\PassOptionsToPackage{dvipsnames}{xcolor}
\documentclass[11pt,letterpaper]{article}
\usepackage[margin=1in]{geometry}
\usepackage[utf8]{inputenc}
\usepackage[T1]{fontenc} %
\usepackage{times}
\usepackage{mathpazo}
\linespread{1.09} %

\usepackage{upgreek}
\renewcommand{\gamma}{\upgamma}
\renewcommand{\pi}{\uppi}

\PassOptionsToPackage{nameinlink}{cleveref}

\usepackage{markdown}
\markdownSetup{
  rendererPrototypes = {
    headingOne    = {\section*{#1}},
    headingTwo    = {\subsection*{#1}},
    headingThree  = {\subsubsection*{#1}},
    headingFour   = {\paragraph{#1}},
    headingFive   = {\subparagraph{#1}},
    headingSix    = {\textbf{#1}\par},
  }
}

\usepackage{cancel}
\usepackage{mathtools}

\usepackage{manfnt} %

\usepackage{figchild}
\usepackage{fontawesome5}

\usepackage{wrapfig}

\usepackage{accents}
 
\usepackage[dvipsnames]{xcolor}
\usepackage{color} 
\definecolor{niceRed}{RGB}{190,38,38}
\definecolor{niceYellow}{HTML}{f5b400}
\definecolor{blueGrotto}{HTML}{059DC0}
\definecolor{royalBlue}{HTML}{057DCD}
\definecolor{navyBlue}{HTML}{0B579C}
\definecolor{yaleBlue}{HTML}{00356b}
\definecolor{limeGreen}{HTML}{81B622}
\definecolor{nicePurple}{HTML}{9c27b0}
\definecolor{lightRoyalBlue}{HTML}{def2ff}  
\definecolor{gold}{HTML}{ffa300}
\definecolor{claudeBlue}{HTML}{4C72B0}
\definecolor{claudeRed}{HTML}{C44E52}

\usepackage{hyperref}
\hypersetup{
  colorlinks = true, 
  urlcolor = {blueGrotto},
  linkcolor = {royalBlue},
  citecolor = {navyBlue}
}
\usepackage[capitalise,noabbrev,sort]{cleveref}

\usepackage[most]{tcolorbox}

\usepackage{xcolor}
\usepackage{xfp} %

\newcommand{\tokenstrut}{\rule[-0.12em]{0pt}{0.52em}}

\definecolor{TokenLow}{HTML}{F7FBFF}   %
\definecolor{TokenHigh}{HTML}{E36A6A}  %
\DeclareUnicodeCharacter{21A9}{\raisebox{0.15ex}{\scriptsize$\hookleftarrow$}}

\newcommand{\tighttracebreak}{\\[-1ex]}

\NewDocumentCommand{\token}{ O{0} +v }{%
  \begingroup%
  \setlength{\fboxsep}{0pt}%
  \ttfamily%
  \scriptsize%
  \colorbox%
    {TokenHigh!\fpeval{min(max(round(#1 * 65, 0), 0), 100)}!TokenLow}%
    {\tokenstrut #2}%
  \endgroup%
}%

\NewDocumentCommand{\cut}{ +v }{%
    \begingroup%
    \setlength{\fboxsep}{0pt}%
    \ttfamily%
    \scriptsize%
    \bfseries
    \colorbox%
    {LimeGreen}%
    {\tokenstrut #1}%
    \endgroup%
}%

\tcbset{
    mybox/.style={
        colframe=black,        %
        colback=gray!0,       %
        boxrule=0.5mm,         %
        width=\linewidth,      %
        arc=3mm,               %
        left=5mm,              %
        right=5mm,             %
        top=5mm,               %
        bottom=5mm,            %
        enhanced,              %
        halign=center        %
    }
}

\usepackage{booktabs} 
\usepackage{multicol} 
\usepackage{multirow} 
\usepackage{makecell} 
\usepackage{longtable}

\usepackage{graphicx}
\usepackage{float} 
\usepackage{tikz}
\usepackage{pgfplots}
\pgfplotsset{compat=1.17}
\usepgfplotslibrary{fillbetween}
\usetikzlibrary{arrows.meta,calc}
\usepackage{setspace}

\tikzset{
  myNodeFlex/.style={
    draw,
    rectangle,
    rounded corners,
    text centered,
    minimum height=1.5em,
  }
}

\tikzset{
  myNode/.style={
    draw,
    rectangle,
    rounded corners,
    text centered,
    minimum height=1.5em,
    minimum width=3cm,
    text width=5cm,    
  }
}

\tikzset{
  myNodeNarrow/.style={
    draw,
    rectangle,
    rounded corners,
    text centered,
    minimum height=1.5em,
    minimum width=1cm,
  }
}

\tikzset{
  myNodeWide/.style={
    draw,
    rectangle,
    rounded corners,
    text centered,
    minimum height=1.5em,
    minimum width=6cm,
  }
}

\usepackage{xinttools} %
\usetikzlibrary{calc}

\def\biglen{20cm} %
\tikzset{
  half plane/.style={ to path={
       ($(\tikztostart)!.5!(\tikztotarget)!#1!(\tikztotarget)!\biglen!90:(\tikztotarget)$)
    -- ($(\tikztostart)!.5!(\tikztotarget)!#1!(\tikztotarget)!\biglen!-90:(\tikztotarget)$)
    -- ([turn]0,2*\biglen) -- ([turn]0,2*\biglen) -- cycle}},
  half plane/.default={1pt}
}

\usepackage{physics} %
\usepackage{amsmath}
\usepackage{amssymb}
\usepackage{amsthm}
\usepackage{bbm}
\usepackage{blkarray}
\usepackage{mathtools}
\usepackage{nicefrac}
\usepackage{dsfont}
\usepackage{xspace}
\usepackage{pifont} %

\usepackage{thm-restate} %

\usepackage{enumerate} 
\usepackage[inline,shortlabels]{enumitem}
\usepackage{typed-checklist}
\usepackage{tcolorbox}
\usepackage[framemethod=TikZ]{mdframed}
\mdfsetup{%
backgroundcolor=white, %
roundcorner=4pt,
linewidth=1pt}
\usepackage{changepage} %

\usepackage[linesnumbered,ruled,vlined,noend,algo2e]{algorithm2e} 
\usepackage{mathrsfs} %
\usepackage{soul} %

\theoremstyle{plain} 
\newtheorem{theorem}{Theorem}[section]

\newtheorem{proposition}[theorem]{Proposition}
\newtheorem{lemma}[theorem]{Lemma}

\newtheorem{condition}{Condition}

\newtheorem{definition}{Definition}
\newtheorem*{definition*}{Definition}

\theoremstyle{definition} 
\newtheorem{example}[theorem]{Example}
\newtheorem{remark}[theorem]{Remark}

\theoremstyle{remark}

\AfterEndEnvironment{definition}{\noindent\ignorespaces}
\AfterEndEnvironment{example}{\noindent\ignorespaces}
\AfterEndEnvironment{assumption}{\noindent\ignorespaces}
\AfterEndEnvironment{condition}{\noindent\ignorespaces}
\AfterEndEnvironment{lemma}{\noindent\ignorespaces}
\AfterEndEnvironment{theorem}{\noindent\ignorespaces}
\AfterEndEnvironment{proposition}{\noindent\ignorespaces}
\AfterEndEnvironment{fact}{\noindent\ignorespaces}
\AfterEndEnvironment{question}{\noindent\ignorespaces}
\AfterEndEnvironment{corollary}{\noindent\ignorespaces}
\AfterEndEnvironment{model}{\noindent\ignorespaces}
\AfterEndEnvironment{remark}{\noindent\ignorespaces}
\AfterEndEnvironment{proof}{\noindent\ignorespaces}
\AfterEndEnvironment{fact}{\noindent\ignorespaces}
\AfterEndEnvironment{minftheorem}{\noindent\ignorespaces}
\AfterEndEnvironment{inftheorem}{\noindent\ignorespaces}
\AfterEndEnvironment{maintheorem}{\noindent\ignorespaces}
\AfterEndEnvironment{restatable}{\noindent\ignorespaces}
\AfterEndEnvironment{infassumption}{\noindent\ignorespaces}
\AfterEndEnvironment{conjecture}{\noindent\ignorespaces}
\AfterEndEnvironment{claim}{\noindent\ignorespaces}

\crefname{section}{Section}{Sections}
\crefname{theorem}{Theorem}{Theorems}
\crefname{lemma}{Lemma}{Lemmas}
\crefname{condition}{Condition}{Conditions}
\crefname{definition}{Definition}{Definitions}
\crefname{conjecture}{Conjecture}{Conjectures}
\crefname{corollary}{Corollary}{Corollaries}
\crefname{construction}{Construction}{Constructions}
\crefname{conjecture}{Conjecture}{Conjectures}
\crefname{claim}{Claim}{Claims}
\crefname{observation}{Observation}{Observations}
\crefname{proposition}{Proposition}{Propositions}
\crefname{fact}{Fact}{Facts}
\crefname{question}{Question}{Questions}
\crefname{problem}{Problem}{Problems}
\crefname{remark}{Remark}{Remarks}
\crefname{model}{Model}{Models}
\crefname{example}{Example}{Examples}
\crefname{equation}{Equation}{Equations}
\crefname{appendix}{Appendix}{Appendices}
\crefname{algorithm}{Algorithm}{Algorithms}
\crefname{algocf}{Algorithm}{Algorithms}
\crefname{model}{Model}{Models}
\crefname{figure}{Figure}{Figures}
\crefname{infdefinition}{Informal Definition}{Informal Definitions}
\crefname{inftheorem}{Informal Theorem}{Informal Theorems}
\crefname{infassumption}{Informal Assumption}{Informal Assumptions}
\crefname{minftheorem}{Main Informal Theorem}{Main Informal Theorems}
\crefname{maintheorem}{Main Theorem}{Main Theorems}
\crefname{assumption}{Assumption}{Assumptions}
\crefname{case}{Case}{Cases}
\crefname{program}{Program}{Programs}
\crefname{inequality}{Inequality}{Inequalities}

\newlist{asmpenum}{enumerate}{1} %
\setlist[asmpenum]{label={\arabic*.},ref=\theassumption.{\arabic*}}
\crefname{asmpenumi}{Assumption}{Assumptions}

\usepackage{etoolbox}

\newcommand{\yesnum}{\addtocounter{equation}{1}\tag{\theequation}}  

\makeatletter
\renewcommand{\eqref}[1]{\textup{\eqrefform@{\ref{#1}}}}
\let\eqrefform@\tagform@

\newcommand{\changetag}[1]{%
  \renewcommand\tagform@[1]{\maketag@@@{(\ignorespaces#1\unskip\@@italiccorr)}}%
}
\makeatother

\makeatletter
\newcommand{\tagnum}[2]{%
    \refstepcounter{equation}%
    \tag{#1) \ (\theequation}%
    \protected@write \@auxout {}{%
        \string \newlabel {#2}{{\theequation}{\thepage}{}{equation.\theequation}{}}%
    }%
}
\makeatother

\newcommand{\quadtext}[1]{\quad\text{#1}\quad}
\newcommand{\qquadtext}[1]{\qquad\text{#1}\qquad}

\newcommand{\qquadand}{\qquadtext{and}}

\newcommand{\quadwhere}{\quadtext{where}}

\def\abs#1{\left| #1 \right|}

\newcommand{\sinparen}[1]{\ensuremath{(#1)}}

\newcommand{\inbrace}[1]{\ensuremath{\left\{#1\right\}}}

\newcommand{\inparen}[1]{\ensuremath{\left(#1\right)}}

\newcommand{\N}{\mathbb{N}}
\newcommand{\R}{\mathbb{R}}

\newcommand{\Z}{\mathbb{Z}}

\newcommand{\tv}[2]{\operatorname{d}_{\mathsf{TV}}{\inparen{#1,#2}}}

\newcommand{\nfrac}[2]{\nicefrac{#1}{#2}}

\newcommand{\eps}{\varepsilon}
\renewcommand{\epsilon}{\varepsilon}
\makeatletter
\newcommand*{\tran}{{\mathpalette\@tran{}}}
\newcommand*{\@tran}[2]{\raisebox{\depth}{$\m@th#1\intercal$}}
\makeatother

\mathchardef\NABLA"272
\newcommand*{\Nabla}{\boldsymbol\NABLA}
\let\nabla\Nabla

\renewcommand{\tilde}{\widetilde}

\DeclareMathAlphabet{\mathdutchcal}{U}{dutchcal}{m}{n}
\SetMathAlphabet{\mathdutchcal}{bold}{U}{dutchcal}{b}{n}
\DeclareMathAlphabet{\mathdutchbcal}{U}{dutchcal}{b}{n}

\DeclareMathAlphabet\urwscr{U}{urwchancal}{b}{n}%
\DeclareMathAlphabet\rsfscr{U}{rsfso}{m}{n}
\DeclareMathAlphabet\euscr{U}{eus}{m}{n}
\DeclareFontEncoding{LS2}{}{}
\DeclareFontSubstitution{LS2}{stix}{m}{n}
\DeclareMathAlphabet\stixcal{LS2}{stixcal}{m} {n}

\renewcommand{\paragraph}[1]{\smallskip \noindent\textbf{#1}~}

\newcommand{\eg}{\emph{e.g.}}
\newcommand{\ie}{\emph{i.e.}}

\newcommand{\eat}[1]{}

\usepackage{caption}
\usepackage{subcaption}
\usepackage{graphicx}

\usepackage{array}
\newcolumntype{L}[1]{>{\raggedright\let\newline\\\arraybackslash\hspace{0pt}}m{#1}}
\newcolumntype{C}[1]{>{\centering\let\newline\\\arraybackslash\hspace{0pt}}m{#1}}
\newcolumntype{R}[1]{>{\raggedleft\let\newline\\\arraybackslash\hspace{0pt}}m{#1}}

\usepackage{fancyhdr}

\fancypagestyle{custom}{
  \fancyhf{}                %
  \fancyfoot[C]{\textcolor{black}{\fcFishD{0.025}{black}{0.5}}} %

}

\makeatletter
\newcommand{\appendixtitle}[1]{%
  {\par
  \vspace{0.2in}
  \@toptitlebar
  \centering
  {\LARGE\bf #1\par}
  \@bottomtitlebar
  \vspace{0.2in}}
}
\makeatother

\newcommand{\highlight}[1]{\emph{\textcolor{royalBlue}{#1}}}
\newcommand{\takeaway}[1]{\textbf{\textcolor{royalBlue}{#1}}}

\newif\iffastcompile %
\newif\ifshowtraces  %
\newif\ifneuripsversion %
\newcommand{\neuripsvspace}[1]{\ifneuripsversion\vspace{#1}\fi}
\newcommand{\codexcontributions}[1]{%
  \ifneuripsversion
    {\setlist*[itemize]{leftmargin=1.5em}#1}%
  \else
    #1%
  \fi}

\newcommand{\codexneuripssamepage}[1]{%
  \ifneuripsversion
    {\begin{samepage}#1\end{samepage}}%
  \else
    #1%
  \fi}

\newtcolorbox{bluecompacttwopartbox}[4][]{%
  enhanced,
  breakable,
  colback=white,
  colframe=black!75,
  boxrule=0.6pt,
  arc=2pt,
  left=2pt, right=2pt, top=2pt, bottom=2pt,
  middle=2pt,
  title={#2},
  coltitle=black!80,
  colbacktitle=claudeBlue!10,
  fonttitle=\small\bfseries,
  toptitle=3pt, bottomtitle=3pt,
  lefttitle=10pt, righttitle=10pt,
  segmentation style={solid, black!75, line width=0.6pt},
  overlay={%
    \node[anchor=base east, inner xsep=6pt, inner ysep=0pt,
          font=\scriptsize\itshape, text=black!55]
      at ([yshift=3pt]segmentation.north east) {#3};
    \node[anchor=base east, inner xsep=6pt, inner ysep=0pt,
          font=\scriptsize\itshape, text=black!55]
      at ([yshift=3pt]interior.south east) {#4};
  },
  #1,
}
\newtcolorbox{redcompacttwopartbox}[4][]{%
  enhanced,
  breakable,
  colback=white,
  colframe=black!75,
  boxrule=0.6pt,
  arc=2pt,
  left=2pt, right=2pt, top=2pt, bottom=2pt,
  middle=2pt,
  title={#2},
  coltitle=black!80,
  colbacktitle=claudeRed!30,
  fonttitle=\small\bfseries,
  toptitle=3pt, bottomtitle=3pt,
  lefttitle=10pt, righttitle=10pt,
  segmentation style={solid, black!75, line width=0.6pt},
  overlay={%
    \node[anchor=base east, inner xsep=6pt, inner ysep=0pt,
          font=\scriptsize\itshape, text=black!55]
      at ([yshift=3pt]segmentation.north east) {#3};
    \node[anchor=base east, inner xsep=6pt, inner ysep=0pt,
          font=\scriptsize\itshape, text=black!55]
      at ([yshift=3pt]interior.south east) {#4};
  },
  #1,
}
\tcbset{
  tighttwopart/.style={
    unbreakable,
    boxsep=1pt,
    left=2pt,right=2pt,top=1pt,bottom=1pt,
    before skip=0pt,after skip=0pt,
    middle=1pt
  }
}

\usepackage{listings}
\lstset{breaklines=true,breakindent=0em,postbreak=\mbox{\textcolor{red}{$\hookrightarrow$}\space},basicstyle=\ttfamily}

\lstdefinestyle{tighttrace}{
  basicstyle=\scriptsize\ttfamily,
  escapechar=|,
  keepspaces=true,
  columns=fullflexible,
  aboveskip=0pt,
  belowskip=0pt,
  xleftmargin=0pt,
  xrightmargin=0pt,
  resetmargins=true
}

\renewcommand{\vec}[1]{#1}

\newtcolorbox{twopartbox}[4][]{%
  enhanced,
  breakable,
  colback=white,
  colframe=black!75,
  boxrule=0.6pt,
  arc=2pt,
  left=2pt, right=2pt, top=2pt, bottom=2pt,
  middle=2pt,
  title={#2},
  coltitle=black!80,
  colbacktitle=claudeBlue!10,
  fonttitle=\small\bfseries,
  toptitle=3pt, bottomtitle=3pt,
  lefttitle=10pt, righttitle=10pt,
  segmentation style={solid, black!75, line width=0.6pt},
  overlay={%
    \node[
      anchor=base east,
      inner xsep=6pt, inner ysep=2pt,
      font=\small\bfseries,
      text=black,
      fill=gold,
      rounded corners=2pt
    ]
    at ([yshift=6pt,xshift=-6pt]segmentation.north east) {#3};

    \node[
      anchor=base east,
      inner xsep=6pt, inner ysep=2pt,
      font=\small\bfseries,
      text=black,
      fill=gold,
      rounded corners=2pt
    ]
    at ([yshift=6pt,xshift=-6pt]interior.south east) {#4};
  },
    #1,
}

\lstdefinestyle{introtrace}{
  style=tighttrace,
  basicstyle=\ttfamily\scriptsize,
  aboveskip=0pt,
  belowskip=0pt,
  breaklines=false,
  breakatwhitespace=false,
  columns=fullflexible,
  escapechar=|
}

\lstdefinestyle{introtracesidebar}{
  style=tighttrace,
  basicstyle=\ttfamily\fontsize{6.6}{6.9}\selectfont,
  aboveskip=0pt,
  belowskip=0pt,
  breaklines=false,
  breakatwhitespace=false,
  columns=fullflexible,
  escapechar=|
}

\tcbset{
  introtracebox/.style={
    tighttwopart,
    before skip=0pt,
    after skip=0pt,
    boxsep=0.55mm,
    left=0.8mm,
    right=0.8mm,
    top=0.55mm,
    bottom=0.55mm,
    middle=0.35mm,
    fonttitle=\bfseries\small,
    fontupper=\scriptsize,
    fontlower=\scriptsize
  },
  introtraceboxtoplabels/.style={
    introtracebox,
    top=2.35mm,
    before lower={\vspace{1.65mm}}
  },
  introtraceboxsidebartoplabels/.style={
    tighttwopart,
    before skip=0pt,
    after skip=0pt,
    boxsep=0.38mm,
    left=0.55mm,
    right=0.55mm,
    top=2.0mm,
    bottom=0.4mm,
    middle=0.25mm,
    before lower={\vspace{1.2mm}},
    fonttitle=\bfseries\footnotesize,
    fontupper=\scriptsize,
    fontlower=\scriptsize
  },
  introtracelabel/.style={
    anchor=north east,
    font=\itshape\scriptsize,
    text=black!58,
    fill=white,
    fill opacity=0.86,
    text opacity=1,
    inner xsep=1pt,
    inner ysep=0pt
  },
  introtracesidebarlabel/.style={
    anchor=north east,
    font=\itshape\tiny,
    text=black!58,
    fill=white,
    fill opacity=0.86,
    text opacity=1,
    inner xsep=0.8pt,
    inner ysep=0pt
  }
}

\tikzset{
  introtracelabel/.style={
    anchor=north east,
    font=\itshape\scriptsize,
    text=black!58,
    fill=white,
    fill opacity=0.86,
    text opacity=1,
    inner xsep=1pt,
    inner ysep=0pt
  },
  introtracesidebarlabel/.style={
    anchor=north east,
    font=\itshape\tiny,
    text=black!58,
    fill=white,
    fill opacity=0.86,
    text opacity=1,
    inner xsep=0.8pt,
    inner ysep=0pt
  }
}

\NewDocumentEnvironment{blueintrotraceboxtoplabels}{O{} m m m}{%
  \begin{bluecompacttwopartbox}[introtraceboxtoplabels,overlay={%
      \node[introtracelabel]
        at ([xshift=-0.9mm,yshift=-0.45mm]interior.north east) {#3};
      \node[introtracelabel]
        at ([xshift=-0.9mm,yshift=-0.45mm]segmentation.east) {#4};
    },#1
  ]{#2}{}{}%
}{%
  \end{bluecompacttwopartbox}%
}

\NewDocumentEnvironment{redintrotraceboxtoplabels}{O{} m m m}{%
  \begin{redcompacttwopartbox}[introtraceboxtoplabels,overlay={%
      \node[introtracelabel]
        at ([xshift=-0.9mm,yshift=-0.45mm]interior.north east) {#3};
      \node[introtracelabel]
        at ([xshift=-0.9mm,yshift=-0.45mm]segmentation.east) {#4};
    },#1
  ]{#2}{}{}%
}{%
  \end{redcompacttwopartbox}%
}

\NewDocumentEnvironment{blueintrotraceboxsidebar}{O{} m m m}{%
  \begin{bluecompacttwopartbox}[introtraceboxsidebartoplabels,overlay={%
      \node[introtracesidebarlabel]
        at ([xshift=-0.7mm,yshift=-0.35mm]interior.north east) {#3};
      \node[introtracesidebarlabel]
        at ([xshift=-0.7mm,yshift=-0.35mm]segmentation.east) {#4};
    },#1
  ]{#2}{}{}%
}{%
  \end{bluecompacttwopartbox}%
}

\NewDocumentEnvironment{redintrotraceboxsidebar}{O{} m m m}{%
  \begin{redcompacttwopartbox}[introtraceboxsidebartoplabels,overlay={%
      \node[introtracesidebarlabel]
        at ([xshift=-0.7mm,yshift=-0.35mm]interior.north east) {#3};
      \node[introtracesidebarlabel]
        at ([xshift=-0.7mm,yshift=-0.35mm]segmentation.east) {#4};
    },#1
  ]{#2}{}{}%
}{%
  \end{redcompacttwopartbox}%
}

\neuripsversionfalse

\usepackage[utf8]{inputenc} %
\usepackage[T1]{fontenc}    %
\usepackage{hyperref}       %
\usepackage{url}            %
\usepackage{booktabs}       %
\usepackage{amsfonts}       %
\usepackage{nicefrac}       %
\usepackage{microtype}      %
\usepackage{xcolor}         %

\usepackage[
  backref=true,
  backend=biber,
  natbib=true,
  style=alphabetic,
  sorting=alphabeticlabel,
  sortcites=true,
  minbibnames=3,
  maxbibnames=100,
  mincitenames=4,
  maxcitenames=4,
  minalphanames=4,
  maxalphanames=4,
  doi=false,
  isbn=false,
]{biblatex}

\DeclareSortingTemplate{alphabeticlabel}{
  \sort[final]{%
    \field{labelalpha}
  }
  \sort{%
    \field{year}
  }
  \sort{%
    \field{title}
  }
}

\AtEveryBibitem{%
  \clearname{editor}%
}
\AtEveryBibitem{%
  \clearlist{location}%
}

\setlength{\bibitemsep}{0pt}
\setlength{\bibnamesep}{0pt}
\setlength{\bibinitsep}{0pt}

\AtBeginRefsection{\GenRefcontextData{sorting=ynt}}
\AtEveryCite{\localrefcontext[sorting=ynt]}
\DeclareSourcemap{
  \maps[datatype=bibtex]{
    \map[overwrite]{
      \step[fieldsource=entrykey, match=\regexp{\Acomanici2025gemini\Z}, final]
      \step[fieldset=author, fieldvalue={
      Gheorghe Comanici and Eric Bieber and Mike Schaekermann and Ice Pasupat and Noveen Sachdeva and Inderjit Dhillon and Marcel Blistein and Ori Ram and Dan Zhang and Evan Rosen and
      Luke Marris and Sam Petulla and Colin Gaffney and Asaf Aharoni and Nathan Lintz and Tiago Cardal Pais and Henrik Jacobsson and Idan Szpektor and Nan-Jiang Jiang and Krishna Haridasan and
      Ahmed Omran and Nikunj Saunshi and Dara Bahri and Gaurav Mishra and Eric Chu and Toby Boyd and Brad Hekman and Aaron Parisi and Chaoyi Zhang and Kornraphop Kawintiranon and
      Tania Bedrax-Weiss and Oliver Wang and Ya Xu and Ollie Purkiss and Uri Mendlovic and Ilaï Deutel and Nam Nguyen and Adam Langley and Flip Korn and Lucia Rossazza and
      Alexandre Ramé and Sagar Waghmare and Helen Miller and Nathan Byrd and Ashrith Sheshan and Raia Hadsell and Sangnie Bhardwaj and Pawel Janus and Tero Rissa and Dan Horgan and
      Alvin Abdagic and Lior Belenki and James Allingham and Anima Singh and Theo Guidroz and Srivatsan Srinivasan and Herman Schmit and Kristen Chiafullo and Andre Elisseeff and Nilpa Jha and
      Prateek Kolhar and Leonard Berrada and Frank Ding and Xiance Si and Shrestha Basu Mallick and Franz Och and Sofia Erell and Eric Ni and Tejasi Latkar and Sherry Yang and
      Petar Sirkovic and Ziqiang Feng and Robert Leland and Rachel Hornung and Gang Wu and Charles Blundell and Hamidreza Alvari and Po-Sen Huang and Cathy Yip and Sanja Deur and
      Li Liu and Gabriela Surita and Pablo Duque and Dima Damen and Johnson Jia and Arthur Guez and Markus Mircea and Animesh Sinha and Alberto Magni and Paweł Stradomski and
      Tal Marian and Vlado Galić and Wenhu Chen and Hisham Husain and Achintya Singhal and Dominik Grewe and François-Xavier Aubet and Shuang Song and Lorenzo Blanco and Leland Rechis and others
      }]
    }
  }
}

\addbibresource{refs.bib}

\showtracesfalse %

\title{Reasoning with Sampling: Cutting at Decision Points}

\author{
    \begin{tabular}{C{5cm}C{5cm}C{5cm}}
        {\bf Felix Zhou}
            & {\bf Anay Mehrotra}
            & {\bf Quanquan C.\ Liu}\\
        {Yale University}
            & {Stanford University}
            & {Yale University}
    \end{tabular}
}
\date{\bf
    \href{https://anaymehrotra.com/reasoning-with-sampling/}{{\normalcolor\faGlobeAmericas~Website}}\quad
    \href{https://github.com/felix-cfzhou/Reasoning-with-Sampling}{{\normalcolor\faGithub~Code}}}

\begin{document}

\maketitle
\begin{abstract}
    Frontier reasoning models are produced by posttraining base language models with reinforcement learning. 
    Recent work has challenged this by showing that sampling from a sharpened version of the base model's distribution, a so-called \emph{power distribution}, elicits comparable reasoning without additional training, curated datasets, or verifiers.
    However, making this method practical requires \textit{efficiently sampling} from the power distribution.
    A sampler needs to ``mix'' to the power distribution, which necessitates moving between modes of the target distribution; intuitively, \eg{}, trying different reasoning strategies.
    The samplers proposed in prior works repeatedly select a ``cut'' position in the current reasoning trace uniformly at random and resample the suffix from that position onward.
    However, reasoning traces typically contain a few consequential decisions (\eg{}, the choice of proof strategy or algorithm), and we observe that a uniformly chosen cut tends to rewrite local details rather than revisit decision points.
    We introduce an algorithm (Entropy-Cut Metropolis--Hastings) that uses the base model's next-token entropy as a proxy to identify key decision points and resample from those positions.
    We empirically verify that entropy jumps are a useful proxy for decision points and, in a stylized model of reasoning, prove that our method's mixing time scales with the number of decisions in a trace rather than with the number of tokens, which can be much larger.
    Across MATH500, HumanEval, GPQA Diamond, and AIME26, our method consistently improves over baselines and RL-trained models,
    including best-of-$N$ at the same token budget.

\end{abstract}

\pagenumbering{arabic}

\section{Introduction}
\label{sec:introduction}

Frontier reasoning models are typically obtained by posttraining base language models with reinforcement learning.
This paradigm has produced striking gains on mathematics, coding, and scientific reasoning
\citep{deepseekai2025deepseekr1,comanici2025gemini,jaech2024openai}.
At the same time, it raises a basic question: does posttraining teach models fundamentally new reasoning procedures, or does it mainly elicit abilities already encoded in the base model?
One explanation for these gains is the \emph{sharpening} view.
Under this view, base models already assign non-negligible probability to many high-quality reasoning traces, and posttraining improves performance by concentrating more mass on these traces.
Indeed, one mechanism by which posttraining improves reasoning performance is
to convert strong top-$k$ performance into stronger top-$1$ performance
\citep{nath2025adaptive, yue2025does}.
A growing body of evidence supports this perspective: RL-trained models often improve pass@1 without expanding the set of reasoning paths available in the base distribution \citep{he2025understanding, yue2025does, shao2025spurious, song2025what}.

\vspace{0pt}

\begin{figure}[thb!]
\centering
\begin{minipage}[t]{0.5\linewidth}
\vspace{0pt}
\begin{tcbraster}[
  raster columns=1,
  raster row skip=1mm,
  raster before skip=0pt,
  raster after skip=0pt
]
\begin{blueintrotraceboxsidebar}{Uniform-Cut MH}{Current Completion}{Proposal Suffix}
\begin{lstlisting}[style=introtracesidebar]
we have $x = 0$ and $y = 3$.
### Step 1: Calculate $r$
$$ r=\sqrt{x^3+y^3}=\sqrt{0^3+3^3|\cut{[CUT]}|} $$
\end{lstlisting}
\tcblower
\begin{lstlisting}[style=introtracesidebar]
}=\sqrt{27}$$
\end{lstlisting}
\end{blueintrotraceboxsidebar}
\begin{redintrotraceboxsidebar}{Entropy-Cut MH (Ours)}{Current Completion}{Proposal Suffix}
\begin{lstlisting}[style=introtracesidebar]
we have $x = 0$ and $y = 3$.
|\cut{[CUT]}|### Step 1: Calculate $r$
$$ r=\sqrt{x^3+y^3}=\sqrt{0^3+3^3}=\sqrt{27} $$
\end{lstlisting}
\tcblower
\begin{lstlisting}[style=introtracesidebar]
### Step 1: Calculate $r$
$$ r=\sqrt{x^2+y^2}=\sqrt{0^2+3^2}=3 $$
\end{lstlisting}
\end{redintrotraceboxsidebar}
\end{tcbraster}
\end{minipage}\hfill
\begin{minipage}[t]{0.5\linewidth}
\vspace{0pt}
\includegraphics[width=\linewidth]{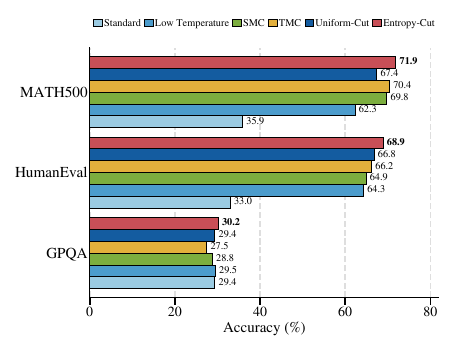}
\end{minipage}

\vspace{0pt}
\caption{
\textbf{Entropy-Cut MH revises reasoning traces at decision points and improves accuracy.}
Left: uniform cuts can splice proposals inside a local calculation, producing suffixes that only rewrite nearby tokens.
Entropy-Cut instead cuts near high-uncertainty reasoning steps, allowing the proposal to reconsider the underlying continuation.
Right: on Qwen2.5-7B, this targeted proposal improves accuracy over standard sampling, low-temperature sampling, SMC, TMC, and uniform-cut MH across MATH500, HumanEval, and GPQA Diamond.
}
\vspace{0pt}
\label{fig:intro_entropy_cut_side_by_side}
\end{figure}

Building on this, \citet{karan2026reasoning} showed that strong reasoning behavior can be elicited from a base model without additional training by sampling from a sharpened version of the base distribution.
Concretely, for a base distribution $p$ over complete traces, they target the power distribution
\[
    \Pi_T(\vec{x}) \propto p(\vec{x})^\alpha\,, 
\]
which upweights traces already assigned high likelihood by the base model.
This approach is attractive because it is training-free, dataset-free, and verifier-free, making it a form of test-time scaling that uses only the base model's own probabilities.
Sampling from $\Pi_T$ is, however, computationally non-trivial (see \cref{sec:preliminaries}).
Karan and Du address this with a stagewise Metropolis--Hastings sampler that iteratively revises a candidate trace: it picks a cut position, keeps the prefix before the cut, resamples the suffix from a proposal model, and accepts or rejects the new trace via the Metropolis--Hastings correction.
This sampler succeeds in eliciting strong reasoning, but its efficiency is the bottleneck for the entire approach, raising the key question: %
\begin{center}
    \vspace{0pt}
    \textit{How can we sample from the power distribution efficiently?}
    \vspace{0pt}
\end{center}

\paragraph{Prior Work.}
Several recent works have sought to address this question by departing from the stagewise Metropolis--Hastings framework.
\citet{azizi2026power} replace it with a Sequential Monte Carlo method that maintains particles and corrects their weights token by token, while \citet{ji2026scalable} sample autoregressively by approximating the next-token conditionals of $\Pi_T$.
Hence, both of these works depart from the approach of prior work and use new and more sophisticated sampling algorithms.
We instead retain the Metropolis--Hastings sampler of \citet{karan2026reasoning} and modify it to better suit the structure of reasoning traces. %

\paragraph{Our Approach.}
For an MCMC method to mix efficiently, it must be able to transition between the modes of its target distribution. For $\Pi_T$, which is a distribution over sequences of tokens, these modes correspond, heuristically, to qualitatively different lines of reasoning: a proof by induction and a proof by contradiction may both have high probability under the sharpened distribution, but moving between them requires revisiting the early decision that chose the proof strategy (\cref{fig:conductance-reasoning}).

\paragraph{Our Contributions.}
Our contributions are as follows.
\codexcontributions{%
\begin{itemize} 
    \item We introduce \emph{Entropy-Cut Metropolis--Hastings} (\cref{sec:entropy-cut}), a training-free sampler for $\Pi_T$ that modifies the stagewise sampler of \citet{karan2026reasoning} by placing cuts at positions of high next-token entropy, a proxy for decision points (see \cref{fig:intro_entropy_cut_side_by_side}), rather than uniformly. The Metropolis--Hastings correction ensures the target distribution is unchanged.
    \item We formalize the benefit of cutting at decision points (\cref{sec:theoretical-results}) in a \emph{reasoning-tree} model, where the possible reasoning traces are encoded as a tree whose root-to-leaf paths are token sequences and whose branch nodes correspond to key decision points. We show that entropy-cut mixing scales with the number of decision points $k$, while uniform-cut mixing can scale with the token depth $T$, which can be much larger than $k$ (\cref{thm:approx-symmetric-mixing-separation}).
    \item We empirically demonstrate the effectiveness of our algorithm over a range of models (\texttt{\small Qwen2.5-7B}, \texttt{\small Qwen2.5-Math-7B}, \texttt{\small Qwen3-8B-Base}, \texttt{\small Phi-3.5-mini-instruct}, \texttt{\small Phi-4-mini-instruct}) and reasoning tasks (MATH500, HumanEval, GPQA Diamond, AIME26).
    As summarized in \cref{fig:intro_entropy_cut_side_by_side}, our results show that more suitable power sampling techniques can \emph{better} extract the latent reasoning capabilities of language models.
    Moreover, over multiple samples, we avoid a collapse in diversity despite the increased single-shot performance.
    Finally, since our method uses more compute than a single sample, we compare with standard and low-temperature best-of-$N$ sampling using each entropy-cut MH run's sampled-token count as a stopping target. Best-of-$N$ retains the sample that crosses this target and may therefore use more tokens; entropy-cut MH scores higher in most settings (\cref{apx:experiments:best-of-n}).
\end{itemize}
}%

    \paragraph{Organization.}
    The next sections review the power-distribution sampling framework, introduce Entropy-Cut MH, and analyze it through a reasoning-tree model.
    We discuss the closest related methods in context and defer broader related work to \cref{sec:appendix:related-work}.

\section{Preliminaries on Sharpening and Metropolis--Hastings Algorithm}
\label{sec:preliminaries}
In this section, we set up notation, introduce the \emph{power distribution}, and describe the stagewise Metropolis--Hastings sampler of \citet{karan2026reasoning} that serves as our starting point.

\paragraph{Notation.}
For nonnegative functions $f,g :\N\to \R_{\geq 0}$, we write $f \lesssim g$ if there exists an absolute constant $C < \infty$ such that $f(x) \leq C g(x)$ for all $x \in \N$. 
Similarly, we write $f \gtrsim g$ if $g\lesssim f$ and $f \simeq g$ if $f \lesssim g$ and $f \gtrsim g$.
We work with finite sequences of tokens and denote a sequence of length $T+1$ by $x_{0:T} \coloneqq (x_0, x_1, \dots, x_T)$. For any $0 \leq t \leq T$, we write $x_{<t} \coloneqq (x_0, \dots, x_{t-1})$ and $x_{>t} \coloneqq (x_{t+1}, \dots, x_T)$ for the prefix and suffix around position~$t$. 
Let $x_P$ be a prompt and let $p$ be an autoregressive language model which, given $x_P$, generates a continuation $x_{0:T}$ with probability
$p(x_{0:T}\mid x_P)
    =
    \prod_{t=0}^{T} p(x_t\mid x_P,x_{<t})$.
Throughout the paper, we condition on the prompt $x_P$ and suppress it from the notation, writing
\[
    \textstyle 
    p(\vec{x})
    \coloneqq
    \prod\nolimits_{t=0}^{\ell} p(x_t\mid x_{<t})
\]
for any prefix $\vec{x}=x_{0:\ell}$.
Sampling from $p$ is straightforward: drawing $x_t \sim p(\cdot \mid x_{<t})$ sequentially for $t = 0, 1, \dots, \ell$ produces a sample from $p(\vec{x})$ in $\ell+1$ forward passes. 
As we discuss, \highlight{this convenience does \emph{not} extend to our target distribution, and sampling efficiently is a central challenge of our paper}.

\paragraph{Power Distribution and Sharpening.}
A growing body of evidence suggests that RL-posttraining does not teach base models fundamentally new reasoning capabilities, but rather concentrates their sampling on high-likelihood reasoning traces that are already present in the base distribution \citep{he2025understanding, yue2025does, shao2025spurious, song2025what}. 
In other words, among candidate continuations of a fixed prompt, those assigned higher probability by the base model tend to be more likely to be correct. %
This motivates \emph{sharpening}: reweighting sequences by a monotone function of $p(\vec{x})$ that enlarges the ratio between likely and unlikely sequences, so that the resulting distribution concentrates on the sequences the base model is already most confident about.
Building on this insight, \citet{karan2026reasoning} proposed sampling from the \emph{power distribution} as a concrete, training-free mechanism for sharpening. Fix a parameter $\alpha>1$, called the \emph{sharpening power}. For each sequence length $\ell$, the power distribution $\Pi_\ell$ is obtained by raising $p(\vec{x})$ to the $\alpha$-th power and renormalizing:
\[
    \Pi_\ell(\vec{x})
    \coloneqq
    \frac{p(\vec{x})^\alpha}{Z_{\ell,\alpha}}
    \quadwhere 
    Z_{\ell,\alpha}
    \coloneqq
    \sum\nolimits_{y_{0:\ell}} p(y_{0:\ell})^\alpha\,.
    \tag{Target: $\alpha$-Power Distribution}
\]
The parameter $\alpha$ controls the strength of sharpening: the probability ratio between two sequences is amplified from $p(\vec{x})/p(\vec{y})$ under $p$ to $\inparen{p(\vec{x})/p(\vec{y})}^{\alpha}$ under $\Pi_\ell$. As $\alpha\to 1$, the power distribution reduces to $p$, and as $\alpha\to\infty$, it concentrates on the most likely completion. The distribution that we ultimately wish to sample from is $\Pi_T$, where $T$ is the total generation length.
Intuitively, sampling from $\Pi_T$ favors tokens whose continuations concentrate on a few high-likelihood futures, rather than tokens that spread probability over many mediocre ones. This implicit bias toward ``planning ahead'' is particularly valuable for reasoning, where a few early token choices can determine whether the rest of the trace succeeds or fails \citep{li2025critical, abdin2024phi}.
\citet{karan2026reasoning} show in their simulations that being able to sample from $\Pi_T$ can elicit comparable reasoning capabilities without additional training.

Next, we explain why sampling from $\Pi_T$ is hard and how prior work overcomes this.

\paragraph{Sampling from the Power Distribution Is Hard.}
At first sight, it might seem that sampling from $\Pi_T$ reduces to low-temperature sampling at temperature $1/\alpha$, where each next-token conditional $p(\cdot\mid x_{<t})$ is independently raised to the $\alpha$-th power and renormalized. This is not the case: low-temperature sampling sharpens and normalizes each next-token distribution locally, whereas the power distribution $\Pi_T$ sharpens complete traces and normalizes globally through $Z_{T,\alpha}$. In fact, the two distributions can be made arbitrarily far apart in total variation, even for sequences of just two tokens (see \cref{sec:appendix:low-temp-fails}).
More fundamentally, computing $Z_{T,\alpha}$ requires summing over an exponentially large space of token sequences, making exact sampling from $\Pi_T$ intractable.

\paragraph{Metropolis--Hastings for Sampling from the Power Distribution.}
Crucially, although the normalizing constant $Z_{T,\alpha}$ is intractable, the \textit{unnormalized} density $p(\vec{x})^\alpha$ is easy to evaluate: a single autoregressive forward pass yields $p(\vec{x})$. 
For sampling from targets specified only through an unnormalized density, a central tool is the Metropolis--Hastings (MH) algorithm \citep{metropolis1953equation,hastings1970monte}. 
Given a proposal distribution $q(\vec{x}|\vec{x}')$, a single MH step at state $\vec{x}$ draws a candidate $\vec{x}' \sim q(\vec{x},\cdot)$ and accepts it with probability
$\min\inbrace{1,~ \nfrac{\pi(\vec{x}')\,q(\vec{x}|\vec{x}')}{(\pi(\vec{x})\,q(\vec{x'}|\vec{x}))}},$
otherwise the chain stays at $\vec{x}$. 
Iterating this rule produces a Markov chain whose stationary distribution is $\pi$ (under minimal assumptions of aperiodicity and irreducibility, which hold in our setting); see \cref{sec:appendix:mh-assumptions-hold}.

The key property of MH is that the acceptance rule depends on $\pi$ only through ratios of densities, so the intractable normalizer $Z_{T,\alpha}$ cancels in $\Pi_T(\vec{x}')/\Pi_T(\vec{x})$, leaving a ratio involving only $p(\vec{x}')^\alpha$ and $p(\vec{x})^\alpha$ that is computable from autoregressive forward passes (see \cref{rem:mh-tractability}).
However, \highlight{the efficiency of MH depends critically on the proposal distribution $q$: a poor choice can cause the chain to mix exponentially slowly, \ie{}, require an exponentially large number of iterations before its distribution converges.} This problem is especially acute for $\Pi_T$, whose state space $\smash{\abs{V}^T}$ grows exponentially in the generation length \citep{gheissari2017mixing,bandeira2022franz}.

To overcome this, \citet{karan2026reasoning} propose a stagewise MH sampler that (partially) addresses this challenge through two design choices (see \cref{alg:karan-du}). First, instead of targeting $\Pi_T$ directly, they grow the generation in blocks of size $B$: samples from $\Pi_{kB}$ are used to initialize an MH chain targeting $\Pi_{(k+1)B}$, avoiding pathological cold starts. Second, within each stage, the proposal distribution is defined by a \emph{cut distribution} $\lambda(\cdot;\vec{x})$, a distribution over positions $\inbrace{1,\ldots,\ell}$: the chain draws a cut index $m \sim \lambda(\cdot;\vec{x})$, retains the prefix $x_{0:m-1}$, and redraws the suffix $x_{m:\ell}$ from a proposal model $p_{\mathrm{prop}}$. Concretely, the proposal distribution assigns zero mass to any candidate $\vec{x}'$ whose prefix disagrees with $\vec{x}$ before position $m$, and autoregressive probability $\prod_{t=m,\dots,\ell} p_{\mathrm{prop}}(x'_t \mid x'_{<t})$ to any candidate that agrees on the prefix (see \cref{rem:cut-kernel} for the explicit form). The resulting acceptance probability, after the standard MH correction for the state-dependent cut distribution, is
\[
    A_m(\vec{x},\vec{x}')
    =
    \min\inbrace{
        1,\,
        \frac{p(\vec{x}')^\alpha}{p(\vec{x})^\alpha}
\cdot
\frac{\lambda\inparen{m;\vec{x}'}}{\lambda\inparen{m;\vec{x}}}
\cdot
\frac{p_{\mathrm{prop}}\inparen{x_{m:\ell}\mid x_{0:m-1}}}
     {p_{\mathrm{prop}}\inparen{x'_{m:\ell}\mid x_{0:m-1}}}
    }\,.
    \yesnum\label{eq:acceptance-probability}
\]
\citet{karan2026reasoning} select a \emph{uniform cut distribution}, $\lambda_{\mathrm{unif}}(m;\vec{x}) \equiv 1/\ell$, paired with a proposal $p_{\mathrm{prop}}$ given by the low-temperature version of $p$, which treats every token position as an equally useful place to revise the trace. 
In \cref{alg:meta}, we present pseudocode for a meta-algorithm that \textit{generalizes \citet{karan2026reasoning}'s method} and allows for a general cut distribution.
As we argue in the next section, \highlight{the uniform cut distribution from \cite{karan2026reasoning} is suboptimal and a key contribution of our work is a carefully designed cut distribution} that improves empirical performance by reducing mixing time.

\newcommand{\tcpCustom}[1]{\tcp{\textnormal{\emph{\textcolor{gray}{#1}}}}}
\newcommand{\inlinetcpCustom}[1]{//~{\textnormal{\emph{\textcolor{gray}{#1}}}}}
\begin{figure}[t]
\begin{minipage}{\linewidth}
\small 
\begin{algorithm2e}[H]
\DontPrintSemicolon
\LinesNumbered
\caption{Meta-Algorithm: Stagewise Metropolis--Hastings}
\label{alg:meta}

\SetKwInOut{Input}{Input}
\SetKwInOut{Hyper}{Hyperparameters}
\SetKwInOut{Output}{Output}

\Input{\mbox{cut-point sampler $\lambda$; proposal models $p_{\mathrm{prop}}$; power $\alpha > 1$; length $T$}}
\Hyper{block size $B$; number of MCMC steps per stage $N_{\mathrm{MCMC}}$}
\Output{continuation $x_{0:T}$ approximately distributed according to $\Pi_T$}

\BlankLine
Set $K \gets \lceil T/B \rceil$, $T_0\gets 0$, and $x^{(0)}$ to the empty continuation\;

\vspace{2mm}

\For{each stage $k \gets 1,2,\dots,K$}{
    Set the $k$-th stage's target length $T_k \gets \min\!\inbrace{kB,\,T}$ \;

    Sample a suffix $s \sim p_{\mathrm{prop}}(\cdot \mid x^{(k-1)})$ of length $T_k - T_{k-1}$\;
    Set the current state $x \gets \sinparen{x^{(k-1)}, s}$\;

    \vspace{4mm}
    
    \tcpCustom{Metropolis--Hastings chain targeting the power distribution for length-$T_k$}
    \For{each MCMC step $n \gets 1,2,\dots,N_{\mathrm{MCMC}}$}{
        Sample a cut point $m \sim \lambda(\cdot\,;x)$ (from which the current continuation will be revised) \;

        Sample a (new) suffix $s \sim p_{\mathrm{prop}}(\cdot \mid x_{0:m-1})$ of length $T_k - m$\;
        Set the proposed state as $x^{\mathrm{prop}} \gets \inparen{x_{0:m-1}, s}$\;

        Compute the acceptance probability $A_m\!\inparen{x, x^{\mathrm{prop}}}$ from \cref{eq:acceptance-probability}
 
        \mbox{With probability $A_m\!\inparen{x, x^{\mathrm{prop}}}$, \textit{accept} the proposal and set $x\gets x^{\mathrm{prop}}$\phantom{..............................}}
        Otherwise, \textit{reject} the proposal and leave $x$ unchanged\; 
    }
    \vspace{2mm}
 
    Set the $k$-th stage's output as $x^{(k)} \gets x$\;
}

\Return{$x^{(K)}$}\;
\end{algorithm2e}
\vspace{0pt}
\end{minipage}
\vspace{0pt}
\end{figure}
 
\section{Entropy-Cut Metropolis--Hastings}
\label{sec:entropy-cut}

In this section, we introduce our algorithm for sampling from the power distribution $\Pi_T$. 

Our algorithm is an instantiation of the meta-algorithm from \cref{alg:meta}.
It is identical to the algorithm of \citet{karan2026reasoning} except in the choice of the
\emph{cut distribution} $\lambda$---the rule that chooses where each MH proposal is cut. 
Rather than cutting uniformly at random, our cut distribution concentrates cuts at positions where the model faces a genuine choice.
The motivation for this change is that the uniform-cut sampler can mix slowly. 

A standard way to think about slow mixing is through \emph{conductance}. 
Roughly speaking, conductance measures how easily a Markov chain can leave one region of its state space and move to another: if the chain tends to get stuck in a region for many steps, then the conductance is low, and mixing is slow.

In our setting, these regions correspond to different high-level lines of reasoning. 
For example, a proof by induction and a proof by contradiction may lie in different regions, and moving from one to the other requires the sampler to revisit the early decision that chose the proof strategy. 
More generally, reasoning traces contain only a few such high-level choices---which proof technique to use, how to split into cases, which algorithm to implement---followed by many tokens that simply work out the consequences of those choices. 
A sampler that treats every token position as equally worth revising will usually cut too late, so it mostly rewrites details within the same line of reasoning rather than moving to a different one. The example below makes this concrete in the setting of a math proof.

\begin{example}[Proof Strategy as an Early Choice]
\label{ex:proof-strategy}
Consider proving a generalization bound. Early in the trace, the model may choose between a uniform-convergence argument and an algorithmic-stability argument. Once this choice is made, later tokens mostly execute the chosen proof strategy (see \cref{fig:conductance-reasoning}). If the choice occurs at position $t^\star$ in a trace of length $\ell$, a uniform cut revisits it with probability only $t^\star/\ell$; when $t^\star \ll \ell$, the sampler rarely tries a different approach, creating a low-conductance bottleneck that increases mixing time.
\end{example} 
A natural way to address this is therefore to revisit such consequential decisions more often.
To turn this idea into an algorithm, however, we need a generic way to identify where these decisions occur.
In some special cases, one could do this using ad hoc rules---for example, in a proof one might ask the model to write in explicit steps and then cut at the beginning of each step.
But such rules do not generalize across the wide variety of reasoning tasks handled by modern LLMs.

 Our first \takeaway{insight is to use the entropy of the base model's next-token distribution as a proxy for a decision point.}
 Intuitively, entropy is high when several continuations are plausible, and low when the continuation is largely determined. For each $0 \leq t \leq \ell$, we define
$h_t(\vec{x})
    \coloneqq
    -\sum_{v \in V} p(v \mid x_{<t})\, \log p(v \mid x_{<t}),$
the Shannon entropy of the base model's next-token distribution after $x_{<t}$. %
 
Using $h_t$ directly has one drawback: after an important decision point, entropy often remains elevated for several subsequent tokens while the consequences of that decision are being worked out. A cut distribution proportional to $h_t$ would therefore spread mass across an entire high-entropy stretch, even though the most informative place to cut is often where the uncertainty first increases. Motivated by this, we use the \emph{positive entropy jump}
\[
    \Delta_t(\vec{x})
    \coloneqq
    \max\inbrace{0,~ h_t(\vec{x}) - h_{t-1}(\vec{x})}
    \qquad (1 \leq t \leq \ell)\,,
    \yesnum\label{def:entropy-jump}
\]
which emphasizes positions where the model's uncertainty has just increased relative to the preceding token, while suppressing later tokens in the same high-entropy region.
{This leads to the following cut law.}
\codexneuripssamepage{%
\begin{definition}[Entropy-Cut Distribution]
\label{def:entropy-cut-law}
    Fix a \emph{cut power} $\beta \geq 0$. The \emph{entropy-cut distribution} %
    \[
        \text{for a state $x$ of length $\ell$\,,}\quad 
        \lambda_\beta(m;\vec{x})
        ~\propto~
        \Delta_m(\vec{x})^\beta
        \qquad (1 \leq m \leq \ell)\,.
    \]
    In the edge case where $\Delta_t(\vec{x}) = 0$ for every $t$, the distribution is uniform over $\inbrace{1,\dots,\ell}$.
\end{definition}
}%
Here, the choice $\beta = 0$ recovers the uniform cut distribution $\lambda_{\mathrm{unif}}$ of \citet{karan2026reasoning}, while larger $\beta$ places more mass on larger entropy jumps.

Entropy has recently been used to identify branching points in chain-of-thought reasoning and to allocate additional computation at such points \citep{wang2025beyond,li2026entropy,zhang2025entropy}. 
Our use of entropy is somewhat different: rather than using it to guide search, training, or evaluation of the output, we use entropy jumps to define the \emph{proposal distribution} of a Metropolis--Hastings sampler targeting the power distribution. 
In particular, entropy affects the MH sampling process. %

\paragraph{Entropy-Cut MH and Its Running Time.}
Our algorithm, \emph{Entropy-Cut MH}, is obtained by instantiating the meta-algorithm of \cref{alg:meta} with the cut distribution $\lambda_\beta$; see \cref{alg:samp} for the full pseudocode. Relative to the uniform-cut sampler of \citet{karan2026reasoning}, the only change is the choice of cut distribution: the stagewise targets, the proposal model $p_{\mathrm{prop}}$, and the Metropolis--Hastings acceptance rule are unchanged. Moreover, the entropies $h_t(\vec{x})$ are already produced by the forward passes used to evaluate $p(\vec{x})^\alpha$, so \highlight{computing $\lambda_\beta$ adds no cost over the uniform-cut baseline}.

 \begin{figure}[t]
\centering
\begin{minipage}[t]{0.44\linewidth}
\vspace{0pt}
\centering
\vspace{0pt}
\begin{tikzpicture}
    \node[anchor=south west, inner sep=0] (img) at (0,0)
    {\includegraphics[width=0.8\linewidth]{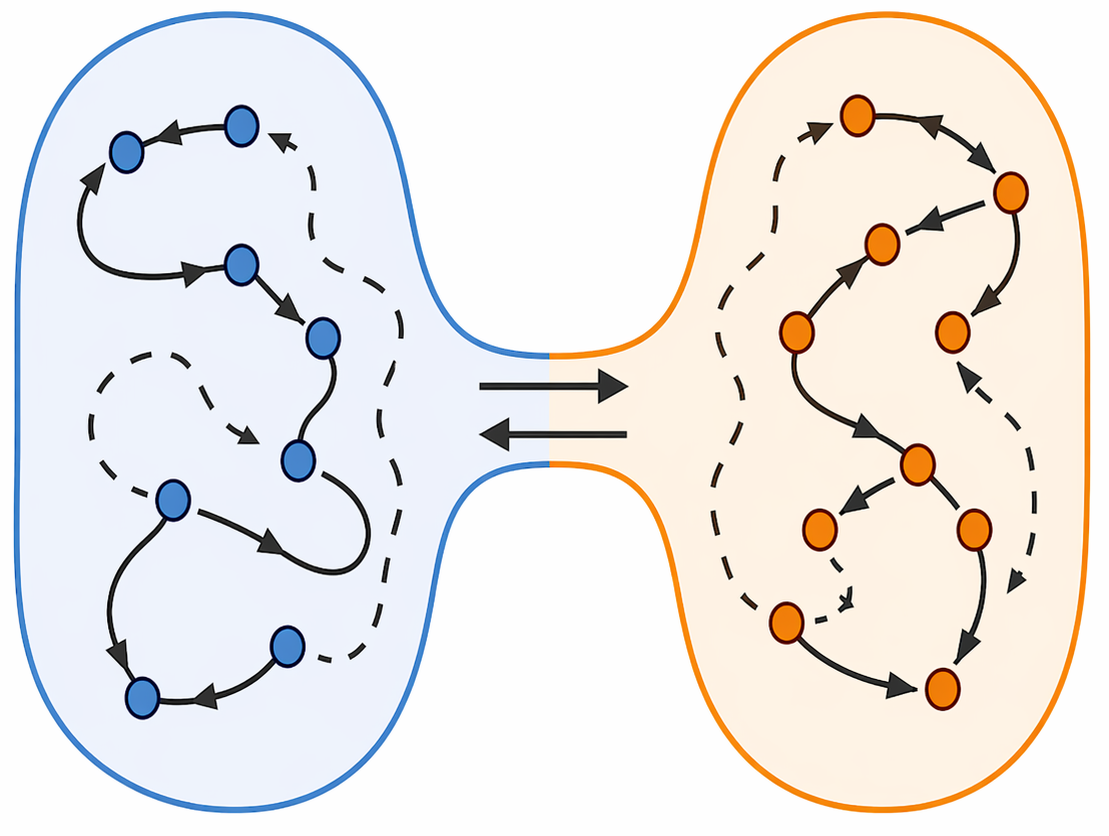}};

    \begin{scope}[x={(img.south east)}, y={(img.north west)}]
        \node[align=center] at (0.2,-0.20)
        {Approach 1\\ (\eg{}, proof by\\ induction)};

        \node[align=center] at (0.8,-0.20)
        {Approach 2\\ (\eg{}, proof by\\ contradiction)};

        \node[align=center] at (0.50,0.25)
        {\textit{early}\\\textit{decision}};

        \draw[densely dashed, -{Latex[length=2mm]}, thick]
            (0.50,0.33) -- (0.50,0.43);
    \end{scope}
\end{tikzpicture}
\end{minipage}
\hfill
\begin{minipage}[t]{0.45\linewidth}
\caption{
Illustration of a low-conductance bottleneck. %
Each region represents a family of reasoning traces that share the same high-level approach (\eg{}, two different proof strategies).
Within a region, the sampler can make many local edits. %
To move from one region to the other, however, the sampler must revisit an \emph{early decision} that determined the approach. %
When this happens rarely, it creates a low-conductance bottleneck leading to slow mixing.
}
\label{fig:conductance-reasoning}
\end{minipage}
\vspace{0pt}
\end{figure}
\subsection{Empirical Validation of the Entropy Proxy}
\label{subsec:entropy-proxy-empirical}

The entropy-cut rule is useful only if large positive entropy jumps tend to occur near genuine decision points in the model's reasoning. We test this directly by comparing what happens when we resample from positions with large $\Delta_t$ to what happens when we resample from positions with very small $\Delta_t$.

\paragraph{Experiment.}
We conduct a simple experiment to test how sensitive an LLM's continuation is to resampling at different points in a chain-of-thought. For each MATH500 problem, we generate a single completion and record the token-level entropies along the way. We then compute the entropy jumps $\Delta_t$, select 5 cut positions from the top decile of the $\Delta_t$ values and 5 from the bottom decile, and resample the suffix 16 times from each selected position while keeping the prefix fixed. We quantify the resulting variation in two ways: \emph{normalized edit distance} between the resampled suffixes, computed as the average pairwise token-level Levenshtein distance divided by the mean suffix length, and \emph{distinct-answer fraction}, computed as the number of unique parsed final answers among the resamples divided by the number of resamples. These measure suffix divergence at each position.

\paragraph{Findings and Interpretation.}
We find that cutting at top-decile $\Delta_t$ positions produces substantially more divergent continuations than cutting at bottom-decile positions. The normalized edit distance is about $1.33\times$ higher (0.287 vs.\ 0.216), and the distinct-answer fraction is about $1.36\times$ higher (0.63 vs.\ 0.47). 
This is direct evidence that large entropy jumps tend to mark branching points in the reasoning trace: resampling at such points leads to more varied downstream text and more varied conclusions, whereas resampling at low-$\Delta_t$ positions typically yields much more consistent suffixes.

\section{Theoretical Results on Mixing Time}
\label{sec:reasoning-tree-model}
\label{sec:theoretical-results}

In this section, we study the mixing time of our method under the following stylized model: %
\begin{definition}[Reasoning Tree]
\label{def:reasoning-tree}
    Fix a prompt and a target length $T$. A \emph{reasoning tree} for $(\Pi_T, p_{\mathrm{prop}})$ is a rooted, token-labeled tree of depth $T$ whose root-to-leaf paths are token sequences $x_{0:T}$, such that $\Pi_T$ and $p_{\mathrm{prop}}$ are both supported on these paths. For a root-to-leaf path $\vec{x} = x_{0:T}$, the node at depth $t \in \inbrace{1, \dots, T}$ is a \emph{branch node} on the path if the prefix $x_{0:t-1}$ has at least two children in the tree, and a \emph{chain node} otherwise.
\end{definition}
Intuitively, a \textit{chain node} represents positions where the next token is effectively determined---for instance, the next step of a calculation already in progress, or completing the implementation of an already-described function. 
A \textit{branch node}, in contrast, represents a position where a decision must be made, \eg{}, which calculation to do or which function to implement. 
For instance, in \cref{ex:proof-strategy} the reasoning tree has a branch at depth $b_1 = O(1)$ where the path splits by proof technique; subsequent branches encode decisions internal to the chosen technique.
 
The reasoning tree constrains the support of the power distribution $\Pi_T$ but does not place any constraint on how it distributes its mass across these paths. 
How this mass is distributed can substantially affect the time a method takes to converge to $\Pi_T$.
Indeed, if $\Pi_T$ places nearly all of its mass on a single path $p$, then any instantiation of Meta-Algorithm~\ref{alg:meta} lands on $p$ quickly and the choice of the cut distribution is irrelevant.
The challenging regime is where $\Pi_T$ places comparable mass on several paths and the sampling algorithm must ``explore'' all of them.
To isolate the effect of the cut distribution, we analyze our method in this regime, formalized in the condition below.

\begin{condition}[Approximately Symmetric Reasoning Tree]
\label{cond:approx-symmetric-reasoning-tree}
    Let $U$ be the uniform distribution over root-to-leaf paths.
    For some $\eta \geq 0$, the reasoning tree satisfies the following symmetry conditions:
    \begin{enumerate}%
        \item For each root-to-leaf path $\vec{x}$, $e^{-\eta} U(\vec{x}) \leq \Pi_T(\vec{x}) \leq e^{\eta} U(\vec{x})$. %
        \item For each root-to-leaf path $\vec{x}$ and $1<t<T$, the proposal distribution satisfies $\frac{U(\vec{x}_{>t}|x_{0:t})}{p_{\mathrm{prop}}(x_{>t}|x_{0:t})} \leq e^\eta$.
        \item For each root-to-leaf path $\vec{x}$,
            the entropy jump $\Delta_t(\vec{x})$ is $\Delta>0$ at branch nodes and $0$ otherwise.
        \item For each branch node in the tree, the subtrees rooted at its children are isomorphic.
    \end{enumerate}
\end{condition}
\begin{figure}[hbt!]
    \centering
    \begin{tikzpicture}[
        x=1cm,
        y=1cm,
        transition/.style={
            -{Latex[length=2mm,width=1.5mm]},
            line width=0.55pt,
            draw=black!80
        },
        entropy/.style={
            -{Latex[length=2.4mm,width=1.8mm]},
            line width=0.8pt,
            draw=red!80!black
        },
        every node/.style={font=\sffamily\scriptsize},
        reasoning node/.style={
            draw=black!70,
            rounded corners=2.5pt,
            minimum width=2.25cm,
            minimum height=0.58cm,
            inner xsep=2pt,
            inner ysep=2pt,
            align=center
        },
        chain node/.style={
            reasoning node,
            fill=black!7
        },
        branch node/.style={
            reasoning node,
            fill=blue!22
        },
        entropy label/.style={
            font=\sffamily\scriptsize,
            text=red!85!black,
            align=left
        },
        legend text/.style={
            font=\sffamily\scriptsize,
            anchor=west
        }
    ]
    
    \node[chain node]  (proof) at (0, -0.78) {proof};
    \node[chain node]  (is)    at (0, -1.68) {is};
    \node[branch node] (via)   at (0, -2.58) {via};
    
    \node[branch node] (ind)    at (-2.5, -3.48) {induction};
    \node[branch node] (contra) at ( 0.00, -3.48) {contradiction};
    \node[branch node] (cons)   at ( 2.5, -3.48) {constructive};
    
    \node[chain node] (i1) at (-2.5, -4.38) {base case};
    \node[chain node] (i2) at (-2.5, -5.28) {assume $k$};
    
    \node[chain node] (c1) at (0.00, -4.38) {suppose not};
    \node[chain node] (c2) at (0.00, -5.28) {derive conflict};
    
    \node[chain node] (s1) at (2.5, -4.38) {define object};
    \node[chain node] (s2) at (2.5, -5.28) {verify};
    
    \draw[transition] (proof) -- (is);
    \draw[transition] (is) -- (via);
    
    \draw[transition] (via) -- (ind);
    \draw[transition] (via) -- (contra);
    \draw[transition] (via) -- (cons);
    
    \draw[transition] (ind) -- (i1);
    \draw[transition] (i1) -- (i2);
    
    \draw[transition] (contra) -- (c1);
    \draw[transition] (c1) -- (c2);
    
    \draw[transition] (cons) -- (s1);
    \draw[transition] (s1) -- (s2);
    
    \draw[entropy]
        ($(via.east)+(0.55,-0.20)$)
        --
        ($(via.east)+(0.55,0.48)$);
    
    \node[entropy label, anchor=west]
        at ($(via.east)+(0.72,0.08)$)
        {Entropy jump (high $\Delta_t$)};

    \end{tikzpicture}
\caption{
Illustration of a reasoning tree. Chain nodes (gray) are positions where the next token is effectively determined; branch nodes (blue) are positions where a semantic decision must be made. Entropy jumps occur at branch nodes. For example, the branch after the token ``via'' corresponds to the choice of proof strategy.
}
\label{fig:reasoning-tree-model}
\vspace{0pt}
\end{figure}
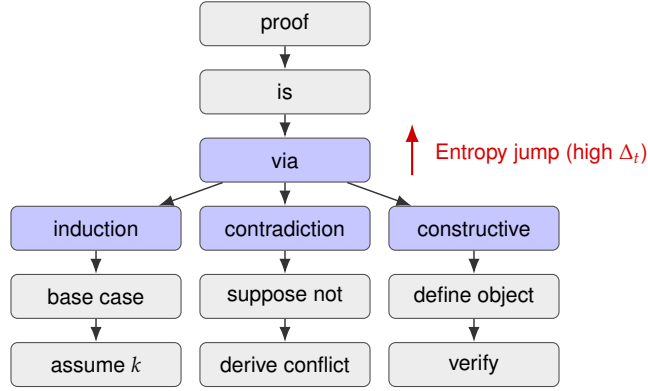
    The first two conditions control how far the power distribution $\Pi_T$ and the proposal distribution $p_{\mathrm{prop}}$ can deviate from the uniform distribution $U$ over the root-to-leaf paths. 
    Here, the parameter $\eta$ controls the amount of deviation, and $\eta=0$ corresponds to perfect symmetry, the hardest regime.
    Condition~3 formalizes our observation in \cref{sec:entropy-cut} that branch nodes have a much larger entropy jump than chain nodes. 
    To simplify the statement, we assume all branch nodes have the same entropy jump $\Delta$.
    Our result extends to the more general case at the cost of a more complicated statement.
    Finally, Condition~4 assumes the reasoning tree is symmetric, so all decisions at a branch point appear ``equally good.''
    Again, this isolates the hardest case for the analysis, and the result extends to more general trees.
    In particular, due to symmetry, this condition implies that each root-to-leaf path has the same number of branch nodes, say $k$.
    We call $k$ the \textit{semantic depth} of the reasoning tree.

    Recall that the \emph{mixing time} $\tau_{\mathrm{mix}}^\lambda(\eps)$ of an MH chain measures how many steps it needs before its distribution is within $\eps$ of the target in TV-distance.\footnote{Formally, $\tau_{\mathrm{mix}}^\lambda(\eps)
        \coloneqq
        \min\inbrace{
            n \geq 0:
            \sup_{\vec{x}}\,
            \tv{P_\lambda^n(\vec{x},\cdot)}{\Pi_T}
            \leq \eps
        },$ where the supremum is over root-to-leaf paths.}
    Our main result compares the mixing times $\tau_{\mathrm{mix}}^{\mathrm{ec}}$ and $\tau_{\mathrm{mix}}^{\mathrm{unif}}$ of the entropy-cut and uniform-cut MH chains, respectively. 

\begin{theorem}[Mixing Separation Under Approximate Symmetry]
\label{thm:approx-symmetric-mixing-separation}
    Consider the reasoning-tree model under \cref{cond:approx-symmetric-reasoning-tree}.
    Fix $\beta > 0$ and set the number of stages $K = 1$. 
    Let the number of branch nodes in a root-to-leaf path be $k$ (\ie{}, the semantic depth is $k$).\footnote{Due to Item~4 in \cref{cond:approx-symmetric-reasoning-tree}, the number of branch nodes in each root-to-leaf path is the same.}
    Then, for every $\eps\in (0,\nfrac{1}{2}),$
    \begin{enumerate}%
        \item The entropy-cut chain satisfies
            $\tau_{\mathrm{mix}}^{\mathrm{ec}}(\eps)
                \lesssim
                e^{2\eta} k \log{\nfrac{1}{\eps}};$
        \item Moreover, if the first branch node is at depth $b_1 \leq T/2$, then $\tau_{\mathrm{mix}}^{\mathrm{unif}}(\eps)
        =
        \Omega_{\eps}(\nfrac{T}{b_1}).$
    \end{enumerate}
\end{theorem}
Thus, for $\eta = O(1)$ and constant accuracy, entropy-cut mixes in $O(k)$ steps while uniform-cut requires $\Omega(T/b_1)$ steps.
\takeaway{Thus, the relevant scale for entropy-cut is the number of semantic decisions $k$, whereas for uniform-cut it is the token depth $T/b_1$ of the earliest decision.}
When $k \ll T/b_1$, the separation is large: if the first decision occurs after $O(1)$ tokens and $k = o(T)$ (as in \cref{fig:conductance-reasoning}), entropy-cut mixes in $o(T)$ steps while uniform-cut requires $\Omega(T)$.
The proof of \cref{thm:approx-symmetric-mixing-separation} appears in \cref{sec:proof-approx-symmetric-mixing-separation}.

\section{Evaluation of Performance on Reasoning and Instruction Following}\label{sec:experiments}
In this section, we present our main empirical results. 
\subsection{Setup}\label{sec:experiments:setup}
We use a standard suite of reasoning benchmarks across math, coding, and STEM.
All of our evaluations are in the \emph{single-shot} setting, \ie{}, with a single response string.
We report the average scores across 8 independent runs for all datasets except AIME26, which we repeat 64 times.
Each run is completed on a single H200 GPU.
Our implementation leverages the \texttt{\small vllm v0.9.2} inference library~\citep{kwon2023efficient}.

\paragraph{Reasoning and Instruction-Following Benchmarks.} 
We evaluate on MATH500, HumanEval, GPQA Diamond, and AIME26, covering competition mathematics, code generation, graduate-level science questions, and olympiad-style mathematical reasoning. Dataset details and evaluation protocols are given in \cref{sec:appendix:benchmark-details}.

\begin{table}[tbh!]
\centering
\vspace{0pt}
\small
\setlength{\tabcolsep}{4pt}
\renewcommand{\arraystretch}{0.9}
\caption{\textbf{Entropy-Cut MH (ours) is competitive with and often improves over prior (power)
sampling techniques across model families and tasks.} We benchmark our sampler on
MATH500, HumanEval, GPQA Diamond, and AIME26. All scores are
percentages. We bold the best score in each column within a model family and
underline our scores (last row per block). 
}
\label{tab:main}
\begin{tabular}{@{}llcccc@{}}
\toprule
Model & Method & MATH500 & HumanEval & GPQA Diamond & AIME26 \\
\midrule
\multirow{6}{*}{\texttt{\small Qwen2.5-7B}}
  & Standard                  & 35.9 & 33.0 & 29.4 & 2.0 \\
  & Low-Temperature       & 62.3 & 64.3 & 29.5 & 5.7 \\
  & SMC                   & 69.8 & 64.9 & 28.8 & 7.4 \\
  & TMC                   & 70.4 & 66.2 & 27.5 & \textbf{9.4} \\
  & Uniform-Cut MH           & 67.4 & 66.8 & 29.4 & 8.2 \\
  & Entropy-Cut MH (ours) & \textbf{\underline{71.9}} & \textbf{\underline{68.9}} & \textbf{\underline{30.2}} & \textbf{\underline{9.4}} \\
\midrule
\multirow{6}{*}{\texttt{\small Qwen2.5-Math-7B}}
  & Standard                  & 50.1 & 31.0 & 29.0 & 4.8 \\
  & Low-Temperature       & 68.3 & 55.8 & 33.1 & 6.4 \\
  & SMC                   & 78.9 & 55.2 & 32.0 & 11.8 \\
  & TMC                   & 78.2 & 51.2 & 32.5 & 10.6 \\
  & Uniform-Cut MH           & 73.1 & 59.8 & 32.4 & 9.5 \\
  & Entropy-Cut MH (ours) & \textbf{\underline{79.0}} & \textbf{\underline{59.9}} & \textbf{\underline{34.1}} & \textbf{\underline{13.1}} \\
\midrule
\multirow{6}{*}{\texttt{\small Qwen3-8B-Base}}
  & Standard                  & 58.5 & 47.6 & 31.7 & 3.2 \\
  & Low-Temperature       & 77.3 & 73.7 & 37.6 & 6.4 \\
  & SMC                   & 79.3 & 76.7 & 38.1 & 10.1 \\
  & TMC                   & 79.7 & 75.2 & 38.6 & 9.8 \\
  & Uniform-Cut MH           & 79.5 & 75.8 & 38.8 & 9.5 \\
  & Entropy-Cut MH (ours) & \textbf{\underline{80.2}} & \textbf{\underline{79.3}} & \textbf{\underline{40.0}} & \textbf{\underline{10.3}} \\
\midrule
\multirow{6}{*}{\texttt{\small Phi-3.5-mini-instruct}}
  & Standard                  & 42.9 & 40.5 & 27.8 & 0.6 \\
  & Low-Temperature       & 45.3 & 62.2 & 32.5 & 1.6 \\
  & SMC                   & \textbf{49.8} & 65.2 & 31.0 & 1.7 \\
  & TMC                   & 49.3 & \textbf{65.6} & 31.5 & 3.0 \\
  & Uniform-Cut MH           & 49.1 & 64.9 & 31.8 & 2.9 \\
  & Entropy-Cut MH (ours) & \underline{49.7} & \textbf{\underline{65.6}} & \textbf{\underline{32.8}} & \textbf{\underline{3.1}} \\
\midrule
\multirow{6}{*}{\texttt{\small Phi-4-mini-instruct}}
  & Standard                  & 55.5 & 39.9 & 30.7 & 3.5 \\
  & Low-Temperature       & 66.4 & 65.5 & 32.1 & 6.8 \\
  & SMC                   & 68.1 & 66.7 & 31.9 & 8.6 \\
  & TMC                   & 66.1 & 66.1 & 32.1 & 6.9 \\
  & Uniform-Cut MH           & 68.7 & 67.9 & 32.0 & 8.0 \\
  & Entropy-Cut MH (ours) & \textbf{\underline{68.8}} & \textbf{\underline{68.4}} & \textbf{\underline{33.3}} & \textbf{\underline{8.7}} \\
\bottomrule
\end{tabular}
\vspace{0pt}
\end{table}

\paragraph{Language Models.}
Following \citet{karan2026reasoning}, we use the base models \texttt{\small Qwen2.5-7B}, \texttt{\small Qwen2.5-Math-7B}, \texttt{\small Qwen3-8B-Base}, as well as the instruction-tuned models \texttt{\small Phi-3.5-mini-instruct} and \texttt{\small Phi-4-mini-instruct}.

\paragraph{Baselines.}
To demonstrate the efficacy of our algorithm, we compare against the following: %
\begin{itemize}%
    \item \textbf{Standard}: 
    Sampling from the conditional next-token distribution provided by the model.
    
    \item \textbf{Low-Temperature}: 
    Sampling from the exponentiated conditional next-token distribution with exponent (\emph{temperature}) $\tau = \nicefrac1\alpha$.
    
    \item \textbf{Sequential Monte Carlo (SMC)}:  SMC~\citep{pitt1999filtering,del2006sequential,johansen2009tutorial} is a well-studied particle Gibbs method for sampling from distributions over sequences given access to unnormalized probabilities.
    In the context of power sampling, SMC maintains a weighted population of partial completions,
    updates their weights as tokens are decoded, and resamples when the weights become too uneven~\citep{azizi2026power}.
    
    \item \textbf{Twisted Monte Carlo (TMC)}: 
    Another well-studied sequential sampling technique is to \emph{modulate} the conditional next-token distribution with some estimate of the future ``rewards'' (a \emph{twist} function). 
    For power sampling, the twist is taken to be an estimate of the future probabilities of the completion~\citep{ji2026scalable}.
    
    \item \textbf{Uniform-Cut Metropolis--Hastings (MH)}:      
    MH with uniformly sampled cuts.
\end{itemize}
As SMC and TMC do not have public implementations, we implement our own following the description in their respective works.
\vspace{0pt}

\paragraph{Hyperparameters.}
Unless otherwise stated, we follow the hyperparameter choices of \citet{karan2026reasoning}.
We set the maximum length $T = 3072$ and block size $B = 3072/16=192$ for MATH500, HumanEval, and double $T$ and $B$ for the more challenging GPQA Diamond and AIME26 datasets.
We set the power exponent $\alpha=4.0$ for all datasets except for HumanEval, for which we found a higher $\alpha=5.0$ to perform better.
The cut law exponent is set to $\beta=4.0$ for all datasets.
The number of MH transition steps is set to $N_{\mathrm{MCMC}} = 10$.
We take the proposal distribution $p_{\mathrm{prop}}$ to be the low-temperature distribution with temperature $\tau = \nicefrac1\alpha$.
The hyperparameters for SMC and TMC are based on their best-reported settings.
See \Cref{apx:experiments} for more details about SMC, TMC, and hyperparameter ablations. 

\subsection{Results}
Our main findings are displayed in \Cref{tab:main}.
\takeaway{Across base and instruction-tuned models from different families, our sampling algorithm consistently improves accuracy on reasoning tasks.}
In particular, using \texttt{\small Qwen2.5-7B}, we observe a \textbf{+36.0} gain on MATH500 and a \textbf{+35.9} gain on HumanEval relative to standard sampling.
The consistent improvement of entropy-cut over uniform-cut MH across model families and tasks supports our theoretical prediction (\Cref{thm:approx-symmetric-mixing-separation}) that cutting at semantic decision points yields faster mixing, while reaffirming that improved power sampling can further extract the latent reasoning capabilities of existing models~\citep{karan2026reasoning}.

{\paragraph{Comparisons at Matched Token Budgets.}
We compare entropy-cut MH against standard and low-temperature best-of-$N$ (BoN) sampling on MATH500, HumanEval, and AIME26 using \texttt{\small Qwen2.5-Math-7B} and \texttt{\small Qwen3-8B-Base}.}
{We use each entropy-cut MH run's sampled-token count, including rejected proposals, as the BoN stopping target. 
We select the response with the highest average log-likelihood under the base model.}
{Entropy-cut MH attains a higher score than standard BoN in all six settings and than low-temperature BoN in five of six; the exception is \texttt{\small Qwen2.5-Math-7B} on HumanEval.
Full results, significance tests, and the budget-matching procedure appear in \cref{apx:experiments:best-of-n}.}

\subsection{Analysis}
\ifshowtraces
We analyze several important characteristics of various (power) sampling algorithms and defer additional details, such as {significance tests (\Cref{apx:experiments:significance}), }hyperparameter ablations, running times, {compute-matched best-of-$N$ (\Cref{apx:experiments:best-of-n}), }and example reasoning traces, to \Cref{apx:experiments,apx:traces}.
\else
We analyze several important characteristics of various (power) sampling algorithms and defer additional details, such as {significance tests (\Cref{apx:experiments:significance}), }hyperparameter ablations, running times, and compute-matched best-of-$N$ (\Cref{apx:experiments:best-of-n}) to \Cref{apx:experiments}.
\fi

\paragraph{Sequence Log-Probabilities.}
We plot the distribution of the sequence log-likelihood values computed via the base model (\texttt{\small Qwen2.5-Math-7B}) on the MATH500 dataset across 8 random seeds in \Cref{fig:SequenceLogprobsSubset}.
In addition, the 25th, 50th (median), and 75th percentile values are marked with dashed, solid, and dotted vertical lines.
{We observe that entropy-cut} MH samples from higher-likelihood regions than standard, low-temperature, and uniform-cut MH.
See \Cref{apx:experiments:likelihood-confidence} for average \emph{confidence} plots.

\begin{figure}[tb!]
    \centering
    \vspace{0pt}
    \includegraphics[width=\linewidth]{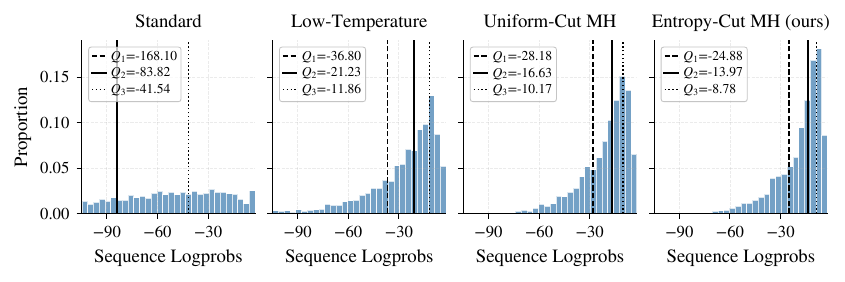}
    \vspace{0pt}
    \caption{\textbf{Entropy-Cut MH (ours) samples from higher-probability regions than baselines.}
    We plot \texttt{\small Qwen2.5-Math-7B} log-likelihoods for MATH500 responses across samplers. The 25th ($Q_1$), 50th ($Q_2$; median), and 75th ($Q_3$) percentile values are marked with dashed, solid, and dotted vertical lines, respectively.}
    \label{fig:SequenceLogprobsSubset}
\vspace{0pt}
\end{figure}

\paragraph{Diversity and Pass@$k$.}
We plot the pass@$k$ accuracy (at least one correct answer among $k$) in \Cref{fig:PassAtK} using an unbiased estimator~\citep{chen2021evaluating}.
This demonstrates that our method maintains diversity comparable to standard sampling despite its improved pass@1 performance.

\begin{figure}[tb!]
    \centering
    \begin{minipage}[t]{0.48\linewidth}
        \vspace{0pt}
        \centering
        \includegraphics[width=\linewidth]{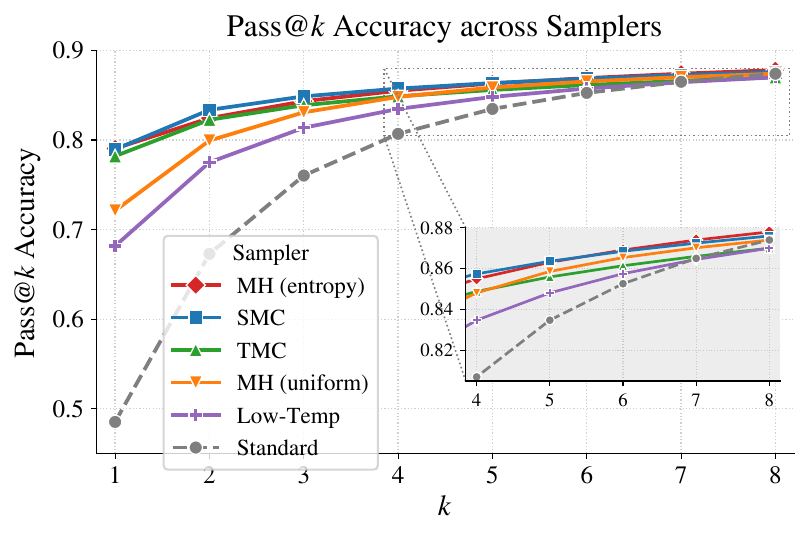}
    \end{minipage}
    \hfill
    \begin{minipage}[t]{0.47\linewidth}
        \vspace{0pt}
        \caption{\textbf{Entropy-Cut MH (ours) maintains diversity across multiple passes despite the improved pass@1 performance.} We plot \texttt{\small Qwen2.5-Math-7B} pass@$k$ accuracy for MATH500 responses across samplers.}
        \label{fig:PassAtK}
    \end{minipage}
    \vspace{0pt}
\end{figure}

\section{Conclusion, Limitations, and Future Work}
\label{sec:conclusion}
We introduced Entropy-Cut Metropolis--Hastings, a simple modification of the stagewise sampler of \citet{karan2026reasoning} that places cuts at positions of high next-token entropy rather than uniformly over tokens, focusing proposals on points where the model appears to face a genuine choice.
This preserves the same power-distribution target while better matching the structure of reasoning traces.
Our theory shows that, in a reasoning-tree model, entropy-cut mixing scales with the number of semantic decisions rather than the full token length, and empirically, Entropy-Cut MH improves performance across models and reasoning benchmarks while maintaining sample diversity.

Several directions remain open.
Our experiments focus on small and medium open models; evaluating entropy-cut sampling at larger scale would clarify how much further reasoning performance can be extracted from stronger base distributions.
We also study single-shot completions, but many practical reasoning systems revise partial solutions, call tools, or branch across attempts; extending entropy-cut proposals to these settings could make the sampler a useful component in broader inference-time reasoning.
Finally, our entropy proxy is deliberately simple; future work could learn or adapt the cut distribution from richer signals such as model uncertainty, verifier feedback, or cross-sample disagreement, while retaining the Metropolis--Hastings correction that keeps the target distribution unchanged.

\subsection*{Acknowledgments}
Felix Zhou acknowledges the support of the Natural Sciences and
Engineering Research Council of Canada (NSERC). Quanquan C. Liu and Felix Zhou are supported in part by the National Science Foundation (NSF) under Grant \#CCF-2453323 and a Google Academic Research Award.

\printbibliography

\appendix
\crefalias{section}{appendix}
\crefalias{subsection}{subappendix}
\crefalias{subsubsection}{subsubappendix}

\addtocontents{toc}{\protect\setcounter{tocdepth}{2}} %
\newpage

\section{Additional Related Work}
    \label{sec:appendix:related-work}
    In this section, we discuss additional related lines of work and how they compare to our work.

    \paragraph{Power Sampling and Diversity.}
    Recent work studies how sharpening a trajectory distribution can affect answer aggregation and path diversity. \citet{yang2026morecorrect} show that power sampling can increase probability mass on correct trajectories while reducing self-consistency accuracy. \citet{ahmadi2026chopthin} propose an SMC resampler and diversity-aware answer selection to preserve distinct reasoning paths under power sampling.

    \paragraph{Speculative Decoding.}
    The structural idea of drawing a candidate from a cheap proposal and accepting it with a ratio-based rule resembles \emph{speculative decoding} \citep{leviathan2023fast,chen2023accelerating}, in which a small draft model proposes continuations that are then accepted or rejected according to ratios of probabilities under a larger target model. The two methods share the accept--reject template but serve different purposes: speculative decoding accelerates exact sampling from a fixed autoregressive target, whereas MH produces samples from a target, here $\Pi_T$, that is not itself autoregressive. 
    
\paragraph{Entropy as an Uncertainty Signal.}
Several works use entropy or related uncertainty estimates from a language model's predictive distribution to identify unreliable generations. In autoregressive sequence prediction, token- and sequence-level uncertainty have been studied for error detection and out-of-domain detection \citep{malinin2021uncertainty}. For factual generation, \citet{fadeeva2024factchecking} use token-level uncertainty to identify claims that are likely to be unreliable. These works use uncertainty primarily as a diagnostic for the generated output. We use the same broad signal for a different purpose: rather than deciding whether a completed answer should be trusted, we use local changes in uncertainty to decide where an MH proposal should cut and resample a reasoning trajectory.

\paragraph{Entropy and Reasoning Decisions.}
Closest to our setting are recent works that treat high-entropy tokens as places where a reasoning trajectory branches. \citet{li2026entropy} propose entropy-gated branching, which allocates additional test-time search at high-entropy moments where multiple continuations are plausible. \citet{wang2025beyond} study entropy patterns in chain-of-thought reasoning and find that a small fraction of high-entropy tokens can act as ``forks'' that steer the model toward different reasoning paths. Some work also considers changes in uncertainty over time: \citet{zhang2025entropy} use changes in entropy and variance-entropy across consecutive reasoning steps to decide whether to deepen, branch, or stop exploration. Taking inspiration from these observations, we use entropy differences as a lightweight proxy for decision points in the model's reasoning. Unlike beam- or tree-expansion methods, however, we do not allocate additional samples at every such point. Instead, we use the signal to choose cut locations inside an MH sampler, with the goal of revisiting uncertain reasoning decisions while {avoiding cuts inside} low-entropy stretches of local computation.

\paragraph{Semantic Entropy.}
A related line of work argues that token-level entropy can be misleading in open-ended generation, because many different strings may express the same meaning. \citet{farquhar2024detecting} therefore propose \emph{semantic entropy}, which measures uncertainty over semantic equivalence classes of sampled responses rather than over surface forms, and use it to detect confabulations in LLM outputs. Follow-up work such as \citet{kossen2024semantic} aims to approximate semantic entropy more efficiently from model representations. This line of work supports the general idea that uncertainty in the model distribution can reveal meaningful properties of the generation, but it also highlights a limitation of purely token-level entropy. In this paper we do not use semantic entropy: our cut rule is based on token-level entropy changes because they are cheap to compute during generation and naturally aligned with token-level cut positions. Understanding whether semantic entropy, or another meaning-aware uncertainty measure, can further improve the choice of cut locations is an interesting direction for future work.
 
\paragraph{Critical Tokens and Decision Points.}
Several recent works observe that reasoning trajectories often contain a small number of high-leverage tokens.
The \emph{Phi-4} technical report introduces \emph{pivotal token search}, which finds tokens whose generation substantially changes the probability of final correctness and uses them to construct targeted DPO pairs \citep{abdin2024phi}.
\citet{lin2024criticaltokens} identify \emph{critical tokens} in incorrect mathematical reasoning traces and show that modifying these tokens can substantially improve downstream correctness.
\citet{li2025critical} study the related phenomenon of \emph{critical windows}, where model behavior changes sharply over a narrow region of the generation process.
\citet{karan2026reasoning} connect this idea to power sampling, showing that the power distribution can favor tokens whose futures concentrate on high-likelihood completions rather than tokens that merely look locally likely.
Our method is in the same spirit, but uses entropy jumps as a verifier-free inference-time proxy for these decision points, allowing the Metropolis--Hastings sampler to revisit consequential choices without changing its target distribution.

\neuripsvspace{-5mm}
\section{Additional Discussion and Remarks}
\neuripsvspace{-3mm}
\subsection{Remarks on Metropolis--Hastings}
    
\begin{remark}[MH Acceptance Ratio Is Tractable for $\Pi_T$]\label{rem:mh-tractability}
    Because $p$ is autoregressive, $p(\vec{x}) = \prod_{t=0}^{\ell} p(x_t\mid x_{<t})$ is computable from a single forward pass, and the MH ratio depends on $p$ only through $p(\vec{x}')^\alpha / p(\vec{x})^\alpha$. The intractable sum defining $Z_{T,\alpha}$ therefore never needs to be computed. The cut-based proposal kernel below can likewise be sampled by a single autoregressive pass, so both the acceptance ratio and the proposal run in polynomial time despite the exponential state space.
\end{remark}
\neuripsvspace{-8mm}

\begin{remark}[Explicit Form of the Cut-Based Proposal Kernel]\label{rem:cut-kernel}
    For a cut index $m$, the proposal kernel that retains $x_{0:m-1}$ and redraws $x_{m:\ell}$ from $p_{\mathrm{prop}}$ is
    \[
        Q_m(\vec{x},\vec{x}')
        =
        \begin{cases}
            p_{\mathrm{prop}}\inparen{x'_{m:\ell}\mid x_{0:m-1}}, & \text{if } x'_{0:m-1}=x_{0:m-1}\,,\\
            0\,, & \text{otherwise}\,,
        \end{cases}
    \]
    where $p_{\mathrm{prop}}\inparen{x_{m:\ell}\mid x_{0:m-1}}
        \coloneqq
        \prod_{t=m}^{\ell} p_{\mathrm{prop}}\inparen{x_t\mid x_{<t}}.$
    The factor $\lambda(m;\vec{x}')/\lambda(m;\vec{x})$ in the acceptance probability corrects for the state-dependence of the cut law and is what ensures that $\Pi_\ell$ is the stationary distribution of the resulting chain.
\end{remark}

\neuripsvspace{-5mm}
\subsection{Irreducibility and Aperiodicity of the Stagewise MH Chain}
\label{sec:appendix:irreducibility-aperiodicity}
\label{sec:appendix:mh-assumptions-hold}

In this section, we recall the standard conditions under which the stagewise Metropolis--Hastings chain converges to its target distribution, and explain why they hold in our setting.

Fix a stage length $\ell$, let
\[
    \Omega_\ell \coloneqq \inbrace{\vec{x} : \Pi_\ell(\vec{x}) > 0}
\]
be the support of the target distribution at that stage, and let $P_\ell$ denote one step of the corresponding MH chain.
We adopt the convention that the earliest admissible cut redraws the entire continuation; with the indexing $x_{0:\ell}$ used in the main text, this corresponds to allowing $m=0$.
\begin{definition}[Irreducible]
    An MH chain $P_\ell$ is said to be \emph{irreducible} on $\Omega_\ell$ if for every $\vec{x},\vec{y}\in\Omega_\ell$, there exists $r\ge 1$ such that $P_\ell^r(\vec{x},\vec{y}) > 0.$
\end{definition}
\neuripsvspace{-7mm}
\begin{definition}[Aperiodic]
    An MH chain $P_\ell$ is said to be \emph{aperiodic} if for every $\vec{x}\in\Omega_\ell$,
    $\gcd \inbrace{r \ge 1 : P_\ell^r(\vec{x},\vec{x}) > 0}$ is 1.
\end{definition}
Since $\Omega_\ell$ is finite and the Metropolis--Hastings construction makes $P_\ell$ reversible with stationary distribution $\Pi_\ell$, irreducibility and aperiodicity imply that $\tv{P_\ell^n(\vec{x},\cdot)}{\Pi_\ell} \to 0$ for every $\vec{x}\in\Omega_\ell$ as $n\to\infty$; see, \eg{}, \citet{levin2017markov}.

It remains to explain why these conditions hold in our setting. Assume first that the proposal model $p_{\mathrm{prop}}$ has full support on $\Omega_\ell$, in the sense that whenever $\vec{y}\in\Omega_\ell$, every suffix of $\vec{y}$ has positive probability under $p_{\mathrm{prop}}$ conditional on its preceding prefix. Assume also that the cut law has full support on the admissible cut locations, \ie{},
\[
    \lambda(m;\vec{x}) > 0
    \qquad
    \text{for every } \vec{x}\in\Omega_\ell
    \text{ and every admissible } m.
\]
This is immediate for the uniform-cut baseline, and for entropy-cut, it can be enforced by adding an arbitrarily small uniform floor, \eg{}, $\tilde{\lambda}_{\beta,\varepsilon}(m;\vec{x})
    \propto
    \Delta_m(\vec{x})^\beta + \varepsilon,$ ($\eps>0$).
Under these two support assumptions, irreducibility and aperiodicity are immediate. For irreducibility, fix any $\vec{x},\vec{y}\in\Omega_\ell$. With positive probability, the chain chooses the earliest cut and then proposes $\vec{y}$ exactly by redrawing the entire continuation from $p_{\mathrm{prop}}$; by the support assumption, this proposal probability is strictly positive. The reverse proposal from $\vec{y}$ back to $\vec{x}$ is also positive, and both $\Pi_\ell(\vec{x})$ and $\Pi_\ell(\vec{y})$ are positive on $\Omega_\ell$, so the MH acceptance probability is strictly positive as well. Hence $P_\ell(\vec{x},\vec{y}) > 0$, so the chain is irreducible. For aperiodicity, from any state $\vec{x}\in\Omega_\ell$, there is positive probability of choosing any admissible cut and then redrawing exactly the current suffix, thereby proposing $\vec{x}$ itself. Such a proposal is accepted with probability one, so $P_\ell(\vec{x},\vec{x}) > 0$ for every $\vec{x}$. Thus the chain has a self-loop at every state and therefore period one.

\subsection{Low-Temperature Sampling Is Not Power Sampling}
\label{sec:appendix:low-temp-fails}
In this section, we show that sampling from the power distribution $\Pi_T(x_{0:T}) \propto p(x_{0:T})^\alpha$ is not equivalent to low-temperature sampling at temperature $1/\alpha$.

Concretely, let $\Phi_T$ denote the low-temperature distribution where each next-token conditional $p(\cdot \mid x_{<t})$ is independently raised to the $\alpha$-th power and renormalized. 
Both power sampling and low-temperature sampling alter the base model, but they do so in different ways: $\Phi_T$ alters and normalizes each next-token distribution locally, whereas $\Pi_T$ raises complete traces to the $\alpha$-th power and normalizes \textit{globally}. 
We show that the two distributions can be almost disjoint, even for sequences of just two tokens.
\begin{proposition}
\label{prop:low-temp-fails}
Fix any $\alpha > 1$ and any $\eps > 0$. There exists an autoregressive distribution $p$ over two-token strings such that the power distribution $\Pi_1$ and the low-temperature distribution $\Phi_1$ satisfy
\[
    \tv{\Pi_1}{\Phi_1}
    \ge 1 - \eps \,.
\]
\end{proposition}
The construction is simple. One first-token choice has smaller marginal probability under the base model but leads to a single high-probability completion. The other has larger marginal probability but spreads its mass over many low-probability completions. Low-temperature sampling favors the second choice because it sharpens marginal next-token probabilities locally; the power distribution favors the first because it sharpens complete traces, rewarding the concentrated completion.

\begin{proof}
Let $\delta = \eps/2$. Choose $R > 1$ large enough that
$\frac{1}{1+R^\alpha} \le \delta,$
and then choose $N$ large enough that
$\frac{R^\alpha}{N^{\alpha-1}} \le \delta.$
Consider two-token strings whose first token is either $a$ or $b$. Let the second-token alphabet contain a special token $\star$ and tokens $1,\ldots,N$. Define the base distribution $p$ by
\[
    p(a,\star) = \frac{1}{1+R}
    \qquadand 
    p(b,i) = \frac{R}{N(1+R)}
    \quad \text{for } i=1,\ldots,N\,,
\]
and assign probability zero to all other two-token strings. Equivalently, the base model chooses $a$ with probability $1/(1+R)$, chooses $b$ with probability $R/(1+R)$, completes $a$ deterministically with $\star$, and completes $b$ uniformly among $N$ choices.

Let $A$ be the event that the first token is $a$. Under low-temperature sampling, the first-token conditional is sharpened locally, so
\[
    \Phi_1(A)
    =
    \frac{\left(\nfrac{1}{1+R}\right)^\alpha}
    {
    \left(\nfrac{1}{1+R}\right)^\alpha
    +
    \left(\nfrac{R}{1+R}\right)^\alpha
    }
    =
    \frac{1}{1+R^\alpha}
    \le \delta \,.
\]
Under the power distribution, however, the probability of $A$ is obtained by sharpening the probabilities of complete two-token strings:
\[
    \Pi_1(A)
    =
    \frac{
    \left(\nfrac{1}{1+R}\right)^\alpha
    }
    {
    \left(\nfrac{1}{1+R}\right)^\alpha
    +
    N\left(\nfrac{R}{(N+NR)}\right)^\alpha
    }
    =
    \frac{1}{1 + R^\alpha/N^{\alpha-1}}
    \ge \frac{1}{1+\delta}
    \ge 1-\delta \,.
\]
Therefore,
\[
    \tv{\Pi_1}{\Phi_1}
    \ge
    \bigl| \Pi_1(A) - \Phi_1(A) \bigr|
    \ge
    1 - 2\delta
    =
    1-\eps\,.\qedhere{}
\] 
\end{proof}

\subsection{Algorithm: Entropy-Cut Stagewise Metropolis--Hastings}
In this section, we present the pseudocode for our algorithm.

Our algorithm is the special case of the meta-algorithm in \cref{alg:meta} obtained by taking the cut distribution to be the entropy-cut distribution $\lambda_\beta$ from \cref{def:entropy-cut-law}. Thus, instead of treating every token position as equally likely to be revised, the sampler places more mass on positions where the model's uncertainty has just increased, as measured by the positive entropy jump $\Delta_t$. In this way, cut points are biased toward likely decision points in the reasoning trace.

Unlike the uniform-cut case, the cut distribution now depends on the current state, so the cut-law ratio does not cancel from the Metropolis--Hastings acceptance probability. This correction is precisely what ensures that the resulting chain still targets the same power distribution $\Pi_T$. The resulting sampler is shown in \cref{alg:samp}. As before, we keep the proposal model $p_{\mathrm{prop}}$ explicit below; in our experiments, it is instantiated as a low-temperature version of the base model.

\begin{figure}[ht!]
\begin{minipage}{\linewidth}
\small
\begin{algorithm2e}[H]
\DontPrintSemicolon
\LinesNumbered
\caption{Entropy-Cut Stagewise Metropolis--Hastings}
\label{alg:samp}

\SetKwInOut{Input}{Input}
\SetKwInOut{Hyper}{Hyperparameters}
\SetKwInOut{Output}{Output}

\Input{\mbox{base model $p$; proposal model $p_{\mathrm{prop}}$; power $\alpha > 1$; cut power $\beta \ge 0$; length $T$}}
\Hyper{block size $B$; number of MCMC steps per stage $N_{\mathrm{MCMC}}$}
\Output{continuation $x_{0:T}$ approximately distributed according to $\Pi_T$}

\BlankLine
Set $K \gets \lceil T/B \rceil$, $T_0 \gets 0$, and $x^{(0)}$ to the empty continuation\;

\vspace{2mm}

\For{each stage $k \gets 1,2,\dots,K$}{
    Set the $k$-th stage's target length $T_k \gets \min\inbrace{kB,\,T}$\;

    Sample a suffix $s \sim p_{\mathrm{prop}}(\cdot \mid x^{(k-1)})$ of length $T_k - T_{k-1}$\;
    Set the current state $x \gets \sinparen{x^{(k-1)}, s}$\;

    \vspace{4mm}

    \tcpCustom{Metropolis--Hastings chain targeting the power distribution for length-$T_k$}
    \For{each MCMC step $n \gets 1,2,\dots,N_{\mathrm{MCMC}}$}{

        Compute entropy profile $h_t(x)
            \gets
            -\sum_{v\in V} p(v\mid x_{<t}) \log p(v\mid x_{<t})$
            for $t=0,1,\dots,T_k$

        Compute entropy jumps $\Delta_t(x)
            \gets
            \max\!\inbrace{0,\; h_t(x)-h_{t-1}(x)}$
            for $t=1,2,\dots,T_k$

	        \eIf{$\sum_{r=1}^{T_k} \Delta_r(x)^\beta = 0$}{
	            Sample a cut point $m$ uniformly from $\inbrace{1,\dots,T_k}$\;
	        }{
	            Sample a cut point $m \sim \lambda_\beta(\cdot\,;x)$\;
	        }

	        Sample a (new) suffix $s \sim p_{\mathrm{prop}}(\cdot \mid x_{0:m-1})$ of length $T_k - m$\;
	        Set the proposed state as $x^{\mathrm{prop}} \gets \inparen{x_{0:m-1}, s}$\;

        Compute the acceptance probability
        \[
            A_m\!\inparen{x, x^{\mathrm{prop}}}
            \gets
            \min\inbrace{
                1,\,
                \frac{\Pi_{T_k}\inparen{x^{\mathrm{prop}}}}{\Pi_{T_k}\inparen{x}}
                \cdot
                \frac{\lambda_\beta\inparen{m; x^{\mathrm{prop}}}}{\lambda_\beta\inparen{m; x}}
                \cdot
                \frac{p_{\mathrm{prop}}\inparen{x_{m:T_k} \mid x^{\mathrm{prop}}_{0:m-1}}}
                     {p_{\mathrm{prop}}\inparen{x^{\mathrm{prop}}_{m:T_k} \mid x_{0:m-1}}}
            }
        \]

        \mbox{With probability $A_m\!\inparen{x, x^{\mathrm{prop}}}$, \textit{accept} the proposal and set $x \gets x^{\mathrm{prop}}$\phantom{.............................}}
        Otherwise, \textit{reject} the proposal and leave $x$ unchanged\;
    }

    \vspace{2mm}

    Set the $k$-th stage's output as $x^{(k)} \gets x$\;
}

\Return{$x^{(K)}$}\;
\end{algorithm2e}
\end{minipage}
\end{figure}

\subsection{Algorithm: Uniform-Cut Stagewise Metropolis--Hastings}
In this section, we present the pseudocode for the sampler of \citet{karan2026reasoning}.

The sampler of \citet{karan2026reasoning} is the special case of the meta-algorithm in \cref{alg:meta} obtained by taking the cut distribution to be uniform:
\[
	    \lambda_{\mathrm{unif}}(m;\vec{x})
	    \;\equiv\;
	    \frac{1}{\ell}
	    \qquad \text{for each } m \in \inbrace{1,\dots,\ell}.
	\]
Thus every token position is treated as equally likely to be revised.
Under this choice, the cut-law ratio in the Metropolis--Hastings acceptance probability is identically one, so the general meta-algorithm simplifies to the uniform-cut sampler shown in \cref{alg:karan-du}.
In the implementation of \citet{karan2026reasoning}, the proposal model $p_{\mathrm{prop}}$ is taken to be a low-temperature version of the base model, but we keep $p_{\mathrm{prop}}$ explicit below.

\begin{figure}[ht!]
\small 
\begin{minipage}{\linewidth}
\begin{algorithm2e}[H]
\DontPrintSemicolon
\LinesNumbered
\caption{Uniform-Cut Stagewise Metropolis--Hastings}
\label{alg:karan-du}

\SetKwInOut{Input}{Input}
\SetKwInOut{Hyper}{Hyperparameters}
\SetKwInOut{Output}{Output}

\Input{\mbox{base model $p$; proposal model $p_{\mathrm{prop}}$; power $\alpha > 1$; length $T$}}
\Hyper{block size $B$; number of MCMC steps per stage $N_{\mathrm{MCMC}}$}
\Output{continuation $x_{0:T}$ approximately distributed according to $\Pi_T$}

\BlankLine
Set $K \gets \lceil T/B \rceil$, $T_0 \gets 0$, and $x^{(0)}$ to the empty continuation\;

\vspace{2mm}

\For{each stage $k \gets 1,2,\dots,K$}{
    Set the $k$-th stage's target length $T_k \gets \min\inbrace{kB,\,T}$\;

    Sample a suffix $s \sim p_{\mathrm{prop}}(\cdot \mid x^{(k-1)})$ of length $T_k - T_{k-1}$\;
    Set the current state $x \gets \sinparen{x^{(k-1)}, s}$\;

    \vspace{4mm}

    \tcpCustom{Metropolis--Hastings chain targeting the power distribution for length-$T_k$}
	    \For{each MCMC step $n \gets 1,2,\dots,N_{\mathrm{MCMC}}$}{
	        Sample a cut point $m$ uniformly from $\inbrace{1,\dots,T_k}$\;

	        Sample a (new) suffix $s \sim p_{\mathrm{prop}}(\cdot \mid x_{0:m-1})$ of length $T_k - m$\;
	        Set the proposed state as $x^{\mathrm{prop}} \gets \inparen{x_{0:m-1}, s}$\;

        Compute the acceptance probability
        \[
            A_m\!\inparen{x, x^{\mathrm{prop}}}
            \gets
            \min\inbrace{
                1,\,
                \frac{\Pi_{T_k}\inparen{x^{\mathrm{prop}}}}{\Pi_{T_k}\inparen{x}}
                \cdot
                \frac{
                    p_{\mathrm{prop}}\inparen{x_{m:T_k} \mid x^{\mathrm{prop}}_{0:m-1}}
                }{
                    p_{\mathrm{prop}}\inparen{x^{\mathrm{prop}}_{m:T_k} \mid x_{0:m-1}}
                }
            }
        \]

        \mbox{With probability $A_m\!\inparen{x, x^{\mathrm{prop}}}$, \textit{accept} the proposal and set $x \gets x^{\mathrm{prop}}$\phantom{.............................}}
        Otherwise, \textit{reject} the proposal and leave $x$ unchanged\;
    }

    \vspace{2mm}

    Set the $k$-th stage's output as $x^{(k)} \gets x$\;
}

\Return{$x^{(K)}$}\;
\end{algorithm2e}
\end{minipage}
\end{figure}

\subsection{Proof of \cref{thm:approx-symmetric-mixing-separation}}
\label{sec:proof-approx-symmetric-mixing-separation}
    In this section, we prove \cref{thm:approx-symmetric-mixing-separation}.
\begin{proof}[Proof of \cref{thm:approx-symmetric-mixing-separation}]
We divide the proof into three steps.
First, we reduce the chain to the space of branch choices and define the relevant conductance quantities.
Second, we show that entropy-cut has large conductance, while uniform-cut has a low-conductance bottleneck at the first branch.
Third, we convert these conductance estimates into the claimed mixing-time bounds.

\paragraph{Step A (Reduction to Branch Choices and Conductance Notation).}
By Item~4 of \cref{cond:approx-symmetric-reasoning-tree}, all subtrees below the children of any branch node are isomorphic.
This implies that every root-to-leaf path has branch nodes at the same depths
\[
    1 \leq b_1 < b_2 < \cdots < b_k \leq T\,,
\]
and that the $j$-th branch has the same number of children, say $d_j\geq 2$, on every path.
Indeed, before the first branch node there is only one possible path, so the first branch depth $b_1$ is common to all leaves.
At that branch, the child subtrees are isomorphic, so the pattern of future branch depths and the number of children at each future branch are identical in every child subtree.
Repeating this argument inductively gives common branch depths $b_1,\dots,b_k$ and common branching factors $d_1,\dots,d_k$.

Since every non-branch node has only one child, a leaf is uniquely determined by its sequence of branch choices.
Thus, we identify the state space with
\[
    \Omega
    \coloneqq
    [d_1]\times [d_2]\times \cdots \times [d_k]\,.
\]
We write $r=(r_1,\dots,r_k)\in\Omega$ for a branch-choice sequence, and we write $\Pi_T(r)$ for the target mass of the corresponding leaf.
Let $x(r)$ denote the root-to-leaf token sequence corresponding to $r$.
Since all tokens before the first branch are deterministic, the prefix $x(r)_{0:b_1-1}$ is the same for every $r$; denote this common prefix by $x^\circ_{0:b_1-1}$.

For a cut distribution $\lambda$, let $P_\lambda$ denote the one-step MH transition kernel on $\Omega$.
For any set $A\subseteq \Omega$ with $0<\Pi_T(A)\leq 1/2$, define its conductance under $\lambda$ by
\[
    \Phi_\lambda(A)
    \coloneqq
    \Pr_\lambda\!\inparen{X_1\in A^c
    \mid X_0\sim \Pi_T(\cdot\mid A)}
    =
    \frac{1}{\Pi_T(A)}
    \sum_{r\in A}\Pi_T(r)P_\lambda(r,A^c)\,.
    \yesnum\label{eq:proof-local-conductance}
\]
The global conductance is
\[
    \Phi_\lambda
    \coloneqq
    \min_{0<\Pi_T(A)\leq 1/2}\Phi_\lambda(A)\,.
\]

We will use the following proposal-imbalance quantity for the proposal obtained by cutting at the first branch and resampling the suffix from $p_{\mathrm{prop}}$:
\[
    M_1
    \coloneqq
    \sup_{r\in\Omega}
    \frac{\Pi_T(r)}
    {p_{\mathrm{prop}}\inparen{x(r)_{b_1:T}\mid x^\circ_{0:b_1-1}}}\,.
    \yesnum\label{eq:def-M-first-branch-proposal}
\]
This definition uses $p_{\mathrm{prop}}$ directly, rather than introducing a separate full-suffix proposal distribution.

We next bound $M_1$ using the first two items of \cref{cond:approx-symmetric-reasoning-tree}.
For any $r\in\Omega$, Item~1 gives
$\Pi_T(r)
    \leq
    e^\eta U(r).$
Since the prefix before the first branch is deterministic,
\[
    U(r)
    =
    U\inparen{x(r)_{b_1:T}\mid x^\circ_{0:b_1-1}}\,.
\]
Applying Item~2 with the cut prefix $x^\circ_{0:b_1-1}$ gives
\[
    U\inparen{x(r)_{b_1:T}\mid x^\circ_{0:b_1-1}}
    \leq
    e^\eta
    p_{\mathrm{prop}}\inparen{x(r)_{b_1:T}\mid x^\circ_{0:b_1-1}}\,.
\]
Combining the above two yields
$\Pi_T(r)
    \leq
    e^{2\eta}
    p_{\mathrm{prop}}\sinparen{x(r)_{b_1:T}\mid x^\circ_{0:b_1-1}}\,,$
and therefore
\[
    M_1 \leq e^{2\eta}\,.
    \yesnum\label{eq:M-first-branch-bound}
\]

\paragraph{Step B (Conductance Bounds).}
We now prove the conductance estimates that drive the separation.
The statement below isolates the two relevant facts: entropy-cut has uniformly large conductance because it can reopen the first semantic decision with probability $\nfrac{1}{k}$, whereas uniform-cut has a low-conductance bottleneck at the first branch because it reopens that decision only with probability $b_1/T$.

\begin{lemma}[Conductance Bounds]
\label{lem:conductance-bounds-symmetric}
Under the assumptions of \cref{thm:approx-symmetric-mixing-separation}, the entropy-cut chain satisfies
$\Phi_{\mathrm{ec}}(A)
    \geq
    \nfrac{1}{2M_1k}$
for every $A\subseteq \Omega$ with $0<\Pi_T(A)\leq 1/2$.
In particular,
\[
    \Phi_{\mathrm{ec}}
    \geq
    \frac{1}{2M_1k}
    \geq
    \frac{e^{-2\eta}}{2k}\,.
\]
Moreover, there exists a first-branch set $S\subseteq \Omega$ with $0<\Pi_T(S)\leq 1/2$ such that
$\Phi_{\mathrm{unif}}(S)
    \leq
    \nfrac{b_1}{T}.$
Consequently, for this same set $S$,
\[
    \Phi_{\mathrm{ec}}(S)
    \geq
    \frac{T}{2M_1kb_1}\,\Phi_{\mathrm{unif}}(S)
    \geq
    \frac{e^{-2\eta}T}{2kb_1}\,\Phi_{\mathrm{unif}}(S)\,.
    \yesnum\label{eq:first-branch-conductance-separation}
\]
\end{lemma}

\begin{proof}[Proof of \cref{lem:conductance-bounds-symmetric}]
We first show that entropy-cut chooses uniformly among the semantic branch positions.
By Item~3 of \cref{cond:approx-symmetric-reasoning-tree}, the positive entropy jumps occur exactly at branch nodes and have the same value $\Delta$.
Thus, for every root-to-leaf path $\vec{x}$,
\[
    \Delta_t(\vec{x}) =
    \begin{cases}
        \Delta\,, & t\in\inbrace{b_1,\dots,b_k}\,,\\
        0\,, & \text{otherwise}\,.
    \end{cases}
\]
Therefore, for any $\beta>0$,
\[
    \lambda_\beta(b_j;\vec{x})
    =
    \frac{\Delta^\beta}{\sum_{i=1}^k \Delta^\beta}
    =
    \frac{1}{k}
    \qquad
    \text{for each } j\in\inbrace{1,\dots,k}\,.
\]
In particular, entropy-cut chooses the first branch depth $b_1$ with probability $\nfrac{1}{k}$.

Now condition on the event that entropy-cut chooses $m=b_1$.
Because all tokens before the first branch are deterministic in the branch-choice representation, this cut redraws the entire branch-choice sequence by sampling
\[
    x(r')_{b_1:T}
    \sim
    p_{\mathrm{prop}}\inparen{\cdot\mid x^\circ_{0:b_1-1}}\,.
\]
Moreover, since $\lambda_\beta(b_1;\vec{x})=\nfrac{1}{k}$ for every state $\vec{x}$, the cut-distribution correction in the MH ratio cancels.
Thus, conditional on choosing $m=b_1$, the transition is the independence Metropolis--Hastings kernel with target $\Pi_T$ and proposal mass
$p_{\mathrm{prop}}\inparen{x(r')_{b_1:T}\mid x^\circ_{0:b_1-1}}.$
We claim that this conditional kernel satisfies the minorization
\[
    K_{b_1}(r,\cdot)
    \geq
    \frac{1}{M_1}\Pi_T(\cdot)
    \qquad
    \text{for every } r\in\Omega\,,
    \yesnum\label{eq:minorization-first-branch-kernel}
\]
where $K_{b_1}$ denotes the MH kernel conditional on cutting at $b_1$.
To prove this, define
\[
    \omega(r)
    \coloneqq
    \frac{\Pi_T(r)}
    {p_{\mathrm{prop}}\inparen{x(r)_{b_1:T}\mid x^\circ_{0:b_1-1}}}\,.
\]
By definition, $\omega(r)\leq M_1$ for every $r$.
For two states $r,r'\in\Omega$, the accepted proposal mass that $K_{b_1}$ assigns to $r'$ is
\[
    p_{\mathrm{prop}}\inparen{x(r')_{b_1:T}\mid x^\circ_{0:b_1-1}}
    \min\inbrace{
        1,\,
        \frac{\Pi_T(r')\,
        p_{\mathrm{prop}}\inparen{x(r)_{b_1:T}\mid x^\circ_{0:b_1-1}}}
        {\Pi_T(r)\,
        p_{\mathrm{prop}}\inparen{x(r')_{b_1:T}\mid x^\circ_{0:b_1-1}}}
    }\,.
\]
Equivalently, this is
\[
    p_{\mathrm{prop}}\inparen{x(r')_{b_1:T}\mid x^\circ_{0:b_1-1}}
    \min\inbrace{1,\frac{\omega(r')}{\omega(r)}}\,.
\]
If $\omega(r')\leq \omega(r)$, then the accepted proposal mass equals
\[
    p_{\mathrm{prop}}\inparen{x(r')_{b_1:T}\mid x^\circ_{0:b_1-1}}
    \frac{\omega(r')}{\omega(r)}
    =
    \frac{\Pi_T(r')}{\omega(r)}
    \geq
    \frac{1}{M_1}\Pi_T(r')\,.
\]
If $\omega(r')>\omega(r)$, then the accepted proposal mass equals
\[
    p_{\mathrm{prop}}\inparen{x(r')_{b_1:T}\mid x^\circ_{0:b_1-1}}
    =
    \frac{\Pi_T(r')}{\omega(r')}
    \geq
    \frac{1}{M_1}\Pi_T(r')\,.
\]
Thus the accepted proposal mass dominates $M_1^{-1}\Pi_T(r')$ for every $r'$.
Rejections only add mass to the current state, so \eqref{eq:minorization-first-branch-kernel} follows.
Since the full entropy-cut kernel uses $K_{b_1}$ with probability $\nfrac{1}{k}$, we obtain
\[
    P_{\mathrm{ec}}(r,\cdot)
    \geq
    \frac{1}{M_1k}\Pi_T(\cdot)
    \qquad
    \text{for every } r\in\Omega\,.
    \yesnum\label{eq:minorization-entropy-cut}
\]
Consequently, for every $A\subseteq \Omega$ with $0<\Pi_T(A)\leq 1/2$, and every $r\in A$,
\[
    P_{\mathrm{ec}}(r,A^c)
    \geq
    \frac{1}{M_1k}\Pi_T(A^c)
    \geq
    \frac{1}{2M_1k}\,.
\]
Averaging over $r\sim \Pi_T(\cdot\mid A)$ gives
$ \Phi_{\mathrm{ec}}(A)
    \geq
    \nfrac{1}{2M_1k}.$
Since this holds for every such $A$,
$\Phi_{\mathrm{ec}}
    \geq
    \nfrac{1}{2M_1k}.$
Using \eqref{eq:M-first-branch-bound}, we further get
$\Phi_{\mathrm{ec}}
    \geq
    \nfrac{e^{-2\eta}}{2k}.$

We next exhibit a low-conductance set for the uniform-cut chain.
For $i\in[d_1]$, let $S_i\subseteq \Omega$ be the set of paths whose first branch choice is $i$.
These sets partition $\Omega$, so there exists $i_\star\in[d_1]$ such that
$ \Pi_T(S_{i_\star})\leq \nfrac{1}{d_1}\leq \nfrac{1}{2}.$
Set $S\coloneqq S_{i_\star}$.
Then $0<\Pi_T(S)\leq \nfrac{1}{2}$.

Under the uniform-cut chain, the first branch can change only if the cut occurs at or before depth $b_1$.
This happens with probability $b_1/T$.
Any later cut preserves the first branch choice, and any rejection also leaves the state unchanged.
Therefore, for every $r\in S$,
$P_{\mathrm{unif}}(r,S^c)
    \leq
    \frac{b_1}{T}.$
Averaging over $r\sim \Pi_T(\cdot\mid S)$ gives
$ \Phi_{\mathrm{unif}}(S)
    \leq
    \frac{b_1}{T}.$
Finally, applying the entropy-cut lower bound to this same set $S$, and using the uniform-cut upper bound, gives
\[
    \Phi_{\mathrm{ec}}(S)
    \geq
    \frac{1}{2M_1k}
    \geq
    \frac{T}{2M_1kb_1}\,\Phi_{\mathrm{unif}}(S)
    \geq
    \frac{e^{-2\eta}T}{2kb_1}\,\Phi_{\mathrm{unif}}(S)\,.\qedhere{}
\]
\end{proof}

\paragraph{Step C (From Conductance to Mixing Time).}
We now convert the above bounds into mixing-time estimates.
For entropy-cut, we use the stronger minorization \eqref{eq:minorization-entropy-cut}.
Let
$\rho\coloneqq \frac{1}{M_1k}.$
Then \eqref{eq:minorization-entropy-cut} implies that for every $r\in\Omega$,
$P_{\mathrm{ec}}(r,\cdot)
    \geq
    \rho\,\Pi_T(\cdot).$
Equivalently, each step of the entropy-cut chain has probability at least $\rho$ of refreshing exactly from stationarity.
Hence, after $n$ steps,
\[
    \sup_{r\in\Omega}
    \tv{P_{\mathrm{ec}}^n(r,\cdot)}{\Pi_T}
    \leq
    (1-\rho)^n
    \leq
    \exp(-\rho n)\,.
\]
Therefore, if
$ n
    \geq
    M_1k\log\nfrac{1}{\eps},$
then the total-variation distance is at most $\eps$.
Using $M_1\leq e^{2\eta}$ from \eqref{eq:M-first-branch-bound}, we obtain
\[
    \tau_{\mathrm{mix}}^{\mathrm{ec}}(\eps)
    \leq
    e^{2\eta} k\log\nfrac{1}{\eps}\,.
    \yesnum\label{eq:entropy-cut-mixing-upper-bound}
\]
It remains to lower bound the mixing time of the uniform-cut chain.
Start the uniform-cut chain from any state $r^{(0)}\in S$, where $S$ is the first-branch set constructed in the proof of \cref{lem:conductance-bounds-symmetric}.
Until the first branch is reopened, the chain must remain in $S$.
Under uniform cuts, the first branch is reopened at each step with probability $b_1/T$.
Hence,
\[
    P_{\mathrm{unif}}^n(r^{(0)},S)
    \geq
    \inparen{1-\frac{b_1}{T}}^n.
    \yesnum\label{eq:uniform-survival-in-S}
\]
Total variation distance is at least the discrepancy on the event $S$, so
\[
    \tv{P_{\mathrm{unif}}^n(r^{(0)},\cdot)}{\Pi_T}
    \geq
    P_{\mathrm{unif}}^n(r^{(0)},S)-\Pi_T(S)
    \geq
    \inparen{1-\frac{b_1}{T}}^n-\frac{1}{2}\,.
    \yesnum\label{eq:uniform-tv-lower-bound}
\]
Fix $\eps\in\inparen{0,\nfrac{1}{2}}$ and define
\[
    c_\eps
    \coloneqq
    \frac{1}{4}
    \log\frac{1}{\eps+\nfrac{1}{2}}
    >0\,.
\]
Since $b_1\leq T/2$, we have
$ -\log\inparen{1-\nfrac{b_1}{T}}
    \leq
    2\nfrac{b_1}{T}.$
Therefore, for every
$n
    \leq
    c_\eps \nfrac{T}{b_1},$
we have
\[
    \inparen{1-\frac{b_1}{T}}^n
    \geq
    \exp\inparen{-2c_\eps}
    =
    \sqrt{\eps+\nfrac{1}{2}}\,.
\]
Since $\eps<1/2$, we have
$   \sqrt{\eps+\nfrac{1}{2}}-\nfrac{1}{2}
    >
    \eps.$
Plugging this into \eqref{eq:uniform-tv-lower-bound} gives
$\tv{P_{\mathrm{unif}}^n(r^{(0)},\cdot)}{\Pi_T}
    >
    \eps$
for all $n\leq c_\eps T/b_1$.
Thus,
$\tau_{\mathrm{mix}}^{\mathrm{unif}}(\eps)
    =
    \Omega_\eps\!\inparen{\nfrac{T}{b_1}}.$
Together with \eqref{eq:entropy-cut-mixing-upper-bound}, this proves the theorem, with the entropy-cut upper bound carrying the factor $e^{2\eta}$ under the symmetry condition as stated.
\end{proof}

\section{Additional Experiments and Details}\label{apx:experiments}

\subsection{Benchmark Details}
\label{sec:appendix:benchmark-details}
In this section, we describe the datasets we use.
\begin{itemize}
    \item \textbf{MATH500}: The MATH dataset~\citep{lightman2024verify} is a set of competition math problems spanning seven categories, including geometry, number theory, and precalculus. 
    There are 12500 total problems, of which 5000 are test problems.
    MATH500 is a particular randomly chosen subset of the test set standardized by OpenAI.
    \item \textbf{HumanEval}: HumanEval~\citep{chen2021evaluating} consists of 164 handwritten programming problems covering algorithms, reasoning, mathematics, and language comprehension. Each problem comes with associated unit tests (7.7 on average) and a solution must pass all tests to be considered correct.
    \item \textbf{GPQA Diamond}: GPQA~\citep{rein2024gpqa} is a dataset of graduate-level multiple-choice science questions. %
    GPQA Diamond is a subset of 198 questions of the highest quality.
    \item \textbf{AIME26}: The AIME 2026~\citep{aime2026} dataset comprises 30 high-school olympiad-level mathematical problems sourced from the 2026 American Invitational Mathematics Examination (AIME I and AIME II). The dataset spans algebra, geometry, number theory, and combinatorics, with each problem requiring multi-step logical deduction.
\end{itemize}

\subsection{Experimental Details}\label{apx:experiments:details}
Here we detail the hyperparameters for the baselines.

\paragraph{SMC Hyperparameters.}
The most important hyperparameter for SMC is the number of particles $N$, which \citet{azizi2026power} sets to be $N=64$.
All other hyperparameters, including the generation length $T$ and power exponent $\alpha$, follow the entropy-cut MH settings (\cref{sec:experiments:setup}).

\paragraph{TMC Hyperparameters.}
TMC has a few more hyperparameters, which we set following \citet{ji2026scalable}.
Specifically, we take the block-length (look-ahead length) to be $B=192$, the number of candidate sequences $K = 8$, and the number of look-ahead completions per candidate to be $M=8$.
Similar to SMC, all other hyperparameters are not specific to TMC, and we take them to be the same as entropy-cut MH (\Cref{sec:experiments:setup}).

{\subsection{Statistical Significance}\label{apx:experiments:significance}
To assess the statistical significance of our results, we use a clustered McNemar test.
Throughout these tests, we fix the models to be \texttt{Qwen2.5-Math-7B} and \texttt{Qwen3-8B-Base}, the datasets to be MATH500, HumanEval, and AIME26, and all default hyperparameter settings described in \Cref{sec:experiments:setup}, while comparing entropy-cut MH against standard sampling, low-temperature sampling, and uniform-cut MH.
As in \Cref{sec:experiments:setup}, we use 8 independent runs for MATH500 and HumanEval and 64 runs for AIME26.}

{\paragraph{Clustered McNemar Test.}
For each question and run, entropy-cut MH and a baseline yield a pair of binary outcomes (correct or incorrect).
For each question $i$, let $b_i$ be the number of runs in which entropy-cut MH is correct and the baseline is incorrect, and $c_i$ the reverse.
The standard McNemar test~\citep{mcnemar1947note} treats all pairs as independent, but outcomes on the same question are correlated across runs.
We therefore use the clustered McNemar test~\citep{durkalski2003analysis,yang2010note}, which treats the runs of each question as a cluster, allowing dependence within a question while assuming independence across questions.
This yields 500, 164, and 30 clusters for MATH500, HumanEval, and AIME26, respectively.
The net discordance $a_i$ of question $i$ and the statistic $z$, whose square is the statistic of \citet{durkalski2003analysis}, are defined respectively as
\[
    \textstyle
    a_i = b_i - c_i\,,
    \qquad
    z = \sum_i a_i \big/ \sqrt{\sum_i a_i^2}\,.
\]
Under the null hypothesis that the two samplers are equally accurate, we compare $z$ against a standard normal distribution and report the two-sided $p$-value.
The denominator estimates the standard error of $\sum_i a_i$ from the variation between questions rather than assuming independence within a question.
We do not correct for multiple comparisons.}

{\paragraph{Results.}
We report the scores and $p$-values in \Cref{tab:significance}.
The scores are computed on the paired runs used for the tests and can differ from those in \Cref{tab:main}.
}
{At the $0.05$ level, we observe that entropy-cut MH significantly improves over standard sampling in all settings except \texttt{Qwen3-8B-Base} on AIME26, and over low-temperature sampling on MATH500 and HumanEval for both models.}

\begin{table}[tbh!]
{\centering
\vspace{0pt}
\small
\setlength{\tabcolsep}{4pt}
\renewcommand{\arraystretch}{0.9}
\caption{\textbf{Entropy-Cut MH (ours) improves significantly over standard and low-temperature sampling on MATH500 and HumanEval.}
We report \texttt{Qwen2.5-Math-7B} and \texttt{Qwen3-8B-Base} scores on MATH500, HumanEval, and AIME26, together with the clustered McNemar $p$-values of entropy-cut MH against each baseline.
All scores are percentages.
We underline our scores (last row per block) and bold $p$-values below $0.05$, all of which favor entropy-cut MH.}
\label{tab:significance}
\begin{tabular}{@{}lllcc@{}}
\toprule
Model & Dataset & Method & Score & $p$-value \\
\midrule
\multirow{12}{*}{\texttt{\small Qwen2.5-Math-7B}}
  & \multirow{4}{*}{MATH500}   & Standard              & 48.5 & $\boldsymbol{3.4\times10^{-58}}$ \\
  &                            & Low-Temperature       & 68.2 & $\boldsymbol{6.0\times10^{-21}}$ \\
  &                            & Uniform-Cut MH        & 72.1 & $\boldsymbol{7.3\times10^{-12}}$ \\
  &                            & Entropy-Cut MH (ours) & \underline{78.3} & \\
\cmidrule(l){2-5}
  & \multirow{4}{*}{HumanEval} & Standard              & 29.0 & $\boldsymbol{1.8\times10^{-18}}$ \\
  &                            & Low-Temperature       & 55.4 & $\boldsymbol{0.0344}$ \\
  &                            & Uniform-Cut MH        & 60.3 & $0.2337$ \\
  &                            & Entropy-Cut MH (ours) & \underline{58.7} & \\
\cmidrule(l){2-5}
  & \multirow{4}{*}{AIME26}    & Standard              & 4.8  & $\boldsymbol{0.0401}$ \\
  &                            & Low-Temperature       & 6.4  & $0.0845$ \\
  &                            & Uniform-Cut MH        & 9.5  & $0.1898$ \\
  &                            & Entropy-Cut MH (ours) & \underline{13.1} & \\
\midrule
\multirow{12}{*}{\texttt{\small Qwen3-8B-Base}}
  & \multirow{4}{*}{MATH500}   & Standard              & 58.5 & $\boldsymbol{2.9\times10^{-50}}$ \\
  &                            & Low-Temperature       & 77.3 & $\boldsymbol{9.4\times10^{-5}}$ \\
  &                            & Uniform-Cut MH        & 79.5 & $0.1945$ \\
  &                            & Entropy-Cut MH (ours) & \underline{80.2} & \\
\cmidrule(l){2-5}
  & \multirow{4}{*}{HumanEval} & Standard              & 47.6 & $\boldsymbol{1.8\times10^{-20}}$ \\
  &                            & Low-Temperature       & 73.7 & $\boldsymbol{4.6\times10^{-5}}$ \\
  &                            & Uniform-Cut MH        & 75.8 & $\boldsymbol{0.0029}$ \\
  &                            & Entropy-Cut MH (ours) & \underline{79.3} & \\
\cmidrule(l){2-5}
  & \multirow{4}{*}{AIME26}    & Standard              & 4.1  & $0.0957$ \\
  &                            & Low-Temperature       & 7.5  & $0.4904$ \\
  &                            & Uniform-Cut MH        & 8.4  & $0.7518$ \\
  &                            & Entropy-Cut MH (ours) & \underline{8.7} & \\
\bottomrule
\end{tabular}
\par}
\vspace{0pt}
\end{table}

\subsection{Likelihood and Confidence}\label{apx:experiments:likelihood-confidence}
The sequence log-likelihood and average \emph{confidence}~\citep{prabhudesai2025maximizing} of a sequence with respect to a model are defined respectively as
\[
    \textstyle
    \operatorname{LL}(x_{0:T}) = \sum_{t=0}^T \log p(x_t\mid x_{<t})\,,
    \qquad
    \operatorname{Conf}(x_{0:T}) = \frac1{T+1} \sum_{t=0}^T \sum_{v\in V} p(v\mid x_{<t})\log p(v\mid x_{<t})\,.
\]
We plot the distribution of the sequence log-likelihood values computed via the base model (\texttt{Qwen2.5-Math-7B}) on the MATH500 dataset across 8 random seeds in \Cref{fig:SequenceLogprobs}.
Similarly, we plot the average confidence in \Cref{fig:AverageConfidence}.
In addition, the 25th, 50th (median), and 75th percentile values are marked with dashed, solid, and dotted vertical lines.
We observe that entropy-cut MH samples from higher-likelihood regions than standard, low-temperature, and uniform-cut MH, and from regions similar to those sampled by SMC and TMC.

\begin{figure}[htbp]
    \centering
    \includegraphics[width=\linewidth]{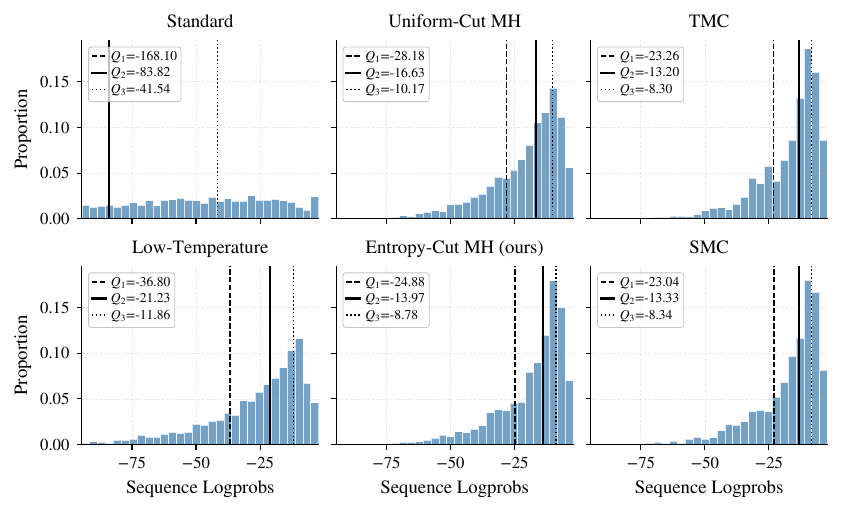}
    \caption{\textbf{Entropy-Cut MH (ours) samples from high-probability regions.} We plot \texttt{Qwen2.5-Math-7B} log-likelihoods for MATH500 responses across samplers.}
    \label{fig:SequenceLogprobs}
\end{figure}

\begin{figure}[hbtp]
    \centering
    \includegraphics[width=\linewidth]{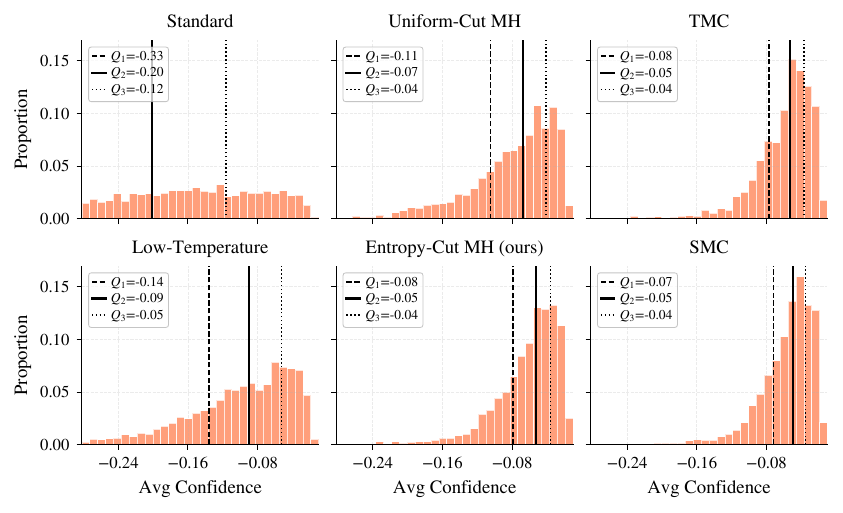}
    \caption{\textbf{Entropy-Cut MH (ours) samples from high-confidence regions.} We plot \texttt{Qwen2.5-Math-7B} average confidence values for MATH500 responses across samplers.}
    \label{fig:AverageConfidence}
\end{figure}

\subsection{Ablations}\label{apx:experiments:ablations}
We demonstrate the stability of our algorithm against the specific choices of hyperparameters.
Throughout these ablations, we fix the model to be \texttt{Qwen2.5-Math-7B}, the MATH500 dataset, and all default hyperparameter settings described in \Cref{sec:experiments:setup}, while varying the parameter in question.
We report the average score across 8 repetitions.

\paragraph{Power Distribution Exponent ($\alpha$).}
We plot the MATH500 scores across varying values of $\alpha$ in \Cref{fig:AblationAlpha}.
At the extreme $\alpha=1$, this is equivalent to the standard sampling algorithm.
As noted by \citet{karan2026reasoning}, a higher $\alpha$ is correlated with better reasoning, but directly optimizing for likelihood is not necessarily optimal for sampling.
Nonetheless, we see that the accuracy of entropy-cut MH remains relatively stable beyond $\alpha\geq 2.0$.

\begin{figure}[tb!]
    \centering
    \includegraphics[width=0.4\linewidth]{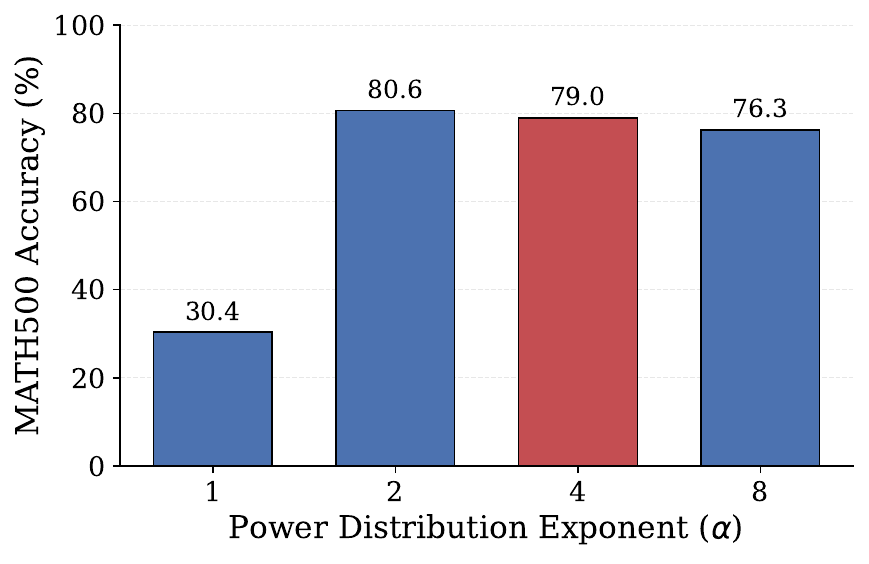}
    \caption{\textbf{Entropy-Cut MH is stable to the choice of power exponent $\alpha$.} We plot \texttt{Qwen2.5-Math-7B} scores for MATH500 across various values of $\alpha$ generated by entropy-cut MH.}
    \label{fig:AblationAlpha}
\end{figure}

\paragraph{Cut Law Exponent ($\beta$).}
Similarly, \Cref{fig:AblationBeta} plots the MATH500 scores across various values of $\beta$.
The accuracy seems similarly stable for values of $\beta\geq 2.0$.

\begin{figure}[tb!]
    \centering
    \includegraphics[width=0.4\linewidth]{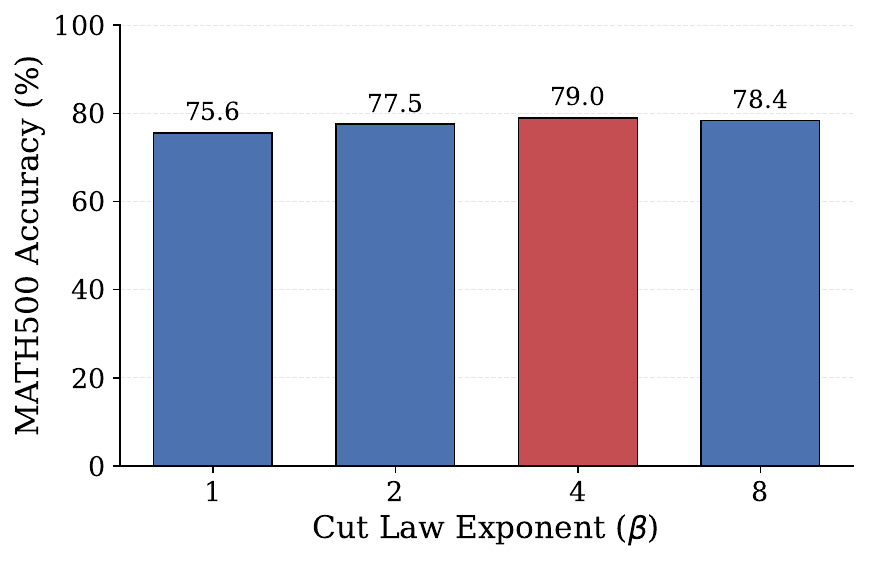}
    \caption{\textbf{Entropy-Cut MH is stable to the choice of cut law exponent $\beta$.} We plot \texttt{Qwen2.5-Math-7B} scores for MATH500 across various values of $\beta$ generated by entropy-cut MH.}
    \label{fig:AblationBeta}
\end{figure}

\paragraph{Metropolis--Hastings Steps ($N_{\mathrm{MCMC}}$).}
In \Cref{fig:AblationMHSteps}, we plot the MATH500 scores across different numbers of $N_{\mathrm{MCMC}}$.
While increasing the number of steps increases the performance as expected, the increase past $10$ steps seems to be relatively minor.

\begin{figure}[tb!]
    \centering
    \includegraphics[width=0.4\linewidth]{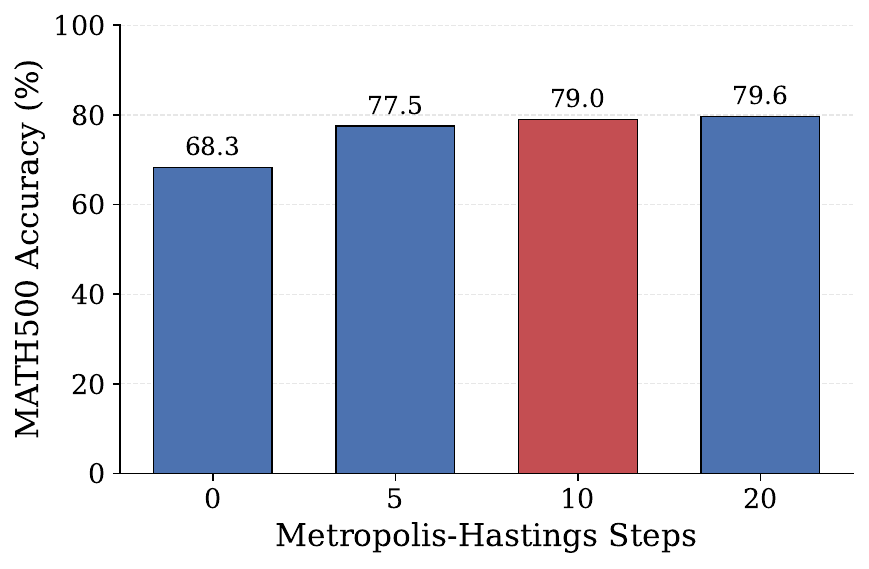}
    \caption{\textbf{Entropy-Cut MH improves with additional transition steps $N_{\mathrm{MCMC}}$.} We plot \texttt{Qwen2.5-Math-7B} scores for MATH500 across various values of $N_{\mathrm{MCMC}}$ generated by entropy-cut MH.}
    \label{fig:AblationMHSteps}
\end{figure}

\subsection{Running Times}
To demonstrate the scalability of our method, we report the average running time to generate the solution to a MATH500 question using \texttt{Qwen2.5-Math-7B} for the various (power) samplers in \Cref{fig:RunningTime}.
As expected, standard and low-temperature sampling are an order of magnitude faster; uniform-cut MH, entropy-cut MH, and SMC have comparable costs. TMC is the most expensive algorithm, requiring nearly double the cost of other samplers.
{See \Cref{apx:experiments:best-of-n} for a comparison of entropy-cut MH against best-of-$N$ using each entropy-cut MH run's sampled-token count as a token budget.}

\begin{figure}[tb!]
    \centering
    \includegraphics[width=0.4\linewidth]{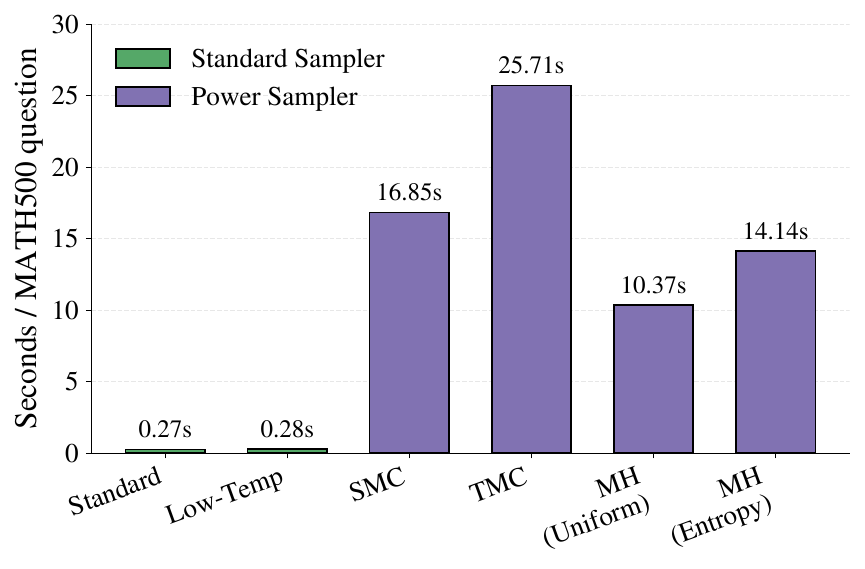}
    \caption{\textbf{Entropy-Cut MH has running times comparable to other power sampling methods.} We plot \texttt{Qwen2.5-Math-7B} average time to generate a solution per MATH500 question across various (power) sampling algorithms.}
    \label{fig:RunningTime}
\end{figure}

{\subsection{Compute-Matched Best-of-\texorpdfstring{$N$}{N}}\label{apx:experiments:best-of-n}}
{To account for the additional compute of entropy-cut MH (\Cref{fig:RunningTime}), we compare it against best-of-$N$ (BoN)~\citep{huang2025is} using each entropy-cut MH run's sampled-token count as a token budget, which we use as a proxy for compute.}
{Throughout this comparison, we fix the models to be \texttt{Qwen2.5-Math-7B} and \texttt{Qwen3-8B-Base}, the datasets to be MATH500, HumanEval, and AIME26, and all default hyperparameter settings described in \Cref{sec:experiments:setup}, while comparing entropy-cut MH against standard and low-temperature sampling and their BoN variants.
For each question, we first run entropy-cut MH and count the total number of tokens it samples, including those of rejected proposals.}
{We then draw independent samples by standard sampling (standard BoN) or by low-temperature sampling with $\tau=\nicefrac1\alpha$ (low-temperature BoN) until their total sampled-token count first reaches or exceeds this target; the crossing sample is kept.}
{As power sampling is verifier-free, we select among these samples the one with the highest average log-likelihood $\operatorname{LL}(x_{0:T})/(T+1)$ under the base model, where $T+1$ is the number of tokens of the sample $x_{0:T}$ (\Cref{apx:experiments:likelihood-confidence}).
Since the token budget is set by the entropy-cut MH run, the number of samples drawn by BoN varies per question.
We repeat this procedure 8 times for MATH500 and HumanEval and 64 times for AIME26, and compute the $p$-values as in \Cref{apx:experiments:significance}.}

{\paragraph{Results.}
We report the scores and $p$-values in \Cref{tab:best-of-n}.
The scores, including those of entropy-cut MH and of standard and low-temperature sampling, are computed on the same runs as the BoN baselines and can differ from those in \Cref{tab:main,tab:significance}.
We observe that entropy-cut MH attains a higher score than standard BoN in all settings, and the improvement is significant on MATH500 and HumanEval for both models.
Entropy-cut MH also significantly improves over low-temperature BoN in all settings except \texttt{Qwen2.5-Math-7B} on HumanEval, where low-temperature BoN attains a higher score (61.0 vs.\ 58.9), though not significantly ($p=0.1072$).
In addition, selecting among low-temperature samples by average log-likelihood can underperform a single low-temperature sample, as seen for \texttt{Qwen2.5-Math-7B} on AIME26 (4.1 vs.\ 6.2) and for \texttt{Qwen3-8B-Base} on MATH500 (74.4 vs.\ 77.6) and AIME26 (3.1 vs.\ 6.2).
This echoes our remark in \Cref{apx:experiments:ablations} that, as noted by \citet{karan2026reasoning}, directly optimizing for likelihood is not necessarily optimal for sampling.}

\begin{table}[tbh!]
{%
\centering
\vspace{0pt}
\small
\setlength{\tabcolsep}{4pt}
\renewcommand{\arraystretch}{0.9}
\caption{%
{\textbf{Entropy-Cut MH (ours) is competitive with and often improves over best-of-$N$ (BoN) using each entropy-cut MH run's sampled-token count as a stopping target.}}
{We report \texttt{Qwen2.5-Math-7B} and \texttt{Qwen3-8B-Base} scores on MATH500, HumanEval, and AIME26, together with the clustered McNemar $p$-values of entropy-cut MH against each baseline.
All scores are percentages.
We underline our scores (last row per block) and bold $p$-values below $0.05$, all of which favor entropy-cut MH.}
}
\label{tab:best-of-n}
{\begin{tabular}{@{}lllcc@{}}
\toprule
Model & Dataset & Method & Score & $p$-value \\
\midrule
\multirow{15}{*}{\texttt{\small Qwen2.5-Math-7B}}
  & \multirow{5}{*}{MATH500}   & Standard              & 48.6 & $\boldsymbol{1.7\times10^{-60}}$ \\
  &                            & Low-Temperature       & 69.5 & $\boldsymbol{4.3\times10^{-19}}$ \\
  &                            & Standard BoN          & 75.7 & $\boldsymbol{0.0004}$ \\
  &                            & Low-Temperature BoN   & 74.4 & $\boldsymbol{6.4\times10^{-8}}$ \\
  &                            & Entropy-Cut MH (ours) & \underline{78.8} & \\
\cmidrule(l){2-5}
  & \multirow{5}{*}{HumanEval} & Standard              & 32.2 & $\boldsymbol{1.7\times10^{-16}}$ \\
  &                            & Low-Temperature       & 56.5 & $0.1131$ \\
  &                            & Standard BoN          & 49.1 & $\boldsymbol{2.6\times10^{-6}}$ \\
  &                            & Low-Temperature BoN   & 61.0 & $0.1072$ \\
  &                            & Entropy-Cut MH (ours) & \underline{58.9} & \\
\cmidrule(l){2-5}
  & \multirow{5}{*}{AIME26}    & Standard              & 4.2  & $\boldsymbol{0.0307}$ \\
  &                            & Low-Temperature       & 6.2  & $0.0670$ \\
  &                            & Standard BoN          & 10.4 & $0.1563$ \\
  &                            & Low-Temperature BoN   & 4.1  & $\boldsymbol{0.0313}$ \\
  &                            & Entropy-Cut MH (ours) & \underline{13.3} & \\
\midrule
\multirow{15}{*}{\texttt{\small Qwen3-8B-Base}}
  & \multirow{5}{*}{MATH500}   & Standard              & 57.4 & $\boldsymbol{6.8\times10^{-50}}$ \\
  &                            & Low-Temperature       & 77.6 & $\boldsymbol{0.0010}$ \\
  &                            & Standard BoN          & 78.2 & $\boldsymbol{0.0225}$ \\
  &                            & Low-Temperature BoN   & 74.4 & $\boldsymbol{2.2\times10^{-9}}$ \\
  &                            & Entropy-Cut MH (ours) & \underline{80.1} & \\
\cmidrule(l){2-5}
  & \multirow{5}{*}{HumanEval} & Standard              & 47.4 & $\boldsymbol{3.9\times10^{-20}}$ \\
  &                            & Low-Temperature       & 72.3 & $\boldsymbol{0.0002}$ \\
  &                            & Standard BoN          & 68.4 & $\boldsymbol{1.7\times10^{-7}}$ \\
  &                            & Low-Temperature BoN   & 74.7 & $\boldsymbol{0.0282}$ \\
  &                            & Entropy-Cut MH (ours) & \underline{77.8} & \\
\cmidrule(l){2-5}
  & \multirow{5}{*}{AIME26}    & Standard              & 4.5  & $0.0569$ \\
  &                            & Low-Temperature       & 6.2  & $0.0793$ \\
  &                            & Standard BoN          & 9.7  & $0.9249$ \\
  &                            & Low-Temperature BoN   & 3.1  & $\boldsymbol{0.0463}$ \\
  &                            & Entropy-Cut MH (ours) & \underline{9.9} & \\
\bottomrule
\end{tabular}}
\par}
\vspace{0pt}
\end{table}

{\subsection{Decision-Point Audit}\label{apx:experiments:decision-points}
In \Cref{subsec:entropy-proxy-empirical}, we find that resampling at positions with large entropy jumps $\Delta_t$ yields more divergent suffixes than resampling at positions with small ones.
However, suffix divergence can also reflect formatting or lexical variation rather than a different choice at the cut.
To examine the cut locations themselves, we conduct a blind audit in which an LLM classifier, \texttt{GPT-5.6 Sol}, labels the cut locations of proposals made by entropy-cut MH and uniform-cut MH.
Each proposal corresponds to one MCMC step of the meta-algorithm in \Cref{alg:meta}, which keeps the prefix $x_{0:m-1}$ of the current state $\vec{x}$ and resamples the suffix from the cut position $m$ onward.
Throughout this audit, we fix two model--dataset pairs: \texttt{Qwen2.5-Math-7B} on all 500 MATH500 problems and \texttt{Qwen3-8B-Base} on all 164 HumanEval problems.
On MATH500, the two methods produce 12,432 proposals in total, 6,231 of them from entropy-cut MH, and we label 1,000 randomly sampled proposals from each method.
On HumanEval, we label all 714 entropy-cut and 771 uniform-cut proposals without subsampling.
The classifier sees only the reasoning or code of the current state $\vec{x}$ around the cut, while the sampling method, the proposed continuation, the entropy values, the acceptance decision, and the final outcome are hidden.
A cut is labeled a decision point when, at the cut, the solution could make either a different high-level choice (proof method, algorithm, data structure, or case split) or a different substantive intermediate choice, such as a mathematical subgoal, equation, predicate, state update, or constraint.
We report 95\% confidence intervals (CIs) for the difference in decision-point rates between the two methods, computed with a bootstrap clustered by problem.
We also rank the entropy-cut proposals by the entropy jump $\Delta_m(\vec{x})$ at their cut position $m$ and split them into four groups; on MATH500, the group boundaries are computed on all 6,231 entropy-cut proposals and then applied to the 1,000 labeled entropy-cut proposals.}

\paragraph{Results.}
We observe that entropy-cut MH cuts at locations labeled as decision points substantially more often than uniform-cut MH in both settings.
On MATH500, the decision-point rate is 19.6\% (196 of 1,000) for entropy-cut MH vs.\ 5.6\% (56 of 1,000) for uniform-cut MH, a difference of 14.0 percentage points (95\% CI $+10.8$ to $+17.3$), so entropy-cut MH cuts at decision points at $3.5\times$ the rate of uniform-cut MH.
On HumanEval, the rates are 29.0\% (207 of 714) vs.\ 9.3\% (72 of 771), a difference of 19.7 percentage points (95\% CI $+14.7$ to $+24.5$), and entropy-cut MH again cuts at decision points at about $3\times$ the rate of uniform-cut MH.
Within entropy-cut MH, the decision-point rate also increases across the four groups: from the lowest- to the highest-jump group, it is 10.1\%, 15.5\%, 20.3\%, and 33.6\% on MATH500, and 11.8\%, 15.1\%, 38.5\%, and 50.6\% on HumanEval.
Thus, in both settings, a cut in the highest-jump group is more than three times as likely to land at a decision point as a cut in the lowest-jump group.
Nevertheless, most cuts of either method are not labeled decision points, so entropy jumps make a cut at a decision point more likely rather than isolating decision points.
Moreover, the labels come from a single LLM classifier and are themselves an imperfect measure of decision points.

\subsection{Prompts}
For each dataset except HumanEval, we have two sets of prompts.
We use one for the Qwen family of base models (\texttt{Qwen2.5-7B}, \texttt{Qwen2.5-Math-7B}, \texttt{Qwen3-8B-Base}), which does not use a chat template, and another for the Phi family of instruction-tuned models (\texttt{Phi-3.5-mini-instruct}, \texttt{Phi-4-mini-instruct}), which uses the respective chat templates.

\subsubsection{MATH500}
\paragraph{Qwen Family.}
We use the following prompt template where \texttt{\{Question\}} is replaced with the specific MATH500 question.
\begin{lstlisting}
Can you solve the following math problem? Please reason step by step, and put your final answer within \\boxed{{}}. 

{Question}

Remember to present your final answer within \\boxed{{}}!"
\end{lstlisting}

\paragraph{Phi Family.}
We format the following messages using the respective model chat templates where \texttt{\{Qwen Prompt\}} is replaced with the prompt above.
\begin{lstlisting}
[
    {"role": "system", "content": "You are an AI math expert."},
    {"role": "user", "content": {Qwen Prompt}},
]
\end{lstlisting}

\subsubsection{HumanEval}
For HumanEval, we {provide} the code stub and enforce the following stop words.
\begin{lstlisting}
[
    "\nclass",
    "\ndef",
    "\n#",
    "\nif",
    "\nprint",
    "\nassert",
    "\nimport",
    "\nfrom",
    "\n```",
    "if __name__",
]
\end{lstlisting}

\subsubsection{GPQA Diamond}
\paragraph{Qwen Family.}
For fairness, we randomly permute each GPQA Diamond multiple-choice answer among the choices and use the following prompt template.
Here \texttt{\{Question\}} is replaced with the specific GPQA Diamond question and \texttt{\{A\}}, \texttt{\{B\}}, \texttt{\{C\}}, and \texttt{\{D\}} are replaced with the actual multiple-choice responses.
\begin{lstlisting}
Answer the following multiple-choice question. The last line of your response should be of the following format: '\\boxed{{$LETTER}}' (without quotes) where LETTER is one of ABCD. Think step by step before answering.

{Question}

A) {A}
B) {B}
C) {C}
D) {D}
\end{lstlisting}

\paragraph{Phi Family.}
We format the following messages using the respective model chat templates where \texttt{\{Qwen Prompt\}} is replaced with the prompt above.
\begin{lstlisting}
[
    {"role": "user", "content": {Qwen Prompt}},
]
\end{lstlisting}

\subsubsection{AIME26}
For AIME26, we use the same prompt templates as MATH500 for both the Qwen and Phi families.

\ifshowtraces\section{Reasoning Traces}\label{apx:traces}
We compare three example solutions to MATH500 questions from both uniform-cut and entropy-cut MH.
We each question, we present the question and correct answer, followed by the generated solutions by \texttt{Qwen2.5-Math-7B}.
We also display the reasoning trace of the transitions performed by both methods, truncated near the sampled cut token for readability (indicated by the symbol [...]).
For each question, we perform the same experiment from \Cref{sec:experiments:setup}, but reduce the number of MH steps to 3 for each block for readability.

\subsection{Math500 Question 51}
\paragraph{Question.}
You have seven bags of gold coins. Each bag has the same number of gold coins. One day, you find a bag of 53 coins. You decide to redistribute the number of coins you have so that all eight bags you hold have the same number of coins. You successfully manage to redistribute all the coins, and you also note that you have more than 200 coins. What is the smallest number of coins you could have had before finding the bag of 53 coins?

\paragraph{Correct Answer.}
$203$.

\subsubsection{Solutions}
\fzedit{In the written solutions, the main difference is that uniform-cut MH stops at the first candidate that superficially looks valid, while entropy-cut MH keeps enforcing the actual integrality constraint. Both solutions derive the same setup \(7x + 53 = 8y\) together with \(7x + 53 > 200\), but uniform-cut MH incorrectly accepts \(x=25\) after obtaining \(228\), even though this would imply \(y=28.5\), not an integer. By contrast, entropy-cut MH continues checking candidates until it reaches \(x=29\), where \(7x+53=256=8\cdot 32\), and therefore returns the correct total \(7x=203\).}

\paragraph{Uniform-Cut Solution.}
\begin{markdown}
To solve this problem, let's denote the number of coins in each of the original seven bags as \(x\). Therefore, the total number of coins before finding the additional bag is \(7x\). After finding the bag with 53 coins, you redistribute the coins so that all eight bags have the same number of coins. Let's denote the new number of coins in each bag as \(y\). The total number of coins after redistribution is \(7x + 53\), and this total is evenly divisible by 8, so we have:
\[7x + 53 = 8y\]

We also know that the total number of coins after redistribution is more than 200, so:
\[7x + 53 > 200\]

To find the smallest number of coins you could have had before finding the bag of 53 coins, we need to find the smallest value of \(x\) that satisfies these conditions. Let's start by solving the inequality for \(x\):
\[7x > 147\]
\[x > 21\]

Since \(x\) must be an integer, the smallest possible value for \(x\) is 22. However, we need to check if this value of \(x\) satisfies the equation \(7x + 53 = 8y\). Substituting \(x = 22\) into the equation:
\[7(22) + 53 = 154 + 53 = 207\]

Since 207 is not divisible by 8, \(x = 22\) is not a solution. We need to find the smallest \(x\) that makes \(7x + 53\) divisible by 8 and greater than 200. Let's try \(x = 23\):
\[7(23) + 53 = 161 + 53 = 214\]

214 is not divisible by 8, so we continue. Let's try \(x = 24\):
\[7(24) + 53 = 168 + 53 = 221\]

221 is not divisible by 8, so we continue. Let's try \(x = 25\):
\[7(25) + 53 = 175 + 53 = 228\]

228 is divisible by 8, and it is greater than 200. Therefore, the smallest number of coins you could have had before finding the bag of 53 coins is \(7 \times 25 = 175\).
\end{markdown}

\paragraph{Entropy-Cut Solution.}
\begin{markdown}
To solve this problem, let's denote the number of coins in each of the original seven bags as \(x\). Therefore, the total number of coins before finding the additional bag is \(7x\). After finding the bag with 53 coins, you redistribute all the coins into eight bags with an equal number of coins, which we'll call \(y\). This means the total number of coins is now \(7x + 53\), and this total is divisible by 8, so \(7x + 53 = 8y\).

Given that you have more than 200 coins in total, we have \(7x + 53 > 200\), which simplifies to \(7x > 147\), and thus \(x > 21\). Since \(x\) must be an integer, the smallest possible value for \(x\) is 22.

Substituting \(x = 22\) into the equation \(7x + 53 = 8y\), we get \(7(22) + 53 = 154 + 53 = 207\). This means \(8y = 207\), but since 207 is not divisible by 8, \(x = 22\) is not the correct value.

We need to find the smallest \(x\) such that \(7x + 53\) is divisible by 8 and greater than 200. Let's increment \(x\) by 1 and check again. For \(x = 23\), we have \(7(23) + 53 = 161 + 53 = 214\), which is divisible by 8 (since \(214 = 8 \times 26.75\), but we need an integer solution, so we continue).

For \(x = 24\), we have \(7(24) + 53 = 168 + 53 = 221\), which is also not divisible by 8. For \(x = 25\), we have \(7(25) + 53 = 175 + 53 = 228\), and \(228\) is divisible by 8 (since \(228 = 8 \times 28.5\), but again, we need an integer solution).

For \(x = 26\), we have \(7(26) + 53 = 182 + 53 = 235\), which is not divisible by 8. For \(x = 27\), we have \(7(27) + 53 = 189 + 53 = 242\), and \(242\) is divisible by 8 (since \(242 = 8 \times 30.25\), but we need an integer solution).

For \(x = 28\), we have \(7(28) + 53 = 196 + 53 = 249\), which is not divisible by 8. For \(x = 29\), we have \(7(29) + 53 = 203 + 53 = 256\), and \(256\) is divisible by 8 (since \(256 = 8 \times 32\)).

Therefore, the smallest number of coins you could have had before finding the bag of 53 coins is \(7 \times 29 = 203\). The final answer is \boxed{203}.
\end{markdown}

\subsubsection{Reasoning Traces}
\fzedit{Entropy-cut first falls into the same bad local branch as MCMC orig by fixing \(y=26\) and reducing the problem to \(7x=155\), but it later escapes that trap. The key transition occurs at \emph{Block 3 Step 1}, where the trace switches back to searching over integer values of \(x\) subject to the divisibility condition. The later \emph{Block 5 Step 3} is best viewed as the point where the corrected search stabilizes on the final correct choice \(x=29\). MCMC orig, in contrast, never makes that recovery and remains on the mistaken branch that terminates at \(175\).}

\paragraph{Uniform-Cut Trace.}
\begingroup
\raggedright %
\setlength{\lineskip}{0pt} %
\setlength{\baselineskip}{0.9em} %
\begin{twopartbox}{Cut Index 106 | Acceptance Probability 1.0000}{Block 1 Step 1}{Proposal (Accepted)}
\token[2.2824]| To|\allowbreak\token[0.4171]| solve|\allowbreak\token[0.4477]| this|\allowbreak\token[0.0840]| problem|\allowbreak\token[0.0111]|,|\allowbreak\token[0.5775]| let|\allowbreak\token[0.0029]|'s|\allowbreak\token[0.8168]| denote|\allowbreak\token[0.0233]| the|\allowbreak\token[0.7993]| number|\allowbreak\token[0.0000]| of|\allowbreak\token[0.1007]| coins|\allowbreak\token[0.0804]| in|\allowbreak\token[0.0205]| each|\allowbreak\token[0.0183]| of|\allowbreak\token[0.0031]| the|\allowbreak\token[0.5575]| original|\allowbreak\token[0.0329]| seven|\allowbreak\token[0.0001]| bags|\allowbreak\token[0.0009]| as|\allowbreak\token[0.6010]| \(|\allowbreak\token[0.8095]|x|\allowbreak\token[0.0048]|\|\allowbreak\token[0.1099]|).|\allowbreak\token[0.9476]| Therefore|\allowbreak\token[0.0000]|,|\allowbreak\token[0.6384]| the|\allowbreak\token[0.0007]| total|\allowbreak\token[0.0001]| number|\allowbreak\token[0.0000]| of|\allowbreak\token[0.0001]| coins|\allowbreak\token[0.5853]| before|\allowbreak\token[0.0055]| finding|\allowbreak\token[0.0003]| the|\allowbreak\token[0.8417]| additional|\allowbreak\token[0.3947]| bag|\allowbreak\token[0.8785]| is|\allowbreak\token[0.0000]| \(|\allowbreak\token[0.0001]|7|\allowbreak\token[0.0000]|x|\allowbreak\token[0.2478]|\|\allowbreak\token[0.0619]|).|\allowbreak\token[0.1434]| After|\allowbreak\token[0.0428]| finding|\allowbreak\token[0.0027]| the|\allowbreak\token[0.1502]| bag|\allowbreak\token[0.6866]| with|\allowbreak\token[0.0002]| |\allowbreak\token[0.0000]|5|\allowbreak\token[0.0001]|3|\allowbreak\token[0.0064]| coins|\allowbreak\token[0.1746]|,|\allowbreak\token[0.7104]| you|\allowbreak\token[0.2338]| redistribute|\allowbreak\token[0.8207]| all|\allowbreak\token[0.1743]| the|\allowbreak\token[0.0002]| coins|\allowbreak\token[0.7935]| into|\allowbreak\token[0.0057]| eight|\allowbreak\token[0.1040]| bags|\allowbreak\token[0.7179]| with|\allowbreak\token[0.1465]| an|\allowbreak\token[0.0013]| equal|\allowbreak\token[0.0010]| number|\allowbreak\token[0.0131]| of|\allowbreak\token[0.0000]| coins|\allowbreak\token[1.0825]|,|\allowbreak\token[0.9683]| which|\allowbreak\token[0.5047]| we|\allowbreak\token[0.0822]|'ll|\allowbreak\token[0.0494]| call|\allowbreak\token[0.0000]| \(|\allowbreak\token[0.0056]|y|\allowbreak\token[0.0002]|\|\allowbreak\token[0.0501]|).|\allowbreak\token[0.8690]| This|\allowbreak\token[0.1325]| means|\allowbreak\token[0.1778]| the|\allowbreak\token[0.0619]| total|\allowbreak\token[0.0001]| number|\allowbreak\token[0.0000]| of|\allowbreak\token[0.0000]| coins|\allowbreak\token[0.7681]| is|\allowbreak\token[0.6525]| now|\allowbreak\token[0.0222]| \(|\allowbreak\token[0.5747]|7|\allowbreak\token[0.0003]|x|\allowbreak\token[0.0000]| +|\allowbreak\token[0.0000]| |\allowbreak\token[0.0000]|5|\allowbreak\token[0.0000]|3|\allowbreak\token[0.6691]|\|\allowbreak\token[0.0232]|),|\allowbreak\token[0.0038]| and|\allowbreak\token[0.0990]| this|\allowbreak\token[0.1893]| total|\allowbreak\token[1.0998]| is|\allowbreak\token[1.2135]| divisible|\allowbreak\token[0.0057]| by|\allowbreak\token[0.0004]| |\allowbreak\token[0.0001]|8|\allowbreak\token[1.2314]|,|\allowbreak\token[1.5080]| so|\allowbreak\token[0.7424]| \(|\allowbreak\token[0.0598]|7|\allowbreak\token[0.0000]|x|\allowbreak\cut{[CUT]}\allowbreak\token[0.0000]| +|\allowbreak\token[0.0000]| |\allowbreak\token[0.0000]|5|\allowbreak\token[0.0001]|3|\allowbreak\token[0.0141]| =|\allowbreak\token[0.0000]| |\allowbreak\token[0.0000]|8|\allowbreak\token[0.0003]|y|\allowbreak\token[0.0008]|\|\allowbreak\token[0.5937]|).↩↩|\tighttracebreak\allowbreak\token[0.4107]|Given|\allowbreak\token[0.0002]| that|\allowbreak\token[0.5522]| you|\allowbreak\token[0.0682]| have|\allowbreak\token[0.0010]| more|\allowbreak\token[0.0000]| than|\allowbreak\token[0.0001]| |\allowbreak\token[0.0000]|2|\allowbreak\token[0.0000]|0|\allowbreak\token[0.0000]|0|\allowbreak\token[0.0001]| coins|\allowbreak\token[1.0382]| after|\allowbreak\token[0.8873]| redistribution|\allowbreak\token[0.0019]|,|\allowbreak\token[0.7739]| we|\allowbreak\token[0.7100]| know|\allowbreak\token[0.5303]| that|\allowbreak\token[0.0008]| \(|\allowbreak\token[0.2071]|8|\allowbreak\token[0.0000]|y|\allowbreak\token[0.0000]| >|\allowbreak\token[0.0000]| |\allowbreak\token[0.0003]|2|\allowbreak\token[0.0001]|0|\allowbreak\token[0.0019]|0|\allowbreak\token[0.0002]|\|\allowbreak\token[0.1095]|),|\allowbreak\token[0.0547]| which|\allowbreak\token[0.7146]| simpl|\allowbreak\token[0.0000]|ifies|\allowbreak\token[0.0000]| to|\allowbreak\token[0.0000]| \(|\allowbreak\token[0.0000]|y|\allowbreak\token[0.0000]| >|\allowbreak\token[0.0000]| |\allowbreak\token[0.0000]|2|\allowbreak\token[0.0000]|5|\allowbreak\token[0.0004]|\|\allowbreak\token[0.0245]|).|\allowbreak\token[1.4162]| Since|\allowbreak\token[0.2454]| \(|\allowbreak\token[0.0419]|y|\allowbreak\token[0.0082]|\)|\allowbreak\token[0.7963]| must|\allowbreak\token[0.0004]| be|\allowbreak\token[0.0617]| an|\allowbreak\token[0.0000]| integer|\allowbreak\token[0.2768]|,|\allowbreak\token[0.0018]| the|\allowbreak\token[0.0016]| smallest|\allowbreak\token[0.8835]| possible|\allowbreak\token[0.0008]| value|\allowbreak\token[0.0000]| for|\allowbreak\token[0.0000]| \(|\allowbreak\token[0.0001]|y|\allowbreak\token[0.0000]|\)|\allowbreak\token[0.3403]| is|\allowbreak\token[0.0042]| |\allowbreak\token[0.0004]|2|\allowbreak\token[0.0032]|6|\allowbreak\token[0.5630]|.|\allowbreak\token[1.2256]| Sub|\allowbreak\token[0.0000]|stit|\allowbreak\token[0.0005]|uting|\allowbreak\token[0.0008]| \(|\allowbreak\token[0.0001]|y|\allowbreak\token[0.0000]| =|\allowbreak\token[0.0000]| |\allowbreak\token[0.0000]|2|\allowbreak\token[0.0000]|6|\allowbreak\token[0.0001]|\)|\allowbreak\token[0.0091]| into|\allowbreak\token[0.6360]| our|\allowbreak\token[0.0000]| equation|\allowbreak\token[0.0748]| gives|\allowbreak\token[0.0002]| us|
\tcblower
\token[0.0000]| +|\allowbreak\token[0.0000]| |\allowbreak\token[0.0000]|5|\allowbreak\token[0.0001]|3|\allowbreak\token[0.0141]| =|\allowbreak\token[0.0000]| |\allowbreak\token[0.0000]|8|\allowbreak\token[0.0003]|y|\allowbreak\token[0.0010]|\|\allowbreak\token[0.6026]|).↩↩|\tighttracebreak\allowbreak\token[0.4106]|Given|\allowbreak\token[0.0002]| that|\allowbreak\token[0.5805]| you|\allowbreak\token[0.0687]| have|\allowbreak\token[0.0010]| more|\allowbreak\token[0.0000]| than|\allowbreak\token[0.0000]| |\allowbreak\token[0.0000]|2|\allowbreak\token[0.0000]|0|\allowbreak\token[0.0000]|0|\allowbreak\token[0.0001]| coins|\allowbreak\token[1.0398]| in|\allowbreak\token[0.0066]| total|\allowbreak\token[0.5339]|,|\allowbreak\token[0.0983]| we|\allowbreak\token[1.1632]| have|\allowbreak\token[0.0954]| \(|\allowbreak\token[0.2488]|7|\allowbreak\token[0.0000]|x|\allowbreak\token[0.0004]| +|\allowbreak\token[0.0000]| |\allowbreak\token[0.0000]|5|\allowbreak\token[0.0001]|3|\allowbreak\token[0.0000]| >|\allowbreak\token[0.0000]| |\allowbreak\token[0.0000]|2|\allowbreak\token[0.0001]|0|\allowbreak\token[0.0003]|0|\allowbreak\token[0.0000]|\|\allowbreak\token[0.6786]|).|\allowbreak\token[1.0683]| Simpl|\allowbreak\token[0.0000]|ifying|\allowbreak\token[0.5825]| this|\allowbreak\token[0.6077]|,|\allowbreak\token[0.0000]| we|\allowbreak\token[0.0496]| get|\allowbreak\token[0.0003]| \(|\allowbreak\token[0.0034]|7|\allowbreak\token[0.0000]|x|\allowbreak\token[0.0000]| >|\allowbreak\token[0.0000]| |\allowbreak\token[0.0000]|1|\allowbreak\token[0.0000]|4|\allowbreak\token[0.0001]|7|\allowbreak\token[0.0009]|\|\allowbreak\token[0.0004]|),|\allowbreak\token[0.4973]| which|\allowbreak\token[0.1144]| means|\allowbreak\token[0.0000]| \(|\allowbreak\token[0.0000]|x|\allowbreak\token[0.0000]| >|\allowbreak\token[0.0002]| |\allowbreak\token[0.0002]|2|\allowbreak\token[0.0000]|1|\allowbreak\token[0.0006]|\|\allowbreak\token[0.1602]|).|\allowbreak\token[0.7226]| Since|\allowbreak\token[0.0077]| \(|\allowbreak\token[0.0012]|x|\allowbreak\token[0.0001]|\)|\allowbreak\token[0.0698]| must|\allowbreak\token[0.0000]| be|\allowbreak\token[0.0494]| an|\allowbreak\token[0.0000]| integer|\allowbreak\token[0.5314]|,|\allowbreak\token[0.0023]| the|\allowbreak\token[0.0038]| smallest|\allowbreak\token[0.9413]| possible|\allowbreak\token[0.0016]| value|\allowbreak\token[0.0000]| for|\allowbreak\token[0.0000]| \(|\allowbreak\token[0.0004]|x|\allowbreak\token[0.0000]|\)|\allowbreak\token[0.4479]| is|\allowbreak\token[0.0030]| |\allowbreak\token[0.0001]|2|
\end{twopartbox}
\begin{twopartbox}{Cut Index 187 | Acceptance Probability 1.0000}{Block 1 Step 2}{Proposal (Accepted)}
[...]\allowbreak\token[0.0428]| finding|\allowbreak\token[0.0027]| the|\allowbreak\token[0.1502]| bag|\allowbreak\token[0.6866]| with|\allowbreak\token[0.0002]| |\allowbreak\token[0.0000]|5|\allowbreak\token[0.0001]|3|\allowbreak\token[0.0064]| coins|\allowbreak\token[0.1746]|,|\allowbreak\token[0.7104]| you|\allowbreak\token[0.2338]| redistribute|\allowbreak\token[0.8207]| all|\allowbreak\token[0.1743]| the|\allowbreak\token[0.0002]| coins|\allowbreak\token[0.7935]| into|\allowbreak\token[0.0057]| eight|\allowbreak\token[0.1040]| bags|\allowbreak\token[0.7179]| with|\allowbreak\token[0.1465]| an|\allowbreak\token[0.0013]| equal|\allowbreak\token[0.0010]| number|\allowbreak\token[0.0131]| of|\allowbreak\token[0.0000]| coins|\allowbreak\token[1.0825]|,|\allowbreak\token[0.9683]| which|\allowbreak\token[0.5047]| we|\allowbreak\token[0.0822]|'ll|\allowbreak\token[0.0494]| call|\allowbreak\token[0.0000]| \(|\allowbreak\token[0.0056]|y|\allowbreak\token[0.0002]|\|\allowbreak\token[0.0501]|).|\allowbreak\token[0.8690]| This|\allowbreak\token[0.1325]| means|\allowbreak\token[0.1778]| the|\allowbreak\token[0.0619]| total|\allowbreak\token[0.0001]| number|\allowbreak\token[0.0000]| of|\allowbreak\token[0.0000]| coins|\allowbreak\token[0.7681]| is|\allowbreak\token[0.6525]| now|\allowbreak\token[0.0222]| \(|\allowbreak\token[0.5747]|7|\allowbreak\token[0.0003]|x|\allowbreak\token[0.0000]| +|\allowbreak\token[0.0000]| |\allowbreak\token[0.0000]|5|\allowbreak\token[0.0000]|3|\allowbreak\token[0.6691]|\|\allowbreak\token[0.0232]|),|\allowbreak\token[0.0038]| and|\allowbreak\token[0.0990]| this|\allowbreak\token[0.1893]| total|\allowbreak\token[1.0998]| is|\allowbreak\token[1.2135]| divisible|\allowbreak\token[0.0057]| by|\allowbreak\token[0.0004]| |\allowbreak\token[0.0001]|8|\allowbreak\token[1.2314]|,|\allowbreak\token[1.5080]| so|\allowbreak\token[0.7424]| \(|\allowbreak\token[0.0598]|7|\allowbreak\token[0.0000]|x|\allowbreak\token[0.0000]| +|\allowbreak\token[0.0000]| |\allowbreak\token[0.0000]|5|\allowbreak\token[0.0001]|3|\allowbreak\token[0.0141]| =|\allowbreak\token[0.0000]| |\allowbreak\token[0.0000]|8|\allowbreak\token[0.0003]|y|\allowbreak\token[0.0010]|\|\allowbreak\token[0.6026]|).↩↩|\tighttracebreak\allowbreak\token[0.4106]|Given|\allowbreak\token[0.0002]| that|\allowbreak\token[0.5805]| you|\allowbreak\token[0.0687]| have|\allowbreak\token[0.0010]| more|\allowbreak\token[0.0000]| than|\allowbreak\token[0.0000]| |\allowbreak\token[0.0000]|2|\allowbreak\token[0.0000]|0|\allowbreak\token[0.0000]|0|\allowbreak\token[0.0001]| coins|\allowbreak\token[1.0398]| in|\allowbreak\token[0.0066]| total|\allowbreak\token[0.5339]|,|\allowbreak\token[0.0983]| we|\allowbreak\token[1.1632]| have|\allowbreak\token[0.0954]| \(|\allowbreak\token[0.2488]|7|\allowbreak\token[0.0000]|x|\allowbreak\token[0.0004]| +|\allowbreak\token[0.0000]| |\allowbreak\token[0.0000]|5|\allowbreak\token[0.0001]|3|\allowbreak\token[0.0000]| >|\allowbreak\token[0.0000]| |\allowbreak\token[0.0000]|2|\allowbreak\token[0.0001]|0|\allowbreak\token[0.0003]|0|\allowbreak\token[0.0000]|\|\allowbreak\token[0.6786]|).|\allowbreak\token[1.0683]| Simpl|\allowbreak\token[0.0000]|ifying|\allowbreak\token[0.5825]| this|\allowbreak\token[0.6077]|,|\allowbreak\token[0.0000]| we|\allowbreak\token[0.0496]| get|\allowbreak\token[0.0003]| \(|\allowbreak\token[0.0034]|7|\allowbreak\token[0.0000]|x|\allowbreak\token[0.0000]| >|\allowbreak\token[0.0000]| |\allowbreak\token[0.0000]|1|\allowbreak\token[0.0000]|4|\allowbreak\token[0.0001]|7|\allowbreak\token[0.0009]|\|\allowbreak\token[0.0004]|),|\allowbreak\token[0.4973]| which|\allowbreak\token[0.1144]| means|\allowbreak\token[0.0000]| \(|\allowbreak\token[0.0000]|x|\allowbreak\token[0.0000]| >|\allowbreak\token[0.0002]| |\allowbreak\token[0.0002]|2|\allowbreak\token[0.0000]|1|\allowbreak\token[0.0006]|\|\allowbreak\token[0.1602]|).|\allowbreak\token[0.7226]| Since|\allowbreak\token[0.0077]| \(|\allowbreak\token[0.0012]|x|\allowbreak\token[0.0001]|\)|\allowbreak\token[0.0698]| must|\allowbreak\token[0.0000]| be|\allowbreak\token[0.0494]| an|\allowbreak\token[0.0000]| integer|\allowbreak\token[0.5314]|,|\allowbreak\token[0.0023]| the|\allowbreak\token[0.0038]| smallest|\allowbreak\token[0.9413]| possible|\allowbreak\token[0.0016]| value|\allowbreak\token[0.0000]| for|\allowbreak\token[0.0000]| \(|\allowbreak\cut{[CUT]}\allowbreak\token[0.0004]|x|\allowbreak\token[0.0000]|\)|\allowbreak\token[0.4479]| is|\allowbreak\token[0.0030]| |\allowbreak\token[0.0001]|2|
\tcblower
\token[0.0004]|x|\allowbreak\token[0.0000]|\)|\allowbreak\token[0.4480]| is|\allowbreak\token[0.0027]| |\allowbreak\token[0.0001]|2|
\end{twopartbox}
\begin{twopartbox}{Cut Index 113 | Acceptance Probability 1.0000}{Block 1 Step 3}{Proposal (Accepted)}
\token[2.2824]| To|\allowbreak\token[0.4171]| solve|\allowbreak\token[0.4477]| this|\allowbreak\token[0.0840]| problem|\allowbreak\token[0.0111]|,|\allowbreak\token[0.5775]| let|\allowbreak\token[0.0029]|'s|\allowbreak\token[0.8168]| denote|\allowbreak\token[0.0233]| the|\allowbreak\token[0.7993]| number|\allowbreak\token[0.0000]| of|\allowbreak\token[0.1007]| coins|\allowbreak\token[0.0804]| in|\allowbreak\token[0.0205]| each|\allowbreak\token[0.0183]| of|\allowbreak\token[0.0031]| the|\allowbreak\token[0.5575]| original|\allowbreak\token[0.0329]| seven|\allowbreak\token[0.0001]| bags|\allowbreak\token[0.0009]| as|\allowbreak\token[0.6010]| \(|\allowbreak\token[0.8095]|x|\allowbreak\token[0.0048]|\|\allowbreak\token[0.1099]|).|\allowbreak\token[0.9476]| Therefore|\allowbreak\token[0.0000]|,|\allowbreak\token[0.6384]| the|\allowbreak\token[0.0007]| total|\allowbreak\token[0.0001]| number|\allowbreak\token[0.0000]| of|\allowbreak\token[0.0001]| coins|\allowbreak\token[0.5853]| before|\allowbreak\token[0.0055]| finding|\allowbreak\token[0.0003]| the|\allowbreak\token[0.8417]| additional|\allowbreak\token[0.3947]| bag|\allowbreak\token[0.8785]| is|\allowbreak\token[0.0000]| \(|\allowbreak\token[0.0001]|7|\allowbreak\token[0.0000]|x|\allowbreak\token[0.2478]|\|\allowbreak\token[0.0619]|).|\allowbreak\token[0.1434]| After|\allowbreak\token[0.0428]| finding|\allowbreak\token[0.0027]| the|\allowbreak\token[0.1502]| bag|\allowbreak\token[0.6866]| with|\allowbreak\token[0.0002]| |\allowbreak\token[0.0000]|5|\allowbreak\token[0.0001]|3|\allowbreak\token[0.0064]| coins|\allowbreak\token[0.1746]|,|\allowbreak\token[0.7104]| you|\allowbreak\token[0.2338]| redistribute|\allowbreak\token[0.8207]| all|\allowbreak\token[0.1743]| the|\allowbreak\token[0.0002]| coins|\allowbreak\token[0.7935]| into|\allowbreak\token[0.0057]| eight|\allowbreak\token[0.1040]| bags|\allowbreak\token[0.7179]| with|\allowbreak\token[0.1465]| an|\allowbreak\token[0.0013]| equal|\allowbreak\token[0.0010]| number|\allowbreak\token[0.0131]| of|\allowbreak\token[0.0000]| coins|\allowbreak\token[1.0825]|,|\allowbreak\token[0.9683]| which|\allowbreak\token[0.5047]| we|\allowbreak\token[0.0822]|'ll|\allowbreak\token[0.0494]| call|\allowbreak\token[0.0000]| \(|\allowbreak\token[0.0056]|y|\allowbreak\token[0.0002]|\|\allowbreak\token[0.0501]|).|\allowbreak\token[0.8690]| This|\allowbreak\token[0.1325]| means|\allowbreak\token[0.1778]| the|\allowbreak\token[0.0619]| total|\allowbreak\token[0.0001]| number|\allowbreak\token[0.0000]| of|\allowbreak\token[0.0000]| coins|\allowbreak\token[0.7681]| is|\allowbreak\token[0.6525]| now|\allowbreak\token[0.0222]| \(|\allowbreak\token[0.5747]|7|\allowbreak\token[0.0003]|x|\allowbreak\token[0.0000]| +|\allowbreak\token[0.0000]| |\allowbreak\token[0.0000]|5|\allowbreak\token[0.0000]|3|\allowbreak\token[0.6691]|\|\allowbreak\token[0.0232]|),|\allowbreak\token[0.0038]| and|\allowbreak\token[0.0990]| this|\allowbreak\token[0.1893]| total|\allowbreak\token[1.0998]| is|\allowbreak\token[1.2135]| divisible|\allowbreak\token[0.0057]| by|\allowbreak\token[0.0004]| |\allowbreak\token[0.0001]|8|\allowbreak\token[1.2314]|,|\allowbreak\token[1.5080]| so|\allowbreak\token[0.7424]| \(|\allowbreak\token[0.0598]|7|\allowbreak\token[0.0000]|x|\allowbreak\token[0.0000]| +|\allowbreak\token[0.0000]| |\allowbreak\token[0.0000]|5|\allowbreak\token[0.0001]|3|\allowbreak\token[0.0141]| =|\allowbreak\token[0.0000]| |\allowbreak\token[0.0000]|8|\allowbreak\cut{[CUT]}\allowbreak\token[0.0003]|y|\allowbreak\token[0.0010]|\|\allowbreak\token[0.6026]|).↩↩|\tighttracebreak\allowbreak\token[0.4106]|Given|\allowbreak\token[0.0002]| that|\allowbreak\token[0.5805]| you|\allowbreak\token[0.0687]| have|\allowbreak\token[0.0010]| more|\allowbreak\token[0.0000]| than|\allowbreak\token[0.0000]| |\allowbreak\token[0.0000]|2|\allowbreak\token[0.0000]|0|\allowbreak\token[0.0000]|0|\allowbreak\token[0.0001]| coins|\allowbreak\token[1.0398]| in|\allowbreak\token[0.0066]| total|\allowbreak\token[0.5339]|,|\allowbreak\token[0.0983]| we|\allowbreak\token[1.1632]| have|\allowbreak\token[0.0954]| \(|\allowbreak\token[0.2488]|7|\allowbreak\token[0.0000]|x|\allowbreak\token[0.0004]| +|\allowbreak\token[0.0000]| |\allowbreak\token[0.0000]|5|\allowbreak\token[0.0001]|3|\allowbreak\token[0.0000]| >|\allowbreak\token[0.0000]| |\allowbreak\token[0.0000]|2|\allowbreak\token[0.0001]|0|\allowbreak\token[0.0003]|0|\allowbreak\token[0.0000]|\|\allowbreak\token[0.6786]|).|\allowbreak\token[1.0683]| Simpl|\allowbreak\token[0.0000]|ifying|\allowbreak\token[0.5825]| this|\allowbreak\token[0.6077]|,|\allowbreak\token[0.0000]| we|\allowbreak\token[0.0496]| get|\allowbreak\token[0.0003]| \(|\allowbreak\token[0.0034]|7|\allowbreak\token[0.0000]|x|\allowbreak\token[0.0000]| >|\allowbreak\token[0.0000]| |\allowbreak\token[0.0000]|1|\allowbreak\token[0.0000]|4|\allowbreak\token[0.0001]|7|\allowbreak\token[0.0009]|\|\allowbreak\token[0.0004]|),|\allowbreak\token[0.4973]| which|\allowbreak\token[0.1144]| means|\allowbreak\token[0.0000]| \(|\allowbreak\token[0.0000]|x|\allowbreak\token[0.0000]| >|\allowbreak\token[0.0002]| |\allowbreak\token[0.0002]|2|\allowbreak\token[0.0000]|1|\allowbreak\token[0.0006]|\|\allowbreak\token[0.1602]|).|\allowbreak\token[0.7226]| Since|\allowbreak\token[0.0077]| \(|\allowbreak\token[0.0012]|x|\allowbreak\token[0.0001]|\)|\allowbreak\token[0.0698]| must|\allowbreak\token[0.0000]| be|\allowbreak\token[0.0494]| an|\allowbreak\token[0.0000]| integer|\allowbreak\token[0.5314]|,|\allowbreak\token[0.0023]| the|\allowbreak\token[0.0038]| smallest|\allowbreak\token[0.9413]| possible|\allowbreak\token[0.0016]| value|\allowbreak\token[0.0000]| for|\allowbreak\token[0.0000]| \(|\allowbreak\token[0.0004]|x|\allowbreak\token[0.0000]|\)|\allowbreak\token[0.4480]| is|\allowbreak\token[0.0027]| |\allowbreak\token[0.0001]|2|
\tcblower
\token[0.0003]|y|\allowbreak\token[0.0009]|\|\allowbreak\token[0.6026]|).↩↩|\tighttracebreak\allowbreak\token[0.4110]|Given|\allowbreak\token[0.0002]| that|\allowbreak\token[0.5804]| you|\allowbreak\token[0.0682]| have|\allowbreak\token[0.0010]| more|\allowbreak\token[0.0000]| than|\allowbreak\token[0.0001]| |\allowbreak\token[0.0000]|2|\allowbreak\token[0.0000]|0|\allowbreak\token[0.0000]|0|\allowbreak\token[0.0001]| coins|\allowbreak\token[1.0382]| in|\allowbreak\token[0.0066]| total|\allowbreak\token[0.5604]|,|\allowbreak\token[0.1088]| we|\allowbreak\token[1.1632]| have|\allowbreak\token[0.0970]| \(|\allowbreak\token[0.2701]|7|\allowbreak\token[0.0000]|x|\allowbreak\token[0.0003]| +|\allowbreak\token[0.0000]| |\allowbreak\token[0.0000]|5|\allowbreak\token[0.0001]|3|\allowbreak\token[0.0000]| >|\allowbreak\token[0.0000]| |\allowbreak\token[0.0000]|2|\allowbreak\token[0.0001]|0|\allowbreak\token[0.0003]|0|\allowbreak\token[0.0000]|\|\allowbreak\token[0.6783]|),|\allowbreak\token[0.0054]| which|\allowbreak\token[0.0017]| simpl|\allowbreak\token[0.0000]|ifies|\allowbreak\token[0.0000]| to|\allowbreak\token[0.0000]| \(|\allowbreak\token[0.0327]|7|\allowbreak\token[0.0000]|x|\allowbreak\token[0.0000]| >|\allowbreak\token[0.0000]| |\allowbreak\token[0.0000]|1|\allowbreak\token[0.0000]|4|\allowbreak\token[0.0001]|7|\allowbreak\token[0.0494]|\|\allowbreak\token[0.0446]|),|\allowbreak\token[0.3285]| and|\allowbreak\token[0.2730]| thus|\allowbreak\token[0.0002]| \(|\allowbreak\token[0.0000]|x|\allowbreak\token[0.0000]| >|\allowbreak\token[0.0001]| |\allowbreak\token[0.0001]|2|\allowbreak\token[0.0000]|1|\allowbreak\token[0.0005]|\|\allowbreak\token[0.1096]|).|\allowbreak\token[1.0603]| Since|\allowbreak\token[0.0127]| \(|\allowbreak\token[0.0036]|x|\allowbreak\token[0.0002]|\)|\allowbreak\token[0.0962]| must|\allowbreak\token[0.0000]| be|\allowbreak\token[0.0294]| an|\allowbreak\token[0.0000]| integer|\allowbreak\token[0.3086]|,|\allowbreak\token[0.0030]| the|\allowbreak\token[0.0059]| smallest|\allowbreak\token[0.8335]| possible|\allowbreak\token[0.0020]| value|\allowbreak\token[0.0000]| for|\allowbreak\token[0.0000]| \(|\allowbreak\token[0.0004]|x|\allowbreak\token[0.0000]|\)|\allowbreak\token[0.4213]| is|\allowbreak\token[0.0042]| |\allowbreak\token[0.0001]|2|\allowbreak\token[0.0016]|2|\allowbreak\token[0.7192]|.↩↩|\tighttracebreak
\end{twopartbox}
\begin{twopartbox}{Cut Index 7 | Acceptance Probability 1.0000}{Block 2 Step 1}{Proposal (Accepted)}
\token[2.2824]| To|\allowbreak\token[0.4171]| solve|\allowbreak\token[0.4477]| this|\allowbreak\token[0.0840]| problem|\allowbreak\token[0.0111]|,|\allowbreak\token[0.5775]| let|\allowbreak\token[0.0029]|'s|\allowbreak\cut{[CUT]}\allowbreak\token[0.8168]| denote|\allowbreak\token[0.0233]| the|\allowbreak\token[0.7993]| number|\allowbreak\token[0.0000]| of|\allowbreak\token[0.1007]| coins|\allowbreak\token[0.0804]| in|\allowbreak\token[0.0205]| each|\allowbreak\token[0.0183]| of|\allowbreak\token[0.0031]| the|\allowbreak\token[0.5575]| original|\allowbreak\token[0.0329]| seven|\allowbreak\token[0.0001]| bags|\allowbreak\token[0.0009]| as|\allowbreak\token[0.6010]| \(|\allowbreak\token[0.8095]|x|\allowbreak\token[0.0048]|\|\allowbreak\token[0.1099]|).|\allowbreak\token[0.9476]| Therefore|\allowbreak\token[0.0000]|,|\allowbreak\token[0.6384]| the|\allowbreak\token[0.0007]| total|\allowbreak\token[0.0001]| number|\allowbreak\token[0.0000]| of|\allowbreak\token[0.0001]| coins|\allowbreak\token[0.5853]| before|\allowbreak\token[0.0055]| finding|\allowbreak\token[0.0003]| the|\allowbreak\token[0.8417]| additional|\allowbreak\token[0.3947]| bag|\allowbreak\token[0.8785]| is|\allowbreak\token[0.0000]| \(|\allowbreak\token[0.0001]|7|\allowbreak\token[0.0000]|x|\allowbreak\token[0.2478]|\|\allowbreak\token[0.0619]|).|\allowbreak\token[0.1434]| After|\allowbreak\token[0.0428]| finding|\allowbreak\token[0.0027]| the|\allowbreak\token[0.1502]| bag|\allowbreak\token[0.6866]| with|\allowbreak\token[0.0002]| |\allowbreak\token[0.0000]|5|\allowbreak\token[0.0001]|3|\allowbreak\token[0.0064]| coins|\allowbreak\token[0.1746]|,|\allowbreak\token[0.7104]| you|\allowbreak\token[0.2338]| redistribute|\allowbreak\token[0.8207]| all|\allowbreak\token[0.1743]| the|\allowbreak\token[0.0002]| coins|\allowbreak\token[0.7935]| into|\allowbreak\token[0.0057]| eight|\allowbreak\token[0.1040]| bags|\allowbreak\token[0.7179]| with|\allowbreak\token[0.1465]| an|\allowbreak\token[0.0013]| equal|\allowbreak\token[0.0010]| number|\allowbreak\token[0.0131]| of|\allowbreak\token[0.0000]| coins|\allowbreak\token[1.0825]|,|\allowbreak\token[0.9683]| which|\allowbreak\token[0.5047]| we|\allowbreak\token[0.0822]|'ll|\allowbreak\token[0.0494]| call|\allowbreak\token[0.0000]| \(|\allowbreak\token[0.0056]|y|\allowbreak\token[0.0002]|\|\allowbreak\token[0.0501]|).|\allowbreak\token[0.8690]| This|\allowbreak\token[0.1325]| means|\allowbreak\token[0.1778]| the|\allowbreak\token[0.0619]| total|\allowbreak\token[0.0001]| number|\allowbreak\token[0.0000]| of|\allowbreak\token[0.0000]| coins|\allowbreak\token[0.7681]| is|\allowbreak\token[0.6525]| now|\allowbreak\token[0.0222]| \(|\allowbreak\token[0.5747]|7|\allowbreak\token[0.0003]|x|\allowbreak\token[0.0000]| +|\allowbreak\token[0.0000]| |\allowbreak\token[0.0000]|5|\allowbreak\token[0.0000]|3|\allowbreak\token[0.6691]|\|\allowbreak\token[0.0232]|),|\allowbreak\token[0.0038]| and|\allowbreak\token[0.0990]| this|\allowbreak\token[0.1893]| total|\allowbreak\token[1.0998]| is|\allowbreak\token[1.2135]| divisible|\allowbreak\token[0.0057]| by|\allowbreak\token[0.0004]| |\allowbreak\token[0.0001]|8|\allowbreak\token[1.2314]|,|\allowbreak\token[1.5080]| so|\allowbreak\token[0.7424]| \(|\allowbreak\token[0.0598]|7|\allowbreak\token[0.0000]|x|\allowbreak\token[0.0000]| +|\allowbreak\token[0.0000]| |\allowbreak\token[0.0000]|5|\allowbreak\token[0.0001]|3|\allowbreak\token[0.0141]| =|\allowbreak\token[0.0000]| |\allowbreak\token[0.0000]|8|\allowbreak\token[0.0003]|y|\allowbreak\token[0.0009]|\|\allowbreak\token[0.6026]|).↩↩|\tighttracebreak\allowbreak\token[0.4110]|Given|\allowbreak\token[0.0002]| that|\allowbreak\token[0.5804]| you|\allowbreak\token[0.0682]| have|\allowbreak\token[0.0010]| more|\allowbreak\token[0.0000]| than|\allowbreak\token[0.0001]| |\allowbreak\token[0.0000]|2|\allowbreak\token[0.0000]|0|\allowbreak\token[0.0000]|0|\allowbreak\token[0.0001]| coins|\allowbreak\token[1.0382]| in|\allowbreak\token[0.0066]| total|\allowbreak\token[0.5604]|,|\allowbreak\token[0.1088]| we|\allowbreak\token[1.1632]| have|\allowbreak\token[0.0970]| \(|\allowbreak\token[0.2701]|7|\allowbreak\token[0.0000]|x|\allowbreak\token[0.0003]| +|\allowbreak\token[0.0000]| |\allowbreak\token[0.0000]|5|\allowbreak\token[0.0001]|3|\allowbreak\token[0.0000]| >|\allowbreak\token[0.0000]| |\allowbreak\token[0.0000]|2|\allowbreak\token[0.0001]|0|\allowbreak\token[0.0003]|0|\allowbreak\token[0.0000]|\|\allowbreak\token[0.6783]|),|\allowbreak\token[0.0054]| which|\allowbreak\token[0.0017]| simpl|\allowbreak\token[0.0000]|ifies|\allowbreak\token[0.0000]| to|\allowbreak\token[0.0000]| \(|\allowbreak[...]
\tcblower
\token[0.7813]| denote|\allowbreak\token[0.0220]| the|\allowbreak\token[0.8006]| number|\allowbreak\token[0.0001]| of|\allowbreak\token[0.1007]| coins|\allowbreak\token[0.0814]| in|\allowbreak\token[0.0184]| each|\allowbreak\token[0.0164]| of|\allowbreak\token[0.0034]| the|\allowbreak\token[0.5548]| original|\allowbreak\token[0.0362]| seven|\allowbreak\token[0.0001]| bags|\allowbreak\token[0.0010]| as|\allowbreak\token[0.6332]| \(|\allowbreak\token[0.7849]|x|\allowbreak\token[0.0053]|\|\allowbreak\token[0.1099]|).|\allowbreak\token[0.9478]| Therefore|\allowbreak\token[0.0000]|,|\allowbreak\token[0.6160]| the|\allowbreak\token[0.0009]| total|\allowbreak\token[0.0001]| number|\allowbreak\token[0.0000]| of|\allowbreak\token[0.0001]| coins|\allowbreak\token[0.5853]| before|\allowbreak\token[0.0061]| finding|\allowbreak\token[0.0003]| the|\allowbreak\token[0.7916]| additional|\allowbreak\token[0.3947]| bag|\allowbreak\token[0.8396]| is|\allowbreak\token[0.0000]| \(|\allowbreak\token[0.0001]|7|\allowbreak\token[0.0000]|x|\allowbreak\token[0.2691]|\|\allowbreak\token[0.0558]|).|\allowbreak\token[0.1432]| After|\allowbreak\token[0.0399]| finding|\allowbreak\token[0.0025]| the|\allowbreak\token[0.1605]| bag|\allowbreak\token[0.6926]| with|\allowbreak\token[0.0002]| |\allowbreak\token[0.0000]|5|\allowbreak\token[0.0001]|3|\allowbreak\token[0.0064]| coins|\allowbreak\token[0.1746]|,|\allowbreak\token[0.7223]| you|\allowbreak\token[0.2138]| redistribute|\allowbreak\token[0.8151]| the|\allowbreak\token[0.2692]| coins|\allowbreak\token[0.9168]| so|\allowbreak\token[0.0000]| that|\allowbreak\token[0.8742]| all|\allowbreak\token[0.0001]| eight|\allowbreak\token[0.0000]| bags|\allowbreak\token[0.0107]| have|\allowbreak\token[0.0498]| the|\allowbreak\token[0.0001]| same|\allowbreak\token[0.0038]| number|\allowbreak\token[0.1279]| of|\allowbreak\token[0.0000]| coins|\allowbreak\token[0.6799]|.|\allowbreak\token[0.6893]| Let|\allowbreak\token[0.0016]|'s|\allowbreak\token[0.4479]| denote|\allowbreak\token[0.6060]| the|\allowbreak\token[0.1672]| new|\allowbreak\token[0.0997]| number|\allowbreak\token[0.0000]| of|\allowbreak\token[0.0000]| coins|\allowbreak\token[0.1204]| in|\allowbreak\token[0.0001]| each|\allowbreak\token[0.2477]| bag|\allowbreak\token[0.1461]| as|\allowbreak\token[0.0000]| \(|\allowbreak\token[0.0029]|y|\allowbreak\token[0.0000]|\|\allowbreak\token[0.1042]|).|\allowbreak\token[1.2032]| The|\allowbreak\token[0.0333]| total|\allowbreak\token[0.0000]| number|\allowbreak\token[0.0000]| of|\allowbreak\token[0.0000]| coins|\allowbreak\token[0.2547]| after|\allowbreak\token[0.3218]| redistribution|\allowbreak\token[0.0016]| is|\allowbreak\token[0.0207]| \(|\allowbreak\token[0.6968]|7|\allowbreak\token[0.0060]|x|\allowbreak\token[0.0000]| +|\allowbreak\token[0.0000]| |\allowbreak\token[0.0000]|5|\allowbreak\token[0.0000]|3|\allowbreak\token[0.5331]|\|\allowbreak\token[0.0273]|),|\allowbreak\token[0.3655]| and|\allowbreak\token[0.5117]| this|\allowbreak\token[0.9310]| total|\allowbreak\token[0.8310]| is|\allowbreak\token[1.5498]| evenly|\allowbreak\token[0.3380]| divisible|\allowbreak\token[0.0038]| by|\allowbreak\token[0.0012]| |\allowbreak\token[0.0002]|8|\allowbreak\token[0.4522]|,|\allowbreak\token[1.0927]| so|\allowbreak\token[0.7020]| we|\allowbreak\token[0.0113]| have|\allowbreak\token[0.9266]|:↩|\newline\allowbreak\token[0.0074]|\[|\allowbreak\token[0.0236]|7|\allowbreak\token[0.0001]|x|\allowbreak\token[0.0000]| +|\allowbreak\token[0.0000]| |\allowbreak\token[0.0000]|5|\allowbreak\token[0.0002]|3|\allowbreak\token[0.0001]| =|\allowbreak\token[0.0000]| |\allowbreak\token[0.0000]|8|\allowbreak\token[0.0000]|y|\allowbreak\token[0.0003]|\|\allowbreak\token[0.3655]|]↩↩|\tighttracebreak\allowbreak\token[0.7411]|We|\allowbreak\token[0.3478]| also|\allowbreak\token[0.0000]| know|\allowbreak\token[0.0001]| that|\allowbreak\token[0.2705]| the|\allowbreak\token[0.0020]| total|\allowbreak\token[0.0055]| number|\allowbreak\token[0.0000]| of|\allowbreak\token[0.0000]| coins|\allowbreak\token[0.7427]| after|\allowbreak\token[0.0057]| redistribution|\allowbreak\token[0.0063]| is|\allowbreak\token[0.0218]| more|\allowbreak\token[0.0000]| than|\allowbreak\token[0.0000]| |\allowbreak\token[0.0000]|2|\allowbreak\token[0.0000]|0|\allowbreak\token[0.0002]|0|\allowbreak\token[0.3491]|,|\allowbreak\token[0.0685]| so|\allowbreak\token[0.0566]|:↩|\newline\allowbreak\token[0.0000]|\[|\allowbreak\token[0.5834]|7|\allowbreak[...]
\end{twopartbox}
\begin{twopartbox}{Cut Index 181 | Acceptance Probability 1.0000}{Block 2 Step 2}{Proposal (Accepted)}
[...]\allowbreak\token[0.0000]| \(|\allowbreak\token[0.0001]|7|\allowbreak\token[0.0000]|x|\allowbreak\token[0.2691]|\|\allowbreak\token[0.0558]|).|\allowbreak\token[0.1432]| After|\allowbreak\token[0.0399]| finding|\allowbreak\token[0.0025]| the|\allowbreak\token[0.1605]| bag|\allowbreak\token[0.6926]| with|\allowbreak\token[0.0002]| |\allowbreak\token[0.0000]|5|\allowbreak\token[0.0001]|3|\allowbreak\token[0.0064]| coins|\allowbreak\token[0.1746]|,|\allowbreak\token[0.7223]| you|\allowbreak\token[0.2138]| redistribute|\allowbreak\token[0.8151]| the|\allowbreak\token[0.2692]| coins|\allowbreak\token[0.9168]| so|\allowbreak\token[0.0000]| that|\allowbreak\token[0.8742]| all|\allowbreak\token[0.0001]| eight|\allowbreak\token[0.0000]| bags|\allowbreak\token[0.0107]| have|\allowbreak\token[0.0498]| the|\allowbreak\token[0.0001]| same|\allowbreak\token[0.0038]| number|\allowbreak\token[0.1279]| of|\allowbreak\token[0.0000]| coins|\allowbreak\token[0.6799]|.|\allowbreak\token[0.6893]| Let|\allowbreak\token[0.0016]|'s|\allowbreak\token[0.4479]| denote|\allowbreak\token[0.6060]| the|\allowbreak\token[0.1672]| new|\allowbreak\token[0.0997]| number|\allowbreak\token[0.0000]| of|\allowbreak\token[0.0000]| coins|\allowbreak\token[0.1204]| in|\allowbreak\token[0.0001]| each|\allowbreak\token[0.2477]| bag|\allowbreak\token[0.1461]| as|\allowbreak\token[0.0000]| \(|\allowbreak\token[0.0029]|y|\allowbreak\token[0.0000]|\|\allowbreak\token[0.1042]|).|\allowbreak\token[1.2032]| The|\allowbreak\token[0.0333]| total|\allowbreak\token[0.0000]| number|\allowbreak\token[0.0000]| of|\allowbreak\token[0.0000]| coins|\allowbreak\token[0.2547]| after|\allowbreak\token[0.3218]| redistribution|\allowbreak\token[0.0016]| is|\allowbreak\token[0.0207]| \(|\allowbreak\token[0.6968]|7|\allowbreak\token[0.0060]|x|\allowbreak\token[0.0000]| +|\allowbreak\token[0.0000]| |\allowbreak\token[0.0000]|5|\allowbreak\token[0.0000]|3|\allowbreak\token[0.5331]|\|\allowbreak\token[0.0273]|),|\allowbreak\token[0.3655]| and|\allowbreak\token[0.5117]| this|\allowbreak\token[0.9310]| total|\allowbreak\token[0.8310]| is|\allowbreak\token[1.5498]| evenly|\allowbreak\token[0.3380]| divisible|\allowbreak\token[0.0038]| by|\allowbreak\token[0.0012]| |\allowbreak\token[0.0002]|8|\allowbreak\token[0.4522]|,|\allowbreak\token[1.0927]| so|\allowbreak\token[0.7020]| we|\allowbreak\token[0.0113]| have|\allowbreak\token[0.9266]|:↩|\newline\allowbreak\token[0.0074]|\[|\allowbreak\token[0.0236]|7|\allowbreak\token[0.0001]|x|\allowbreak\token[0.0000]| +|\allowbreak\token[0.0000]| |\allowbreak\token[0.0000]|5|\allowbreak\token[0.0002]|3|\allowbreak\token[0.0001]| =|\allowbreak\token[0.0000]| |\allowbreak\token[0.0000]|8|\allowbreak\token[0.0000]|y|\allowbreak\token[0.0003]|\|\allowbreak\token[0.3655]|]↩↩|\tighttracebreak\allowbreak\token[0.7411]|We|\allowbreak\token[0.3478]| also|\allowbreak\token[0.0000]| know|\allowbreak\token[0.0001]| that|\allowbreak\token[0.2705]| the|\allowbreak\token[0.0020]| total|\allowbreak\token[0.0055]| number|\allowbreak\token[0.0000]| of|\allowbreak\token[0.0000]| coins|\allowbreak\token[0.7427]| after|\allowbreak\token[0.0057]| redistribution|\allowbreak\token[0.0063]| is|\allowbreak\token[0.0218]| more|\allowbreak\token[0.0000]| than|\allowbreak\token[0.0000]| |\allowbreak\token[0.0000]|2|\allowbreak\token[0.0000]|0|\allowbreak\token[0.0002]|0|\allowbreak\token[0.3491]|,|\allowbreak\token[0.0685]| so|\allowbreak\token[0.0566]|:↩|\newline\allowbreak\token[0.0000]|\[|\allowbreak\token[0.5834]|7|\allowbreak\token[0.0002]|x|\allowbreak\token[0.0000]| +|\allowbreak\token[0.0000]| |\allowbreak\token[0.0000]|5|\allowbreak\token[0.0001]|3|\allowbreak\token[0.0001]| >|\allowbreak\token[0.0000]| |\allowbreak\token[0.0000]|2|\allowbreak\token[0.0000]|0|\allowbreak\token[0.0002]|0|\allowbreak\token[0.0000]|\|\allowbreak\token[0.4472]|]↩↩|\tighttracebreak\allowbreak\token[0.7990]|To|\allowbreak\token[0.0005]| find|\allowbreak\token[0.0001]| the|\allowbreak\token[0.0020]| smallest|\allowbreak\token[0.0738]| number|\allowbreak\token[0.0000]| of|\allowbreak\token[0.0013]| coins|\allowbreak\token[0.9062]| you|\allowbreak\token[0.0001]| could|\allowbreak\token[0.0000]| have|\allowbreak\token[0.0037]| had|\allowbreak\token[0.5101]| before|\allowbreak\token[0.0058]| finding|\allowbreak\token[0.0000]| the|\allowbreak\token[0.2427]| bag|\allowbreak\token[0.0011]| of|\allowbreak\token[0.0000]| |\allowbreak\token[0.0000]|5|\allowbreak\cut{[CUT]}\allowbreak\token[0.0001]|3|\allowbreak\token[0.0047]| coins|\allowbreak\token[0.0266]|,|\allowbreak\token[0.0002]| we|\allowbreak\token[0.0072]| need|\allowbreak\token[0.0000]| to|\allowbreak\token[0.0403]| find|\allowbreak\token[0.0000]| the|\allowbreak\token[0.0005]| smallest|\allowbreak\token[0.7630]| value|\allowbreak\token[0.0000]| of|\allowbreak\token[0.0000]| \(|\allowbreak\token[0.1627]|x|\allowbreak\token[0.0000]|\)|\allowbreak\token[0.0001]| that|\allowbreak\token[0.0001]| satisfies|\allowbreak\token[0.5978]| these|\allowbreak\token[0.0000]| conditions|\allowbreak\token[0.1894]|.|\allowbreak\token[0.8435]| Let|\allowbreak\token[0.0000]|'s|\allowbreak\token[0.8865]| start|\allowbreak\token[0.0004]| by|\allowbreak\token[0.7690]| solving|\allowbreak\token[0.0819]| the|\allowbreak\token[0.3394]| inequality|\allowbreak\token[0.6897]|:↩|\newline\allowbreak\token[0.0000]|\[|\allowbreak\token[0.0001]|7|\allowbreak\token[0.0000]|x|\allowbreak\token[0.3148]| >|\allowbreak\token[0.0000]| |\allowbreak\token[0.0326]|1|\allowbreak\token[0.0000]|4|\allowbreak\token[0.0001]|7|\allowbreak\token[0.0002]|\|\allowbreak\token[0.0059]|]↩|\newline\allowbreak\token[0.0009]|\[|\allowbreak\token[0.0000]|x|\allowbreak\token[0.0000]| >|\allowbreak\token[0.0008]| |\allowbreak\token[0.0001]|2|\allowbreak\token[0.0000]|1|\allowbreak\token[0.0020]|\|\allowbreak\token[0.0002]|]↩↩|\tighttracebreak\allowbreak\token[0.0966]|Since|\allowbreak\token[0.0012]| \(|\allowbreak\token[0.0013]|x|\allowbreak\token[0.0000]|\)|\allowbreak\token[0.0108]| must|\allowbreak\token[0.0000]| be|\allowbreak\token[0.0740]| an|\allowbreak\token[0.0000]| integer|\allowbreak\token[0.9570]|,|\allowbreak\token[0.0074]| the|\allowbreak\token[0.0003]| smallest|\allowbreak\token[0.6142]| possible|\allowbreak\token[0.0000]| value|\allowbreak\token[0.0002]| for|\allowbreak\token[0.0000]| \(|\allowbreak\token[0.0004]|x|\allowbreak\token[0.0000]|\)|\allowbreak\token[0.0551]| is|\allowbreak\token[0.0011]| |\allowbreak\token[0.0000]|2|\allowbreak\token[0.0007]|2|\allowbreak\token[0.0015]|.|\allowbreak\token[0.1316]| However|\allowbreak\token[0.0000]|,|\allowbreak\token[0.0422]| we|\allowbreak\token[0.1763]| need|\allowbreak\token[0.0001]| to|\allowbreak\token[0.7193]| check|\allowbreak\token[0.0001]| if|\allowbreak\token[0.6715]| this|\allowbreak\token[0.0853]| value|\allowbreak\token[0.7378]| of|\allowbreak\token[0.0000]| \(|\allowbreak\token[0.0002]|x|\allowbreak\token[0.0003]|\)|\allowbreak\token[1.3230]| satisfies|\allowbreak\token[0.0258]| the|\allowbreak\token[0.9062]| equation|\allowbreak\token[0.0055]| \(|\allowbreak\token[0.0001]|7|\allowbreak\token[0.0000]|x|\allowbreak\token[0.0000]| +|\allowbreak\token[0.0000]| |\allowbreak\token[0.0000]|5|\allowbreak\token[0.0000]|3|\allowbreak\token[0.0020]| =|\allowbreak\token[0.0000]| |\allowbreak\token[0.0000]|8|\allowbreak\token[0.0000]|y|\allowbreak\token[0.1094]|\|\allowbreak\token[0.0900]|).|\allowbreak\token[0.2594]| Sub|\allowbreak\token[0.0000]|stit|\allowbreak\token[0.0012]|uting|\allowbreak\token[0.0000]| \(|\allowbreak\token[0.0000]|x|\allowbreak\token[0.0000]| =|\allowbreak\token[0.0000]| |\allowbreak\token[0.0000]|2|\allowbreak\token[0.0001]|2|\allowbreak\token[0.6059]|\)|\allowbreak\token[0.0446]| into|\allowbreak\token[0.0000]| the|\allowbreak\token[0.0000]| equation|\allowbreak\token[0.6390]|:↩|\newline\allowbreak\token[0.0000]|\[|\allowbreak\token[0.0015]|7|\allowbreak\token[0.0127]|(|\allowbreak\token[0.0000]|2|\allowbreak\token[0.0000]|2|\allowbreak\token[0.0000]|)|\allowbreak\token[0.0000]| +|\allowbreak\token[0.0000]| |\allowbreak\token[0.0000]|5|\allowbreak\token[0.0000]|3|\allowbreak\token[0.0000]| =|\allowbreak\token[0.0000]| |\allowbreak\token[0.1390]|1|\allowbreak\token[0.2614]|5|\allowbreak\token[0.0088]|4|\allowbreak\token[0.0001]| +|\allowbreak\token[0.0000]| |\allowbreak\token[0.0000]|5|\allowbreak\token[0.0000]|3|\allowbreak\token[0.0000]| =|\allowbreak\token[0.0000]| |\allowbreak\token[0.0004]|2|\allowbreak\token[0.0000]|0|\allowbreak\token[0.0000]|7|\allowbreak\token[0.0073]|\|\allowbreak\token[0.1745]|]↩↩|\tighttracebreak\allowbreak\token[0.9890]|Since|\allowbreak\token[0.2911]| |\allowbreak\token[0.0000]|2|\allowbreak\token[0.0000]|0|\allowbreak\token[0.0000]|7|\allowbreak\token[0.0015]| is|\allowbreak\token[0.6378]| not|\allowbreak\token[0.0297]| divisible|\allowbreak[...]
\tcblower
\token[0.0001]|3|\allowbreak\token[0.0047]| coins|\allowbreak\token[0.0266]|,|\allowbreak\token[0.0002]| we|\allowbreak\token[0.0070]| need|\allowbreak\token[0.0000]| to|\allowbreak\token[0.0403]| find|\allowbreak\token[0.0000]| the|\allowbreak\token[0.0005]| smallest|\allowbreak\token[0.7712]| value|\allowbreak\token[0.0000]| of|\allowbreak\token[0.0000]| \(|\allowbreak\token[0.1778]|x|\allowbreak\token[0.0000]|\)|\allowbreak\token[0.0001]| that|\allowbreak\token[0.0001]| satisfies|\allowbreak\token[0.6093]| these|\allowbreak\token[0.0000]| conditions|\allowbreak\token[0.1894]|.|\allowbreak\token[0.7665]| Let|\allowbreak\token[0.0000]|'s|\allowbreak\token[0.8862]| start|\allowbreak\token[0.0004]| by|\allowbreak\token[0.7551]| solving|\allowbreak\token[0.0742]| the|\allowbreak\token[0.2850]| inequality|\allowbreak\token[0.6954]| for|\allowbreak\token[0.0908]| \(|\allowbreak\token[0.0324]|x|\allowbreak\token[0.0000]|\|\allowbreak\token[0.0000]|):↩|\newline\allowbreak\token[0.0001]|\[|\allowbreak\token[0.0001]|7|\allowbreak\token[0.0000]|x|\allowbreak\token[0.0140]| >|\allowbreak\token[0.0000]| |\allowbreak\token[0.0669]|1|\allowbreak\token[0.0000]|4|\allowbreak\token[0.0001]|7|\allowbreak\token[0.0001]|\|\allowbreak\token[0.0010]|]↩|\newline\allowbreak\token[0.0005]|\[|\allowbreak\token[0.0000]|x|\allowbreak\token[0.0000]| >|\allowbreak\token[0.0012]| |\allowbreak\token[0.0000]|2|\allowbreak\token[0.0000]|1|\allowbreak\token[0.0025]|\|\allowbreak\token[0.0002]|]↩↩|\tighttracebreak\allowbreak\token[0.0684]|Since|\allowbreak\token[0.0009]| \(|\allowbreak\token[0.0009]|x|\allowbreak\token[0.0000]|\)|\allowbreak\token[0.0155]| must|\allowbreak\token[0.0000]| be|\allowbreak\token[0.0902]| an|\allowbreak\token[0.0000]| integer|\allowbreak\token[1.0248]|,|\allowbreak\token[0.0092]| the|\allowbreak\token[0.0003]| smallest|\allowbreak\token[0.6842]| possible|\allowbreak\token[0.0000]| value|\allowbreak\token[0.0001]| for|\allowbreak\token[0.0000]| \(|\allowbreak\token[0.0002]|x|\allowbreak\token[0.0000]|\)|\allowbreak\token[0.0747]| is|\allowbreak\token[0.0015]| |\allowbreak\token[0.0000]|2|\allowbreak\token[0.0005]|2|\allowbreak\token[0.0015]|.|\allowbreak\token[0.1082]| However|\allowbreak\token[0.0000]|,|\allowbreak\token[0.0402]| we|\allowbreak\token[0.2163]| need|\allowbreak\token[0.0001]| to|\allowbreak\token[0.7207]| check|\allowbreak\token[0.0002]| if|\allowbreak\token[0.5524]| this|\allowbreak\token[0.0862]| value|\allowbreak\token[0.6565]| of|\allowbreak\token[0.0000]| \(|\allowbreak\token[0.0002]|x|\allowbreak\token[0.0004]|\)|\allowbreak\token[1.3582]| satisfies|\allowbreak\token[0.0216]| the|\allowbreak\token[0.8785]| equation|\allowbreak\token[0.0071]| \(|\allowbreak\token[0.0001]|7|\allowbreak\token[0.0000]|x|\allowbreak\token[0.0000]| +|\allowbreak\token[0.0000]| |\allowbreak\token[0.0000]|5|\allowbreak\token[0.0000]|3|\allowbreak\token[0.0013]| =|\allowbreak\token[0.0000]| |\allowbreak\token[0.0000]|8|\allowbreak\token[0.0000]|y|\allowbreak\token[0.1453]|\|\allowbreak\token[0.0862]|).|\allowbreak\token[0.2563]| Sub|\allowbreak\token[0.0000]|stit|\allowbreak\token[0.0009]|uting|\allowbreak\token[0.0000]| \(|\allowbreak\token[0.0000]|x|\allowbreak\token[0.0000]| =|\allowbreak\token[0.0000]| |\allowbreak\token[0.0000]|2|\allowbreak\token[0.0001]|2|\allowbreak\token[0.5029]|\)|\allowbreak\token[0.0362]| into|\allowbreak\token[0.0001]| the|\allowbreak\token[0.0000]| equation|\allowbreak\token[0.6985]|:↩|\newline\allowbreak\token[0.0000]|\[|\allowbreak\token[0.0014]|7|\allowbreak\token[0.0102]|(|\allowbreak\token[0.0000]|2|\allowbreak\token[0.0000]|2|\allowbreak\token[0.0000]|)|\allowbreak\token[0.0000]| +|\allowbreak\token[0.0000]| |\allowbreak\token[0.0000]|5|\allowbreak\token[0.0000]|3|\allowbreak\token[0.0000]| =|\allowbreak\token[0.0000]| |\allowbreak\token[0.1525]|1|\allowbreak\token[0.2788]|5|\allowbreak\token[0.0097]|4|\allowbreak\token[0.0001]| +|\allowbreak\token[0.0000]| |\allowbreak\token[0.0000]|5|\allowbreak\token[0.0000]|3|\allowbreak\token[0.0000]| =|\allowbreak\token[0.0000]| |\allowbreak\token[0.0004]|2|\allowbreak\token[0.0000]|0|\allowbreak\token[0.0000]|7|\allowbreak\token[0.0065]|\|\allowbreak\token[0.1453]|]↩↩|\tighttracebreak\allowbreak\token[0.9716]|Since|\allowbreak\token[0.3397]| |\allowbreak\token[0.0000]|2|\allowbreak\token[0.0000]|0|\allowbreak[...]
\end{twopartbox}
\begin{twopartbox}{Cut Index 46 | Acceptance Probability 0.0000}{Block 2 Step 3}{Proposal (Rejected)}
\token[2.2824]| To|\allowbreak\token[0.4171]| solve|\allowbreak\token[0.4477]| this|\allowbreak\token[0.0840]| problem|\allowbreak\token[0.0111]|,|\allowbreak\token[0.5775]| let|\allowbreak\token[0.0029]|'s|\allowbreak\token[0.7813]| denote|\allowbreak\token[0.0220]| the|\allowbreak\token[0.8006]| number|\allowbreak\token[0.0001]| of|\allowbreak\token[0.1007]| coins|\allowbreak\token[0.0814]| in|\allowbreak\token[0.0184]| each|\allowbreak\token[0.0164]| of|\allowbreak\token[0.0034]| the|\allowbreak\token[0.5548]| original|\allowbreak\token[0.0362]| seven|\allowbreak\token[0.0001]| bags|\allowbreak\token[0.0010]| as|\allowbreak\token[0.6332]| \(|\allowbreak\token[0.7849]|x|\allowbreak\token[0.0053]|\|\allowbreak\token[0.1099]|).|\allowbreak\token[0.9478]| Therefore|\allowbreak\token[0.0000]|,|\allowbreak\token[0.6160]| the|\allowbreak\token[0.0009]| total|\allowbreak\token[0.0001]| number|\allowbreak\token[0.0000]| of|\allowbreak\token[0.0001]| coins|\allowbreak\token[0.5853]| before|\allowbreak\token[0.0061]| finding|\allowbreak\token[0.0003]| the|\allowbreak\token[0.7916]| additional|\allowbreak\token[0.3947]| bag|\allowbreak\token[0.8396]| is|\allowbreak\token[0.0000]| \(|\allowbreak\token[0.0001]|7|\allowbreak\token[0.0000]|x|\allowbreak\token[0.2691]|\|\allowbreak\token[0.0558]|).|\allowbreak\token[0.1432]| After|\allowbreak\token[0.0399]| finding|\allowbreak\token[0.0025]| the|\allowbreak\token[0.1605]| bag|\allowbreak\cut{[CUT]}\allowbreak\token[0.6926]| with|\allowbreak\token[0.0002]| |\allowbreak\token[0.0000]|5|\allowbreak\token[0.0001]|3|\allowbreak\token[0.0064]| coins|\allowbreak\token[0.1746]|,|\allowbreak\token[0.7223]| you|\allowbreak\token[0.2138]| redistribute|\allowbreak\token[0.8151]| the|\allowbreak\token[0.2692]| coins|\allowbreak\token[0.9168]| so|\allowbreak\token[0.0000]| that|\allowbreak\token[0.8742]| all|\allowbreak\token[0.0001]| eight|\allowbreak\token[0.0000]| bags|\allowbreak\token[0.0107]| have|\allowbreak\token[0.0498]| the|\allowbreak\token[0.0001]| same|\allowbreak\token[0.0038]| number|\allowbreak\token[0.1279]| of|\allowbreak\token[0.0000]| coins|\allowbreak\token[0.6799]|.|\allowbreak\token[0.6893]| Let|\allowbreak\token[0.0016]|'s|\allowbreak\token[0.4479]| denote|\allowbreak\token[0.6060]| the|\allowbreak\token[0.1672]| new|\allowbreak\token[0.0997]| number|\allowbreak\token[0.0000]| of|\allowbreak\token[0.0000]| coins|\allowbreak\token[0.1204]| in|\allowbreak\token[0.0001]| each|\allowbreak\token[0.2477]| bag|\allowbreak\token[0.1461]| as|\allowbreak\token[0.0000]| \(|\allowbreak\token[0.0029]|y|\allowbreak\token[0.0000]|\|\allowbreak\token[0.1042]|).|\allowbreak\token[1.2032]| The|\allowbreak\token[0.0333]| total|\allowbreak\token[0.0000]| number|\allowbreak\token[0.0000]| of|\allowbreak\token[0.0000]| coins|\allowbreak\token[0.2547]| after|\allowbreak\token[0.3218]| redistribution|\allowbreak\token[0.0016]| is|\allowbreak\token[0.0207]| \(|\allowbreak\token[0.6968]|7|\allowbreak\token[0.0060]|x|\allowbreak\token[0.0000]| +|\allowbreak\token[0.0000]| |\allowbreak\token[0.0000]|5|\allowbreak\token[0.0000]|3|\allowbreak\token[0.5331]|\|\allowbreak\token[0.0273]|),|\allowbreak\token[0.3655]| and|\allowbreak\token[0.5117]| this|\allowbreak\token[0.9310]| total|\allowbreak\token[0.8310]| is|\allowbreak\token[1.5498]| evenly|\allowbreak\token[0.3380]| divisible|\allowbreak\token[0.0038]| by|\allowbreak\token[0.0012]| |\allowbreak\token[0.0002]|8|\allowbreak\token[0.4522]|,|\allowbreak\token[1.0927]| so|\allowbreak\token[0.7020]| we|\allowbreak\token[0.0113]| have|\allowbreak\token[0.9266]|:↩|\newline\allowbreak\token[0.0074]|\[|\allowbreak\token[0.0236]|7|\allowbreak\token[0.0001]|x|\allowbreak\token[0.0000]| +|\allowbreak\token[0.0000]| |\allowbreak\token[0.0000]|5|\allowbreak\token[0.0002]|3|\allowbreak\token[0.0001]| =|\allowbreak\token[0.0000]| |\allowbreak\token[0.0000]|8|\allowbreak\token[0.0000]|y|\allowbreak\token[0.0003]|\|\allowbreak\token[0.3655]|]↩↩|\tighttracebreak\allowbreak\token[0.7411]|We|\allowbreak\token[0.3478]| also|\allowbreak\token[0.0000]| know|\allowbreak\token[0.0001]| that|\allowbreak\token[0.2705]| the|\allowbreak\token[0.0020]| total|\allowbreak\token[0.0055]| number|\allowbreak\token[0.0000]| of|\allowbreak\token[0.0000]| coins|\allowbreak\token[0.7427]| after|\allowbreak\token[0.0057]| redistribution|\allowbreak\token[0.0063]| is|\allowbreak\token[0.0218]| more|\allowbreak\token[0.0000]| than|\allowbreak\token[0.0000]| |\allowbreak\token[0.0000]|2|\allowbreak\token[0.0000]|0|\allowbreak\token[0.0002]|0|\allowbreak\token[0.3491]|,|\allowbreak\token[0.0685]| so|\allowbreak\token[0.0566]|:↩|\newline\allowbreak\token[0.0000]|\[|\allowbreak\token[0.5834]|7|\allowbreak\token[0.0002]|x|\allowbreak\token[0.0000]| +|\allowbreak\token[0.0000]| |\allowbreak\token[0.0000]|5|\allowbreak\token[0.0001]|3|\allowbreak\token[0.0001]| >|\allowbreak\token[0.0000]| |\allowbreak\token[0.0000]|2|\allowbreak\token[0.0000]|0|\allowbreak\token[0.0002]|0|\allowbreak\token[0.0000]|\|\allowbreak\token[0.4472]|]↩↩|\tighttracebreak\allowbreak\token[0.7990]|To|\allowbreak\token[0.0005]| find|\allowbreak\token[0.0001]| the|\allowbreak\token[0.0020]| smallest|\allowbreak\token[0.0738]| number|\allowbreak\token[0.0000]| of|\allowbreak\token[0.0013]| coins|\allowbreak\token[0.9062]| you|\allowbreak\token[0.0001]| could|\allowbreak\token[0.0000]| have|\allowbreak\token[0.0037]| had|\allowbreak\token[0.5101]| before|\allowbreak\token[0.0058]| finding|\allowbreak\token[0.0000]| the|\allowbreak\token[0.2427]| bag|\allowbreak\token[0.0011]| of|\allowbreak\token[0.0000]| |\allowbreak\token[0.0000]|5|\allowbreak\token[0.0001]|3|\allowbreak\token[0.0047]| coins|\allowbreak\token[0.0266]|,|\allowbreak\token[0.0002]| we|\allowbreak\token[0.0070]| need|\allowbreak\token[0.0000]| to|\allowbreak\token[0.0403]| find|\allowbreak\token[0.0000]| the|\allowbreak\token[0.0005]| smallest|\allowbreak[...]
\tcblower
\token[0.6869]| with|\allowbreak\token[0.0002]| |\allowbreak\token[0.0000]|5|\allowbreak\token[0.0001]|3|\allowbreak\token[0.0064]| coins|\allowbreak\token[0.1746]|,|\allowbreak\token[0.7219]| you|\allowbreak\token[0.2139]| redistribute|\allowbreak\token[0.8151]| the|\allowbreak\token[0.2479]| coins|\allowbreak\token[0.9387]| so|\allowbreak\token[0.0000]| that|\allowbreak\token[0.8521]| all|\allowbreak\token[0.0001]| eight|\allowbreak\token[0.0000]| bags|\allowbreak\token[0.0098]| have|\allowbreak\token[0.0498]| the|\allowbreak\token[0.0001]| same|\allowbreak\token[0.0038]| number|\allowbreak\token[0.1288]| of|\allowbreak\token[0.0000]| coins|\allowbreak\token[0.6670]|.|\allowbreak\token[0.6951]| Let|\allowbreak\token[0.0017]|'s|\allowbreak\token[0.4202]| denote|\allowbreak\token[0.6277]| the|\allowbreak\token[0.1673]| new|\allowbreak\token[0.0988]| number|\allowbreak\token[0.0000]| of|\allowbreak\token[0.0000]| coins|\allowbreak\token[0.1204]| in|\allowbreak\token[0.0001]| each|\allowbreak\token[0.2477]| bag|\allowbreak\token[0.1461]| as|\allowbreak\token[0.0000]| \(|\allowbreak\token[0.0026]|y|\allowbreak\token[0.0000]|\|\allowbreak\token[0.1039]|).|\allowbreak\token[1.2022]| Therefore|\allowbreak\token[0.0000]|,|\allowbreak\token[0.0601]| the|\allowbreak\token[0.0328]| total|\allowbreak\token[0.0000]| number|\allowbreak\token[0.0000]| of|\allowbreak\token[0.0000]| coins|\allowbreak\token[0.2847]| after|\allowbreak\token[0.1077]| redistribution|\allowbreak\token[0.0006]| is|\allowbreak\token[0.0007]| \(|\allowbreak\token[0.0279]|8|\allowbreak\token[0.0000]|y|\allowbreak\token[0.0083]|\|\allowbreak\token[0.6695]|).↩↩|\tighttracebreak\allowbreak\token[0.4782]|Given|\allowbreak\token[0.0008]| that|\allowbreak\token[0.6183]| you|\allowbreak\token[0.4592]| have|\allowbreak\token[0.0316]| more|\allowbreak\token[0.0000]| than|\allowbreak\token[0.0000]| |\allowbreak\token[0.0000]|2|\allowbreak\token[0.0000]|0|\allowbreak\token[0.0001]|0|\allowbreak\token[0.0000]| coins|\allowbreak\token[1.0446]| after|\allowbreak\token[0.1546]| redistribution|\allowbreak\token[0.0367]|,|\allowbreak\token[0.0284]| we|\allowbreak\token[0.4651]| have|\allowbreak\token[0.5852]| \(|\allowbreak\token[0.0007]|8|\allowbreak\token[0.0000]|y|\allowbreak\token[0.0000]| >|\allowbreak\token[0.0000]| |\allowbreak\token[0.0000]|2|\allowbreak\token[0.0000]|0|\allowbreak\token[0.0000]|0|\allowbreak\token[0.0002]|\|\allowbreak\token[0.4479]|),|\allowbreak\token[0.0004]| which|\allowbreak\token[0.0876]| simpl|\allowbreak\token[0.0000]|ifies|\allowbreak\token[0.0000]| to|\allowbreak\token[0.0000]| \(|\allowbreak\token[0.0000]|y|\allowbreak\token[0.0000]| >|\allowbreak\token[0.0000]| |\allowbreak\token[0.0000]|2|\allowbreak\token[0.0000]|5|\allowbreak\token[0.0002]|\|\allowbreak\token[0.2914]|).|\allowbreak\token[0.6398]| Since|\allowbreak\token[0.9221]| \(|\allowbreak\token[0.1132]|y|\allowbreak\token[0.0364]|\)|\allowbreak\token[0.7748]| must|\allowbreak\token[0.0001]| be|\allowbreak\token[0.1097]| an|\allowbreak\token[0.0000]| integer|\allowbreak\token[0.6042]|,|\allowbreak\token[0.0021]| the|\allowbreak\token[0.0038]| smallest|\allowbreak\token[0.4706]| possible|\allowbreak\token[0.0000]| value|\allowbreak\token[0.0000]| for|\allowbreak\token[0.0000]| \(|\allowbreak\token[0.0000]|y|\allowbreak\token[0.0000]|\)|\allowbreak\token[0.0140]| is|\allowbreak\token[0.0156]| |\allowbreak\token[0.0001]|2|\allowbreak\token[0.0010]|6|\allowbreak\token[0.5709]|.↩↩|\tighttracebreak\allowbreak\token[0.8820]|The|\allowbreak\token[0.0474]| total|\allowbreak\token[0.0001]| number|\allowbreak\token[0.0000]| of|\allowbreak\token[0.0000]| coins|\allowbreak\token[0.2551]| after|\allowbreak\token[0.1406]| redistribution|\allowbreak\token[0.0232]| is|\allowbreak\token[0.2485]| \(|\allowbreak\token[0.2971]|8|\allowbreak\token[0.4475]| \|\allowbreak\token[0.0000]|times|\allowbreak\token[0.0001]| |\allowbreak\token[0.0000]|2|\allowbreak\token[0.0000]|6|\allowbreak\token[0.0000]| =|\allowbreak\token[0.0000]| |\allowbreak\token[0.0007]|2|\allowbreak\token[0.0000]|0|\allowbreak\token[0.0000]|8|\allowbreak\token[0.0011]|\|\allowbreak\token[0.0890]|).|\allowbreak\token[1.3458]| Before|\allowbreak\token[0.6909]| finding|\allowbreak\token[0.0001]| the|\allowbreak\token[0.5426]| additional|\allowbreak\token[0.1211]| bag|\allowbreak\token[0.4545]|,|\allowbreak\token[0.6504]| you|\allowbreak\token[0.0362]| had|\allowbreak\token[0.0461]| \(|\allowbreak[...]
\end{twopartbox}
\begin{twopartbox}{Cut Index 266 | Acceptance Probability 0.0000}{Block 3 Step 1}{Proposal (Rejected)}
[...]\allowbreak\token[0.0001]| =|\allowbreak\token[0.0000]| |\allowbreak\token[0.0000]|8|\allowbreak\token[0.0000]|y|\allowbreak\token[0.0003]|\|\allowbreak\token[0.3655]|]↩↩|\tighttracebreak\allowbreak\token[0.7411]|We|\allowbreak\token[0.3478]| also|\allowbreak\token[0.0000]| know|\allowbreak\token[0.0001]| that|\allowbreak\token[0.2705]| the|\allowbreak\token[0.0020]| total|\allowbreak\token[0.0055]| number|\allowbreak\token[0.0000]| of|\allowbreak\token[0.0000]| coins|\allowbreak\token[0.7427]| after|\allowbreak\token[0.0057]| redistribution|\allowbreak\token[0.0063]| is|\allowbreak\token[0.0218]| more|\allowbreak\token[0.0000]| than|\allowbreak\token[0.0000]| |\allowbreak\token[0.0000]|2|\allowbreak\token[0.0000]|0|\allowbreak\token[0.0002]|0|\allowbreak\token[0.3491]|,|\allowbreak\token[0.0685]| so|\allowbreak\token[0.0566]|:↩|\newline\allowbreak\token[0.0000]|\[|\allowbreak\token[0.5834]|7|\allowbreak\token[0.0002]|x|\allowbreak\token[0.0000]| +|\allowbreak\token[0.0000]| |\allowbreak\token[0.0000]|5|\allowbreak\token[0.0001]|3|\allowbreak\token[0.0001]| >|\allowbreak\token[0.0000]| |\allowbreak\token[0.0000]|2|\allowbreak\token[0.0000]|0|\allowbreak\token[0.0002]|0|\allowbreak\token[0.0000]|\|\allowbreak\token[0.4472]|]↩↩|\tighttracebreak\allowbreak\token[0.7990]|To|\allowbreak\token[0.0005]| find|\allowbreak\token[0.0001]| the|\allowbreak\token[0.0020]| smallest|\allowbreak\token[0.0738]| number|\allowbreak\token[0.0000]| of|\allowbreak\token[0.0013]| coins|\allowbreak\token[0.9062]| you|\allowbreak\token[0.0001]| could|\allowbreak\token[0.0000]| have|\allowbreak\token[0.0037]| had|\allowbreak\token[0.5101]| before|\allowbreak\token[0.0058]| finding|\allowbreak\token[0.0000]| the|\allowbreak\token[0.2427]| bag|\allowbreak\token[0.0011]| of|\allowbreak\token[0.0000]| |\allowbreak\token[0.0000]|5|\allowbreak\token[0.0001]|3|\allowbreak\token[0.0047]| coins|\allowbreak\token[0.0266]|,|\allowbreak\token[0.0002]| we|\allowbreak\token[0.0070]| need|\allowbreak\token[0.0000]| to|\allowbreak\token[0.0403]| find|\allowbreak\token[0.0000]| the|\allowbreak\token[0.0005]| smallest|\allowbreak\token[0.7712]| value|\allowbreak\token[0.0000]| of|\allowbreak\token[0.0000]| \(|\allowbreak\token[0.1778]|x|\allowbreak\token[0.0000]|\)|\allowbreak\token[0.0001]| that|\allowbreak\token[0.0001]| satisfies|\allowbreak\token[0.6093]| these|\allowbreak\token[0.0000]| conditions|\allowbreak\token[0.1894]|.|\allowbreak\token[0.7665]| Let|\allowbreak\token[0.0000]|'s|\allowbreak\token[0.8862]| start|\allowbreak\token[0.0004]| by|\allowbreak\token[0.7551]| solving|\allowbreak\token[0.0742]| the|\allowbreak\token[0.2850]| inequality|\allowbreak\token[0.6954]| for|\allowbreak\token[0.0908]| \(|\allowbreak\token[0.0324]|x|\allowbreak\token[0.0000]|\|\allowbreak\token[0.0000]|):↩|\newline\allowbreak\token[0.0001]|\[|\allowbreak\token[0.0001]|7|\allowbreak\token[0.0000]|x|\allowbreak\token[0.0140]| >|\allowbreak\token[0.0000]| |\allowbreak\token[0.0669]|1|\allowbreak\token[0.0000]|4|\allowbreak\token[0.0001]|7|\allowbreak\token[0.0001]|\|\allowbreak\token[0.0010]|]↩|\newline\allowbreak\token[0.0005]|\[|\allowbreak\token[0.0000]|x|\allowbreak\token[0.0000]| >|\allowbreak\token[0.0012]| |\allowbreak\token[0.0000]|2|\allowbreak\token[0.0000]|1|\allowbreak\token[0.0025]|\|\allowbreak\token[0.0002]|]↩↩|\tighttracebreak\allowbreak\token[0.0684]|Since|\allowbreak\token[0.0009]| \(|\allowbreak\token[0.0009]|x|\allowbreak\token[0.0000]|\)|\allowbreak\token[0.0155]| must|\allowbreak\token[0.0000]| be|\allowbreak\token[0.0902]| an|\allowbreak\token[0.0000]| integer|\allowbreak\token[1.0248]|,|\allowbreak\token[0.0092]| the|\allowbreak\token[0.0003]| smallest|\allowbreak\token[0.6842]| possible|\allowbreak\token[0.0000]| value|\allowbreak\token[0.0001]| for|\allowbreak\token[0.0000]| \(|\allowbreak\token[0.0002]|x|\allowbreak\token[0.0000]|\)|\allowbreak\token[0.0747]| is|\allowbreak\token[0.0015]| |\allowbreak\token[0.0000]|2|\allowbreak\token[0.0005]|2|\allowbreak\token[0.0015]|.|\allowbreak\token[0.1082]| However|\allowbreak\token[0.0000]|,|\allowbreak\token[0.0402]| we|\allowbreak\token[0.2163]| need|\allowbreak\token[0.0001]| to|\allowbreak\token[0.7207]| check|\allowbreak\token[0.0002]| if|\allowbreak\token[0.5524]| this|\allowbreak\token[0.0862]| value|\allowbreak\token[0.6565]| of|\allowbreak\token[0.0000]| \(|\allowbreak\token[0.0002]|x|\allowbreak\token[0.0004]|\)|\allowbreak\token[1.3582]| satisfies|\allowbreak\cut{[CUT]}\allowbreak\token[0.0216]| the|\allowbreak\token[0.8785]| equation|\allowbreak\token[0.0071]| \(|\allowbreak\token[0.0001]|7|\allowbreak\token[0.0000]|x|\allowbreak\token[0.0000]| +|\allowbreak\token[0.0000]| |\allowbreak\token[0.0000]|5|\allowbreak\token[0.0000]|3|\allowbreak\token[0.0013]| =|\allowbreak\token[0.0000]| |\allowbreak\token[0.0000]|8|\allowbreak\token[0.0000]|y|\allowbreak\token[0.1453]|\|\allowbreak\token[0.0862]|).|\allowbreak\token[0.2563]| Sub|\allowbreak\token[0.0000]|stit|\allowbreak\token[0.0009]|uting|\allowbreak\token[0.0000]| \(|\allowbreak\token[0.0000]|x|\allowbreak\token[0.0000]| =|\allowbreak\token[0.0000]| |\allowbreak\token[0.0000]|2|\allowbreak\token[0.0001]|2|\allowbreak\token[0.5029]|\)|\allowbreak\token[0.0362]| into|\allowbreak\token[0.0001]| the|\allowbreak\token[0.0000]| equation|\allowbreak\token[0.6985]|:↩|\newline\allowbreak\token[0.0000]|\[|\allowbreak\token[0.0014]|7|\allowbreak\token[0.0102]|(|\allowbreak\token[0.0000]|2|\allowbreak\token[0.0000]|2|\allowbreak\token[0.0000]|)|\allowbreak\token[0.0000]| +|\allowbreak\token[0.0000]| |\allowbreak\token[0.0000]|5|\allowbreak\token[0.0000]|3|\allowbreak\token[0.0000]| =|\allowbreak\token[0.0000]| |\allowbreak\token[0.1525]|1|\allowbreak\token[0.2788]|5|\allowbreak\token[0.0097]|4|\allowbreak\token[0.0001]| +|\allowbreak\token[0.0000]| |\allowbreak\token[0.0000]|5|\allowbreak\token[0.0000]|3|\allowbreak\token[0.0000]| =|\allowbreak\token[0.0000]| |\allowbreak\token[0.0004]|2|\allowbreak\token[0.0000]|0|\allowbreak\token[0.0000]|7|\allowbreak\token[0.0065]|\|\allowbreak\token[0.1453]|]↩↩|\tighttracebreak\allowbreak\token[0.9716]|Since|\allowbreak\token[0.3397]| |\allowbreak\token[0.0000]|2|\allowbreak\token[0.0000]|0|\allowbreak\token[0.0000]|7|\allowbreak\token[0.0015]| is|\allowbreak\token[0.6322]| not|\allowbreak\token[0.0330]| divisible|\allowbreak\token[0.0001]| by|\allowbreak\token[0.0000]| |\allowbreak\token[0.0001]|8|\allowbreak\token[0.0003]|,|\allowbreak\token[0.4347]| \(|\allowbreak\token[0.0000]|x|\allowbreak\token[0.0002]| =|\allowbreak\token[0.0000]| |\allowbreak\token[0.0000]|2|\allowbreak\token[0.0000]|2|\allowbreak\token[0.0000]|\)|\allowbreak\token[0.4471]| is|\allowbreak\token[0.0000]| not|\allowbreak\token[0.5029]| a|\allowbreak\token[0.6630]| solution|\allowbreak\token[0.0000]|.|\allowbreak\token[0.1803]| We|\allowbreak\token[0.7428]| need|\allowbreak\token[0.0015]| to|\allowbreak\token[0.0065]| find|\allowbreak\token[0.0014]| the|\allowbreak\token[0.2272]| smallest|\allowbreak\token[0.1538]| \(|\allowbreak\token[0.0000]|x|\allowbreak\token[0.0326]|\)|\allowbreak\token[0.7490]| that|\allowbreak\token[0.3681]| makes|\allowbreak\token[0.0087]| \(|\allowbreak\token[0.0003]|7|\allowbreak\token[0.0000]|x|\allowbreak\token[0.0000]| +|\allowbreak\token[0.0000]| |\allowbreak\token[0.0000]|5|\allowbreak\token[0.0001]|3|\allowbreak\token[0.0000]|\)|\allowbreak\token[0.0364]| divisible|\allowbreak\token[0.0000]| by|\allowbreak\token[0.0000]| |\allowbreak\token[0.0000]|8|\allowbreak\token[0.6856]| and|\allowbreak\token[0.3847]| greater|\allowbreak\token[0.0000]| than|\allowbreak\token[0.0001]| |\allowbreak\token[0.0016]|2|\allowbreak\token[0.0001]|0|\allowbreak\token[0.0090]|0|\allowbreak\token[0.0009]|.|\allowbreak\token[0.6951]| Let|\allowbreak\token[0.0000]|'s|\allowbreak\token[0.7438]| try|\allowbreak\token[0.0802]| \(|\allowbreak\token[0.0000]|x|\allowbreak\token[0.0000]| =|\allowbreak\token[0.0000]| |\allowbreak\token[0.0029]|2|\allowbreak\token[0.1013]|3|\allowbreak\token[0.0011]|\|\allowbreak\token[0.0000]|):↩|\newline\allowbreak\token[0.0000]|\[|\allowbreak\token[0.0000]|7|\allowbreak\token[0.0000]|(|\allowbreak\token[0.0000]|2|\allowbreak\token[0.0001]|3|\allowbreak\token[0.0000]|)|\allowbreak\token[0.0000]| +|\allowbreak\token[0.0000]| |\allowbreak\token[0.0000]|5|\allowbreak\token[0.0000]|3|\allowbreak\token[0.0000]| =|\allowbreak\token[0.0000]| |\allowbreak\token[0.0017]|1|\allowbreak\token[0.0011]|6|\allowbreak\token[0.0001]|1|\allowbreak\token[0.0000]| +|\allowbreak\token[0.0000]| |\allowbreak\token[0.0000]|5|\allowbreak\token[0.0000]|3|\allowbreak\token[0.0000]| =|\allowbreak\token[0.0000]| |\allowbreak\token[0.0000]|2|\allowbreak\token[0.0000]|1|\allowbreak[...]
\tcblower
\token[0.0216]| the|\allowbreak\token[0.8787]| equation|\allowbreak\token[0.0064]| \(|\allowbreak\token[0.0001]|7|\allowbreak\token[0.0000]|x|\allowbreak\token[0.0000]| +|\allowbreak\token[0.0000]| |\allowbreak\token[0.0000]|5|\allowbreak\token[0.0000]|3|\allowbreak\token[0.0012]| =|\allowbreak\token[0.0000]| |\allowbreak\token[0.0000]|8|\allowbreak\token[0.0000]|y|\allowbreak\token[0.1594]|\|\allowbreak\token[0.0863]|).|\allowbreak\token[0.2792]| Sub|\allowbreak\token[0.0000]|stit|\allowbreak\token[0.0008]|uting|\allowbreak\token[0.0000]| \(|\allowbreak\token[0.0000]|x|\allowbreak\token[0.0000]| =|\allowbreak\token[0.0000]| |\allowbreak\token[0.0000]|2|\allowbreak\token[0.0001]|2|\allowbreak\token[0.5029]|\)|\allowbreak\token[0.0362]| into|\allowbreak\token[0.0001]| the|\allowbreak\token[0.0000]| equation|\allowbreak\token[0.6985]|:↩|\newline\allowbreak\token[0.0000]|\[|\allowbreak\token[0.0014]|7|\allowbreak\token[0.0102]|(|\allowbreak\token[0.0000]|2|\allowbreak\token[0.0000]|2|\allowbreak\token[0.0000]|)|\allowbreak\token[0.0000]| +|\allowbreak\token[0.0000]| |\allowbreak\token[0.0000]|5|\allowbreak\token[0.0000]|3|\allowbreak\token[0.0000]| =|\allowbreak\token[0.0000]| |\allowbreak\token[0.1520]|1|\allowbreak\token[0.3065]|5|\allowbreak\token[0.0088]|4|\allowbreak\token[0.0001]| +|\allowbreak\token[0.0000]| |\allowbreak\token[0.0000]|5|\allowbreak\token[0.0000]|3|\allowbreak\token[0.0000]| =|\allowbreak\token[0.0000]| |\allowbreak\token[0.0004]|2|\allowbreak\token[0.0000]|0|\allowbreak\token[0.0000]|7|\allowbreak\token[0.0073]|\|\allowbreak\token[0.1594]|]↩↩|\tighttracebreak\allowbreak\token[0.9729]|Since|\allowbreak\token[0.3397]| |\allowbreak\token[0.0000]|2|\allowbreak\token[0.0000]|0|\allowbreak\token[0.0000]|7|\allowbreak\token[0.0015]| is|\allowbreak\token[0.6329]| not|\allowbreak\token[0.0330]| divisible|\allowbreak\token[0.0001]| by|\allowbreak\token[0.0000]| |\allowbreak\token[0.0001]|8|\allowbreak\token[0.0004]|,|\allowbreak\token[0.5078]| \(|\allowbreak\token[0.0000]|x|\allowbreak\token[0.0002]| =|\allowbreak\token[0.0000]| |\allowbreak\token[0.0000]|2|\allowbreak\token[0.0000]|2|\allowbreak\token[0.0000]|\)|\allowbreak\token[0.4751]| is|\allowbreak\token[0.0000]| not|\allowbreak\token[0.5303]| a|\allowbreak\token[0.6276]| solution|\allowbreak\token[0.0000]|.|\allowbreak\token[0.2063]| We|\allowbreak\token[0.7427]| need|\allowbreak\token[0.0015]| to|\allowbreak\token[0.0052]| find|\allowbreak\token[0.0014]| the|\allowbreak\token[0.2272]| smallest|\allowbreak\token[0.1585]| \(|\allowbreak\token[0.0000]|x|\allowbreak\token[0.0326]|\)|\allowbreak\token[0.7485]| that|\allowbreak\token[0.4373]| makes|\allowbreak\token[0.0086]| \(|\allowbreak\token[0.0003]|7|\allowbreak\token[0.0000]|x|\allowbreak\token[0.0000]| +|\allowbreak\token[0.0000]| |\allowbreak\token[0.0000]|5|\allowbreak\token[0.0001]|3|\allowbreak\token[0.0000]|\)|\allowbreak\token[0.0328]| divisible|\allowbreak\token[0.0000]| by|\allowbreak\token[0.0000]| |\allowbreak\token[0.0000]|8|\allowbreak\token[0.6761]|.|\allowbreak\token[1.1875]| By|\allowbreak\token[0.9035]| trial|\allowbreak\token[0.0001]| and|\allowbreak\token[0.0000]| error|\allowbreak\token[0.0327]| or|\allowbreak\token[0.8100]| systematic|\allowbreak\token[0.5689]| calculation|\allowbreak\token[0.0000]|,|\allowbreak\token[0.0000]| we|\allowbreak\token[0.0000]| find|\allowbreak\token[0.0000]| that|\allowbreak\token[0.1047]| \(|\allowbreak\token[0.0000]|x|\allowbreak\token[0.0000]| =|\allowbreak\token[0.0000]| |\allowbreak\token[0.0986]|2|\allowbreak\token[1.3263]|3|\allowbreak\token[0.0001]|\)|\allowbreak\token[0.1468]| is|\allowbreak\token[0.0005]| the|\allowbreak\token[0.0786]| smallest|\allowbreak\token[0.0756]| value|\allowbreak\token[0.0059]| that|\allowbreak\token[0.5326]| works|\allowbreak\token[0.4333]|:↩|\newline\allowbreak\token[0.0000]|\[|\allowbreak\token[0.0000]|7|\allowbreak\token[0.0000]|(|\allowbreak\token[0.0000]|2|\allowbreak\token[0.0000]|3|\allowbreak\token[0.0000]|)|\allowbreak\token[0.0001]| +|\allowbreak\token[0.0000]| |\allowbreak\token[0.0000]|5|\allowbreak\token[0.0000]|3|\allowbreak\token[0.0000]| =|\allowbreak\token[0.0000]| |\allowbreak\token[0.0053]|1|\allowbreak\token[0.0026]|6|\allowbreak\token[0.0003]|1|\allowbreak\token[0.0000]| +|\allowbreak[...]
\end{twopartbox}
\begin{twopartbox}{Cut Index 299 | Acceptance Probability 0.0011}{Block 3 Step 2}{Proposal (Rejected)}
[...]\allowbreak\token[0.0001]|3|\allowbreak\token[0.0001]| >|\allowbreak\token[0.0000]| |\allowbreak\token[0.0000]|2|\allowbreak\token[0.0000]|0|\allowbreak\token[0.0002]|0|\allowbreak\token[0.0000]|\|\allowbreak\token[0.4472]|]↩↩|\tighttracebreak\allowbreak\token[0.7990]|To|\allowbreak\token[0.0005]| find|\allowbreak\token[0.0001]| the|\allowbreak\token[0.0020]| smallest|\allowbreak\token[0.0738]| number|\allowbreak\token[0.0000]| of|\allowbreak\token[0.0013]| coins|\allowbreak\token[0.9062]| you|\allowbreak\token[0.0001]| could|\allowbreak\token[0.0000]| have|\allowbreak\token[0.0037]| had|\allowbreak\token[0.5101]| before|\allowbreak\token[0.0058]| finding|\allowbreak\token[0.0000]| the|\allowbreak\token[0.2427]| bag|\allowbreak\token[0.0011]| of|\allowbreak\token[0.0000]| |\allowbreak\token[0.0000]|5|\allowbreak\token[0.0001]|3|\allowbreak\token[0.0047]| coins|\allowbreak\token[0.0266]|,|\allowbreak\token[0.0002]| we|\allowbreak\token[0.0070]| need|\allowbreak\token[0.0000]| to|\allowbreak\token[0.0403]| find|\allowbreak\token[0.0000]| the|\allowbreak\token[0.0005]| smallest|\allowbreak\token[0.7712]| value|\allowbreak\token[0.0000]| of|\allowbreak\token[0.0000]| \(|\allowbreak\token[0.1778]|x|\allowbreak\token[0.0000]|\)|\allowbreak\token[0.0001]| that|\allowbreak\token[0.0001]| satisfies|\allowbreak\token[0.6093]| these|\allowbreak\token[0.0000]| conditions|\allowbreak\token[0.1894]|.|\allowbreak\token[0.7665]| Let|\allowbreak\token[0.0000]|'s|\allowbreak\token[0.8862]| start|\allowbreak\token[0.0004]| by|\allowbreak\token[0.7551]| solving|\allowbreak\token[0.0742]| the|\allowbreak\token[0.2850]| inequality|\allowbreak\token[0.6954]| for|\allowbreak\token[0.0908]| \(|\allowbreak\token[0.0324]|x|\allowbreak\token[0.0000]|\|\allowbreak\token[0.0000]|):↩|\newline\allowbreak\token[0.0001]|\[|\allowbreak\token[0.0001]|7|\allowbreak\token[0.0000]|x|\allowbreak\token[0.0140]| >|\allowbreak\token[0.0000]| |\allowbreak\token[0.0669]|1|\allowbreak\token[0.0000]|4|\allowbreak\token[0.0001]|7|\allowbreak\token[0.0001]|\|\allowbreak\token[0.0010]|]↩|\newline\allowbreak\token[0.0005]|\[|\allowbreak\token[0.0000]|x|\allowbreak\token[0.0000]| >|\allowbreak\token[0.0012]| |\allowbreak\token[0.0000]|2|\allowbreak\token[0.0000]|1|\allowbreak\token[0.0025]|\|\allowbreak\token[0.0002]|]↩↩|\tighttracebreak\allowbreak\token[0.0684]|Since|\allowbreak\token[0.0009]| \(|\allowbreak\token[0.0009]|x|\allowbreak\token[0.0000]|\)|\allowbreak\token[0.0155]| must|\allowbreak\token[0.0000]| be|\allowbreak\token[0.0902]| an|\allowbreak\token[0.0000]| integer|\allowbreak\token[1.0248]|,|\allowbreak\token[0.0092]| the|\allowbreak\token[0.0003]| smallest|\allowbreak\token[0.6842]| possible|\allowbreak\token[0.0000]| value|\allowbreak\token[0.0001]| for|\allowbreak\token[0.0000]| \(|\allowbreak\token[0.0002]|x|\allowbreak\token[0.0000]|\)|\allowbreak\token[0.0747]| is|\allowbreak\token[0.0015]| |\allowbreak\token[0.0000]|2|\allowbreak\token[0.0005]|2|\allowbreak\token[0.0015]|.|\allowbreak\token[0.1082]| However|\allowbreak\token[0.0000]|,|\allowbreak\token[0.0402]| we|\allowbreak\token[0.2163]| need|\allowbreak\token[0.0001]| to|\allowbreak\token[0.7207]| check|\allowbreak\token[0.0002]| if|\allowbreak\token[0.5524]| this|\allowbreak\token[0.0862]| value|\allowbreak\token[0.6565]| of|\allowbreak\token[0.0000]| \(|\allowbreak\token[0.0002]|x|\allowbreak\token[0.0004]|\)|\allowbreak\token[1.3582]| satisfies|\allowbreak\token[0.0216]| the|\allowbreak\token[0.8785]| equation|\allowbreak\token[0.0071]| \(|\allowbreak\token[0.0001]|7|\allowbreak\token[0.0000]|x|\allowbreak\token[0.0000]| +|\allowbreak\token[0.0000]| |\allowbreak\token[0.0000]|5|\allowbreak\token[0.0000]|3|\allowbreak\token[0.0013]| =|\allowbreak\token[0.0000]| |\allowbreak\token[0.0000]|8|\allowbreak\token[0.0000]|y|\allowbreak\token[0.1453]|\|\allowbreak\token[0.0862]|).|\allowbreak\token[0.2563]| Sub|\allowbreak\token[0.0000]|stit|\allowbreak\token[0.0009]|uting|\allowbreak\token[0.0000]| \(|\allowbreak\token[0.0000]|x|\allowbreak\token[0.0000]| =|\allowbreak\token[0.0000]| |\allowbreak\token[0.0000]|2|\allowbreak\token[0.0001]|2|\allowbreak\token[0.5029]|\)|\allowbreak\token[0.0362]| into|\allowbreak\token[0.0001]| the|\allowbreak\token[0.0000]| equation|\allowbreak\token[0.6985]|:↩|\newline\allowbreak\token[0.0000]|\[|\allowbreak\token[0.0014]|7|\allowbreak\token[0.0102]|(|\allowbreak\token[0.0000]|2|\allowbreak\cut{[CUT]}\allowbreak\token[0.0000]|2|\allowbreak\token[0.0000]|)|\allowbreak\token[0.0000]| +|\allowbreak\token[0.0000]| |\allowbreak\token[0.0000]|5|\allowbreak\token[0.0000]|3|\allowbreak\token[0.0000]| =|\allowbreak\token[0.0000]| |\allowbreak\token[0.1525]|1|\allowbreak\token[0.2788]|5|\allowbreak\token[0.0097]|4|\allowbreak\token[0.0001]| +|\allowbreak\token[0.0000]| |\allowbreak\token[0.0000]|5|\allowbreak\token[0.0000]|3|\allowbreak\token[0.0000]| =|\allowbreak\token[0.0000]| |\allowbreak\token[0.0004]|2|\allowbreak\token[0.0000]|0|\allowbreak\token[0.0000]|7|\allowbreak\token[0.0065]|\|\allowbreak\token[0.1453]|]↩↩|\tighttracebreak\allowbreak\token[0.9716]|Since|\allowbreak\token[0.3397]| |\allowbreak\token[0.0000]|2|\allowbreak\token[0.0000]|0|\allowbreak\token[0.0000]|7|\allowbreak\token[0.0015]| is|\allowbreak\token[0.6322]| not|\allowbreak\token[0.0330]| divisible|\allowbreak\token[0.0001]| by|\allowbreak\token[0.0000]| |\allowbreak\token[0.0001]|8|\allowbreak\token[0.0003]|,|\allowbreak\token[0.4347]| \(|\allowbreak\token[0.0000]|x|\allowbreak\token[0.0002]| =|\allowbreak\token[0.0000]| |\allowbreak\token[0.0000]|2|\allowbreak\token[0.0000]|2|\allowbreak\token[0.0000]|\)|\allowbreak\token[0.4471]| is|\allowbreak\token[0.0000]| not|\allowbreak\token[0.5029]| a|\allowbreak\token[0.6630]| solution|\allowbreak\token[0.0000]|.|\allowbreak\token[0.1803]| We|\allowbreak\token[0.7428]| need|\allowbreak\token[0.0015]| to|\allowbreak\token[0.0065]| find|\allowbreak\token[0.0014]| the|\allowbreak\token[0.2272]| smallest|\allowbreak\token[0.1538]| \(|\allowbreak\token[0.0000]|x|\allowbreak\token[0.0326]|\)|\allowbreak\token[0.7490]| that|\allowbreak\token[0.3681]| makes|\allowbreak\token[0.0087]| \(|\allowbreak\token[0.0003]|7|\allowbreak\token[0.0000]|x|\allowbreak\token[0.0000]| +|\allowbreak\token[0.0000]| |\allowbreak\token[0.0000]|5|\allowbreak\token[0.0001]|3|\allowbreak\token[0.0000]|\)|\allowbreak\token[0.0364]| divisible|\allowbreak\token[0.0000]| by|\allowbreak\token[0.0000]| |\allowbreak\token[0.0000]|8|\allowbreak\token[0.6856]| and|\allowbreak\token[0.3847]| greater|\allowbreak\token[0.0000]| than|\allowbreak\token[0.0001]| |\allowbreak\token[0.0016]|2|\allowbreak\token[0.0001]|0|\allowbreak\token[0.0090]|0|\allowbreak\token[0.0009]|.|\allowbreak\token[0.6951]| Let|\allowbreak\token[0.0000]|'s|\allowbreak\token[0.7438]| try|\allowbreak\token[0.0802]| \(|\allowbreak\token[0.0000]|x|\allowbreak\token[0.0000]| =|\allowbreak\token[0.0000]| |\allowbreak\token[0.0029]|2|\allowbreak\token[0.1013]|3|\allowbreak\token[0.0011]|\|\allowbreak\token[0.0000]|):↩|\newline\allowbreak\token[0.0000]|\[|\allowbreak\token[0.0000]|7|\allowbreak\token[0.0000]|(|\allowbreak\token[0.0000]|2|\allowbreak\token[0.0001]|3|\allowbreak\token[0.0000]|)|\allowbreak\token[0.0000]| +|\allowbreak\token[0.0000]| |\allowbreak\token[0.0000]|5|\allowbreak\token[0.0000]|3|\allowbreak\token[0.0000]| =|\allowbreak\token[0.0000]| |\allowbreak\token[0.0017]|1|\allowbreak\token[0.0011]|6|\allowbreak\token[0.0001]|1|\allowbreak\token[0.0000]| +|\allowbreak\token[0.0000]| |\allowbreak\token[0.0000]|5|\allowbreak\token[0.0000]|3|\allowbreak\token[0.0000]| =|\allowbreak\token[0.0000]| |\allowbreak\token[0.0000]|2|\allowbreak\token[0.0000]|1|\allowbreak\token[0.0001]|4|\allowbreak\token[0.0000]|\|\allowbreak\token[0.0140]|]↩↩|\tighttracebreak\allowbreak\token[0.7837]|2|\allowbreak\token[0.0000]|1|\allowbreak\token[0.0000]|4|\allowbreak\token[0.0113]| is|\allowbreak\token[0.9358]| not|\allowbreak\token[0.0007]| divisible|\allowbreak\token[0.0000]| by|\allowbreak\token[0.0000]| |\allowbreak\token[0.0000]|8|\allowbreak\token[1.0783]|,|\allowbreak\token[0.0914]| so|\allowbreak\token[0.8892]| we|\allowbreak\token[0.1628]| continue|\allowbreak\token[1.0739]|.|\allowbreak\token[1.3819]| Let|\allowbreak\token[0.0001]|'s|\allowbreak\token[0.0000]| try|\allowbreak\token[0.0000]| \(|\allowbreak\token[0.0000]|x|\allowbreak\token[0.0000]| =|\allowbreak\token[0.0000]| |\allowbreak\token[0.0000]|2|\allowbreak\token[0.0316]|4|\allowbreak\token[0.0001]|\|\allowbreak\token[0.0000]|):↩|\newline\allowbreak\token[0.0000]|\[|\allowbreak\token[0.0000]|7|\allowbreak\token[0.0000]|(|\allowbreak\token[0.0000]|2|\allowbreak\token[0.0001]|4|\allowbreak[...]
\tcblower
\token[0.0000]|2|\allowbreak\token[0.0000]|)|\allowbreak\token[0.0000]| +|\allowbreak\token[0.0000]| |\allowbreak\token[0.0000]|5|\allowbreak\token[0.0000]|3|\allowbreak\token[0.0000]| =|\allowbreak\token[0.0000]| |\allowbreak\token[0.1525]|1|\allowbreak\token[0.2959]|5|\allowbreak\token[0.0089]|4|\allowbreak\token[0.0001]| +|\allowbreak\token[0.0000]| |\allowbreak\token[0.0000]|5|\allowbreak\token[0.0000]|3|\allowbreak\token[0.0000]| =|\allowbreak\token[0.0000]| |\allowbreak\token[0.0004]|2|\allowbreak\token[0.0000]|0|\allowbreak\token[0.0000]|7|\allowbreak\token[0.0091]|\|\allowbreak\token[0.1594]|]↩↩|\tighttracebreak\allowbreak\token[0.9770]|Since|\allowbreak\token[0.3397]| |\allowbreak\token[0.0000]|2|\allowbreak\token[0.0000]|0|\allowbreak\token[0.0000]|7|\allowbreak\token[0.0015]| is|\allowbreak\token[0.6310]| not|\allowbreak\token[0.0330]| divisible|\allowbreak\token[0.0001]| by|\allowbreak\token[0.0000]| |\allowbreak\token[0.0001]|8|\allowbreak\token[0.0004]|,|\allowbreak\token[0.4703]| \(|\allowbreak\token[0.0000]|x|\allowbreak\token[0.0002]| =|\allowbreak\token[0.0000]| |\allowbreak\token[0.0000]|2|\allowbreak\token[0.0000]|2|\allowbreak\token[0.0000]|\)|\allowbreak\token[0.5030]| is|\allowbreak\token[0.0000]| not|\allowbreak\token[0.5029]| a|\allowbreak\token[0.6630]| solution|\allowbreak\token[0.0000]|.|\allowbreak\token[0.2191]| We|\allowbreak\token[0.7234]| need|\allowbreak\token[0.0015]| to|\allowbreak\token[0.0063]| find|\allowbreak\token[0.0013]| the|\allowbreak\token[0.2272]| smallest|\allowbreak\token[0.1585]| \(|\allowbreak\token[0.0000]|x|\allowbreak\token[0.0326]|\)|\allowbreak\token[0.7490]| that|\allowbreak\token[0.4373]| makes|\allowbreak\token[0.0096]| \(|\allowbreak\token[0.0003]|7|\allowbreak\token[0.0000]|x|\allowbreak\token[0.0000]| +|\allowbreak\token[0.0000]| |\allowbreak\token[0.0000]|5|\allowbreak\token[0.0001]|3|\allowbreak\token[0.0000]|\)|\allowbreak\token[0.0328]| divisible|\allowbreak\token[0.0000]| by|\allowbreak\token[0.0000]| |\allowbreak\token[0.0000]|8|\allowbreak\token[0.6761]|.|\allowbreak\token[1.2125]| Let|\allowbreak\token[0.0000]|'s|\allowbreak\token[0.7651]| increment|\allowbreak\token[0.0173]| \(|\allowbreak\token[0.0000]|x|\allowbreak\token[0.0000]|\)|\allowbreak\token[0.2886]| by|\allowbreak\token[0.0019]| |\allowbreak\token[0.0113]|1|\allowbreak\token[0.2971]| and|\allowbreak\token[0.0011]| check|\allowbreak\token[0.6475]| again|\allowbreak\token[0.1128]|:↩|\newline\allowbreak\token[0.0497]|\[|\allowbreak\token[0.2273]|7|\allowbreak\token[0.0000]|(|\allowbreak\token[0.0000]|2|\allowbreak\token[0.0010]|3|\allowbreak\token[0.0000]|)|\allowbreak\token[0.0000]| +|\allowbreak\token[0.0000]| |\allowbreak\token[0.0000]|5|\allowbreak\token[0.0000]|3|\allowbreak\token[0.0000]| =|\allowbreak\token[0.0000]| |\allowbreak\token[0.0739]|1|\allowbreak\token[0.0045]|6|\allowbreak\token[0.0008]|1|\allowbreak\token[0.0001]| +|\allowbreak\token[0.0000]| |\allowbreak\token[0.0000]|5|\allowbreak\token[0.0000]|3|\allowbreak\token[0.0000]| =|\allowbreak\token[0.0000]| |\allowbreak\token[0.0000]|2|\allowbreak\token[0.0000]|1|\allowbreak\token[0.0001]|4|\allowbreak\token[0.0000]|\|\allowbreak\token[0.0362]|]↩↩|\tighttracebreak\allowbreak\token[0.5650]|2|\allowbreak\token[0.0000]|1|\allowbreak\token[0.0000]|4|\allowbreak\token[0.0193]| is|\allowbreak\token[1.0511]| not|\allowbreak\token[0.0000]| divisible|\allowbreak\token[0.0000]| by|\allowbreak\token[0.0000]| |\allowbreak\token[0.0000]|8|\allowbreak\token[0.8323]| either|\allowbreak\token[0.0026]|.|\allowbreak\token[1.1495]| Continuing|\allowbreak\token[0.0001]| this|\allowbreak\token[0.0000]| process|\allowbreak\token[0.0739]|,|\allowbreak\token[0.0052]| we|\allowbreak\token[0.0000]| find|\allowbreak\token[0.0006]| that|\allowbreak\token[0.5779]| when|\allowbreak\token[0.0000]| \(|\allowbreak\token[0.0000]|x|\allowbreak\token[0.0000]| =|\allowbreak\token[0.0000]| |\allowbreak\token[0.0325]|2|\allowbreak\token[1.4465]|5|\allowbreak\token[0.0002]|\|\allowbreak\token[0.2476]|):↩|\newline\allowbreak\token[0.0000]|\[|\allowbreak\token[0.0000]|7|\allowbreak\token[0.0000]|(|\allowbreak\token[0.0000]|2|\allowbreak\token[0.0000]|5|\allowbreak\token[0.0000]|)|\allowbreak\token[0.0000]| +|\allowbreak\token[0.0000]| |\allowbreak[...]
\end{twopartbox}
\begin{twopartbox}{Cut Index 478 | Acceptance Probability 0.9719}{Block 3 Step 3}{Proposal (Accepted)}
[...]\allowbreak\token[0.0000]|x|\allowbreak\token[0.0002]| =|\allowbreak\token[0.0000]| |\allowbreak\token[0.0000]|2|\allowbreak\token[0.0000]|2|\allowbreak\token[0.0000]|\)|\allowbreak\token[0.4471]| is|\allowbreak\token[0.0000]| not|\allowbreak\token[0.5029]| a|\allowbreak\token[0.6630]| solution|\allowbreak\token[0.0000]|.|\allowbreak\token[0.1803]| We|\allowbreak\token[0.7428]| need|\allowbreak\token[0.0015]| to|\allowbreak\token[0.0065]| find|\allowbreak\token[0.0014]| the|\allowbreak\token[0.2272]| smallest|\allowbreak\token[0.1538]| \(|\allowbreak\token[0.0000]|x|\allowbreak\token[0.0326]|\)|\allowbreak\token[0.7490]| that|\allowbreak\token[0.3681]| makes|\allowbreak\token[0.0087]| \(|\allowbreak\token[0.0003]|7|\allowbreak\token[0.0000]|x|\allowbreak\token[0.0000]| +|\allowbreak\token[0.0000]| |\allowbreak\token[0.0000]|5|\allowbreak\token[0.0001]|3|\allowbreak\token[0.0000]|\)|\allowbreak\token[0.0364]| divisible|\allowbreak\token[0.0000]| by|\allowbreak\token[0.0000]| |\allowbreak\token[0.0000]|8|\allowbreak\token[0.6856]| and|\allowbreak\token[0.3847]| greater|\allowbreak\token[0.0000]| than|\allowbreak\token[0.0001]| |\allowbreak\token[0.0016]|2|\allowbreak\token[0.0001]|0|\allowbreak\token[0.0090]|0|\allowbreak\token[0.0009]|.|\allowbreak\token[0.6951]| Let|\allowbreak\token[0.0000]|'s|\allowbreak\token[0.7438]| try|\allowbreak\token[0.0802]| \(|\allowbreak\token[0.0000]|x|\allowbreak\token[0.0000]| =|\allowbreak\token[0.0000]| |\allowbreak\token[0.0029]|2|\allowbreak\token[0.1013]|3|\allowbreak\token[0.0011]|\|\allowbreak\token[0.0000]|):↩|\newline\allowbreak\token[0.0000]|\[|\allowbreak\token[0.0000]|7|\allowbreak\token[0.0000]|(|\allowbreak\token[0.0000]|2|\allowbreak\token[0.0001]|3|\allowbreak\token[0.0000]|)|\allowbreak\token[0.0000]| +|\allowbreak\token[0.0000]| |\allowbreak\token[0.0000]|5|\allowbreak\token[0.0000]|3|\allowbreak\token[0.0000]| =|\allowbreak\token[0.0000]| |\allowbreak\token[0.0017]|1|\allowbreak\token[0.0011]|6|\allowbreak\token[0.0001]|1|\allowbreak\token[0.0000]| +|\allowbreak\token[0.0000]| |\allowbreak\token[0.0000]|5|\allowbreak\token[0.0000]|3|\allowbreak\token[0.0000]| =|\allowbreak\token[0.0000]| |\allowbreak\token[0.0000]|2|\allowbreak\token[0.0000]|1|\allowbreak\token[0.0001]|4|\allowbreak\token[0.0000]|\|\allowbreak\token[0.0140]|]↩↩|\tighttracebreak\allowbreak\token[0.7837]|2|\allowbreak\token[0.0000]|1|\allowbreak\token[0.0000]|4|\allowbreak\token[0.0113]| is|\allowbreak\token[0.9358]| not|\allowbreak\token[0.0007]| divisible|\allowbreak\token[0.0000]| by|\allowbreak\token[0.0000]| |\allowbreak\token[0.0000]|8|\allowbreak\token[1.0783]|,|\allowbreak\token[0.0914]| so|\allowbreak\token[0.8892]| we|\allowbreak\token[0.1628]| continue|\allowbreak\token[1.0739]|.|\allowbreak\token[1.3819]| Let|\allowbreak\token[0.0001]|'s|\allowbreak\token[0.0000]| try|\allowbreak\token[0.0000]| \(|\allowbreak\token[0.0000]|x|\allowbreak\token[0.0000]| =|\allowbreak\token[0.0000]| |\allowbreak\token[0.0000]|2|\allowbreak\token[0.0316]|4|\allowbreak\token[0.0001]|\|\allowbreak\token[0.0000]|):↩|\newline\allowbreak\token[0.0000]|\[|\allowbreak\token[0.0000]|7|\allowbreak\token[0.0000]|(|\allowbreak\token[0.0000]|2|\allowbreak\token[0.0001]|4|\allowbreak\token[0.0000]|)|\allowbreak\token[0.0000]| +|\allowbreak\token[0.0000]| |\allowbreak\token[0.0000]|5|\allowbreak\token[0.0000]|3|\allowbreak\token[0.0000]| =|\allowbreak\token[0.0000]| |\allowbreak\token[0.0000]|1|\allowbreak\token[0.0001]|6|\allowbreak\token[0.0000]|8|\allowbreak\token[0.0000]| +|\allowbreak\token[0.0000]| |\allowbreak\token[0.0000]|5|\allowbreak\token[0.0000]|3|\allowbreak\token[0.0000]| =|\allowbreak\token[0.0000]| |\allowbreak\token[0.0000]|2|\allowbreak\token[0.0001]|2|\allowbreak\token[0.0000]|1|\allowbreak\token[0.0000]|\|\allowbreak\token[0.0003]|]↩↩|\tighttracebreak\allowbreak\token[0.0198]|2|\allowbreak\token[0.0000]|2|\allowbreak\token[0.0000]|1|\allowbreak\token[0.0066]| is|\allowbreak\token[0.3220]| not|\allowbreak\token[0.0001]| divisible|\allowbreak\token[0.0000]| by|\allowbreak\token[0.0000]| |\allowbreak\token[0.0000]|8|\allowbreak\token[0.7791]|,|\allowbreak\token[0.2702]| so|\allowbreak\token[0.0018]| we|\allowbreak\token[0.0315]| continue|\allowbreak\token[0.0002]|.|\allowbreak\cut{[CUT]}\allowbreak\token[0.0256]| Let|\allowbreak\token[0.0000]|'s|\allowbreak\token[0.0000]| try|\allowbreak\token[0.0000]| \(|\allowbreak\token[0.0000]|x|\allowbreak\token[0.0000]| =|\allowbreak\token[0.0000]| |\allowbreak\token[0.0000]|2|\allowbreak\token[0.0043]|5|\allowbreak\token[0.0001]|\|\allowbreak\token[0.0000]|):↩|\newline\allowbreak\token[0.0000]|\[|\allowbreak\token[0.0000]|7|\allowbreak\token[0.0000]|(|\allowbreak\token[0.0000]|2|\allowbreak\token[0.0001]|5|\allowbreak\token[0.0000]|)|\allowbreak\token[0.0000]| +|\allowbreak\token[0.0000]| |\allowbreak\token[0.0000]|5|\allowbreak\token[0.0000]|3|\allowbreak\token[0.0000]| =|\allowbreak\token[0.0000]| |\allowbreak\token[0.0000]|1|\allowbreak\token[0.0000]|7|\allowbreak\token[0.0000]|5|\allowbreak\token[0.0000]| +|\allowbreak\token[0.0000]| |\allowbreak\token[0.0000]|5|\allowbreak\token[0.0000]|3|\allowbreak\token[0.0000]| =|\allowbreak\token[0.0000]| |\allowbreak\token[0.0000]|2|\allowbreak\token[0.0001]|2|\allowbreak\token[0.0001]|8|\allowbreak\token[0.0000]|\|\allowbreak\token[0.0002]|]↩↩|\tighttracebreak\allowbreak\token[0.0207]|2|\allowbreak\token[0.0000]|2|\allowbreak\token[0.0000]|8|\allowbreak\token[0.0066]| is|\allowbreak\token[0.6075]| divisible|\allowbreak\token[0.0000]| by|\allowbreak\token[0.0000]| |\allowbreak\token[0.0000]|8|\allowbreak\token[0.4276]|,|\allowbreak\token[0.5119]| and|\allowbreak\token[0.9086]| it|\allowbreak\token[0.0606]| is|\allowbreak\token[0.9637]| greater|\allowbreak\token[0.0000]| than|\allowbreak\token[0.0000]| |\allowbreak\token[0.0000]|2|\allowbreak\token[0.0000]|0|\allowbreak\token[0.0001]|0|\allowbreak\token[0.2912]|.|\allowbreak\token[0.0008]| Therefore|\allowbreak\token[0.0000]|,|\allowbreak\token[0.0402]| the|\allowbreak\token[0.0021]| smallest|\allowbreak\token[0.0013]| number|\allowbreak\token[0.0000]| of|\allowbreak\token[0.0000]| coins|\allowbreak\token[0.0009]| you|\allowbreak\token[0.0000]| could|\allowbreak\token[0.0000]| have|\allowbreak\token[0.0000]| had|\allowbreak\token[0.0002]| before|\allowbreak\token[0.0000]| finding|\allowbreak\token[0.0000]| the|\allowbreak\token[0.0024]| bag|\allowbreak\token[0.0000]| of|\allowbreak\token[0.0000]| |\allowbreak\token[0.0000]|5|\allowbreak\token[0.0001]|3|\allowbreak\token[0.0001]| coins|\allowbreak\token[0.0016]| is|\allowbreak\token[0.3672]| \(|\allowbreak\token[0.0015]|7|\allowbreak\token[0.4191]| \|\allowbreak\token[0.0000]|times|\allowbreak\token[0.0000]| |\allowbreak\token[0.0000]|2|\allowbreak\token[0.0015]|5|\allowbreak\token[0.1210]| =|\allowbreak\token[0.0004]| |\allowbreak\token[0.0000]|1|\allowbreak\token[0.0000]|7|\allowbreak\token[0.0000]|5|\allowbreak\token[0.0888]|\|\allowbreak\token[0.6006]|).|\allowbreak\token[1.3388]!<|endoftext|>!
\tcblower
\token[0.0206]| Let|\allowbreak\token[0.0000]|'s|\allowbreak\token[0.0000]| try|\allowbreak\token[0.0000]| \(|\allowbreak\token[0.0000]|x|\allowbreak\token[0.0000]| =|\allowbreak\token[0.0000]| |\allowbreak\token[0.0000]|2|\allowbreak\token[0.0046]|5|\allowbreak\token[0.0001]|\|\allowbreak\token[0.0000]|):↩|\newline\allowbreak\token[0.0000]|\[|\allowbreak\token[0.0000]|7|\allowbreak\token[0.0000]|(|\allowbreak\token[0.0000]|2|\allowbreak\token[0.0001]|5|\allowbreak\token[0.0000]|)|\allowbreak\token[0.0000]| +|\allowbreak\token[0.0000]| |\allowbreak\token[0.0000]|5|\allowbreak\token[0.0000]|3|\allowbreak\token[0.0000]| =|\allowbreak\token[0.0000]| |\allowbreak\token[0.0000]|1|\allowbreak\token[0.0000]|7|\allowbreak\token[0.0000]|5|\allowbreak\token[0.0000]| +|\allowbreak\token[0.0000]| |\allowbreak\token[0.0000]|5|\allowbreak\token[0.0000]|3|\allowbreak\token[0.0000]| =|\allowbreak\token[0.0000]| |\allowbreak\token[0.0000]|2|\allowbreak\token[0.0001]|2|\allowbreak\token[0.0001]|8|\allowbreak\token[0.0000]|\|\allowbreak\token[0.0002]|]↩↩|\tighttracebreak\allowbreak\token[0.0206]|2|\allowbreak\token[0.0000]|2|\allowbreak\token[0.0000]|8|\allowbreak\token[0.0073]| is|\allowbreak\token[0.6077]| divisible|\allowbreak\token[0.0000]| by|\allowbreak\token[0.0000]| |\allowbreak\token[0.0000]|8|\allowbreak\token[0.4633]|,|\allowbreak\token[0.5123]| and|\allowbreak\token[0.9784]| it|\allowbreak\token[0.0696]| is|\allowbreak\token[0.9636]| greater|\allowbreak\token[0.0000]| than|\allowbreak\token[0.0000]| |\allowbreak\token[0.0000]|2|\allowbreak\token[0.0000]|0|\allowbreak\token[0.0001]|0|\allowbreak\token[0.2687]|.|\allowbreak\token[0.0009]| Therefore|\allowbreak\token[0.0000]|,|\allowbreak\token[0.0403]| the|\allowbreak\token[0.0021]| smallest|\allowbreak\token[0.0012]| number|\allowbreak\token[0.0000]| of|\allowbreak\token[0.0000]| coins|\allowbreak\token[0.0009]| you|\allowbreak\token[0.0000]| could|\allowbreak\token[0.0000]| have|\allowbreak\token[0.0000]| had|\allowbreak\token[0.0002]| before|\allowbreak\token[0.0000]| finding|\allowbreak\token[0.0000]| the|\allowbreak\token[0.0022]| bag|\allowbreak\token[0.0000]| of|\allowbreak\token[0.0000]| |\allowbreak\token[0.0000]|5|\allowbreak\token[0.0001]|3|\allowbreak\token[0.0001]| coins|\allowbreak\token[0.0014]| is|\allowbreak\token[0.3388]| \(|\allowbreak\token[0.0015]|7|\allowbreak\token[0.4524]| \|\allowbreak\token[0.0000]|times|\allowbreak\token[0.0000]| |\allowbreak\token[0.0000]|2|\allowbreak\token[0.0015]|5|\allowbreak\token[0.1100]| =|\allowbreak\token[0.0004]| |\allowbreak\token[0.0000]|1|\allowbreak\token[0.0000]|7|\allowbreak\token[0.0000]|5|\allowbreak\token[0.0972]|\|\allowbreak\token[0.6006]|).|\allowbreak\token[1.2592]!<|endoftext|>!
\end{twopartbox}

\endgroup

\paragraph{Entropy-Cut Trace.}
\begingroup
\raggedright %
\setlength{\lineskip}{0pt} %
\setlength{\baselineskip}{0.9em} %
\begin{twopartbox}{Cut Index 66 | Acceptance Probability 1.0000}{Block 1 Step 1}{Proposal (Accepted)}
\token[2.2824]| To|\allowbreak\token[0.4171]| solve|\allowbreak\token[0.4477]| this|\allowbreak\token[0.0840]| problem|\allowbreak\token[0.0111]|,|\allowbreak\token[0.5775]| let|\allowbreak\token[0.0029]|'s|\allowbreak\token[0.8168]| denote|\allowbreak\token[0.0233]| the|\allowbreak\token[0.7993]| number|\allowbreak\token[0.0000]| of|\allowbreak\token[0.1007]| coins|\allowbreak\token[0.0804]| in|\allowbreak\token[0.0205]| each|\allowbreak\token[0.0183]| of|\allowbreak\token[0.0031]| the|\allowbreak\token[0.5575]| original|\allowbreak\token[0.0329]| seven|\allowbreak\token[0.0001]| bags|\allowbreak\token[0.0009]| as|\allowbreak\token[0.6010]| \(|\allowbreak\token[0.8095]|x|\allowbreak\token[0.0048]|\|\allowbreak\token[0.1099]|).|\allowbreak\token[0.9476]| Therefore|\allowbreak\token[0.0000]|,|\allowbreak\token[0.6384]| the|\allowbreak\token[0.0007]| total|\allowbreak\token[0.0001]| number|\allowbreak\token[0.0000]| of|\allowbreak\token[0.0001]| coins|\allowbreak\token[0.5853]| before|\allowbreak\token[0.0055]| finding|\allowbreak\token[0.0003]| the|\allowbreak\token[0.8417]| additional|\allowbreak\token[0.3947]| bag|\allowbreak\token[0.8785]| is|\allowbreak\token[0.0000]| \(|\allowbreak\token[0.0001]|7|\allowbreak\token[0.0000]|x|\allowbreak\token[0.2478]|\|\allowbreak\token[0.0619]|).|\allowbreak\token[0.1434]| After|\allowbreak\token[0.0428]| finding|\allowbreak\token[0.0027]| the|\allowbreak\token[0.1502]| bag|\allowbreak\token[0.6866]| with|\allowbreak\token[0.0002]| |\allowbreak\token[0.0000]|5|\allowbreak\token[0.0001]|3|\allowbreak\token[0.0064]| coins|\allowbreak\token[0.1746]|,|\allowbreak\token[0.7104]| you|\allowbreak\token[0.2338]| redistribute|\allowbreak\token[0.8207]| all|\allowbreak\token[0.1743]| the|\allowbreak\token[0.0002]| coins|\allowbreak\token[0.7935]| into|\allowbreak\token[0.0057]| eight|\allowbreak\token[0.1040]| bags|\allowbreak\token[0.7179]| with|\allowbreak\token[0.1465]| an|\allowbreak\token[0.0013]| equal|\allowbreak\token[0.0010]| number|\allowbreak\token[0.0131]| of|\allowbreak\token[0.0000]| coins|\allowbreak\cut{[CUT]}\allowbreak\token[1.0825]|,|\allowbreak\token[0.9683]| which|\allowbreak\token[0.5047]| we|\allowbreak\token[0.0822]|'ll|\allowbreak\token[0.0494]| call|\allowbreak\token[0.0000]| \(|\allowbreak\token[0.0056]|y|\allowbreak\token[0.0002]|\|\allowbreak\token[0.0501]|).|\allowbreak\token[0.8690]| This|\allowbreak\token[0.1325]| means|\allowbreak\token[0.1778]| the|\allowbreak\token[0.0619]| total|\allowbreak\token[0.0001]| number|\allowbreak\token[0.0000]| of|\allowbreak\token[0.0000]| coins|\allowbreak\token[0.7681]| is|\allowbreak\token[0.6525]| now|\allowbreak\token[0.0222]| \(|\allowbreak\token[0.5747]|7|\allowbreak\token[0.0003]|x|\allowbreak\token[0.0000]| +|\allowbreak\token[0.0000]| |\allowbreak\token[0.0000]|5|\allowbreak\token[0.0000]|3|\allowbreak\token[0.6691]|\|\allowbreak\token[0.0232]|),|\allowbreak\token[0.0038]| and|\allowbreak\token[0.0990]| this|\allowbreak\token[0.1893]| total|\allowbreak\token[1.0998]| is|\allowbreak\token[1.2135]| divisible|\allowbreak\token[0.0057]| by|\allowbreak\token[0.0004]| |\allowbreak\token[0.0001]|8|\allowbreak\token[1.2314]|,|\allowbreak\token[1.5080]| so|\allowbreak\token[0.7424]| \(|\allowbreak\token[0.0598]|7|\allowbreak\token[0.0000]|x|\allowbreak\token[0.0000]| +|\allowbreak\token[0.0000]| |\allowbreak\token[0.0000]|5|\allowbreak\token[0.0001]|3|\allowbreak\token[0.0141]| =|\allowbreak\token[0.0000]| |\allowbreak\token[0.0000]|8|\allowbreak\token[0.0003]|y|\allowbreak\token[0.0008]|\|\allowbreak\token[0.5937]|).↩↩|\tighttracebreak\allowbreak\token[0.4107]|Given|\allowbreak\token[0.0002]| that|\allowbreak\token[0.5522]| you|\allowbreak\token[0.0682]| have|\allowbreak\token[0.0010]| more|\allowbreak\token[0.0000]| than|\allowbreak\token[0.0001]| |\allowbreak\token[0.0000]|2|\allowbreak\token[0.0000]|0|\allowbreak\token[0.0000]|0|\allowbreak\token[0.0001]| coins|\allowbreak\token[1.0382]| after|\allowbreak\token[0.8873]| redistribution|\allowbreak\token[0.0019]|,|\allowbreak\token[0.7739]| we|\allowbreak\token[0.7100]| know|\allowbreak\token[0.5303]| that|\allowbreak\token[0.0008]| \(|\allowbreak\token[0.2071]|8|\allowbreak\token[0.0000]|y|\allowbreak\token[0.0000]| >|\allowbreak\token[0.0000]| |\allowbreak\token[0.0003]|2|\allowbreak\token[0.0001]|0|\allowbreak\token[0.0019]|0|\allowbreak\token[0.0002]|\|\allowbreak\token[0.1095]|),|\allowbreak\token[0.0547]| which|\allowbreak\token[0.7146]| simpl|\allowbreak\token[0.0000]|ifies|\allowbreak\token[0.0000]| to|\allowbreak\token[0.0000]| \(|\allowbreak\token[0.0000]|y|\allowbreak\token[0.0000]| >|\allowbreak\token[0.0000]| |\allowbreak\token[0.0000]|2|\allowbreak\token[0.0000]|5|\allowbreak\token[0.0004]|\|\allowbreak\token[0.0245]|).|\allowbreak\token[1.4162]| Since|\allowbreak\token[0.2454]| \(|\allowbreak\token[0.0419]|y|\allowbreak\token[0.0082]|\)|\allowbreak\token[0.7963]| must|\allowbreak\token[0.0004]| be|\allowbreak\token[0.0617]| an|\allowbreak\token[0.0000]| integer|\allowbreak\token[0.2768]|,|\allowbreak\token[0.0018]| the|\allowbreak\token[0.0016]| smallest|\allowbreak\token[0.8835]| possible|\allowbreak\token[0.0008]| value|\allowbreak\token[0.0000]| for|\allowbreak\token[0.0000]| \(|\allowbreak\token[0.0001]|y|\allowbreak\token[0.0000]|\)|\allowbreak\token[0.3403]| is|\allowbreak\token[0.0042]| |\allowbreak\token[0.0004]|2|\allowbreak\token[0.0032]|6|\allowbreak\token[0.5630]|.|\allowbreak\token[1.2256]| Sub|\allowbreak\token[0.0000]|stit|\allowbreak\token[0.0005]|uting|\allowbreak\token[0.0008]| \(|\allowbreak\token[0.0001]|y|\allowbreak\token[0.0000]| =|\allowbreak\token[0.0000]| |\allowbreak\token[0.0000]|2|\allowbreak\token[0.0000]|6|\allowbreak\token[0.0001]|\)|\allowbreak\token[0.0073]| into|\allowbreak\token[0.6566]| our|\allowbreak\token[0.0000]| equation|\allowbreak\token[0.1005]| gives|\allowbreak\token[0.0002]| us|
\tcblower
\token[1.0839]|,|\allowbreak\token[0.9759]| let|\allowbreak\token[0.0000]|'s|\allowbreak\token[0.0518]| call|\allowbreak\token[0.0001]| this|\allowbreak\token[0.0940]| number|\allowbreak\token[0.0034]| \(|\allowbreak\token[0.0048]|y|\allowbreak\token[0.0000]|\|\allowbreak\token[0.0825]|).|\allowbreak\token[1.2411]| So|\allowbreak\token[0.0000]|,|\allowbreak\token[0.0321]| the|\allowbreak\token[0.0783]| total|\allowbreak\token[0.0000]| number|\allowbreak\token[0.0000]| of|\allowbreak\token[0.0000]| coins|\allowbreak\token[0.6227]| after|\allowbreak\token[0.3939]| redistribution|\allowbreak\token[0.0004]| is|\allowbreak\token[0.0082]| \(|\allowbreak\token[0.6478]|8|\allowbreak\token[0.0000]|y|\allowbreak\token[0.0177]|\|\allowbreak\token[0.9802]|).↩↩|\tighttracebreak\allowbreak\token[0.2592]|Given|\allowbreak\token[0.0010]| that|\allowbreak\token[0.7182]| you|\allowbreak\token[0.3932]| have|\allowbreak\token[0.0104]| more|\allowbreak\token[0.0000]| than|\allowbreak\token[0.0000]| |\allowbreak\token[0.0000]|2|\allowbreak\token[0.0000]|0|\allowbreak\token[0.0001]|0|\allowbreak\token[0.0000]| coins|\allowbreak\token[0.9615]| after|\allowbreak\token[0.2218]| redistribution|\allowbreak\token[0.0331]|,|\allowbreak\token[0.0294]| we|\allowbreak\token[0.5636]| have|\allowbreak\token[0.3679]| \(|\allowbreak\token[0.0011]|8|\allowbreak\token[0.0000]|y|\allowbreak\token[0.0000]| >|\allowbreak\token[0.0000]| |\allowbreak\token[0.0000]|2|\allowbreak\token[0.0000]|0|\allowbreak\token[0.0001]|0|\allowbreak\token[0.0001]|\|\allowbreak\token[0.2913]|),|\allowbreak\token[0.0003]| which|\allowbreak\token[0.0865]| simpl|\allowbreak\token[0.0000]|ifies|\allowbreak\token[0.0000]| to|\allowbreak\token[0.0000]| \(|\allowbreak\token[0.0000]|y|\allowbreak\token[0.0000]| >|\allowbreak\token[0.0000]| |\allowbreak\token[0.0000]|2|\allowbreak\token[0.0000]|5|\allowbreak\token[0.0001]|\|\allowbreak\token[0.1752]|).|\allowbreak\token[0.8695]| Since|\allowbreak\token[0.7806]| \(|\allowbreak\token[0.1099]|y|\allowbreak\token[0.0215]|\)|\allowbreak\token[0.8159]| must|\allowbreak\token[0.0002]| be|\allowbreak\token[0.1456]| an|\allowbreak\token[0.0000]| integer|\allowbreak\token[0.4575]|,|\allowbreak\token[0.0010]| the|\allowbreak\token[0.0038]| smallest|\allowbreak\token[0.6622]| possible|\allowbreak\token[0.0002]| value|\allowbreak\token[0.0000]| for|\allowbreak\token[0.0000]| \(|\allowbreak\token[0.0000]|y|\allowbreak\token[0.0000]|\)|\allowbreak\token[0.0173]| is|\allowbreak\token[0.0140]| |\allowbreak\token[0.0001]|2|\allowbreak\token[0.0014]|6|\allowbreak\token[0.5697]|.↩↩|\tighttracebreak\allowbreak\token[0.8369]|Now|\allowbreak\token[0.0000]|,|\allowbreak\token[0.5838]| we|\allowbreak\token[0.1722]| know|\allowbreak\token[0.0013]| that|\allowbreak\token[0.8064]| the|\allowbreak\token[0.0004]| total|\allowbreak\token[0.0002]| number|\allowbreak\token[0.0000]| of|\allowbreak\token[0.0000]| coins|\allowbreak\token[0.5237]| after|\allowbreak\token[0.2302]| redistribution|\allowbreak\token[0.0241]| is|\allowbreak\token[0.8472]| \(|\allowbreak\token[0.5679]|8|\allowbreak\token[0.6473]| \|\allowbreak\token[0.0000]|times|\allowbreak\token[0.0001]| |\allowbreak\token[0.0000]|2|\allowbreak\token[0.0000]|6|\allowbreak\token[0.0001]| =|\allowbreak\token[0.0000]| |\allowbreak\token[0.0007]|2|\allowbreak\token[0.0000]|0|\allowbreak\token[0.0000]|8|\allowbreak\token[0.0015]|\|\allowbreak\token[0.0637]|).|\allowbreak\token[0.9022]| This|\allowbreak\token[0.3194]| total|\allowbreak\token[0.0876]| includes|\allowbreak\token[0.0003]| the|\allowbreak\token[0.2205]| |\allowbreak\token[0.0001]|5|\allowbreak\token[0.0000]|3|\allowbreak\token[0.0067]| coins|\allowbreak\token[0.3874]| from|\allowbreak\token[0.0000]| the|\allowbreak\token[0.9499]| additional|\allowbreak\token[0.0001]| bag|\allowbreak\token[0.3194]|,|\allowbreak\token[0.0235]| so|
\end{twopartbox}
\begin{twopartbox}{Cut Index 90 | Acceptance Probability 0.1065}{Block 1 Step 2}{Proposal (Rejected)}
\token[2.2824]| To|\allowbreak\token[0.4171]| solve|\allowbreak\token[0.4477]| this|\allowbreak\token[0.0840]| problem|\allowbreak\token[0.0111]|,|\allowbreak\token[0.5775]| let|\allowbreak\token[0.0029]|'s|\allowbreak\token[0.8168]| denote|\allowbreak\token[0.0233]| the|\allowbreak\token[0.7993]| number|\allowbreak\token[0.0000]| of|\allowbreak\token[0.1007]| coins|\allowbreak\token[0.0804]| in|\allowbreak\token[0.0205]| each|\allowbreak\token[0.0183]| of|\allowbreak\token[0.0031]| the|\allowbreak\token[0.5575]| original|\allowbreak\token[0.0329]| seven|\allowbreak\token[0.0001]| bags|\allowbreak\token[0.0009]| as|\allowbreak\token[0.6010]| \(|\allowbreak\token[0.8095]|x|\allowbreak\token[0.0048]|\|\allowbreak\token[0.1099]|).|\allowbreak\token[0.9476]| Therefore|\allowbreak\token[0.0000]|,|\allowbreak\token[0.6384]| the|\allowbreak\token[0.0007]| total|\allowbreak\token[0.0001]| number|\allowbreak\token[0.0000]| of|\allowbreak\token[0.0001]| coins|\allowbreak\token[0.5853]| before|\allowbreak\token[0.0055]| finding|\allowbreak\token[0.0003]| the|\allowbreak\token[0.8417]| additional|\allowbreak\token[0.3947]| bag|\allowbreak\token[0.8785]| is|\allowbreak\token[0.0000]| \(|\allowbreak\token[0.0001]|7|\allowbreak\token[0.0000]|x|\allowbreak\token[0.2478]|\|\allowbreak\token[0.0619]|).|\allowbreak\token[0.1434]| After|\allowbreak\token[0.0428]| finding|\allowbreak\token[0.0027]| the|\allowbreak\token[0.1502]| bag|\allowbreak\token[0.6866]| with|\allowbreak\token[0.0002]| |\allowbreak\token[0.0000]|5|\allowbreak\token[0.0001]|3|\allowbreak\token[0.0064]| coins|\allowbreak\token[0.1746]|,|\allowbreak\token[0.7104]| you|\allowbreak\token[0.2338]| redistribute|\allowbreak\token[0.8207]| all|\allowbreak\token[0.1743]| the|\allowbreak\token[0.0002]| coins|\allowbreak\token[0.7935]| into|\allowbreak\token[0.0057]| eight|\allowbreak\token[0.1040]| bags|\allowbreak\token[0.7179]| with|\allowbreak\token[0.1465]| an|\allowbreak\token[0.0013]| equal|\allowbreak\token[0.0010]| number|\allowbreak\token[0.0131]| of|\allowbreak\token[0.0000]| coins|\allowbreak\token[1.0839]|,|\allowbreak\token[0.9759]| let|\allowbreak\token[0.0000]|'s|\allowbreak\token[0.0518]| call|\allowbreak\token[0.0001]| this|\allowbreak\token[0.0940]| number|\allowbreak\token[0.0034]| \(|\allowbreak\token[0.0048]|y|\allowbreak\token[0.0000]|\|\allowbreak\token[0.0825]|).|\allowbreak\token[1.2411]| So|\allowbreak\token[0.0000]|,|\allowbreak\token[0.0321]| the|\allowbreak\token[0.0783]| total|\allowbreak\token[0.0000]| number|\allowbreak\token[0.0000]| of|\allowbreak\token[0.0000]| coins|\allowbreak\token[0.6227]| after|\allowbreak\token[0.3939]| redistribution|\allowbreak\token[0.0004]| is|\allowbreak\token[0.0082]| \(|\allowbreak\token[0.6478]|8|\allowbreak\token[0.0000]|y|\allowbreak\token[0.0177]|\|\allowbreak\cut{[CUT]}\allowbreak\token[0.9802]|).↩↩|\tighttracebreak\allowbreak\token[0.2592]|Given|\allowbreak\token[0.0010]| that|\allowbreak\token[0.7182]| you|\allowbreak\token[0.3932]| have|\allowbreak\token[0.0104]| more|\allowbreak\token[0.0000]| than|\allowbreak\token[0.0000]| |\allowbreak\token[0.0000]|2|\allowbreak\token[0.0000]|0|\allowbreak\token[0.0001]|0|\allowbreak\token[0.0000]| coins|\allowbreak\token[0.9615]| after|\allowbreak\token[0.2218]| redistribution|\allowbreak\token[0.0331]|,|\allowbreak\token[0.0294]| we|\allowbreak\token[0.5636]| have|\allowbreak\token[0.3679]| \(|\allowbreak\token[0.0011]|8|\allowbreak\token[0.0000]|y|\allowbreak\token[0.0000]| >|\allowbreak\token[0.0000]| |\allowbreak\token[0.0000]|2|\allowbreak\token[0.0000]|0|\allowbreak\token[0.0001]|0|\allowbreak\token[0.0001]|\|\allowbreak\token[0.2913]|),|\allowbreak\token[0.0003]| which|\allowbreak\token[0.0865]| simpl|\allowbreak\token[0.0000]|ifies|\allowbreak\token[0.0000]| to|\allowbreak\token[0.0000]| \(|\allowbreak\token[0.0000]|y|\allowbreak\token[0.0000]| >|\allowbreak\token[0.0000]| |\allowbreak\token[0.0000]|2|\allowbreak\token[0.0000]|5|\allowbreak\token[0.0001]|\|\allowbreak\token[0.1752]|).|\allowbreak\token[0.8695]| Since|\allowbreak\token[0.7806]| \(|\allowbreak\token[0.1099]|y|\allowbreak\token[0.0215]|\)|\allowbreak\token[0.8159]| must|\allowbreak\token[0.0002]| be|\allowbreak\token[0.1456]| an|\allowbreak\token[0.0000]| integer|\allowbreak\token[0.4575]|,|\allowbreak\token[0.0010]| the|\allowbreak\token[0.0038]| smallest|\allowbreak\token[0.6622]| possible|\allowbreak\token[0.0002]| value|\allowbreak\token[0.0000]| for|\allowbreak\token[0.0000]| \(|\allowbreak\token[0.0000]|y|\allowbreak\token[0.0000]|\)|\allowbreak\token[0.0173]| is|\allowbreak\token[0.0140]| |\allowbreak\token[0.0001]|2|\allowbreak\token[0.0014]|6|\allowbreak\token[0.5697]|.↩↩|\tighttracebreak\allowbreak\token[0.8369]|Now|\allowbreak\token[0.0000]|,|\allowbreak\token[0.5838]| we|\allowbreak\token[0.1722]| know|\allowbreak\token[0.0013]| that|\allowbreak\token[0.8064]| the|\allowbreak\token[0.0004]| total|\allowbreak\token[0.0002]| number|\allowbreak\token[0.0000]| of|\allowbreak\token[0.0000]| coins|\allowbreak\token[0.5237]| after|\allowbreak\token[0.2302]| redistribution|\allowbreak\token[0.0241]| is|\allowbreak\token[0.8472]| \(|\allowbreak\token[0.5679]|8|\allowbreak\token[0.6473]| \|\allowbreak\token[0.0000]|times|\allowbreak\token[0.0001]| |\allowbreak\token[0.0000]|2|\allowbreak\token[0.0000]|6|\allowbreak\token[0.0001]| =|\allowbreak\token[0.0000]| |\allowbreak\token[0.0007]|2|\allowbreak\token[0.0000]|0|\allowbreak\token[0.0000]|8|\allowbreak\token[0.0015]|\|\allowbreak\token[0.0637]|).|\allowbreak\token[0.9022]| This|\allowbreak\token[0.3194]| total|\allowbreak\token[0.0876]| includes|\allowbreak\token[0.0003]| the|\allowbreak\token[0.2205]| |\allowbreak\token[0.0001]|5|\allowbreak\token[0.0000]|3|\allowbreak\token[0.0067]| coins|\allowbreak\token[0.3874]| from|\allowbreak\token[0.0000]| the|\allowbreak\token[0.9499]| additional|\allowbreak\token[0.0001]| bag|\allowbreak\token[0.3194]|,|\allowbreak\token[0.0235]| so|
\tcblower
\token[0.9802]|).↩↩|\tighttracebreak\allowbreak\token[0.2592]|Given|\allowbreak\token[0.0011]| that|\allowbreak\token[0.6980]| you|\allowbreak\token[0.4030]| have|\allowbreak\token[0.0115]| more|\allowbreak\token[0.0000]| than|\allowbreak\token[0.0001]| |\allowbreak\token[0.0000]|2|\allowbreak\token[0.0000]|0|\allowbreak\token[0.0001]|0|\allowbreak\token[0.0000]| coins|\allowbreak\token[0.9850]| in|\allowbreak\token[0.0023]| total|\allowbreak\token[0.7103]| after|\allowbreak\token[0.1289]| redistribution|\allowbreak\token[0.0145]|,|\allowbreak\token[0.0215]| we|\allowbreak\token[0.3861]| have|\allowbreak\token[0.5540]| \(|\allowbreak\token[0.0020]|8|\allowbreak\token[0.0000]|y|\allowbreak\token[0.0000]| >|\allowbreak\token[0.0000]| |\allowbreak\token[0.0000]|2|\allowbreak\token[0.0000]|0|\allowbreak\token[0.0001]|0|\allowbreak\token[0.0001]|\|\allowbreak\token[0.3659]|),|\allowbreak\token[0.0003]| which|\allowbreak\token[0.0925]| simpl|\allowbreak\token[0.0000]|ifies|\allowbreak\token[0.0000]| to|\allowbreak\token[0.0000]| \(|\allowbreak\token[0.0000]|y|\allowbreak\token[0.0000]| >|\allowbreak\token[0.0000]| |\allowbreak\token[0.0000]|2|\allowbreak\token[0.0000]|5|\allowbreak\token[0.0001]|\|\allowbreak\token[0.2282]|).|\allowbreak\token[0.8726]| Since|\allowbreak\token[0.6489]| \(|\allowbreak\token[0.0823]|y|\allowbreak\token[0.0140]|\)|\allowbreak\token[0.8133]| must|\allowbreak\token[0.0001]| be|\allowbreak\token[0.1456]| an|\allowbreak\token[0.0000]| integer|\allowbreak\token[0.4331]|,|\allowbreak\token[0.0008]| the|\allowbreak\token[0.0030]| smallest|\allowbreak\token[0.6066]| possible|\allowbreak\token[0.0001]| value|\allowbreak\token[0.0000]| for|\allowbreak\token[0.0000]| \(|\allowbreak\token[0.0000]|y|\allowbreak\token[0.0000]|\)|\allowbreak\token[0.0193]| is|\allowbreak\token[0.0156]| |\allowbreak\token[0.0001]|2|\allowbreak\token[0.0014]|6|\allowbreak\token[0.5167]|.↩↩|\tighttracebreak\allowbreak\token[0.8411]|The|\allowbreak\token[0.0790]| total|\allowbreak\token[0.0001]| number|\allowbreak\token[0.0000]| of|\allowbreak\token[0.0000]| coins|\allowbreak\token[0.4734]| after|\allowbreak\token[0.1856]| redistribution|\allowbreak\token[0.0347]| is|\allowbreak\token[0.4458]| \(|\allowbreak\token[0.4715]|8|\allowbreak\token[0.3657]| \|\allowbreak\token[0.0000]|times|\allowbreak\token[0.0001]| |\allowbreak\token[0.0000]|2|\allowbreak\token[0.0000]|6|\allowbreak\token[0.0000]| =|\allowbreak\token[0.0000]| |\allowbreak\token[0.0004]|2|\allowbreak\token[0.0000]|0|\allowbreak\token[0.0000]|8|\allowbreak\token[0.0008]|\|\allowbreak\token[0.0851]|).|\allowbreak\token[1.1720]| This|\allowbreak\token[0.5618]| total|\allowbreak\token[0.0579]| includes|\allowbreak\token[0.0003]| the|\allowbreak\token[0.2107]| |\allowbreak\token[0.0001]|5|\allowbreak\token[0.0000]|3|\allowbreak\token[0.0081]| coins|\allowbreak\token[0.4243]| from|\allowbreak\token[0.0000]| the|\allowbreak\token[0.9783]| additional|\allowbreak\token[0.0001]| bag|\allowbreak\token[0.2494]|,|\allowbreak\token[0.0242]| so|\allowbreak\token[0.3659]| the|\allowbreak\token[0.9356]| original|\allowbreak\token[0.8674]| total|
\end{twopartbox}
\begin{twopartbox}{Cut Index 24 | Acceptance Probability 0.2232}{Block 1 Step 3}{Proposal (Rejected)}
\token[2.2824]| To|\allowbreak\token[0.4171]| solve|\allowbreak\token[0.4477]| this|\allowbreak\token[0.0840]| problem|\allowbreak\token[0.0111]|,|\allowbreak\token[0.5775]| let|\allowbreak\token[0.0029]|'s|\allowbreak\token[0.8168]| denote|\allowbreak\token[0.0233]| the|\allowbreak\token[0.7993]| number|\allowbreak\token[0.0000]| of|\allowbreak\token[0.1007]| coins|\allowbreak\token[0.0804]| in|\allowbreak\token[0.0205]| each|\allowbreak\token[0.0183]| of|\allowbreak\token[0.0031]| the|\allowbreak\token[0.5575]| original|\allowbreak\token[0.0329]| seven|\allowbreak\token[0.0001]| bags|\allowbreak\token[0.0009]| as|\allowbreak\token[0.6010]| \(|\allowbreak\token[0.8095]|x|\allowbreak\token[0.0048]|\|\allowbreak\token[0.1099]|).|\allowbreak\cut{[CUT]}\allowbreak\token[0.9476]| Therefore|\allowbreak\token[0.0000]|,|\allowbreak\token[0.6384]| the|\allowbreak\token[0.0007]| total|\allowbreak\token[0.0001]| number|\allowbreak\token[0.0000]| of|\allowbreak\token[0.0001]| coins|\allowbreak\token[0.5853]| before|\allowbreak\token[0.0055]| finding|\allowbreak\token[0.0003]| the|\allowbreak\token[0.8417]| additional|\allowbreak\token[0.3947]| bag|\allowbreak\token[0.8785]| is|\allowbreak\token[0.0000]| \(|\allowbreak\token[0.0001]|7|\allowbreak\token[0.0000]|x|\allowbreak\token[0.2478]|\|\allowbreak\token[0.0619]|).|\allowbreak\token[0.1434]| After|\allowbreak\token[0.0428]| finding|\allowbreak\token[0.0027]| the|\allowbreak\token[0.1502]| bag|\allowbreak\token[0.6866]| with|\allowbreak\token[0.0002]| |\allowbreak\token[0.0000]|5|\allowbreak\token[0.0001]|3|\allowbreak\token[0.0064]| coins|\allowbreak\token[0.1746]|,|\allowbreak\token[0.7104]| you|\allowbreak\token[0.2338]| redistribute|\allowbreak\token[0.8207]| all|\allowbreak\token[0.1743]| the|\allowbreak\token[0.0002]| coins|\allowbreak\token[0.7935]| into|\allowbreak\token[0.0057]| eight|\allowbreak\token[0.1040]| bags|\allowbreak\token[0.7179]| with|\allowbreak\token[0.1465]| an|\allowbreak\token[0.0013]| equal|\allowbreak\token[0.0010]| number|\allowbreak\token[0.0131]| of|\allowbreak\token[0.0000]| coins|\allowbreak\token[1.0839]|,|\allowbreak\token[0.9759]| let|\allowbreak\token[0.0000]|'s|\allowbreak\token[0.0518]| call|\allowbreak\token[0.0001]| this|\allowbreak\token[0.0940]| number|\allowbreak\token[0.0034]| \(|\allowbreak\token[0.0048]|y|\allowbreak\token[0.0000]|\|\allowbreak\token[0.0825]|).|\allowbreak\token[1.2411]| So|\allowbreak\token[0.0000]|,|\allowbreak\token[0.0321]| the|\allowbreak\token[0.0783]| total|\allowbreak\token[0.0000]| number|\allowbreak\token[0.0000]| of|\allowbreak\token[0.0000]| coins|\allowbreak\token[0.6227]| after|\allowbreak\token[0.3939]| redistribution|\allowbreak\token[0.0004]| is|\allowbreak\token[0.0082]| \(|\allowbreak\token[0.6478]|8|\allowbreak\token[0.0000]|y|\allowbreak\token[0.0177]|\|\allowbreak\token[0.9802]|).↩↩|\tighttracebreak\allowbreak\token[0.2592]|Given|\allowbreak\token[0.0010]| that|\allowbreak\token[0.7182]| you|\allowbreak\token[0.3932]| have|\allowbreak\token[0.0104]| more|\allowbreak\token[0.0000]| than|\allowbreak\token[0.0000]| |\allowbreak\token[0.0000]|2|\allowbreak\token[0.0000]|0|\allowbreak\token[0.0001]|0|\allowbreak\token[0.0000]| coins|\allowbreak\token[0.9615]| after|\allowbreak\token[0.2218]| redistribution|\allowbreak\token[0.0331]|,|\allowbreak\token[0.0294]| we|\allowbreak\token[0.5636]| have|\allowbreak\token[0.3679]| \(|\allowbreak\token[0.0011]|8|\allowbreak\token[0.0000]|y|\allowbreak\token[0.0000]| >|\allowbreak\token[0.0000]| |\allowbreak\token[0.0000]|2|\allowbreak\token[0.0000]|0|\allowbreak\token[0.0001]|0|\allowbreak\token[0.0001]|\|\allowbreak\token[0.2913]|),|\allowbreak\token[0.0003]| which|\allowbreak\token[0.0865]| simpl|\allowbreak\token[0.0000]|ifies|\allowbreak\token[0.0000]| to|\allowbreak\token[0.0000]| \(|\allowbreak\token[0.0000]|y|\allowbreak\token[0.0000]| >|\allowbreak\token[0.0000]| |\allowbreak\token[0.0000]|2|\allowbreak\token[0.0000]|5|\allowbreak\token[0.0001]|\|\allowbreak\token[0.1752]|).|\allowbreak\token[0.8695]| Since|\allowbreak\token[0.7806]| \(|\allowbreak\token[0.1099]|y|\allowbreak\token[0.0215]|\)|\allowbreak\token[0.8159]| must|\allowbreak\token[0.0002]| be|\allowbreak\token[0.1456]| an|\allowbreak\token[0.0000]| integer|\allowbreak\token[0.4575]|,|\allowbreak\token[0.0010]| the|\allowbreak\token[0.0038]| smallest|\allowbreak\token[0.6622]| possible|\allowbreak\token[0.0002]| value|\allowbreak\token[0.0000]| for|\allowbreak\token[0.0000]| \(|\allowbreak\token[0.0000]|y|\allowbreak\token[0.0000]|\)|\allowbreak\token[0.0173]| is|\allowbreak\token[0.0140]| |\allowbreak\token[0.0001]|2|\allowbreak\token[0.0014]|6|\allowbreak\token[0.5697]|.↩↩|\tighttracebreak\allowbreak\token[0.8369]|Now|\allowbreak\token[0.0000]|,|\allowbreak\token[0.5838]| we|\allowbreak\token[0.1722]| know|\allowbreak\token[0.0013]| that|\allowbreak\token[0.8064]| the|\allowbreak\token[0.0004]| total|\allowbreak\token[0.0002]| number|\allowbreak\token[0.0000]| of|\allowbreak\token[0.0000]| coins|\allowbreak\token[0.5237]| after|\allowbreak\token[0.2302]| redistribution|\allowbreak\token[0.0241]| is|\allowbreak\token[0.8472]| \(|\allowbreak\token[0.5679]|8|\allowbreak\token[0.6473]| \|\allowbreak\token[0.0000]|times|\allowbreak[...]
\tcblower
\token[0.9474]| Therefore|\allowbreak\token[0.0000]|,|\allowbreak\token[0.6396]| the|\allowbreak\token[0.0006]| total|\allowbreak\token[0.0001]| number|\allowbreak\token[0.0000]| of|\allowbreak\token[0.0001]| coins|\allowbreak\token[0.5853]| before|\allowbreak\token[0.0060]| finding|\allowbreak\token[0.0003]| the|\allowbreak\token[0.8076]| additional|\allowbreak\token[0.3947]| bag|\allowbreak\token[0.8786]| is|\allowbreak\token[0.0000]| \(|\allowbreak\token[0.0001]|7|\allowbreak\token[0.0000]|x|\allowbreak\token[0.2478]|\|\allowbreak\token[0.0619]|).|\allowbreak\token[0.1434]| After|\allowbreak\token[0.0428]| finding|\allowbreak\token[0.0027]| the|\allowbreak\token[0.1605]| bag|\allowbreak\token[0.6926]| with|\allowbreak\token[0.0002]| |\allowbreak\token[0.0000]|5|\allowbreak\token[0.0001]|3|\allowbreak\token[0.0057]| coins|\allowbreak\token[0.1594]|,|\allowbreak\token[0.7222]| you|\allowbreak\token[0.2354]| redistribute|\allowbreak\token[0.8206]| all|\allowbreak\token[0.1524]| the|\allowbreak\token[0.0002]| coins|\allowbreak\token[0.8794]| into|\allowbreak\token[0.0064]| eight|\allowbreak\token[0.0939]| bags|\allowbreak\token[0.7149]| with|\allowbreak\token[0.1606]| an|\allowbreak\token[0.0016]| equal|\allowbreak\token[0.0014]| number|\allowbreak\token[0.0130]| of|\allowbreak\token[0.0000]| coins|\allowbreak\token[1.0803]|,|\allowbreak\token[0.9683]| which|\allowbreak\token[0.4500]| we|\allowbreak\token[0.0907]|'ll|\allowbreak\token[0.0494]| call|\allowbreak\token[0.0000]| \(|\allowbreak\token[0.0059]|y|\allowbreak\token[0.0002]|\|\allowbreak\token[0.0502]|).|\allowbreak\token[0.8690]| This|\allowbreak\token[0.1326]| means|\allowbreak\token[0.1775]| the|\allowbreak\token[0.0617]| total|\allowbreak\token[0.0001]| number|\allowbreak\token[0.0000]| of|\allowbreak\token[0.0000]| coins|\allowbreak\token[0.7692]| after|\allowbreak\token[0.3629]| redistribution|\allowbreak\token[0.0003]| is|\allowbreak\token[0.0081]| \(|\allowbreak\token[0.3541]|8|\allowbreak\token[0.0000]|y|\allowbreak\token[0.0095]|\|\allowbreak\token[0.9877]|).|\allowbreak\token[1.2552]| We|\allowbreak\token[0.8143]| know|\allowbreak\token[0.0045]| that|\allowbreak\token[0.4256]| \(|\allowbreak\token[0.7214]|8|\allowbreak\token[0.0002]|y|\allowbreak\token[0.4738]| =|\allowbreak\token[0.0002]| |\allowbreak\token[0.0008]|7|\allowbreak\token[0.0000]|x|\allowbreak\token[0.0000]| +|\allowbreak\token[0.0000]| |\allowbreak\token[0.0000]|5|\allowbreak\token[0.0001]|3|\allowbreak\token[0.5030]|\)|\allowbreak\token[0.0007]| and|\allowbreak\token[0.1756]| that|\allowbreak\token[0.2138]| \(|\allowbreak\token[0.1231]|8|\allowbreak\token[0.0002]|y|\allowbreak\token[0.0085]| >|\allowbreak\token[0.0000]| |\allowbreak\token[0.0003]|2|\allowbreak\token[0.0001]|0|\allowbreak\token[0.0017]|0|\allowbreak\token[0.0015]|\|\allowbreak\token[0.3414]|).↩↩|\tighttracebreak\allowbreak\token[0.7763]|To|\allowbreak\token[0.0032]| find|\allowbreak\token[0.0006]| the|\allowbreak\token[0.0075]| smallest|\allowbreak\token[0.2374]| number|\allowbreak\token[0.0000]| of|\allowbreak\token[0.0010]| coins|\allowbreak\token[0.9889]| you|\allowbreak\token[0.0005]| could|\allowbreak\token[0.0000]| have|\allowbreak\token[0.0133]| had|\allowbreak\token[1.0033]| before|\allowbreak\token[0.0251]| finding|\allowbreak\token[0.0000]| the|\allowbreak\token[0.4912]| bag|\allowbreak\token[0.0056]| of|\allowbreak\token[0.0000]| |\allowbreak\token[0.0000]|5|\allowbreak\token[0.0001]|3|\allowbreak\token[0.0091]| coins|\allowbreak\token[0.0074]|,|\allowbreak\token[0.0003]| we|\allowbreak\token[0.0132]| need|\allowbreak\token[0.0000]| to|\allowbreak\token[0.0332]| find|\allowbreak\token[0.0000]| the|\allowbreak\token[0.0019]| smallest|\allowbreak\token[0.7751]| value|\allowbreak\token[0.0000]| of|\allowbreak\token[0.0000]| \(|\allowbreak\token[0.3135]|x|\allowbreak\token[0.0001]|\)|\allowbreak\token[0.0042]| that|\allowbreak\token[0.0021]| satisfies|\allowbreak\token[0.1370]| these|\allowbreak\token[0.0000]| conditions|\allowbreak\token[0.0832]|.|\allowbreak\token[0.8255]| Since|\allowbreak\token[0.1962]| \(|\allowbreak\token[0.0956]|8|\allowbreak\token[0.0002]|y|\allowbreak\token[0.2511]| >|\allowbreak\token[0.0000]| |\allowbreak\token[0.0000]|2|\allowbreak\token[0.0001]|0|\allowbreak\token[0.0002]|0|\allowbreak\token[0.0001]|\|\allowbreak\token[0.0000]|),|\allowbreak\token[0.6646]| the|\allowbreak\token[0.0216]| smallest|\allowbreak\token[1.1690]| \(|\allowbreak[...]
\end{twopartbox}
\begin{twopartbox}{Cut Index 102 | Acceptance Probability 1.0000}{Block 2 Step 1}{Proposal (Accepted)}
\token[2.2824]| To|\allowbreak\token[0.4171]| solve|\allowbreak\token[0.4477]| this|\allowbreak\token[0.0840]| problem|\allowbreak\token[0.0111]|,|\allowbreak\token[0.5775]| let|\allowbreak\token[0.0029]|'s|\allowbreak\token[0.8168]| denote|\allowbreak\token[0.0233]| the|\allowbreak\token[0.7993]| number|\allowbreak\token[0.0000]| of|\allowbreak\token[0.1007]| coins|\allowbreak\token[0.0804]| in|\allowbreak\token[0.0205]| each|\allowbreak\token[0.0183]| of|\allowbreak\token[0.0031]| the|\allowbreak\token[0.5575]| original|\allowbreak\token[0.0329]| seven|\allowbreak\token[0.0001]| bags|\allowbreak\token[0.0009]| as|\allowbreak\token[0.6010]| \(|\allowbreak\token[0.8095]|x|\allowbreak\token[0.0048]|\|\allowbreak\token[0.1099]|).|\allowbreak\token[0.9476]| Therefore|\allowbreak\token[0.0000]|,|\allowbreak\token[0.6384]| the|\allowbreak\token[0.0007]| total|\allowbreak\token[0.0001]| number|\allowbreak\token[0.0000]| of|\allowbreak\token[0.0001]| coins|\allowbreak\token[0.5853]| before|\allowbreak\token[0.0055]| finding|\allowbreak\token[0.0003]| the|\allowbreak\token[0.8417]| additional|\allowbreak\token[0.3947]| bag|\allowbreak\token[0.8785]| is|\allowbreak\token[0.0000]| \(|\allowbreak\token[0.0001]|7|\allowbreak\token[0.0000]|x|\allowbreak\token[0.2478]|\|\allowbreak\token[0.0619]|).|\allowbreak\token[0.1434]| After|\allowbreak\token[0.0428]| finding|\allowbreak\token[0.0027]| the|\allowbreak\token[0.1502]| bag|\allowbreak\token[0.6866]| with|\allowbreak\token[0.0002]| |\allowbreak\token[0.0000]|5|\allowbreak\token[0.0001]|3|\allowbreak\token[0.0064]| coins|\allowbreak\token[0.1746]|,|\allowbreak\token[0.7104]| you|\allowbreak\token[0.2338]| redistribute|\allowbreak\token[0.8207]| all|\allowbreak\token[0.1743]| the|\allowbreak\token[0.0002]| coins|\allowbreak\token[0.7935]| into|\allowbreak\token[0.0057]| eight|\allowbreak\token[0.1040]| bags|\allowbreak\token[0.7179]| with|\allowbreak\token[0.1465]| an|\allowbreak\token[0.0013]| equal|\allowbreak\token[0.0010]| number|\allowbreak\token[0.0131]| of|\allowbreak\token[0.0000]| coins|\allowbreak\token[1.0839]|,|\allowbreak\token[0.9759]| let|\allowbreak\token[0.0000]|'s|\allowbreak\token[0.0518]| call|\allowbreak\token[0.0001]| this|\allowbreak\token[0.0940]| number|\allowbreak\token[0.0034]| \(|\allowbreak\token[0.0048]|y|\allowbreak\token[0.0000]|\|\allowbreak\token[0.0825]|).|\allowbreak\token[1.2411]| So|\allowbreak\token[0.0000]|,|\allowbreak\token[0.0321]| the|\allowbreak\token[0.0783]| total|\allowbreak\token[0.0000]| number|\allowbreak\token[0.0000]| of|\allowbreak\token[0.0000]| coins|\allowbreak\token[0.6227]| after|\allowbreak\token[0.3939]| redistribution|\allowbreak\token[0.0004]| is|\allowbreak\token[0.0082]| \(|\allowbreak\token[0.6478]|8|\allowbreak\token[0.0000]|y|\allowbreak\token[0.0177]|\|\allowbreak\token[0.9802]|).↩↩|\tighttracebreak\allowbreak\token[0.2592]|Given|\allowbreak\token[0.0010]| that|\allowbreak\token[0.7182]| you|\allowbreak\token[0.3932]| have|\allowbreak\token[0.0104]| more|\allowbreak\token[0.0000]| than|\allowbreak\token[0.0000]| |\allowbreak\token[0.0000]|2|\allowbreak\token[0.0000]|0|\allowbreak\token[0.0001]|0|\allowbreak\token[0.0000]| coins|\allowbreak\cut{[CUT]}\allowbreak\token[0.9615]| after|\allowbreak\token[0.2218]| redistribution|\allowbreak\token[0.0331]|,|\allowbreak\token[0.0294]| we|\allowbreak\token[0.5636]| have|\allowbreak\token[0.3679]| \(|\allowbreak\token[0.0011]|8|\allowbreak\token[0.0000]|y|\allowbreak\token[0.0000]| >|\allowbreak\token[0.0000]| |\allowbreak\token[0.0000]|2|\allowbreak\token[0.0000]|0|\allowbreak\token[0.0001]|0|\allowbreak\token[0.0001]|\|\allowbreak\token[0.2913]|),|\allowbreak\token[0.0003]| which|\allowbreak\token[0.0865]| simpl|\allowbreak\token[0.0000]|ifies|\allowbreak\token[0.0000]| to|\allowbreak\token[0.0000]| \(|\allowbreak\token[0.0000]|y|\allowbreak\token[0.0000]| >|\allowbreak\token[0.0000]| |\allowbreak\token[0.0000]|2|\allowbreak\token[0.0000]|5|\allowbreak\token[0.0001]|\|\allowbreak\token[0.1752]|).|\allowbreak\token[0.8695]| Since|\allowbreak\token[0.7806]| \(|\allowbreak\token[0.1099]|y|\allowbreak\token[0.0215]|\)|\allowbreak\token[0.8159]| must|\allowbreak\token[0.0002]| be|\allowbreak\token[0.1456]| an|\allowbreak\token[0.0000]| integer|\allowbreak\token[0.4575]|,|\allowbreak\token[0.0010]| the|\allowbreak\token[0.0038]| smallest|\allowbreak\token[0.6622]| possible|\allowbreak\token[0.0002]| value|\allowbreak\token[0.0000]| for|\allowbreak\token[0.0000]| \(|\allowbreak\token[0.0000]|y|\allowbreak\token[0.0000]|\)|\allowbreak\token[0.0173]| is|\allowbreak\token[0.0140]| |\allowbreak\token[0.0001]|2|\allowbreak\token[0.0014]|6|\allowbreak\token[0.5697]|.↩↩|\tighttracebreak\allowbreak\token[0.8369]|Now|\allowbreak\token[0.0000]|,|\allowbreak\token[0.5838]| we|\allowbreak\token[0.1722]| know|\allowbreak\token[0.0013]| that|\allowbreak\token[0.8064]| the|\allowbreak\token[0.0004]| total|\allowbreak\token[0.0002]| number|\allowbreak\token[0.0000]| of|\allowbreak\token[0.0000]| coins|\allowbreak\token[0.5237]| after|\allowbreak\token[0.2302]| redistribution|\allowbreak\token[0.0241]| is|\allowbreak\token[0.8472]| \(|\allowbreak\token[0.5679]|8|\allowbreak\token[0.6473]| \|\allowbreak\token[0.0000]|times|\allowbreak\token[0.0001]| |\allowbreak\token[0.0000]|2|\allowbreak\token[0.0000]|6|\allowbreak\token[0.0001]| =|\allowbreak\token[0.0000]| |\allowbreak\token[0.0007]|2|\allowbreak\token[0.0000]|0|\allowbreak\token[0.0000]|8|\allowbreak\token[0.0015]|\|\allowbreak\token[0.0637]|).|\allowbreak\token[0.9022]| This|\allowbreak\token[0.3194]| total|\allowbreak\token[0.0876]| includes|\allowbreak\token[0.0003]| the|\allowbreak\token[0.2205]| |\allowbreak\token[0.0001]|5|\allowbreak\token[0.0000]|3|\allowbreak\token[0.0067]| coins|\allowbreak\token[0.3874]| from|\allowbreak\token[0.0000]| the|\allowbreak\token[0.9499]| additional|\allowbreak\token[0.0001]| bag|\allowbreak\token[0.3194]|,|\allowbreak\token[0.0235]| so|\allowbreak\token[0.2692]| the|\allowbreak\token[1.0088]| original|\allowbreak\token[0.8358]| total|\allowbreak\token[0.8178]| number|\allowbreak\token[0.0000]| of|\allowbreak\token[0.0000]| coins|\allowbreak\token[0.7526]| was|\allowbreak\token[0.0000]| \(|\allowbreak\token[0.0001]|2|\allowbreak\token[0.0000]|0|\allowbreak\token[0.0000]|8|\allowbreak\token[0.0000]| -|\allowbreak\token[0.0000]| |\allowbreak\token[0.0000]|5|\allowbreak\token[0.0000]|3|\allowbreak\token[0.0000]| =|\allowbreak\token[0.0000]| |\allowbreak\token[0.0000]|1|\allowbreak\token[0.0000]|5|\allowbreak\token[0.0001]|5|\allowbreak\token[0.0016]|\|\allowbreak\token[0.6926]|).↩↩|\tighttracebreak\allowbreak\token[0.9242]|Since|\allowbreak\token[0.9426]| the|\allowbreak\token[0.7577]| original|\allowbreak\token[1.1103]| total|\allowbreak\token[0.3947]| was|\allowbreak\token[1.1891]| \(|\allowbreak\token[0.0157]|7|\allowbreak\token[0.0000]|x|\allowbreak\token[0.4525]|\|\allowbreak\token[0.0000]|),|\allowbreak\token[0.3157]| we|\allowbreak\token[0.2040]| have|\allowbreak\token[0.0001]| \(|\allowbreak\token[0.0011]|7|\allowbreak\token[0.0000]|x|\allowbreak\token[0.0000]| =|\allowbreak\token[0.0000]| |\allowbreak\token[0.0000]|1|\allowbreak\token[0.0000]|5|\allowbreak\token[0.0001]|5|\allowbreak\token[0.0000]|\|\allowbreak\token[0.4193]|).|\allowbreak\token[0.7442]| Sol|\allowbreak\token[0.0000]|ving|\allowbreak\token[0.0000]| for|\allowbreak\token[0.0000]| \(|\allowbreak\token[0.0000]|x|\allowbreak\token[0.0739]|\|\allowbreak\token[0.0000]|),|\allowbreak\token[0.0000]| we|\allowbreak\token[0.6765]| get|\allowbreak\token[0.0000]| \(|\allowbreak[...]
\tcblower
\token[0.9826]| after|\allowbreak\token[0.2219]| redistribution|\allowbreak\token[0.0331]|,|\allowbreak\token[0.0307]| we|\allowbreak\token[0.5636]| have|\allowbreak\token[0.3679]| \(|\allowbreak\token[0.0010]|8|\allowbreak\token[0.0000]|y|\allowbreak\token[0.0000]| >|\allowbreak\token[0.0000]| |\allowbreak\token[0.0000]|2|\allowbreak\token[0.0000]|0|\allowbreak\token[0.0001]|0|\allowbreak\token[0.0001]|\|\allowbreak\token[0.3151]|),|\allowbreak\token[0.0003]| which|\allowbreak\token[0.1058]| simpl|\allowbreak\token[0.0000]|ifies|\allowbreak\token[0.0000]| to|\allowbreak\token[0.0000]| \(|\allowbreak\token[0.0000]|y|\allowbreak\token[0.0000]| >|\allowbreak\token[0.0000]| |\allowbreak\token[0.0000]|2|\allowbreak\token[0.0000]|5|\allowbreak\token[0.0001]|\|\allowbreak\token[0.1915]|).|\allowbreak\token[0.8687]| Since|\allowbreak\token[0.7162]| \(|\allowbreak\token[0.1102]|y|\allowbreak\token[0.0193]|\)|\allowbreak\token[0.8193]| is|\allowbreak\token[0.2219]| the|\allowbreak\token[0.0858]| number|\allowbreak\token[0.0000]| of|\allowbreak\token[0.0000]| coins|\allowbreak\token[0.0039]| in|\allowbreak\token[0.0002]| each|\allowbreak\token[0.2691]| of|\allowbreak\token[0.0000]| the|\allowbreak\token[0.0053]| eight|\allowbreak\token[0.0006]| bags|\allowbreak\token[0.3178]| after|\allowbreak\token[0.0001]| redistribution|\allowbreak\token[0.0007]|,|\allowbreak\token[0.4196]| and|\allowbreak\token[1.4536]| you|\allowbreak\token[2.1574]| redistributed|\allowbreak\token[0.3095]| all|\allowbreak\token[0.1122]| the|\allowbreak\token[0.0161]| coins|\allowbreak\token[0.6253]|,|\allowbreak\token[0.5869]| we|\allowbreak\token[0.7468]| have|\allowbreak\token[0.1412]| \(|\allowbreak\token[0.6637]|7|\allowbreak\token[0.0000]|x|\allowbreak\token[0.0000]| +|\allowbreak\token[0.0000]| |\allowbreak\token[0.0000]|5|\allowbreak\token[0.0000]|3|\allowbreak\token[0.0000]| =|\allowbreak\token[0.0000]| |\allowbreak\token[0.0000]|8|\allowbreak\token[0.0000]|y|\allowbreak\token[0.0000]|\|\allowbreak\token[0.0053]|).↩↩|\tighttracebreak\allowbreak\token[0.4369]|To|\allowbreak\token[0.0006]| find|\allowbreak\token[0.0000]| the|\allowbreak\token[0.0025]| smallest|\allowbreak\token[0.0261]| number|\allowbreak\token[0.0000]| of|\allowbreak\token[0.0002]| coins|\allowbreak\token[0.2616]| you|\allowbreak\token[0.0002]| could|\allowbreak\token[0.0000]| have|\allowbreak\token[0.0026]| had|\allowbreak\token[0.2728]| before|\allowbreak\token[0.0045]| finding|\allowbreak\token[0.0000]| the|\allowbreak\token[0.2073]| bag|\allowbreak\token[0.0020]| of|\allowbreak\token[0.0000]| |\allowbreak\token[0.0000]|5|\allowbreak\token[0.0000]|3|\allowbreak\token[0.0066]| coins|\allowbreak\token[0.0496]|,|\allowbreak\token[0.0000]| we|\allowbreak\token[0.0013]| need|\allowbreak\token[0.0002]| to|\allowbreak\token[0.0013]| find|\allowbreak\token[0.0000]| the|\allowbreak\token[0.0004]| smallest|\allowbreak\token[0.7334]| \(|\allowbreak\token[0.0419]|x|\allowbreak\token[0.0001]|\)|\allowbreak\token[0.1331]| that|\allowbreak\token[0.0116]| satisfies|\allowbreak\token[0.8051]| these|\allowbreak\token[0.0000]| conditions|\allowbreak\token[0.0103]|.|\allowbreak\token[0.6609]| We|\allowbreak\token[0.2249]| know|\allowbreak\token[0.5822]| that|\allowbreak\token[0.0006]| \(|\allowbreak\token[1.1788]|y|\allowbreak\token[0.5635]|\)|\allowbreak\token[0.0128]| must|\allowbreak\token[0.0001]| be|\allowbreak\token[0.1046]| an|\allowbreak\token[0.0000]| integer|\allowbreak\token[0.5408]|,|\allowbreak\token[0.0687]| and|\allowbreak\token[0.5179]| since|\allowbreak\token[0.0424]| \(|\allowbreak\token[0.7812]|y|\allowbreak\token[0.0027]| >|\allowbreak\token[0.0000]| |\allowbreak\token[0.0000]|2|\allowbreak\token[0.0002]|5|\allowbreak\token[0.0003]|\|\allowbreak\token[0.0000]|),|\allowbreak\token[0.6853]| the|\allowbreak\token[0.0003]| smallest|\allowbreak\token[0.8615]| possible|\allowbreak\token[0.2900]| value|\allowbreak\token[0.0000]| for|\allowbreak\token[0.0000]| \(|\allowbreak\token[0.0015]|y|\allowbreak\token[0.0000]|\)|\allowbreak\token[0.5570]| is|\allowbreak\token[0.0073]| |\allowbreak\token[0.0004]|2|\allowbreak\token[0.0024]|6|\allowbreak\token[0.5418]|.|\allowbreak\token[0.8446]| Sub|\allowbreak\token[0.0000]|stit|\allowbreak\token[0.0016]|uting|\allowbreak\token[0.0001]| \(|\allowbreak\token[0.0000]|y|\allowbreak\token[0.0000]| =|\allowbreak\token[0.0000]| |\allowbreak\token[0.0000]|2|\allowbreak[...]
\end{twopartbox}
\begin{twopartbox}{Cut Index 366 | Acceptance Probability 1.0000}{Block 2 Step 2}{Proposal (Accepted)}
[...]\allowbreak\token[0.0002]|5|\allowbreak\token[0.0003]|\|\allowbreak\token[0.0000]|),|\allowbreak\token[0.6853]| the|\allowbreak\token[0.0003]| smallest|\allowbreak\token[0.8615]| possible|\allowbreak\token[0.2900]| value|\allowbreak\token[0.0000]| for|\allowbreak\token[0.0000]| \(|\allowbreak\token[0.0015]|y|\allowbreak\token[0.0000]|\)|\allowbreak\token[0.5570]| is|\allowbreak\token[0.0073]| |\allowbreak\token[0.0004]|2|\allowbreak\token[0.0024]|6|\allowbreak\token[0.5418]|.|\allowbreak\token[0.8446]| Sub|\allowbreak\token[0.0000]|stit|\allowbreak\token[0.0016]|uting|\allowbreak\token[0.0001]| \(|\allowbreak\token[0.0000]|y|\allowbreak\token[0.0000]| =|\allowbreak\token[0.0000]| |\allowbreak\token[0.0000]|2|\allowbreak\token[0.0000]|6|\allowbreak\token[0.0003]|\)|\allowbreak\token[0.0001]| into|\allowbreak\token[0.3175]| the|\allowbreak\token[0.0001]| equation|\allowbreak\token[0.5772]| gives|\allowbreak\token[0.0000]| us|\allowbreak\token[0.0198]| \(|\allowbreak\token[0.0239]|7|\allowbreak\token[0.0000]|x|\allowbreak\token[0.0000]| +|\allowbreak\token[0.0000]| |\allowbreak\token[0.0000]|5|\allowbreak\token[0.0000]|3|\allowbreak\token[0.0000]| =|\allowbreak\token[0.0000]| |\allowbreak\token[0.2474]|8|\allowbreak\token[0.0268]| \|\allowbreak\token[0.0000]|times|\allowbreak\token[0.0000]| |\allowbreak\token[0.0000]|2|\allowbreak\token[0.0001]|6|\allowbreak\token[0.0140]|\|\allowbreak\token[0.0353]|),|\allowbreak\token[0.0000]| which|\allowbreak\token[0.0000]| simpl|\allowbreak\token[0.0000]|ifies|\allowbreak\token[0.0000]| to|\allowbreak\token[0.0000]| \(|\allowbreak\token[0.0000]|7|\allowbreak\token[0.0000]|x|\allowbreak\token[0.6466]| =|\allowbreak\token[0.0001]| |\allowbreak\token[0.5905]|2|\allowbreak\token[0.1225]|0|\allowbreak\token[0.0661]|8|\allowbreak\token[0.0053]| -|\allowbreak\token[0.0000]| |\allowbreak\token[0.0000]|5|\allowbreak\token[0.0000]|3|\allowbreak\token[0.6631]|\|\allowbreak\token[0.0001]|),|\allowbreak\token[0.5531]| and|\allowbreak\token[0.0786]| further|\allowbreak\token[0.0113]| to|\allowbreak\token[0.0000]| \(|\allowbreak\token[0.0000]|7|\allowbreak\token[0.0000]|x|\allowbreak\token[0.0000]| =|\allowbreak\token[0.0000]| |\allowbreak\token[0.0000]|1|\allowbreak\token[0.0000]|5|\allowbreak\token[0.0001]|5|\allowbreak\token[0.0000]|\|\allowbreak\token[0.1493]|).|\allowbreak\token[0.7309]| Sol|\allowbreak\token[0.0000]|ving|\allowbreak\token[0.0000]| for|\allowbreak\token[0.0000]| \(|\allowbreak\token[0.0000]|x|\allowbreak\token[0.0817]|\)|\allowbreak\token[0.0002]| gives|\allowbreak\token[0.1094]| us|\allowbreak\token[0.0000]| \(|\allowbreak\token[0.0000]|x|\allowbreak\token[0.0066]| =|\allowbreak\token[0.2472]| \|\allowbreak\token[0.0000]|frac|\allowbreak\token[0.0000]|{|\allowbreak\token[0.0000]|1|\allowbreak\token[0.0000]|5|\allowbreak\token[0.0000]|5|\allowbreak\token[0.0000]|}{|\allowbreak\token[0.0000]|7|\allowbreak\token[0.0605]|}\|\allowbreak\token[0.0107]|),|\allowbreak\token[0.1094]| which|\allowbreak\token[0.0047]| is|\allowbreak\token[0.0214]| not|\allowbreak\token[0.0024]| an|\allowbreak\token[0.0000]| integer|\allowbreak\token[1.0421]|.↩↩|\tighttracebreak\allowbreak\token[0.9851]|We|\allowbreak\token[0.1916]| need|\allowbreak\token[0.2722]| to|\allowbreak\token[0.0034]| find|\allowbreak\token[0.0006]| the|\allowbreak\token[0.0742]| smallest|\allowbreak\token[0.3669]| integer|\allowbreak\token[0.2720]| value|\allowbreak\token[0.0445]| of|\allowbreak\token[0.0000]| \(|\allowbreak\token[0.6760]|x|\allowbreak\token[0.0030]|\)|\allowbreak\token[0.0135]| that|\allowbreak\token[0.9107]| makes|\allowbreak\token[0.1639]| \(|\allowbreak\token[0.4046]|7|\allowbreak\token[0.0000]|x|\allowbreak\token[0.0406]| +|\allowbreak\token[0.0000]| |\allowbreak\token[0.0000]|5|\allowbreak\token[0.0000]|3|\allowbreak\token[0.0020]|\)|\allowbreak\token[0.4764]| divisible|\allowbreak\token[0.0000]| by|\allowbreak\token[0.0001]| |\allowbreak\token[0.0002]|8|\allowbreak\token[0.5645]|.|\allowbreak\token[1.4133]| By|\allowbreak\token[0.5451]| trial|\allowbreak\token[0.0018]| and|\allowbreak\token[0.0001]| error|\allowbreak\token[0.0156]| or|\allowbreak\token[0.5605]| systematic|\allowbreak\token[0.1292]| calculation|\allowbreak\token[0.0000]|,|\allowbreak\token[0.0000]| we|\allowbreak\token[0.0000]| find|\allowbreak\token[0.0000]| that|\allowbreak\cut{[CUT]}\allowbreak\token[0.7327]| when|\allowbreak\token[0.0000]| \(|\allowbreak\token[0.6953]|x|\allowbreak\token[0.0000]| =|\allowbreak\token[0.0001]| |\allowbreak\token[0.3471]|2|\allowbreak\token[1.6396]|2|\allowbreak\token[0.0058]|\|\allowbreak\token[0.0000]|),|\allowbreak\token[0.2726]| \(|\allowbreak\token[0.0000]|7|\allowbreak\token[0.3122]|x|\allowbreak\token[0.5570]| +|\allowbreak\token[0.0000]| |\allowbreak\token[0.0000]|5|\allowbreak\token[0.0000]|3|\allowbreak\token[0.0193]| =|\allowbreak\token[0.0000]| |
\tcblower
\token[0.7253]| \(|\allowbreak\token[0.1021]|x|\allowbreak\token[0.0000]| =|\allowbreak\token[0.0002]| |\allowbreak\token[0.3289]|2|\allowbreak\token[1.5403]|2|\allowbreak\token[0.0004]|\)|\allowbreak\token[0.2730]| is|\allowbreak\token[0.0012]| the|\allowbreak\token[0.0093]| smallest|\allowbreak\token[0.3354]| integer|\allowbreak\token[0.0754]| that|\allowbreak\token[0.6285]| works|\allowbreak\token[0.1752]|,|\allowbreak\token[0.4760]| because|\allowbreak\token[0.0001]| \(|\allowbreak\token[0.0002]|7|\allowbreak\token[0.0053]| \|
\end{twopartbox}
\begin{twopartbox}{Cut Index 90 | Acceptance Probability 0.0000}{Block 2 Step 3}{Proposal (Rejected)}
\token[2.2824]| To|\allowbreak\token[0.4171]| solve|\allowbreak\token[0.4477]| this|\allowbreak\token[0.0840]| problem|\allowbreak\token[0.0111]|,|\allowbreak\token[0.5775]| let|\allowbreak\token[0.0029]|'s|\allowbreak\token[0.8168]| denote|\allowbreak\token[0.0233]| the|\allowbreak\token[0.7993]| number|\allowbreak\token[0.0000]| of|\allowbreak\token[0.1007]| coins|\allowbreak\token[0.0804]| in|\allowbreak\token[0.0205]| each|\allowbreak\token[0.0183]| of|\allowbreak\token[0.0031]| the|\allowbreak\token[0.5575]| original|\allowbreak\token[0.0329]| seven|\allowbreak\token[0.0001]| bags|\allowbreak\token[0.0009]| as|\allowbreak\token[0.6010]| \(|\allowbreak\token[0.8095]|x|\allowbreak\token[0.0048]|\|\allowbreak\token[0.1099]|).|\allowbreak\token[0.9476]| Therefore|\allowbreak\token[0.0000]|,|\allowbreak\token[0.6384]| the|\allowbreak\token[0.0007]| total|\allowbreak\token[0.0001]| number|\allowbreak\token[0.0000]| of|\allowbreak\token[0.0001]| coins|\allowbreak\token[0.5853]| before|\allowbreak\token[0.0055]| finding|\allowbreak\token[0.0003]| the|\allowbreak\token[0.8417]| additional|\allowbreak\token[0.3947]| bag|\allowbreak\token[0.8785]| is|\allowbreak\token[0.0000]| \(|\allowbreak\token[0.0001]|7|\allowbreak\token[0.0000]|x|\allowbreak\token[0.2478]|\|\allowbreak\token[0.0619]|).|\allowbreak\token[0.1434]| After|\allowbreak\token[0.0428]| finding|\allowbreak\token[0.0027]| the|\allowbreak\token[0.1502]| bag|\allowbreak\token[0.6866]| with|\allowbreak\token[0.0002]| |\allowbreak\token[0.0000]|5|\allowbreak\token[0.0001]|3|\allowbreak\token[0.0064]| coins|\allowbreak\token[0.1746]|,|\allowbreak\token[0.7104]| you|\allowbreak\token[0.2338]| redistribute|\allowbreak\token[0.8207]| all|\allowbreak\token[0.1743]| the|\allowbreak\token[0.0002]| coins|\allowbreak\token[0.7935]| into|\allowbreak\token[0.0057]| eight|\allowbreak\token[0.1040]| bags|\allowbreak\token[0.7179]| with|\allowbreak\token[0.1465]| an|\allowbreak\token[0.0013]| equal|\allowbreak\token[0.0010]| number|\allowbreak\token[0.0131]| of|\allowbreak\token[0.0000]| coins|\allowbreak\token[1.0839]|,|\allowbreak\token[0.9759]| let|\allowbreak\token[0.0000]|'s|\allowbreak\token[0.0518]| call|\allowbreak\token[0.0001]| this|\allowbreak\token[0.0940]| number|\allowbreak\token[0.0034]| \(|\allowbreak\token[0.0048]|y|\allowbreak\token[0.0000]|\|\allowbreak\token[0.0825]|).|\allowbreak\token[1.2411]| So|\allowbreak\token[0.0000]|,|\allowbreak\token[0.0321]| the|\allowbreak\token[0.0783]| total|\allowbreak\token[0.0000]| number|\allowbreak\token[0.0000]| of|\allowbreak\token[0.0000]| coins|\allowbreak\token[0.6227]| after|\allowbreak\token[0.3939]| redistribution|\allowbreak\token[0.0004]| is|\allowbreak\token[0.0082]| \(|\allowbreak\token[0.6478]|8|\allowbreak\token[0.0000]|y|\allowbreak\token[0.0177]|\|\allowbreak\cut{[CUT]}\allowbreak\token[0.9802]|).↩↩|\tighttracebreak\allowbreak\token[0.2592]|Given|\allowbreak\token[0.0010]| that|\allowbreak\token[0.7182]| you|\allowbreak\token[0.3932]| have|\allowbreak\token[0.0104]| more|\allowbreak\token[0.0000]| than|\allowbreak\token[0.0000]| |\allowbreak\token[0.0000]|2|\allowbreak\token[0.0000]|0|\allowbreak\token[0.0001]|0|\allowbreak\token[0.0000]| coins|\allowbreak\token[0.9826]| after|\allowbreak\token[0.2219]| redistribution|\allowbreak\token[0.0331]|,|\allowbreak\token[0.0307]| we|\allowbreak\token[0.5636]| have|\allowbreak\token[0.3679]| \(|\allowbreak\token[0.0010]|8|\allowbreak\token[0.0000]|y|\allowbreak\token[0.0000]| >|\allowbreak\token[0.0000]| |\allowbreak\token[0.0000]|2|\allowbreak\token[0.0000]|0|\allowbreak\token[0.0001]|0|\allowbreak\token[0.0001]|\|\allowbreak\token[0.3151]|),|\allowbreak\token[0.0003]| which|\allowbreak\token[0.1058]| simpl|\allowbreak\token[0.0000]|ifies|\allowbreak\token[0.0000]| to|\allowbreak\token[0.0000]| \(|\allowbreak\token[0.0000]|y|\allowbreak\token[0.0000]| >|\allowbreak\token[0.0000]| |\allowbreak\token[0.0000]|2|\allowbreak\token[0.0000]|5|\allowbreak\token[0.0001]|\|\allowbreak\token[0.1915]|).|\allowbreak\token[0.8687]| Since|\allowbreak\token[0.7162]| \(|\allowbreak\token[0.1102]|y|\allowbreak\token[0.0193]|\)|\allowbreak\token[0.8193]| is|\allowbreak\token[0.2219]| the|\allowbreak\token[0.0858]| number|\allowbreak\token[0.0000]| of|\allowbreak\token[0.0000]| coins|\allowbreak\token[0.0039]| in|\allowbreak\token[0.0002]| each|\allowbreak\token[0.2691]| of|\allowbreak\token[0.0000]| the|\allowbreak\token[0.0053]| eight|\allowbreak\token[0.0006]| bags|\allowbreak\token[0.3178]| after|\allowbreak\token[0.0001]| redistribution|\allowbreak\token[0.0007]|,|\allowbreak\token[0.4196]| and|\allowbreak\token[1.4536]| you|\allowbreak\token[2.1574]| redistributed|\allowbreak\token[0.3095]| all|\allowbreak\token[0.1122]| the|\allowbreak\token[0.0161]| coins|\allowbreak\token[0.6253]|,|\allowbreak\token[0.5869]| we|\allowbreak\token[0.7468]| have|\allowbreak\token[0.1412]| \(|\allowbreak\token[0.6637]|7|\allowbreak\token[0.0000]|x|\allowbreak\token[0.0000]| +|\allowbreak\token[0.0000]| |\allowbreak\token[0.0000]|5|\allowbreak\token[0.0000]|3|\allowbreak\token[0.0000]| =|\allowbreak\token[0.0000]| |\allowbreak\token[0.0000]|8|\allowbreak\token[0.0000]|y|\allowbreak\token[0.0000]|\|\allowbreak\token[0.0053]|).↩↩|\tighttracebreak\allowbreak\token[0.4369]|To|\allowbreak\token[0.0006]| find|\allowbreak\token[0.0000]| the|\allowbreak\token[0.0025]| smallest|\allowbreak\token[0.0261]| number|\allowbreak\token[0.0000]| of|\allowbreak\token[0.0002]| coins|\allowbreak\token[0.2616]| you|\allowbreak\token[0.0002]| could|\allowbreak\token[0.0000]| have|\allowbreak\token[0.0026]| had|\allowbreak\token[0.2728]| before|\allowbreak\token[0.0045]| finding|\allowbreak\token[0.0000]| the|\allowbreak\token[0.2073]| bag|\allowbreak\token[0.0020]| of|\allowbreak\token[0.0000]| |\allowbreak\token[0.0000]|5|\allowbreak\token[0.0000]|3|\allowbreak\token[0.0066]| coins|\allowbreak\token[0.0496]|,|\allowbreak\token[0.0000]| we|\allowbreak\token[0.0013]| need|\allowbreak\token[0.0002]| to|\allowbreak\token[0.0013]| find|\allowbreak\token[0.0000]| the|\allowbreak\token[0.0004]| smallest|\allowbreak\token[0.7334]| \(|\allowbreak\token[0.0419]|x|\allowbreak\token[0.0001]|\)|\allowbreak\token[0.1331]| that|\allowbreak\token[0.0116]| satisfies|\allowbreak\token[0.8051]| these|\allowbreak\token[0.0000]| conditions|\allowbreak\token[0.0103]|.|\allowbreak\token[0.6609]| We|\allowbreak\token[0.2249]| know|\allowbreak\token[0.5822]| that|\allowbreak\token[0.0006]| \(|\allowbreak\token[1.1788]|y|\allowbreak\token[0.5635]|\)|\allowbreak\token[0.0128]| must|\allowbreak\token[0.0001]| be|\allowbreak\token[0.1046]| an|\allowbreak\token[0.0000]| integer|\allowbreak\token[0.5408]|,|\allowbreak\token[0.0687]| and|\allowbreak\token[0.5179]| since|\allowbreak\token[0.0424]| \(|\allowbreak\token[0.7812]|y|\allowbreak\token[0.0027]| >|\allowbreak\token[0.0000]| |\allowbreak\token[0.0000]|2|\allowbreak\token[0.0002]|5|\allowbreak\token[0.0003]|\|\allowbreak\token[0.0000]|),|\allowbreak\token[0.6853]| the|\allowbreak\token[0.0003]| smallest|\allowbreak\token[0.8615]| possible|\allowbreak\token[0.2900]| value|\allowbreak\token[0.0000]| for|\allowbreak\token[0.0000]| \(|\allowbreak\token[0.0015]|y|\allowbreak\token[0.0000]|\)|\allowbreak\token[0.5570]| is|\allowbreak[...]
\tcblower
\token[0.9877]|).↩↩|\tighttracebreak\allowbreak\token[0.2579]|Given|\allowbreak\token[0.0010]| that|\allowbreak\token[0.6971]| you|\allowbreak\token[0.3961]| have|\allowbreak\token[0.0104]| more|\allowbreak\token[0.0000]| than|\allowbreak\token[0.0000]| |\allowbreak\token[0.0000]|2|\allowbreak\token[0.0000]|0|\allowbreak\token[0.0001]|0|\allowbreak\token[0.0000]| coins|\allowbreak\token[0.9828]| after|\allowbreak\token[0.2219]| redistribution|\allowbreak\token[0.0367]|,|\allowbreak\token[0.0306]| we|\allowbreak\token[0.5636]| have|\allowbreak\token[0.3679]| \(|\allowbreak\token[0.0011]|8|\allowbreak\token[0.0000]|y|\allowbreak\token[0.0000]| >|\allowbreak\token[0.0000]| |\allowbreak\token[0.0000]|2|\allowbreak\token[0.0000]|0|\allowbreak\token[0.0001]|0|\allowbreak\token[0.0001]|\|\allowbreak\token[0.3150]|),|\allowbreak\token[0.0003]| which|\allowbreak\token[0.0998]| simpl|\allowbreak\token[0.0000]|ifies|\allowbreak\token[0.0000]| to|\allowbreak\token[0.0000]| \(|\allowbreak\token[0.0000]|y|\allowbreak\token[0.0000]| >|\allowbreak\token[0.0000]| |\allowbreak\token[0.0000]|2|\allowbreak\token[0.0000]|5|\allowbreak\token[0.0001]|\|\allowbreak\token[0.1752]|).|\allowbreak\token[0.9793]| Since|\allowbreak\token[0.7158]| \(|\allowbreak\token[0.1095]|y|\allowbreak\token[0.0215]|\)|\allowbreak\token[0.8258]| must|\allowbreak\token[0.0002]| be|\allowbreak\token[0.1456]| an|\allowbreak\token[0.0000]| integer|\allowbreak\token[0.4248]|,|\allowbreak\token[0.0013]| the|\allowbreak\token[0.0038]| smallest|\allowbreak\token[0.5809]| possible|\allowbreak\token[0.0001]| value|\allowbreak\token[0.0000]| for|\allowbreak\token[0.0000]| \(|\allowbreak\token[0.0000]|y|\allowbreak\token[0.0000]|\)|\allowbreak\token[0.0173]| is|\allowbreak\token[0.0156]| |\allowbreak\token[0.0001]|2|\allowbreak\token[0.0014]|6|\allowbreak\token[0.5696]|.↩↩|\tighttracebreak\allowbreak\token[0.8301]|The|\allowbreak\token[0.0967]| total|\allowbreak\token[0.0001]| number|\allowbreak\token[0.0000]| of|\allowbreak\token[0.0000]| coins|\allowbreak\token[0.3614]| after|\allowbreak\token[0.2024]| redistribution|\allowbreak\token[0.0341]| is|\allowbreak\token[0.4837]| \(|\allowbreak\token[0.4718]|8|\allowbreak\token[0.4475]| \|\allowbreak\token[0.0000]|times|\allowbreak\token[0.0001]| |\allowbreak\token[0.0000]|2|\allowbreak\token[0.0000]|6|\allowbreak\token[0.0000]| =|\allowbreak\token[0.0000]| |\allowbreak\token[0.0004]|2|\allowbreak\token[0.0000]|0|\allowbreak\token[0.0000]|8|\allowbreak\token[0.0010]|\|\allowbreak\token[0.1063]|).|\allowbreak\token[1.1501]| Before|\allowbreak\token[0.7529]| finding|\allowbreak\token[0.0001]| the|\allowbreak\token[0.6993]| additional|\allowbreak\token[0.1753]| bag|\allowbreak\token[0.3437]|,|\allowbreak\token[0.6934]| you|\allowbreak\token[0.0323]| had|\allowbreak\token[0.0520]| \(|\allowbreak\token[0.2069]|2|\allowbreak\token[0.0000]|0|\allowbreak\token[0.0001]|8|\allowbreak\token[0.0006]| -|\allowbreak\token[0.0000]| |\allowbreak\token[0.0001]|5|\allowbreak\token[0.0000]|3|\allowbreak\token[0.0008]| =|\allowbreak\token[0.0000]| |\allowbreak\token[0.0000]|1|\allowbreak\token[0.0000]|5|\allowbreak\token[0.0000]|5|\allowbreak\token[0.0009]|\)|\allowbreak\token[0.0003]| coins|\allowbreak\token[0.8978]|.|\allowbreak\token[0.8506]| However|\allowbreak\token[0.0000]|,|\allowbreak\token[0.9037]| this|\allowbreak\token[1.1292]| is|\allowbreak\token[0.4732]| the|\allowbreak\token[0.1665]| total|\allowbreak\token[1.1061]| number|\allowbreak\token[0.0183]| of|\allowbreak\token[0.0002]| coins|\allowbreak\token[0.2485]| in|\allowbreak\token[0.0159]| the|\allowbreak\token[0.5455]| seven|\allowbreak\token[0.6289]| bags|\allowbreak\token[0.8430]|,|\allowbreak\token[0.8333]| so|\allowbreak\token[1.2229]| the|\allowbreak\token[0.6212]| number|\allowbreak\token[0.0006]| of|\allowbreak\token[0.0002]| coins|\allowbreak\token[0.0746]| in|\allowbreak\token[0.0029]| each|\allowbreak\token[0.5930]| of|\allowbreak\token[0.1727]| the|\allowbreak\token[0.3404]| original|\allowbreak\token[0.6480]| bags|\allowbreak\token[0.3958]| is|\allowbreak\token[0.1462]| \(|\allowbreak\token[0.0458]|1|\allowbreak\token[0.0000]|5|\allowbreak\token[0.0000]|5|\allowbreak\token[0.5854]| /|\allowbreak\token[0.0000]| |\allowbreak\token[0.0000]|7|\allowbreak\token[0.3423]|\|\allowbreak\token[0.2183]|),|\allowbreak\token[0.0081]| which|\allowbreak\token[0.0688]| is|\allowbreak[...]
\end{twopartbox}
\begin{twopartbox}{Cut Index 60 | Acceptance Probability 1.0000}{Block 3 Step 1}{Proposal (Accepted)}
\token[2.2824]| To|\allowbreak\token[0.4171]| solve|\allowbreak\token[0.4477]| this|\allowbreak\token[0.0840]| problem|\allowbreak\token[0.0111]|,|\allowbreak\token[0.5775]| let|\allowbreak\token[0.0029]|'s|\allowbreak\token[0.8168]| denote|\allowbreak\token[0.0233]| the|\allowbreak\token[0.7993]| number|\allowbreak\token[0.0000]| of|\allowbreak\token[0.1007]| coins|\allowbreak\token[0.0804]| in|\allowbreak\token[0.0205]| each|\allowbreak\token[0.0183]| of|\allowbreak\token[0.0031]| the|\allowbreak\token[0.5575]| original|\allowbreak\token[0.0329]| seven|\allowbreak\token[0.0001]| bags|\allowbreak\token[0.0009]| as|\allowbreak\token[0.6010]| \(|\allowbreak\token[0.8095]|x|\allowbreak\token[0.0048]|\|\allowbreak\token[0.1099]|).|\allowbreak\token[0.9476]| Therefore|\allowbreak\token[0.0000]|,|\allowbreak\token[0.6384]| the|\allowbreak\token[0.0007]| total|\allowbreak\token[0.0001]| number|\allowbreak\token[0.0000]| of|\allowbreak\token[0.0001]| coins|\allowbreak\token[0.5853]| before|\allowbreak\token[0.0055]| finding|\allowbreak\token[0.0003]| the|\allowbreak\token[0.8417]| additional|\allowbreak\token[0.3947]| bag|\allowbreak\token[0.8785]| is|\allowbreak\token[0.0000]| \(|\allowbreak\token[0.0001]|7|\allowbreak\token[0.0000]|x|\allowbreak\token[0.2478]|\|\allowbreak\token[0.0619]|).|\allowbreak\token[0.1434]| After|\allowbreak\token[0.0428]| finding|\allowbreak\token[0.0027]| the|\allowbreak\token[0.1502]| bag|\allowbreak\token[0.6866]| with|\allowbreak\token[0.0002]| |\allowbreak\token[0.0000]|5|\allowbreak\token[0.0001]|3|\allowbreak\token[0.0064]| coins|\allowbreak\token[0.1746]|,|\allowbreak\token[0.7104]| you|\allowbreak\token[0.2338]| redistribute|\allowbreak\token[0.8207]| all|\allowbreak\token[0.1743]| the|\allowbreak\token[0.0002]| coins|\allowbreak\token[0.7935]| into|\allowbreak\token[0.0057]| eight|\allowbreak\token[0.1040]| bags|\allowbreak\cut{[CUT]}\allowbreak\token[0.7179]| with|\allowbreak\token[0.1465]| an|\allowbreak\token[0.0013]| equal|\allowbreak\token[0.0010]| number|\allowbreak\token[0.0131]| of|\allowbreak\token[0.0000]| coins|\allowbreak\token[1.0839]|,|\allowbreak\token[0.9759]| let|\allowbreak\token[0.0000]|'s|\allowbreak\token[0.0518]| call|\allowbreak\token[0.0001]| this|\allowbreak\token[0.0940]| number|\allowbreak\token[0.0034]| \(|\allowbreak\token[0.0048]|y|\allowbreak\token[0.0000]|\|\allowbreak\token[0.0825]|).|\allowbreak\token[1.2411]| So|\allowbreak\token[0.0000]|,|\allowbreak\token[0.0321]| the|\allowbreak\token[0.0783]| total|\allowbreak\token[0.0000]| number|\allowbreak\token[0.0000]| of|\allowbreak\token[0.0000]| coins|\allowbreak\token[0.6227]| after|\allowbreak\token[0.3939]| redistribution|\allowbreak\token[0.0004]| is|\allowbreak\token[0.0082]| \(|\allowbreak\token[0.6478]|8|\allowbreak\token[0.0000]|y|\allowbreak\token[0.0177]|\|\allowbreak\token[0.9802]|).↩↩|\tighttracebreak\allowbreak\token[0.2592]|Given|\allowbreak\token[0.0010]| that|\allowbreak\token[0.7182]| you|\allowbreak\token[0.3932]| have|\allowbreak\token[0.0104]| more|\allowbreak\token[0.0000]| than|\allowbreak\token[0.0000]| |\allowbreak\token[0.0000]|2|\allowbreak\token[0.0000]|0|\allowbreak\token[0.0001]|0|\allowbreak\token[0.0000]| coins|\allowbreak\token[0.9826]| after|\allowbreak\token[0.2219]| redistribution|\allowbreak\token[0.0331]|,|\allowbreak\token[0.0307]| we|\allowbreak\token[0.5636]| have|\allowbreak\token[0.3679]| \(|\allowbreak\token[0.0010]|8|\allowbreak\token[0.0000]|y|\allowbreak\token[0.0000]| >|\allowbreak\token[0.0000]| |\allowbreak\token[0.0000]|2|\allowbreak\token[0.0000]|0|\allowbreak\token[0.0001]|0|\allowbreak\token[0.0001]|\|\allowbreak\token[0.3151]|),|\allowbreak\token[0.0003]| which|\allowbreak\token[0.1058]| simpl|\allowbreak\token[0.0000]|ifies|\allowbreak\token[0.0000]| to|\allowbreak\token[0.0000]| \(|\allowbreak\token[0.0000]|y|\allowbreak\token[0.0000]| >|\allowbreak\token[0.0000]| |\allowbreak\token[0.0000]|2|\allowbreak\token[0.0000]|5|\allowbreak\token[0.0001]|\|\allowbreak\token[0.1915]|).|\allowbreak\token[0.8687]| Since|\allowbreak\token[0.7162]| \(|\allowbreak\token[0.1102]|y|\allowbreak\token[0.0193]|\)|\allowbreak\token[0.8193]| is|\allowbreak\token[0.2219]| the|\allowbreak\token[0.0858]| number|\allowbreak\token[0.0000]| of|\allowbreak\token[0.0000]| coins|\allowbreak\token[0.0039]| in|\allowbreak\token[0.0002]| each|\allowbreak\token[0.2691]| of|\allowbreak\token[0.0000]| the|\allowbreak\token[0.0053]| eight|\allowbreak\token[0.0006]| bags|\allowbreak\token[0.3178]| after|\allowbreak\token[0.0001]| redistribution|\allowbreak\token[0.0007]|,|\allowbreak\token[0.4196]| and|\allowbreak\token[1.4536]| you|\allowbreak\token[2.1574]| redistributed|\allowbreak\token[0.3095]| all|\allowbreak\token[0.1122]| the|\allowbreak\token[0.0161]| coins|\allowbreak\token[0.6253]|,|\allowbreak\token[0.5869]| we|\allowbreak\token[0.7468]| have|\allowbreak\token[0.1412]| \(|\allowbreak\token[0.6637]|7|\allowbreak\token[0.0000]|x|\allowbreak\token[0.0000]| +|\allowbreak\token[0.0000]| |\allowbreak\token[0.0000]|5|\allowbreak\token[0.0000]|3|\allowbreak\token[0.0000]| =|\allowbreak\token[0.0000]| |\allowbreak\token[0.0000]|8|\allowbreak\token[0.0000]|y|\allowbreak\token[0.0000]|\|\allowbreak\token[0.0053]|).↩↩|\tighttracebreak\allowbreak\token[0.4369]|To|\allowbreak\token[0.0006]| find|\allowbreak\token[0.0000]| the|\allowbreak\token[0.0025]| smallest|\allowbreak\token[0.0261]| number|\allowbreak\token[0.0000]| of|\allowbreak\token[0.0002]| coins|\allowbreak\token[0.2616]| you|\allowbreak\token[0.0002]| could|\allowbreak\token[0.0000]| have|\allowbreak\token[0.0026]| had|\allowbreak\token[0.2728]| before|\allowbreak\token[0.0045]| finding|\allowbreak\token[0.0000]| the|\allowbreak\token[0.2073]| bag|\allowbreak\token[0.0020]| of|\allowbreak\token[0.0000]| |\allowbreak\token[0.0000]|5|\allowbreak\token[0.0000]|3|\allowbreak\token[0.0066]| coins|\allowbreak\token[0.0496]|,|\allowbreak\token[0.0000]| we|\allowbreak\token[0.0013]| need|\allowbreak\token[0.0002]| to|\allowbreak\token[0.0013]| find|\allowbreak\token[0.0000]| the|\allowbreak\token[0.0004]| smallest|\allowbreak\token[0.7334]| \(|\allowbreak\token[0.0419]|x|\allowbreak\token[0.0001]|\)|\allowbreak\token[0.1331]| that|\allowbreak\token[0.0116]| satisfies|\allowbreak\token[0.8051]| these|\allowbreak\token[0.0000]| conditions|\allowbreak\token[0.0103]|.|\allowbreak[...]
\tcblower
\token[0.7756]| with|\allowbreak\token[0.1464]| an|\allowbreak\token[0.0016]| equal|\allowbreak\token[0.0013]| number|\allowbreak\token[0.0144]| of|\allowbreak\token[0.0000]| coins|\allowbreak\token[1.1025]|,|\allowbreak\token[0.9695]| which|\allowbreak\token[0.5044]| we|\allowbreak\token[0.0822]|'ll|\allowbreak\token[0.0494]| call|\allowbreak\token[0.0000]| \(|\allowbreak\token[0.0062]|y|\allowbreak\token[0.0002]|\|\allowbreak\token[0.0502]|).|\allowbreak\token[0.8690]| This|\allowbreak\token[0.1096]| means|\allowbreak\token[0.1777]| the|\allowbreak\token[0.0753]| total|\allowbreak\token[0.0001]| number|\allowbreak\token[0.0000]| of|\allowbreak\token[0.0000]| coins|\allowbreak\token[0.7641]| is|\allowbreak\token[0.6064]| now|\allowbreak\token[0.0227]| \(|\allowbreak\token[0.5764]|7|\allowbreak\token[0.0003]|x|\allowbreak\token[0.0000]| +|\allowbreak\token[0.0000]| |\allowbreak\token[0.0000]|5|\allowbreak\token[0.0000]|3|\allowbreak\token[0.6698]|\|\allowbreak\token[0.0232]|),|\allowbreak\token[0.0038]| and|\allowbreak\token[0.0894]| this|\allowbreak\token[0.1714]| total|\allowbreak\token[1.0999]| is|\allowbreak\token[1.1997]| divisible|\allowbreak\token[0.0057]| by|\allowbreak\token[0.0004]| |\allowbreak\token[0.0001]|8|\allowbreak\token[1.3067]|,|\allowbreak\token[1.4839]| so|\allowbreak\token[0.7600]| \(|\allowbreak\token[0.0621]|7|\allowbreak\token[0.0000]|x|\allowbreak\token[0.0000]| +|\allowbreak\token[0.0000]| |\allowbreak\token[0.0000]|5|\allowbreak\token[0.0001]|3|\allowbreak\token[0.0127]| =|\allowbreak\token[0.0000]| |\allowbreak\token[0.0000]|8|\allowbreak\token[0.0003]|y|\allowbreak\token[0.0008]|\|\allowbreak\token[0.6026]|).↩↩|\tighttracebreak\allowbreak\token[0.4085]|Given|\allowbreak\token[0.0002]| that|\allowbreak\token[0.5546]| you|\allowbreak\token[0.0619]| have|\allowbreak\token[0.0010]| more|\allowbreak\token[0.0000]| than|\allowbreak\token[0.0001]| |\allowbreak\token[0.0000]|2|\allowbreak\token[0.0000]|0|\allowbreak\token[0.0000]|0|\allowbreak\token[0.0001]| coins|\allowbreak\token[1.0395]| in|\allowbreak\token[0.0074]| total|\allowbreak\token[0.5339]|,|\allowbreak\token[0.0983]| we|\allowbreak\token[1.1646]| have|\allowbreak\token[0.0967]| \(|\allowbreak\token[0.2702]|7|\allowbreak\token[0.0000]|x|\allowbreak\token[0.0003]| +|\allowbreak\token[0.0000]| |\allowbreak\token[0.0000]|5|\allowbreak\token[0.0001]|3|\allowbreak\token[0.0000]| >|\allowbreak\token[0.0000]| |\allowbreak\token[0.0000]|2|\allowbreak\token[0.0001]|0|\allowbreak\token[0.0002]|0|\allowbreak\token[0.0000]|\|\allowbreak\token[0.6786]|),|\allowbreak\token[0.0054]| which|\allowbreak\token[0.0017]| simpl|\allowbreak\token[0.0000]|ifies|\allowbreak\token[0.0000]| to|\allowbreak\token[0.0000]| \(|\allowbreak\token[0.0294]|7|\allowbreak\token[0.0000]|x|\allowbreak\token[0.0000]| >|\allowbreak\token[0.0000]| |\allowbreak\token[0.0000]|1|\allowbreak\token[0.0000]|4|\allowbreak\token[0.0001]|7|\allowbreak\token[0.0494]|\|\allowbreak\token[0.0446]|),|\allowbreak\token[0.3258]| and|\allowbreak\token[0.3191]| thus|\allowbreak\token[0.0002]| \(|\allowbreak\token[0.0000]|x|\allowbreak\token[0.0000]| >|\allowbreak\token[0.0001]| |\allowbreak\token[0.0001]|2|\allowbreak\token[0.0000]|1|\allowbreak\token[0.0004]|\|\allowbreak\token[0.1096]|).|\allowbreak\token[0.9808]| Since|\allowbreak\token[0.0127]| \(|\allowbreak\token[0.0035]|x|\allowbreak\token[0.0002]|\)|\allowbreak\token[0.0945]| must|\allowbreak\token[0.0000]| be|\allowbreak\token[0.0294]| an|\allowbreak\token[0.0000]| integer|\allowbreak\token[0.3227]|,|\allowbreak\token[0.0030]| the|\allowbreak\token[0.0073]| smallest|\allowbreak\token[0.9362]| possible|\allowbreak\token[0.0018]| value|\allowbreak\token[0.0000]| for|\allowbreak\token[0.0000]| \(|\allowbreak\token[0.0004]|x|\allowbreak\token[0.0000]|\)|\allowbreak\token[0.4213]| is|\allowbreak\token[0.0042]| |\allowbreak\token[0.0001]|2|\allowbreak\token[0.0016]|2|\allowbreak\token[0.7191]|.↩↩|\tighttracebreak\allowbreak\token[0.9831]|Sub|\allowbreak\token[0.0000]|stit|\allowbreak\token[0.0005]|uting|\allowbreak\token[0.0000]| \(|\allowbreak\token[0.0000]|x|\allowbreak\token[0.0000]| =|\allowbreak\token[0.0000]| |\allowbreak\token[0.0000]|2|\allowbreak\token[0.0000]|2|\allowbreak\token[0.0001]|\)|\allowbreak\token[0.0010]| into|\allowbreak\token[0.7102]| the|\allowbreak[...]
\end{twopartbox}
\begin{twopartbox}{Cut Index 419 | Acceptance Probability 1.0000}{Block 3 Step 2}{Proposal (Accepted)}
[...]\allowbreak\token[0.0000]|2|\allowbreak\token[0.0000]|2|\allowbreak\token[0.0000]|\)|\allowbreak\token[0.4211]| is|\allowbreak\token[0.0013]| not|\allowbreak\token[0.4197]| the|\allowbreak\token[0.2724]| correct|\allowbreak\token[0.8378]| value|\allowbreak\token[0.7992]|.↩↩|\tighttracebreak\allowbreak\token[1.3209]|We|\allowbreak\token[0.8921]| need|\allowbreak\token[0.0063]| to|\allowbreak\token[0.0001]| find|\allowbreak\token[0.0004]| the|\allowbreak\token[0.0016]| smallest|\allowbreak\token[0.0233]| \(|\allowbreak\token[0.0001]|x|\allowbreak\token[0.0993]|\)|\allowbreak\token[0.4080]| such|\allowbreak\token[0.0000]| that|\allowbreak\token[0.0002]| \(|\allowbreak\token[0.0007]|7|\allowbreak\token[0.0000]|x|\allowbreak\token[0.0000]| +|\allowbreak\token[0.0000]| |\allowbreak\token[0.0000]|5|\allowbreak\token[0.0001]|3|\allowbreak\token[0.0001]|\)|\allowbreak\token[0.0000]| is|\allowbreak\token[0.0414]| divisible|\allowbreak\token[0.0000]| by|\allowbreak\token[0.0000]| |\allowbreak\token[0.0000]|8|\allowbreak\token[0.1746]| and|\allowbreak\token[0.3258]| greater|\allowbreak\token[0.0000]| than|\allowbreak\token[0.0005]| |\allowbreak\token[0.0024]|2|\allowbreak\token[0.0002]|0|\allowbreak\token[0.0164]|0|\allowbreak\token[0.0001]|.|\allowbreak\token[0.6441]| Let|\allowbreak\token[0.0000]|'s|\allowbreak\token[0.3903]| increment|\allowbreak\token[0.0402]| \(|\allowbreak\token[0.0000]|x|\allowbreak\token[0.0000]|\)|\allowbreak\token[0.5173]| by|\allowbreak\token[0.0051]| |\allowbreak\token[0.0239]|1|\allowbreak\token[1.0275]| and|\allowbreak\token[0.0935]| check|\allowbreak\token[1.4560]| again|\allowbreak\token[0.4429]|.|\allowbreak\token[0.0968]| For|\allowbreak\token[0.0000]| \(|\allowbreak\token[0.0000]|x|\allowbreak\token[0.0000]| =|\allowbreak\token[0.0000]| |\allowbreak\token[0.0002]|2|\allowbreak\token[0.0020]|3|\allowbreak\token[0.0000]|\|\allowbreak\token[0.0000]|),|\allowbreak\token[0.0487]| we|\allowbreak\token[0.0103]| have|\allowbreak\token[0.0001]| \(|\allowbreak\token[0.0001]|7|\allowbreak\token[0.0007]|(|\allowbreak\token[0.0000]|2|\allowbreak\token[0.0001]|3|\allowbreak\token[0.0000]|)|\allowbreak\token[0.0001]| +|\allowbreak\token[0.0000]| |\allowbreak\token[0.0000]|5|\allowbreak\token[0.0000]|3|\allowbreak\token[0.0000]| =|\allowbreak\token[0.0000]| |\allowbreak\token[0.1324]|1|\allowbreak\token[0.0238]|6|\allowbreak\token[0.0022]|1|\allowbreak\token[0.0001]| +|\allowbreak\token[0.0000]| |\allowbreak\token[0.0000]|5|\allowbreak\token[0.0000]|3|\allowbreak\token[0.0000]| =|\allowbreak\token[0.0000]| |\allowbreak\token[0.0001]|2|\allowbreak\token[0.0001]|1|\allowbreak\token[0.0001]|4|\allowbreak\token[0.0000]|\|\allowbreak\token[0.1454]|),|\allowbreak\token[0.6115]| which|\allowbreak\token[0.0005]| is|\allowbreak\token[1.3057]| still|\allowbreak\token[0.0054]| not|\allowbreak\token[0.0001]| divisible|\allowbreak\token[0.0000]| by|\allowbreak\token[0.0000]| |\allowbreak\token[0.0000]|8|\allowbreak\token[0.6630]|.↩↩|\tighttracebreak\allowbreak\token[0.3222]|Contin|\allowbreak\token[0.0000]|uing|\allowbreak\token[0.0000]| this|\allowbreak\token[0.0000]| process|\allowbreak\token[0.0003]|,|\allowbreak\token[0.5751]| for|\allowbreak\token[0.0000]| \(|\allowbreak\token[0.0000]|x|\allowbreak\token[0.0000]| =|\allowbreak\token[0.0000]| |\allowbreak\token[0.0004]|2|\allowbreak\token[0.2636]|4|\allowbreak\token[0.0000]|\|\allowbreak\token[0.0000]|),|\allowbreak\token[0.0001]| we|\allowbreak\token[0.4357]| have|\allowbreak\token[0.0000]| \(|\allowbreak\token[0.0000]|7|\allowbreak\token[0.0000]|(|\allowbreak\token[0.0000]|2|\allowbreak\token[0.0001]|4|\allowbreak\token[0.0000]|)|\allowbreak\token[0.0000]| +|\allowbreak\token[0.0000]| |\allowbreak\token[0.0000]|5|\allowbreak\token[0.0000]|3|\allowbreak\token[0.0000]| =|\allowbreak\token[0.0000]| |\allowbreak\token[0.0000]|1|\allowbreak\token[0.0005]|6|\allowbreak\token[0.0000]|8|\allowbreak\token[0.0000]| +|\allowbreak\token[0.0000]| |\allowbreak\token[0.0000]|5|\allowbreak\token[0.0000]|3|\allowbreak\token[0.0000]| =|\allowbreak\token[0.0000]| |\allowbreak\token[0.0000]|2|\allowbreak\token[0.0003]|2|\allowbreak\token[0.0000]|1|\allowbreak\token[0.0000]|\|\allowbreak\token[0.0547]|),|\allowbreak\token[0.2690]| which|\allowbreak\token[0.0002]| is|\allowbreak\cut{[CUT]}\allowbreak\token[1.2710]| not|\allowbreak\token[0.0012]| divisible|\allowbreak\token[0.0000]| by|\allowbreak\token[0.0000]| |\allowbreak\token[0.0000]|8|\allowbreak\token[0.4985]|.|\allowbreak\token[0.4480]| For|\allowbreak\token[0.0000]| \(|\allowbreak\token[0.0000]|x|\allowbreak\token[0.0000]| =|\allowbreak\token[0.0000]| |\allowbreak\token[0.0000]|2|\allowbreak\token[0.0014]|5|\allowbreak\token[0.0000]|\|\allowbreak\token[0.0000]|),|\allowbreak\token[0.0004]| we|\allowbreak\token[0.0309]| have|\allowbreak\token[0.0000]| \(|\allowbreak\token[0.0059]|7|\allowbreak\token[0.0000]|(|\allowbreak\token[0.0000]|2|\allowbreak\token[0.0001]|5|\allowbreak\token[0.0000]|)|\allowbreak\token[0.0000]| +|\allowbreak\token[0.0000]| |\allowbreak\token[0.0000]|5|\allowbreak\token[0.0000]|3|\allowbreak\token[0.0000]| =|\allowbreak\token[0.0000]| |\allowbreak\token[0.0000]|1|\allowbreak\token[0.0000]|7|\allowbreak\token[0.0000]|5|\allowbreak\token[0.0000]| +|\allowbreak\token[0.0000]| |\allowbreak\token[0.0000]|5|\allowbreak\token[0.0000]|3|\allowbreak\token[0.0000]| =|\allowbreak\token[0.0000]| |\allowbreak\token[0.0000]|2|\allowbreak\token[0.0001]|2|\allowbreak\token[0.0001]|8|\allowbreak\token[0.0000]|\|\allowbreak\token[0.0025]|),|\allowbreak\token[0.6062]| which|\allowbreak\token[0.0002]| is|\allowbreak\token[0.9286]| divisible|\allowbreak\token[0.0000]| by|\allowbreak\token[0.0000]| |\allowbreak\token[0.0000]|8|\allowbreak\token[1.3941]| (|\allowbreak\token[0.5933]|since|\allowbreak\token[0.0018]| \(|\allowbreak\token[0.1328]|2|\allowbreak\token[0.0008]|2|\allowbreak\token[0.0001]|8|\allowbreak\token[0.6989]| =|\allowbreak\token[0.0000]| |\allowbreak\token[0.0050]|8|\allowbreak\token[0.0001]| \|\allowbreak\token[0.0000]|times|\allowbreak\token[0.0000]| |\allowbreak\token[0.0001]|2|\allowbreak\token[0.0003]|8|\allowbreak\token[0.0532]|.|\allowbreak\token[0.0003]|5|\allowbreak\token[0.0011]|\|\allowbreak\token[0.1036]|),|\allowbreak\token[0.0117]| but|\allowbreak\token[0.6165]| we|\allowbreak\token[0.4815]| need|\allowbreak\token[0.7825]| an|\allowbreak\token[0.0182]| integer|\allowbreak\token[1.2005]| solution|\allowbreak\token[0.9218]|,|\allowbreak\token[0.0107]| so|\allowbreak\token[0.4842]| we|\allowbreak\token[1.3012]| look|\allowbreak\token[0.0360]| for|\allowbreak\token[0.0065]| the|\allowbreak\token[0.1260]| next|\allowbreak\token[1.0651]| integer|\allowbreak\token[1.3665]| value|\allowbreak\token[0.5555]| of|\allowbreak\token[0.0007]| \(|\allowbreak\token[0.0607]|x|\allowbreak\token[0.3396]|\|\allowbreak\token[0.0300]|))|\allowbreak\token[0.0000]|.↩↩|\tighttracebreak\allowbreak\token[0.0638]|For|\allowbreak\token[0.0000]| \(|\allowbreak\token[0.0000]|x|\allowbreak\token[0.0000]| =|\allowbreak\token[0.0000]| |\allowbreak\token[0.0009]|2|\allowbreak\token[0.1048]|6|\allowbreak\token[0.0000]|\|\allowbreak\token[0.0000]|),|\allowbreak\token[0.0143]| we|\allowbreak\token[0.0112]| have|\allowbreak\token[0.0000]| \(|\allowbreak\token[0.0000]|7|\allowbreak\token[0.0000]|(|\allowbreak\token[0.0000]|2|\allowbreak\token[0.0001]|6|\allowbreak\token[0.0000]|)|\allowbreak\token[0.0000]| +|\allowbreak\token[0.0000]| |\allowbreak\token[0.0000]|5|\allowbreak\token[0.0001]|3|\allowbreak\token[0.0000]| =|\allowbreak\token[0.0000]| |\allowbreak\token[0.0000]|1|\allowbreak\token[0.0002]|8|\allowbreak\token[0.0000]|2|\allowbreak\token[0.0000]| +|\allowbreak\token[0.0000]| |\allowbreak\token[0.0000]|5|\allowbreak\token[0.0000]|3|\allowbreak\token[0.0000]| =|\allowbreak\token[0.0000]| |\allowbreak\token[0.0000]|2|
\tcblower
\token[1.2670]| also|\allowbreak\token[0.0000]| not|\allowbreak\token[0.0005]| divisible|\allowbreak\token[0.0000]| by|\allowbreak\token[0.0000]| |\allowbreak\token[0.0000]|8|\allowbreak\token[0.0993]|.|\allowbreak\token[0.7500]| For|\allowbreak\token[0.0000]| \(|\allowbreak\token[0.0000]|x|\allowbreak\token[0.0000]| =|\allowbreak\token[0.0000]| |\allowbreak\token[0.0000]|2|\allowbreak\token[0.0014]|5|\allowbreak\token[0.0000]|\|\allowbreak\token[0.0000]|),|\allowbreak\token[0.0005]| we|\allowbreak\token[0.0789]| have|\allowbreak\token[0.0000]| \(|\allowbreak\token[0.0042]|7|\allowbreak\token[0.0000]|(|\allowbreak\token[0.0000]|2|\allowbreak\token[0.0001]|5|\allowbreak\token[0.0000]|)|\allowbreak\token[0.0000]| +|\allowbreak\token[0.0000]| |\allowbreak\token[0.0000]|5|\allowbreak\token[0.0000]|3|\allowbreak\token[0.0000]| =|\allowbreak\token[0.0000]| |\allowbreak\token[0.0000]|1|\allowbreak\token[0.0000]|7|\allowbreak\token[0.0000]|5|\allowbreak\token[0.0000]| +|\allowbreak\token[0.0000]| |\allowbreak\token[0.0000]|5|\allowbreak\token[0.0000]|3|\allowbreak\token[0.0000]| =|\allowbreak\token[0.0000]| |\allowbreak\token[0.0000]|2|\allowbreak\token[0.0001]|2|\allowbreak\token[0.0001]|8|\allowbreak\token[0.0000]|\|\allowbreak\token[0.0041]|),|\allowbreak\token[0.6763]| and|\allowbreak\token[0.7388]| \(|\allowbreak\token[0.0126]|2|\allowbreak\token[0.0000]|2|\allowbreak\token[0.0000]|8|\allowbreak\token[0.0255]|\)|\allowbreak\token[0.0605]| is|\allowbreak\token[0.1459]| divisible|\allowbreak\token[0.0000]| by|\allowbreak\token[0.2911]| |\allowbreak\token[0.0000]|8|\allowbreak\token[1.2947]|,|\allowbreak\token[1.2634]| giving|\allowbreak\token[0.1594]| us|\allowbreak\token[0.0007]| \(|\allowbreak\token[0.4197]|y|\allowbreak\token[0.0000]| =|\allowbreak\token[0.0008]| |\allowbreak\token[0.0001]|2|\allowbreak\token[0.0281]|8|\allowbreak\token[0.1465]|.|\allowbreak\token[0.0003]|5|\allowbreak\token[0.0011]|\|\allowbreak\token[0.0504]|),|\allowbreak\token[0.6854]| but|\allowbreak\token[0.5913]| since|\allowbreak\token[0.0254]| \(|\allowbreak\token[0.0001]|y|\allowbreak\token[0.0000]|\)|\allowbreak\token[0.0003]| must|\allowbreak\token[0.3920]| be|\allowbreak\token[0.0140]| an|\allowbreak\token[0.0001]| integer|\allowbreak\token[0.0054]|,|\allowbreak\token[0.7654]| this|\allowbreak\token[0.0575]| is|\allowbreak\token[0.2197]| not|\allowbreak\token[0.6561]| the|\allowbreak\token[0.1608]| correct|\allowbreak\token[1.2719]| solution|\allowbreak\token[0.1618]|.↩↩|\tighttracebreak\allowbreak\token[0.4113]|For|\allowbreak\token[0.0000]| \(|\allowbreak\token[0.0000]|x|\allowbreak\token[0.0000]| =|\allowbreak\token[0.0000]| |\allowbreak\token[0.0000]|2|\allowbreak\token[0.0142]|6|\allowbreak\token[0.0000]|\|\allowbreak\token[0.0000]|),|\allowbreak\token[0.0004]| we|\allowbreak\token[0.0034]| have|\allowbreak\token[0.0000]| \(|\allowbreak\token[0.0000]|7|\allowbreak\token[0.0000]|(|\allowbreak\token[0.0000]|2|\allowbreak\token[0.0003]|6|\allowbreak\token[0.0000]|)|\allowbreak\token[0.0000]| +|\allowbreak\token[0.0000]| |\allowbreak\token[0.0000]|5|\allowbreak\token[0.0000]|3|\allowbreak\token[0.0000]| =|\allowbreak\token[0.0000]| |\allowbreak\token[0.0000]|1|\allowbreak\token[0.0002]|8|\allowbreak\token[0.0000]|2|\allowbreak\token[0.0000]| +|\allowbreak\token[0.0000]| |\allowbreak\token[0.0000]|5|\allowbreak\token[0.0000]|3|\allowbreak\token[0.0000]| =|\allowbreak\token[0.0000]| |\allowbreak\token[0.0000]|2|\allowbreak\token[0.0001]|3|\allowbreak\token[0.0001]|5|\allowbreak\token[0.0000]|\|
\end{twopartbox}
\begin{twopartbox}{Cut Index 129 | Acceptance Probability 0.0000}{Block 3 Step 3}{Proposal (Rejected)}
\token[2.2824]| To|\allowbreak\token[0.4171]| solve|\allowbreak\token[0.4477]| this|\allowbreak\token[0.0840]| problem|\allowbreak\token[0.0111]|,|\allowbreak\token[0.5775]| let|\allowbreak\token[0.0029]|'s|\allowbreak\token[0.8168]| denote|\allowbreak\token[0.0233]| the|\allowbreak\token[0.7993]| number|\allowbreak\token[0.0000]| of|\allowbreak\token[0.1007]| coins|\allowbreak\token[0.0804]| in|\allowbreak\token[0.0205]| each|\allowbreak\token[0.0183]| of|\allowbreak\token[0.0031]| the|\allowbreak\token[0.5575]| original|\allowbreak\token[0.0329]| seven|\allowbreak\token[0.0001]| bags|\allowbreak\token[0.0009]| as|\allowbreak\token[0.6010]| \(|\allowbreak\token[0.8095]|x|\allowbreak\token[0.0048]|\|\allowbreak\token[0.1099]|).|\allowbreak\token[0.9476]| Therefore|\allowbreak\token[0.0000]|,|\allowbreak\token[0.6384]| the|\allowbreak\token[0.0007]| total|\allowbreak\token[0.0001]| number|\allowbreak\token[0.0000]| of|\allowbreak\token[0.0001]| coins|\allowbreak\token[0.5853]| before|\allowbreak\token[0.0055]| finding|\allowbreak\token[0.0003]| the|\allowbreak\token[0.8417]| additional|\allowbreak\token[0.3947]| bag|\allowbreak\token[0.8785]| is|\allowbreak\token[0.0000]| \(|\allowbreak\token[0.0001]|7|\allowbreak\token[0.0000]|x|\allowbreak\token[0.2478]|\|\allowbreak\token[0.0619]|).|\allowbreak\token[0.1434]| After|\allowbreak\token[0.0428]| finding|\allowbreak\token[0.0027]| the|\allowbreak\token[0.1502]| bag|\allowbreak\token[0.6866]| with|\allowbreak\token[0.0002]| |\allowbreak\token[0.0000]|5|\allowbreak\token[0.0001]|3|\allowbreak\token[0.0064]| coins|\allowbreak\token[0.1746]|,|\allowbreak\token[0.7104]| you|\allowbreak\token[0.2338]| redistribute|\allowbreak\token[0.8207]| all|\allowbreak\token[0.1743]| the|\allowbreak\token[0.0002]| coins|\allowbreak\token[0.7935]| into|\allowbreak\token[0.0057]| eight|\allowbreak\token[0.1040]| bags|\allowbreak\token[0.7756]| with|\allowbreak\token[0.1464]| an|\allowbreak\token[0.0016]| equal|\allowbreak\token[0.0013]| number|\allowbreak\token[0.0144]| of|\allowbreak\token[0.0000]| coins|\allowbreak\token[1.1025]|,|\allowbreak\token[0.9695]| which|\allowbreak\token[0.5044]| we|\allowbreak\token[0.0822]|'ll|\allowbreak\token[0.0494]| call|\allowbreak\token[0.0000]| \(|\allowbreak\token[0.0062]|y|\allowbreak\token[0.0002]|\|\allowbreak\token[0.0502]|).|\allowbreak\token[0.8690]| This|\allowbreak\token[0.1096]| means|\allowbreak\token[0.1777]| the|\allowbreak\token[0.0753]| total|\allowbreak\token[0.0001]| number|\allowbreak\token[0.0000]| of|\allowbreak\token[0.0000]| coins|\allowbreak\token[0.7641]| is|\allowbreak\token[0.6064]| now|\allowbreak\token[0.0227]| \(|\allowbreak\token[0.5764]|7|\allowbreak\token[0.0003]|x|\allowbreak\token[0.0000]| +|\allowbreak\token[0.0000]| |\allowbreak\token[0.0000]|5|\allowbreak\token[0.0000]|3|\allowbreak\token[0.6698]|\|\allowbreak\token[0.0232]|),|\allowbreak\token[0.0038]| and|\allowbreak\token[0.0894]| this|\allowbreak\token[0.1714]| total|\allowbreak\token[1.0999]| is|\allowbreak\token[1.1997]| divisible|\allowbreak\token[0.0057]| by|\allowbreak\token[0.0004]| |\allowbreak\token[0.0001]|8|\allowbreak\token[1.3067]|,|\allowbreak\token[1.4839]| so|\allowbreak\token[0.7600]| \(|\allowbreak\token[0.0621]|7|\allowbreak\token[0.0000]|x|\allowbreak\token[0.0000]| +|\allowbreak\token[0.0000]| |\allowbreak\token[0.0000]|5|\allowbreak\token[0.0001]|3|\allowbreak\token[0.0127]| =|\allowbreak\token[0.0000]| |\allowbreak\token[0.0000]|8|\allowbreak\token[0.0003]|y|\allowbreak\token[0.0008]|\|\allowbreak\token[0.6026]|).↩↩|\tighttracebreak\allowbreak\token[0.4085]|Given|\allowbreak\token[0.0002]| that|\allowbreak\token[0.5546]| you|\allowbreak\token[0.0619]| have|\allowbreak\token[0.0010]| more|\allowbreak\token[0.0000]| than|\allowbreak\token[0.0001]| |\allowbreak\token[0.0000]|2|\allowbreak\token[0.0000]|0|\allowbreak\token[0.0000]|0|\allowbreak\token[0.0001]| coins|\allowbreak\token[1.0395]| in|\allowbreak\token[0.0074]| total|\allowbreak\cut{[CUT]}\allowbreak\token[0.5339]|,|\allowbreak\token[0.0983]| we|\allowbreak\token[1.1646]| have|\allowbreak\token[0.0967]| \(|\allowbreak\token[0.2702]|7|\allowbreak\token[0.0000]|x|\allowbreak\token[0.0003]| +|\allowbreak\token[0.0000]| |\allowbreak\token[0.0000]|5|\allowbreak\token[0.0001]|3|\allowbreak\token[0.0000]| >|\allowbreak\token[0.0000]| |\allowbreak\token[0.0000]|2|\allowbreak\token[0.0001]|0|\allowbreak\token[0.0002]|0|\allowbreak\token[0.0000]|\|\allowbreak\token[0.6786]|),|\allowbreak\token[0.0054]| which|\allowbreak\token[0.0017]| simpl|\allowbreak\token[0.0000]|ifies|\allowbreak\token[0.0000]| to|\allowbreak\token[0.0000]| \(|\allowbreak\token[0.0294]|7|\allowbreak\token[0.0000]|x|\allowbreak\token[0.0000]| >|\allowbreak\token[0.0000]| |\allowbreak\token[0.0000]|1|\allowbreak\token[0.0000]|4|\allowbreak\token[0.0001]|7|\allowbreak\token[0.0494]|\|\allowbreak\token[0.0446]|),|\allowbreak\token[0.3258]| and|\allowbreak\token[0.3191]| thus|\allowbreak\token[0.0002]| \(|\allowbreak\token[0.0000]|x|\allowbreak\token[0.0000]| >|\allowbreak\token[0.0001]| |\allowbreak\token[0.0001]|2|\allowbreak\token[0.0000]|1|\allowbreak\token[0.0004]|\|\allowbreak\token[0.1096]|).|\allowbreak\token[0.9808]| Since|\allowbreak\token[0.0127]| \(|\allowbreak\token[0.0035]|x|\allowbreak\token[0.0002]|\)|\allowbreak\token[0.0945]| must|\allowbreak\token[0.0000]| be|\allowbreak\token[0.0294]| an|\allowbreak\token[0.0000]| integer|\allowbreak\token[0.3227]|,|\allowbreak\token[0.0030]| the|\allowbreak\token[0.0073]| smallest|\allowbreak\token[0.9362]| possible|\allowbreak\token[0.0018]| value|\allowbreak\token[0.0000]| for|\allowbreak\token[0.0000]| \(|\allowbreak\token[0.0004]|x|\allowbreak\token[0.0000]|\)|\allowbreak\token[0.4213]| is|\allowbreak\token[0.0042]| |\allowbreak\token[0.0001]|2|\allowbreak\token[0.0016]|2|\allowbreak\token[0.7191]|.↩↩|\tighttracebreak\allowbreak\token[0.9831]|Sub|\allowbreak\token[0.0000]|stit|\allowbreak\token[0.0005]|uting|\allowbreak\token[0.0000]| \(|\allowbreak\token[0.0000]|x|\allowbreak\token[0.0000]| =|\allowbreak\token[0.0000]| |\allowbreak\token[0.0000]|2|\allowbreak\token[0.0000]|2|\allowbreak\token[0.0001]|\)|\allowbreak\token[0.0010]| into|\allowbreak\token[0.7102]| the|\allowbreak\token[0.0135]| equation|\allowbreak\token[0.1743]| \(|\allowbreak\token[0.0002]|7|\allowbreak\token[0.0000]|x|\allowbreak\token[0.0001]| +|\allowbreak\token[0.0000]| |\allowbreak\token[0.0000]|5|\allowbreak\token[0.0001]|3|\allowbreak\token[0.1096]| =|\allowbreak\token[0.0000]| |\allowbreak\token[0.0002]|8|\allowbreak\token[0.0000]|y|\allowbreak\token[0.0017]|\|\allowbreak\token[0.0000]|),|\allowbreak\token[0.0000]| we|\allowbreak\token[0.0104]| get|\allowbreak\token[0.0670]| \(|\allowbreak\token[0.0739]|7|\allowbreak\token[0.0337]|(|\allowbreak\token[0.0000]|2|\allowbreak\token[0.0000]|2|\allowbreak\token[0.0000]|)|\allowbreak\token[0.0000]| +|\allowbreak\token[0.0000]| |\allowbreak\token[0.0000]|5|\allowbreak\token[0.0000]|3|\allowbreak\token[0.0000]| =|\allowbreak\token[0.0000]| |\allowbreak\token[0.5392]|1|\allowbreak\token[0.7184]|5|\allowbreak\token[0.0361]|4|\allowbreak\token[0.0000]| +|\allowbreak\token[0.0000]| |\allowbreak\token[0.0000]|5|\allowbreak\token[0.0000]|3|\allowbreak\token[0.0000]| =|\allowbreak\token[0.0000]| |\allowbreak\token[0.0020]|2|\allowbreak\token[0.0000]|0|\allowbreak\token[0.0000]|7|\allowbreak\token[0.1748]|\|\allowbreak\token[0.2686]|).|\allowbreak\token[1.2096]| This|\allowbreak\token[1.2500]| means|\allowbreak\token[0.7949]| \(|\allowbreak\token[0.6895]|8|\allowbreak\token[0.0000]|y|\allowbreak\token[0.0000]| =|\allowbreak\token[0.0000]| |\allowbreak\token[0.0000]|2|\allowbreak\token[0.0000]|0|\allowbreak\token[0.0001]|7|\allowbreak\token[0.0007]|\|\allowbreak\token[0.0032]|),|\allowbreak\token[0.6428]| but|\allowbreak\token[0.6889]| since|\allowbreak\token[0.2738]| |\allowbreak\token[0.0000]|2|\allowbreak\token[0.0000]|0|\allowbreak\token[0.0000]|7|\allowbreak\token[0.0005]| is|\allowbreak\token[0.0000]| not|\allowbreak\token[0.0011]| divisible|\allowbreak\token[0.0000]| by|\allowbreak\token[0.0000]| |\allowbreak\token[0.0000]|8|\allowbreak\token[0.0001]|,|\allowbreak\token[0.6828]| \(|\allowbreak\token[0.0025]|x|\allowbreak[...]
\tcblower
\token[0.5600]|,|\allowbreak\token[0.0983]| we|\allowbreak\token[1.1632]| know|\allowbreak\token[0.4472]| that|\allowbreak\token[0.0016]| \(|\allowbreak\token[0.6995]|8|\allowbreak\token[0.0000]|y|\allowbreak\token[0.0001]| >|\allowbreak\token[0.0000]| |\allowbreak\token[0.0004]|2|\allowbreak\token[0.0002]|0|\allowbreak\token[0.0066]|0|\allowbreak\token[0.0003]|\|\allowbreak\token[0.1205]|),|\allowbreak\token[0.0647]| which|\allowbreak\token[0.6335]| simpl|\allowbreak\token[0.0000]|ifies|\allowbreak\token[0.0000]| to|\allowbreak\token[0.0000]| \(|\allowbreak\token[0.0001]|y|\allowbreak\token[0.0000]| >|\allowbreak\token[0.0001]| |\allowbreak\token[0.0000]|2|\allowbreak\token[0.0000]|5|\allowbreak\token[0.0006]|\|\allowbreak\token[0.0189]|).|\allowbreak\token[1.4165]| Since|\allowbreak\token[0.1698]| \(|\allowbreak\token[0.0303]|y|\allowbreak\token[0.0043]|\)|\allowbreak\token[0.8165]| is|\allowbreak\token[0.3104]| the|\allowbreak\token[0.0263]| number|\allowbreak\token[0.0001]| of|\allowbreak\token[0.0000]| coins|\allowbreak\token[0.0119]| in|\allowbreak\token[0.0008]| each|\allowbreak\token[0.0904]| of|\allowbreak\token[0.0000]| the|\allowbreak\token[0.0098]| eight|\allowbreak\token[0.0007]| bags|\allowbreak\token[0.4825]| after|\allowbreak\token[0.0020]| redistribution|\allowbreak\token[0.0016]|,|\allowbreak\token[0.3280]| and|\allowbreak\token[0.9796]| it|\allowbreak\token[0.1378]| must|\allowbreak\token[0.0002]| be|\allowbreak\token[0.3789]| an|\allowbreak\token[0.0001]| integer|\allowbreak\token[0.0317]|,|\allowbreak\token[0.0619]| the|\allowbreak\token[0.0048]| smallest|\allowbreak\token[1.1099]| possible|\allowbreak\token[0.0027]| value|\allowbreak\token[0.0000]| for|\allowbreak\token[0.0000]| \(|\allowbreak\token[0.0009]|y|\allowbreak\token[0.0000]|\)|\allowbreak\token[0.3147]| that|\allowbreak\token[0.4806]| is|\allowbreak\token[0.0021]| greater|\allowbreak\token[0.0000]| than|\allowbreak\token[0.0000]| |\allowbreak\token[0.0000]|2|\allowbreak\token[0.0002]|5|\allowbreak\token[0.4752]| is|\allowbreak\token[0.0327]| |\allowbreak\token[0.0063]|2|\allowbreak\token[0.0276]|6|\allowbreak\token[0.7907]|.↩↩|\tighttracebreak\allowbreak\token[0.6053]|Sub|\allowbreak\token[0.0000]|stit|\allowbreak\token[0.0009]|uting|\allowbreak\token[0.0000]| \(|\allowbreak\token[0.0000]|y|\allowbreak\token[0.0000]| =|\allowbreak\token[0.0000]| |\allowbreak\token[0.0000]|2|\allowbreak\token[0.0000]|6|\allowbreak\token[0.0001]|\)|\allowbreak\token[0.0047]| into|\allowbreak\token[0.7121]| the|\allowbreak\token[0.0001]| equation|\allowbreak\token[0.0944]| \(|\allowbreak\token[0.0605]|7|\allowbreak\token[0.0000]|x|\allowbreak\token[0.0000]| +|\allowbreak\token[0.0000]| |\allowbreak\token[0.0000]|5|\allowbreak\token[0.0001]|3|\allowbreak\token[0.0000]| =|\allowbreak\token[0.0000]| |\allowbreak\token[0.0000]|8|\allowbreak\token[0.0034]|y|\allowbreak\token[0.0027]|\|\allowbreak\token[0.0000]|),|\allowbreak\token[0.0000]| we|\allowbreak\token[0.0014]| get|\allowbreak\token[0.5825]|:↩|\newline\allowbreak\token[0.0005]|\[|\allowbreak\token[0.0000]|7|\allowbreak\token[0.0000]|x|\allowbreak\token[0.0000]| +|\allowbreak\token[0.0000]| |\allowbreak\token[0.0000]|5|\allowbreak\token[0.0000]|3|\allowbreak\token[0.0000]| =|\allowbreak\token[0.0000]| |\allowbreak\token[0.0402]|8|\allowbreak\token[0.0403]| \|\allowbreak\token[0.0001]|times|\allowbreak\token[0.0000]| |\allowbreak\token[0.0000]|2|\allowbreak\token[0.0001]|6|\allowbreak\token[0.0027]|\|\allowbreak\token[0.0001]|]↩|\newline\allowbreak\token[0.0000]|\[|\allowbreak\token[0.0000]|7|\allowbreak\token[0.0000]|x|\allowbreak\token[0.0362]| +|\allowbreak\token[0.0000]| |\allowbreak\token[0.0000]|5|\allowbreak\token[0.0000]|3|\allowbreak\token[0.0000]| =|\allowbreak\token[0.0000]| |\allowbreak\token[0.0000]|2|\allowbreak\token[0.0000]|0|\allowbreak\token[0.0000]|8|\allowbreak\token[0.0000]|\|\allowbreak\token[0.0010]|]↩|\newline\allowbreak\token[0.0001]|\[|\allowbreak\token[0.0000]|7|\allowbreak\token[0.0000]|x|\allowbreak\token[0.0000]| =|\allowbreak\token[0.0000]| |\allowbreak\token[0.3653]|1|\allowbreak\token[0.0000]|5|\allowbreak\token[0.0000]|5|\allowbreak\token[0.0000]|\|\allowbreak\token[0.1094]|]↩|\newline\allowbreak\token[0.0210]|\[|\allowbreak\token[0.0000]|x|\allowbreak\token[0.0010]| =|\allowbreak[...]
\end{twopartbox}
\begin{twopartbox}{Cut Index 34 | Acceptance Probability 0.0000}{Block 4 Step 1}{Proposal (Rejected)}
\token[2.2824]| To|\allowbreak\token[0.4171]| solve|\allowbreak\token[0.4477]| this|\allowbreak\token[0.0840]| problem|\allowbreak\token[0.0111]|,|\allowbreak\token[0.5775]| let|\allowbreak\token[0.0029]|'s|\allowbreak\token[0.8168]| denote|\allowbreak\token[0.0233]| the|\allowbreak\token[0.7993]| number|\allowbreak\token[0.0000]| of|\allowbreak\token[0.1007]| coins|\allowbreak\token[0.0804]| in|\allowbreak\token[0.0205]| each|\allowbreak\token[0.0183]| of|\allowbreak\token[0.0031]| the|\allowbreak\token[0.5575]| original|\allowbreak\token[0.0329]| seven|\allowbreak\token[0.0001]| bags|\allowbreak\token[0.0009]| as|\allowbreak\token[0.6010]| \(|\allowbreak\token[0.8095]|x|\allowbreak\token[0.0048]|\|\allowbreak\token[0.1099]|).|\allowbreak\token[0.9476]| Therefore|\allowbreak\token[0.0000]|,|\allowbreak\token[0.6384]| the|\allowbreak\token[0.0007]| total|\allowbreak\token[0.0001]| number|\allowbreak\token[0.0000]| of|\allowbreak\token[0.0001]| coins|\allowbreak\token[0.5853]| before|\allowbreak\token[0.0055]| finding|\allowbreak\token[0.0003]| the|\allowbreak\cut{[CUT]}\allowbreak\token[0.8417]| additional|\allowbreak\token[0.3947]| bag|\allowbreak\token[0.8785]| is|\allowbreak\token[0.0000]| \(|\allowbreak\token[0.0001]|7|\allowbreak\token[0.0000]|x|\allowbreak\token[0.2478]|\|\allowbreak\token[0.0619]|).|\allowbreak\token[0.1434]| After|\allowbreak\token[0.0428]| finding|\allowbreak\token[0.0027]| the|\allowbreak\token[0.1502]| bag|\allowbreak\token[0.6866]| with|\allowbreak\token[0.0002]| |\allowbreak\token[0.0000]|5|\allowbreak\token[0.0001]|3|\allowbreak\token[0.0064]| coins|\allowbreak\token[0.1746]|,|\allowbreak\token[0.7104]| you|\allowbreak\token[0.2338]| redistribute|\allowbreak\token[0.8207]| all|\allowbreak\token[0.1743]| the|\allowbreak\token[0.0002]| coins|\allowbreak\token[0.7935]| into|\allowbreak\token[0.0057]| eight|\allowbreak\token[0.1040]| bags|\allowbreak\token[0.7756]| with|\allowbreak\token[0.1464]| an|\allowbreak\token[0.0016]| equal|\allowbreak\token[0.0013]| number|\allowbreak\token[0.0144]| of|\allowbreak\token[0.0000]| coins|\allowbreak\token[1.1025]|,|\allowbreak\token[0.9695]| which|\allowbreak\token[0.5044]| we|\allowbreak\token[0.0822]|'ll|\allowbreak\token[0.0494]| call|\allowbreak\token[0.0000]| \(|\allowbreak\token[0.0062]|y|\allowbreak\token[0.0002]|\|\allowbreak\token[0.0502]|).|\allowbreak\token[0.8690]| This|\allowbreak\token[0.1096]| means|\allowbreak\token[0.1777]| the|\allowbreak\token[0.0753]| total|\allowbreak\token[0.0001]| number|\allowbreak\token[0.0000]| of|\allowbreak\token[0.0000]| coins|\allowbreak\token[0.7641]| is|\allowbreak\token[0.6064]| now|\allowbreak\token[0.0227]| \(|\allowbreak\token[0.5764]|7|\allowbreak\token[0.0003]|x|\allowbreak\token[0.0000]| +|\allowbreak\token[0.0000]| |\allowbreak\token[0.0000]|5|\allowbreak\token[0.0000]|3|\allowbreak\token[0.6698]|\|\allowbreak\token[0.0232]|),|\allowbreak\token[0.0038]| and|\allowbreak\token[0.0894]| this|\allowbreak\token[0.1714]| total|\allowbreak\token[1.0999]| is|\allowbreak\token[1.1997]| divisible|\allowbreak\token[0.0057]| by|\allowbreak\token[0.0004]| |\allowbreak\token[0.0001]|8|\allowbreak\token[1.3067]|,|\allowbreak\token[1.4839]| so|\allowbreak\token[0.7600]| \(|\allowbreak\token[0.0621]|7|\allowbreak\token[0.0000]|x|\allowbreak\token[0.0000]| +|\allowbreak\token[0.0000]| |\allowbreak\token[0.0000]|5|\allowbreak\token[0.0001]|3|\allowbreak\token[0.0127]| =|\allowbreak\token[0.0000]| |\allowbreak\token[0.0000]|8|\allowbreak\token[0.0003]|y|\allowbreak\token[0.0008]|\|\allowbreak\token[0.6026]|).↩↩|\tighttracebreak\allowbreak\token[0.4085]|Given|\allowbreak\token[0.0002]| that|\allowbreak\token[0.5546]| you|\allowbreak\token[0.0619]| have|\allowbreak\token[0.0010]| more|\allowbreak\token[0.0000]| than|\allowbreak\token[0.0001]| |\allowbreak\token[0.0000]|2|\allowbreak\token[0.0000]|0|\allowbreak\token[0.0000]|0|\allowbreak\token[0.0001]| coins|\allowbreak\token[1.0395]| in|\allowbreak\token[0.0074]| total|\allowbreak\token[0.5339]|,|\allowbreak\token[0.0983]| we|\allowbreak\token[1.1646]| have|\allowbreak\token[0.0967]| \(|\allowbreak\token[0.2702]|7|\allowbreak\token[0.0000]|x|\allowbreak\token[0.0003]| +|\allowbreak\token[0.0000]| |\allowbreak\token[0.0000]|5|\allowbreak\token[0.0001]|3|\allowbreak\token[0.0000]| >|\allowbreak\token[0.0000]| |\allowbreak\token[0.0000]|2|\allowbreak\token[0.0001]|0|\allowbreak\token[0.0002]|0|\allowbreak\token[0.0000]|\|\allowbreak\token[0.6786]|),|\allowbreak\token[0.0054]| which|\allowbreak\token[0.0017]| simpl|\allowbreak\token[0.0000]|ifies|\allowbreak\token[0.0000]| to|\allowbreak\token[0.0000]| \(|\allowbreak\token[0.0294]|7|\allowbreak\token[0.0000]|x|\allowbreak\token[0.0000]| >|\allowbreak\token[0.0000]| |\allowbreak\token[0.0000]|1|\allowbreak\token[0.0000]|4|\allowbreak\token[0.0001]|7|\allowbreak\token[0.0494]|\|\allowbreak\token[0.0446]|),|\allowbreak\token[0.3258]| and|\allowbreak\token[0.3191]| thus|\allowbreak\token[0.0002]| \(|\allowbreak\token[0.0000]|x|\allowbreak\token[0.0000]| >|\allowbreak\token[0.0001]| |\allowbreak\token[0.0001]|2|\allowbreak\token[0.0000]|1|\allowbreak\token[0.0004]|\|\allowbreak\token[0.1096]|).|\allowbreak\token[0.9808]| Since|\allowbreak\token[0.0127]| \(|\allowbreak\token[0.0035]|x|\allowbreak\token[0.0002]|\)|\allowbreak\token[0.0945]| must|\allowbreak\token[0.0000]| be|\allowbreak\token[0.0294]| an|\allowbreak\token[0.0000]| integer|\allowbreak[...]
\tcblower
\token[0.7624]| additional|\allowbreak\token[0.3946]| bag|\allowbreak\token[0.8400]| is|\allowbreak\token[0.0000]| \(|\allowbreak\token[0.0002]|7|\allowbreak\token[0.0000]|x|\allowbreak\token[0.2479]|\|\allowbreak\token[0.0619]|).|\allowbreak\token[0.1431]| After|\allowbreak\token[0.0419]| finding|\allowbreak\token[0.0029]| the|\allowbreak\token[0.1497]| bag|\allowbreak\token[0.6867]| with|\allowbreak\token[0.0002]| |\allowbreak\token[0.0000]|5|\allowbreak\token[0.0001]|3|\allowbreak\token[0.0064]| coins|\allowbreak\token[0.1746]|,|\allowbreak\token[0.7143]| the|\allowbreak\token[0.0463]| total|\allowbreak\token[0.6932]| number|\allowbreak\token[0.0004]| of|\allowbreak\token[0.0092]| coins|\allowbreak\token[0.0146]| becomes|\allowbreak\token[0.0000]| \(|\allowbreak\token[0.0001]|7|\allowbreak\token[0.0000]|x|\allowbreak\token[0.0000]| +|\allowbreak\token[0.0000]| |\allowbreak\token[0.0000]|5|\allowbreak\token[0.0001]|3|\allowbreak\token[0.0001]|\|\allowbreak\token[0.3016]|).|\allowbreak\token[1.7580]| We|\allowbreak\token[0.5028]| are|\allowbreak\token[0.1228]| told|\allowbreak\token[0.0059]| that|\allowbreak\token[1.2052]| this|\allowbreak\token[0.4254]| total|\allowbreak\token[0.8852]| can|\allowbreak\token[0.0004]| be|\allowbreak\token[0.0780]| evenly|\allowbreak\token[0.7035]| distributed|\allowbreak\token[0.3447]| among|\allowbreak\token[0.3366]| eight|\allowbreak\token[0.0000]| bags|\allowbreak\token[0.0105]|,|\allowbreak\token[0.7304]| meaning|\allowbreak\token[0.9521]| \(|\allowbreak\token[0.0016]|7|\allowbreak\token[0.0000]|x|\allowbreak\token[0.0000]| +|\allowbreak\token[0.0000]| |\allowbreak\token[0.0000]|5|\allowbreak\token[0.0000]|3|\allowbreak\token[0.0494]|\)|\allowbreak\token[0.1611]| is|\allowbreak\token[0.0497]| divisible|\allowbreak\token[0.0000]| by|\allowbreak\token[0.0002]| |\allowbreak\token[0.0000]|8|\allowbreak\token[0.6446]|.|\allowbreak\token[0.3850]| Additionally|\allowbreak\token[0.0000]|,|\allowbreak\token[0.4298]| we|\allowbreak\token[0.0626]| know|\allowbreak\token[0.1411]| that|\allowbreak\token[0.2957]| the|\allowbreak\token[0.0057]| total|\allowbreak\token[0.1290]| number|\allowbreak\token[0.0000]| of|\allowbreak\token[0.0000]| coins|\allowbreak\token[0.1499]| is|\allowbreak\token[0.0370]| more|\allowbreak\token[0.0000]| than|\allowbreak\token[0.0000]| |\allowbreak\token[0.0000]|2|\allowbreak\token[0.0000]|0|\allowbreak\token[0.0001]|0|\allowbreak\token[0.2596]|.↩↩|\tighttracebreak\allowbreak\token[1.0650]|Let|\allowbreak\token[0.0010]|'s|\allowbreak\token[1.1876]| find|\allowbreak\token[0.0004]| the|\allowbreak\token[0.0015]| smallest|\allowbreak\token[0.6779]| \(|\allowbreak\token[0.0075]|x|\allowbreak\token[0.0000]|\)|\allowbreak\token[0.0227]| that|\allowbreak\token[0.0018]| satisfies|\allowbreak\token[0.0003]| these|\allowbreak\token[0.0000]| conditions|\allowbreak\token[0.1156]|.|\allowbreak\token[0.6108]| We|\allowbreak\token[0.8661]| can|\allowbreak\token[0.9641]| start|\allowbreak\token[0.0000]| by|\allowbreak\token[1.5854]| finding|\allowbreak\token[0.1394]| the|\allowbreak\token[0.0243]| smallest|\allowbreak\token[0.8114]| multiple|\allowbreak\token[0.0000]| of|\allowbreak\token[0.0001]| |\allowbreak\token[0.0844]|8|\allowbreak\token[0.0496]| that|\allowbreak\token[0.0025]| is|\allowbreak\token[0.1762]| greater|\allowbreak\token[0.0000]| than|\allowbreak\token[0.0038]| |\allowbreak\token[0.0017]|2|\allowbreak\token[0.0002]|0|\allowbreak\token[0.0197]|0|\allowbreak\token[0.6623]|,|\allowbreak\token[0.0709]| which|\allowbreak\token[0.0002]| is|\allowbreak\token[0.0047]| |\allowbreak\token[0.0002]|2|\allowbreak\token[0.0092]|0|\allowbreak\token[0.0508]|8|\allowbreak\token[0.0816]|.|\allowbreak\token[0.3196]| However|\allowbreak\token[0.0000]|,|\allowbreak\token[1.2859]| this|\allowbreak\token[1.2252]| is|\allowbreak\token[0.7771]| the|\allowbreak\token[0.0296]| total|\allowbreak\token[0.7510]| after|\allowbreak\token[0.6216]| adding|\allowbreak\token[0.2273]| the|\allowbreak\token[0.0164]| |\allowbreak\token[0.0000]|5|\allowbreak\token[0.0000]|3|\allowbreak\token[0.0025]| coins|\allowbreak\token[0.1888]|,|\allowbreak\token[0.6653]| so|\allowbreak\token[0.9761]| we|\allowbreak\token[0.6816]| need|\allowbreak\token[0.0007]| to|\allowbreak\token[0.5736]| subtract|\allowbreak\token[0.0290]| |\allowbreak\token[0.0001]|5|\allowbreak\token[0.0000]|3|\allowbreak\token[0.6304]| to|\allowbreak\token[0.0036]| find|\allowbreak[...]
\end{twopartbox}
\begin{twopartbox}{Cut Index 85 | Acceptance Probability 0.0000}{Block 4 Step 2}{Proposal (Rejected)}
\token[2.2824]| To|\allowbreak\token[0.4171]| solve|\allowbreak\token[0.4477]| this|\allowbreak\token[0.0840]| problem|\allowbreak\token[0.0111]|,|\allowbreak\token[0.5775]| let|\allowbreak\token[0.0029]|'s|\allowbreak\token[0.8168]| denote|\allowbreak\token[0.0233]| the|\allowbreak\token[0.7993]| number|\allowbreak\token[0.0000]| of|\allowbreak\token[0.1007]| coins|\allowbreak\token[0.0804]| in|\allowbreak\token[0.0205]| each|\allowbreak\token[0.0183]| of|\allowbreak\token[0.0031]| the|\allowbreak\token[0.5575]| original|\allowbreak\token[0.0329]| seven|\allowbreak\token[0.0001]| bags|\allowbreak\token[0.0009]| as|\allowbreak\token[0.6010]| \(|\allowbreak\token[0.8095]|x|\allowbreak\token[0.0048]|\|\allowbreak\token[0.1099]|).|\allowbreak\token[0.9476]| Therefore|\allowbreak\token[0.0000]|,|\allowbreak\token[0.6384]| the|\allowbreak\token[0.0007]| total|\allowbreak\token[0.0001]| number|\allowbreak\token[0.0000]| of|\allowbreak\token[0.0001]| coins|\allowbreak\token[0.5853]| before|\allowbreak\token[0.0055]| finding|\allowbreak\token[0.0003]| the|\allowbreak\token[0.8417]| additional|\allowbreak\token[0.3947]| bag|\allowbreak\token[0.8785]| is|\allowbreak\token[0.0000]| \(|\allowbreak\token[0.0001]|7|\allowbreak\token[0.0000]|x|\allowbreak\token[0.2478]|\|\allowbreak\token[0.0619]|).|\allowbreak\token[0.1434]| After|\allowbreak\token[0.0428]| finding|\allowbreak\token[0.0027]| the|\allowbreak\token[0.1502]| bag|\allowbreak\token[0.6866]| with|\allowbreak\token[0.0002]| |\allowbreak\token[0.0000]|5|\allowbreak\token[0.0001]|3|\allowbreak\token[0.0064]| coins|\allowbreak\token[0.1746]|,|\allowbreak\token[0.7104]| you|\allowbreak\token[0.2338]| redistribute|\allowbreak\token[0.8207]| all|\allowbreak\token[0.1743]| the|\allowbreak\token[0.0002]| coins|\allowbreak\token[0.7935]| into|\allowbreak\token[0.0057]| eight|\allowbreak\token[0.1040]| bags|\allowbreak\token[0.7756]| with|\allowbreak\token[0.1464]| an|\allowbreak\token[0.0016]| equal|\allowbreak\token[0.0013]| number|\allowbreak\token[0.0144]| of|\allowbreak\token[0.0000]| coins|\allowbreak\token[1.1025]|,|\allowbreak\token[0.9695]| which|\allowbreak\token[0.5044]| we|\allowbreak\token[0.0822]|'ll|\allowbreak\token[0.0494]| call|\allowbreak\token[0.0000]| \(|\allowbreak\token[0.0062]|y|\allowbreak\token[0.0002]|\|\allowbreak\token[0.0502]|).|\allowbreak\token[0.8690]| This|\allowbreak\token[0.1096]| means|\allowbreak\token[0.1777]| the|\allowbreak\token[0.0753]| total|\allowbreak\token[0.0001]| number|\allowbreak\token[0.0000]| of|\allowbreak\token[0.0000]| coins|\allowbreak\token[0.7641]| is|\allowbreak\token[0.6064]| now|\allowbreak\token[0.0227]| \(|\allowbreak\cut{[CUT]}\allowbreak\token[0.5764]|7|\allowbreak\token[0.0003]|x|\allowbreak\token[0.0000]| +|\allowbreak\token[0.0000]| |\allowbreak\token[0.0000]|5|\allowbreak\token[0.0000]|3|\allowbreak\token[0.6698]|\|\allowbreak\token[0.0232]|),|\allowbreak\token[0.0038]| and|\allowbreak\token[0.0894]| this|\allowbreak\token[0.1714]| total|\allowbreak\token[1.0999]| is|\allowbreak\token[1.1997]| divisible|\allowbreak\token[0.0057]| by|\allowbreak\token[0.0004]| |\allowbreak\token[0.0001]|8|\allowbreak\token[1.3067]|,|\allowbreak\token[1.4839]| so|\allowbreak\token[0.7600]| \(|\allowbreak\token[0.0621]|7|\allowbreak\token[0.0000]|x|\allowbreak\token[0.0000]| +|\allowbreak\token[0.0000]| |\allowbreak\token[0.0000]|5|\allowbreak\token[0.0001]|3|\allowbreak\token[0.0127]| =|\allowbreak\token[0.0000]| |\allowbreak\token[0.0000]|8|\allowbreak\token[0.0003]|y|\allowbreak\token[0.0008]|\|\allowbreak\token[0.6026]|).↩↩|\tighttracebreak\allowbreak\token[0.4085]|Given|\allowbreak\token[0.0002]| that|\allowbreak\token[0.5546]| you|\allowbreak\token[0.0619]| have|\allowbreak\token[0.0010]| more|\allowbreak\token[0.0000]| than|\allowbreak\token[0.0001]| |\allowbreak\token[0.0000]|2|\allowbreak\token[0.0000]|0|\allowbreak\token[0.0000]|0|\allowbreak\token[0.0001]| coins|\allowbreak\token[1.0395]| in|\allowbreak\token[0.0074]| total|\allowbreak\token[0.5339]|,|\allowbreak\token[0.0983]| we|\allowbreak\token[1.1646]| have|\allowbreak\token[0.0967]| \(|\allowbreak\token[0.2702]|7|\allowbreak\token[0.0000]|x|\allowbreak\token[0.0003]| +|\allowbreak\token[0.0000]| |\allowbreak\token[0.0000]|5|\allowbreak\token[0.0001]|3|\allowbreak\token[0.0000]| >|\allowbreak\token[0.0000]| |\allowbreak\token[0.0000]|2|\allowbreak\token[0.0001]|0|\allowbreak\token[0.0002]|0|\allowbreak\token[0.0000]|\|\allowbreak\token[0.6786]|),|\allowbreak\token[0.0054]| which|\allowbreak\token[0.0017]| simpl|\allowbreak\token[0.0000]|ifies|\allowbreak\token[0.0000]| to|\allowbreak\token[0.0000]| \(|\allowbreak\token[0.0294]|7|\allowbreak\token[0.0000]|x|\allowbreak\token[0.0000]| >|\allowbreak\token[0.0000]| |\allowbreak\token[0.0000]|1|\allowbreak\token[0.0000]|4|\allowbreak\token[0.0001]|7|\allowbreak\token[0.0494]|\|\allowbreak\token[0.0446]|),|\allowbreak\token[0.3258]| and|\allowbreak\token[0.3191]| thus|\allowbreak\token[0.0002]| \(|\allowbreak\token[0.0000]|x|\allowbreak\token[0.0000]| >|\allowbreak\token[0.0001]| |\allowbreak\token[0.0001]|2|\allowbreak\token[0.0000]|1|\allowbreak\token[0.0004]|\|\allowbreak\token[0.1096]|).|\allowbreak\token[0.9808]| Since|\allowbreak\token[0.0127]| \(|\allowbreak\token[0.0035]|x|\allowbreak\token[0.0002]|\)|\allowbreak\token[0.0945]| must|\allowbreak\token[0.0000]| be|\allowbreak\token[0.0294]| an|\allowbreak\token[0.0000]| integer|\allowbreak\token[0.3227]|,|\allowbreak\token[0.0030]| the|\allowbreak\token[0.0073]| smallest|\allowbreak\token[0.9362]| possible|\allowbreak\token[0.0018]| value|\allowbreak\token[0.0000]| for|\allowbreak\token[0.0000]| \(|\allowbreak\token[0.0004]|x|\allowbreak\token[0.0000]|\)|\allowbreak\token[0.4213]| is|\allowbreak\token[0.0042]| |\allowbreak\token[0.0001]|2|\allowbreak\token[0.0016]|2|\allowbreak\token[0.7191]|.↩↩|\tighttracebreak\allowbreak\token[0.9831]|Sub|\allowbreak\token[0.0000]|stit|\allowbreak\token[0.0005]|uting|\allowbreak\token[0.0000]| \(|\allowbreak\token[0.0000]|x|\allowbreak\token[0.0000]| =|\allowbreak\token[0.0000]| |\allowbreak\token[0.0000]|2|\allowbreak\token[0.0000]|2|\allowbreak\token[0.0001]|\)|\allowbreak\token[0.0010]| into|\allowbreak\token[0.7102]| the|\allowbreak\token[0.0135]| equation|\allowbreak\token[0.1743]| \(|\allowbreak\token[0.0002]|7|\allowbreak\token[0.0000]|x|\allowbreak\token[0.0001]| +|\allowbreak\token[0.0000]| |\allowbreak\token[0.0000]|5|\allowbreak\token[0.0001]|3|\allowbreak\token[0.1096]| =|\allowbreak\token[0.0000]| |\allowbreak\token[0.0002]|8|\allowbreak\token[0.0000]|y|\allowbreak\token[0.0017]|\|\allowbreak\token[0.0000]|),|\allowbreak\token[0.0000]| we|\allowbreak\token[0.0104]| get|\allowbreak\token[0.0670]| \(|\allowbreak\token[0.0739]|7|\allowbreak\token[0.0337]|(|\allowbreak\token[0.0000]|2|\allowbreak\token[0.0000]|2|\allowbreak\token[0.0000]|)|\allowbreak\token[0.0000]| +|\allowbreak\token[0.0000]| |\allowbreak\token[0.0000]|5|\allowbreak[...]
\tcblower
\token[0.5761]|7|\allowbreak\token[0.0003]|x|\allowbreak\token[0.0000]| +|\allowbreak\token[0.0000]| |\allowbreak\token[0.0000]|5|\allowbreak\token[0.0000]|3|\allowbreak\token[0.6698]|\|\allowbreak\token[0.0207]|),|\allowbreak\token[0.0034]| and|\allowbreak\token[0.0878]| this|\allowbreak\token[0.2029]| total|\allowbreak\token[1.0431]| is|\allowbreak\token[1.2853]| divisible|\allowbreak\token[0.0051]| by|\allowbreak\token[0.0004]| |\allowbreak\token[0.0001]|8|\allowbreak\token[1.2232]|,|\allowbreak\token[1.4512]| as|\allowbreak\token[1.1631]| it|\allowbreak\token[0.7259]| is|\allowbreak\token[0.8247]| equally|\allowbreak\token[0.1153]| distributed|\allowbreak\token[0.6770]| among|\allowbreak\token[0.3201]| the|\allowbreak\token[0.0010]| eight|\allowbreak\token[0.0001]| bags|\allowbreak\token[0.5398]|.|\allowbreak\token[1.5957]| So|\allowbreak\token[0.0000]|,|\allowbreak\token[0.0158]| we|\allowbreak\token[0.0003]| have|\allowbreak\token[0.7921]|:↩|\newline\allowbreak\token[0.0028]|\[|\allowbreak\token[0.0062]|7|\allowbreak\token[0.0000]|x|\allowbreak\token[0.0000]| +|\allowbreak\token[0.0000]| |\allowbreak\token[0.0000]|5|\allowbreak\token[0.0001]|3|\allowbreak\token[0.0007]| =|\allowbreak\token[0.0000]| |\allowbreak\token[0.0000]|8|\allowbreak\token[0.0001]|y|\allowbreak\token[0.0003]|\|\allowbreak\token[0.5824]|]↩↩|\tighttracebreak\allowbreak\token[0.7526]|Given|\allowbreak\token[0.0002]| that|\allowbreak\token[0.7396]| the|\allowbreak\token[0.0014]| total|\allowbreak\token[0.2130]| number|\allowbreak\token[0.0000]| of|\allowbreak\token[0.0000]| coins|\allowbreak\token[0.2542]| is|\allowbreak\token[0.0062]| more|\allowbreak\token[0.0000]| than|\allowbreak\token[0.0000]| |\allowbreak\token[0.0000]|2|\allowbreak\token[0.0000]|0|\allowbreak\token[0.0001]|0|\allowbreak\token[0.0497]|,|\allowbreak\token[0.0284]| we|\allowbreak\token[0.9664]| have|\allowbreak\token[0.1775]|:↩|\newline\allowbreak\token[0.0001]|\[|\allowbreak\token[0.0611]|7|\allowbreak\token[0.0000]|x|\allowbreak\token[0.0006]| +|\allowbreak\token[0.0000]| |\allowbreak\token[0.0000]|5|\allowbreak\token[0.0001]|3|\allowbreak\token[0.0001]| >|\allowbreak\token[0.0000]| |\allowbreak\token[0.0000]|2|\allowbreak\token[0.0000]|0|\allowbreak\token[0.0001]|0|\allowbreak\token[0.0000]|\|\allowbreak\token[0.6630]|]↩|\newline\allowbreak\token[0.0050]|\[|\allowbreak\token[0.0004]|7|\allowbreak\token[0.0000]|x|\allowbreak\token[0.0000]| >|\allowbreak\token[0.0000]| |\allowbreak\token[0.0000]|1|\allowbreak\token[0.0000]|4|\allowbreak\token[0.0000]|7|\allowbreak\token[0.0000]|\|\allowbreak\token[0.0006]|]↩|\newline\allowbreak\token[0.0002]|\[|\allowbreak\token[0.0000]|x|\allowbreak\token[0.0000]| >|\allowbreak\token[0.0002]| |\allowbreak\token[0.0000]|2|\allowbreak\token[0.0000]|1|\allowbreak\token[0.0005]|\|\allowbreak\token[0.0004]|]↩↩|\tighttracebreak\allowbreak\token[0.5556]|Since|\allowbreak\token[0.0083]| \(|\allowbreak\token[0.0184]|x|\allowbreak\token[0.0005]|\)|\allowbreak\token[0.2388]| must|\allowbreak\token[0.0000]| be|\allowbreak\token[0.0295]| an|\allowbreak\token[0.0000]| integer|\allowbreak\token[0.7549]|,|\allowbreak\token[0.0496]| the|\allowbreak\token[0.0059]| smallest|\allowbreak\token[1.1931]| possible|\allowbreak\token[0.0009]| value|\allowbreak\token[0.0000]| for|\allowbreak\token[0.0000]| \(|\allowbreak\token[0.0010]|x|\allowbreak\token[0.0000]|\)|\allowbreak\token[0.6834]| is|\allowbreak\token[0.0011]| |\allowbreak\token[0.0001]|2|\allowbreak\token[0.0015]|2|\allowbreak\token[0.0092]|.|\allowbreak\token[0.5534]| However|\allowbreak\token[0.0000]|,|\allowbreak\token[0.1862]| we|\allowbreak\token[0.0648]| need|\allowbreak\token[0.0004]| to|\allowbreak\token[0.6851]| find|\allowbreak\token[0.0032]| the|\allowbreak\token[0.0101]| smallest|\allowbreak\token[0.7988]| number|\allowbreak\token[0.0007]| of|\allowbreak\token[0.0006]| coins|\allowbreak\token[1.1240]| you|\allowbreak\token[0.0009]| could|\allowbreak\token[0.0000]| have|\allowbreak\token[0.0059]| had|\allowbreak\token[0.1309]| before|\allowbreak\token[0.0050]| finding|\allowbreak\token[0.0000]| the|\allowbreak\token[0.2679]| bag|\allowbreak\token[0.0013]| of|\allowbreak\token[0.0000]| |\allowbreak\token[0.0000]|5|\allowbreak\token[0.0001]|3|\allowbreak\token[0.0036]| coins|\allowbreak\token[0.1337]|,|\allowbreak\token[0.1008]| which|\allowbreak\token[0.3728]| means|\allowbreak[...]
\end{twopartbox}
\begin{twopartbox}{Cut Index 497 | Acceptance Probability 0.0268}{Block 4 Step 3}{Proposal (Rejected)}
[...]\allowbreak\token[0.0238]|6|\allowbreak\token[0.0022]|1|\allowbreak\token[0.0001]| +|\allowbreak\token[0.0000]| |\allowbreak\token[0.0000]|5|\allowbreak\token[0.0000]|3|\allowbreak\token[0.0000]| =|\allowbreak\token[0.0000]| |\allowbreak\token[0.0001]|2|\allowbreak\token[0.0001]|1|\allowbreak\token[0.0001]|4|\allowbreak\token[0.0000]|\|\allowbreak\token[0.1454]|),|\allowbreak\token[0.6115]| which|\allowbreak\token[0.0005]| is|\allowbreak\token[1.3057]| still|\allowbreak\token[0.0054]| not|\allowbreak\token[0.0001]| divisible|\allowbreak\token[0.0000]| by|\allowbreak\token[0.0000]| |\allowbreak\token[0.0000]|8|\allowbreak\token[0.6630]|.↩↩|\tighttracebreak\allowbreak\token[0.3222]|Contin|\allowbreak\token[0.0000]|uing|\allowbreak\token[0.0000]| this|\allowbreak\token[0.0000]| process|\allowbreak\token[0.0003]|,|\allowbreak\token[0.5751]| for|\allowbreak\token[0.0000]| \(|\allowbreak\token[0.0000]|x|\allowbreak\token[0.0000]| =|\allowbreak\token[0.0000]| |\allowbreak\token[0.0004]|2|\allowbreak\token[0.2636]|4|\allowbreak\token[0.0000]|\|\allowbreak\token[0.0000]|),|\allowbreak\token[0.0001]| we|\allowbreak\token[0.4357]| have|\allowbreak\token[0.0000]| \(|\allowbreak\token[0.0000]|7|\allowbreak\token[0.0000]|(|\allowbreak\token[0.0000]|2|\allowbreak\token[0.0001]|4|\allowbreak\token[0.0000]|)|\allowbreak\token[0.0000]| +|\allowbreak\token[0.0000]| |\allowbreak\token[0.0000]|5|\allowbreak\token[0.0000]|3|\allowbreak\token[0.0000]| =|\allowbreak\token[0.0000]| |\allowbreak\token[0.0000]|1|\allowbreak\token[0.0005]|6|\allowbreak\token[0.0000]|8|\allowbreak\token[0.0000]| +|\allowbreak\token[0.0000]| |\allowbreak\token[0.0000]|5|\allowbreak\token[0.0000]|3|\allowbreak\token[0.0000]| =|\allowbreak\token[0.0000]| |\allowbreak\token[0.0000]|2|\allowbreak\token[0.0003]|2|\allowbreak\token[0.0000]|1|\allowbreak\token[0.0000]|\|\allowbreak\token[0.0547]|),|\allowbreak\token[0.2690]| which|\allowbreak\token[0.0002]| is|\allowbreak\token[1.2670]| also|\allowbreak\token[0.0000]| not|\allowbreak\token[0.0005]| divisible|\allowbreak\token[0.0000]| by|\allowbreak\token[0.0000]| |\allowbreak\token[0.0000]|8|\allowbreak\token[0.0993]|.|\allowbreak\token[0.7500]| For|\allowbreak\token[0.0000]| \(|\allowbreak\token[0.0000]|x|\allowbreak\token[0.0000]| =|\allowbreak\token[0.0000]| |\allowbreak\token[0.0000]|2|\allowbreak\token[0.0014]|5|\allowbreak\token[0.0000]|\|\allowbreak\token[0.0000]|),|\allowbreak\token[0.0005]| we|\allowbreak\token[0.0789]| have|\allowbreak\token[0.0000]| \(|\allowbreak\token[0.0042]|7|\allowbreak\token[0.0000]|(|\allowbreak\token[0.0000]|2|\allowbreak\token[0.0001]|5|\allowbreak\token[0.0000]|)|\allowbreak\token[0.0000]| +|\allowbreak\token[0.0000]| |\allowbreak\token[0.0000]|5|\allowbreak\token[0.0000]|3|\allowbreak\token[0.0000]| =|\allowbreak\token[0.0000]| |\allowbreak\token[0.0000]|1|\allowbreak\token[0.0000]|7|\allowbreak\token[0.0000]|5|\allowbreak\token[0.0000]| +|\allowbreak\token[0.0000]| |\allowbreak\token[0.0000]|5|\allowbreak\token[0.0000]|3|\allowbreak\token[0.0000]| =|\allowbreak\token[0.0000]| |\allowbreak\token[0.0000]|2|\allowbreak\token[0.0001]|2|\allowbreak\token[0.0001]|8|\allowbreak\token[0.0000]|\|\allowbreak\token[0.0041]|),|\allowbreak\token[0.6763]| and|\allowbreak\token[0.7388]| \(|\allowbreak\token[0.0126]|2|\allowbreak\token[0.0000]|2|\allowbreak\token[0.0000]|8|\allowbreak\token[0.0255]|\)|\allowbreak\token[0.0605]| is|\allowbreak\token[0.1459]| divisible|\allowbreak\token[0.0000]| by|\allowbreak\token[0.2911]| |\allowbreak\token[0.0000]|8|\allowbreak\token[1.2947]|,|\allowbreak\token[1.2634]| giving|\allowbreak\token[0.1594]| us|\allowbreak\token[0.0007]| \(|\allowbreak\token[0.4197]|y|\allowbreak\token[0.0000]| =|\allowbreak\token[0.0008]| |\allowbreak\token[0.0001]|2|\allowbreak\token[0.0281]|8|\allowbreak\token[0.1465]|.|\allowbreak\token[0.0003]|5|\allowbreak\token[0.0011]|\|\allowbreak\token[0.0504]|),|\allowbreak\token[0.6854]| but|\allowbreak\token[0.5913]| since|\allowbreak\token[0.0254]| \(|\allowbreak\token[0.0001]|y|\allowbreak\token[0.0000]|\)|\allowbreak\token[0.0003]| must|\allowbreak\token[0.3920]| be|\allowbreak\token[0.0140]| an|\allowbreak\token[0.0001]| integer|\allowbreak\token[0.0054]|,|\allowbreak\cut{[CUT]}\allowbreak\token[0.7654]| this|\allowbreak\token[0.0575]| is|\allowbreak\token[0.2197]| not|\allowbreak\token[0.6561]| the|\allowbreak\token[0.1608]| correct|\allowbreak\token[1.2719]| solution|\allowbreak\token[0.1618]|.↩↩|\tighttracebreak\allowbreak\token[0.4113]|For|\allowbreak\token[0.0000]| \(|\allowbreak\token[0.0000]|x|\allowbreak\token[0.0000]| =|\allowbreak\token[0.0000]| |\allowbreak\token[0.0000]|2|\allowbreak\token[0.0142]|6|\allowbreak\token[0.0000]|\|\allowbreak\token[0.0000]|),|\allowbreak\token[0.0004]| we|\allowbreak\token[0.0034]| have|\allowbreak\token[0.0000]| \(|\allowbreak\token[0.0000]|7|\allowbreak\token[0.0000]|(|\allowbreak\token[0.0000]|2|\allowbreak\token[0.0003]|6|\allowbreak\token[0.0000]|)|\allowbreak\token[0.0000]| +|\allowbreak\token[0.0000]| |\allowbreak\token[0.0000]|5|\allowbreak\token[0.0000]|3|\allowbreak\token[0.0000]| =|\allowbreak\token[0.0000]| |\allowbreak\token[0.0000]|1|\allowbreak\token[0.0002]|8|\allowbreak\token[0.0000]|2|\allowbreak\token[0.0000]| +|\allowbreak\token[0.0000]| |\allowbreak\token[0.0000]|5|\allowbreak\token[0.0000]|3|\allowbreak\token[0.0000]| =|\allowbreak\token[0.0000]| |\allowbreak\token[0.0000]|2|\allowbreak\token[0.0001]|3|\allowbreak\token[0.0001]|5|\allowbreak\token[0.0000]|\|\allowbreak\token[0.0004]|),|\allowbreak\token[0.1752]| which|\allowbreak\token[0.0006]| is|\allowbreak\token[0.1688]| not|\allowbreak\token[0.0013]| divisible|\allowbreak\token[0.0000]| by|\allowbreak\token[0.0000]| |\allowbreak\token[0.0000]|8|\allowbreak\token[0.0068]|.|\allowbreak\token[0.1847]| For|\allowbreak\token[0.0000]| \(|\allowbreak\token[0.0000]|x|\allowbreak\token[0.0000]| =|\allowbreak\token[0.0000]| |\allowbreak\token[0.0000]|2|\allowbreak\token[0.0001]|7|\allowbreak\token[0.0000]|\|\allowbreak\token[0.0000]|),|\allowbreak\token[0.0000]| we|\allowbreak\token[0.0005]| have|\allowbreak\token[0.0000]| \(|\allowbreak\token[0.0007]|7|\allowbreak\token[0.0000]|(|\allowbreak\token[0.0000]|2|\allowbreak\token[0.0002]|7|\allowbreak\token[0.0000]|)|\allowbreak\token[0.0000]| +|\allowbreak\token[0.0001]| |\allowbreak\token[0.0000]|5|\allowbreak\token[0.0000]|3|\allowbreak\token[0.0000]| =|\allowbreak\token[0.0000]| |\allowbreak\token[0.0000]|1|\allowbreak\token[0.0003]|8|\allowbreak\token[0.0000]|9|\allowbreak\token[0.0000]| +|\allowbreak\token[0.0000]| |\allowbreak\token[0.0000]|5|\allowbreak\token[0.0000]|3|\allowbreak\token[0.0000]| =|\allowbreak\token[0.0000]| |\allowbreak\token[0.0000]|2|\allowbreak\token[0.0000]|4|\allowbreak\token[0.0001]|2|\allowbreak\token[0.0000]|\|\allowbreak\token[0.0001]|),|\allowbreak\token[0.6283]| and|\allowbreak\token[0.0015]| \(|\allowbreak\token[0.0000]|2|\allowbreak\token[0.0000]|4|\allowbreak\token[0.0000]|2|\allowbreak\token[0.0022]|\)|\allowbreak\token[0.0265]| is|\allowbreak\token[0.5992]| divisible|\allowbreak\token[0.0000]| by|\allowbreak\token[0.0000]| |\allowbreak\token[0.0000]|8|\allowbreak\token[0.0001]|,|\allowbreak\token[0.0062]| giving|\allowbreak\token[0.0000]| us|\allowbreak\token[0.0001]| \(|\allowbreak\token[0.0008]|y|\allowbreak\token[0.0000]| =|\allowbreak\token[0.0000]| |\allowbreak\token[0.0156]|3|\allowbreak\token[0.0005]|0|\allowbreak\token[0.0407]|.|\allowbreak\token[0.0007]|2|\allowbreak\token[0.0000]|5|\allowbreak\token[0.0000]|\|\allowbreak\token[0.0006]|),|\allowbreak\token[0.5420]| which|\allowbreak\token[0.0327]| is|\allowbreak\token[0.2796]| not|\allowbreak\token[0.1985]| an|\allowbreak\token[0.0000]| integer|\allowbreak\token[0.2723]|.↩↩|\tighttracebreak\allowbreak\token[0.7013]|For|\allowbreak\token[0.0000]| \(|\allowbreak\token[0.0000]|x|\allowbreak\token[0.0000]| =|\allowbreak\token[0.0000]| |\allowbreak\token[0.0000]|2|\allowbreak\token[0.0032]|8|\allowbreak\token[0.0000]|\|\allowbreak\token[0.0000]|),|\allowbreak\token[0.0000]| we|\allowbreak\token[0.0003]| have|\allowbreak\token[0.0000]| \(|\allowbreak\token[0.0000]|7|\allowbreak\token[0.0000]|(|\allowbreak\token[0.0000]|2|\allowbreak\token[0.0000]|8|\allowbreak\token[0.0000]|)|\allowbreak\token[0.0000]| +|\allowbreak\token[0.0000]| |\allowbreak\token[0.0000]|5|\allowbreak\token[0.0000]|3|\allowbreak\token[0.0000]| =|\allowbreak\token[0.0000]| |\allowbreak\token[0.0002]|1|\allowbreak[...]
\tcblower
\token[0.7651]| this|\allowbreak\token[0.0532]| is|\allowbreak\token[0.2198]| not|\allowbreak\token[0.7140]| the|\allowbreak\token[0.1608]| correct|\allowbreak\token[1.2726]| solution|\allowbreak\token[0.1618]|.↩↩|\tighttracebreak\allowbreak\token[0.4113]|For|\allowbreak\token[0.0000]| \(|\allowbreak\token[0.0000]|x|\allowbreak\token[0.0000]| =|\allowbreak\token[0.0000]| |\allowbreak\token[0.0000]|2|\allowbreak\token[0.0141]|6|\allowbreak\token[0.0000]|\|\allowbreak\token[0.0000]|),|\allowbreak\token[0.0004]| we|\allowbreak\token[0.0034]| have|\allowbreak\token[0.0000]| \(|\allowbreak\token[0.0000]|7|\allowbreak\token[0.0000]|(|\allowbreak\token[0.0000]|2|\allowbreak\token[0.0003]|6|\allowbreak\token[0.0000]|)|\allowbreak\token[0.0000]| +|\allowbreak\token[0.0000]| |\allowbreak\token[0.0000]|5|\allowbreak\token[0.0000]|3|\allowbreak\token[0.0000]| =|\allowbreak\token[0.0000]| |\allowbreak\token[0.0000]|1|\allowbreak\token[0.0002]|8|\allowbreak\token[0.0000]|2|\allowbreak\token[0.0000]| +|\allowbreak\token[0.0000]| |\allowbreak\token[0.0000]|5|\allowbreak\token[0.0000]|3|\allowbreak\token[0.0000]| =|\allowbreak\token[0.0000]| |\allowbreak\token[0.0000]|2|\allowbreak\token[0.0001]|3|\allowbreak\token[0.0001]|5|\allowbreak\token[0.0000]|\|\allowbreak\token[0.0005]|),|\allowbreak\token[0.1752]| which|\allowbreak\token[0.0007]| is|\allowbreak\token[0.1514]| not|\allowbreak\token[0.0011]| divisible|\allowbreak\token[0.0000]| by|\allowbreak\token[0.0000]| |\allowbreak\token[0.0000]|8|\allowbreak\token[0.0061]|.|\allowbreak\token[0.2069]| For|\allowbreak\token[0.0000]| \(|\allowbreak\token[0.0000]|x|\allowbreak\token[0.0000]| =|\allowbreak\token[0.0000]| |\allowbreak\token[0.0000]|2|\allowbreak\token[0.0001]|7|\allowbreak\token[0.0000]|\|\allowbreak\token[0.0000]|),|\allowbreak\token[0.0000]| we|\allowbreak\token[0.0004]| have|\allowbreak\token[0.0000]| \(|\allowbreak\token[0.0006]|7|\allowbreak\token[0.0000]|(|\allowbreak\token[0.0000]|2|\allowbreak\token[0.0002]|7|\allowbreak\token[0.0000]|)|\allowbreak\token[0.0000]| +|\allowbreak\token[0.0001]| |\allowbreak\token[0.0000]|5|\allowbreak\token[0.0000]|3|\allowbreak\token[0.0000]| =|\allowbreak\token[0.0000]| |\allowbreak\token[0.0000]|1|\allowbreak\token[0.0003]|8|\allowbreak\token[0.0000]|9|\allowbreak\token[0.0000]| +|\allowbreak\token[0.0000]| |\allowbreak\token[0.0000]|5|\allowbreak\token[0.0000]|3|\allowbreak\token[0.0000]| =|\allowbreak\token[0.0000]| |\allowbreak\token[0.0000]|2|\allowbreak\token[0.0000]|4|\allowbreak\token[0.0001]|2|\allowbreak\token[0.0000]|\|\allowbreak\token[0.0001]|),|\allowbreak\token[0.6473]| and|\allowbreak\token[0.0016]| \(|\allowbreak\token[0.0000]|2|\allowbreak\token[0.0000]|4|\allowbreak\token[0.0000]|2|\allowbreak\token[0.0021]|\)|\allowbreak\token[0.0214]| is|\allowbreak\token[0.5993]| divisible|\allowbreak\token[0.0000]| by|\allowbreak\token[0.0000]| |\allowbreak\token[0.0000]|8|\allowbreak\token[0.0001]|,|\allowbreak\token[0.0056]| giving|\allowbreak\token[0.0000]| us|\allowbreak\token[0.0001]| \(|\allowbreak\token[0.0009]|y|\allowbreak\token[0.0000]| =|\allowbreak\token[0.0000]| |\allowbreak\token[0.0174]|3|\allowbreak\token[0.0005]|0|\allowbreak\token[0.0451]|.|\allowbreak\token[0.0007]|2|\allowbreak\token[0.0000]|5|\allowbreak\token[0.0000]|\|\allowbreak\token[0.0006]|),|\allowbreak\token[0.4851]| which|\allowbreak\token[0.0327]| is|\allowbreak\token[0.2796]| not|\allowbreak\token[0.1990]| an|\allowbreak\token[0.0000]| integer|\allowbreak\token[0.2609]|.↩↩|\tighttracebreak\allowbreak\token[0.6982]|Finally|\allowbreak\token[0.0000]|,|\allowbreak\token[0.0000]| for|\allowbreak\token[0.0000]| \(|\allowbreak\token[0.0000]|x|\allowbreak\token[0.0000]| =|\allowbreak\token[0.0000]| |\allowbreak\token[0.0001]|2|\allowbreak\token[0.0289]|8|\allowbreak\token[0.0000]|\|\allowbreak\token[0.0000]|),|\allowbreak\token[0.0000]| we|\allowbreak\token[0.0001]| have|\allowbreak\token[0.0000]| \(|\allowbreak\token[0.0000]|7|\allowbreak\token[0.0000]|(|\allowbreak\token[0.0000]|2|\allowbreak\token[0.0001]|8|\allowbreak\token[0.0000]|)|\allowbreak\token[0.0000]| +|\allowbreak\token[0.0000]| |\allowbreak\token[0.0000]|5|\allowbreak\token[0.0001]|3|\allowbreak\token[0.0000]| =|\allowbreak[...]
\end{twopartbox}
\begin{twopartbox}{Cut Index 754 | Acceptance Probability 0.9999}{Block 5 Step 1}{Proposal (Accepted)}
[...]\allowbreak\token[0.0006]|),|\allowbreak\token[0.5420]| which|\allowbreak\token[0.0327]| is|\allowbreak\token[0.2796]| not|\allowbreak\token[0.1985]| an|\allowbreak\token[0.0000]| integer|\allowbreak\token[0.2723]|.↩↩|\tighttracebreak\allowbreak\token[0.7013]|For|\allowbreak\token[0.0000]| \(|\allowbreak\token[0.0000]|x|\allowbreak\token[0.0000]| =|\allowbreak\token[0.0000]| |\allowbreak\token[0.0000]|2|\allowbreak\token[0.0032]|8|\allowbreak\token[0.0000]|\|\allowbreak\token[0.0000]|),|\allowbreak\token[0.0000]| we|\allowbreak\token[0.0003]| have|\allowbreak\token[0.0000]| \(|\allowbreak\token[0.0000]|7|\allowbreak\token[0.0000]|(|\allowbreak\token[0.0000]|2|\allowbreak\token[0.0000]|8|\allowbreak\token[0.0000]|)|\allowbreak\token[0.0000]| +|\allowbreak\token[0.0000]| |\allowbreak\token[0.0000]|5|\allowbreak\token[0.0000]|3|\allowbreak\token[0.0000]| =|\allowbreak\token[0.0000]| |\allowbreak\token[0.0002]|1|\allowbreak\token[0.0000]|9|\allowbreak\token[0.0002]|6|\allowbreak\token[0.0000]| +|\allowbreak\token[0.0000]| |\allowbreak\token[0.0000]|5|\allowbreak\token[0.0000]|3|\allowbreak\token[0.0000]| =|\allowbreak\token[0.0000]| |\allowbreak\token[0.0000]|2|\allowbreak\token[0.0001]|4|\allowbreak\token[0.0000]|9|\allowbreak\token[0.0001]|\|\allowbreak\token[0.0001]|),|\allowbreak\token[0.5830]| which|\allowbreak\token[0.0001]| is|\allowbreak\token[0.2041]| not|\allowbreak\token[0.0001]| divisible|\allowbreak\token[0.0000]| by|\allowbreak\token[0.0000]| |\allowbreak\token[0.0000]|8|\allowbreak\token[0.0000]|.|\allowbreak\token[0.2061]| For|\allowbreak\token[0.0000]| \(|\allowbreak\token[0.0000]|x|\allowbreak\token[0.0000]| =|\allowbreak\token[0.0000]| |\allowbreak\token[0.0001]|2|\allowbreak\token[0.0000]|9|\allowbreak\token[0.0000]|\|\allowbreak\token[0.0000]|),|\allowbreak\token[0.0000]| we|\allowbreak\token[0.0002]| have|\allowbreak\token[0.0000]| \(|\allowbreak\token[0.0000]|7|\allowbreak\token[0.0000]|(|\allowbreak\token[0.0000]|2|\allowbreak\token[0.0000]|9|\allowbreak\token[0.0000]|)|\allowbreak\token[0.0000]| +|\allowbreak\token[0.0000]| |\allowbreak\token[0.0000]|5|\allowbreak\token[0.0000]|3|\allowbreak\token[0.0000]| =|\allowbreak\token[0.0000]| |\allowbreak\token[0.0003]|2|\allowbreak\token[0.0020]|0|\allowbreak\token[0.0007]|3|\allowbreak\token[0.0004]| +|\allowbreak\token[0.0000]| |\allowbreak\token[0.0000]|5|\allowbreak\token[0.0000]|3|\allowbreak\token[0.0000]| =|\allowbreak\token[0.0000]| |\allowbreak\token[0.0000]|2|\allowbreak\token[0.0002]|5|\allowbreak\token[0.0002]|6|\allowbreak\token[0.0000]|\|\allowbreak\token[0.0002]|),|\allowbreak\token[0.0326]| and|\allowbreak\token[0.0012]| \(|\allowbreak\token[0.0000]|2|\allowbreak\token[0.0000]|5|\allowbreak\token[0.0000]|6|\allowbreak\token[0.0006]|\)|\allowbreak\token[0.0002]| is|\allowbreak\token[0.0019]| divisible|\allowbreak\token[0.0000]| by|\allowbreak\token[0.0001]| |\allowbreak\token[0.0000]|8|\allowbreak\token[0.0004]|,|\allowbreak\token[0.0022]| giving|\allowbreak\token[0.0000]| us|\allowbreak\token[0.0001]| \(|\allowbreak\token[0.0006]|y|\allowbreak\token[0.0000]| =|\allowbreak\token[0.0000]| |\allowbreak\token[0.0000]|3|\allowbreak\token[0.0001]|2|\allowbreak\token[0.0006]|\|\allowbreak\token[0.4737]|),|\allowbreak\token[0.0039]| which|\allowbreak\token[0.0015]| is|\allowbreak\token[0.0330]| an|\allowbreak\token[0.0000]| integer|\allowbreak\token[0.1338]|.↩↩|\tighttracebreak\allowbreak\token[0.0287]|Therefore|\allowbreak\token[0.0000]|,|\allowbreak\token[0.0022]| the|\allowbreak\token[0.0034]| smallest|\allowbreak\token[0.0015]| number|\allowbreak\token[0.0000]| of|\allowbreak\token[0.0001]| coins|\allowbreak\token[0.0022]| you|\allowbreak\token[0.0001]| could|\allowbreak\token[0.0000]| have|\allowbreak\token[0.0004]| had|\allowbreak\token[0.0047]| before|\allowbreak\token[0.0003]| finding|\allowbreak\token[0.0000]| the|\allowbreak\token[0.0435]| bag|\allowbreak\token[0.0005]| of|\allowbreak\token[0.0000]| |\allowbreak\token[0.0000]|5|\allowbreak\token[0.0000]|3|\allowbreak\token[0.0004]| coins|\allowbreak\token[0.0301]| is|\allowbreak\token[0.0571]| \(|\allowbreak\token[0.0110]|7|\allowbreak\token[0.6515]| \|\allowbreak\token[0.0000]|times|\allowbreak\token[0.0000]| |\allowbreak\token[0.0000]|2|\allowbreak\token[0.0006]|9|\allowbreak\cut{[CUT]}\allowbreak\token[0.3400]| =|\allowbreak\token[0.0014]| |\allowbreak\token[0.0004]|2|\allowbreak\token[0.0000]|0|\allowbreak\token[0.0001]|3|\allowbreak\token[0.3318]|\|\allowbreak\token[0.7773]|),|\allowbreak\token[0.8848]| and|\allowbreak\token[0.9087]| the|\allowbreak\token[1.1240]| total|\allowbreak\token[0.2455]| number|\allowbreak\token[0.0005]| of|\allowbreak\token[0.0045]| coins|\allowbreak\token[0.5433]| after|\allowbreak\token[0.2509]| redistribution|\allowbreak\token[0.0228]| is|\allowbreak\token[0.1114]| \(|\allowbreak\token[0.4257]|2|\allowbreak\token[0.1226]|5|\allowbreak\token[0.0000]|6|\allowbreak\token[0.0016]|\|\allowbreak\token[0.7682]|).|\allowbreak\token[0.9679]!<|endoftext|>!
\tcblower
\token[0.3400]| =|\allowbreak\token[0.0014]| |\allowbreak\token[0.0004]|2|\allowbreak\token[0.0000]|0|\allowbreak\token[0.0001]|3|\allowbreak\token[0.3318]|\|\allowbreak\token[0.7773]|),|\allowbreak\token[0.8848]| and|\allowbreak\token[0.9087]| the|\allowbreak\token[1.1240]| total|\allowbreak\token[0.2455]| number|\allowbreak\token[0.0005]| of|\allowbreak\token[0.0045]| coins|\allowbreak\token[0.5433]| after|\allowbreak\token[0.2509]| redistribution|\allowbreak\token[0.0228]| is|\allowbreak\token[0.1114]| \(|\allowbreak\token[0.4257]|2|\allowbreak\token[0.1226]|5|\allowbreak\token[0.0000]|6|\allowbreak\token[0.0016]|\|\allowbreak\token[0.7682]|).|\allowbreak\token[0.9679]!<|endoftext|>!
\end{twopartbox}
\begin{twopartbox}{Cut Index 246 | Acceptance Probability 1.0000}{Block 5 Step 2}{Proposal (Accepted)}
[...]\allowbreak\token[1.4839]| so|\allowbreak\token[0.7600]| \(|\allowbreak\token[0.0621]|7|\allowbreak\token[0.0000]|x|\allowbreak\token[0.0000]| +|\allowbreak\token[0.0000]| |\allowbreak\token[0.0000]|5|\allowbreak\token[0.0001]|3|\allowbreak\token[0.0127]| =|\allowbreak\token[0.0000]| |\allowbreak\token[0.0000]|8|\allowbreak\token[0.0003]|y|\allowbreak\token[0.0008]|\|\allowbreak\token[0.6026]|).↩↩|\tighttracebreak\allowbreak\token[0.4085]|Given|\allowbreak\token[0.0002]| that|\allowbreak\token[0.5546]| you|\allowbreak\token[0.0619]| have|\allowbreak\token[0.0010]| more|\allowbreak\token[0.0000]| than|\allowbreak\token[0.0001]| |\allowbreak\token[0.0000]|2|\allowbreak\token[0.0000]|0|\allowbreak\token[0.0000]|0|\allowbreak\token[0.0001]| coins|\allowbreak\token[1.0395]| in|\allowbreak\token[0.0074]| total|\allowbreak\token[0.5339]|,|\allowbreak\token[0.0983]| we|\allowbreak\token[1.1646]| have|\allowbreak\token[0.0967]| \(|\allowbreak\token[0.2702]|7|\allowbreak\token[0.0000]|x|\allowbreak\token[0.0003]| +|\allowbreak\token[0.0000]| |\allowbreak\token[0.0000]|5|\allowbreak\token[0.0001]|3|\allowbreak\token[0.0000]| >|\allowbreak\token[0.0000]| |\allowbreak\token[0.0000]|2|\allowbreak\token[0.0001]|0|\allowbreak\token[0.0002]|0|\allowbreak\token[0.0000]|\|\allowbreak\token[0.6786]|),|\allowbreak\token[0.0054]| which|\allowbreak\token[0.0017]| simpl|\allowbreak\token[0.0000]|ifies|\allowbreak\token[0.0000]| to|\allowbreak\token[0.0000]| \(|\allowbreak\token[0.0294]|7|\allowbreak\token[0.0000]|x|\allowbreak\token[0.0000]| >|\allowbreak\token[0.0000]| |\allowbreak\token[0.0000]|1|\allowbreak\token[0.0000]|4|\allowbreak\token[0.0001]|7|\allowbreak\token[0.0494]|\|\allowbreak\token[0.0446]|),|\allowbreak\token[0.3258]| and|\allowbreak\token[0.3191]| thus|\allowbreak\token[0.0002]| \(|\allowbreak\token[0.0000]|x|\allowbreak\token[0.0000]| >|\allowbreak\token[0.0001]| |\allowbreak\token[0.0001]|2|\allowbreak\token[0.0000]|1|\allowbreak\token[0.0004]|\|\allowbreak\token[0.1096]|).|\allowbreak\token[0.9808]| Since|\allowbreak\token[0.0127]| \(|\allowbreak\token[0.0035]|x|\allowbreak\token[0.0002]|\)|\allowbreak\token[0.0945]| must|\allowbreak\token[0.0000]| be|\allowbreak\token[0.0294]| an|\allowbreak\token[0.0000]| integer|\allowbreak\token[0.3227]|,|\allowbreak\token[0.0030]| the|\allowbreak\token[0.0073]| smallest|\allowbreak\token[0.9362]| possible|\allowbreak\token[0.0018]| value|\allowbreak\token[0.0000]| for|\allowbreak\token[0.0000]| \(|\allowbreak\token[0.0004]|x|\allowbreak\token[0.0000]|\)|\allowbreak\token[0.4213]| is|\allowbreak\token[0.0042]| |\allowbreak\token[0.0001]|2|\allowbreak\token[0.0016]|2|\allowbreak\token[0.7191]|.↩↩|\tighttracebreak\allowbreak\token[0.9831]|Sub|\allowbreak\token[0.0000]|stit|\allowbreak\token[0.0005]|uting|\allowbreak\token[0.0000]| \(|\allowbreak\token[0.0000]|x|\allowbreak\token[0.0000]| =|\allowbreak\token[0.0000]| |\allowbreak\token[0.0000]|2|\allowbreak\token[0.0000]|2|\allowbreak\token[0.0001]|\)|\allowbreak\token[0.0010]| into|\allowbreak\token[0.7102]| the|\allowbreak\token[0.0135]| equation|\allowbreak\token[0.1743]| \(|\allowbreak\token[0.0002]|7|\allowbreak\token[0.0000]|x|\allowbreak\token[0.0001]| +|\allowbreak\token[0.0000]| |\allowbreak\token[0.0000]|5|\allowbreak\token[0.0001]|3|\allowbreak\token[0.1096]| =|\allowbreak\token[0.0000]| |\allowbreak\token[0.0002]|8|\allowbreak\token[0.0000]|y|\allowbreak\token[0.0017]|\|\allowbreak\token[0.0000]|),|\allowbreak\token[0.0000]| we|\allowbreak\token[0.0104]| get|\allowbreak\token[0.0670]| \(|\allowbreak\token[0.0739]|7|\allowbreak\token[0.0337]|(|\allowbreak\token[0.0000]|2|\allowbreak\token[0.0000]|2|\allowbreak\token[0.0000]|)|\allowbreak\token[0.0000]| +|\allowbreak\token[0.0000]| |\allowbreak\token[0.0000]|5|\allowbreak\token[0.0000]|3|\allowbreak\token[0.0000]| =|\allowbreak\token[0.0000]| |\allowbreak\token[0.5392]|1|\allowbreak\token[0.7184]|5|\allowbreak\token[0.0361]|4|\allowbreak\token[0.0000]| +|\allowbreak\token[0.0000]| |\allowbreak\token[0.0000]|5|\allowbreak\token[0.0000]|3|\allowbreak\token[0.0000]| =|\allowbreak\token[0.0000]| |\allowbreak\token[0.0020]|2|\allowbreak\token[0.0000]|0|\allowbreak\token[0.0000]|7|\allowbreak\token[0.1748]|\|\allowbreak\token[0.2686]|).|\allowbreak\cut{[CUT]}\allowbreak\token[1.2096]| This|\allowbreak\token[1.2500]| means|\allowbreak\token[0.7949]| \(|\allowbreak\token[0.6895]|8|\allowbreak\token[0.0000]|y|\allowbreak\token[0.0000]| =|\allowbreak\token[0.0000]| |\allowbreak\token[0.0000]|2|\allowbreak\token[0.0000]|0|\allowbreak\token[0.0001]|7|\allowbreak\token[0.0007]|\|\allowbreak\token[0.0032]|),|\allowbreak\token[0.6428]| but|\allowbreak\token[0.6889]| since|\allowbreak\token[0.2738]| |\allowbreak\token[0.0000]|2|\allowbreak\token[0.0000]|0|\allowbreak\token[0.0000]|7|\allowbreak\token[0.0005]| is|\allowbreak\token[0.0000]| not|\allowbreak\token[0.0011]| divisible|\allowbreak\token[0.0000]| by|\allowbreak\token[0.0000]| |\allowbreak\token[0.0000]|8|\allowbreak\token[0.0001]|,|\allowbreak\token[0.6828]| \(|\allowbreak\token[0.0025]|x|\allowbreak\token[0.0082]| =|\allowbreak\token[0.0000]| |\allowbreak\token[0.0000]|2|\allowbreak\token[0.0000]|2|\allowbreak\token[0.0000]|\)|\allowbreak\token[0.4211]| is|\allowbreak\token[0.0013]| not|\allowbreak\token[0.4197]| the|\allowbreak\token[0.2724]| correct|\allowbreak\token[0.8378]| value|\allowbreak\token[0.7992]|.↩↩|\tighttracebreak\allowbreak\token[1.3209]|We|\allowbreak\token[0.8921]| need|\allowbreak\token[0.0063]| to|\allowbreak\token[0.0001]| find|\allowbreak\token[0.0004]| the|\allowbreak\token[0.0016]| smallest|\allowbreak\token[0.0233]| \(|\allowbreak\token[0.0001]|x|\allowbreak\token[0.0993]|\)|\allowbreak\token[0.4080]| such|\allowbreak\token[0.0000]| that|\allowbreak\token[0.0002]| \(|\allowbreak\token[0.0007]|7|\allowbreak\token[0.0000]|x|\allowbreak\token[0.0000]| +|\allowbreak\token[0.0000]| |\allowbreak\token[0.0000]|5|\allowbreak\token[0.0001]|3|\allowbreak\token[0.0001]|\)|\allowbreak\token[0.0000]| is|\allowbreak\token[0.0414]| divisible|\allowbreak\token[0.0000]| by|\allowbreak\token[0.0000]| |\allowbreak\token[0.0000]|8|\allowbreak\token[0.1746]| and|\allowbreak\token[0.3258]| greater|\allowbreak\token[0.0000]| than|\allowbreak\token[0.0005]| |\allowbreak\token[0.0024]|2|\allowbreak\token[0.0002]|0|\allowbreak\token[0.0164]|0|\allowbreak\token[0.0001]|.|\allowbreak\token[0.6441]| Let|\allowbreak\token[0.0000]|'s|\allowbreak\token[0.3903]| increment|\allowbreak\token[0.0402]| \(|\allowbreak\token[0.0000]|x|\allowbreak\token[0.0000]|\)|\allowbreak\token[0.5173]| by|\allowbreak\token[0.0051]| |\allowbreak\token[0.0239]|1|\allowbreak\token[1.0275]| and|\allowbreak\token[0.0935]| check|\allowbreak\token[1.4560]| again|\allowbreak\token[0.4429]|.|\allowbreak\token[0.0968]| For|\allowbreak\token[0.0000]| \(|\allowbreak\token[0.0000]|x|\allowbreak\token[0.0000]| =|\allowbreak\token[0.0000]| |\allowbreak\token[0.0002]|2|\allowbreak\token[0.0020]|3|\allowbreak\token[0.0000]|\|\allowbreak\token[0.0000]|),|\allowbreak\token[0.0487]| we|\allowbreak\token[0.0103]| have|\allowbreak\token[0.0001]| \(|\allowbreak\token[0.0001]|7|\allowbreak\token[0.0007]|(|\allowbreak\token[0.0000]|2|\allowbreak\token[0.0001]|3|\allowbreak\token[0.0000]|)|\allowbreak\token[0.0001]| +|\allowbreak\token[0.0000]| |\allowbreak\token[0.0000]|5|\allowbreak\token[0.0000]|3|\allowbreak\token[0.0000]| =|\allowbreak\token[0.0000]| |\allowbreak\token[0.1324]|1|\allowbreak\token[0.0238]|6|\allowbreak\token[0.0022]|1|\allowbreak\token[0.0001]| +|\allowbreak\token[0.0000]| |\allowbreak\token[0.0000]|5|\allowbreak\token[0.0000]|3|\allowbreak\token[0.0000]| =|\allowbreak\token[0.0000]| |\allowbreak\token[0.0001]|2|\allowbreak\token[0.0001]|1|\allowbreak\token[0.0001]|4|\allowbreak\token[0.0000]|\|\allowbreak\token[0.1454]|),|\allowbreak\token[0.6115]| which|\allowbreak\token[0.0005]| is|\allowbreak\token[1.3057]| still|\allowbreak\token[0.0054]| not|\allowbreak\token[0.0001]| divisible|\allowbreak\token[0.0000]| by|\allowbreak\token[0.0000]| |\allowbreak\token[0.0000]|8|\allowbreak\token[0.6630]|.↩↩|\tighttracebreak\allowbreak\token[0.3222]|Contin|\allowbreak\token[0.0000]|uing|\allowbreak\token[0.0000]| this|\allowbreak\token[0.0000]| process|\allowbreak\token[0.0003]|,|\allowbreak\token[0.5751]| for|\allowbreak\token[0.0000]| \(|\allowbreak\token[0.0000]|x|\allowbreak\token[0.0000]| =|\allowbreak\token[0.0000]| |\allowbreak\token[0.0004]|2|\allowbreak\token[0.2636]|4|\allowbreak\token[0.0000]|\|\allowbreak\token[0.0000]|),|\allowbreak\token[0.0001]| we|\allowbreak[...]
\tcblower
\token[1.2099]| This|\allowbreak\token[1.2497]| means|\allowbreak\token[0.8052]| \(|\allowbreak\token[0.7138]|8|\allowbreak\token[0.0000]|y|\allowbreak\token[0.0000]| =|\allowbreak\token[0.0000]| |\allowbreak\token[0.0000]|2|\allowbreak\token[0.0000]|0|\allowbreak\token[0.0002]|7|\allowbreak\token[0.0007]|\|\allowbreak\token[0.0034]|),|\allowbreak\token[0.6824]| but|\allowbreak\token[0.6512]| since|\allowbreak\token[0.2515]| |\allowbreak\token[0.0000]|2|\allowbreak\token[0.0000]|0|\allowbreak\token[0.0000]|7|\allowbreak\token[0.0004]| is|\allowbreak\token[0.0000]| not|\allowbreak\token[0.0013]| divisible|\allowbreak\token[0.0000]| by|\allowbreak\token[0.0000]| |\allowbreak\token[0.0000]|8|\allowbreak\token[0.0001]|,|\allowbreak\token[0.6841]| \(|\allowbreak\token[0.0022]|x|\allowbreak\token[0.0073]| =|\allowbreak\token[0.0000]| |\allowbreak\token[0.0000]|2|\allowbreak\token[0.0000]|2|\allowbreak\token[0.0000]|\)|\allowbreak\token[0.4213]| is|\allowbreak\token[0.0015]| not|\allowbreak\token[0.4197]| the|\allowbreak\token[0.2723]| correct|\allowbreak\token[0.8406]| value|\allowbreak\token[0.8114]|.↩↩|\tighttracebreak\allowbreak\token[1.3165]|We|\allowbreak\token[0.8518]| need|\allowbreak\token[0.0061]| to|\allowbreak\token[0.0001]| find|\allowbreak\token[0.0004]| the|\allowbreak\token[0.0016]| smallest|\allowbreak\token[0.0209]| \(|\allowbreak\token[0.0001]|x|\allowbreak\token[0.0901]|\)|\allowbreak\token[0.4363]| such|\allowbreak\token[0.0000]| that|\allowbreak\token[0.0002]| \(|\allowbreak\token[0.0006]|7|\allowbreak\token[0.0000]|x|\allowbreak\token[0.0000]| +|\allowbreak\token[0.0000]| |\allowbreak\token[0.0000]|5|\allowbreak\token[0.0000]|3|\allowbreak\token[0.0001]|\)|\allowbreak\token[0.0000]| is|\allowbreak\token[0.0457]| divisible|\allowbreak\token[0.0000]| by|\allowbreak\token[0.0000]| |\allowbreak\token[0.0000]|8|\allowbreak\token[0.1910]| and|\allowbreak\token[0.3383]| greater|\allowbreak\token[0.0000]| than|\allowbreak\token[0.0006]| |\allowbreak\token[0.0027]|2|\allowbreak\token[0.0002]|0|\allowbreak\token[0.0183]|0|\allowbreak\token[0.0001]|.|\allowbreak\token[0.6860]| Let|\allowbreak\token[0.0000]|'s|\allowbreak\token[0.3903]| increment|\allowbreak\token[0.0402]| \(|\allowbreak\token[0.0000]|x|\allowbreak\token[0.0000]|\)|\allowbreak\token[0.5674]| by|\allowbreak\token[0.0057]| |\allowbreak\token[0.0245]|1|\allowbreak\token[0.9522]| and|\allowbreak\token[0.0936]| check|\allowbreak\token[1.4251]| again|\allowbreak\token[0.4584]|.|\allowbreak\token[0.0968]| For|\allowbreak\token[0.0000]| \(|\allowbreak\token[0.0000]|x|\allowbreak\token[0.0000]| =|\allowbreak\token[0.0000]| |\allowbreak\token[0.0002]|2|\allowbreak\token[0.0020]|3|\allowbreak\token[0.0000]|\|\allowbreak\token[0.0000]|),|\allowbreak\token[0.0443]| we|\allowbreak\token[0.0092]| have|\allowbreak\token[0.0002]| \(|\allowbreak\token[0.0001]|7|\allowbreak\token[0.0008]|(|\allowbreak\token[0.0000]|2|\allowbreak\token[0.0001]|3|\allowbreak\token[0.0000]|)|\allowbreak\token[0.0001]| +|\allowbreak\token[0.0000]| |\allowbreak\token[0.0000]|5|\allowbreak\token[0.0000]|3|\allowbreak\token[0.0000]| =|\allowbreak\token[0.0000]| |\allowbreak\token[0.1453]|1|\allowbreak\token[0.0231]|6|\allowbreak\token[0.0021]|1|\allowbreak\token[0.0001]| +|\allowbreak\token[0.0000]| |\allowbreak\token[0.0000]|5|\allowbreak\token[0.0000]|3|\allowbreak\token[0.0000]| =|\allowbreak\token[0.0000]| |\allowbreak\token[0.0000]|2|\allowbreak\token[0.0001]|1|\allowbreak\token[0.0001]|4|\allowbreak\token[0.0000]|\|\allowbreak\token[0.1454]|),|\allowbreak\token[0.6115]| which|\allowbreak\token[0.0003]| is|\allowbreak\token[1.3081]| divisible|\allowbreak\token[0.0000]| by|\allowbreak\token[0.0000]| |\allowbreak\token[0.0000]|8|\allowbreak\token[1.2192]| (|\allowbreak\token[0.2863]|since|\allowbreak\token[0.0031]| \(|\allowbreak\token[0.1771]|2|\allowbreak\token[0.0001]|1|\allowbreak\token[0.0008]|4|\allowbreak\token[0.4694]| =|\allowbreak\token[0.0000]| |\allowbreak\token[0.0075]|8|\allowbreak\token[0.0016]| \|\allowbreak\token[0.0000]|times|\allowbreak\token[0.0000]| |\allowbreak\token[0.0001]|2|\allowbreak\token[0.0084]|6|\allowbreak\token[0.3863]|.|\allowbreak\token[0.0001]|7|\allowbreak\token[0.0001]|5|\allowbreak\token[0.0004]|\|\allowbreak[...]
\end{twopartbox}
\begin{twopartbox}{Cut Index 610 | Acceptance Probability 1.0000}{Block 5 Step 3}{Proposal (Accepted)}
[...]\allowbreak\token[0.0000]| |\allowbreak\token[0.0000]|5|\allowbreak\token[0.0000]|3|\allowbreak\token[0.0000]| =|\allowbreak\token[0.0000]| |\allowbreak\token[0.0001]|1|\allowbreak\token[0.0000]|7|\allowbreak\token[0.0000]|5|\allowbreak\token[0.0000]| +|\allowbreak\token[0.0000]| |\allowbreak\token[0.0000]|5|\allowbreak\token[0.0000]|3|\allowbreak\token[0.0000]| =|\allowbreak\token[0.0000]| |\allowbreak\token[0.0000]|2|\allowbreak\token[0.0001]|2|\allowbreak\token[0.0001]|8|\allowbreak\token[0.0001]|\|\allowbreak\token[0.0011]|),|\allowbreak\token[0.6913]| and|\allowbreak\token[0.2315]| \(|\allowbreak\token[0.0003]|2|\allowbreak\token[0.0000]|2|\allowbreak\token[0.0000]|8|\allowbreak\token[0.4942]|\)|\allowbreak\token[0.0214]| is|\allowbreak\token[0.0447]| divisible|\allowbreak\token[0.0000]| by|\allowbreak\token[0.4193]| |\allowbreak\token[0.0000]|8|\allowbreak\token[0.5591]| (|\allowbreak\token[0.0013]|since|\allowbreak\token[0.0000]| \(|\allowbreak\token[0.0005]|2|\allowbreak\token[0.0000]|2|\allowbreak\token[0.0000]|8|\allowbreak\token[0.0004]| =|\allowbreak\token[0.0000]| |\allowbreak\token[0.0001]|8|\allowbreak\token[0.0000]| \|\allowbreak\token[0.0000]|times|\allowbreak\token[0.0000]| |\allowbreak\token[0.0001]|2|\allowbreak\token[0.0000]|8|\allowbreak\token[0.0035]|.|\allowbreak\token[0.0001]|5|\allowbreak\token[0.0005]|\|\allowbreak\token[0.1592]|),|\allowbreak\token[0.2843]| but|\allowbreak\token[0.3873]| again|\allowbreak\token[0.0001]|,|\allowbreak\token[0.0100]| we|\allowbreak\token[0.0098]| need|\allowbreak\token[0.0823]| an|\allowbreak\token[0.0741]| integer|\allowbreak\token[0.0853]| solution|\allowbreak\token[0.6242]|).↩↩|\tighttracebreak\allowbreak\token[1.2030]|Contin|\allowbreak\token[0.0000]|uing|\allowbreak\token[0.0000]| this|\allowbreak\token[0.0002]| process|\allowbreak\token[0.0000]|,|\allowbreak\token[0.4780]| for|\allowbreak\token[0.0000]| \(|\allowbreak\token[0.0000]|x|\allowbreak\token[0.0000]| =|\allowbreak\token[0.0000]| |\allowbreak\token[0.0066]|2|\allowbreak\token[0.6199]|6|\allowbreak\token[0.0000]|\|\allowbreak\token[0.0000]|),|\allowbreak\token[0.0002]| we|\allowbreak\token[0.0803]| have|\allowbreak\token[0.0000]| \(|\allowbreak\token[0.0000]|7|\allowbreak\token[0.0000]|(|\allowbreak\token[0.0000]|2|\allowbreak\token[0.0001]|6|\allowbreak\token[0.0000]|)|\allowbreak\token[0.0000]| +|\allowbreak\token[0.0000]| |\allowbreak\token[0.0000]|5|\allowbreak\token[0.0001]|3|\allowbreak\token[0.0000]| =|\allowbreak\token[0.0000]| |\allowbreak\token[0.0001]|1|\allowbreak\token[0.0005]|8|\allowbreak\token[0.0001]|2|\allowbreak\token[0.0000]| +|\allowbreak\token[0.0000]| |\allowbreak\token[0.0000]|5|\allowbreak\token[0.0000]|3|\allowbreak\token[0.0000]| =|\allowbreak\token[0.0000]| |\allowbreak\token[0.0001]|2|\allowbreak\token[0.0002]|3|\allowbreak\token[0.0002]|5|\allowbreak\token[0.0000]|\|\allowbreak\token[0.0006]|),|\allowbreak\token[0.0618]| which|\allowbreak\token[0.0001]| is|\allowbreak\token[0.0562]| not|\allowbreak\token[0.0011]| divisible|\allowbreak\token[0.0000]| by|\allowbreak\token[0.0000]| |\allowbreak\token[0.0000]|8|\allowbreak\token[0.0001]|.|\allowbreak\token[0.1133]| For|\allowbreak\token[0.0000]| \(|\allowbreak\token[0.0000]|x|\allowbreak\token[0.0000]| =|\allowbreak\token[0.0000]| |\allowbreak\token[0.0000]|2|\allowbreak\token[0.0004]|7|\allowbreak\token[0.0000]|\|\allowbreak\token[0.0000]|),|\allowbreak\token[0.0000]| we|\allowbreak\token[0.0018]| have|\allowbreak\token[0.0000]| \(|\allowbreak\token[0.0011]|7|\allowbreak\token[0.0000]|(|\allowbreak\token[0.0000]|2|\allowbreak\token[0.0002]|7|\allowbreak\token[0.0000]|)|\allowbreak\token[0.0000]| +|\allowbreak\token[0.0001]| |\allowbreak\token[0.0000]|5|\allowbreak\token[0.0000]|3|\allowbreak\token[0.0000]| =|\allowbreak\token[0.0000]| |\allowbreak\token[0.0000]|1|\allowbreak\token[0.0003]|8|\allowbreak\token[0.0000]|9|\allowbreak\token[0.0000]| +|\allowbreak\token[0.0000]| |\allowbreak\token[0.0000]|5|\allowbreak\token[0.0000]|3|\allowbreak\token[0.0000]| =|\allowbreak\token[0.0000]| |\allowbreak\token[0.0000]|2|\allowbreak\token[0.0001]|4|\allowbreak\token[0.0001]|2|\allowbreak\token[0.0000]|\|\allowbreak\token[0.0001]|),|\allowbreak\cut{[CUT]}\allowbreak\token[0.5827]| and|\allowbreak\token[0.0051]| \(|\allowbreak\token[0.0000]|2|\allowbreak\token[0.0000]|4|\allowbreak\token[0.0000]|2|\allowbreak\token[0.0406]|\)|\allowbreak\token[0.0091]| is|\allowbreak\token[0.5844]| divisible|\allowbreak\token[0.0000]| by|\allowbreak\token[0.0000]| |\allowbreak\token[0.0001]|8|\allowbreak\token[0.0140]| (|\allowbreak\token[0.0001]|since|\allowbreak\token[0.0000]| \(|\allowbreak\token[0.0000]|2|\allowbreak\token[0.0000]|4|\allowbreak\token[0.0000]|2|\allowbreak\token[0.0002]| =|\allowbreak\token[0.0000]| |\allowbreak\token[0.0003]|8|\allowbreak\token[0.0000]| \|\allowbreak\token[0.0000]|times|\allowbreak\token[0.0000]| |\allowbreak\token[0.0113]|3|\allowbreak\token[0.0001]|0|\allowbreak\token[0.0035]|.|\allowbreak\token[0.0000]|2|\allowbreak\token[0.0000]|5|\allowbreak\token[0.0001]|\|\allowbreak\token[0.1751]|),|\allowbreak\token[0.0175]| but|\allowbreak\token[0.0914]| we|\allowbreak\token[0.2248]| need|\allowbreak\token[0.1999]| an|\allowbreak\token[0.0174]| integer|\allowbreak\token[0.0068]| solution|\allowbreak\token[0.1113]|).↩↩|\tighttracebreak\allowbreak\token[0.7602]|Finally|\allowbreak\token[0.0000]|,|\allowbreak\token[0.0000]| for|\allowbreak\token[0.0000]| \(|\allowbreak\token[0.0000]|x|\allowbreak\token[0.0000]| =|\allowbreak\token[0.0000]| |\allowbreak\token[0.0011]|2|\allowbreak\token[0.1873]|8|\allowbreak\token[0.0000]|\|\allowbreak\token[0.0000]|),|\allowbreak\token[0.0000]| we|\allowbreak\token[0.0002]| have|\allowbreak\token[0.0000]| \(|\allowbreak\token[0.0000]|7|\allowbreak\token[0.0000]|(|\allowbreak\token[0.0000]|2|\allowbreak\token[0.0000]|8|\allowbreak\token[0.0000]|)|\allowbreak\token[0.0000]| +|\allowbreak\token[0.0000]| |\allowbreak\token[0.0000]|5|\allowbreak\token[0.0001]|3|\allowbreak\token[0.0000]| =|\allowbreak\token[0.0000]| |\allowbreak\token[0.0002]|1|\allowbreak\token[0.0001]|9|\allowbreak\token[0.0002]|6|\allowbreak\token[0.0000]| +|\allowbreak\token[0.0000]| |\allowbreak\token[0.0000]|5|\allowbreak\token[0.0000]|3|\allowbreak\token[0.0000]| =|\allowbreak\token[0.0000]| |\allowbreak\token[0.0000]|2|\allowbreak\token[0.0003]|4|\allowbreak\token[0.0000]|9|\allowbreak\token[0.0001]|\|\allowbreak\token[0.0006]|),|\allowbreak\token[0.6899]| and|\allowbreak\token[0.0019]| \(|\allowbreak\token[0.0000]|2|\allowbreak\token[0.0001]|4|\allowbreak\token[0.0001]|9|\allowbreak\token[0.0106]|\)|\allowbreak\token[0.0059]| is|\allowbreak\token[0.7276]| divisible|\allowbreak\token[0.0000]| by|\allowbreak\token[0.0000]| |\allowbreak\token[0.0006]|8|\allowbreak\token[0.0144]| (|\allowbreak\token[0.0002]|since|\allowbreak\token[0.0000]| \(|\allowbreak\token[0.0000]|2|\allowbreak\token[0.0000]|4|\allowbreak\token[0.0000]|9|\allowbreak\token[0.0001]| =|\allowbreak\token[0.0000]| |\allowbreak\token[0.0008]|8|\allowbreak\token[0.0000]| \|\allowbreak\token[0.0000]|times|\allowbreak\token[0.0000]| |\allowbreak\token[0.0000]|3|\allowbreak\token[0.0001]|1|\allowbreak\token[0.0824]|.|\allowbreak\token[0.0000]|1|\allowbreak\token[0.0000]|2|\allowbreak\token[0.0000]|5|\allowbreak\token[0.0000]|\|\allowbreak\token[0.1152]|),|\allowbreak\token[0.0188]| but|\allowbreak\token[0.5635]| we|\allowbreak\token[0.0610]| need|\allowbreak\token[0.0722]| an|\allowbreak\token[0.0059]| integer|\allowbreak\token[0.0005]| solution|\allowbreak\token[0.6451]|).|\allowbreak\token[0.7464]| However|\allowbreak\token[0.0000]|,|\allowbreak\token[1.4172]| for|\allowbreak\token[0.0023]| \(|\allowbreak\token[0.0000]|x|\allowbreak\token[0.0000]| =|\allowbreak\token[0.0000]| |\allowbreak\token[0.0327]|2|\allowbreak\token[0.1169]|9|\allowbreak\token[0.0000]|\|\allowbreak\token[0.0000]|),|\allowbreak\token[0.0003]| we|\allowbreak\token[0.0105]| have|\allowbreak\token[0.0000]| \(|\allowbreak\token[0.0000]|7|\allowbreak\token[0.0000]|(|\allowbreak\token[0.0000]|2|\allowbreak\token[0.0000]|9|\allowbreak\token[0.0000]|)|\allowbreak\token[0.0000]| +|\allowbreak\token[0.0000]| |\allowbreak\token[0.0000]|5|\allowbreak\token[0.0000]|3|\allowbreak\token[0.0000]| =|\allowbreak\token[0.0000]| |\allowbreak\token[0.0003]|2|\allowbreak\token[0.0024]|0|\allowbreak\token[0.0012]|3|\allowbreak\token[0.0033]| +|\allowbreak\token[0.0000]| |\allowbreak[...]
\tcblower
\token[0.6064]| and|\allowbreak\token[0.0059]| \(|\allowbreak\token[0.0000]|2|\allowbreak\token[0.0000]|4|\allowbreak\token[0.0000]|2|\allowbreak\token[0.0450]|\)|\allowbreak\token[0.0091]| is|\allowbreak\token[0.6375]| divisible|\allowbreak\token[0.0000]| by|\allowbreak\token[0.0000]| |\allowbreak\token[0.0001]|8|\allowbreak\token[0.0117]| (|\allowbreak\token[0.0001]|since|\allowbreak\token[0.0000]| \(|\allowbreak\token[0.0000]|2|\allowbreak\token[0.0000]|4|\allowbreak\token[0.0000]|2|\allowbreak\token[0.0002]| =|\allowbreak\token[0.0000]| |\allowbreak\token[0.0003]|8|\allowbreak\token[0.0000]| \|\allowbreak\token[0.0000]|times|\allowbreak\token[0.0000]| |\allowbreak\token[0.0113]|3|\allowbreak\token[0.0001]|0|\allowbreak\token[0.0040]|.|\allowbreak\token[0.0000]|2|\allowbreak\token[0.0000]|5|\allowbreak\token[0.0001]|\|\allowbreak\token[0.1751]|),|\allowbreak\token[0.0176]| but|\allowbreak\token[0.0922]| we|\allowbreak\token[0.2025]| need|\allowbreak\token[0.2367]| an|\allowbreak\token[0.0215]| integer|\allowbreak\token[0.0068]| solution|\allowbreak\token[0.1017]|).↩↩|\tighttracebreak\allowbreak\token[0.7436]|Finally|\allowbreak\token[0.0000]|,|\allowbreak\token[0.0000]| for|\allowbreak\token[0.0000]| \(|\allowbreak\token[0.0000]|x|\allowbreak\token[0.0000]| =|\allowbreak\token[0.0000]| |\allowbreak\token[0.0011]|2|\allowbreak\token[0.1874]|8|\allowbreak\token[0.0000]|\|\allowbreak\token[0.0000]|),|\allowbreak\token[0.0000]| we|\allowbreak\token[0.0002]| have|\allowbreak\token[0.0000]| \(|\allowbreak\token[0.0000]|7|\allowbreak\token[0.0000]|(|\allowbreak\token[0.0000]|2|\allowbreak\token[0.0000]|8|\allowbreak\token[0.0000]|)|\allowbreak\token[0.0000]| +|\allowbreak\token[0.0000]| |\allowbreak\token[0.0000]|5|\allowbreak\token[0.0001]|3|\allowbreak\token[0.0000]| =|\allowbreak\token[0.0000]| |\allowbreak\token[0.0002]|1|\allowbreak\token[0.0001]|9|\allowbreak\token[0.0002]|6|\allowbreak\token[0.0000]| +|\allowbreak\token[0.0000]| |\allowbreak\token[0.0000]|5|\allowbreak\token[0.0000]|3|\allowbreak\token[0.0000]| =|\allowbreak\token[0.0000]| |\allowbreak\token[0.0000]|2|\allowbreak\token[0.0003]|4|\allowbreak\token[0.0000]|9|\allowbreak\token[0.0001]|\|\allowbreak\token[0.0007]|),|\allowbreak\token[0.6801]| which|\allowbreak\token[0.0001]| is|\allowbreak\token[0.7309]| not|\allowbreak\token[0.0030]| divisible|\allowbreak\token[0.0000]| by|\allowbreak\token[0.0000]| |\allowbreak\token[0.0000]|8|\allowbreak\token[0.0004]|.|\allowbreak\token[0.2842]| For|\allowbreak\token[0.0000]| \(|\allowbreak\token[0.0000]|x|\allowbreak\token[0.0000]| =|\allowbreak\token[0.0000]| |\allowbreak\token[0.0007]|2|\allowbreak\token[0.0001]|9|\allowbreak\token[0.0000]|\|\allowbreak\token[0.0000]|),|\allowbreak\token[0.0001]| we|\allowbreak\token[0.0002]| have|\allowbreak\token[0.0000]| \(|\allowbreak\token[0.0001]|7|\allowbreak\token[0.0000]|(|\allowbreak\token[0.0000]|2|\allowbreak\token[0.0000]|9|\allowbreak\token[0.0000]|)|\allowbreak\token[0.0000]| +|\allowbreak\token[0.0001]| |\allowbreak\token[0.0000]|5|\allowbreak\token[0.0000]|3|\allowbreak\token[0.0000]| =|\allowbreak\token[0.0000]| |\allowbreak\token[0.0006]|2|\allowbreak\token[0.0012]|0|\allowbreak\token[0.0006]|3|\allowbreak\token[0.0009]| +|\allowbreak\token[0.0000]| |\allowbreak\token[0.0000]|5|\allowbreak\token[0.0000]|3|\allowbreak\token[0.0000]| =|\allowbreak\token[0.0000]| |\allowbreak\token[0.0000]|2|\allowbreak\token[0.0003]|5|\allowbreak\token[0.0004]|6|\allowbreak\token[0.0000]|\|\allowbreak\token[0.0003]|),|\allowbreak\token[0.0113]| and|\allowbreak\token[0.0010]| \(|\allowbreak\token[0.0000]|2|\allowbreak\token[0.0000]|5|\allowbreak\token[0.0000]|6|\allowbreak\token[0.0327]|\)|\allowbreak\token[0.0000]| is|\allowbreak\token[0.0120]| divisible|\allowbreak\token[0.0000]| by|\allowbreak\token[0.0001]| |\allowbreak\token[0.0000]|8|\allowbreak\token[0.0169]| (|\allowbreak\token[0.0001]|since|\allowbreak\token[0.0000]| \(|\allowbreak\token[0.0000]|2|\allowbreak\token[0.0000]|5|\allowbreak\token[0.0000]|6|\allowbreak\token[0.0001]| =|\allowbreak\token[0.0000]| |\allowbreak\token[0.0005]|8|\allowbreak\token[0.0002]| \|\allowbreak\token[0.0000]|times|\allowbreak\token[0.0000]| |\allowbreak[...]
\end{twopartbox}
\begin{twopartbox}{Cut Index 203 | Acceptance Probability 0.0040}{Block 6 Step 1}{Proposal (Rejected)}
[...]\allowbreak\token[0.1040]| bags|\allowbreak\token[0.7756]| with|\allowbreak\token[0.1464]| an|\allowbreak\token[0.0016]| equal|\allowbreak\token[0.0013]| number|\allowbreak\token[0.0144]| of|\allowbreak\token[0.0000]| coins|\allowbreak\token[1.1025]|,|\allowbreak\token[0.9695]| which|\allowbreak\token[0.5044]| we|\allowbreak\token[0.0822]|'ll|\allowbreak\token[0.0494]| call|\allowbreak\token[0.0000]| \(|\allowbreak\token[0.0062]|y|\allowbreak\token[0.0002]|\|\allowbreak\token[0.0502]|).|\allowbreak\token[0.8690]| This|\allowbreak\token[0.1096]| means|\allowbreak\token[0.1777]| the|\allowbreak\token[0.0753]| total|\allowbreak\token[0.0001]| number|\allowbreak\token[0.0000]| of|\allowbreak\token[0.0000]| coins|\allowbreak\token[0.7641]| is|\allowbreak\token[0.6064]| now|\allowbreak\token[0.0227]| \(|\allowbreak\token[0.5764]|7|\allowbreak\token[0.0003]|x|\allowbreak\token[0.0000]| +|\allowbreak\token[0.0000]| |\allowbreak\token[0.0000]|5|\allowbreak\token[0.0000]|3|\allowbreak\token[0.6698]|\|\allowbreak\token[0.0232]|),|\allowbreak\token[0.0038]| and|\allowbreak\token[0.0894]| this|\allowbreak\token[0.1714]| total|\allowbreak\token[1.0999]| is|\allowbreak\token[1.1997]| divisible|\allowbreak\token[0.0057]| by|\allowbreak\token[0.0004]| |\allowbreak\token[0.0001]|8|\allowbreak\token[1.3067]|,|\allowbreak\token[1.4839]| so|\allowbreak\token[0.7600]| \(|\allowbreak\token[0.0621]|7|\allowbreak\token[0.0000]|x|\allowbreak\token[0.0000]| +|\allowbreak\token[0.0000]| |\allowbreak\token[0.0000]|5|\allowbreak\token[0.0001]|3|\allowbreak\token[0.0127]| =|\allowbreak\token[0.0000]| |\allowbreak\token[0.0000]|8|\allowbreak\token[0.0003]|y|\allowbreak\token[0.0008]|\|\allowbreak\token[0.6026]|).↩↩|\tighttracebreak\allowbreak\token[0.4085]|Given|\allowbreak\token[0.0002]| that|\allowbreak\token[0.5546]| you|\allowbreak\token[0.0619]| have|\allowbreak\token[0.0010]| more|\allowbreak\token[0.0000]| than|\allowbreak\token[0.0001]| |\allowbreak\token[0.0000]|2|\allowbreak\token[0.0000]|0|\allowbreak\token[0.0000]|0|\allowbreak\token[0.0001]| coins|\allowbreak\token[1.0395]| in|\allowbreak\token[0.0074]| total|\allowbreak\token[0.5339]|,|\allowbreak\token[0.0983]| we|\allowbreak\token[1.1646]| have|\allowbreak\token[0.0967]| \(|\allowbreak\token[0.2702]|7|\allowbreak\token[0.0000]|x|\allowbreak\token[0.0003]| +|\allowbreak\token[0.0000]| |\allowbreak\token[0.0000]|5|\allowbreak\token[0.0001]|3|\allowbreak\token[0.0000]| >|\allowbreak\token[0.0000]| |\allowbreak\token[0.0000]|2|\allowbreak\token[0.0001]|0|\allowbreak\token[0.0002]|0|\allowbreak\token[0.0000]|\|\allowbreak\token[0.6786]|),|\allowbreak\token[0.0054]| which|\allowbreak\token[0.0017]| simpl|\allowbreak\token[0.0000]|ifies|\allowbreak\token[0.0000]| to|\allowbreak\token[0.0000]| \(|\allowbreak\token[0.0294]|7|\allowbreak\token[0.0000]|x|\allowbreak\token[0.0000]| >|\allowbreak\token[0.0000]| |\allowbreak\token[0.0000]|1|\allowbreak\token[0.0000]|4|\allowbreak\token[0.0001]|7|\allowbreak\token[0.0494]|\|\allowbreak\token[0.0446]|),|\allowbreak\token[0.3258]| and|\allowbreak\token[0.3191]| thus|\allowbreak\token[0.0002]| \(|\allowbreak\token[0.0000]|x|\allowbreak\token[0.0000]| >|\allowbreak\token[0.0001]| |\allowbreak\token[0.0001]|2|\allowbreak\token[0.0000]|1|\allowbreak\token[0.0004]|\|\allowbreak\token[0.1096]|).|\allowbreak\token[0.9808]| Since|\allowbreak\token[0.0127]| \(|\allowbreak\token[0.0035]|x|\allowbreak\token[0.0002]|\)|\allowbreak\token[0.0945]| must|\allowbreak\token[0.0000]| be|\allowbreak\token[0.0294]| an|\allowbreak\token[0.0000]| integer|\allowbreak\token[0.3227]|,|\allowbreak\token[0.0030]| the|\allowbreak\token[0.0073]| smallest|\allowbreak\token[0.9362]| possible|\allowbreak\token[0.0018]| value|\allowbreak\token[0.0000]| for|\allowbreak\token[0.0000]| \(|\allowbreak\token[0.0004]|x|\allowbreak\token[0.0000]|\)|\allowbreak\token[0.4213]| is|\allowbreak\token[0.0042]| |\allowbreak\token[0.0001]|2|\allowbreak\token[0.0016]|2|\allowbreak\token[0.7191]|.↩↩|\tighttracebreak\allowbreak\token[0.9831]|Sub|\allowbreak\token[0.0000]|stit|\allowbreak\token[0.0005]|uting|\allowbreak\token[0.0000]| \(|\allowbreak\token[0.0000]|x|\allowbreak\token[0.0000]| =|\allowbreak\token[0.0000]| |\allowbreak\token[0.0000]|2|\allowbreak\token[0.0000]|2|\allowbreak\token[0.0001]|\)|\allowbreak\token[0.0010]| into|\allowbreak\cut{[CUT]}\allowbreak\token[0.7102]| the|\allowbreak\token[0.0135]| equation|\allowbreak\token[0.1743]| \(|\allowbreak\token[0.0002]|7|\allowbreak\token[0.0000]|x|\allowbreak\token[0.0001]| +|\allowbreak\token[0.0000]| |\allowbreak\token[0.0000]|5|\allowbreak\token[0.0001]|3|\allowbreak\token[0.1096]| =|\allowbreak\token[0.0000]| |\allowbreak\token[0.0002]|8|\allowbreak\token[0.0000]|y|\allowbreak\token[0.0017]|\|\allowbreak\token[0.0000]|),|\allowbreak\token[0.0000]| we|\allowbreak\token[0.0104]| get|\allowbreak\token[0.0670]| \(|\allowbreak\token[0.0739]|7|\allowbreak\token[0.0337]|(|\allowbreak\token[0.0000]|2|\allowbreak\token[0.0000]|2|\allowbreak\token[0.0000]|)|\allowbreak\token[0.0000]| +|\allowbreak\token[0.0000]| |\allowbreak\token[0.0000]|5|\allowbreak\token[0.0000]|3|\allowbreak\token[0.0000]| =|\allowbreak\token[0.0000]| |\allowbreak\token[0.5392]|1|\allowbreak\token[0.7184]|5|\allowbreak\token[0.0361]|4|\allowbreak\token[0.0000]| +|\allowbreak\token[0.0000]| |\allowbreak\token[0.0000]|5|\allowbreak\token[0.0000]|3|\allowbreak\token[0.0000]| =|\allowbreak\token[0.0000]| |\allowbreak\token[0.0020]|2|\allowbreak\token[0.0000]|0|\allowbreak\token[0.0000]|7|\allowbreak\token[0.1748]|\|\allowbreak\token[0.2686]|).|\allowbreak\token[1.2099]| This|\allowbreak\token[1.2497]| means|\allowbreak\token[0.8052]| \(|\allowbreak\token[0.7138]|8|\allowbreak\token[0.0000]|y|\allowbreak\token[0.0000]| =|\allowbreak\token[0.0000]| |\allowbreak\token[0.0000]|2|\allowbreak\token[0.0000]|0|\allowbreak\token[0.0002]|7|\allowbreak\token[0.0007]|\|\allowbreak\token[0.0034]|),|\allowbreak\token[0.6824]| but|\allowbreak\token[0.6512]| since|\allowbreak\token[0.2515]| |\allowbreak\token[0.0000]|2|\allowbreak\token[0.0000]|0|\allowbreak\token[0.0000]|7|\allowbreak\token[0.0004]| is|\allowbreak\token[0.0000]| not|\allowbreak\token[0.0013]| divisible|\allowbreak\token[0.0000]| by|\allowbreak\token[0.0000]| |\allowbreak\token[0.0000]|8|\allowbreak\token[0.0001]|,|\allowbreak\token[0.6841]| \(|\allowbreak\token[0.0022]|x|\allowbreak\token[0.0073]| =|\allowbreak\token[0.0000]| |\allowbreak\token[0.0000]|2|\allowbreak\token[0.0000]|2|\allowbreak\token[0.0000]|\)|\allowbreak\token[0.4213]| is|\allowbreak\token[0.0015]| not|\allowbreak\token[0.4197]| the|\allowbreak\token[0.2723]| correct|\allowbreak\token[0.8406]| value|\allowbreak\token[0.8114]|.↩↩|\tighttracebreak\allowbreak\token[1.3165]|We|\allowbreak\token[0.8518]| need|\allowbreak\token[0.0061]| to|\allowbreak\token[0.0001]| find|\allowbreak\token[0.0004]| the|\allowbreak\token[0.0016]| smallest|\allowbreak\token[0.0209]| \(|\allowbreak\token[0.0001]|x|\allowbreak\token[0.0901]|\)|\allowbreak\token[0.4363]| such|\allowbreak\token[0.0000]| that|\allowbreak\token[0.0002]| \(|\allowbreak\token[0.0006]|7|\allowbreak\token[0.0000]|x|\allowbreak\token[0.0000]| +|\allowbreak\token[0.0000]| |\allowbreak\token[0.0000]|5|\allowbreak\token[0.0000]|3|\allowbreak\token[0.0001]|\)|\allowbreak\token[0.0000]| is|\allowbreak\token[0.0457]| divisible|\allowbreak\token[0.0000]| by|\allowbreak\token[0.0000]| |\allowbreak\token[0.0000]|8|\allowbreak\token[0.1910]| and|\allowbreak\token[0.3383]| greater|\allowbreak\token[0.0000]| than|\allowbreak\token[0.0006]| |\allowbreak\token[0.0027]|2|\allowbreak\token[0.0002]|0|\allowbreak\token[0.0183]|0|\allowbreak\token[0.0001]|.|\allowbreak\token[0.6860]| Let|\allowbreak\token[0.0000]|'s|\allowbreak\token[0.3903]| increment|\allowbreak\token[0.0402]| \(|\allowbreak\token[0.0000]|x|\allowbreak\token[0.0000]|\)|\allowbreak\token[0.5674]| by|\allowbreak\token[0.0057]| |\allowbreak\token[0.0245]|1|\allowbreak\token[0.9522]| and|\allowbreak\token[0.0936]| check|\allowbreak\token[1.4251]| again|\allowbreak\token[0.4584]|.|\allowbreak\token[0.0968]| For|\allowbreak\token[0.0000]| \(|\allowbreak\token[0.0000]|x|\allowbreak\token[0.0000]| =|\allowbreak\token[0.0000]| |\allowbreak\token[0.0002]|2|\allowbreak\token[0.0020]|3|\allowbreak\token[0.0000]|\|\allowbreak\token[0.0000]|),|\allowbreak\token[0.0443]| we|\allowbreak\token[0.0092]| have|\allowbreak\token[0.0002]| \(|\allowbreak\token[0.0001]|7|\allowbreak\token[0.0008]|(|\allowbreak\token[0.0000]|2|\allowbreak\token[0.0001]|3|\allowbreak\token[0.0000]|)|\allowbreak\token[0.0001]| +|\allowbreak[...]
\tcblower
\token[0.7102]| the|\allowbreak\token[0.0145]| equation|\allowbreak\token[0.1743]| \(|\allowbreak\token[0.0002]|7|\allowbreak\token[0.0000]|x|\allowbreak\token[0.0000]| +|\allowbreak\token[0.0000]| |\allowbreak\token[0.0000]|5|\allowbreak\token[0.0001]|3|\allowbreak\token[0.1325]| =|\allowbreak\token[0.0000]| |\allowbreak\token[0.0003]|8|\allowbreak\token[0.0000]|y|\allowbreak\token[0.0017]|\|\allowbreak\token[0.0000]|),|\allowbreak\token[0.0000]| we|\allowbreak\token[0.0104]| get|\allowbreak\token[0.0740]| \(|\allowbreak\token[0.0739]|7|\allowbreak\token[0.0303]|(|\allowbreak\token[0.0000]|2|\allowbreak\token[0.0000]|2|\allowbreak\token[0.0000]|)|\allowbreak\token[0.0000]| +|\allowbreak\token[0.0000]| |\allowbreak\token[0.0000]|5|\allowbreak\token[0.0000]|3|\allowbreak\token[0.0000]| =|\allowbreak\token[0.0000]| |\allowbreak\token[0.5392]|1|\allowbreak\token[0.7007]|5|\allowbreak\token[0.0341]|4|\allowbreak\token[0.0000]| +|\allowbreak\token[0.0000]| |\allowbreak\token[0.0000]|5|\allowbreak\token[0.0000]|3|\allowbreak\token[0.0000]| =|\allowbreak\token[0.0000]| |\allowbreak\token[0.0021]|2|\allowbreak\token[0.0000]|0|\allowbreak\token[0.0000]|7|\allowbreak\token[0.1748]|\|\allowbreak\token[0.2686]|).|\allowbreak\token[1.2365]| This|\allowbreak\token[1.2335]| means|\allowbreak\token[0.7669]| \(|\allowbreak\token[0.7313]|8|\allowbreak\token[0.0000]|y|\allowbreak\token[0.0000]| =|\allowbreak\token[0.0000]| |\allowbreak\token[0.0000]|2|\allowbreak\token[0.0000]|0|\allowbreak\token[0.0001]|7|\allowbreak\token[0.0008]|\|\allowbreak\token[0.0034]|),|\allowbreak\token[0.6667]| but|\allowbreak\token[0.6513]| since|\allowbreak\token[0.2744]| |\allowbreak\token[0.0000]|2|\allowbreak\token[0.0000]|0|\allowbreak\token[0.0000]|7|\allowbreak\token[0.0004]| is|\allowbreak\token[0.0000]| not|\allowbreak\token[0.0011]| divisible|\allowbreak\token[0.0000]| by|\allowbreak\token[0.0000]| |\allowbreak\token[0.0000]|8|\allowbreak\token[0.0001]|,|\allowbreak\token[0.6893]| \(|\allowbreak\token[0.0024]|x|\allowbreak\token[0.0073]| =|\allowbreak\token[0.0000]| |\allowbreak\token[0.0000]|2|\allowbreak\token[0.0000]|2|\allowbreak\token[0.0000]|\)|\allowbreak\token[0.4209]| is|\allowbreak\token[0.0013]| not|\allowbreak\token[0.4197]| the|\allowbreak\token[0.2724]| correct|\allowbreak\token[0.8421]| value|\allowbreak\token[0.7989]|.↩↩|\tighttracebreak\allowbreak\token[1.2894]|We|\allowbreak\token[0.8520]| need|\allowbreak\token[0.0063]| to|\allowbreak\token[0.0001]| find|\allowbreak\token[0.0004]| the|\allowbreak\token[0.0016]| smallest|\allowbreak\token[0.0247]| \(|\allowbreak\token[0.0001]|x|\allowbreak\token[0.0993]|\)|\allowbreak\token[0.4368]| such|\allowbreak\token[0.0000]| that|\allowbreak\token[0.0002]| \(|\allowbreak\token[0.0007]|7|\allowbreak\token[0.0000]|x|\allowbreak\token[0.0000]| +|\allowbreak\token[0.0000]| |\allowbreak\token[0.0000]|5|\allowbreak\token[0.0000]|3|\allowbreak\token[0.0001]|\)|\allowbreak\token[0.0000]| is|\allowbreak\token[0.0413]| divisible|\allowbreak\token[0.0000]| by|\allowbreak\token[0.0000]| |\allowbreak\token[0.0000]|8|\allowbreak\token[0.1910]| and|\allowbreak\token[0.3258]| greater|\allowbreak\token[0.0000]| than|\allowbreak\token[0.0005]| |\allowbreak\token[0.0024]|2|\allowbreak\token[0.0002]|0|\allowbreak\token[0.0183]|0|\allowbreak\token[0.0001]|.|\allowbreak\token[0.7065]| Let|\allowbreak\token[0.0000]|'s|\allowbreak\token[0.3903]| increment|\allowbreak\token[0.0446]| \(|\allowbreak\token[0.0000]|x|\allowbreak\token[0.0000]|\)|\allowbreak\token[0.5174]| by|\allowbreak\token[0.0053]| |\allowbreak\token[0.0220]|1|\allowbreak\token[1.0275]| and|\allowbreak\token[0.0859]| check|\allowbreak\token[1.4733]| again|\allowbreak\token[0.4659]|.|\allowbreak\token[0.0969]| For|\allowbreak\token[0.0000]| \(|\allowbreak\token[0.0000]|x|\allowbreak\token[0.0000]| =|\allowbreak\token[0.0000]| |\allowbreak\token[0.0002]|2|\allowbreak\token[0.0020]|3|\allowbreak\token[0.0000]|\|\allowbreak\token[0.0000]|),|\allowbreak\token[0.0487]| we|\allowbreak\token[0.0103]| have|\allowbreak\token[0.0001]| \(|\allowbreak\token[0.0001]|7|\allowbreak\token[0.0007]|(|\allowbreak\token[0.0000]|2|\allowbreak\token[0.0001]|3|\allowbreak\token[0.0000]|)|\allowbreak\token[0.0001]| +|\allowbreak[...]
\end{twopartbox}
\begin{twopartbox}{Cut Index 818 | Acceptance Probability 0.9958}{Block 6 Step 2}{Proposal (Accepted)}
[...]\allowbreak\token[0.0002]|6|\allowbreak\token[0.0000]| +|\allowbreak\token[0.0000]| |\allowbreak\token[0.0000]|5|\allowbreak\token[0.0000]|3|\allowbreak\token[0.0000]| =|\allowbreak\token[0.0000]| |\allowbreak\token[0.0000]|2|\allowbreak\token[0.0003]|4|\allowbreak\token[0.0000]|9|\allowbreak\token[0.0001]|\|\allowbreak\token[0.0007]|),|\allowbreak\token[0.6801]| which|\allowbreak\token[0.0001]| is|\allowbreak\token[0.7309]| not|\allowbreak\token[0.0030]| divisible|\allowbreak\token[0.0000]| by|\allowbreak\token[0.0000]| |\allowbreak\token[0.0000]|8|\allowbreak\token[0.0004]|.|\allowbreak\token[0.2842]| For|\allowbreak\token[0.0000]| \(|\allowbreak\token[0.0000]|x|\allowbreak\token[0.0000]| =|\allowbreak\token[0.0000]| |\allowbreak\token[0.0007]|2|\allowbreak\token[0.0001]|9|\allowbreak\token[0.0000]|\|\allowbreak\token[0.0000]|),|\allowbreak\token[0.0001]| we|\allowbreak\token[0.0002]| have|\allowbreak\token[0.0000]| \(|\allowbreak\token[0.0001]|7|\allowbreak\token[0.0000]|(|\allowbreak\token[0.0000]|2|\allowbreak\token[0.0000]|9|\allowbreak\token[0.0000]|)|\allowbreak\token[0.0000]| +|\allowbreak\token[0.0001]| |\allowbreak\token[0.0000]|5|\allowbreak\token[0.0000]|3|\allowbreak\token[0.0000]| =|\allowbreak\token[0.0000]| |\allowbreak\token[0.0006]|2|\allowbreak\token[0.0012]|0|\allowbreak\token[0.0006]|3|\allowbreak\token[0.0009]| +|\allowbreak\token[0.0000]| |\allowbreak\token[0.0000]|5|\allowbreak\token[0.0000]|3|\allowbreak\token[0.0000]| =|\allowbreak\token[0.0000]| |\allowbreak\token[0.0000]|2|\allowbreak\token[0.0003]|5|\allowbreak\token[0.0004]|6|\allowbreak\token[0.0000]|\|\allowbreak\token[0.0003]|),|\allowbreak\token[0.0113]| and|\allowbreak\token[0.0010]| \(|\allowbreak\token[0.0000]|2|\allowbreak\token[0.0000]|5|\allowbreak\token[0.0000]|6|\allowbreak\token[0.0327]|\)|\allowbreak\token[0.0000]| is|\allowbreak\token[0.0120]| divisible|\allowbreak\token[0.0000]| by|\allowbreak\token[0.0001]| |\allowbreak\token[0.0000]|8|\allowbreak\token[0.0169]| (|\allowbreak\token[0.0001]|since|\allowbreak\token[0.0000]| \(|\allowbreak\token[0.0000]|2|\allowbreak\token[0.0000]|5|\allowbreak\token[0.0000]|6|\allowbreak\token[0.0001]| =|\allowbreak\token[0.0000]| |\allowbreak\token[0.0005]|8|\allowbreak\token[0.0002]| \|\allowbreak\token[0.0000]|times|\allowbreak\token[0.0000]| |\allowbreak\token[0.0001]|3|\allowbreak\token[0.0000]|2|\allowbreak\token[0.0000]|\|\allowbreak\token[0.4904]|))|\allowbreak\token[0.0294]|.↩↩|\tighttracebreak\allowbreak\token[0.0761]|Therefore|\allowbreak\token[0.0000]|,|\allowbreak\token[0.0019]| the|\allowbreak\token[0.0034]| smallest|\allowbreak\token[0.0042]| number|\allowbreak\token[0.0000]| of|\allowbreak\token[0.0001]| coins|\allowbreak\token[0.0100]| you|\allowbreak\token[0.0002]| could|\allowbreak\token[0.0000]| have|\allowbreak\token[0.0005]| had|\allowbreak\token[0.0059]| before|\allowbreak\token[0.0002]| finding|\allowbreak\token[0.0000]| the|\allowbreak\token[0.0562]| bag|\allowbreak\token[0.0005]| of|\allowbreak\token[0.0000]| |\allowbreak\token[0.0000]|5|\allowbreak\token[0.0000]|3|\allowbreak\token[0.0003]| coins|\allowbreak\token[0.0550]| is|\allowbreak\token[0.1240]| \(|\allowbreak\token[0.0134]|7|\allowbreak\token[0.3563]| \|\allowbreak\token[0.0000]|times|\allowbreak\token[0.0000]| |\allowbreak\token[0.0000]|2|\allowbreak\token[0.0014]|9|\allowbreak\token[0.5582]| =|\allowbreak\token[0.0001]| |\allowbreak\token[0.0003]|2|\allowbreak\token[0.0000]|0|\allowbreak\token[0.0001]|3|\allowbreak\token[0.3505]|\|\allowbreak\token[0.7364]|),|\allowbreak\token[0.6511]| and|\allowbreak\token[0.8730]| the|\allowbreak\token[0.9158]| total|\allowbreak\token[0.1834]| number|\allowbreak\token[0.0004]| of|\allowbreak\token[0.0042]| coins|\allowbreak\token[0.6438]| after|\allowbreak\token[0.4921]| redistribution|\allowbreak\token[0.0592]| is|\allowbreak\token[0.0916]| \(|\allowbreak\token[0.1913]|2|\allowbreak\token[0.0689]|5|\allowbreak\token[0.0000]|6|\allowbreak\token[0.0023]|\|\allowbreak\token[0.7552]|),|\allowbreak\token[0.4158]| which|\allowbreak\token[0.4887]| is|\allowbreak\token[1.6224]| indeed|\allowbreak\token[0.6051]| more|\allowbreak\token[0.0000]| than|\allowbreak\token[0.0018]| |\allowbreak\token[0.0000]|2|\allowbreak\token[0.0000]|0|\allowbreak\token[0.0003]|0|\allowbreak\cut{[CUT]}\allowbreak\token[0.8913]|.|\allowbreak\token[0.9978]!<|endoftext|>!
\tcblower
\token[0.8879]|.|\allowbreak\token[0.9976]| The|
\end{twopartbox}
\begin{twopartbox}{Cut Index 441 | Acceptance Probability 1.0000}{Block 6 Step 3}{Proposal (Accepted)}
[...]\allowbreak\token[0.0000]|x|\allowbreak\token[0.0000]| +|\allowbreak\token[0.0000]| |\allowbreak\token[0.0000]|5|\allowbreak\token[0.0000]|3|\allowbreak\token[0.0001]|\)|\allowbreak\token[0.0000]| is|\allowbreak\token[0.0457]| divisible|\allowbreak\token[0.0000]| by|\allowbreak\token[0.0000]| |\allowbreak\token[0.0000]|8|\allowbreak\token[0.1910]| and|\allowbreak\token[0.3383]| greater|\allowbreak\token[0.0000]| than|\allowbreak\token[0.0006]| |\allowbreak\token[0.0027]|2|\allowbreak\token[0.0002]|0|\allowbreak\token[0.0183]|0|\allowbreak\token[0.0001]|.|\allowbreak\token[0.6860]| Let|\allowbreak\token[0.0000]|'s|\allowbreak\token[0.3903]| increment|\allowbreak\token[0.0402]| \(|\allowbreak\token[0.0000]|x|\allowbreak\token[0.0000]|\)|\allowbreak\token[0.5674]| by|\allowbreak\token[0.0057]| |\allowbreak\token[0.0245]|1|\allowbreak\token[0.9522]| and|\allowbreak\token[0.0936]| check|\allowbreak\token[1.4251]| again|\allowbreak\token[0.4584]|.|\allowbreak\token[0.0968]| For|\allowbreak\token[0.0000]| \(|\allowbreak\token[0.0000]|x|\allowbreak\token[0.0000]| =|\allowbreak\token[0.0000]| |\allowbreak\token[0.0002]|2|\allowbreak\token[0.0020]|3|\allowbreak\token[0.0000]|\|\allowbreak\token[0.0000]|),|\allowbreak\token[0.0443]| we|\allowbreak\token[0.0092]| have|\allowbreak\token[0.0002]| \(|\allowbreak\token[0.0001]|7|\allowbreak\token[0.0008]|(|\allowbreak\token[0.0000]|2|\allowbreak\token[0.0001]|3|\allowbreak\token[0.0000]|)|\allowbreak\token[0.0001]| +|\allowbreak\token[0.0000]| |\allowbreak\token[0.0000]|5|\allowbreak\token[0.0000]|3|\allowbreak\token[0.0000]| =|\allowbreak\token[0.0000]| |\allowbreak\token[0.1453]|1|\allowbreak\token[0.0231]|6|\allowbreak\token[0.0021]|1|\allowbreak\token[0.0001]| +|\allowbreak\token[0.0000]| |\allowbreak\token[0.0000]|5|\allowbreak\token[0.0000]|3|\allowbreak\token[0.0000]| =|\allowbreak\token[0.0000]| |\allowbreak\token[0.0000]|2|\allowbreak\token[0.0001]|1|\allowbreak\token[0.0001]|4|\allowbreak\token[0.0000]|\|\allowbreak\token[0.1454]|),|\allowbreak\token[0.6115]| which|\allowbreak\token[0.0003]| is|\allowbreak\token[1.3081]| divisible|\allowbreak\token[0.0000]| by|\allowbreak\token[0.0000]| |\allowbreak\token[0.0000]|8|\allowbreak\token[1.2192]| (|\allowbreak\token[0.2863]|since|\allowbreak\token[0.0031]| \(|\allowbreak\token[0.1771]|2|\allowbreak\token[0.0001]|1|\allowbreak\token[0.0008]|4|\allowbreak\token[0.4694]| =|\allowbreak\token[0.0000]| |\allowbreak\token[0.0075]|8|\allowbreak\token[0.0016]| \|\allowbreak\token[0.0000]|times|\allowbreak\token[0.0000]| |\allowbreak\token[0.0001]|2|\allowbreak\token[0.0084]|6|\allowbreak\token[0.3863]|.|\allowbreak\token[0.0001]|7|\allowbreak\token[0.0001]|5|\allowbreak\token[0.0004]|\|\allowbreak\token[0.1548]|),|\allowbreak\token[0.0167]| but|\allowbreak\token[0.1822]| we|\allowbreak\token[0.6753]| need|\allowbreak\token[0.7536]| an|\allowbreak\token[0.1911]| integer|\allowbreak\token[1.1204]| solution|\allowbreak\token[1.2542]|,|\allowbreak\token[0.0182]| so|\allowbreak\token[0.4553]| we|\allowbreak\token[1.0862]| continue|\allowbreak\token[0.1836]|).↩↩|\tighttracebreak\allowbreak\token[0.4698]|For|\allowbreak\token[0.0000]| \(|\allowbreak\token[0.0000]|x|\allowbreak\token[0.0000]| =|\allowbreak\token[0.0000]| |\allowbreak\token[0.0001]|2|\allowbreak\token[0.1798]|4|\allowbreak\token[0.0000]|\|\allowbreak\token[0.0000]|),|\allowbreak\token[0.0461]| we|\allowbreak\token[0.1888]| have|\allowbreak\token[0.0000]| \(|\allowbreak\token[0.0000]|7|\allowbreak\token[0.0000]|(|\allowbreak\token[0.0000]|2|\allowbreak\token[0.0001]|4|\allowbreak\token[0.0000]|)|\allowbreak\token[0.0000]| +|\allowbreak\token[0.0000]| |\allowbreak\token[0.0000]|5|\allowbreak\token[0.0000]|3|\allowbreak\token[0.0000]| =|\allowbreak\token[0.0000]| |\allowbreak\token[0.0000]|1|\allowbreak\token[0.0009]|6|\allowbreak\token[0.0000]|8|\allowbreak\token[0.0000]| +|\allowbreak\token[0.0000]| |\allowbreak\token[0.0000]|5|\allowbreak\token[0.0000]|3|\allowbreak\token[0.0000]| =|\allowbreak\token[0.0000]| |\allowbreak\token[0.0000]|2|\allowbreak\token[0.0004]|2|\allowbreak\token[0.0000]|1|\allowbreak\token[0.0000]|\|\allowbreak\token[0.0091]|),|\allowbreak\token[0.0913]| which|\allowbreak\token[0.0001]| is|\allowbreak\cut{[CUT]}\allowbreak\token[0.8340]| also|\allowbreak\token[0.6440]| not|\allowbreak\token[0.0007]| divisible|\allowbreak\token[0.0000]| by|\allowbreak\token[0.0000]| |\allowbreak\token[0.0000]|8|\allowbreak\token[0.0995]|.|\allowbreak\token[1.1062]| For|\allowbreak\token[0.0000]| \(|\allowbreak\token[0.0000]|x|\allowbreak\token[0.0000]| =|\allowbreak\token[0.0000]| |\allowbreak\token[0.0000]|2|\allowbreak\token[0.0033]|5|\allowbreak\token[0.0000]|\|\allowbreak\token[0.0000]|),|\allowbreak\token[0.0004]| we|\allowbreak\token[0.0587]| have|\allowbreak\token[0.0000]| \(|\allowbreak\token[0.0043]|7|\allowbreak\token[0.0000]|(|\allowbreak\token[0.0000]|2|\allowbreak\token[0.0001]|5|\allowbreak\token[0.0000]|)|\allowbreak\token[0.0000]| +|\allowbreak\token[0.0000]| |\allowbreak\token[0.0000]|5|\allowbreak\token[0.0000]|3|\allowbreak\token[0.0000]| =|\allowbreak\token[0.0000]| |\allowbreak\token[0.0001]|1|\allowbreak\token[0.0000]|7|\allowbreak\token[0.0000]|5|\allowbreak\token[0.0000]| +|\allowbreak\token[0.0000]| |\allowbreak\token[0.0000]|5|\allowbreak\token[0.0000]|3|\allowbreak\token[0.0000]| =|\allowbreak\token[0.0000]| |\allowbreak\token[0.0000]|2|\allowbreak\token[0.0001]|2|\allowbreak\token[0.0001]|8|\allowbreak\token[0.0001]|\|\allowbreak\token[0.0011]|),|\allowbreak\token[0.6913]| and|\allowbreak\token[0.2315]| \(|\allowbreak\token[0.0003]|2|\allowbreak\token[0.0000]|2|\allowbreak\token[0.0000]|8|\allowbreak\token[0.4942]|\)|\allowbreak\token[0.0214]| is|\allowbreak\token[0.0447]| divisible|\allowbreak\token[0.0000]| by|\allowbreak\token[0.4193]| |\allowbreak\token[0.0000]|8|\allowbreak\token[0.5591]| (|\allowbreak\token[0.0013]|since|\allowbreak\token[0.0000]| \(|\allowbreak\token[0.0005]|2|\allowbreak\token[0.0000]|2|\allowbreak\token[0.0000]|8|\allowbreak\token[0.0004]| =|\allowbreak\token[0.0000]| |\allowbreak\token[0.0001]|8|\allowbreak\token[0.0000]| \|\allowbreak\token[0.0000]|times|\allowbreak\token[0.0000]| |\allowbreak\token[0.0001]|2|\allowbreak\token[0.0000]|8|\allowbreak\token[0.0035]|.|\allowbreak\token[0.0001]|5|\allowbreak\token[0.0005]|\|\allowbreak\token[0.1592]|),|\allowbreak\token[0.2843]| but|\allowbreak\token[0.3873]| again|\allowbreak\token[0.0001]|,|\allowbreak\token[0.0100]| we|\allowbreak\token[0.0098]| need|\allowbreak\token[0.0823]| an|\allowbreak\token[0.0741]| integer|\allowbreak\token[0.0853]| solution|\allowbreak\token[0.6242]|).↩↩|\tighttracebreak\allowbreak\token[1.2030]|Contin|\allowbreak\token[0.0000]|uing|\allowbreak\token[0.0000]| this|\allowbreak\token[0.0002]| process|\allowbreak\token[0.0000]|,|\allowbreak\token[0.4780]| for|\allowbreak\token[0.0000]| \(|\allowbreak\token[0.0000]|x|\allowbreak\token[0.0000]| =|\allowbreak\token[0.0000]| |\allowbreak\token[0.0066]|2|\allowbreak\token[0.6199]|6|\allowbreak\token[0.0000]|\|\allowbreak\token[0.0000]|),|\allowbreak\token[0.0002]| we|\allowbreak\token[0.0803]| have|\allowbreak\token[0.0000]| \(|\allowbreak\token[0.0000]|7|\allowbreak\token[0.0000]|(|\allowbreak\token[0.0000]|2|\allowbreak\token[0.0001]|6|\allowbreak\token[0.0000]|)|\allowbreak\token[0.0000]| +|\allowbreak\token[0.0000]| |\allowbreak\token[0.0000]|5|\allowbreak\token[0.0001]|3|\allowbreak\token[0.0000]| =|\allowbreak\token[0.0000]| |\allowbreak\token[0.0001]|1|\allowbreak\token[0.0005]|8|\allowbreak\token[0.0001]|2|\allowbreak\token[0.0000]| +|\allowbreak\token[0.0000]| |\allowbreak\token[0.0000]|5|\allowbreak\token[0.0000]|3|\allowbreak\token[0.0000]| =|\allowbreak\token[0.0000]| |\allowbreak\token[0.0001]|2|\allowbreak\token[0.0002]|3|\allowbreak\token[0.0002]|5|\allowbreak\token[0.0000]|\|\allowbreak\token[0.0006]|),|\allowbreak\token[0.0618]| which|\allowbreak\token[0.0001]| is|\allowbreak\token[0.0562]| not|\allowbreak\token[0.0011]| divisible|\allowbreak\token[0.0000]| by|\allowbreak\token[0.0000]| |\allowbreak\token[0.0000]|8|\allowbreak\token[0.0001]|.|\allowbreak\token[0.1133]| For|\allowbreak\token[0.0000]| \(|\allowbreak\token[0.0000]|x|\allowbreak\token[0.0000]| =|\allowbreak\token[0.0000]| |\allowbreak\token[0.0000]|2|\allowbreak\token[0.0004]|7|\allowbreak\token[0.0000]|\|\allowbreak\token[0.0000]|),|\allowbreak\token[0.0000]| we|\allowbreak\token[0.0018]| have|\allowbreak\token[0.0000]| \(|\allowbreak[...]
\tcblower
\token[0.8338]| also|\allowbreak\token[0.6436]| not|\allowbreak\token[0.0007]| divisible|\allowbreak\token[0.0000]| by|\allowbreak\token[0.0000]| |\allowbreak\token[0.0000]|8|\allowbreak\token[0.0903]|.|\allowbreak\token[1.1060]| For|\allowbreak\token[0.0000]| \(|\allowbreak\token[0.0000]|x|\allowbreak\token[0.0000]| =|\allowbreak\token[0.0000]| |\allowbreak\token[0.0000]|2|\allowbreak\token[0.0033]|5|\allowbreak\token[0.0000]|\|\allowbreak\token[0.0000]|),|\allowbreak\token[0.0003]| we|\allowbreak\token[0.0482]| have|\allowbreak\token[0.0000]| \(|\allowbreak\token[0.0048]|7|\allowbreak\token[0.0000]|(|\allowbreak\token[0.0000]|2|\allowbreak\token[0.0001]|5|\allowbreak\token[0.0000]|)|\allowbreak\token[0.0000]| +|\allowbreak\token[0.0000]| |\allowbreak\token[0.0000]|5|\allowbreak\token[0.0000]|3|\allowbreak\token[0.0000]| =|\allowbreak\token[0.0000]| |\allowbreak\token[0.0001]|1|\allowbreak\token[0.0000]|7|\allowbreak\token[0.0000]|5|\allowbreak\token[0.0000]| +|\allowbreak\token[0.0000]| |\allowbreak\token[0.0000]|5|\allowbreak\token[0.0000]|3|\allowbreak\token[0.0000]| =|\allowbreak\token[0.0000]| |\allowbreak\token[0.0000]|2|\allowbreak\token[0.0001]|2|\allowbreak\token[0.0001]|8|\allowbreak\token[0.0001]|\|\allowbreak\token[0.0011]|),|\allowbreak\token[0.6914]| and|\allowbreak\token[0.2386]| \(|\allowbreak\token[0.0003]|2|\allowbreak\token[0.0000]|2|\allowbreak\token[0.0000]|8|\allowbreak\token[0.4931]|\)|\allowbreak\token[0.0214]| is|\allowbreak\token[0.0401]| divisible|\allowbreak\token[0.0000]| by|\allowbreak\token[0.4193]| |\allowbreak\token[0.0000]|8|\allowbreak\token[0.5591]| (|\allowbreak\token[0.0013]|since|\allowbreak\token[0.0000]| \(|\allowbreak\token[0.0005]|2|\allowbreak\token[0.0000]|2|\allowbreak\token[0.0000]|8|\allowbreak\token[0.0004]| =|\allowbreak\token[0.0000]| |\allowbreak\token[0.0001]|8|\allowbreak\token[0.0000]| \|\allowbreak\token[0.0000]|times|\allowbreak\token[0.0000]| |\allowbreak\token[0.0001]|2|\allowbreak\token[0.0000]|8|\allowbreak\token[0.0032]|.|\allowbreak\token[0.0001]|5|\allowbreak\token[0.0005]|\|\allowbreak\token[0.1758]|),|\allowbreak\token[0.2849]| but|\allowbreak\token[0.3367]| again|\allowbreak\token[0.0001]|,|\allowbreak\token[0.0101]| we|\allowbreak\token[0.0115]| need|\allowbreak\token[0.0684]| an|\allowbreak\token[0.0741]| integer|\allowbreak\token[0.1050]| solution|\allowbreak\token[0.7025]|).↩↩|\tighttracebreak\allowbreak\token[1.1860]|For|\allowbreak\token[0.0000]| \(|\allowbreak\token[0.0000]|x|\allowbreak\token[0.0000]| =|\allowbreak\token[0.0000]| |\allowbreak\token[0.0001]|2|\allowbreak\token[0.0522]|6|\allowbreak\token[0.0000]|\|\allowbreak\token[0.0000]|),|\allowbreak\token[0.0007]| we|\allowbreak\token[0.0076]| have|\allowbreak\token[0.0000]| \(|\allowbreak\token[0.0000]|7|\allowbreak\token[0.0000]|(|\allowbreak\token[0.0000]|2|\allowbreak\token[0.0001]|6|\allowbreak\token[0.0000]|)|\allowbreak\token[0.0000]| +|\allowbreak\token[0.0000]| |\allowbreak\token[0.0000]|5|\allowbreak\token[0.0001]|3|\allowbreak\token[0.0000]| =|\allowbreak\token[0.0000]| |\allowbreak\token[0.0000]|1|\allowbreak\token[0.0003]|8|\allowbreak\token[0.0000]|2|\allowbreak\token[0.0000]| +|\allowbreak\token[0.0000]| |\allowbreak\token[0.0000]|5|\allowbreak\token[0.0000]|3|\allowbreak\token[0.0000]| =|\allowbreak\token[0.0000]| |\allowbreak\token[0.0001]|2|\allowbreak\token[0.0002]|3|\allowbreak\token[0.0002]|5|\allowbreak\token[0.0000]|\|\allowbreak\token[0.0004]|),|\allowbreak\token[0.0910]| which|\allowbreak\token[0.0001]| is|\allowbreak\token[0.1008]| not|\allowbreak\token[0.0009]| divisible|\allowbreak\token[0.0000]| by|\allowbreak\token[0.0000]| |\allowbreak\token[0.0000]|8|\allowbreak\token[0.0006]|.|\allowbreak\token[0.2185]| For|\allowbreak\token[0.0000]| \(|\allowbreak\token[0.0000]|x|\allowbreak\token[0.0000]| =|\allowbreak\token[0.0000]| |\allowbreak\token[0.0000]|2|\allowbreak\token[0.0002]|7|\allowbreak\token[0.0000]|\|\allowbreak\token[0.0000]|),|\allowbreak\token[0.0000]| we|\allowbreak\token[0.0007]| have|\allowbreak\token[0.0000]| \(|\allowbreak\token[0.0003]|7|\allowbreak\token[0.0000]|(|\allowbreak\token[0.0000]|2|\allowbreak\token[0.0001]|7|\allowbreak\token[0.0000]|)|\allowbreak[...]
\end{twopartbox}

\endgroup

\subsection{Math500 Question 124}
\paragraph{Question.}
In the diagram (\Cref{fig:Q124}), $D$ and $E$ are the midpoints of $\overline{AB}$ and $\overline{BC}$ respectively.  Determine the area of quadrilateral $DBEF$.

\begin{figure}[H]
    \centering
    \includegraphics[width=0.5\linewidth]{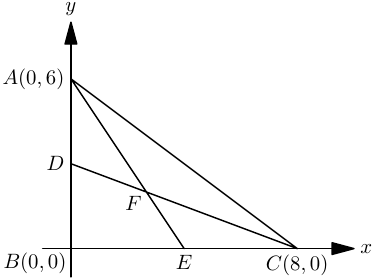}
    \caption{Caption}
    \label{fig:Q124}
\end{figure}

\paragraph{Correct Answer.}
$8$.

\subsubsection{Solutions}
\fzedit{Here the difference is primarily a method-choice error. uniform-cut MH identifies the midpoint coordinates correctly, but then uses the wrong construction for \(F\), taking \(F\) to be the intersection of \(AE\) and \(BD\). Since \(BD\) is the vertical line \(x=0\), this forces \(F=(0,6)\) and leads to the incorrect area \(12\). Entropy-cut MH instead uses the intended construction with \(F\) as the intersection of \(AE\) and \(CD\), solves for \(F=\left(\frac{8}{3},2\right)\), and then computes the correct area \(8\).}

\paragraph{Uniform-Cut Solution.}
\begin{markdown}
To solve the problem, we need to determine the area of quadrilateral \(DBEF\) given that \(D\) and \(E\) are the midpoints of \(\overline{AB}\) and \(\overline{BC}\) respectively.

First, let's identify the coordinates of the points:
- \(A(0, 6)\)
- \(B(0, 0)\)
- \(C(8, 0)\)

Since \(D\) is the midpoint of \(\overline{AB}\):
\[ D = \left( \frac{0+0}{2}, \frac{6+0}{2} \right) = (0, 3) \]

Since \(E\) is the midpoint of \(\overline{BC}\):
\[ E = \left( \frac{0+8}{2}, \frac{0+0}{2} \right) = (4, 0) \]

Next, we need to find the coordinates of point \(F\). Since \(F\) is the intersection of line segments \(AE\) and \(BD\), we need to find the equations of these lines and solve for their intersection.

The equation of line \(AE\) can be found using the points \(A(0, 6)\) and \(E(4, 0)\):
\[ \text{slope of } AE = \frac{0-6}{4-0} = -\frac{3}{2} \]
\[ \text{equation of } AE: y - 6 = -\frac{3}{2}(x - 0) \implies y = -\frac{3}{2}x + 6 \]

The equation of line \(BD\) can be found using the points \(B(0, 0)\) and \(D(0, 3)\):
\[ \text{slope of } BD = \frac{3-0}{0-0} = \text{undefined (vertical line)} \]
\[ \text{equation of } BD: x = 0 \]

Since \(F\) is the intersection of \(AE\) and \(BD\), we substitute \(x = 0\) into the equation of \(AE\):
\[ y = -\frac{3}{2}(0) + 6 = 6 \]
So, \(F(0, 6)\).

Now we have the coordinates of the vertices of quadrilateral \(DBEF\):
- \(D(0, 3)\)
- \(B(0, 0)\)
- \(E(4, 0)\)
- \(F(0, 6)\)

We can use the shoelace formula to find the area of quadrilateral \(DBEF\):
\[ \text{Area} = \frac{1}{2} \left| x_1y_2 + x_2y_3 + x_3y_4 + x_4y_1 - (y_1x_2 + y_2x_3 + y_3x_4 + y_4x_1) \right| \]
Substituting the coordinates:
\[ \text{Area} = \frac{1}{2} \left| 0 \cdot 0 + 0 \cdot 0 + 4 \cdot 6 + 0 \cdot 3 - (3 \cdot 0 + 0 \cdot 4 + 0 \cdot 0 + 6 \cdot 0) \right| \]
\[ \text{Area} = \frac{1}{2} \left| 0 + 0 + 24 + 0 - (0 + 0 + 0 + 0) \right| \]
\[ \text{Area} = \frac{1}{2} \left| 24 \right| \]
\[ \text{Area} = \frac{1}{2} \cdot 24 \]
\[ \text{Area} = 12 \]

So, the area of quadrilateral \(DBEF\) is \(\boxed{12}\).
\end{markdown}

\paragraph{Entropy-Cut Solution.}
\begin{markdown}
To solve the problem, we need to determine the area of quadrilateral \(DBEF\). Here's the step-by-step approach:

1. **Identify Coordinates of Points:**
   - Point \(A\) is \((0, 6)\).
   - Point \(B\) is \((0, 0)\).
   - Point \(C\) is \((8, 0)\).
   - Point \(D\) is the midpoint of \(\overline{AB}\), so \(D\) is \((0, 3)\).
   - Point \(E\) is the midpoint of \(\overline{BC}\), so \(E\) is \((4, 0)\).
   - Point \(F\) is the intersection of \(\overline{AE}\) and \(\overline{CD}\).

2. **Find the Equation of Line \(AE\):**
   - The slope of \(AE\) is \(\frac{6-0}{0-4} = -\frac{3}{2}\).
   - The equation of the line \(AE\) using the point-slope form \(y - y_1 = m(x - x_1)\) is:
     \[
     y - 0 = -\frac{3}{2}(x - 4) \implies y = -\frac{3}{2}x + 6.
     \]

3. **Find the Equation of Line \(CD\):**
   - The slope of \(CD\) is \(\frac{3-0}{0-8} = -\frac{3}{8}\).
   - The equation of the line \(CD\) using the point-slope form is:
     \[
     y - 3 = -\frac{3}{8}(x - 0) \implies y = -\frac{3}{8}x + 3.
     \]

4. **Find the Intersection Point \(F\):**
   - Set the equations equal to each other to find the intersection point:
     \[
     -\frac{3}{2}x + 6 = -\frac{3}{8}x + 3.
     \]
   - Solving for \(x\):
     \[
     -\frac{3}{2}x + \frac{3}{8}x = 3 - 6 \implies -\frac{12}{8}x + \frac{3}{8}x = -3 \implies -\frac{9}{8}x = -3 \implies x = \frac{8 \cdot 3}{9} = \frac{8}{3}.
     \]
   - Substituting \(x = \frac{8}{3}\) back into the equation \(y = -\frac{3}{2}x + 6\):
     \[
     y = -\frac{3}{2} \left(\frac{8}{3}\right) + 6 = -4 + 6 = 2.
     \]
   - So, point \(F\) is \(\left(\frac{8}{3}, 2\right)\).

5. **Calculate the Area of Quadrilateral \(DBEF\):**
   - We can use the shoelace formula to find the area of quadrilateral \(DBEF\):
     \[
     \text{Area} = \frac{1}{2} \left| x_1y_2 + x_2y_3 + x_3y_4 + x_4y_1 - (y_1x_2 + y_2x_3 + y_3x_4 + y_4x_1) \right|.
     \]
   - Substituting the coordinates of \(D(0, 3)\), \(B(0, 0)\), \(E(4, 0)\), and \(F\left(\frac{8}{3}, 2\right)\):
     \[
     \text{Area} = \frac{1}{2} \left| 0 \cdot 0 + 0 \cdot 0 + 4 \cdot 2 + \frac{8}{3} \cdot 3 - (3 \cdot 0 + 0 \cdot 4 + 0 \cdot \frac{8}{3} + 2 \cdot 0) \right|.
     \]
     \[
     \text{Area} = \frac{1}{2} \left| 0 + 0 + 8 + 8 - (0 + 0 + 0 + 0) \right| = \frac{1}{2} \left| 16 \right| = 8.
     \]

So, the area of quadrilateral \(DBEF\) is \(\boxed{8}\).
\end{markdown}

\subsubsection{Reasoning Trace}
\fzedit{For this problem, the first clear escape from a bad local branch occurs at \emph{Block 1 Step 3}, where an accepted proposal switches the geometry to the correct \(AE\)--\(CD\) construction. \emph{Block 2 Step 1} then continues with the same setup by explicitly writing the equations for \(AE\) and \(CD\). The important contrast with uniform-cut MH is that uniform-cut MH later commits to the wrong \(AE\)--\(BD\) interpretation and propagates that bad branch through the rest of the derivation, whereas entropy-cut MH reopens the construction early enough to stay on the correct geometric method.}

\paragraph{Uniform-Cut Trace.}
\begingroup
\raggedright %
\setlength{\lineskip}{0pt} %
\setlength{\baselineskip}{0.9em} %
\begin{twopartbox}{Cut Index 143 | Acceptance Probability 0.0070}{Block 1 Step 1}{Proposal (Rejected)}
\token[2.6002]| To|\allowbreak\token[0.6725]| solve|\allowbreak\token[0.8316]| the|\allowbreak\token[0.0687]| problem|\allowbreak\token[0.5135]|,|\allowbreak\token[0.3361]| we|\allowbreak\token[0.4652]| need|\allowbreak\token[0.0016]| to|\allowbreak\token[0.7385]| determine|\allowbreak\token[0.0004]| the|\allowbreak\token[0.1338]| area|\allowbreak\token[0.0002]| of|\allowbreak\token[0.1912]| quadr|\allowbreak\token[0.0000]|ilateral|\allowbreak\token[0.5312]| \(|\allowbreak\token[0.3013]|DB|\allowbreak\token[0.0001]|EF|\allowbreak\token[0.6859]|\|\allowbreak\token[0.3149]|).|\allowbreak\token[1.6625]| Here|\allowbreak\token[1.0700]|'s|\allowbreak\token[0.6866]| the|\allowbreak\token[0.1679]| step|\allowbreak\token[0.0002]|-by|\allowbreak\token[0.0001]|-step|\allowbreak\token[1.4678]| reasoning|\allowbreak\token[0.6935]|:↩↩|\tighttracebreak\allowbreak\token[0.0038]|1|\allowbreak\token[0.0000]|.|\allowbreak\token[0.2652]| **|\allowbreak\token[0.2757]|Ident|\allowbreak\token[0.0027]|ify|\allowbreak\token[0.6453]| Coordinates|\allowbreak\token[1.0409]| of|\allowbreak\token[0.6452]| Points|\allowbreak\token[0.9793]|:|\allowbreak\token[0.0009]|**↩|\newline\allowbreak\token[0.2687]|  |\allowbreak\token[0.0035]| -|\allowbreak\token[0.9269]| \(|\allowbreak\token[0.1889]|A|\allowbreak\token[0.7277]|(|\allowbreak\token[0.0000]|0|\allowbreak\token[0.0000]|,|\allowbreak\token[0.5823]| |\allowbreak\token[0.0000]|6|\allowbreak\token[0.0000]|)\|\allowbreak\token[0.0347]|)↩|\newline\allowbreak\token[0.0000]|  |\allowbreak\token[0.0000]| -|\allowbreak\token[0.0000]| \(|\allowbreak\token[0.0004]|B|\allowbreak\token[0.0000]|(|\allowbreak\token[0.0000]|0|\allowbreak\token[0.0000]|,|\allowbreak\token[0.0001]| |\allowbreak\token[0.0000]|0|\allowbreak\token[0.0000]|)\|\allowbreak\token[0.0002]|)↩|\newline\allowbreak\token[0.0000]|  |\allowbreak\token[0.0000]| -|\allowbreak\token[0.0002]| \(|\allowbreak\token[0.0002]|C|\allowbreak\token[0.0000]|(|\allowbreak\token[0.0001]|8|\allowbreak\token[0.0000]|,|\allowbreak\token[0.0001]| |\allowbreak\token[0.0000]|0|\allowbreak\token[0.0000]|)\|\allowbreak\token[0.0268]|)↩|\newline\allowbreak\token[0.0082]|  |\allowbreak\token[0.0005]| -|\allowbreak\token[0.1978]| \(|\allowbreak\token[0.0011]|D|\allowbreak\token[0.4275]|\)|\allowbreak\token[0.0367]| is|\allowbreak\token[0.0001]| the|\allowbreak\token[0.0000]| midpoint|\allowbreak\token[0.0000]| of|\allowbreak\token[0.1094]| \|\allowbreak\token[0.0001]|(\|\allowbreak\token[0.0000]|over|\allowbreak\token[0.0000]|line|\allowbreak\token[0.0000]|{|\allowbreak\token[0.0000]|AB|\allowbreak\token[0.0008]|}\|\allowbreak\token[0.0572]|),|\allowbreak\token[0.0003]| so|\allowbreak\token[0.0014]| \(|\allowbreak\token[0.0000]|D|\allowbreak\token[0.4586]|(|\allowbreak\token[0.0005]|0|\allowbreak\token[0.0000]|,|\allowbreak\token[0.0329]| |\allowbreak\token[0.0001]|3|\allowbreak\token[0.0000]|)\|\allowbreak\token[0.0483]|)↩|\newline\allowbreak\token[0.0000]|  |\allowbreak\token[0.0000]| -|\allowbreak\token[0.0000]| \(|\allowbreak\token[0.0000]|E|\allowbreak\token[0.0000]|\)|\allowbreak\token[0.0001]| is|\allowbreak\token[0.0000]| the|\allowbreak\token[0.0000]| midpoint|\allowbreak\token[0.0000]| of|\allowbreak\token[0.0000]| \|\allowbreak\token[0.0000]|(\|\allowbreak\token[0.0001]|over|\allowbreak\token[0.0001]|line|\allowbreak\token[0.0000]|{|\allowbreak\token[0.0004]|BC|\allowbreak\token[0.0000]|}\|\allowbreak\token[0.0000]|),|\allowbreak\token[0.0000]| so|\allowbreak\token[0.0000]| \(|\allowbreak\token[0.0001]|E|\allowbreak\token[0.0009]|(|\allowbreak\token[0.0000]|4|\allowbreak\token[0.0000]|,|\allowbreak\token[0.0000]| |\allowbreak\token[0.0001]|0|\allowbreak\token[0.0000]|)\|\allowbreak\token[0.5040]|)↩|\newline\allowbreak\token[0.1003]|  |\allowbreak\token[0.0001]| -|\allowbreak\token[0.8572]| \(|\allowbreak\token[0.0046]|F|\allowbreak\token[0.0137]|\)|\allowbreak\token[0.0735]| is|\allowbreak\token[0.5274]| the|\allowbreak\token[0.9145]| intersection|\allowbreak\token[0.0215]| of|\allowbreak\token[0.6304]| \|\allowbreak\token[0.0004]|(\|\allowbreak\token[0.0002]|over|\allowbreak\token[0.0008]|line|\allowbreak\token[0.0001]|{|\allowbreak\token[1.2387]|AE|\allowbreak\token[0.0007]|}\|\allowbreak\token[0.0001]|)|\allowbreak\token[0.0035]| and|\allowbreak\token[0.0029]| \|\allowbreak\cut{[CUT]}\allowbreak\token[0.0001]|(\|\allowbreak\token[0.0001]|over|\allowbreak\token[0.0002]|line|\allowbreak\token[0.0000]|{|\allowbreak\token[0.1767]|CD|\allowbreak\token[0.0001]|}\|\allowbreak\token[1.2073]|).↩↩|\tighttracebreak\allowbreak\token[0.0005]|2|\allowbreak\token[0.0000]|.|\allowbreak\token[0.0000]| **|\allowbreak\token[0.2157]|Find|\allowbreak\token[0.2821]| the|\allowbreak\token[1.1300]| Equation|\allowbreak\token[0.0000]| of|\allowbreak\token[0.7482]| Lines|\allowbreak\token[0.9091]| \(|\allowbreak\token[0.0047]|AE|\allowbreak\token[0.0009]|\)|\allowbreak\token[0.0000]| and|\allowbreak\token[0.0001]| \(|\allowbreak\token[0.0001]|CD|\allowbreak\token[0.0002]|\|\allowbreak\token[0.0005]|):|\allowbreak\token[0.0157]|**↩|\newline\allowbreak\token[0.0000]|  |\allowbreak\token[0.0000]| -|\allowbreak\token[0.6082]| Line|\allowbreak\token[0.0000]| \(|\allowbreak\token[0.0005]|AE|\allowbreak\token[0.3411]|\)|\allowbreak\token[0.0191]| passes|\allowbreak\token[0.0000]| through|\allowbreak\token[0.2690]| points|\allowbreak\token[0.0001]| \(|\allowbreak\token[0.0000]|A|\allowbreak\token[0.0000]|(|\allowbreak\token[0.0000]|0|\allowbreak\token[0.0000]|,|\allowbreak\token[0.0027]| |\allowbreak\token[0.0000]|6|\allowbreak\token[0.0000]|)\|\allowbreak\token[0.0000]|)|\allowbreak\token[0.0000]| and|\allowbreak\token[0.0000]| \(|\allowbreak\token[0.0001]|E|\allowbreak\token[0.0000]|(|\allowbreak\token[0.0000]|4|\allowbreak\token[0.0000]|,|\allowbreak\token[0.0000]| |
\tcblower
\token[0.0001]|(\|\allowbreak\token[0.0001]|over|\allowbreak\token[0.0002]|line|\allowbreak\token[0.0000]|{|\allowbreak\token[0.1767]|CD|\allowbreak\token[0.0001]|}\|\allowbreak\token[1.2073]|)↩↩|\tighttracebreak\allowbreak\token[0.0005]|2|\allowbreak\token[0.0000]|.|\allowbreak\token[0.0000]| **|\allowbreak\token[0.2603]|Find|\allowbreak\token[0.3208]| the|\allowbreak\token[1.1346]| Equation|\allowbreak\token[0.0000]| of|\allowbreak\token[0.7733]| Line|\allowbreak\token[0.4612]| \(|\allowbreak\token[0.0094]|AE|\allowbreak\token[0.0239]|\|\allowbreak\token[0.0008]|):|\allowbreak\token[0.0003]|**↩|\newline\allowbreak\token[0.0000]|  |\allowbreak\token[0.6201]| -|\allowbreak\token[1.1969]| The|\allowbreak\token[0.1084]| slope|\allowbreak\token[0.0193]| of|\allowbreak\token[0.1949]| \(|\allowbreak\token[0.0004]|AE|\allowbreak\token[0.0002]|\)|\allowbreak\token[0.0001]| is|\allowbreak\token[0.0010]| \|\allowbreak\token[0.0107]|(\|\allowbreak\token[0.0003]|frac|\allowbreak\token[0.0005]|{|\allowbreak\token[0.6952]|6|\allowbreak\token[0.6467]|-|\allowbreak\token[0.0002]|0|\allowbreak\token[0.0000]|}{|\allowbreak\token[0.0001]|0|\allowbreak\token[0.0000]|-|\allowbreak\token[0.0026]|4|\allowbreak\token[0.0003]|}|\allowbreak\token[0.0003]| =|\allowbreak\token[0.0013]| -|\allowbreak\token[0.0111]|\|\allowbreak\token[0.0000]|frac|\allowbreak\token[0.0000]|{|\allowbreak\token[0.0082]|3|\allowbreak\token[0.0000]|}{|\allowbreak\token[0.0000]|2|
\end{twopartbox}
\begin{twopartbox}{Cut Index 76 | Acceptance Probability 0.0080}{Block 1 Step 2}{Proposal (Rejected)}
\token[2.6002]| To|\allowbreak\token[0.6725]| solve|\allowbreak\token[0.8316]| the|\allowbreak\token[0.0687]| problem|\allowbreak\token[0.5135]|,|\allowbreak\token[0.3361]| we|\allowbreak\token[0.4652]| need|\allowbreak\token[0.0016]| to|\allowbreak\token[0.7385]| determine|\allowbreak\token[0.0004]| the|\allowbreak\token[0.1338]| area|\allowbreak\token[0.0002]| of|\allowbreak\token[0.1912]| quadr|\allowbreak\token[0.0000]|ilateral|\allowbreak\token[0.5312]| \(|\allowbreak\token[0.3013]|DB|\allowbreak\token[0.0001]|EF|\allowbreak\token[0.6859]|\|\allowbreak\token[0.3149]|).|\allowbreak\token[1.6625]| Here|\allowbreak\token[1.0700]|'s|\allowbreak\token[0.6866]| the|\allowbreak\token[0.1679]| step|\allowbreak\token[0.0002]|-by|\allowbreak\token[0.0001]|-step|\allowbreak\token[1.4678]| reasoning|\allowbreak\token[0.6935]|:↩↩|\tighttracebreak\allowbreak\token[0.0038]|1|\allowbreak\token[0.0000]|.|\allowbreak\token[0.2652]| **|\allowbreak\token[0.2757]|Ident|\allowbreak\token[0.0027]|ify|\allowbreak\token[0.6453]| Coordinates|\allowbreak\token[1.0409]| of|\allowbreak\token[0.6452]| Points|\allowbreak\token[0.9793]|:|\allowbreak\token[0.0009]|**↩|\newline\allowbreak\token[0.2687]|  |\allowbreak\token[0.0035]| -|\allowbreak\token[0.9269]| \(|\allowbreak\token[0.1889]|A|\allowbreak\token[0.7277]|(|\allowbreak\token[0.0000]|0|\allowbreak\token[0.0000]|,|\allowbreak\token[0.5823]| |\allowbreak\token[0.0000]|6|\allowbreak\token[0.0000]|)\|\allowbreak\token[0.0347]|)↩|\newline\allowbreak\token[0.0000]|  |\allowbreak\token[0.0000]| -|\allowbreak\token[0.0000]| \(|\allowbreak\token[0.0004]|B|\allowbreak\token[0.0000]|(|\allowbreak\token[0.0000]|0|\allowbreak\token[0.0000]|,|\allowbreak\token[0.0001]| |\allowbreak\token[0.0000]|0|\allowbreak\token[0.0000]|)\|\allowbreak\token[0.0002]|)↩|\newline\allowbreak\token[0.0000]|  |\allowbreak\token[0.0000]| -|\allowbreak\token[0.0002]| \(|\allowbreak\token[0.0002]|C|\allowbreak\token[0.0000]|(|\allowbreak\token[0.0001]|8|\allowbreak\token[0.0000]|,|\allowbreak\token[0.0001]| |\allowbreak\token[0.0000]|0|\allowbreak\token[0.0000]|)\|\allowbreak\token[0.0268]|)↩|\newline\allowbreak\token[0.0082]|  |\allowbreak\token[0.0005]| -|\allowbreak\token[0.1978]| \(|\allowbreak\token[0.0011]|D|\allowbreak\token[0.4275]|\)|\allowbreak\token[0.0367]| is|\allowbreak\cut{[CUT]}\allowbreak\token[0.0001]| the|\allowbreak\token[0.0000]| midpoint|\allowbreak\token[0.0000]| of|\allowbreak\token[0.1094]| \|\allowbreak\token[0.0001]|(\|\allowbreak\token[0.0000]|over|\allowbreak\token[0.0000]|line|\allowbreak\token[0.0000]|{|\allowbreak\token[0.0000]|AB|\allowbreak\token[0.0008]|}\|\allowbreak\token[0.0572]|),|\allowbreak\token[0.0003]| so|\allowbreak\token[0.0014]| \(|\allowbreak\token[0.0000]|D|\allowbreak\token[0.4586]|(|\allowbreak\token[0.0005]|0|\allowbreak\token[0.0000]|,|\allowbreak\token[0.0329]| |\allowbreak\token[0.0001]|3|\allowbreak\token[0.0000]|)\|\allowbreak\token[0.0483]|)↩|\newline\allowbreak\token[0.0000]|  |\allowbreak\token[0.0000]| -|\allowbreak\token[0.0000]| \(|\allowbreak\token[0.0000]|E|\allowbreak\token[0.0000]|\)|\allowbreak\token[0.0001]| is|\allowbreak\token[0.0000]| the|\allowbreak\token[0.0000]| midpoint|\allowbreak\token[0.0000]| of|\allowbreak\token[0.0000]| \|\allowbreak\token[0.0000]|(\|\allowbreak\token[0.0001]|over|\allowbreak\token[0.0001]|line|\allowbreak\token[0.0000]|{|\allowbreak\token[0.0004]|BC|\allowbreak\token[0.0000]|}\|\allowbreak\token[0.0000]|),|\allowbreak\token[0.0000]| so|\allowbreak\token[0.0000]| \(|\allowbreak\token[0.0001]|E|\allowbreak\token[0.0009]|(|\allowbreak\token[0.0000]|4|\allowbreak\token[0.0000]|,|\allowbreak\token[0.0000]| |\allowbreak\token[0.0001]|0|\allowbreak\token[0.0000]|)\|\allowbreak\token[0.5040]|)↩|\newline\allowbreak\token[0.1003]|  |\allowbreak\token[0.0001]| -|\allowbreak\token[0.8572]| \(|\allowbreak\token[0.0046]|F|\allowbreak\token[0.0137]|\)|\allowbreak\token[0.0735]| is|\allowbreak\token[0.5274]| the|\allowbreak\token[0.9145]| intersection|\allowbreak\token[0.0215]| of|\allowbreak\token[0.6304]| \|\allowbreak\token[0.0004]|(\|\allowbreak\token[0.0002]|over|\allowbreak\token[0.0008]|line|\allowbreak\token[0.0001]|{|\allowbreak\token[1.2387]|AE|\allowbreak\token[0.0007]|}\|\allowbreak\token[0.0001]|)|\allowbreak\token[0.0035]| and|\allowbreak\token[0.0029]| \|\allowbreak\token[0.0001]|(\|\allowbreak\token[0.0001]|over|\allowbreak\token[0.0002]|line|\allowbreak\token[0.0000]|{|\allowbreak\token[0.1767]|CD|\allowbreak\token[0.0001]|}\|\allowbreak\token[1.2073]|).↩↩|\tighttracebreak\allowbreak\token[0.0005]|2|\allowbreak\token[0.0000]|.|\allowbreak\token[0.0000]| **|\allowbreak\token[0.2157]|Find|\allowbreak\token[0.2821]| the|\allowbreak\token[1.1300]| Equation|\allowbreak\token[0.0000]| of|\allowbreak\token[0.7482]| Lines|\allowbreak\token[0.9091]| \(|\allowbreak\token[0.0047]|AE|\allowbreak\token[0.0009]|\)|\allowbreak\token[0.0000]| and|\allowbreak\token[0.0001]| \(|\allowbreak\token[0.0001]|CD|\allowbreak\token[0.0002]|\|\allowbreak\token[0.0005]|):|\allowbreak\token[0.0157]|**↩|\newline\allowbreak\token[0.0000]|  |\allowbreak\token[0.0000]| -|\allowbreak\token[0.6082]| Line|\allowbreak\token[0.0000]| \(|\allowbreak\token[0.0005]|AE|\allowbreak\token[0.3411]|\)|\allowbreak\token[0.0191]| passes|\allowbreak\token[0.0000]| through|\allowbreak\token[0.2690]| points|\allowbreak\token[0.0001]| \(|\allowbreak\token[0.0000]|A|\allowbreak\token[0.0000]|(|\allowbreak\token[0.0000]|0|\allowbreak\token[0.0000]|,|\allowbreak\token[0.0027]| |\allowbreak\token[0.0000]|6|\allowbreak\token[0.0000]|)\|\allowbreak\token[0.0000]|)|\allowbreak\token[0.0000]| and|\allowbreak\token[0.0000]| \(|\allowbreak\token[0.0001]|E|\allowbreak\token[0.0000]|(|\allowbreak\token[0.0000]|4|\allowbreak\token[0.0000]|,|\allowbreak\token[0.0000]| |
\tcblower
\token[0.0001]| the|\allowbreak\token[0.0000]| midpoint|\allowbreak\token[0.0000]| of|\allowbreak\token[0.0901]| \|\allowbreak\token[0.0001]|(\|\allowbreak\token[0.0000]|over|\allowbreak\token[0.0000]|line|\allowbreak\token[0.0000]|{|\allowbreak\token[0.0000]|AB|\allowbreak\token[0.0008]|}\|\allowbreak\token[0.0572]|),|\allowbreak\token[0.0003]| so|\allowbreak\token[0.0014]| \(|\allowbreak\token[0.0000]|D|\allowbreak\token[0.4585]|(|\allowbreak\token[0.0006]|0|\allowbreak\token[0.0000]|,|\allowbreak\token[0.0297]| |\allowbreak\token[0.0001]|3|\allowbreak\token[0.0000]|)\|\allowbreak\token[0.0483]|)↩|\newline\allowbreak\token[0.0000]|  |\allowbreak\token[0.0000]| -|\allowbreak\token[0.0000]| \(|\allowbreak\token[0.0000]|E|\allowbreak\token[0.0000]|\)|\allowbreak\token[0.0001]| is|\allowbreak\token[0.0000]| the|\allowbreak\token[0.0000]| midpoint|\allowbreak\token[0.0000]| of|\allowbreak\token[0.0000]| \|\allowbreak\token[0.0000]|(\|\allowbreak\token[0.0002]|over|\allowbreak\token[0.0001]|line|\allowbreak\token[0.0000]|{|\allowbreak\token[0.0003]|BC|\allowbreak\token[0.0000]|}\|\allowbreak\token[0.0000]|),|\allowbreak\token[0.0000]| so|\allowbreak\token[0.0000]| \(|\allowbreak\token[0.0001]|E|\allowbreak\token[0.0010]|(|\allowbreak\token[0.0000]|4|\allowbreak\token[0.0000]|,|\allowbreak\token[0.0000]| |\allowbreak\token[0.0001]|0|\allowbreak\token[0.0000]|)\|\allowbreak\token[0.5040]|)↩|\newline\allowbreak\token[0.1003]|  |\allowbreak\token[0.0001]| -|\allowbreak\token[0.8410]| \(|\allowbreak\token[0.0047]|F|\allowbreak\token[0.0143]|\)|\allowbreak\token[0.0604]| is|\allowbreak\token[0.5276]| the|\allowbreak\token[0.8958]| intersection|\allowbreak\token[0.0215]| of|\allowbreak\token[0.6304]| \|\allowbreak\token[0.0004]|(\|\allowbreak\token[0.0002]|over|\allowbreak\token[0.0008]|line|\allowbreak\token[0.0001]|{|\allowbreak\token[1.2393]|AE|\allowbreak\token[0.0006]|}\|\allowbreak\token[0.0001]|)|\allowbreak\token[0.0032]| and|\allowbreak\token[0.0033]| \|\allowbreak\token[0.0001]|(\|\allowbreak\token[0.0001]|over|\allowbreak\token[0.0002]|line|\allowbreak\token[0.0000]|{|\allowbreak\token[0.1747]|CD|\allowbreak\token[0.0001]|}\|\allowbreak\token[1.2228]|)↩↩|\tighttracebreak\allowbreak\token[0.0005]|2|\allowbreak\token[0.0000]|.|\allowbreak\token[0.0000]| **|\allowbreak\token[0.2603]|Find|\allowbreak\token[0.2949]| the|\allowbreak\token[1.1259]| Equation|\allowbreak\token[0.0000]| of|\allowbreak\token[0.7705]| Line|\allowbreak\token[0.4641]| \(|\allowbreak\token[0.0094]|AE|\allowbreak\token[0.0215]|\|\allowbreak\token[0.0009]|):|\allowbreak\token[0.0003]|**↩|\newline\allowbreak\token[0.0000]|  |\allowbreak\token[0.6195]| -|\allowbreak\token[1.1970]| The|\allowbreak\token[0.1065]| slope|\allowbreak\token[0.0193]| of|\allowbreak\token[0.1778]| \(|\allowbreak\token[0.0004]|AE|\allowbreak\token[0.0002]|\)|\allowbreak\token[0.0001]| is|\allowbreak\token[0.0011]| \|\allowbreak\token[0.0119]|(\|\allowbreak\token[0.0003]|frac|\allowbreak\token[0.0005]|{|\allowbreak\token[0.7007]|6|\allowbreak\token[0.6467]|-|\allowbreak\token[0.0002]|0|\allowbreak\token[0.0000]|}{|\allowbreak\token[0.0001]|0|\allowbreak\token[0.0000]|-|\allowbreak\token[0.0032]|4|\allowbreak\token[0.0003]|}|\allowbreak\token[0.0004]| =|\allowbreak\token[0.0011]| -|\allowbreak\token[0.0111]|\|\allowbreak\token[0.0000]|frac|\allowbreak\token[0.0000]|{|\allowbreak\token[0.0082]|3|\allowbreak\token[0.0000]|}{|\allowbreak\token[0.0000]|2|
\end{twopartbox}
\begin{twopartbox}{Cut Index 14 | Acceptance Probability 1.0000}{Block 1 Step 3}{Proposal (Accepted)}
\token[2.6002]| To|\allowbreak\token[0.6725]| solve|\allowbreak\token[0.8316]| the|\allowbreak\token[0.0687]| problem|\allowbreak\token[0.5135]|,|\allowbreak\token[0.3361]| we|\allowbreak\token[0.4652]| need|\allowbreak\token[0.0016]| to|\allowbreak\token[0.7385]| determine|\allowbreak\token[0.0004]| the|\allowbreak\token[0.1338]| area|\allowbreak\token[0.0002]| of|\allowbreak\token[0.1912]| quadr|\allowbreak\token[0.0000]|ilateral|\allowbreak\cut{[CUT]}\allowbreak\token[0.5312]| \(|\allowbreak\token[0.3013]|DB|\allowbreak\token[0.0001]|EF|\allowbreak\token[0.6859]|\|\allowbreak\token[0.3149]|).|\allowbreak\token[1.6625]| Here|\allowbreak\token[1.0700]|'s|\allowbreak\token[0.6866]| the|\allowbreak\token[0.1679]| step|\allowbreak\token[0.0002]|-by|\allowbreak\token[0.0001]|-step|\allowbreak\token[1.4678]| reasoning|\allowbreak\token[0.6935]|:↩↩|\tighttracebreak\allowbreak\token[0.0038]|1|\allowbreak\token[0.0000]|.|\allowbreak\token[0.2652]| **|\allowbreak\token[0.2757]|Ident|\allowbreak\token[0.0027]|ify|\allowbreak\token[0.6453]| Coordinates|\allowbreak\token[1.0409]| of|\allowbreak\token[0.6452]| Points|\allowbreak\token[0.9793]|:|\allowbreak\token[0.0009]|**↩|\newline\allowbreak\token[0.2687]|  |\allowbreak\token[0.0035]| -|\allowbreak\token[0.9269]| \(|\allowbreak\token[0.1889]|A|\allowbreak\token[0.7277]|(|\allowbreak\token[0.0000]|0|\allowbreak\token[0.0000]|,|\allowbreak\token[0.5823]| |\allowbreak\token[0.0000]|6|\allowbreak\token[0.0000]|)\|\allowbreak\token[0.0347]|)↩|\newline\allowbreak\token[0.0000]|  |\allowbreak\token[0.0000]| -|\allowbreak\token[0.0000]| \(|\allowbreak\token[0.0004]|B|\allowbreak\token[0.0000]|(|\allowbreak\token[0.0000]|0|\allowbreak\token[0.0000]|,|\allowbreak\token[0.0001]| |\allowbreak\token[0.0000]|0|\allowbreak\token[0.0000]|)\|\allowbreak\token[0.0002]|)↩|\newline\allowbreak\token[0.0000]|  |\allowbreak\token[0.0000]| -|\allowbreak\token[0.0002]| \(|\allowbreak\token[0.0002]|C|\allowbreak\token[0.0000]|(|\allowbreak\token[0.0001]|8|\allowbreak\token[0.0000]|,|\allowbreak\token[0.0001]| |\allowbreak\token[0.0000]|0|\allowbreak\token[0.0000]|)\|\allowbreak\token[0.0268]|)↩|\newline\allowbreak\token[0.0082]|  |\allowbreak\token[0.0005]| -|\allowbreak\token[0.1978]| \(|\allowbreak\token[0.0011]|D|\allowbreak\token[0.4275]|\)|\allowbreak\token[0.0367]| is|\allowbreak\token[0.0001]| the|\allowbreak\token[0.0000]| midpoint|\allowbreak\token[0.0000]| of|\allowbreak\token[0.1094]| \|\allowbreak\token[0.0001]|(\|\allowbreak\token[0.0000]|over|\allowbreak\token[0.0000]|line|\allowbreak\token[0.0000]|{|\allowbreak\token[0.0000]|AB|\allowbreak\token[0.0008]|}\|\allowbreak\token[0.0572]|),|\allowbreak\token[0.0003]| so|\allowbreak\token[0.0014]| \(|\allowbreak\token[0.0000]|D|\allowbreak\token[0.4586]|(|\allowbreak\token[0.0005]|0|\allowbreak\token[0.0000]|,|\allowbreak\token[0.0329]| |\allowbreak\token[0.0001]|3|\allowbreak\token[0.0000]|)\|\allowbreak\token[0.0483]|)↩|\newline\allowbreak\token[0.0000]|  |\allowbreak\token[0.0000]| -|\allowbreak\token[0.0000]| \(|\allowbreak\token[0.0000]|E|\allowbreak\token[0.0000]|\)|\allowbreak\token[0.0001]| is|\allowbreak\token[0.0000]| the|\allowbreak\token[0.0000]| midpoint|\allowbreak\token[0.0000]| of|\allowbreak\token[0.0000]| \|\allowbreak\token[0.0000]|(\|\allowbreak\token[0.0001]|over|\allowbreak\token[0.0001]|line|\allowbreak\token[0.0000]|{|\allowbreak\token[0.0004]|BC|\allowbreak\token[0.0000]|}\|\allowbreak\token[0.0000]|),|\allowbreak\token[0.0000]| so|\allowbreak\token[0.0000]| \(|\allowbreak\token[0.0001]|E|\allowbreak\token[0.0009]|(|\allowbreak\token[0.0000]|4|\allowbreak\token[0.0000]|,|\allowbreak\token[0.0000]| |\allowbreak\token[0.0001]|0|\allowbreak\token[0.0000]|)\|\allowbreak\token[0.5040]|)↩|\newline\allowbreak\token[0.1003]|  |\allowbreak\token[0.0001]| -|\allowbreak\token[0.8572]| \(|\allowbreak\token[0.0046]|F|\allowbreak\token[0.0137]|\)|\allowbreak\token[0.0735]| is|\allowbreak\token[0.5274]| the|\allowbreak\token[0.9145]| intersection|\allowbreak\token[0.0215]| of|\allowbreak\token[0.6304]| \|\allowbreak\token[0.0004]|(\|\allowbreak\token[0.0002]|over|\allowbreak\token[0.0008]|line|\allowbreak\token[0.0001]|{|\allowbreak\token[1.2387]|AE|\allowbreak\token[0.0007]|}\|\allowbreak\token[0.0001]|)|\allowbreak\token[0.0035]| and|\allowbreak\token[0.0029]| \|\allowbreak\token[0.0001]|(\|\allowbreak\token[0.0001]|over|\allowbreak\token[0.0002]|line|\allowbreak\token[0.0000]|{|\allowbreak\token[0.1767]|CD|\allowbreak\token[0.0001]|}\|\allowbreak\token[1.2073]|).↩↩|\tighttracebreak\allowbreak\token[0.0005]|2|\allowbreak\token[0.0000]|.|\allowbreak\token[0.0000]| **|\allowbreak\token[0.2157]|Find|\allowbreak\token[0.2821]| the|\allowbreak\token[1.1300]| Equation|\allowbreak\token[0.0000]| of|\allowbreak\token[0.7482]| Lines|\allowbreak[...]
\tcblower
\token[0.5313]| \(|\allowbreak\token[0.3203]|DB|\allowbreak\token[0.0001]|EF|\allowbreak\token[0.6764]|\)|\allowbreak\token[1.2481]| given|\allowbreak\token[0.7398]| that|\allowbreak\token[0.0695]| \(|\allowbreak\token[0.0023]|D|\allowbreak\token[0.0001]|\)|\allowbreak\token[0.0362]| and|\allowbreak\token[0.0000]| \(|\allowbreak\token[0.0002]|E|\allowbreak\token[0.0001]|\)|\allowbreak\token[0.0000]| are|\allowbreak\token[0.4197]| the|\allowbreak\token[0.0001]| mid|\allowbreak\token[0.0000]|points|\allowbreak\token[0.0000]| of|\allowbreak\token[0.0381]| \|\allowbreak\token[0.0000]|(\|\allowbreak\token[0.0000]|over|\allowbreak\token[0.0000]|line|\allowbreak\token[0.0000]|{|\allowbreak\token[0.0002]|AB|\allowbreak\token[0.0001]|}\|\allowbreak\token[0.0001]|)|\allowbreak\token[0.0001]| and|\allowbreak\token[0.0000]| \|\allowbreak\token[0.0000]|(\|\allowbreak\token[0.0010]|over|\allowbreak\token[0.0000]|line|\allowbreak\token[0.0000]|{|\allowbreak\token[0.0016]|BC|\allowbreak\token[0.0002]|}\|\allowbreak\token[0.2279]|)|\allowbreak\token[0.0003]| respectively|\allowbreak\token[1.1826]|.↩↩|\tighttracebreak\allowbreak\token[1.7545]|First|\allowbreak\token[0.0168]|,|\allowbreak\token[0.3726]| let|\allowbreak\token[0.0907]|'s|\allowbreak\token[1.1646]| identify|\allowbreak\token[0.0261]| the|\allowbreak\token[0.0024]| coordinates|\allowbreak\token[0.0240]| of|\allowbreak\token[0.7604]| the|\allowbreak\token[0.5344]| points|\allowbreak\token[1.2730]|:↩|\newline\allowbreak\token[0.0041]|-|\allowbreak\token[0.0943]| \(|\allowbreak\token[0.1553]|A|\allowbreak\token[0.7984]|(|\allowbreak\token[0.0000]|0|\allowbreak\token[0.0001]|,|\allowbreak\token[0.4196]| |\allowbreak\token[0.0002]|6|\allowbreak\token[0.0001]|)\|\allowbreak\token[0.0052]|)↩|\newline\allowbreak\token[0.0000]|-|\allowbreak\token[0.0000]| \(|\allowbreak\token[0.0003]|B|\allowbreak\token[0.0001]|(|\allowbreak\token[0.0000]|0|\allowbreak\token[0.0000]|,|\allowbreak\token[0.0002]| |\allowbreak\token[0.0001]|0|\allowbreak\token[0.0001]|)\|\allowbreak\token[0.0003]|)↩|\newline\allowbreak\token[0.0000]|-|\allowbreak\token[0.0004]| \(|\allowbreak\token[0.0003]|C|\allowbreak\token[0.0001]|(|\allowbreak\token[0.0002]|8|\allowbreak\token[0.0000]|,|\allowbreak\token[0.0002]| |\allowbreak\token[0.0001]|0|\allowbreak\token[0.0002]|)\|\allowbreak\token[0.6871]|)↩↩|\tighttracebreak\allowbreak\token[1.4622]|Since|\allowbreak\token[0.0036]| \(|\allowbreak\token[0.0005]|D|\allowbreak\token[0.0002]|\)|\allowbreak\token[0.6760]| and|\allowbreak\token[0.0001]| \(|\allowbreak\token[0.0001]|E|\allowbreak\token[0.0002]|\)|\allowbreak\token[0.0001]| are|\allowbreak\token[0.4758]| mid|\allowbreak\token[0.0002]|points|\allowbreak\token[0.1028]|:↩|\newline\allowbreak\token[0.0064]|-|\allowbreak\token[0.0421]| \(|\allowbreak\token[0.0002]|D|\allowbreak\token[0.2559]|\)|\allowbreak\token[0.0481]| is|\allowbreak\token[0.0023]| the|\allowbreak\token[0.0000]| midpoint|\allowbreak\token[0.0001]| of|\allowbreak\token[0.0402]| \|\allowbreak\token[0.0001]|(\|\allowbreak\token[0.0000]|over|\allowbreak\token[0.0000]|line|\allowbreak\token[0.0000]|{|\allowbreak\token[0.0001]|AB|\allowbreak\token[0.0019]|}\|\allowbreak\token[1.0841]|),|\allowbreak\token[0.0241]| so|\allowbreak\token[0.3774]| \(|\allowbreak\token[0.0011]|D|\allowbreak\token[1.2363]|\|\allowbreak\token[0.0002]|left|\allowbreak\token[0.2089]|(\|\allowbreak\token[0.0039]|frac|\allowbreak\token[0.0039]|{|\allowbreak\token[0.0008]|0|\allowbreak\token[0.2274]|+|\allowbreak\token[0.0000]|0|\allowbreak\token[0.0000]|}{|\allowbreak\token[0.0000]|2|\allowbreak\token[0.0001]|},|\allowbreak\token[0.0010]| \|\allowbreak\token[0.0000]|frac|\allowbreak\token[0.0000]|{|\allowbreak\token[0.1748]|6|\allowbreak\token[0.0002]|+|\allowbreak\token[0.0000]|0|\allowbreak\token[0.0001]|}{|\allowbreak\token[0.0000]|2|\allowbreak\token[0.0001]|}\|\allowbreak\token[0.0002]|right|\allowbreak\token[0.0006]|)|\allowbreak\token[0.0002]| =|\allowbreak\token[0.6923]| (|\allowbreak\token[0.0000]|0|\allowbreak\token[0.0000]|,|\allowbreak\token[0.0008]| |\allowbreak\token[0.0002]|3|\allowbreak\token[0.0001]|)\|\allowbreak\token[0.3665]|)↩|\newline\allowbreak\token[0.0001]|-|\allowbreak\token[0.0001]| \(|\allowbreak\token[0.0001]|E|\allowbreak\token[0.0001]|\)|\allowbreak\token[0.0003]| is|\allowbreak[...]
\end{twopartbox}
\begin{twopartbox}{Cut Index 198 | Acceptance Probability 0.1667}{Block 2 Step 1}{Proposal (Rejected)}
[...]\allowbreak\token[0.0907]|'s|\allowbreak\token[1.1646]| identify|\allowbreak\token[0.0261]| the|\allowbreak\token[0.0024]| coordinates|\allowbreak\token[0.0240]| of|\allowbreak\token[0.7604]| the|\allowbreak\token[0.5344]| points|\allowbreak\token[1.2730]|:↩|\newline\allowbreak\token[0.0041]|-|\allowbreak\token[0.0943]| \(|\allowbreak\token[0.1553]|A|\allowbreak\token[0.7984]|(|\allowbreak\token[0.0000]|0|\allowbreak\token[0.0001]|,|\allowbreak\token[0.4196]| |\allowbreak\token[0.0002]|6|\allowbreak\token[0.0001]|)\|\allowbreak\token[0.0052]|)↩|\newline\allowbreak\token[0.0000]|-|\allowbreak\token[0.0000]| \(|\allowbreak\token[0.0003]|B|\allowbreak\token[0.0001]|(|\allowbreak\token[0.0000]|0|\allowbreak\token[0.0000]|,|\allowbreak\token[0.0002]| |\allowbreak\token[0.0001]|0|\allowbreak\token[0.0001]|)\|\allowbreak\token[0.0003]|)↩|\newline\allowbreak\token[0.0000]|-|\allowbreak\token[0.0004]| \(|\allowbreak\token[0.0003]|C|\allowbreak\token[0.0001]|(|\allowbreak\token[0.0002]|8|\allowbreak\token[0.0000]|,|\allowbreak\token[0.0002]| |\allowbreak\token[0.0001]|0|\allowbreak\token[0.0002]|)\|\allowbreak\token[0.6871]|)↩↩|\tighttracebreak\allowbreak\token[1.4622]|Since|\allowbreak\token[0.0036]| \(|\allowbreak\token[0.0005]|D|\allowbreak\token[0.0002]|\)|\allowbreak\token[0.6760]| and|\allowbreak\token[0.0001]| \(|\allowbreak\token[0.0001]|E|\allowbreak\token[0.0002]|\)|\allowbreak\token[0.0001]| are|\allowbreak\token[0.4758]| mid|\allowbreak\token[0.0002]|points|\allowbreak\token[0.1028]|:↩|\newline\allowbreak\token[0.0064]|-|\allowbreak\token[0.0421]| \(|\allowbreak\token[0.0002]|D|\allowbreak\token[0.2559]|\)|\allowbreak\token[0.0481]| is|\allowbreak\token[0.0023]| the|\allowbreak\token[0.0000]| midpoint|\allowbreak\token[0.0001]| of|\allowbreak\token[0.0402]| \|\allowbreak\token[0.0001]|(\|\allowbreak\token[0.0000]|over|\allowbreak\token[0.0000]|line|\allowbreak\token[0.0000]|{|\allowbreak\token[0.0001]|AB|\allowbreak\token[0.0019]|}\|\allowbreak\token[1.0841]|),|\allowbreak\token[0.0241]| so|\allowbreak\token[0.3774]| \(|\allowbreak\token[0.0011]|D|\allowbreak\token[1.2363]|\|\allowbreak\token[0.0002]|left|\allowbreak\token[0.2089]|(\|\allowbreak\token[0.0039]|frac|\allowbreak\token[0.0039]|{|\allowbreak\token[0.0008]|0|\allowbreak\token[0.2274]|+|\allowbreak\token[0.0000]|0|\allowbreak\token[0.0000]|}{|\allowbreak\token[0.0000]|2|\allowbreak\token[0.0001]|},|\allowbreak\token[0.0010]| \|\allowbreak\token[0.0000]|frac|\allowbreak\token[0.0000]|{|\allowbreak\token[0.1748]|6|\allowbreak\token[0.0002]|+|\allowbreak\token[0.0000]|0|\allowbreak\token[0.0001]|}{|\allowbreak\token[0.0000]|2|\allowbreak\token[0.0001]|}\|\allowbreak\token[0.0002]|right|\allowbreak\token[0.0006]|)|\allowbreak\token[0.0002]| =|\allowbreak\token[0.6923]| (|\allowbreak\token[0.0000]|0|\allowbreak\token[0.0000]|,|\allowbreak\token[0.0008]| |\allowbreak\token[0.0002]|3|\allowbreak\token[0.0001]|)\|\allowbreak\token[0.3665]|)↩|\newline\allowbreak\token[0.0001]|-|\allowbreak\token[0.0001]| \(|\allowbreak\token[0.0001]|E|\allowbreak\token[0.0001]|\)|\allowbreak\token[0.0003]| is|\allowbreak\token[0.0001]| the|\allowbreak\token[0.0000]| midpoint|\allowbreak\token[0.0000]| of|\allowbreak\token[0.0001]| \|\allowbreak\token[0.0001]|(\|\allowbreak\token[0.0007]|over|\allowbreak\token[0.0001]|line|\allowbreak\token[0.0000]|{|\allowbreak\token[0.0005]|BC|\allowbreak\token[0.0001]|}\|\allowbreak\token[0.0001]|),|\allowbreak\token[0.0001]| so|\allowbreak\token[0.0000]| \(|\allowbreak\token[0.0002]|E|\allowbreak\token[0.0000]|\|\allowbreak\token[0.0000]|left|\allowbreak\token[0.0000]|(\|\allowbreak\token[0.0000]|frac|\allowbreak\token[0.0000]|{|\allowbreak\token[0.0494]|0|\allowbreak\token[0.0002]|+|\allowbreak\token[0.0000]|8|\allowbreak\token[0.0000]|}{|\allowbreak\token[0.0000]|2|\allowbreak\token[0.0000]|},|\allowbreak\token[0.0001]| \|\allowbreak\token[0.0008]|frac|\allowbreak\token[0.0000]|{|\allowbreak\token[0.0000]|0|\allowbreak\token[0.0001]|+|\allowbreak\token[0.0000]|0|\allowbreak\token[0.0001]|}{|\allowbreak\token[0.0002]|2|\allowbreak\token[0.0001]|}\|\allowbreak\token[0.0001]|right|\allowbreak\token[0.0001]|)|\allowbreak\token[0.0002]| =|\allowbreak\token[0.0002]| (|\allowbreak\token[0.0000]|4|\allowbreak\token[0.0000]|,|\allowbreak\cut{[CUT]}\allowbreak\token[0.0001]| |\allowbreak\token[0.0001]|0|\allowbreak\token[0.0002]|)\|\allowbreak\token[0.0246]|)↩↩|\tighttracebreak\allowbreak\token[1.5519]|We|\allowbreak\token[1.4358]| also|\allowbreak\token[0.2664]| need|\allowbreak\token[0.7105]| the|\allowbreak\token[0.0351]| coordinates|\allowbreak\token[0.0367]| of|\allowbreak\token[0.6566]| \(|\allowbreak\token[0.0013]|F|\allowbreak\token[0.0116]|\|\allowbreak\token[0.7443]|),|\allowbreak\token[0.3280]| which|\allowbreak\token[0.6152]| is|\allowbreak\token[0.6459]| the|\allowbreak\token[0.2742]| intersection|\allowbreak\token[0.0363]| of|\allowbreak\token[1.4571]| \|\allowbreak\token[0.0014]|(\|\allowbreak\token[0.0002]|over|\allowbreak\token[0.0011]|line|\allowbreak\token[0.0001]|{|\allowbreak\token[0.9998]|AE|\allowbreak\token[0.0005]|}\|\allowbreak\token[0.0001]|)|\allowbreak\token[0.0047]| and|\allowbreak\token[0.0073]| \|\allowbreak\token[0.0005]|(\|\allowbreak\token[0.0004]|over|\allowbreak\token[0.0001]|line|\allowbreak\token[0.0000]|{|\allowbreak\token[0.5351]|CD|\allowbreak\token[0.0004]|}\|\allowbreak\token[0.7845]|).|\allowbreak\token[1.5953]| To|\allowbreak\token[0.0086]| find|\allowbreak\token[0.4626]| \(|\allowbreak\token[0.0001]|F|\allowbreak\token[0.0007]|\|\allowbreak\token[0.2358]|),|\allowbreak\token[0.0126]| we|\allowbreak\token[1.4109]| need|\allowbreak\token[0.6760]| the|\allowbreak\token[0.0029]| equations|\allowbreak\token[0.0001]| of|\allowbreak\token[0.6737]| lines|\allowbreak\token[0.6276]| \(|\allowbreak\token[0.0004]|AE|\allowbreak\token[0.0001]|\)|\allowbreak\token[0.0000]| and|\allowbreak\token[0.0001]| \(|\allowbreak\token[0.0002]|CD|\allowbreak\token[0.0081]|\|\allowbreak\token[0.5966]|).↩↩|\tighttracebreak\allowbreak\token[1.4105]|1|\allowbreak\token[0.0006]|.|\allowbreak\token[1.0174]| **|\allowbreak\token[0.1629]|Equ|\allowbreak\token[0.0009]|ation|\allowbreak\token[0.0004]| of|\allowbreak\token[0.5463]| line|\allowbreak\token[0.0005]| \(|\allowbreak\token[0.0006]|AE|\allowbreak\token[0.0023]|\|\allowbreak\token[0.5575]|)**|\allowbreak\token[0.0225]|:↩|\newline\allowbreak\token[0.1645]|  |\allowbreak\token[0.1799]| -|\allowbreak\token[0.4860]| Slo|\allowbreak\token[0.0000]|pe|\allowbreak\token[0.2687]| of|\allowbreak\token[0.0013]| \(|\allowbreak\token[0.0010]|AE|\allowbreak\token[0.8302]|\)|\allowbreak\token[0.7750]| =|\allowbreak\token[0.0002]| \|\allowbreak\token[0.0002]|(\|\allowbreak\token[0.0001]|frac|\allowbreak\token[0.0002]|{|\allowbreak\token[0.1234]|0|\allowbreak\token[0.6630]|-|\allowbreak\token[0.0051]|6|\allowbreak\token[0.0000]|}{|\allowbreak\token[0.0001]|4|\allowbreak\token[0.0000]|-|\allowbreak\token[0.0000]|0|\allowbreak\token[0.0001]|}|\allowbreak\token[0.0000]| =|\allowbreak\token[0.0053]| -|\allowbreak\token[0.0044]|\|\allowbreak\token[0.0000]|frac|\allowbreak\token[0.0000]|{|\allowbreak\token[0.0091]|3|\allowbreak\token[0.0000]|}{|\allowbreak\token[0.0000]|2|\allowbreak\token[0.0001]|}\|\allowbreak\token[0.0003]|)↩|\newline\allowbreak\token[0.0000]|  |\allowbreak\token[0.0001]| -|\allowbreak\token[0.4395]| Equation|\allowbreak\token[0.9362]|:|\allowbreak\token[0.0018]| \(|\allowbreak\token[0.1095]|y|\allowbreak\token[0.4218]| -|\allowbreak\token[0.0000]| |\allowbreak\token[0.5824]|6|\allowbreak\token[0.0000]| =|\allowbreak\token[0.0000]| -|\allowbreak\token[0.0000]|\|\allowbreak\token[0.0000]|frac|\allowbreak\token[0.0000]|{|\allowbreak\token[0.0000]|3|\allowbreak\token[0.0000]|}{|\allowbreak\token[0.0000]|2|\allowbreak\token[0.6914]|}|\allowbreak\token[0.0175]|x|\allowbreak\token[0.9276]| \|\allowbreak\token[0.8962]|Rightarrow|\allowbreak\token[0.0671]| y|\allowbreak\token[0.0000]| =|\allowbreak\token[0.2473]| -|\allowbreak\token[0.0000]|\|\allowbreak\token[0.0000]|frac|\allowbreak\token[0.0000]|{|\allowbreak\token[0.0000]|3|\allowbreak\token[0.0000]|}{|\allowbreak\token[0.0000]|2|\allowbreak\token[0.0000]|}|\allowbreak\token[0.0000]|x|\allowbreak\token[0.0000]| +|\allowbreak\token[0.0000]| |\allowbreak\token[0.0000]|6|\allowbreak\token[0.0009]|\|\allowbreak\token[0.0007]|)↩↩|\tighttracebreak\allowbreak\token[0.0000]|2|\allowbreak\token[0.0000]|.|\allowbreak\token[0.0000]| **|\allowbreak\token[0.0000]|Equ|\allowbreak\token[0.0003]|ation|\allowbreak\token[0.0000]| of|\allowbreak\token[0.0000]| line|\allowbreak\token[0.0000]| \(|\allowbreak[...]
\tcblower
\token[0.0001]| |\allowbreak\token[0.0001]|0|\allowbreak\token[0.0002]|)\|\allowbreak\token[0.0245]|)↩↩|\tighttracebreak\allowbreak\token[1.5726]|We|\allowbreak\token[1.4296]| also|\allowbreak\token[0.2663]| need|\allowbreak\token[0.7113]| the|\allowbreak\token[0.0426]| coordinates|\allowbreak\token[0.0331]| of|\allowbreak\token[0.6567]| \(|\allowbreak\token[0.0013]|F|\allowbreak\token[0.0116]|\|\allowbreak\token[0.7440]|),|\allowbreak\token[0.2878]| which|\allowbreak\token[0.5747]| is|\allowbreak\token[0.7023]| the|\allowbreak\token[0.2597]| intersection|\allowbreak\token[0.0363]| of|\allowbreak\token[1.3847]| \|\allowbreak\token[0.0016]|(\|\allowbreak\token[0.0002]|over|\allowbreak\token[0.0010]|line|\allowbreak\token[0.0001]|{|\allowbreak\token[0.9777]|AE|\allowbreak\token[0.0005]|}\|\allowbreak\token[0.0001]|)|\allowbreak\token[0.0046]| and|\allowbreak\token[0.0073]| \|\allowbreak\token[0.0005]|(\|\allowbreak\token[0.0004]|over|\allowbreak\token[0.0001]|line|\allowbreak\token[0.0000]|{|\allowbreak\token[0.5351]|CD|\allowbreak\token[0.0004]|}\|\allowbreak\token[0.7675]|).|\allowbreak\token[1.5664]| To|\allowbreak\token[0.0078]| find|\allowbreak\token[0.4767]| \(|\allowbreak\token[0.0001]|F|\allowbreak\token[0.0007]|\|\allowbreak\token[0.2471]|),|\allowbreak\token[0.0126]| we|\allowbreak\token[1.4307]| need|\allowbreak\token[0.6760]| the|\allowbreak\token[0.0029]| equations|\allowbreak\token[0.0001]| of|\allowbreak\token[0.6737]| lines|\allowbreak\token[0.6276]| \(|\allowbreak\token[0.0004]|AE|\allowbreak\token[0.0001]|\)|\allowbreak\token[0.0000]| and|\allowbreak\token[0.0000]| \(|\allowbreak\token[0.0003]|CD|\allowbreak\token[0.0073]|\|\allowbreak\token[0.5969]|).↩↩|\tighttracebreak\allowbreak\token[1.4747]|1|\allowbreak\token[0.0007]|.|\allowbreak\token[1.0170]| **|\allowbreak\token[0.1788]|Equ|\allowbreak\token[0.0009]|ation|\allowbreak\token[0.0004]| of|\allowbreak\token[0.5493]| line|\allowbreak\token[0.0004]| \(|\allowbreak\token[0.0006]|AE|\allowbreak\token[0.0025]|\|\allowbreak\token[0.5574]|)**|\allowbreak\token[0.0226]|:↩|\newline\allowbreak\token[0.1645]|  |\allowbreak\token[0.1799]| -|\allowbreak\token[0.4946]| Slo|\allowbreak\token[0.0000]|pe|\allowbreak\token[0.3149]| of|\allowbreak\token[0.0015]| \(|\allowbreak\token[0.0009]|AE|\allowbreak\token[0.8304]|\)|\allowbreak\token[0.7750]| =|\allowbreak\token[0.0002]| \|\allowbreak\token[0.0002]|(\|\allowbreak\token[0.0001]|frac|\allowbreak\token[0.0002]|{|\allowbreak\token[0.1122]|0|\allowbreak\token[0.6630]|-|\allowbreak\token[0.0051]|6|\allowbreak\token[0.0000]|}{|\allowbreak\token[0.0001]|4|\allowbreak\token[0.0000]|-|\allowbreak\token[0.0000]|0|\allowbreak\token[0.0002]|}|\allowbreak\token[0.0000]| =|\allowbreak\token[0.0053]| -|\allowbreak\token[0.0049]|\|\allowbreak\token[0.0000]|frac|\allowbreak\token[0.0000]|{|\allowbreak\token[0.0082]|3|\allowbreak\token[0.0000]|}{|\allowbreak\token[0.0000]|2|\allowbreak\token[0.0001]|}\|\allowbreak\token[0.0002]|)↩|\newline\allowbreak\token[0.0000]|  |\allowbreak\token[0.0000]| -|\allowbreak\token[0.5048]| Equation|\allowbreak\token[0.9035]|:|\allowbreak\token[0.0018]| \(|\allowbreak\token[0.1204]|y|\allowbreak\token[0.4218]| -|\allowbreak\token[0.0000]| |\allowbreak\token[0.5824]|6|\allowbreak\token[0.0000]| =|\allowbreak\token[0.0000]| -|\allowbreak\token[0.0000]|\|\allowbreak\token[0.0000]|frac|\allowbreak\token[0.0000]|{|\allowbreak\token[0.0000]|3|\allowbreak\token[0.0000]|}{|\allowbreak\token[0.0000]|2|\allowbreak\token[0.6914]|}(|\allowbreak\token[0.0000]|x|\allowbreak\token[0.0006]| -|\allowbreak\token[0.0000]| |\allowbreak\token[0.0001]|0|\allowbreak\token[0.6467]|)\|\allowbreak\token[0.6802]|)↩|\newline\allowbreak\token[0.1473]|  |\allowbreak\token[0.3738]| -|\allowbreak\token[0.2265]| Simpl|\allowbreak\token[0.7346]|ified|\allowbreak\token[0.2218]|:|\allowbreak\token[0.0001]| \(|\allowbreak\token[0.0153]|y|\allowbreak\token[0.0001]| =|\allowbreak\token[0.2473]| -|\allowbreak\token[0.0000]|\|\allowbreak\token[0.0000]|frac|\allowbreak\token[0.0000]|{|\allowbreak\token[0.0000]|3|\allowbreak\token[0.0000]|}{|\allowbreak\token[0.0000]|2|\allowbreak\token[0.0000]|}|\allowbreak\token[0.0003]|x|\allowbreak\token[0.0000]| +|\allowbreak\token[0.0000]| |\allowbreak\token[0.0000]|6|\allowbreak\token[0.0011]|\|\allowbreak[...]
\end{twopartbox}
\begin{twopartbox}{Cut Index 95 | Acceptance Probability 1.0000}{Block 2 Step 2}{Proposal (Accepted)}
\token[2.6002]| To|\allowbreak\token[0.6725]| solve|\allowbreak\token[0.8316]| the|\allowbreak\token[0.0687]| problem|\allowbreak\token[0.5135]|,|\allowbreak\token[0.3361]| we|\allowbreak\token[0.4652]| need|\allowbreak\token[0.0016]| to|\allowbreak\token[0.7385]| determine|\allowbreak\token[0.0004]| the|\allowbreak\token[0.1338]| area|\allowbreak\token[0.0002]| of|\allowbreak\token[0.1912]| quadr|\allowbreak\token[0.0000]|ilateral|\allowbreak\token[0.5313]| \(|\allowbreak\token[0.3203]|DB|\allowbreak\token[0.0001]|EF|\allowbreak\token[0.6764]|\)|\allowbreak\token[1.2481]| given|\allowbreak\token[0.7398]| that|\allowbreak\token[0.0695]| \(|\allowbreak\token[0.0023]|D|\allowbreak\token[0.0001]|\)|\allowbreak\token[0.0362]| and|\allowbreak\token[0.0000]| \(|\allowbreak\token[0.0002]|E|\allowbreak\token[0.0001]|\)|\allowbreak\token[0.0000]| are|\allowbreak\token[0.4197]| the|\allowbreak\token[0.0001]| mid|\allowbreak\token[0.0000]|points|\allowbreak\token[0.0000]| of|\allowbreak\token[0.0381]| \|\allowbreak\token[0.0000]|(\|\allowbreak\token[0.0000]|over|\allowbreak\token[0.0000]|line|\allowbreak\token[0.0000]|{|\allowbreak\token[0.0002]|AB|\allowbreak\token[0.0001]|}\|\allowbreak\token[0.0001]|)|\allowbreak\token[0.0001]| and|\allowbreak\token[0.0000]| \|\allowbreak\token[0.0000]|(\|\allowbreak\token[0.0010]|over|\allowbreak\token[0.0000]|line|\allowbreak\token[0.0000]|{|\allowbreak\token[0.0016]|BC|\allowbreak\token[0.0002]|}\|\allowbreak\token[0.2279]|)|\allowbreak\token[0.0003]| respectively|\allowbreak\token[1.1826]|.↩↩|\tighttracebreak\allowbreak\token[1.7545]|First|\allowbreak\token[0.0168]|,|\allowbreak\token[0.3726]| let|\allowbreak\token[0.0907]|'s|\allowbreak\token[1.1646]| identify|\allowbreak\token[0.0261]| the|\allowbreak\token[0.0024]| coordinates|\allowbreak\token[0.0240]| of|\allowbreak\token[0.7604]| the|\allowbreak\token[0.5344]| points|\allowbreak\token[1.2730]|:↩|\newline\allowbreak\token[0.0041]|-|\allowbreak\token[0.0943]| \(|\allowbreak\token[0.1553]|A|\allowbreak\token[0.7984]|(|\allowbreak\token[0.0000]|0|\allowbreak\token[0.0001]|,|\allowbreak\token[0.4196]| |\allowbreak\token[0.0002]|6|\allowbreak\token[0.0001]|)\|\allowbreak\token[0.0052]|)↩|\newline\allowbreak\token[0.0000]|-|\allowbreak\token[0.0000]| \(|\allowbreak\token[0.0003]|B|\allowbreak\token[0.0001]|(|\allowbreak\token[0.0000]|0|\allowbreak\token[0.0000]|,|\allowbreak\token[0.0002]| |\allowbreak\token[0.0001]|0|\allowbreak\token[0.0001]|)\|\allowbreak\token[0.0003]|)↩|\newline\allowbreak\token[0.0000]|-|\allowbreak\token[0.0004]| \(|\allowbreak\token[0.0003]|C|\allowbreak\token[0.0001]|(|\allowbreak\token[0.0002]|8|\allowbreak\token[0.0000]|,|\allowbreak\token[0.0002]| |\allowbreak\token[0.0001]|0|\allowbreak\token[0.0002]|)\|\allowbreak\token[0.6871]|)↩↩|\tighttracebreak\allowbreak\token[1.4622]|Since|\allowbreak\token[0.0036]| \(|\allowbreak\token[0.0005]|D|\allowbreak\cut{[CUT]}\allowbreak\token[0.0002]|\)|\allowbreak\token[0.6760]| and|\allowbreak\token[0.0001]| \(|\allowbreak\token[0.0001]|E|\allowbreak\token[0.0002]|\)|\allowbreak\token[0.0001]| are|\allowbreak\token[0.4758]| mid|\allowbreak\token[0.0002]|points|\allowbreak\token[0.1028]|:↩|\newline\allowbreak\token[0.0064]|-|\allowbreak\token[0.0421]| \(|\allowbreak\token[0.0002]|D|\allowbreak\token[0.2559]|\)|\allowbreak\token[0.0481]| is|\allowbreak\token[0.0023]| the|\allowbreak\token[0.0000]| midpoint|\allowbreak\token[0.0001]| of|\allowbreak\token[0.0402]| \|\allowbreak\token[0.0001]|(\|\allowbreak\token[0.0000]|over|\allowbreak\token[0.0000]|line|\allowbreak\token[0.0000]|{|\allowbreak\token[0.0001]|AB|\allowbreak\token[0.0019]|}\|\allowbreak\token[1.0841]|),|\allowbreak\token[0.0241]| so|\allowbreak\token[0.3774]| \(|\allowbreak\token[0.0011]|D|\allowbreak\token[1.2363]|\|\allowbreak\token[0.0002]|left|\allowbreak\token[0.2089]|(\|\allowbreak\token[0.0039]|frac|\allowbreak\token[0.0039]|{|\allowbreak\token[0.0008]|0|\allowbreak\token[0.2274]|+|\allowbreak\token[0.0000]|0|\allowbreak\token[0.0000]|}{|\allowbreak\token[0.0000]|2|\allowbreak\token[0.0001]|},|\allowbreak\token[0.0010]| \|\allowbreak\token[0.0000]|frac|\allowbreak\token[0.0000]|{|\allowbreak\token[0.1748]|6|\allowbreak\token[0.0002]|+|\allowbreak\token[0.0000]|0|\allowbreak\token[0.0001]|}{|\allowbreak\token[0.0000]|2|\allowbreak\token[0.0001]|}\|\allowbreak\token[0.0002]|right|\allowbreak\token[0.0006]|)|\allowbreak\token[0.0002]| =|\allowbreak\token[0.6923]| (|\allowbreak\token[0.0000]|0|\allowbreak\token[0.0000]|,|\allowbreak\token[0.0008]| |\allowbreak\token[0.0002]|3|\allowbreak\token[0.0001]|)\|\allowbreak\token[0.3665]|)↩|\newline\allowbreak\token[0.0001]|-|\allowbreak\token[0.0001]| \(|\allowbreak\token[0.0001]|E|\allowbreak\token[0.0001]|\)|\allowbreak\token[0.0003]| is|\allowbreak\token[0.0001]| the|\allowbreak\token[0.0000]| midpoint|\allowbreak\token[0.0000]| of|\allowbreak\token[0.0001]| \|\allowbreak\token[0.0001]|(\|\allowbreak\token[0.0007]|over|\allowbreak\token[0.0001]|line|\allowbreak\token[0.0000]|{|\allowbreak\token[0.0005]|BC|\allowbreak\token[0.0001]|}\|\allowbreak\token[0.0001]|),|\allowbreak\token[0.0001]| so|\allowbreak\token[0.0000]| \(|\allowbreak\token[0.0002]|E|\allowbreak\token[0.0000]|\|\allowbreak\token[0.0000]|left|\allowbreak\token[0.0000]|(\|\allowbreak\token[0.0000]|frac|\allowbreak\token[0.0000]|{|\allowbreak\token[0.0494]|0|\allowbreak\token[0.0002]|+|\allowbreak\token[0.0000]|8|\allowbreak\token[0.0000]|}{|\allowbreak\token[0.0000]|2|\allowbreak\token[0.0000]|},|\allowbreak\token[0.0001]| \|\allowbreak\token[0.0008]|frac|\allowbreak\token[0.0000]|{|\allowbreak\token[0.0000]|0|\allowbreak\token[0.0001]|+|\allowbreak\token[0.0000]|0|\allowbreak\token[0.0001]|}{|\allowbreak\token[0.0002]|2|\allowbreak\token[0.0001]|}\|\allowbreak\token[0.0001]|right|\allowbreak\token[0.0001]|)|\allowbreak\token[0.0002]| =|\allowbreak\token[0.0002]| (|\allowbreak\token[0.0000]|4|\allowbreak\token[0.0000]|,|\allowbreak\token[0.0001]| |\allowbreak\token[0.0001]|0|\allowbreak\token[0.0002]|)\|\allowbreak\token[0.0246]|)↩↩|\tighttracebreak\allowbreak\token[1.5519]|We|\allowbreak\token[1.4358]| also|\allowbreak\token[0.2664]| need|\allowbreak\token[0.7105]| the|\allowbreak\token[0.0351]| coordinates|\allowbreak\token[0.0367]| of|\allowbreak\token[0.6566]| \(|\allowbreak\token[0.0013]|F|\allowbreak\token[0.0116]|\|\allowbreak\token[0.7443]|),|\allowbreak\token[0.3280]| which|\allowbreak\token[0.6152]| is|\allowbreak\token[0.6459]| the|\allowbreak\token[0.2742]| intersection|\allowbreak\token[0.0363]| of|\allowbreak\token[1.4571]| \|\allowbreak\token[0.0014]|(\|\allowbreak\token[0.0002]|over|\allowbreak\token[0.0011]|line|\allowbreak\token[0.0001]|{|\allowbreak\token[0.9998]|AE|\allowbreak\token[0.0005]|}\|\allowbreak\token[0.0001]|)|\allowbreak\token[0.0047]| and|\allowbreak\token[0.0073]| \|\allowbreak\token[0.0005]|(\|\allowbreak\token[0.0004]|over|\allowbreak\token[0.0001]|line|\allowbreak\token[0.0000]|{|\allowbreak\token[0.5351]|CD|\allowbreak\token[0.0004]|}\|\allowbreak\token[0.7845]|).|\allowbreak\token[1.5953]| To|\allowbreak\token[0.0086]| find|\allowbreak\token[0.4626]| \(|\allowbreak\token[0.0001]|F|\allowbreak\token[0.0007]|\|\allowbreak[...]
\tcblower
\token[0.0002]|\)|\allowbreak\token[0.6760]| is|\allowbreak\token[0.0001]| the|\allowbreak\token[0.0000]| midpoint|\allowbreak\token[0.0002]| of|\allowbreak\token[0.0113]| \|\allowbreak\token[0.0001]|(\|\allowbreak\token[0.0002]|over|\allowbreak\token[0.0001]|line|\allowbreak\token[0.0001]|{|\allowbreak\token[0.0003]|AB|\allowbreak\token[0.0002]|}\|\allowbreak\token[0.8417]|):↩|\newline\allowbreak\token[0.5591]|\[|\allowbreak\token[0.2927]| D|\allowbreak\token[0.8893]| =|\allowbreak\token[0.0001]| \|\allowbreak\token[0.0001]|left|\allowbreak\token[0.2691]|(|\allowbreak\token[0.0116]| \|\allowbreak\token[0.0008]|frac|\allowbreak\token[0.0329]|{|\allowbreak\token[0.0024]|0|\allowbreak\token[0.6631]|+|\allowbreak\token[0.0000]|0|\allowbreak\token[0.0000]|}{|\allowbreak\token[0.0000]|2|\allowbreak\token[0.0001]|},|\allowbreak\token[0.0000]| \|\allowbreak\token[0.0001]|frac|\allowbreak\token[0.0000]|{|\allowbreak\token[0.1326]|6|\allowbreak\token[0.0002]|+|\allowbreak\token[0.0000]|0|\allowbreak\token[0.0001]|}{|\allowbreak\token[0.0000]|2|\allowbreak\token[0.0000]|}|\allowbreak\token[0.0000]| \|\allowbreak\token[0.0000]|right|\allowbreak\token[0.0004]|)|\allowbreak\token[0.0007]| =|\allowbreak\token[0.0018]| (|\allowbreak\token[0.0000]|0|\allowbreak\token[0.0000]|,|\allowbreak\token[0.0002]| |\allowbreak\token[0.0002]|3|\allowbreak\token[0.0004]|)|\allowbreak\token[0.0001]| \|\allowbreak\token[0.0059]|]↩↩|\tighttracebreak\allowbreak\token[0.1003]|Since|\allowbreak\token[0.0000]| \(|\allowbreak\token[0.0001]|E|\allowbreak\token[0.0003]|\)|\allowbreak\token[0.0003]| is|\allowbreak\token[0.0000]| the|\allowbreak\token[0.0000]| midpoint|\allowbreak\token[0.0000]| of|\allowbreak\token[0.0000]| \|\allowbreak\token[0.0005]|(\|\allowbreak\token[0.0002]|over|\allowbreak\token[0.0000]|line|\allowbreak\token[0.0000]|{|\allowbreak\token[0.0007]|BC|\allowbreak\token[0.0001]|}\|\allowbreak\token[0.0002]|):↩|\newline\allowbreak\token[0.0000]|\[|\allowbreak\token[0.0001]| E|\allowbreak\token[0.0001]| =|\allowbreak\token[0.0000]| \|\allowbreak\token[0.0000]|left|\allowbreak\token[0.0000]|(|\allowbreak\token[0.0000]| \|\allowbreak\token[0.0000]|frac|\allowbreak\token[0.0000]|{|\allowbreak\token[0.0494]|0|\allowbreak\token[0.0008]|+|\allowbreak\token[0.0000]|8|\allowbreak\token[0.0000]|}{|\allowbreak\token[0.0000]|2|\allowbreak\token[0.0000]|},|\allowbreak\token[0.0000]| \|\allowbreak\token[0.0001]|frac|\allowbreak\token[0.0000]|{|\allowbreak\token[0.0001]|0|\allowbreak\token[0.0000]|+|\allowbreak\token[0.0000]|0|\allowbreak\token[0.0001]|}{|\allowbreak\token[0.0002]|2|\allowbreak\token[0.0000]|}|\allowbreak\token[0.0004]| \|\allowbreak\token[0.0004]|right|\allowbreak\token[0.0001]|)|\allowbreak\token[0.0002]| =|\allowbreak\token[0.0002]| (|\allowbreak\token[0.0000]|4|\allowbreak\token[0.0000]|,|\allowbreak\token[0.0001]| |\allowbreak\token[0.0002]|0|\allowbreak\token[0.0005]|)|\allowbreak\token[0.0000]| \|\allowbreak\token[0.0014]|]↩↩|\tighttracebreak\allowbreak\token[1.5966]|Next|\allowbreak\token[0.0009]|,|\allowbreak\token[0.3245]| we|\allowbreak\token[0.3536]| need|\allowbreak\token[0.1951]| to|\allowbreak\token[0.6223]| find|\allowbreak\token[0.0583]| the|\allowbreak\token[0.1246]| coordinates|\allowbreak\token[0.0007]| of|\allowbreak\token[0.6730]| point|\allowbreak\token[0.0001]| \(|\allowbreak\token[0.0021]|F|\allowbreak\token[0.0117]|\|\allowbreak\token[0.6301]|).|\allowbreak\token[1.5853]| Since|\allowbreak\token[0.1894]| \(|\allowbreak\token[0.2436]|F|\allowbreak\token[0.0004]|\)|\allowbreak\token[0.5669]| is|\allowbreak\token[0.8988]| the|\allowbreak\token[0.3516]| intersection|\allowbreak\token[0.0141]| of|\allowbreak\token[1.5777]| line|\allowbreak\token[0.6444]| segments|\allowbreak\token[0.6633]| \(|\allowbreak\token[1.0602]|AE|\allowbreak\token[0.0003]|\)|\allowbreak\token[0.0001]| and|\allowbreak\token[0.0018]| \(|\allowbreak\token[0.9437]|BD|\allowbreak\token[0.0017]|\|\allowbreak\token[0.0306]|),|\allowbreak\token[0.0350]| we|\allowbreak\token[1.2384]| need|\allowbreak\token[0.0125]| to|\allowbreak\token[0.9057]| find|\allowbreak\token[0.0356]| the|\allowbreak\token[0.1261]| equations|\allowbreak\token[0.0003]| of|\allowbreak\token[0.6777]| these|\allowbreak\token[0.0235]| lines|\allowbreak\token[0.5417]| and|\allowbreak\token[0.6232]| solve|\allowbreak[...]
\end{twopartbox}
\begin{twopartbox}{Cut Index 42 | Acceptance Probability 0.0000}{Block 2 Step 3}{Proposal (Rejected)}
\token[2.6002]| To|\allowbreak\token[0.6725]| solve|\allowbreak\token[0.8316]| the|\allowbreak\token[0.0687]| problem|\allowbreak\token[0.5135]|,|\allowbreak\token[0.3361]| we|\allowbreak\token[0.4652]| need|\allowbreak\token[0.0016]| to|\allowbreak\token[0.7385]| determine|\allowbreak\token[0.0004]| the|\allowbreak\token[0.1338]| area|\allowbreak\token[0.0002]| of|\allowbreak\token[0.1912]| quadr|\allowbreak\token[0.0000]|ilateral|\allowbreak\token[0.5313]| \(|\allowbreak\token[0.3203]|DB|\allowbreak\token[0.0001]|EF|\allowbreak\token[0.6764]|\)|\allowbreak\token[1.2481]| given|\allowbreak\token[0.7398]| that|\allowbreak\token[0.0695]| \(|\allowbreak\token[0.0023]|D|\allowbreak\token[0.0001]|\)|\allowbreak\token[0.0362]| and|\allowbreak\token[0.0000]| \(|\allowbreak\token[0.0002]|E|\allowbreak\token[0.0001]|\)|\allowbreak\token[0.0000]| are|\allowbreak\token[0.4197]| the|\allowbreak\token[0.0001]| mid|\allowbreak\token[0.0000]|points|\allowbreak\token[0.0000]| of|\allowbreak\token[0.0381]| \|\allowbreak\token[0.0000]|(\|\allowbreak\token[0.0000]|over|\allowbreak\token[0.0000]|line|\allowbreak\token[0.0000]|{|\allowbreak\token[0.0002]|AB|\allowbreak\token[0.0001]|}\|\allowbreak\token[0.0001]|)|\allowbreak\token[0.0001]| and|\allowbreak\token[0.0000]| \|\allowbreak\cut{[CUT]}\allowbreak\token[0.0000]|(\|\allowbreak\token[0.0010]|over|\allowbreak\token[0.0000]|line|\allowbreak\token[0.0000]|{|\allowbreak\token[0.0016]|BC|\allowbreak\token[0.0002]|}\|\allowbreak\token[0.2279]|)|\allowbreak\token[0.0003]| respectively|\allowbreak\token[1.1826]|.↩↩|\tighttracebreak\allowbreak\token[1.7545]|First|\allowbreak\token[0.0168]|,|\allowbreak\token[0.3726]| let|\allowbreak\token[0.0907]|'s|\allowbreak\token[1.1646]| identify|\allowbreak\token[0.0261]| the|\allowbreak\token[0.0024]| coordinates|\allowbreak\token[0.0240]| of|\allowbreak\token[0.7604]| the|\allowbreak\token[0.5344]| points|\allowbreak\token[1.2730]|:↩|\newline\allowbreak\token[0.0041]|-|\allowbreak\token[0.0943]| \(|\allowbreak\token[0.1553]|A|\allowbreak\token[0.7984]|(|\allowbreak\token[0.0000]|0|\allowbreak\token[0.0001]|,|\allowbreak\token[0.4196]| |\allowbreak\token[0.0002]|6|\allowbreak\token[0.0001]|)\|\allowbreak\token[0.0052]|)↩|\newline\allowbreak\token[0.0000]|-|\allowbreak\token[0.0000]| \(|\allowbreak\token[0.0003]|B|\allowbreak\token[0.0001]|(|\allowbreak\token[0.0000]|0|\allowbreak\token[0.0000]|,|\allowbreak\token[0.0002]| |\allowbreak\token[0.0001]|0|\allowbreak\token[0.0001]|)\|\allowbreak\token[0.0003]|)↩|\newline\allowbreak\token[0.0000]|-|\allowbreak\token[0.0004]| \(|\allowbreak\token[0.0003]|C|\allowbreak\token[0.0001]|(|\allowbreak\token[0.0002]|8|\allowbreak\token[0.0000]|,|\allowbreak\token[0.0002]| |\allowbreak\token[0.0001]|0|\allowbreak\token[0.0002]|)\|\allowbreak\token[0.6871]|)↩↩|\tighttracebreak\allowbreak\token[1.4622]|Since|\allowbreak\token[0.0036]| \(|\allowbreak\token[0.0005]|D|\allowbreak\token[0.0002]|\)|\allowbreak\token[0.6760]| is|\allowbreak\token[0.0001]| the|\allowbreak\token[0.0000]| midpoint|\allowbreak\token[0.0002]| of|\allowbreak\token[0.0113]| \|\allowbreak\token[0.0001]|(\|\allowbreak\token[0.0002]|over|\allowbreak\token[0.0001]|line|\allowbreak\token[0.0001]|{|\allowbreak\token[0.0003]|AB|\allowbreak\token[0.0002]|}\|\allowbreak\token[0.8417]|):↩|\newline\allowbreak\token[0.5591]|\[|\allowbreak\token[0.2927]| D|\allowbreak\token[0.8893]| =|\allowbreak\token[0.0001]| \|\allowbreak\token[0.0001]|left|\allowbreak\token[0.2691]|(|\allowbreak\token[0.0116]| \|\allowbreak\token[0.0008]|frac|\allowbreak\token[0.0329]|{|\allowbreak\token[0.0024]|0|\allowbreak\token[0.6631]|+|\allowbreak\token[0.0000]|0|\allowbreak\token[0.0000]|}{|\allowbreak\token[0.0000]|2|\allowbreak\token[0.0001]|},|\allowbreak\token[0.0000]| \|\allowbreak\token[0.0001]|frac|\allowbreak\token[0.0000]|{|\allowbreak\token[0.1326]|6|\allowbreak\token[0.0002]|+|\allowbreak\token[0.0000]|0|\allowbreak\token[0.0001]|}{|\allowbreak\token[0.0000]|2|\allowbreak\token[0.0000]|}|\allowbreak\token[0.0000]| \|\allowbreak\token[0.0000]|right|\allowbreak\token[0.0004]|)|\allowbreak\token[0.0007]| =|\allowbreak\token[0.0018]| (|\allowbreak\token[0.0000]|0|\allowbreak\token[0.0000]|,|\allowbreak\token[0.0002]| |\allowbreak\token[0.0002]|3|\allowbreak\token[0.0004]|)|\allowbreak\token[0.0001]| \|\allowbreak\token[0.0059]|]↩↩|\tighttracebreak\allowbreak\token[0.1003]|Since|\allowbreak\token[0.0000]| \(|\allowbreak\token[0.0001]|E|\allowbreak\token[0.0003]|\)|\allowbreak\token[0.0003]| is|\allowbreak\token[0.0000]| the|\allowbreak\token[0.0000]| midpoint|\allowbreak\token[0.0000]| of|\allowbreak\token[0.0000]| \|\allowbreak\token[0.0005]|(\|\allowbreak\token[0.0002]|over|\allowbreak\token[0.0000]|line|\allowbreak\token[0.0000]|{|\allowbreak\token[0.0007]|BC|\allowbreak\token[0.0001]|}\|\allowbreak\token[0.0002]|):↩|\newline\allowbreak\token[0.0000]|\[|\allowbreak\token[0.0001]| E|\allowbreak\token[0.0001]| =|\allowbreak\token[0.0000]| \|\allowbreak\token[0.0000]|left|\allowbreak\token[0.0000]|(|\allowbreak\token[0.0000]| \|\allowbreak\token[0.0000]|frac|\allowbreak\token[0.0000]|{|\allowbreak\token[0.0494]|0|\allowbreak\token[0.0008]|+|\allowbreak\token[0.0000]|8|\allowbreak\token[0.0000]|}{|\allowbreak\token[0.0000]|2|\allowbreak\token[0.0000]|},|\allowbreak\token[0.0000]| \|\allowbreak\token[0.0001]|frac|\allowbreak\token[0.0000]|{|\allowbreak\token[0.0001]|0|\allowbreak\token[0.0000]|+|\allowbreak\token[0.0000]|0|\allowbreak\token[0.0001]|}{|\allowbreak\token[0.0002]|2|\allowbreak\token[0.0000]|}|\allowbreak\token[0.0004]| \|\allowbreak\token[0.0004]|right|\allowbreak[...]
\tcblower
\token[0.0000]|(\|\allowbreak\token[0.0013]|over|\allowbreak\token[0.0000]|line|\allowbreak\token[0.0000]|{|\allowbreak\token[0.0018]|BC|\allowbreak\token[0.0002]|}\|\allowbreak\token[0.2279]|)|\allowbreak\token[0.0003]| respectively|\allowbreak\token[1.1720]|.↩↩|\tighttracebreak\allowbreak\token[1.7846]|Here|\allowbreak\token[0.9159]| are|\allowbreak\token[0.0103]| the|\allowbreak\token[0.1392]| steps|\allowbreak\token[1.0190]| to|\allowbreak\token[0.2884]| solve|\allowbreak\token[0.3030]| the|\allowbreak\token[0.0003]| problem|\allowbreak\token[0.0240]|:↩↩|\tighttracebreak\allowbreak\token[0.0005]|1|\allowbreak\token[0.0008]|.|\allowbreak\token[1.0531]| **|\allowbreak\token[0.7659]|Ident|\allowbreak\token[0.0020]|ify|\allowbreak\token[0.8612]| Coordinates|\allowbreak\token[0.9641]| of|\allowbreak\token[1.0550]| Points|\allowbreak\token[0.9842]|:|\allowbreak\token[0.0011]|**↩|\newline\allowbreak\token[0.2282]|  |\allowbreak\token[0.0567]| -|\allowbreak\token[1.1384]| \(|\allowbreak\token[0.2840]|A|\allowbreak\token[0.7924]|(|\allowbreak\token[0.0000]|0|\allowbreak\token[0.0001]|,|\allowbreak\token[0.3151]| |\allowbreak\token[0.0002]|6|\allowbreak\token[0.0001]|)\|\allowbreak\token[0.0260]|)↩|\newline\allowbreak\token[0.0000]|  |\allowbreak\token[0.0000]| -|\allowbreak\token[0.0001]| \(|\allowbreak\token[0.0006]|B|\allowbreak\token[0.0001]|(|\allowbreak\token[0.0000]|0|\allowbreak\token[0.0000]|,|\allowbreak\token[0.0005]| |\allowbreak\token[0.0001]|0|\allowbreak\token[0.0001]|)\|\allowbreak\token[0.0004]|)↩|\newline\allowbreak\token[0.0000]|  |\allowbreak\token[0.0000]| -|\allowbreak\token[0.0003]| \(|\allowbreak\token[0.0004]|C|\allowbreak\token[0.0001]|(|\allowbreak\token[0.0002]|8|\allowbreak\token[0.0000]|,|\allowbreak\token[0.0003]| |\allowbreak\token[0.0000]|0|\allowbreak\token[0.0003]|)\|\allowbreak\token[0.1922]|)↩|\newline\allowbreak\token[0.0413]|  |\allowbreak\token[0.0011]| -|\allowbreak\token[0.3044]| \(|\allowbreak\token[0.0007]|D|\allowbreak\token[0.6832]|\)|\allowbreak\token[0.0974]| is|\allowbreak\token[0.0048]| the|\allowbreak\token[0.0001]| midpoint|\allowbreak\token[0.0001]| of|\allowbreak\token[0.0740]| \|\allowbreak\token[0.0001]|(\|\allowbreak\token[0.0000]|over|\allowbreak\token[0.0001]|line|\allowbreak\token[0.0001]|{|\allowbreak\token[0.0001]|AB|\allowbreak\token[0.0049]|}\|\allowbreak\token[1.1205]|),|\allowbreak\token[0.0247]| so|\allowbreak\token[0.0095]| \(|\allowbreak\token[0.0003]|D|\allowbreak\token[1.0423]|(|\allowbreak\token[0.0099]|0|\allowbreak\token[0.0006]|,|\allowbreak\token[0.0633]| |\allowbreak\token[0.0007]|3|\allowbreak\token[0.0002]|)\|\allowbreak\token[0.1272]|)↩|\newline\allowbreak\token[0.0000]|  |\allowbreak\token[0.0000]| -|\allowbreak\token[0.0002]| \(|\allowbreak\token[0.0001]|E|\allowbreak\token[0.0001]|\)|\allowbreak\token[0.0004]| is|\allowbreak\token[0.0001]| the|\allowbreak\token[0.0001]| midpoint|\allowbreak\token[0.0001]| of|\allowbreak\token[0.0001]| \|\allowbreak\token[0.0001]|(\|\allowbreak\token[0.0012]|over|\allowbreak\token[0.0001]|line|\allowbreak\token[0.0000]|{|\allowbreak\token[0.0006]|BC|\allowbreak\token[0.0002]|}\|\allowbreak\token[0.0002]|),|\allowbreak\token[0.0004]| so|\allowbreak\token[0.0001]| \(|\allowbreak\token[0.0003]|E|\allowbreak\token[0.0028]|(|\allowbreak\token[0.0001]|4|\allowbreak\token[0.0001]|,|\allowbreak\token[0.0002]| |\allowbreak\token[0.0002]|0|\allowbreak\token[0.0002]|)\|\allowbreak\token[0.5844]|)↩↩|\tighttracebreak\allowbreak\token[0.0055]|2|\allowbreak\token[0.0001]|.|\allowbreak\token[0.0002]| **|\allowbreak\token[1.1419]|Find|\allowbreak\token[0.7460]| Coordinates|\allowbreak\token[0.0007]| of|\allowbreak\token[0.7232]| \(|\allowbreak\token[0.0065]|F|\allowbreak\token[0.0063]|\|\allowbreak\token[0.0042]|):|\allowbreak\token[0.0068]|**↩|\newline\allowbreak\token[0.0024]|  |\allowbreak\token[0.9454]| -|\allowbreak\token[1.0034]| \(|\allowbreak\token[0.0456]|F|\allowbreak\token[0.0043]|\)|\allowbreak\token[0.3067]| is|\allowbreak\token[0.4554]| the|\allowbreak\token[0.3992]| intersection|\allowbreak\token[0.0673]| of|\allowbreak\token[1.3811]| \|\allowbreak\token[0.0018]|(\|\allowbreak\token[0.0008]|over|\allowbreak\token[0.0026]|line|\allowbreak\token[0.0004]|{|\allowbreak\token[1.4155]|AE|\allowbreak\token[0.0014]|}\|\allowbreak\token[0.0002]|)|\allowbreak\token[0.0049]| and|\allowbreak[...]
\end{twopartbox}
\begin{twopartbox}{Cut Index 533 | Acceptance Probability 0.7086}{Block 3 Step 1}{Proposal (Rejected)}
[...]\allowbreak\token[0.0000]| }|\allowbreak\token[0.0000]| BD|\allowbreak\token[0.0196]| =|\allowbreak\token[0.0009]| \|\allowbreak\token[0.2335]|frac|\allowbreak\token[0.0000]|{|\allowbreak\token[0.0014]|3|\allowbreak\token[0.0016]| -|\allowbreak\token[0.0000]| |\allowbreak\token[0.0000]|0|\allowbreak\token[0.0000]|}{|\allowbreak\token[0.0003]|0|\allowbreak\token[0.0001]| -|\allowbreak\token[0.0000]| |\allowbreak\token[0.0009]|0|\allowbreak\token[0.0005]|}|\allowbreak\token[0.3410]| =|\allowbreak\token[0.0005]| \|\allowbreak\token[0.1248]|text|\allowbreak\token[0.0002]|{|\allowbreak\token[0.0005]|undefined|\allowbreak\token[0.5035]| (|\allowbreak\token[0.0009]|vertical|\allowbreak\token[0.0006]| line|\allowbreak\token[0.0009]|)}|\allowbreak\token[0.0000]| \|\allowbreak\token[0.0035]|]↩|\newline\allowbreak\token[0.0017]|\[|\allowbreak\token[0.0002]| \|\allowbreak\token[0.0000]|text|\allowbreak\token[0.0000]|{|\allowbreak\token[0.0000]|equ|\allowbreak\token[0.0000]|ation|\allowbreak\token[0.0000]| of|\allowbreak\token[0.0000]| }|\allowbreak\token[0.0000]| BD|\allowbreak\token[0.0000]|:|\allowbreak\token[0.0004]| x|\allowbreak\token[0.0001]| =|\allowbreak\token[0.0003]| |\allowbreak\token[0.0011]|0|\allowbreak\token[0.0032]| \|\allowbreak\token[0.0094]|]↩↩|\tighttracebreak\allowbreak\token[1.1238]|Since|\allowbreak\token[0.6334]| \(|\allowbreak\token[1.0681]|x|\allowbreak\token[0.0003]| =|\allowbreak\token[0.0011]| |\allowbreak\token[0.0108]|0|\allowbreak\token[0.2912]|\)|\allowbreak\token[0.7673]| for|\allowbreak\token[0.1491]| line|\allowbreak\token[0.0020]| \(|\allowbreak\token[0.0006]|BD|\allowbreak\token[0.0065]|\|\allowbreak\token[0.0005]|),|\allowbreak\token[0.7597]| substit|\allowbreak\token[0.0000]|uting|\allowbreak\token[0.0286]| \(|\allowbreak\token[0.0000]|x|\allowbreak\token[0.0000]| =|\allowbreak\token[0.0000]| |\allowbreak\token[0.0002]|0|\allowbreak\token[0.0000]|\)|\allowbreak\token[0.0238]| into|\allowbreak\token[0.0001]| the|\allowbreak\token[0.0000]| equation|\allowbreak\token[0.0123]| of|\allowbreak\token[0.6916]| line|\allowbreak\token[0.0000]| \(|\allowbreak\token[0.0000]|AE|\allowbreak\token[0.5823]|\|\allowbreak\token[0.0333]|):↩|\newline\allowbreak\token[0.0000]|\[|\allowbreak\token[0.0000]| y|\allowbreak\token[0.0000]| =|\allowbreak\token[0.0004]| -|\allowbreak\token[0.0000]|\|\allowbreak\token[0.0000]|frac|\allowbreak\token[0.0000]|{|\allowbreak\token[0.0000]|3|\allowbreak\token[0.0000]|}{|\allowbreak\token[0.0000]|2|\allowbreak\token[0.0007]|}(|\allowbreak\token[0.0000]|0|\allowbreak\token[0.0000]|)|\allowbreak\token[0.0000]| +|\allowbreak\token[0.0000]| |\allowbreak\token[0.0000]|6|\allowbreak\token[0.1909]| =|\allowbreak\token[0.0000]| |\allowbreak\token[0.0143]|6|\allowbreak\token[0.1572]| \|\allowbreak\token[0.3192]|]↩|\newline\allowbreak\token[1.0599]|So|\allowbreak\token[0.0363]|,|\allowbreak\token[1.0792]| point|\allowbreak\token[0.0000]| \(|\allowbreak\token[0.0015]|F|\allowbreak\token[0.0135]|\)|\allowbreak\token[0.1152]| is|\allowbreak\token[0.3917]| \|\allowbreak\token[0.0032]|((|\allowbreak\token[0.1314]|0|\allowbreak\token[0.0001]|,|\allowbreak\token[0.0315]| |\allowbreak\token[0.8556]|3|\allowbreak\token[0.0461]|)\|\allowbreak\token[0.7838]|).↩↩|\tighttracebreak\allowbreak\token[0.2194]|Now|\allowbreak\token[0.7125]| we|\allowbreak\token[0.1614]| have|\allowbreak\token[0.0097]| the|\allowbreak\token[0.0141]| coordinates|\allowbreak\token[0.2031]| of|\allowbreak\token[1.2021]| the|\allowbreak\token[0.2981]| vertices|\allowbreak\token[0.0002]| of|\allowbreak\token[0.0020]| quadr|\allowbreak\token[0.0000]|ilateral|\allowbreak\token[0.0000]| \(|\allowbreak\token[0.0001]|DB|\allowbreak\token[0.0000]|EF|\allowbreak\token[0.0672]|\|\allowbreak\token[0.5601]|):↩|\newline\allowbreak\token[0.3156]|-|\allowbreak\token[0.0000]| \(|\allowbreak\token[0.0007]|D|\allowbreak\token[0.0018]|(|\allowbreak\token[0.0000]|0|\allowbreak\token[0.0000]|,|\allowbreak\token[0.0001]| |\allowbreak\token[0.0002]|3|\allowbreak\token[0.0000]|)\|\allowbreak\token[0.0000]|)↩|\newline\allowbreak\token[0.0000]|-|\allowbreak\token[0.0000]| \(|\allowbreak\token[0.0363]|B|\allowbreak\token[0.0000]|(|\allowbreak\token[0.0001]|0|\allowbreak\token[0.0000]|,|\allowbreak\token[0.0001]| |\allowbreak\token[0.0000]|0|\allowbreak\token[0.0000]|)\|\allowbreak\cut{[CUT]}\allowbreak\token[0.0000]|)↩|\newline\allowbreak\token[0.0000]|-|\allowbreak\token[0.0000]| \(|\allowbreak\token[0.0008]|E|\allowbreak\token[0.0001]|(|\allowbreak\token[0.0001]|4|\allowbreak\token[0.0000]|,|\allowbreak\token[0.0000]| |\allowbreak\token[0.0000]|0|\allowbreak\token[0.0000]|)\|\allowbreak\token[0.0005]|)↩|\newline\allowbreak\token[0.0000]|-|\allowbreak\token[0.0000]| \(|\allowbreak\token[0.0002]|F|\allowbreak\token[0.1139]|(|\allowbreak\token[0.2045]|0|\allowbreak\token[0.0004]|,|\allowbreak\token[0.0242]| |\allowbreak\token[0.0634]|3|\allowbreak\token[0.0032]|)\|\allowbreak\token[0.0823]|)↩↩|\tighttracebreak\allowbreak\token[1.2590]|We|\allowbreak\token[0.2123]| can|\allowbreak\token[0.7037]| see|\allowbreak\token[0.0000]| that|\allowbreak\token[1.0517]| points|\allowbreak\token[0.0000]| \(|\allowbreak\token[0.0017]|D|\allowbreak\token[0.0001]|\)|\allowbreak\token[0.0000]| and|\allowbreak\token[0.0000]| \(|\allowbreak\token[0.0001]|F|\allowbreak\token[0.0000]|\)|\allowbreak\token[0.7610]| are|\allowbreak\token[0.0120]| the|\allowbreak\token[0.0000]| same|\allowbreak\token[0.0771]|,|\allowbreak\token[0.0522]| so|\allowbreak\token[0.7911]| the|\allowbreak\token[0.0669]| quadr|\allowbreak\token[0.0000]|ilateral|\allowbreak\token[0.1188]| \(|\allowbreak\token[0.0000]|DB|
\tcblower
\token[0.0000]|)↩|\newline\allowbreak\token[0.0000]|-|\allowbreak\token[0.0000]| \(|\allowbreak\token[0.0008]|E|\allowbreak\token[0.0001]|(|\allowbreak\token[0.0001]|4|\allowbreak\token[0.0000]|,|\allowbreak\token[0.0000]| |\allowbreak\token[0.0000]|0|\allowbreak\token[0.0000]|)\|\allowbreak\token[0.0004]|)↩|\newline\allowbreak\token[0.0000]|-|\allowbreak\token[0.0000]| \(|\allowbreak\token[0.0001]|F|\allowbreak\token[0.1370]|(|\allowbreak\token[0.2028]|0|\allowbreak\token[0.0004]|,|\allowbreak\token[0.0242]| |\allowbreak\token[0.0703]|3|\allowbreak\token[0.0032]|)\|\allowbreak\token[0.0908]|)↩↩|\tighttracebreak\allowbreak\token[1.2848]|We|\allowbreak\token[0.2078]| can|\allowbreak\token[0.7242]| see|\allowbreak\token[0.0000]| that|\allowbreak\token[1.0670]| \(|\allowbreak\token[0.1772]|D|\allowbreak\token[0.0064]|\)|\allowbreak\token[0.0000]| and|\allowbreak\token[0.0000]| \(|\allowbreak\token[0.0002]|F|\allowbreak\token[0.0000]|\)|\allowbreak\token[0.5741]| have|\allowbreak\token[0.0000]| the|\allowbreak\token[0.0000]| same|\allowbreak\token[0.0018]| coordinates|\allowbreak\token[0.0090]|,|\allowbreak\token[0.1058]| so|\allowbreak\token[1.2345]| the|\allowbreak\token[0.1364]| quadr|\allowbreak\token[0.0000]|ilateral|\allowbreak\token[0.1776]| \(|\allowbreak\token[0.0000]|DB|
\end{twopartbox}
\begin{twopartbox}{Cut Index 308 | Acceptance Probability 0.0084}{Block 3 Step 2}{Proposal (Rejected)}
[...]\allowbreak\token[0.0000]|left|\allowbreak\token[0.0000]|(|\allowbreak\token[0.0000]| \|\allowbreak\token[0.0000]|frac|\allowbreak\token[0.0000]|{|\allowbreak\token[0.0494]|0|\allowbreak\token[0.0008]|+|\allowbreak\token[0.0000]|8|\allowbreak\token[0.0000]|}{|\allowbreak\token[0.0000]|2|\allowbreak\token[0.0000]|},|\allowbreak\token[0.0000]| \|\allowbreak\token[0.0001]|frac|\allowbreak\token[0.0000]|{|\allowbreak\token[0.0001]|0|\allowbreak\token[0.0000]|+|\allowbreak\token[0.0000]|0|\allowbreak\token[0.0001]|}{|\allowbreak\token[0.0002]|2|\allowbreak\token[0.0000]|}|\allowbreak\token[0.0004]| \|\allowbreak\token[0.0004]|right|\allowbreak\token[0.0001]|)|\allowbreak\token[0.0002]| =|\allowbreak\token[0.0002]| (|\allowbreak\token[0.0000]|4|\allowbreak\token[0.0000]|,|\allowbreak\token[0.0001]| |\allowbreak\token[0.0002]|0|\allowbreak\token[0.0005]|)|\allowbreak\token[0.0000]| \|\allowbreak\token[0.0014]|]↩↩|\tighttracebreak\allowbreak\token[1.5966]|Next|\allowbreak\token[0.0009]|,|\allowbreak\token[0.3245]| we|\allowbreak\token[0.3536]| need|\allowbreak\token[0.1951]| to|\allowbreak\token[0.6223]| find|\allowbreak\token[0.0583]| the|\allowbreak\token[0.1246]| coordinates|\allowbreak\token[0.0007]| of|\allowbreak\token[0.6730]| point|\allowbreak\token[0.0001]| \(|\allowbreak\token[0.0021]|F|\allowbreak\token[0.0117]|\|\allowbreak\token[0.6301]|).|\allowbreak\token[1.5853]| Since|\allowbreak\token[0.1894]| \(|\allowbreak\token[0.2436]|F|\allowbreak\token[0.0004]|\)|\allowbreak\token[0.5669]| is|\allowbreak\token[0.8988]| the|\allowbreak\token[0.3516]| intersection|\allowbreak\token[0.0141]| of|\allowbreak\token[1.5777]| line|\allowbreak\token[0.6444]| segments|\allowbreak\token[0.6633]| \(|\allowbreak\token[1.0602]|AE|\allowbreak\token[0.0003]|\)|\allowbreak\token[0.0001]| and|\allowbreak\token[0.0018]| \(|\allowbreak\token[0.9437]|BD|\allowbreak\token[0.0017]|\|\allowbreak\token[0.0306]|),|\allowbreak\token[0.0350]| we|\allowbreak\token[1.2384]| need|\allowbreak\token[0.0125]| to|\allowbreak\token[0.9057]| find|\allowbreak\token[0.0356]| the|\allowbreak\token[0.1261]| equations|\allowbreak\token[0.0003]| of|\allowbreak\token[0.6777]| these|\allowbreak\token[0.0235]| lines|\allowbreak\token[0.5417]| and|\allowbreak\token[0.6232]| solve|\allowbreak\token[0.0630]| for|\allowbreak\token[0.1762]| their|\allowbreak\token[0.0003]| intersection|\allowbreak\token[0.1750]|.↩↩|\tighttracebreak\allowbreak\token[1.3623]|The|\allowbreak\token[0.5901]| equation|\allowbreak\token[0.0022]| of|\allowbreak\token[0.0118]| line|\allowbreak\token[0.0297]| \(|\allowbreak\token[0.0586]|AE|\allowbreak\token[0.6934]|\)|\allowbreak\token[0.6600]| can|\allowbreak\token[0.0000]| be|\allowbreak\token[0.8074]| found|\allowbreak\token[0.1120]| using|\allowbreak\token[0.1335]| the|\allowbreak\token[0.6864]| points|\allowbreak\token[0.0003]| \(|\allowbreak\token[0.0001]|A|\allowbreak\token[0.0006]|(|\allowbreak\token[0.0000]|0|\allowbreak\token[0.0000]|,|\allowbreak\token[0.0605]| |\allowbreak\token[0.0001]|6|\allowbreak\token[0.0000]|)\|\allowbreak\token[0.0000]|)|\allowbreak\token[0.0000]| and|\allowbreak\token[0.0000]| \(|\allowbreak\token[0.0002]|E|\allowbreak\token[0.0001]|(|\allowbreak\token[0.0001]|4|\allowbreak\token[0.0000]|,|\allowbreak\token[0.0001]| |\allowbreak\token[0.0000]|0|\allowbreak\token[0.0001]|)\|\allowbreak\token[0.0941]|):↩|\newline\allowbreak\token[0.3132]|\[|\allowbreak\token[0.0095]| \|\allowbreak\token[0.0002]|text|\allowbreak\token[0.3148]|{s|\allowbreak\token[0.0000]|lope|\allowbreak\token[0.0118]| of|\allowbreak\token[0.0013]| }|\allowbreak\token[0.0003]| AE|\allowbreak\token[0.0002]| =|\allowbreak\token[0.0000]| \|\allowbreak\token[0.0000]|frac|\allowbreak\token[0.0000]|{|\allowbreak\token[0.0016]|0|\allowbreak\token[0.6759]| -|\allowbreak\token[0.0000]| |\allowbreak\token[0.0000]|6|\allowbreak\token[0.0000]|}{|\allowbreak\token[0.0000]|4|\allowbreak\token[0.0000]| -|\allowbreak\token[0.0000]| |\allowbreak\token[0.0000]|0|\allowbreak\token[0.0000]|}|\allowbreak\token[0.0000]| =|\allowbreak\token[0.0014]| -|\allowbreak\token[0.0020]|\|\allowbreak\token[0.0000]|frac|\allowbreak\token[0.0000]|{|\allowbreak\token[0.0173]|3|\allowbreak\token[0.0000]|}{|\allowbreak\token[0.0000]|2|\allowbreak\token[0.0000]|}|\allowbreak\token[0.0000]| \|\allowbreak\token[0.0002]|]↩|\newline\allowbreak\cut{[CUT]}\allowbreak\token[1.2316]|\[|\allowbreak\token[0.6783]| \|\allowbreak\token[0.0000]|text|\allowbreak\token[0.0000]|{|\allowbreak\token[0.4240]|equ|\allowbreak\token[0.0000]|ation|\allowbreak\token[0.0025]| of|\allowbreak\token[0.0450]| }|\allowbreak\token[0.0000]| AE|\allowbreak\token[0.0388]|:|\allowbreak\token[0.0022]| y|\allowbreak\token[0.3396]| -|\allowbreak\token[0.0000]| |\allowbreak\token[0.0817]|6|\allowbreak\token[0.0000]| =|\allowbreak\token[0.0000]| -|\allowbreak\token[0.0000]|\|\allowbreak\token[0.0000]|frac|\allowbreak\token[0.0000]|{|\allowbreak\token[0.0000]|3|\allowbreak\token[0.0000]|}{|\allowbreak\token[0.0000]|2|\allowbreak\token[0.6913]|}|\allowbreak\token[0.0671]|x|\allowbreak\token[0.0001]| \|\allowbreak\token[0.9195]|impl|\allowbreak\token[0.0000]|ies|\allowbreak\token[0.0031]| y|\allowbreak\token[0.0000]| =|\allowbreak\token[0.1324]| -|\allowbreak\token[0.0000]|\|\allowbreak\token[0.0000]|frac|\allowbreak\token[0.0000]|{|\allowbreak\token[0.0000]|3|\allowbreak\token[0.0000]|}{|\allowbreak\token[0.0000]|2|\allowbreak\token[0.0000]|}|\allowbreak\token[0.0000]|x|\allowbreak\token[0.0000]| +|\allowbreak\token[0.0000]| |\allowbreak\token[0.0000]|6|\allowbreak\token[0.0000]| \|\allowbreak\token[0.0001]|]↩↩|\tighttracebreak\allowbreak\token[0.0001]|The|\allowbreak\token[0.0000]| equation|\allowbreak\token[0.0000]| of|\allowbreak\token[0.0000]| line|\allowbreak\token[0.0000]| \(|\allowbreak\token[0.0001]|BD|\allowbreak\token[0.0001]|\)|\allowbreak\token[0.0003]| can|\allowbreak\token[0.0000]| be|\allowbreak\token[0.0000]| found|\allowbreak\token[0.0000]| using|\allowbreak\token[0.0014]| the|\allowbreak\token[0.0000]| points|\allowbreak\token[0.0000]| \(|\allowbreak\token[0.0008]|B|\allowbreak\token[0.0000]|(|\allowbreak\token[0.0000]|0|\allowbreak\token[0.0000]|,|\allowbreak\token[0.0006]| |\allowbreak\token[0.0000]|0|\allowbreak\token[0.0000]|)\|\allowbreak\token[0.0007]|)|\allowbreak\token[0.0000]| and|\allowbreak\token[0.0000]| \(|\allowbreak\token[0.0003]|D|\allowbreak\token[0.0001]|(|\allowbreak\token[0.0002]|0|\allowbreak\token[0.0000]|,|\allowbreak\token[0.0000]| |\allowbreak\token[0.0004]|3|\allowbreak\token[0.0001]|)\|\allowbreak\token[0.0613]|):↩|\newline\allowbreak\token[0.1976]|\[|\allowbreak\token[0.0011]| \|\allowbreak\token[0.0000]|text|\allowbreak\token[0.2474]|{s|\allowbreak\token[0.0000]|lope|\allowbreak\token[0.0000]| of|\allowbreak\token[0.0000]| }|\allowbreak\token[0.0000]| BD|\allowbreak\token[0.0196]| =|\allowbreak\token[0.0009]| \|\allowbreak\token[0.2335]|frac|\allowbreak\token[0.0000]|{|\allowbreak\token[0.0014]|3|\allowbreak\token[0.0016]| -|\allowbreak\token[0.0000]| |\allowbreak\token[0.0000]|0|\allowbreak\token[0.0000]|}{|\allowbreak\token[0.0003]|0|\allowbreak\token[0.0001]| -|\allowbreak\token[0.0000]| |\allowbreak\token[0.0009]|0|\allowbreak\token[0.0005]|}|\allowbreak\token[0.3410]| =|\allowbreak\token[0.0005]| \|\allowbreak\token[0.1248]|text|\allowbreak\token[0.0002]|{|\allowbreak\token[0.0005]|undefined|\allowbreak\token[0.5035]| (|\allowbreak\token[0.0009]|vertical|\allowbreak\token[0.0006]| line|\allowbreak\token[0.0009]|)}|\allowbreak\token[0.0000]| \|\allowbreak\token[0.0035]|]↩|\newline\allowbreak\token[0.0017]|\[|\allowbreak\token[0.0002]| \|\allowbreak\token[0.0000]|text|\allowbreak\token[0.0000]|{|\allowbreak\token[0.0000]|equ|\allowbreak\token[0.0000]|ation|\allowbreak\token[0.0000]| of|\allowbreak\token[0.0000]| }|\allowbreak\token[0.0000]| BD|\allowbreak\token[0.0000]|:|\allowbreak\token[0.0004]| x|\allowbreak\token[0.0001]| =|\allowbreak\token[0.0003]| |\allowbreak\token[0.0011]|0|\allowbreak\token[0.0032]| \|\allowbreak\token[0.0094]|]↩↩|\tighttracebreak\allowbreak\token[1.1238]|Since|\allowbreak\token[0.6334]| \(|\allowbreak\token[1.0681]|x|\allowbreak\token[0.0003]| =|\allowbreak\token[0.0011]| |\allowbreak\token[0.0108]|0|\allowbreak\token[0.2912]|\)|\allowbreak\token[0.7673]| for|\allowbreak\token[0.1491]| line|\allowbreak\token[0.0020]| \(|\allowbreak\token[0.0006]|BD|\allowbreak\token[0.0065]|\|\allowbreak\token[0.0005]|),|\allowbreak\token[0.7597]| substit|\allowbreak\token[0.0000]|uting|\allowbreak\token[0.0286]| \(|\allowbreak\token[0.0000]|x|\allowbreak\token[0.0000]| =|\allowbreak\token[0.0000]| |\allowbreak\token[0.0002]|0|\allowbreak[...]
\tcblower
\token[1.2326]|\[|\allowbreak\token[0.6877]| y|\allowbreak\token[0.5032]| -|\allowbreak\token[0.0002]| |\allowbreak\token[0.0363]|6|\allowbreak\token[0.0000]| =|\allowbreak\token[0.0000]| -|\allowbreak\token[0.0000]|\|\allowbreak\token[0.0000]|frac|\allowbreak\token[0.0000]|{|\allowbreak\token[0.0000]|3|\allowbreak\token[0.0000]|}{|\allowbreak\token[0.0000]|2|\allowbreak\token[0.6630]|}|\allowbreak\token[0.1750]|x|\allowbreak\token[0.0000]| \|\allowbreak\token[0.6787]|]↩|\newline\allowbreak\token[0.0002]|\[|\allowbreak\token[0.0154]| y|\allowbreak\token[0.0000]| =|\allowbreak\token[0.0265]| -|\allowbreak\token[0.0000]|\|\allowbreak\token[0.0000]|frac|\allowbreak\token[0.0000]|{|\allowbreak\token[0.0000]|3|\allowbreak\token[0.0000]|}{|\allowbreak\token[0.0000]|2|\allowbreak\token[0.0000]|}|\allowbreak\token[0.0000]|x|\allowbreak\token[0.0000]| +|\allowbreak\token[0.0000]| |\allowbreak\token[0.0000]|6|\allowbreak\token[0.0000]| \|\allowbreak\token[0.0028]|]↩↩|\tighttracebreak\allowbreak\token[0.0002]|The|\allowbreak\token[0.0000]| equation|\allowbreak\token[0.0000]| of|\allowbreak\token[0.0001]| line|\allowbreak\token[0.0000]| \(|\allowbreak\token[0.0001]|BD|\allowbreak\token[0.0001]|\)|\allowbreak\token[0.0008]| can|\allowbreak\token[0.0000]| be|\allowbreak\token[0.0000]| found|\allowbreak\token[0.0000]| using|\allowbreak\token[0.0019]| the|\allowbreak\token[0.0000]| points|\allowbreak\token[0.0000]| \(|\allowbreak\token[0.0012]|B|\allowbreak\token[0.0000]|(|\allowbreak\token[0.0000]|0|\allowbreak\token[0.0000]|,|\allowbreak\token[0.0006]| |\allowbreak\token[0.0000]|0|\allowbreak\token[0.0000]|)\|\allowbreak\token[0.0008]|)|\allowbreak\token[0.0000]| and|\allowbreak\token[0.0000]| \(|\allowbreak\token[0.0005]|D|\allowbreak\token[0.0001]|(|\allowbreak\token[0.0004]|0|\allowbreak\token[0.0001]|,|\allowbreak\token[0.0000]| |\allowbreak\token[0.0006]|3|\allowbreak\token[0.0002]|)\|\allowbreak\token[0.0925]|):↩|\newline\allowbreak\token[0.3875]|\[|\allowbreak\token[0.1053]| \|\allowbreak\token[0.0001]|text|\allowbreak\token[0.0113]|{s|\allowbreak\token[0.0000]|lope|\allowbreak\token[0.0000]| of|\allowbreak\token[0.0000]| }|\allowbreak\token[0.0001]| BD|\allowbreak\token[0.0240]| =|\allowbreak\token[0.0014]| \|\allowbreak\token[0.2967]|frac|\allowbreak\token[0.0000]|{|\allowbreak\token[0.0016]|3|\allowbreak\token[0.0025]| -|\allowbreak\token[0.0000]| |\allowbreak\token[0.0000]|0|\allowbreak\token[0.0000]|}{|\allowbreak\token[0.0003]|0|\allowbreak\token[0.0001]| -|\allowbreak\token[0.0000]| |\allowbreak\token[0.0011]|0|\allowbreak\token[0.0004]|}|\allowbreak\token[0.3406]| =|\allowbreak\token[0.0006]| \|\allowbreak\token[0.1284]|text|\allowbreak\token[0.0002]|{|\allowbreak\token[0.0005]|undefined|\allowbreak\token[0.6281]| (|\allowbreak\token[0.0010]|vertical|\allowbreak\token[0.0007]| line|\allowbreak\token[0.0011]|)}|\allowbreak\token[0.0000]| \|\allowbreak\token[0.0002]|]↩|\newline\allowbreak\token[0.1766]|\[|\allowbreak\token[0.0021]| x|\allowbreak\token[0.0001]| =|\allowbreak\token[0.0003]| |\allowbreak\token[0.0049]|0|\allowbreak\token[0.0014]| \|\allowbreak\token[0.0276]|]↩↩|\tighttracebreak\allowbreak\token[1.5666]|To|\allowbreak\token[0.0003]| find|\allowbreak\token[0.0988]| the|\allowbreak\token[0.1618]| intersection|\allowbreak\token[1.0381]| of|\allowbreak\token[0.7125]| these|\allowbreak\token[0.0495]| lines|\allowbreak\token[0.4346]|,|\allowbreak\token[0.4768]| we|\allowbreak\token[0.0024]| substitute|\allowbreak\token[0.0000]| \(|\allowbreak\token[0.0101]|x|\allowbreak\token[0.0003]| =|\allowbreak\token[0.0101]| |\allowbreak\token[0.0309]|0|\allowbreak\token[0.0000]|\)|\allowbreak\token[0.0004]| into|\allowbreak\token[0.0000]| the|\allowbreak\token[0.0001]| equation|\allowbreak\token[0.0880]| of|\allowbreak\token[0.3398]| line|\allowbreak\token[0.0000]| \(|\allowbreak\token[0.0004]|AE|\allowbreak\token[0.0000]|\|\allowbreak\token[0.0000]|):↩|\newline\allowbreak\token[0.0000]|\[|\allowbreak\token[0.0003]| y|\allowbreak\token[0.0000]| =|\allowbreak\token[0.0012]| -|\allowbreak\token[0.0000]|\|\allowbreak\token[0.0000]|frac|\allowbreak\token[0.0000]|{|\allowbreak\token[0.0000]|3|\allowbreak\token[0.0000]|}{|\allowbreak\token[0.0000]|2|\allowbreak\token[0.0011]|}(|\allowbreak\token[0.0000]|0|\allowbreak\token[0.0000]|)|\allowbreak[...]
\end{twopartbox}
\begin{twopartbox}{Cut Index 542 | Acceptance Probability 1.0000}{Block 3 Step 3}{Proposal (Accepted)}
[...]\allowbreak\token[0.0000]|0|\allowbreak\token[0.0000]|}{|\allowbreak\token[0.0003]|0|\allowbreak\token[0.0001]| -|\allowbreak\token[0.0000]| |\allowbreak\token[0.0009]|0|\allowbreak\token[0.0005]|}|\allowbreak\token[0.3410]| =|\allowbreak\token[0.0005]| \|\allowbreak\token[0.1248]|text|\allowbreak\token[0.0002]|{|\allowbreak\token[0.0005]|undefined|\allowbreak\token[0.5035]| (|\allowbreak\token[0.0009]|vertical|\allowbreak\token[0.0006]| line|\allowbreak\token[0.0009]|)}|\allowbreak\token[0.0000]| \|\allowbreak\token[0.0035]|]↩|\newline\allowbreak\token[0.0017]|\[|\allowbreak\token[0.0002]| \|\allowbreak\token[0.0000]|text|\allowbreak\token[0.0000]|{|\allowbreak\token[0.0000]|equ|\allowbreak\token[0.0000]|ation|\allowbreak\token[0.0000]| of|\allowbreak\token[0.0000]| }|\allowbreak\token[0.0000]| BD|\allowbreak\token[0.0000]|:|\allowbreak\token[0.0004]| x|\allowbreak\token[0.0001]| =|\allowbreak\token[0.0003]| |\allowbreak\token[0.0011]|0|\allowbreak\token[0.0032]| \|\allowbreak\token[0.0094]|]↩↩|\tighttracebreak\allowbreak\token[1.1238]|Since|\allowbreak\token[0.6334]| \(|\allowbreak\token[1.0681]|x|\allowbreak\token[0.0003]| =|\allowbreak\token[0.0011]| |\allowbreak\token[0.0108]|0|\allowbreak\token[0.2912]|\)|\allowbreak\token[0.7673]| for|\allowbreak\token[0.1491]| line|\allowbreak\token[0.0020]| \(|\allowbreak\token[0.0006]|BD|\allowbreak\token[0.0065]|\|\allowbreak\token[0.0005]|),|\allowbreak\token[0.7597]| substit|\allowbreak\token[0.0000]|uting|\allowbreak\token[0.0286]| \(|\allowbreak\token[0.0000]|x|\allowbreak\token[0.0000]| =|\allowbreak\token[0.0000]| |\allowbreak\token[0.0002]|0|\allowbreak\token[0.0000]|\)|\allowbreak\token[0.0238]| into|\allowbreak\token[0.0001]| the|\allowbreak\token[0.0000]| equation|\allowbreak\token[0.0123]| of|\allowbreak\token[0.6916]| line|\allowbreak\token[0.0000]| \(|\allowbreak\token[0.0000]|AE|\allowbreak\token[0.5823]|\|\allowbreak\token[0.0333]|):↩|\newline\allowbreak\token[0.0000]|\[|\allowbreak\token[0.0000]| y|\allowbreak\token[0.0000]| =|\allowbreak\token[0.0004]| -|\allowbreak\token[0.0000]|\|\allowbreak\token[0.0000]|frac|\allowbreak\token[0.0000]|{|\allowbreak\token[0.0000]|3|\allowbreak\token[0.0000]|}{|\allowbreak\token[0.0000]|2|\allowbreak\token[0.0007]|}(|\allowbreak\token[0.0000]|0|\allowbreak\token[0.0000]|)|\allowbreak\token[0.0000]| +|\allowbreak\token[0.0000]| |\allowbreak\token[0.0000]|6|\allowbreak\token[0.1909]| =|\allowbreak\token[0.0000]| |\allowbreak\token[0.0143]|6|\allowbreak\token[0.1572]| \|\allowbreak\token[0.3192]|]↩|\newline\allowbreak\token[1.0599]|So|\allowbreak\token[0.0363]|,|\allowbreak\token[1.0792]| point|\allowbreak\token[0.0000]| \(|\allowbreak\token[0.0015]|F|\allowbreak\token[0.0135]|\)|\allowbreak\token[0.1152]| is|\allowbreak\token[0.3917]| \|\allowbreak\token[0.0032]|((|\allowbreak\token[0.1314]|0|\allowbreak\token[0.0001]|,|\allowbreak\token[0.0315]| |\allowbreak\token[0.8556]|3|\allowbreak\token[0.0461]|)\|\allowbreak\token[0.7838]|).↩↩|\tighttracebreak\allowbreak\token[0.2194]|Now|\allowbreak\token[0.7125]| we|\allowbreak\token[0.1614]| have|\allowbreak\token[0.0097]| the|\allowbreak\token[0.0141]| coordinates|\allowbreak\token[0.2031]| of|\allowbreak\token[1.2021]| the|\allowbreak\token[0.2981]| vertices|\allowbreak\token[0.0002]| of|\allowbreak\token[0.0020]| quadr|\allowbreak\token[0.0000]|ilateral|\allowbreak\token[0.0000]| \(|\allowbreak\token[0.0001]|DB|\allowbreak\token[0.0000]|EF|\allowbreak\token[0.0672]|\|\allowbreak\token[0.5601]|):↩|\newline\allowbreak\token[0.3156]|-|\allowbreak\token[0.0000]| \(|\allowbreak\token[0.0007]|D|\allowbreak\token[0.0018]|(|\allowbreak\token[0.0000]|0|\allowbreak\token[0.0000]|,|\allowbreak\token[0.0001]| |\allowbreak\token[0.0002]|3|\allowbreak\token[0.0000]|)\|\allowbreak\token[0.0000]|)↩|\newline\allowbreak\token[0.0000]|-|\allowbreak\token[0.0000]| \(|\allowbreak\token[0.0363]|B|\allowbreak\token[0.0000]|(|\allowbreak\token[0.0001]|0|\allowbreak\token[0.0000]|,|\allowbreak\token[0.0001]| |\allowbreak\token[0.0000]|0|\allowbreak\token[0.0000]|)\|\allowbreak\token[0.0000]|)↩|\newline\allowbreak\token[0.0000]|-|\allowbreak\token[0.0000]| \(|\allowbreak\token[0.0008]|E|\allowbreak\token[0.0001]|(|\allowbreak\token[0.0001]|4|\allowbreak\token[0.0000]|,|\allowbreak\token[0.0000]| |\allowbreak\token[0.0000]|0|\allowbreak\cut{[CUT]}\allowbreak\token[0.0000]|)\|\allowbreak\token[0.0005]|)↩|\newline\allowbreak\token[0.0000]|-|\allowbreak\token[0.0000]| \(|\allowbreak\token[0.0002]|F|\allowbreak\token[0.1139]|(|\allowbreak\token[0.2045]|0|\allowbreak\token[0.0004]|,|\allowbreak\token[0.0242]| |\allowbreak\token[0.0634]|3|\allowbreak\token[0.0032]|)\|\allowbreak\token[0.0823]|)↩↩|\tighttracebreak\allowbreak\token[1.2590]|We|\allowbreak\token[0.2123]| can|\allowbreak\token[0.7037]| see|\allowbreak\token[0.0000]| that|\allowbreak\token[1.0517]| points|\allowbreak\token[0.0000]| \(|\allowbreak\token[0.0017]|D|\allowbreak\token[0.0001]|\)|\allowbreak\token[0.0000]| and|\allowbreak\token[0.0000]| \(|\allowbreak\token[0.0001]|F|\allowbreak\token[0.0000]|\)|\allowbreak\token[0.7610]| are|\allowbreak\token[0.0120]| the|\allowbreak\token[0.0000]| same|\allowbreak\token[0.0771]|,|\allowbreak\token[0.0522]| so|\allowbreak\token[0.7911]| the|\allowbreak\token[0.0669]| quadr|\allowbreak\token[0.0000]|ilateral|\allowbreak\token[0.1188]| \(|\allowbreak\token[0.0000]|DB|
\tcblower
\token[0.0000]|)\|\allowbreak\token[0.0005]|)↩|\newline\allowbreak\token[0.0000]|-|\allowbreak\token[0.0000]| \(|\allowbreak\token[0.0001]|F|\allowbreak\token[0.1370]|(|\allowbreak\token[0.2028]|0|\allowbreak\token[0.0004]|,|\allowbreak\token[0.0269]| |\allowbreak\token[0.0637]|3|\allowbreak\token[0.0029]|)\|\allowbreak\token[0.0908]|)↩↩|\tighttracebreak\allowbreak\token[1.2879]|Notice|\allowbreak\token[0.0000]| that|\allowbreak\token[1.1315]| points|\allowbreak\token[0.0000]| \(|\allowbreak\token[0.0022]|D|\allowbreak\token[0.0001]|\)|\allowbreak\token[0.0000]| and|\allowbreak\token[0.0000]| \(|\allowbreak\token[0.0001]|F|\allowbreak\token[0.0000]|\)|\allowbreak\token[0.6432]| have|\allowbreak\token[0.0000]| the|\allowbreak\token[0.0000]| same|\allowbreak\token[0.0008]| coordinates|\allowbreak\token[0.0380]|,|\allowbreak\token[0.1974]| so|\allowbreak\token[1.2532]| quadr|\allowbreak\token[0.0000]|ilateral|\allowbreak\token[0.0000]| \(|\allowbreak\token[0.0000]|DB|\allowbreak\token[0.0001]|EF|\allowbreak\token[0.0000]|\)|
\end{twopartbox}
\begin{twopartbox}{Cut Index 696 | Acceptance Probability 1.0000}{Block 4 Step 1}{Proposal (Accepted)}
[...]\allowbreak\token[0.0029]|)\|\allowbreak\token[0.0908]|)↩↩|\tighttracebreak\allowbreak\token[1.2879]|Notice|\allowbreak\token[0.0000]| that|\allowbreak\token[1.1315]| points|\allowbreak\token[0.0000]| \(|\allowbreak\token[0.0022]|D|\allowbreak\token[0.0001]|\)|\allowbreak\token[0.0000]| and|\allowbreak\token[0.0000]| \(|\allowbreak\token[0.0001]|F|\allowbreak\token[0.0000]|\)|\allowbreak\token[0.6432]| have|\allowbreak\token[0.0000]| the|\allowbreak\token[0.0000]| same|\allowbreak\token[0.0008]| coordinates|\allowbreak\token[0.0380]|,|\allowbreak\token[0.1974]| so|\allowbreak\token[1.2532]| quadr|\allowbreak\token[0.0000]|ilateral|\allowbreak\token[0.0000]| \(|\allowbreak\token[0.0000]|DB|\allowbreak\token[0.0001]|EF|\allowbreak\token[0.0000]|\)|\allowbreak\token[0.1320]| is|\allowbreak\token[0.0418]| actually|\allowbreak\token[0.0618]| a|\allowbreak\token[0.2454]| triangle|\allowbreak\token[0.2995]| \(|\allowbreak\token[0.5223]|DB|\allowbreak\token[0.0011]|E|\allowbreak\token[0.6468]|\|\allowbreak\token[0.5405]|).↩↩|\tighttracebreak\allowbreak\token[1.0834]|To|\allowbreak\token[0.0011]| find|\allowbreak\token[0.0000]| the|\allowbreak\token[0.0000]| area|\allowbreak\token[0.0000]| of|\allowbreak\token[0.0327]| triangle|\allowbreak\token[0.0000]| \(|\allowbreak\token[0.0001]|DB|\allowbreak\token[0.0000]|E|\allowbreak\token[0.0024]|\|\allowbreak\token[0.1325]|),|\allowbreak\token[0.0000]| we|\allowbreak\token[0.2475]| can|\allowbreak\token[0.0000]| use|\allowbreak\token[0.0000]| the|\allowbreak\token[0.1487]| formula|\allowbreak\token[0.0006]| for|\allowbreak\token[0.0000]| the|\allowbreak\token[0.0000]| area|\allowbreak\token[0.0000]| of|\allowbreak\token[0.0000]| a|\allowbreak\token[0.0012]| triangle|\allowbreak\token[0.5992]| with|\allowbreak\token[0.0148]| vertices|\allowbreak\token[0.1323]| \|\allowbreak\token[0.0000]|((|\allowbreak\token[0.0000]|x|\allowbreak\token[0.0001]|_|\allowbreak\token[0.0000]|1|\allowbreak\token[0.0000]|,|\allowbreak\token[0.0000]| y|\allowbreak\token[0.0000]|_|\allowbreak\token[0.0000]|1|\allowbreak\token[0.0000]|)\|\allowbreak\token[0.0000]|),|\allowbreak\token[0.0000]| \|\allowbreak\token[0.0000]|((|\allowbreak\token[0.0000]|x|\allowbreak\token[0.0000]|_|\allowbreak\token[0.0000]|2|\allowbreak\token[0.0000]|,|\allowbreak\token[0.0000]| y|\allowbreak\token[0.0001]|_|\allowbreak\token[0.0001]|2|\allowbreak\token[0.0000]|)\|\allowbreak\token[0.0000]|),|\allowbreak\token[0.6629]| and|\allowbreak\token[0.0000]| \|\allowbreak\token[0.0003]|((|\allowbreak\token[0.0000]|x|\allowbreak\token[0.0000]|_|\allowbreak\token[0.0000]|3|\allowbreak\token[0.0000]|,|\allowbreak\token[0.0001]| y|\allowbreak\token[0.0002]|_|\allowbreak\token[0.0001]|3|\allowbreak\token[0.0000]|)\|\allowbreak\token[0.0003]|):↩|\newline\allowbreak\token[0.0125]|\[|\allowbreak\token[0.0000]| \|\allowbreak\token[0.0000]|text|\allowbreak\token[0.0000]|{|\allowbreak\token[0.0000]|Area|\allowbreak\token[0.0000]|}|\allowbreak\token[0.0000]| =|\allowbreak\token[0.0000]| \|\allowbreak\token[0.0000]|frac|\allowbreak\token[0.0000]|{|\allowbreak\token[0.0000]|1|\allowbreak\token[0.0000]|}{|\allowbreak\token[0.0000]|2|\allowbreak\token[0.0000]|}|\allowbreak\token[0.0000]| \|\allowbreak\token[0.0000]|left|\allowbreak\token[0.0000]!|!\allowbreak\token[0.0000]| x|\allowbreak\token[0.0000]|_|\allowbreak\token[0.0000]|1|\allowbreak\token[0.0000]|(y|\allowbreak\token[0.0000]|_|\allowbreak\token[0.0000]|2|\allowbreak\token[0.0003]| -|\allowbreak\token[0.0000]| y|\allowbreak\token[0.0000]|_|\allowbreak\token[0.0000]|3|\allowbreak\token[0.0000]|)|\allowbreak\token[0.0000]| +|\allowbreak\token[0.0000]| x|\allowbreak\token[0.0000]|_|\allowbreak\token[0.0000]|2|\allowbreak\token[0.0000]|(y|\allowbreak\token[0.0000]|_|\allowbreak\token[0.0000]|3|\allowbreak\token[0.0000]| -|\allowbreak\token[0.0000]| y|\allowbreak\token[0.0001]|_|\allowbreak\token[0.0001]|1|\allowbreak\token[0.0000]|)|\allowbreak\token[0.0000]| +|\allowbreak\token[0.0000]| x|\allowbreak\token[0.0000]|_|\allowbreak\token[0.0000]|3|\allowbreak\token[0.0000]|(y|\allowbreak\token[0.0000]|_|\allowbreak\token[0.0000]|1|\allowbreak\token[0.0000]| -|\allowbreak\token[0.0000]| y|\allowbreak\token[0.0001]|_|\allowbreak\token[0.0000]|2|\allowbreak\token[0.0000]|)|\allowbreak\token[0.0000]| \|\allowbreak\cut{[CUT]}\allowbreak\token[0.0001]|right|\allowbreak\token[0.0000]!|!\allowbreak\token[0.0000]| \|\allowbreak\token[0.1324]|]↩↩|\tighttracebreak\allowbreak\token[0.1297]|Sub|\allowbreak\token[0.0073]|stit|\allowbreak\token[0.0000]|uting|\allowbreak\token[0.0551]| the|\allowbreak\token[0.0002]| coordinates|\allowbreak\token[0.0979]| of|\allowbreak\token[0.0405]| \(|\allowbreak\token[0.0003]|D|\allowbreak\token[0.0053]|(|\allowbreak\token[0.0000]|0|\allowbreak\token[0.0000]|,|\allowbreak\token[0.0015]| |\allowbreak\token[0.0000]|3|\allowbreak\token[0.0000]|)\|\allowbreak\token[0.0000]|),|\allowbreak\token[0.0000]| \(|\allowbreak\token[0.0010]|B|\allowbreak\token[0.0000]|(|\allowbreak\token[0.0000]|0|\allowbreak\token[0.0000]|,|\allowbreak\token[0.0000]| |\allowbreak\token[0.0000]|0|\allowbreak\token[0.0000]|)\|\allowbreak\token[0.0000]|),|\allowbreak\token[0.0001]| and|\allowbreak\token[0.0000]| \(|\allowbreak\token[0.0000]|E|\allowbreak\token[0.0000]|(|\allowbreak\token[0.0000]|4|\allowbreak\token[0.0000]|,|\allowbreak\token[0.0000]| |\allowbreak\token[0.0000]|0|\allowbreak\token[0.0000]|)\|\allowbreak\token[0.0002]|):↩|\newline\allowbreak\token[0.0002]|\[|\allowbreak\token[0.0000]| \|\allowbreak\token[0.0000]|text|\allowbreak\token[0.0000]|{|\allowbreak\token[0.0000]|Area|\allowbreak\token[0.0002]|}|\allowbreak\token[0.0000]| =|\allowbreak\token[0.0000]| \|\allowbreak\token[0.0000]|frac|\allowbreak\token[0.0000]|{|\allowbreak\token[0.0000]|1|\allowbreak\token[0.0001]|}{|\allowbreak\token[0.0000]|2|\allowbreak\token[0.0000]|}|\allowbreak\token[0.0000]| \|\allowbreak\token[0.0000]|left|\allowbreak\token[0.0000]!|!\allowbreak\token[0.0000]| |\allowbreak\token[0.0000]|0|\allowbreak\token[0.0001]|(|\allowbreak\token[0.0002]|0|\allowbreak\token[0.0000]| -|\allowbreak\token[0.0000]| |\allowbreak\token[0.0195]|0|\allowbreak\token[0.0000]|)|\allowbreak\token[0.0000]| +|\allowbreak\token[0.0000]| |\allowbreak\token[0.0001]|0|\allowbreak\token[0.0001]|(|\allowbreak\token[0.0006]|0|\allowbreak\token[0.0000]| -|\allowbreak\token[0.0000]| |\allowbreak\token[0.0001]|3|\allowbreak\token[0.0000]|)|
\tcblower
\token[0.0001]|right|\allowbreak\token[0.0000]!|!\allowbreak\token[0.0000]| \|\allowbreak\token[0.1204]|]↩↩|\tighttracebreak\allowbreak\token[0.1209]|Sub|\allowbreak\token[0.0073]|stit|\allowbreak\token[0.0000]|uting|\allowbreak\token[0.0565]| the|\allowbreak\token[0.0002]| coordinates|\allowbreak\token[0.1066]| of|\allowbreak\token[0.0405]| \(|\allowbreak\token[0.0003]|D|\allowbreak\token[0.0053]|(|\allowbreak\token[0.0000]|0|\allowbreak\token[0.0000]|,|\allowbreak\token[0.0015]| |\allowbreak\token[0.0000]|3|\allowbreak\token[0.0000]|)\|\allowbreak\token[0.0000]|),|\allowbreak\token[0.0000]| \(|\allowbreak\token[0.0010]|B|\allowbreak\token[0.0000]|(|\allowbreak\token[0.0000]|0|\allowbreak\token[0.0000]|,|\allowbreak\token[0.0000]| |\allowbreak\token[0.0000]|0|\allowbreak\token[0.0000]|)\|\allowbreak\token[0.0000]|),|\allowbreak\token[0.0001]| and|\allowbreak\token[0.0000]| \(|\allowbreak\token[0.0000]|E|\allowbreak\token[0.0000]|(|\allowbreak\token[0.0000]|4|\allowbreak\token[0.0000]|,|\allowbreak\token[0.0000]| |\allowbreak\token[0.0000]|0|\allowbreak\token[0.0000]|)\|\allowbreak\token[0.0002]|):↩|\newline\allowbreak\token[0.0002]|\[|\allowbreak\token[0.0000]| \|\allowbreak\token[0.0000]|text|\allowbreak\token[0.0000]|{|\allowbreak\token[0.0000]|Area|\allowbreak\token[0.0002]|}|\allowbreak\token[0.0000]| =|\allowbreak\token[0.0000]| \|\allowbreak\token[0.0000]|frac|\allowbreak\token[0.0000]|{|\allowbreak\token[0.0000]|1|\allowbreak\token[0.0001]|}{|\allowbreak\token[0.0000]|2|\allowbreak\token[0.0000]|}|\allowbreak\token[0.0000]| \|\allowbreak\token[0.0000]|left|\allowbreak\token[0.0000]!|!\allowbreak\token[0.0000]| |\allowbreak\token[0.0000]|0|\allowbreak\token[0.0001]|(|\allowbreak\token[0.0002]|0|\allowbreak\token[0.0000]| -|\allowbreak\token[0.0000]| |\allowbreak\token[0.0195]|0|\allowbreak\token[0.0000]|)|\allowbreak\token[0.0000]| +|\allowbreak\token[0.0000]| |\allowbreak\token[0.0001]|0|\allowbreak\token[0.0001]|(|\allowbreak\token[0.0006]|0|\allowbreak\token[0.0000]| -|\allowbreak\token[0.0000]| |\allowbreak\token[0.0001]|3|\allowbreak\token[0.0000]|)|
\end{twopartbox}
\begin{twopartbox}{Cut Index 260 | Acceptance Probability 1.0000}{Block 4 Step 2}{Proposal (Accepted)}
[...]\allowbreak\token[0.0329]|{|\allowbreak\token[0.0024]|0|\allowbreak\token[0.6631]|+|\allowbreak\token[0.0000]|0|\allowbreak\token[0.0000]|}{|\allowbreak\token[0.0000]|2|\allowbreak\token[0.0001]|},|\allowbreak\token[0.0000]| \|\allowbreak\token[0.0001]|frac|\allowbreak\token[0.0000]|{|\allowbreak\token[0.1326]|6|\allowbreak\token[0.0002]|+|\allowbreak\token[0.0000]|0|\allowbreak\token[0.0001]|}{|\allowbreak\token[0.0000]|2|\allowbreak\token[0.0000]|}|\allowbreak\token[0.0000]| \|\allowbreak\token[0.0000]|right|\allowbreak\token[0.0004]|)|\allowbreak\token[0.0007]| =|\allowbreak\token[0.0018]| (|\allowbreak\token[0.0000]|0|\allowbreak\token[0.0000]|,|\allowbreak\token[0.0002]| |\allowbreak\token[0.0002]|3|\allowbreak\token[0.0004]|)|\allowbreak\token[0.0001]| \|\allowbreak\token[0.0059]|]↩↩|\tighttracebreak\allowbreak\token[0.1003]|Since|\allowbreak\token[0.0000]| \(|\allowbreak\token[0.0001]|E|\allowbreak\token[0.0003]|\)|\allowbreak\token[0.0003]| is|\allowbreak\token[0.0000]| the|\allowbreak\token[0.0000]| midpoint|\allowbreak\token[0.0000]| of|\allowbreak\token[0.0000]| \|\allowbreak\token[0.0005]|(\|\allowbreak\token[0.0002]|over|\allowbreak\token[0.0000]|line|\allowbreak\token[0.0000]|{|\allowbreak\token[0.0007]|BC|\allowbreak\token[0.0001]|}\|\allowbreak\token[0.0002]|):↩|\newline\allowbreak\token[0.0000]|\[|\allowbreak\token[0.0001]| E|\allowbreak\token[0.0001]| =|\allowbreak\token[0.0000]| \|\allowbreak\token[0.0000]|left|\allowbreak\token[0.0000]|(|\allowbreak\token[0.0000]| \|\allowbreak\token[0.0000]|frac|\allowbreak\token[0.0000]|{|\allowbreak\token[0.0494]|0|\allowbreak\token[0.0008]|+|\allowbreak\token[0.0000]|8|\allowbreak\token[0.0000]|}{|\allowbreak\token[0.0000]|2|\allowbreak\token[0.0000]|},|\allowbreak\token[0.0000]| \|\allowbreak\token[0.0001]|frac|\allowbreak\token[0.0000]|{|\allowbreak\token[0.0001]|0|\allowbreak\token[0.0000]|+|\allowbreak\token[0.0000]|0|\allowbreak\token[0.0001]|}{|\allowbreak\token[0.0002]|2|\allowbreak\token[0.0000]|}|\allowbreak\token[0.0004]| \|\allowbreak\token[0.0004]|right|\allowbreak\token[0.0001]|)|\allowbreak\token[0.0002]| =|\allowbreak\token[0.0002]| (|\allowbreak\token[0.0000]|4|\allowbreak\token[0.0000]|,|\allowbreak\token[0.0001]| |\allowbreak\token[0.0002]|0|\allowbreak\token[0.0005]|)|\allowbreak\token[0.0000]| \|\allowbreak\token[0.0014]|]↩↩|\tighttracebreak\allowbreak\token[1.5966]|Next|\allowbreak\token[0.0009]|,|\allowbreak\token[0.3245]| we|\allowbreak\token[0.3536]| need|\allowbreak\token[0.1951]| to|\allowbreak\token[0.6223]| find|\allowbreak\token[0.0583]| the|\allowbreak\token[0.1246]| coordinates|\allowbreak\token[0.0007]| of|\allowbreak\token[0.6730]| point|\allowbreak\token[0.0001]| \(|\allowbreak\token[0.0021]|F|\allowbreak\token[0.0117]|\|\allowbreak\token[0.6301]|).|\allowbreak\token[1.5853]| Since|\allowbreak\token[0.1894]| \(|\allowbreak\token[0.2436]|F|\allowbreak\token[0.0004]|\)|\allowbreak\token[0.5669]| is|\allowbreak\token[0.8988]| the|\allowbreak\token[0.3516]| intersection|\allowbreak\token[0.0141]| of|\allowbreak\token[1.5777]| line|\allowbreak\token[0.6444]| segments|\allowbreak\token[0.6633]| \(|\allowbreak\token[1.0602]|AE|\allowbreak\token[0.0003]|\)|\allowbreak\token[0.0001]| and|\allowbreak\token[0.0018]| \(|\allowbreak\token[0.9437]|BD|\allowbreak\token[0.0017]|\|\allowbreak\token[0.0306]|),|\allowbreak\token[0.0350]| we|\allowbreak\token[1.2384]| need|\allowbreak\token[0.0125]| to|\allowbreak\token[0.9057]| find|\allowbreak\token[0.0356]| the|\allowbreak\token[0.1261]| equations|\allowbreak\token[0.0003]| of|\allowbreak\token[0.6777]| these|\allowbreak\token[0.0235]| lines|\allowbreak\token[0.5417]| and|\allowbreak\token[0.6232]| solve|\allowbreak\token[0.0630]| for|\allowbreak\token[0.1762]| their|\allowbreak\token[0.0003]| intersection|\allowbreak\token[0.1750]|.↩↩|\tighttracebreak\allowbreak\token[1.3623]|The|\allowbreak\token[0.5901]| equation|\allowbreak\token[0.0022]| of|\allowbreak\token[0.0118]| line|\allowbreak\token[0.0297]| \(|\allowbreak\token[0.0586]|AE|\allowbreak\token[0.6934]|\)|\allowbreak\token[0.6600]| can|\allowbreak\token[0.0000]| be|\allowbreak\token[0.8074]| found|\allowbreak\token[0.1120]| using|\allowbreak\token[0.1335]| the|\allowbreak\token[0.6864]| points|\allowbreak\token[0.0003]| \(|\allowbreak\token[0.0001]|A|\allowbreak\token[0.0006]|(|\allowbreak\token[0.0000]|0|\allowbreak\cut{[CUT]}\allowbreak\token[0.0000]|,|\allowbreak\token[0.0605]| |\allowbreak\token[0.0001]|6|\allowbreak\token[0.0000]|)\|\allowbreak\token[0.0000]|)|\allowbreak\token[0.0000]| and|\allowbreak\token[0.0000]| \(|\allowbreak\token[0.0002]|E|\allowbreak\token[0.0001]|(|\allowbreak\token[0.0001]|4|\allowbreak\token[0.0000]|,|\allowbreak\token[0.0001]| |\allowbreak\token[0.0000]|0|\allowbreak\token[0.0001]|)\|\allowbreak\token[0.0941]|):↩|\newline\allowbreak\token[0.3132]|\[|\allowbreak\token[0.0095]| \|\allowbreak\token[0.0002]|text|\allowbreak\token[0.3148]|{s|\allowbreak\token[0.0000]|lope|\allowbreak\token[0.0118]| of|\allowbreak\token[0.0013]| }|\allowbreak\token[0.0003]| AE|\allowbreak\token[0.0002]| =|\allowbreak\token[0.0000]| \|\allowbreak\token[0.0000]|frac|\allowbreak\token[0.0000]|{|\allowbreak\token[0.0016]|0|\allowbreak\token[0.6759]| -|\allowbreak\token[0.0000]| |\allowbreak\token[0.0000]|6|\allowbreak\token[0.0000]|}{|\allowbreak\token[0.0000]|4|\allowbreak\token[0.0000]| -|\allowbreak\token[0.0000]| |\allowbreak\token[0.0000]|0|\allowbreak\token[0.0000]|}|\allowbreak\token[0.0000]| =|\allowbreak\token[0.0014]| -|\allowbreak\token[0.0020]|\|\allowbreak\token[0.0000]|frac|\allowbreak\token[0.0000]|{|\allowbreak\token[0.0173]|3|\allowbreak\token[0.0000]|}{|\allowbreak\token[0.0000]|2|\allowbreak\token[0.0000]|}|\allowbreak\token[0.0000]| \|\allowbreak\token[0.0002]|]↩|\newline\allowbreak\token[1.2316]|\[|\allowbreak\token[0.6783]| \|\allowbreak\token[0.0000]|text|\allowbreak\token[0.0000]|{|\allowbreak\token[0.4240]|equ|\allowbreak\token[0.0000]|ation|\allowbreak\token[0.0025]| of|\allowbreak\token[0.0450]| }|\allowbreak\token[0.0000]| AE|\allowbreak\token[0.0388]|:|\allowbreak\token[0.0022]| y|\allowbreak\token[0.3396]| -|\allowbreak\token[0.0000]| |\allowbreak\token[0.0817]|6|\allowbreak\token[0.0000]| =|\allowbreak\token[0.0000]| -|\allowbreak\token[0.0000]|\|\allowbreak\token[0.0000]|frac|\allowbreak\token[0.0000]|{|\allowbreak\token[0.0000]|3|\allowbreak\token[0.0000]|}{|\allowbreak\token[0.0000]|2|\allowbreak\token[0.6913]|}|\allowbreak\token[0.0671]|x|\allowbreak\token[0.0001]| \|\allowbreak\token[0.9195]|impl|\allowbreak\token[0.0000]|ies|\allowbreak\token[0.0031]| y|\allowbreak\token[0.0000]| =|\allowbreak\token[0.1324]| -|\allowbreak\token[0.0000]|\|\allowbreak\token[0.0000]|frac|\allowbreak\token[0.0000]|{|\allowbreak\token[0.0000]|3|\allowbreak\token[0.0000]|}{|\allowbreak\token[0.0000]|2|\allowbreak\token[0.0000]|}|\allowbreak\token[0.0000]|x|\allowbreak\token[0.0000]| +|\allowbreak\token[0.0000]| |\allowbreak\token[0.0000]|6|\allowbreak\token[0.0000]| \|\allowbreak\token[0.0001]|]↩↩|\tighttracebreak\allowbreak\token[0.0001]|The|\allowbreak\token[0.0000]| equation|\allowbreak\token[0.0000]| of|\allowbreak\token[0.0000]| line|\allowbreak\token[0.0000]| \(|\allowbreak\token[0.0001]|BD|\allowbreak\token[0.0001]|\)|\allowbreak\token[0.0003]| can|\allowbreak\token[0.0000]| be|\allowbreak\token[0.0000]| found|\allowbreak\token[0.0000]| using|\allowbreak\token[0.0014]| the|\allowbreak\token[0.0000]| points|\allowbreak\token[0.0000]| \(|\allowbreak\token[0.0008]|B|\allowbreak\token[0.0000]|(|\allowbreak\token[0.0000]|0|\allowbreak\token[0.0000]|,|\allowbreak\token[0.0006]| |\allowbreak\token[0.0000]|0|\allowbreak\token[0.0000]|)\|\allowbreak\token[0.0007]|)|\allowbreak\token[0.0000]| and|\allowbreak\token[0.0000]| \(|\allowbreak\token[0.0003]|D|\allowbreak\token[0.0001]|(|\allowbreak\token[0.0002]|0|\allowbreak\token[0.0000]|,|\allowbreak\token[0.0000]| |\allowbreak\token[0.0004]|3|\allowbreak\token[0.0001]|)\|\allowbreak\token[0.0613]|):↩|\newline\allowbreak\token[0.1976]|\[|\allowbreak\token[0.0011]| \|\allowbreak\token[0.0000]|text|\allowbreak\token[0.2474]|{s|\allowbreak\token[0.0000]|lope|\allowbreak\token[0.0000]| of|\allowbreak\token[0.0000]| }|\allowbreak\token[0.0000]| BD|\allowbreak\token[0.0196]| =|\allowbreak\token[0.0009]| \|\allowbreak\token[0.2335]|frac|\allowbreak\token[0.0000]|{|\allowbreak\token[0.0014]|3|\allowbreak\token[0.0016]| -|\allowbreak\token[0.0000]| |\allowbreak\token[0.0000]|0|\allowbreak\token[0.0000]|}{|\allowbreak\token[0.0003]|0|\allowbreak\token[0.0001]| -|\allowbreak\token[0.0000]| |\allowbreak\token[0.0009]|0|\allowbreak[...]
\tcblower
\token[0.0000]|,|\allowbreak\token[0.0547]| |\allowbreak\token[0.0001]|6|\allowbreak\token[0.0000]|)\|\allowbreak\token[0.0000]|)|\allowbreak\token[0.0000]| and|\allowbreak\token[0.0000]| \(|\allowbreak\token[0.0002]|E|\allowbreak\token[0.0001]|(|\allowbreak\token[0.0001]|4|\allowbreak\token[0.0000]|,|\allowbreak\token[0.0001]| |\allowbreak\token[0.0000]|0|\allowbreak\token[0.0001]|)\|\allowbreak\token[0.0943]|):↩|\newline\allowbreak\token[0.3425]|\[|\allowbreak\token[0.0090]| \|\allowbreak\token[0.0002]|text|\allowbreak\token[0.3148]|{s|\allowbreak\token[0.0000]|lope|\allowbreak\token[0.0162]| of|\allowbreak\token[0.0010]| }|\allowbreak\token[0.0002]| AE|\allowbreak\token[0.0001]| =|\allowbreak\token[0.0000]| \|\allowbreak\token[0.0000]|frac|\allowbreak\token[0.0000]|{|\allowbreak\token[0.0018]|0|\allowbreak\token[0.6759]|-|\allowbreak\token[0.0000]|6|\allowbreak\token[0.0000]|}{|\allowbreak\token[0.0000]|4|\allowbreak\token[0.0000]|-|\allowbreak\token[0.0000]|0|\allowbreak\token[0.0000]|}|\allowbreak\token[0.0000]| =|\allowbreak\token[0.0009]| -|\allowbreak\token[0.0011]|\|\allowbreak\token[0.0000]|frac|\allowbreak\token[0.0000]|{|\allowbreak\token[0.0156]|3|\allowbreak\token[0.0000]|}{|\allowbreak\token[0.0000]|2|\allowbreak\token[0.0001]|}|\allowbreak\token[0.0000]| \|\allowbreak\token[0.0002]|]↩|\newline\allowbreak\token[1.2549]|\[|\allowbreak\token[0.6877]| \|\allowbreak\token[0.0000]|text|\allowbreak\token[0.0000]|{|\allowbreak\token[0.4811]|equ|\allowbreak\token[0.0000]|ation|\allowbreak\token[0.0031]| of|\allowbreak\token[0.0450]| }|\allowbreak\token[0.0000]| AE|\allowbreak\token[0.0365]|:|\allowbreak\token[0.0025]| y|\allowbreak\token[0.3400]| -|\allowbreak\token[0.0000]| |\allowbreak\token[0.0740]|6|\allowbreak\token[0.0000]| =|\allowbreak\token[0.0000]| -|\allowbreak\token[0.0000]|\|\allowbreak\token[0.0000]|frac|\allowbreak\token[0.0000]|{|\allowbreak\token[0.0000]|3|\allowbreak\token[0.0000]|}{|\allowbreak\token[0.0000]|2|\allowbreak\token[0.6760]|}(|\allowbreak\token[0.0000]|x|\allowbreak\token[0.0004]| -|\allowbreak\token[0.0000]| |\allowbreak\token[0.0001]|0|\allowbreak\token[0.0000]|)|\allowbreak\token[0.0001]| \|\allowbreak\token[0.9203]|impl|\allowbreak\token[0.0000]|ies|\allowbreak\token[0.0003]| y|\allowbreak\token[0.0000]| =|\allowbreak\token[0.3148]| -|\allowbreak\token[0.0000]|\|\allowbreak\token[0.0000]|frac|\allowbreak\token[0.0000]|{|\allowbreak\token[0.0000]|3|\allowbreak\token[0.0000]|}{|\allowbreak\token[0.0000]|2|\allowbreak\token[0.0000]|}|\allowbreak\token[0.0001]|x|\allowbreak\token[0.0000]| +|\allowbreak\token[0.0000]| |\allowbreak\token[0.0000]|6|\allowbreak\token[0.0000]| \|\allowbreak\token[0.0002]|]↩↩|\tighttracebreak\allowbreak\token[0.0003]|The|\allowbreak\token[0.0000]| equation|\allowbreak\token[0.0000]| of|\allowbreak\token[0.0000]| line|\allowbreak\token[0.0000]| \(|\allowbreak\token[0.0001]|BD|\allowbreak\token[0.0001]|\)|\allowbreak\token[0.0005]| can|\allowbreak\token[0.0000]| be|\allowbreak\token[0.0000]| found|\allowbreak\token[0.0000]| using|\allowbreak\token[0.0015]| the|\allowbreak\token[0.0000]| points|\allowbreak\token[0.0000]| \(|\allowbreak\token[0.0010]|B|\allowbreak\token[0.0000]|(|\allowbreak\token[0.0000]|0|\allowbreak\token[0.0000]|,|\allowbreak\token[0.0027]| |\allowbreak\token[0.0000]|0|\allowbreak\token[0.0000]|)\|\allowbreak\token[0.0010]|)|\allowbreak\token[0.0000]| and|\allowbreak\token[0.0000]| \(|\allowbreak\token[0.0003]|D|\allowbreak\token[0.0001]|(|\allowbreak\token[0.0002]|0|\allowbreak\token[0.0000]|,|\allowbreak\token[0.0000]| |\allowbreak\token[0.0001]|3|\allowbreak\token[0.0002]|)\|\allowbreak\token[0.0552]|):↩|\newline\allowbreak\token[0.1657]|\[|\allowbreak\token[0.0009]| \|\allowbreak\token[0.0000]|text|\allowbreak\token[0.2474]|{s|\allowbreak\token[0.0000]|lope|\allowbreak\token[0.0000]| of|\allowbreak\token[0.0000]| }|\allowbreak\token[0.0000]| BD|\allowbreak\token[0.0181]| =|\allowbreak\token[0.0011]| \|\allowbreak\token[0.2142]|frac|\allowbreak\token[0.0000]|{|\allowbreak\token[0.0014]|3|\allowbreak\token[0.0005]|-|\allowbreak\token[0.0000]|0|\allowbreak\token[0.0000]|}{|\allowbreak\token[0.0002]|0|\allowbreak\token[0.0001]|-|\allowbreak\token[0.0005]|0|\allowbreak[...]
\end{twopartbox}
\begin{twopartbox}{Cut Index 342 | Acceptance Probability 1.0000}{Block 4 Step 3}{Proposal (Accepted)}
[...]\allowbreak\token[0.3245]| we|\allowbreak\token[0.3536]| need|\allowbreak\token[0.1951]| to|\allowbreak\token[0.6223]| find|\allowbreak\token[0.0583]| the|\allowbreak\token[0.1246]| coordinates|\allowbreak\token[0.0007]| of|\allowbreak\token[0.6730]| point|\allowbreak\token[0.0001]| \(|\allowbreak\token[0.0021]|F|\allowbreak\token[0.0117]|\|\allowbreak\token[0.6301]|).|\allowbreak\token[1.5853]| Since|\allowbreak\token[0.1894]| \(|\allowbreak\token[0.2436]|F|\allowbreak\token[0.0004]|\)|\allowbreak\token[0.5669]| is|\allowbreak\token[0.8988]| the|\allowbreak\token[0.3516]| intersection|\allowbreak\token[0.0141]| of|\allowbreak\token[1.5777]| line|\allowbreak\token[0.6444]| segments|\allowbreak\token[0.6633]| \(|\allowbreak\token[1.0602]|AE|\allowbreak\token[0.0003]|\)|\allowbreak\token[0.0001]| and|\allowbreak\token[0.0018]| \(|\allowbreak\token[0.9437]|BD|\allowbreak\token[0.0017]|\|\allowbreak\token[0.0306]|),|\allowbreak\token[0.0350]| we|\allowbreak\token[1.2384]| need|\allowbreak\token[0.0125]| to|\allowbreak\token[0.9057]| find|\allowbreak\token[0.0356]| the|\allowbreak\token[0.1261]| equations|\allowbreak\token[0.0003]| of|\allowbreak\token[0.6777]| these|\allowbreak\token[0.0235]| lines|\allowbreak\token[0.5417]| and|\allowbreak\token[0.6232]| solve|\allowbreak\token[0.0630]| for|\allowbreak\token[0.1762]| their|\allowbreak\token[0.0003]| intersection|\allowbreak\token[0.1750]|.↩↩|\tighttracebreak\allowbreak\token[1.3623]|The|\allowbreak\token[0.5901]| equation|\allowbreak\token[0.0022]| of|\allowbreak\token[0.0118]| line|\allowbreak\token[0.0297]| \(|\allowbreak\token[0.0586]|AE|\allowbreak\token[0.6934]|\)|\allowbreak\token[0.6600]| can|\allowbreak\token[0.0000]| be|\allowbreak\token[0.8074]| found|\allowbreak\token[0.1120]| using|\allowbreak\token[0.1335]| the|\allowbreak\token[0.6864]| points|\allowbreak\token[0.0003]| \(|\allowbreak\token[0.0001]|A|\allowbreak\token[0.0006]|(|\allowbreak\token[0.0000]|0|\allowbreak\token[0.0000]|,|\allowbreak\token[0.0547]| |\allowbreak\token[0.0001]|6|\allowbreak\token[0.0000]|)\|\allowbreak\token[0.0000]|)|\allowbreak\token[0.0000]| and|\allowbreak\token[0.0000]| \(|\allowbreak\token[0.0002]|E|\allowbreak\token[0.0001]|(|\allowbreak\token[0.0001]|4|\allowbreak\token[0.0000]|,|\allowbreak\token[0.0001]| |\allowbreak\token[0.0000]|0|\allowbreak\token[0.0001]|)\|\allowbreak\token[0.0943]|):↩|\newline\allowbreak\token[0.3425]|\[|\allowbreak\token[0.0090]| \|\allowbreak\token[0.0002]|text|\allowbreak\token[0.3148]|{s|\allowbreak\token[0.0000]|lope|\allowbreak\token[0.0162]| of|\allowbreak\token[0.0010]| }|\allowbreak\token[0.0002]| AE|\allowbreak\token[0.0001]| =|\allowbreak\token[0.0000]| \|\allowbreak\token[0.0000]|frac|\allowbreak\token[0.0000]|{|\allowbreak\token[0.0018]|0|\allowbreak\token[0.6759]|-|\allowbreak\token[0.0000]|6|\allowbreak\token[0.0000]|}{|\allowbreak\token[0.0000]|4|\allowbreak\token[0.0000]|-|\allowbreak\token[0.0000]|0|\allowbreak\token[0.0000]|}|\allowbreak\token[0.0000]| =|\allowbreak\token[0.0009]| -|\allowbreak\token[0.0011]|\|\allowbreak\token[0.0000]|frac|\allowbreak\token[0.0000]|{|\allowbreak\token[0.0156]|3|\allowbreak\token[0.0000]|}{|\allowbreak\token[0.0000]|2|\allowbreak\token[0.0001]|}|\allowbreak\token[0.0000]| \|\allowbreak\token[0.0002]|]↩|\newline\allowbreak\token[1.2549]|\[|\allowbreak\token[0.6877]| \|\allowbreak\token[0.0000]|text|\allowbreak\token[0.0000]|{|\allowbreak\token[0.4811]|equ|\allowbreak\token[0.0000]|ation|\allowbreak\token[0.0031]| of|\allowbreak\token[0.0450]| }|\allowbreak\token[0.0000]| AE|\allowbreak\token[0.0365]|:|\allowbreak\token[0.0025]| y|\allowbreak\token[0.3400]| -|\allowbreak\token[0.0000]| |\allowbreak\token[0.0740]|6|\allowbreak\token[0.0000]| =|\allowbreak\token[0.0000]| -|\allowbreak\token[0.0000]|\|\allowbreak\token[0.0000]|frac|\allowbreak\token[0.0000]|{|\allowbreak\token[0.0000]|3|\allowbreak\token[0.0000]|}{|\allowbreak\token[0.0000]|2|\allowbreak\token[0.6760]|}(|\allowbreak\token[0.0000]|x|\allowbreak\token[0.0004]| -|\allowbreak\token[0.0000]| |\allowbreak\token[0.0001]|0|\allowbreak\token[0.0000]|)|\allowbreak\token[0.0001]| \|\allowbreak\token[0.9203]|impl|\allowbreak\token[0.0000]|ies|\allowbreak\token[0.0003]| y|\allowbreak\token[0.0000]| =|\allowbreak\token[0.3148]| -|\allowbreak\token[0.0000]|\|\allowbreak\token[0.0000]|frac|\allowbreak\cut{[CUT]}\allowbreak\token[0.0000]|{|\allowbreak\token[0.0000]|3|\allowbreak\token[0.0000]|}{|\allowbreak\token[0.0000]|2|\allowbreak\token[0.0000]|}|\allowbreak\token[0.0001]|x|\allowbreak\token[0.0000]| +|\allowbreak\token[0.0000]| |\allowbreak\token[0.0000]|6|\allowbreak\token[0.0000]| \|\allowbreak\token[0.0002]|]↩↩|\tighttracebreak\allowbreak\token[0.0003]|The|\allowbreak\token[0.0000]| equation|\allowbreak\token[0.0000]| of|\allowbreak\token[0.0000]| line|\allowbreak\token[0.0000]| \(|\allowbreak\token[0.0001]|BD|\allowbreak\token[0.0001]|\)|\allowbreak\token[0.0005]| can|\allowbreak\token[0.0000]| be|\allowbreak\token[0.0000]| found|\allowbreak\token[0.0000]| using|\allowbreak\token[0.0015]| the|\allowbreak\token[0.0000]| points|\allowbreak\token[0.0000]| \(|\allowbreak\token[0.0010]|B|\allowbreak\token[0.0000]|(|\allowbreak\token[0.0000]|0|\allowbreak\token[0.0000]|,|\allowbreak\token[0.0027]| |\allowbreak\token[0.0000]|0|\allowbreak\token[0.0000]|)\|\allowbreak\token[0.0010]|)|\allowbreak\token[0.0000]| and|\allowbreak\token[0.0000]| \(|\allowbreak\token[0.0003]|D|\allowbreak\token[0.0001]|(|\allowbreak\token[0.0002]|0|\allowbreak\token[0.0000]|,|\allowbreak\token[0.0000]| |\allowbreak\token[0.0001]|3|\allowbreak\token[0.0002]|)\|\allowbreak\token[0.0552]|):↩|\newline\allowbreak\token[0.1657]|\[|\allowbreak\token[0.0009]| \|\allowbreak\token[0.0000]|text|\allowbreak\token[0.2474]|{s|\allowbreak\token[0.0000]|lope|\allowbreak\token[0.0000]| of|\allowbreak\token[0.0000]| }|\allowbreak\token[0.0000]| BD|\allowbreak\token[0.0181]| =|\allowbreak\token[0.0011]| \|\allowbreak\token[0.2142]|frac|\allowbreak\token[0.0000]|{|\allowbreak\token[0.0014]|3|\allowbreak\token[0.0005]|-|\allowbreak\token[0.0000]|0|\allowbreak\token[0.0000]|}{|\allowbreak\token[0.0002]|0|\allowbreak\token[0.0001]|-|\allowbreak\token[0.0005]|0|\allowbreak\token[0.0007]|}|\allowbreak\token[0.3675]| =|\allowbreak\token[0.0005]| \|\allowbreak\token[0.1522]|text|\allowbreak\token[0.0002]|{|\allowbreak\token[0.0005]|undefined|\allowbreak\token[0.4476]| (|\allowbreak\token[0.0011]|vertical|\allowbreak\token[0.0008]| line|\allowbreak\token[0.0011]|)}|\allowbreak\token[0.0000]| \|\allowbreak\token[0.0022]|]↩|\newline\allowbreak\token[0.0018]|\[|\allowbreak\token[0.0003]| \|\allowbreak\token[0.0000]|text|\allowbreak\token[0.0000]|{|\allowbreak\token[0.0001]|equ|\allowbreak\token[0.0000]|ation|\allowbreak\token[0.0000]| of|\allowbreak\token[0.0000]| }|\allowbreak\token[0.0000]| BD|\allowbreak\token[0.0000]|:|\allowbreak\token[0.0005]| x|\allowbreak\token[0.0001]| =|\allowbreak\token[0.0002]| |\allowbreak\token[0.0010]|0|\allowbreak\token[0.0028]| \|\allowbreak\token[0.0096]|]↩↩|\tighttracebreak\allowbreak\token[1.1751]|Since|\allowbreak\token[0.6552]| \(|\allowbreak\token[1.0679]|BD|\allowbreak\token[0.0000]|\)|\allowbreak\token[0.0001]| is|\allowbreak\token[0.0270]| a|\allowbreak\token[0.0000]| vertical|\allowbreak\token[0.0000]| line|\allowbreak\token[0.8931]| at|\allowbreak\token[0.0000]| \(|\allowbreak\token[0.0001]|x|\allowbreak\token[0.0101]| =|\allowbreak\token[0.0000]| |\allowbreak\token[0.0103]|0|\allowbreak\token[0.0238]|\|\allowbreak\token[0.0002]|),|\allowbreak\token[0.6070]| we|\allowbreak\token[0.2003]| substitute|\allowbreak\token[0.0000]| \(|\allowbreak\token[0.0000]|x|\allowbreak\token[0.0000]| =|\allowbreak\token[0.0001]| |\allowbreak\token[0.0015]|0|\allowbreak\token[0.0000]|\)|\allowbreak\token[0.0008]| into|\allowbreak\token[0.0000]| the|\allowbreak\token[0.0000]| equation|\allowbreak\token[0.0026]| of|\allowbreak\token[0.1326]| \(|\allowbreak\token[0.0001]|AE|\allowbreak\token[0.5304]|\)|\allowbreak\token[0.0000]| to|\allowbreak\token[0.0001]| find|\allowbreak\token[0.1747]| the|\allowbreak\token[0.7892]| y|\allowbreak\token[0.0000]|-coordinate|\allowbreak\token[0.0009]| of|\allowbreak\token[0.0945]| \(|\allowbreak\token[0.0001]|F|\allowbreak\token[0.0000]|\|\allowbreak\token[0.0002]|):↩|\newline\allowbreak\token[0.0000]|\[|\allowbreak\token[0.0012]| y|\allowbreak\token[0.0000]| =|\allowbreak\token[0.0000]| -|\allowbreak\token[0.0000]|\|\allowbreak\token[0.0000]|frac|\allowbreak\token[0.0000]|{|\allowbreak\token[0.0000]|3|\allowbreak\token[0.0000]|}{|\allowbreak\token[0.0000]|2|\allowbreak\token[0.0002]|}(|\allowbreak\token[0.0000]|0|\allowbreak\token[0.0000]|)|\allowbreak[...]
\tcblower
\token[0.0000]|{|\allowbreak\token[0.0000]|3|\allowbreak\token[0.0000]|}{|\allowbreak\token[0.0000]|2|\allowbreak\token[0.0000]|}|\allowbreak\token[0.0001]|x|\allowbreak\token[0.0000]| +|\allowbreak\token[0.0000]| |\allowbreak\token[0.0000]|6|\allowbreak\token[0.0000]| \|\allowbreak\token[0.0002]|]↩↩|\tighttracebreak\allowbreak\token[0.0003]|The|\allowbreak\token[0.0000]| equation|\allowbreak\token[0.0000]| of|\allowbreak\token[0.0000]| line|\allowbreak\token[0.0000]| \(|\allowbreak\token[0.0001]|BD|\allowbreak\token[0.0001]|\)|\allowbreak\token[0.0005]| can|\allowbreak\token[0.0000]| be|\allowbreak\token[0.0000]| found|\allowbreak\token[0.0000]| using|\allowbreak\token[0.0015]| the|\allowbreak\token[0.0000]| points|\allowbreak\token[0.0000]| \(|\allowbreak\token[0.0010]|B|\allowbreak\token[0.0000]|(|\allowbreak\token[0.0000]|0|\allowbreak\token[0.0000]|,|\allowbreak\token[0.0027]| |\allowbreak\token[0.0000]|0|\allowbreak\token[0.0000]|)\|\allowbreak\token[0.0010]|)|\allowbreak\token[0.0000]| and|\allowbreak\token[0.0000]| \(|\allowbreak\token[0.0003]|D|\allowbreak\token[0.0001]|(|\allowbreak\token[0.0002]|0|\allowbreak\token[0.0000]|,|\allowbreak\token[0.0000]| |\allowbreak\token[0.0001]|3|\allowbreak\token[0.0002]|)\|\allowbreak\token[0.0552]|):↩|\newline\allowbreak\token[0.1657]|\[|\allowbreak\token[0.0009]| \|\allowbreak\token[0.0000]|text|\allowbreak\token[0.2474]|{s|\allowbreak\token[0.0000]|lope|\allowbreak\token[0.0000]| of|\allowbreak\token[0.0000]| }|\allowbreak\token[0.0000]| BD|\allowbreak\token[0.0181]| =|\allowbreak\token[0.0011]| \|\allowbreak\token[0.2142]|frac|\allowbreak\token[0.0000]|{|\allowbreak\token[0.0014]|3|\allowbreak\token[0.0005]|-|\allowbreak\token[0.0000]|0|\allowbreak\token[0.0000]|}{|\allowbreak\token[0.0002]|0|\allowbreak\token[0.0001]|-|\allowbreak\token[0.0005]|0|\allowbreak\token[0.0007]|}|\allowbreak\token[0.3675]| =|\allowbreak\token[0.0005]| \|\allowbreak\token[0.1522]|text|\allowbreak\token[0.0002]|{|\allowbreak\token[0.0005]|undefined|\allowbreak\token[0.4476]| (|\allowbreak\token[0.0011]|vertical|\allowbreak\token[0.0008]| line|\allowbreak\token[0.0011]|)}|\allowbreak\token[0.0000]| \|\allowbreak\token[0.0022]|]↩|\newline\allowbreak\token[0.0018]|\[|\allowbreak\token[0.0003]| \|\allowbreak\token[0.0000]|text|\allowbreak\token[0.0000]|{|\allowbreak\token[0.0001]|equ|\allowbreak\token[0.0000]|ation|\allowbreak\token[0.0000]| of|\allowbreak\token[0.0000]| }|\allowbreak\token[0.0000]| BD|\allowbreak\token[0.0000]|:|\allowbreak\token[0.0005]| x|\allowbreak\token[0.0001]| =|\allowbreak\token[0.0002]| |\allowbreak\token[0.0010]|0|\allowbreak\token[0.0028]| \|\allowbreak\token[0.0096]|]↩↩|\tighttracebreak\allowbreak\token[1.1751]|Since|\allowbreak\token[0.6552]| \(|\allowbreak\token[1.0679]|F|\allowbreak\token[0.0000]|\)|\allowbreak\token[0.4194]| is|\allowbreak\token[0.1330]| the|\allowbreak\token[0.0034]| intersection|\allowbreak\token[0.0002]| of|\allowbreak\token[0.4560]| \(|\allowbreak\token[0.1575]|AE|\allowbreak\token[0.0001]|\)|\allowbreak\token[0.0296]| and|\allowbreak\token[0.0091]| \(|\allowbreak\token[0.0003]|BD|\allowbreak\token[0.0006]|\|\allowbreak\token[0.0215]|),|\allowbreak\token[1.1989]| we|\allowbreak\token[0.0066]| substitute|\allowbreak\token[0.0000]| \(|\allowbreak\token[0.0000]|x|\allowbreak\token[0.0113]| =|\allowbreak\token[0.0193]| |\allowbreak\token[0.0039]|0|\allowbreak\token[0.0000]|\)|\allowbreak\token[0.0019]| into|\allowbreak\token[0.0000]| the|\allowbreak\token[0.0000]| equation|\allowbreak\token[0.0327]| of|\allowbreak\token[0.0494]| \(|\allowbreak\token[0.0001]|AE|\allowbreak\token[0.0022]|\|\allowbreak\token[0.0000]|):↩|\newline\allowbreak\token[0.0000]|\[|\allowbreak\token[0.0005]| y|\allowbreak\token[0.0000]| =|\allowbreak\token[0.0001]| -|\allowbreak\token[0.0000]|\|\allowbreak\token[0.0000]|frac|\allowbreak\token[0.0000]|{|\allowbreak\token[0.0000]|3|\allowbreak\token[0.0000]|}{|\allowbreak\token[0.0000]|2|\allowbreak\token[0.0002]|}(|\allowbreak\token[0.0000]|0|\allowbreak\token[0.0000]|)|\allowbreak\token[0.0000]| +|\allowbreak\token[0.0000]| |\allowbreak\token[0.0000]|6|\allowbreak\token[0.1909]| =|\allowbreak\token[0.0000]| |\allowbreak\token[0.0216]|6|\allowbreak\token[0.1966]| \|\allowbreak\token[0.0637]|]↩|\newline\allowbreak\token[0.6701]|So|\allowbreak[...]
\end{twopartbox}
\begin{twopartbox}{Cut Index 193 | Acceptance Probability 0.0000}{Block 5 Step 1}{Proposal (Rejected)}
[...]\allowbreak\token[0.0003]| respectively|\allowbreak\token[1.1826]|.↩↩|\tighttracebreak\allowbreak\token[1.7545]|First|\allowbreak\token[0.0168]|,|\allowbreak\token[0.3726]| let|\allowbreak\token[0.0907]|'s|\allowbreak\token[1.1646]| identify|\allowbreak\token[0.0261]| the|\allowbreak\token[0.0024]| coordinates|\allowbreak\token[0.0240]| of|\allowbreak\token[0.7604]| the|\allowbreak\token[0.5344]| points|\allowbreak\token[1.2730]|:↩|\newline\allowbreak\token[0.0041]|-|\allowbreak\token[0.0943]| \(|\allowbreak\token[0.1553]|A|\allowbreak\token[0.7984]|(|\allowbreak\token[0.0000]|0|\allowbreak\token[0.0001]|,|\allowbreak\token[0.4196]| |\allowbreak\token[0.0002]|6|\allowbreak\token[0.0001]|)\|\allowbreak\token[0.0052]|)↩|\newline\allowbreak\token[0.0000]|-|\allowbreak\token[0.0000]| \(|\allowbreak\token[0.0003]|B|\allowbreak\token[0.0001]|(|\allowbreak\token[0.0000]|0|\allowbreak\token[0.0000]|,|\allowbreak\token[0.0002]| |\allowbreak\token[0.0001]|0|\allowbreak\token[0.0001]|)\|\allowbreak\token[0.0003]|)↩|\newline\allowbreak\token[0.0000]|-|\allowbreak\token[0.0004]| \(|\allowbreak\token[0.0003]|C|\allowbreak\token[0.0001]|(|\allowbreak\token[0.0002]|8|\allowbreak\token[0.0000]|,|\allowbreak\token[0.0002]| |\allowbreak\token[0.0001]|0|\allowbreak\token[0.0002]|)\|\allowbreak\token[0.6871]|)↩↩|\tighttracebreak\allowbreak\token[1.4622]|Since|\allowbreak\token[0.0036]| \(|\allowbreak\token[0.0005]|D|\allowbreak\token[0.0002]|\)|\allowbreak\token[0.6760]| is|\allowbreak\token[0.0001]| the|\allowbreak\token[0.0000]| midpoint|\allowbreak\token[0.0002]| of|\allowbreak\token[0.0113]| \|\allowbreak\token[0.0001]|(\|\allowbreak\token[0.0002]|over|\allowbreak\token[0.0001]|line|\allowbreak\token[0.0001]|{|\allowbreak\token[0.0003]|AB|\allowbreak\token[0.0002]|}\|\allowbreak\token[0.8417]|):↩|\newline\allowbreak\token[0.5591]|\[|\allowbreak\token[0.2927]| D|\allowbreak\token[0.8893]| =|\allowbreak\token[0.0001]| \|\allowbreak\token[0.0001]|left|\allowbreak\token[0.2691]|(|\allowbreak\token[0.0116]| \|\allowbreak\token[0.0008]|frac|\allowbreak\token[0.0329]|{|\allowbreak\token[0.0024]|0|\allowbreak\token[0.6631]|+|\allowbreak\token[0.0000]|0|\allowbreak\token[0.0000]|}{|\allowbreak\token[0.0000]|2|\allowbreak\token[0.0001]|},|\allowbreak\token[0.0000]| \|\allowbreak\token[0.0001]|frac|\allowbreak\token[0.0000]|{|\allowbreak\token[0.1326]|6|\allowbreak\token[0.0002]|+|\allowbreak\token[0.0000]|0|\allowbreak\token[0.0001]|}{|\allowbreak\token[0.0000]|2|\allowbreak\token[0.0000]|}|\allowbreak\token[0.0000]| \|\allowbreak\token[0.0000]|right|\allowbreak\token[0.0004]|)|\allowbreak\token[0.0007]| =|\allowbreak\token[0.0018]| (|\allowbreak\token[0.0000]|0|\allowbreak\token[0.0000]|,|\allowbreak\token[0.0002]| |\allowbreak\token[0.0002]|3|\allowbreak\token[0.0004]|)|\allowbreak\token[0.0001]| \|\allowbreak\token[0.0059]|]↩↩|\tighttracebreak\allowbreak\token[0.1003]|Since|\allowbreak\token[0.0000]| \(|\allowbreak\token[0.0001]|E|\allowbreak\token[0.0003]|\)|\allowbreak\token[0.0003]| is|\allowbreak\token[0.0000]| the|\allowbreak\token[0.0000]| midpoint|\allowbreak\token[0.0000]| of|\allowbreak\token[0.0000]| \|\allowbreak\token[0.0005]|(\|\allowbreak\token[0.0002]|over|\allowbreak\token[0.0000]|line|\allowbreak\token[0.0000]|{|\allowbreak\token[0.0007]|BC|\allowbreak\token[0.0001]|}\|\allowbreak\token[0.0002]|):↩|\newline\allowbreak\token[0.0000]|\[|\allowbreak\token[0.0001]| E|\allowbreak\token[0.0001]| =|\allowbreak\token[0.0000]| \|\allowbreak\token[0.0000]|left|\allowbreak\token[0.0000]|(|\allowbreak\token[0.0000]| \|\allowbreak\token[0.0000]|frac|\allowbreak\token[0.0000]|{|\allowbreak\token[0.0494]|0|\allowbreak\token[0.0008]|+|\allowbreak\token[0.0000]|8|\allowbreak\token[0.0000]|}{|\allowbreak\token[0.0000]|2|\allowbreak\token[0.0000]|},|\allowbreak\token[0.0000]| \|\allowbreak\token[0.0001]|frac|\allowbreak\token[0.0000]|{|\allowbreak\token[0.0001]|0|\allowbreak\token[0.0000]|+|\allowbreak\token[0.0000]|0|\allowbreak\token[0.0001]|}{|\allowbreak\token[0.0002]|2|\allowbreak\token[0.0000]|}|\allowbreak\token[0.0004]| \|\allowbreak\token[0.0004]|right|\allowbreak\token[0.0001]|)|\allowbreak\token[0.0002]| =|\allowbreak\token[0.0002]| (|\allowbreak\token[0.0000]|4|\allowbreak\token[0.0000]|,|\allowbreak\token[0.0001]| |\allowbreak\token[0.0002]|0|\allowbreak\cut{[CUT]}\allowbreak\token[0.0005]|)|\allowbreak\token[0.0000]| \|\allowbreak\token[0.0014]|]↩↩|\tighttracebreak\allowbreak\token[1.5966]|Next|\allowbreak\token[0.0009]|,|\allowbreak\token[0.3245]| we|\allowbreak\token[0.3536]| need|\allowbreak\token[0.1951]| to|\allowbreak\token[0.6223]| find|\allowbreak\token[0.0583]| the|\allowbreak\token[0.1246]| coordinates|\allowbreak\token[0.0007]| of|\allowbreak\token[0.6730]| point|\allowbreak\token[0.0001]| \(|\allowbreak\token[0.0021]|F|\allowbreak\token[0.0117]|\|\allowbreak\token[0.6301]|).|\allowbreak\token[1.5853]| Since|\allowbreak\token[0.1894]| \(|\allowbreak\token[0.2436]|F|\allowbreak\token[0.0004]|\)|\allowbreak\token[0.5669]| is|\allowbreak\token[0.8988]| the|\allowbreak\token[0.3516]| intersection|\allowbreak\token[0.0141]| of|\allowbreak\token[1.5777]| line|\allowbreak\token[0.6444]| segments|\allowbreak\token[0.6633]| \(|\allowbreak\token[1.0602]|AE|\allowbreak\token[0.0003]|\)|\allowbreak\token[0.0001]| and|\allowbreak\token[0.0018]| \(|\allowbreak\token[0.9437]|BD|\allowbreak\token[0.0017]|\|\allowbreak\token[0.0306]|),|\allowbreak\token[0.0350]| we|\allowbreak\token[1.2384]| need|\allowbreak\token[0.0125]| to|\allowbreak\token[0.9057]| find|\allowbreak\token[0.0356]| the|\allowbreak\token[0.1261]| equations|\allowbreak\token[0.0003]| of|\allowbreak\token[0.6777]| these|\allowbreak\token[0.0235]| lines|\allowbreak\token[0.5417]| and|\allowbreak\token[0.6232]| solve|\allowbreak\token[0.0630]| for|\allowbreak\token[0.1762]| their|\allowbreak\token[0.0003]| intersection|\allowbreak\token[0.1750]|.↩↩|\tighttracebreak\allowbreak\token[1.3623]|The|\allowbreak\token[0.5901]| equation|\allowbreak\token[0.0022]| of|\allowbreak\token[0.0118]| line|\allowbreak\token[0.0297]| \(|\allowbreak\token[0.0586]|AE|\allowbreak\token[0.6934]|\)|\allowbreak\token[0.6600]| can|\allowbreak\token[0.0000]| be|\allowbreak\token[0.8074]| found|\allowbreak\token[0.1120]| using|\allowbreak\token[0.1335]| the|\allowbreak\token[0.6864]| points|\allowbreak\token[0.0003]| \(|\allowbreak\token[0.0001]|A|\allowbreak\token[0.0006]|(|\allowbreak\token[0.0000]|0|\allowbreak\token[0.0000]|,|\allowbreak\token[0.0547]| |\allowbreak\token[0.0001]|6|\allowbreak\token[0.0000]|)\|\allowbreak\token[0.0000]|)|\allowbreak\token[0.0000]| and|\allowbreak\token[0.0000]| \(|\allowbreak\token[0.0002]|E|\allowbreak\token[0.0001]|(|\allowbreak\token[0.0001]|4|\allowbreak\token[0.0000]|,|\allowbreak\token[0.0001]| |\allowbreak\token[0.0000]|0|\allowbreak\token[0.0001]|)\|\allowbreak\token[0.0943]|):↩|\newline\allowbreak\token[0.3425]|\[|\allowbreak\token[0.0090]| \|\allowbreak\token[0.0002]|text|\allowbreak\token[0.3148]|{s|\allowbreak\token[0.0000]|lope|\allowbreak\token[0.0162]| of|\allowbreak\token[0.0010]| }|\allowbreak\token[0.0002]| AE|\allowbreak\token[0.0001]| =|\allowbreak\token[0.0000]| \|\allowbreak\token[0.0000]|frac|\allowbreak\token[0.0000]|{|\allowbreak\token[0.0018]|0|\allowbreak\token[0.6759]|-|\allowbreak\token[0.0000]|6|\allowbreak\token[0.0000]|}{|\allowbreak\token[0.0000]|4|\allowbreak\token[0.0000]|-|\allowbreak\token[0.0000]|0|\allowbreak\token[0.0000]|}|\allowbreak\token[0.0000]| =|\allowbreak\token[0.0009]| -|\allowbreak\token[0.0011]|\|\allowbreak\token[0.0000]|frac|\allowbreak\token[0.0000]|{|\allowbreak\token[0.0156]|3|\allowbreak\token[0.0000]|}{|\allowbreak\token[0.0000]|2|\allowbreak\token[0.0001]|}|\allowbreak\token[0.0000]| \|\allowbreak\token[0.0002]|]↩|\newline\allowbreak\token[1.2549]|\[|\allowbreak\token[0.6877]| \|\allowbreak\token[0.0000]|text|\allowbreak\token[0.0000]|{|\allowbreak\token[0.4811]|equ|\allowbreak\token[0.0000]|ation|\allowbreak\token[0.0031]| of|\allowbreak\token[0.0450]| }|\allowbreak\token[0.0000]| AE|\allowbreak\token[0.0365]|:|\allowbreak\token[0.0025]| y|\allowbreak\token[0.3400]| -|\allowbreak\token[0.0000]| |\allowbreak\token[0.0740]|6|\allowbreak\token[0.0000]| =|\allowbreak\token[0.0000]| -|\allowbreak\token[0.0000]|\|\allowbreak\token[0.0000]|frac|\allowbreak\token[0.0000]|{|\allowbreak\token[0.0000]|3|\allowbreak\token[0.0000]|}{|\allowbreak\token[0.0000]|2|\allowbreak\token[0.6760]|}(|\allowbreak\token[0.0000]|x|\allowbreak\token[0.0004]| -|\allowbreak\token[0.0000]| |\allowbreak\token[0.0001]|0|\allowbreak\token[0.0000]|)|\allowbreak\token[0.0001]| \|\allowbreak\token[0.9203]|impl|\allowbreak\token[0.0000]|ies|\allowbreak[...]
\tcblower
\token[0.0005]|)|\allowbreak\token[0.0001]| \|\allowbreak\token[0.0013]|]↩↩|\tighttracebreak\allowbreak\token[1.6161]|Next|\allowbreak\token[0.0010]|,|\allowbreak\token[0.3185]| we|\allowbreak\token[0.3782]| need|\allowbreak\token[0.1952]| to|\allowbreak\token[0.6218]| find|\allowbreak\token[0.0583]| the|\allowbreak\token[0.1260]| coordinates|\allowbreak\token[0.0009]| of|\allowbreak\token[0.6950]| point|\allowbreak\token[0.0001]| \(|\allowbreak\token[0.0021]|F|\allowbreak\token[0.0117]|\|\allowbreak\token[0.6048]|).|\allowbreak\token[1.6337]| Since|\allowbreak\token[0.1871]| \(|\allowbreak\token[0.2483]|F|\allowbreak\token[0.0004]|\)|\allowbreak\token[0.5668]| is|\allowbreak\token[0.8707]| the|\allowbreak\token[0.3694]| intersection|\allowbreak\token[0.0141]| of|\allowbreak\token[1.5637]| line|\allowbreak\token[0.6433]| segments|\allowbreak\token[0.6633]| \(|\allowbreak\token[1.0445]|AE|\allowbreak\token[0.0002]|\)|\allowbreak\token[0.0001]| and|\allowbreak\token[0.0017]| \(|\allowbreak\token[0.9202]|BD|\allowbreak\token[0.0019]|\|\allowbreak\token[0.0277]|),|\allowbreak\token[0.0320]| we|\allowbreak\token[1.2419]| will|\allowbreak\token[1.6154]| find|\allowbreak\token[0.0625]| the|\allowbreak\token[0.0591]| equations|\allowbreak\token[0.0017]| of|\allowbreak\token[0.7751]| these|\allowbreak\token[0.0326]| lines|\allowbreak\token[0.1839]| and|\allowbreak\token[0.5209]| solve|\allowbreak\token[0.0692]| for|\allowbreak\token[0.2146]| their|\allowbreak\token[0.0002]| intersection|\allowbreak\token[0.1690]|.↩↩|\tighttracebreak\allowbreak\token[1.6407]|###|\allowbreak\token[1.6458]| Step|\allowbreak\token[0.3574]| |\allowbreak\token[0.0031]|1|\allowbreak\token[0.0006]|:|\allowbreak\token[0.2879]| Find|\allowbreak\token[0.0070]| the|\allowbreak\token[0.4031]| equation|\allowbreak\token[0.0005]| of|\allowbreak\token[0.0017]| line|\allowbreak\token[0.0143]| \(|\allowbreak\token[0.0346]|AE|\allowbreak\token[0.0008]|\|\allowbreak\token[0.3244]|)↩|\newline\allowbreak\token[0.8131]|The|\allowbreak\token[0.5203]| slope|\allowbreak\token[0.2697]| of|\allowbreak\token[0.5839]| \(|\allowbreak\token[0.0018]|AE|\allowbreak\token[0.0102]|\)|\allowbreak\token[0.0362]| is|\allowbreak\token[0.0210]|:↩|\newline\allowbreak\token[0.0156]|\[|\allowbreak\token[0.6862]| m|\allowbreak\token[0.0034]|_{|\allowbreak\token[0.0001]|AE|\allowbreak\token[0.0000]|}|\allowbreak\token[0.0000]| =|\allowbreak\token[0.0000]| \|\allowbreak\token[0.0000]|frac|\allowbreak\token[0.0002]|{|\allowbreak\token[0.0986]|0|\allowbreak\token[0.4194]| -|\allowbreak\token[0.0000]| |\allowbreak\token[0.0031]|6|\allowbreak\token[0.0000]|}{|\allowbreak\token[0.0002]|4|\allowbreak\token[0.0000]| -|\allowbreak\token[0.0000]| |\allowbreak\token[0.0000]|0|\allowbreak\token[0.0000]|}|\allowbreak\token[0.0001]| =|\allowbreak\token[0.0296]| -|\allowbreak\token[0.0047]|\|\allowbreak\token[0.0000]|frac|\allowbreak\token[0.0000]|{|\allowbreak\token[0.0669]|3|\allowbreak\token[0.0000]|}{|\allowbreak\token[0.0000]|2|\allowbreak\token[0.0001]|}|\allowbreak\token[0.0001]| \|\allowbreak\token[0.5572]|]↩↩|\tighttracebreak\allowbreak\token[0.6948]|Using|\allowbreak\token[0.2913]| the|\allowbreak\token[0.0037]| point|\allowbreak\token[0.0003]|-s|\allowbreak\token[0.0000]|lope|\allowbreak\token[0.0004]| form|\allowbreak\token[0.6521]| of|\allowbreak\token[0.0363]| the|\allowbreak\token[0.6880]| line|\allowbreak\token[0.0075]| equation|\allowbreak\token[0.7311]| \(|\allowbreak\token[0.4196]| y|\allowbreak\token[0.1097]| -|\allowbreak\token[0.0000]| y|\allowbreak\token[0.0005]|_|\allowbreak\token[0.0001]|1|\allowbreak\token[0.0000]| =|\allowbreak\token[0.0000]| m|\allowbreak\token[0.0014]|(x|\allowbreak\token[0.0000]| -|\allowbreak\token[0.0000]| x|\allowbreak\token[0.0001]|_|\allowbreak\token[0.0001]|1|\allowbreak\token[0.0019]|)|\allowbreak\token[0.0000]| \|\allowbreak\token[0.9733]|)|\allowbreak\token[0.2178]| with|\allowbreak\token[0.0038]| point|\allowbreak\token[0.0005]| \(|\allowbreak\token[0.2980]|A|\allowbreak\token[0.0010]|(|\allowbreak\token[0.0000]|0|\allowbreak\token[0.0000]|,|\allowbreak\token[0.0042]| |\allowbreak\token[0.0001]|6|\allowbreak\token[0.0002]|)\|\allowbreak\token[0.0112]|):↩|\newline\allowbreak\token[0.0001]|\[|\allowbreak\token[0.0000]| y|\allowbreak\token[0.0014]| -|\allowbreak\token[0.0000]| |\allowbreak\token[0.0030]|6|\allowbreak\token[0.0000]| =|\allowbreak[...]
\end{twopartbox}
\begin{twopartbox}{Cut Index 260 | Acceptance Probability 0.0073}{Block 5 Step 2}{Proposal (Rejected)}
[...]\allowbreak\token[0.0329]|{|\allowbreak\token[0.0024]|0|\allowbreak\token[0.6631]|+|\allowbreak\token[0.0000]|0|\allowbreak\token[0.0000]|}{|\allowbreak\token[0.0000]|2|\allowbreak\token[0.0001]|},|\allowbreak\token[0.0000]| \|\allowbreak\token[0.0001]|frac|\allowbreak\token[0.0000]|{|\allowbreak\token[0.1326]|6|\allowbreak\token[0.0002]|+|\allowbreak\token[0.0000]|0|\allowbreak\token[0.0001]|}{|\allowbreak\token[0.0000]|2|\allowbreak\token[0.0000]|}|\allowbreak\token[0.0000]| \|\allowbreak\token[0.0000]|right|\allowbreak\token[0.0004]|)|\allowbreak\token[0.0007]| =|\allowbreak\token[0.0018]| (|\allowbreak\token[0.0000]|0|\allowbreak\token[0.0000]|,|\allowbreak\token[0.0002]| |\allowbreak\token[0.0002]|3|\allowbreak\token[0.0004]|)|\allowbreak\token[0.0001]| \|\allowbreak\token[0.0059]|]↩↩|\tighttracebreak\allowbreak\token[0.1003]|Since|\allowbreak\token[0.0000]| \(|\allowbreak\token[0.0001]|E|\allowbreak\token[0.0003]|\)|\allowbreak\token[0.0003]| is|\allowbreak\token[0.0000]| the|\allowbreak\token[0.0000]| midpoint|\allowbreak\token[0.0000]| of|\allowbreak\token[0.0000]| \|\allowbreak\token[0.0005]|(\|\allowbreak\token[0.0002]|over|\allowbreak\token[0.0000]|line|\allowbreak\token[0.0000]|{|\allowbreak\token[0.0007]|BC|\allowbreak\token[0.0001]|}\|\allowbreak\token[0.0002]|):↩|\newline\allowbreak\token[0.0000]|\[|\allowbreak\token[0.0001]| E|\allowbreak\token[0.0001]| =|\allowbreak\token[0.0000]| \|\allowbreak\token[0.0000]|left|\allowbreak\token[0.0000]|(|\allowbreak\token[0.0000]| \|\allowbreak\token[0.0000]|frac|\allowbreak\token[0.0000]|{|\allowbreak\token[0.0494]|0|\allowbreak\token[0.0008]|+|\allowbreak\token[0.0000]|8|\allowbreak\token[0.0000]|}{|\allowbreak\token[0.0000]|2|\allowbreak\token[0.0000]|},|\allowbreak\token[0.0000]| \|\allowbreak\token[0.0001]|frac|\allowbreak\token[0.0000]|{|\allowbreak\token[0.0001]|0|\allowbreak\token[0.0000]|+|\allowbreak\token[0.0000]|0|\allowbreak\token[0.0001]|}{|\allowbreak\token[0.0002]|2|\allowbreak\token[0.0000]|}|\allowbreak\token[0.0004]| \|\allowbreak\token[0.0004]|right|\allowbreak\token[0.0001]|)|\allowbreak\token[0.0002]| =|\allowbreak\token[0.0002]| (|\allowbreak\token[0.0000]|4|\allowbreak\token[0.0000]|,|\allowbreak\token[0.0001]| |\allowbreak\token[0.0002]|0|\allowbreak\token[0.0005]|)|\allowbreak\token[0.0000]| \|\allowbreak\token[0.0014]|]↩↩|\tighttracebreak\allowbreak\token[1.5966]|Next|\allowbreak\token[0.0009]|,|\allowbreak\token[0.3245]| we|\allowbreak\token[0.3536]| need|\allowbreak\token[0.1951]| to|\allowbreak\token[0.6223]| find|\allowbreak\token[0.0583]| the|\allowbreak\token[0.1246]| coordinates|\allowbreak\token[0.0007]| of|\allowbreak\token[0.6730]| point|\allowbreak\token[0.0001]| \(|\allowbreak\token[0.0021]|F|\allowbreak\token[0.0117]|\|\allowbreak\token[0.6301]|).|\allowbreak\token[1.5853]| Since|\allowbreak\token[0.1894]| \(|\allowbreak\token[0.2436]|F|\allowbreak\token[0.0004]|\)|\allowbreak\token[0.5669]| is|\allowbreak\token[0.8988]| the|\allowbreak\token[0.3516]| intersection|\allowbreak\token[0.0141]| of|\allowbreak\token[1.5777]| line|\allowbreak\token[0.6444]| segments|\allowbreak\token[0.6633]| \(|\allowbreak\token[1.0602]|AE|\allowbreak\token[0.0003]|\)|\allowbreak\token[0.0001]| and|\allowbreak\token[0.0018]| \(|\allowbreak\token[0.9437]|BD|\allowbreak\token[0.0017]|\|\allowbreak\token[0.0306]|),|\allowbreak\token[0.0350]| we|\allowbreak\token[1.2384]| need|\allowbreak\token[0.0125]| to|\allowbreak\token[0.9057]| find|\allowbreak\token[0.0356]| the|\allowbreak\token[0.1261]| equations|\allowbreak\token[0.0003]| of|\allowbreak\token[0.6777]| these|\allowbreak\token[0.0235]| lines|\allowbreak\token[0.5417]| and|\allowbreak\token[0.6232]| solve|\allowbreak\token[0.0630]| for|\allowbreak\token[0.1762]| their|\allowbreak\token[0.0003]| intersection|\allowbreak\token[0.1750]|.↩↩|\tighttracebreak\allowbreak\token[1.3623]|The|\allowbreak\token[0.5901]| equation|\allowbreak\token[0.0022]| of|\allowbreak\token[0.0118]| line|\allowbreak\token[0.0297]| \(|\allowbreak\token[0.0586]|AE|\allowbreak\token[0.6934]|\)|\allowbreak\token[0.6600]| can|\allowbreak\token[0.0000]| be|\allowbreak\token[0.8074]| found|\allowbreak\token[0.1120]| using|\allowbreak\token[0.1335]| the|\allowbreak\token[0.6864]| points|\allowbreak\token[0.0003]| \(|\allowbreak\token[0.0001]|A|\allowbreak\token[0.0006]|(|\allowbreak\token[0.0000]|0|\allowbreak\cut{[CUT]}\allowbreak\token[0.0000]|,|\allowbreak\token[0.0547]| |\allowbreak\token[0.0001]|6|\allowbreak\token[0.0000]|)\|\allowbreak\token[0.0000]|)|\allowbreak\token[0.0000]| and|\allowbreak\token[0.0000]| \(|\allowbreak\token[0.0002]|E|\allowbreak\token[0.0001]|(|\allowbreak\token[0.0001]|4|\allowbreak\token[0.0000]|,|\allowbreak\token[0.0001]| |\allowbreak\token[0.0000]|0|\allowbreak\token[0.0001]|)\|\allowbreak\token[0.0943]|):↩|\newline\allowbreak\token[0.3425]|\[|\allowbreak\token[0.0090]| \|\allowbreak\token[0.0002]|text|\allowbreak\token[0.3148]|{s|\allowbreak\token[0.0000]|lope|\allowbreak\token[0.0162]| of|\allowbreak\token[0.0010]| }|\allowbreak\token[0.0002]| AE|\allowbreak\token[0.0001]| =|\allowbreak\token[0.0000]| \|\allowbreak\token[0.0000]|frac|\allowbreak\token[0.0000]|{|\allowbreak\token[0.0018]|0|\allowbreak\token[0.6759]|-|\allowbreak\token[0.0000]|6|\allowbreak\token[0.0000]|}{|\allowbreak\token[0.0000]|4|\allowbreak\token[0.0000]|-|\allowbreak\token[0.0000]|0|\allowbreak\token[0.0000]|}|\allowbreak\token[0.0000]| =|\allowbreak\token[0.0009]| -|\allowbreak\token[0.0011]|\|\allowbreak\token[0.0000]|frac|\allowbreak\token[0.0000]|{|\allowbreak\token[0.0156]|3|\allowbreak\token[0.0000]|}{|\allowbreak\token[0.0000]|2|\allowbreak\token[0.0001]|}|\allowbreak\token[0.0000]| \|\allowbreak\token[0.0002]|]↩|\newline\allowbreak\token[1.2549]|\[|\allowbreak\token[0.6877]| \|\allowbreak\token[0.0000]|text|\allowbreak\token[0.0000]|{|\allowbreak\token[0.4811]|equ|\allowbreak\token[0.0000]|ation|\allowbreak\token[0.0031]| of|\allowbreak\token[0.0450]| }|\allowbreak\token[0.0000]| AE|\allowbreak\token[0.0365]|:|\allowbreak\token[0.0025]| y|\allowbreak\token[0.3400]| -|\allowbreak\token[0.0000]| |\allowbreak\token[0.0740]|6|\allowbreak\token[0.0000]| =|\allowbreak\token[0.0000]| -|\allowbreak\token[0.0000]|\|\allowbreak\token[0.0000]|frac|\allowbreak\token[0.0000]|{|\allowbreak\token[0.0000]|3|\allowbreak\token[0.0000]|}{|\allowbreak\token[0.0000]|2|\allowbreak\token[0.6760]|}(|\allowbreak\token[0.0000]|x|\allowbreak\token[0.0004]| -|\allowbreak\token[0.0000]| |\allowbreak\token[0.0001]|0|\allowbreak\token[0.0000]|)|\allowbreak\token[0.0001]| \|\allowbreak\token[0.9203]|impl|\allowbreak\token[0.0000]|ies|\allowbreak\token[0.0003]| y|\allowbreak\token[0.0000]| =|\allowbreak\token[0.3148]| -|\allowbreak\token[0.0000]|\|\allowbreak\token[0.0000]|frac|\allowbreak\token[0.0000]|{|\allowbreak\token[0.0000]|3|\allowbreak\token[0.0000]|}{|\allowbreak\token[0.0000]|2|\allowbreak\token[0.0000]|}|\allowbreak\token[0.0001]|x|\allowbreak\token[0.0000]| +|\allowbreak\token[0.0000]| |\allowbreak\token[0.0000]|6|\allowbreak\token[0.0000]| \|\allowbreak\token[0.0002]|]↩↩|\tighttracebreak\allowbreak\token[0.0003]|The|\allowbreak\token[0.0000]| equation|\allowbreak\token[0.0000]| of|\allowbreak\token[0.0000]| line|\allowbreak\token[0.0000]| \(|\allowbreak\token[0.0001]|BD|\allowbreak\token[0.0001]|\)|\allowbreak\token[0.0005]| can|\allowbreak\token[0.0000]| be|\allowbreak\token[0.0000]| found|\allowbreak\token[0.0000]| using|\allowbreak\token[0.0015]| the|\allowbreak\token[0.0000]| points|\allowbreak\token[0.0000]| \(|\allowbreak\token[0.0010]|B|\allowbreak\token[0.0000]|(|\allowbreak\token[0.0000]|0|\allowbreak\token[0.0000]|,|\allowbreak\token[0.0027]| |\allowbreak\token[0.0000]|0|\allowbreak\token[0.0000]|)\|\allowbreak\token[0.0010]|)|\allowbreak\token[0.0000]| and|\allowbreak\token[0.0000]| \(|\allowbreak\token[0.0003]|D|\allowbreak\token[0.0001]|(|\allowbreak\token[0.0002]|0|\allowbreak\token[0.0000]|,|\allowbreak\token[0.0000]| |\allowbreak\token[0.0001]|3|\allowbreak\token[0.0002]|)\|\allowbreak\token[0.0552]|):↩|\newline\allowbreak\token[0.1657]|\[|\allowbreak\token[0.0009]| \|\allowbreak\token[0.0000]|text|\allowbreak\token[0.2474]|{s|\allowbreak\token[0.0000]|lope|\allowbreak\token[0.0000]| of|\allowbreak\token[0.0000]| }|\allowbreak\token[0.0000]| BD|\allowbreak\token[0.0181]| =|\allowbreak\token[0.0011]| \|\allowbreak\token[0.2142]|frac|\allowbreak\token[0.0000]|{|\allowbreak\token[0.0014]|3|\allowbreak\token[0.0005]|-|\allowbreak\token[0.0000]|0|\allowbreak\token[0.0000]|}{|\allowbreak\token[0.0002]|0|\allowbreak\token[0.0001]|-|\allowbreak\token[0.0005]|0|\allowbreak[...]
\tcblower
\token[0.0000]|,|\allowbreak\token[0.0547]| |\allowbreak\token[0.0001]|6|\allowbreak\token[0.0000]|)\|\allowbreak\token[0.0000]|)|\allowbreak\token[0.0000]| and|\allowbreak\token[0.0000]| \(|\allowbreak\token[0.0002]|E|\allowbreak\token[0.0001]|(|\allowbreak\token[0.0001]|4|\allowbreak\token[0.0000]|,|\allowbreak\token[0.0001]| |\allowbreak\token[0.0000]|0|\allowbreak\token[0.0001]|)\|\allowbreak\token[0.1037]|):↩|\newline\allowbreak\token[0.3116]|\[|\allowbreak\token[0.0087]| \|\allowbreak\token[0.0002]|text|\allowbreak\token[0.3148]|{s|\allowbreak\token[0.0000]|lope|\allowbreak\token[0.0131]| of|\allowbreak\token[0.0010]| }|\allowbreak\token[0.0003]| AE|\allowbreak\token[0.0002]| =|\allowbreak\token[0.0000]| \|\allowbreak\token[0.0000]|frac|\allowbreak\token[0.0000]|{|\allowbreak\token[0.0018]|0|\allowbreak\token[0.6759]| -|\allowbreak\token[0.0000]| |\allowbreak\token[0.0000]|6|\allowbreak\token[0.0000]|}{|\allowbreak\token[0.0000]|4|\allowbreak\token[0.0000]| -|\allowbreak\token[0.0000]| |\allowbreak\token[0.0000]|0|\allowbreak\token[0.0000]|}|\allowbreak\token[0.0000]| =|\allowbreak\token[0.0014]| -|\allowbreak\token[0.0020]|\|\allowbreak\token[0.0000]|frac|\allowbreak\token[0.0000]|{|\allowbreak\token[0.0173]|3|\allowbreak\token[0.0000]|}{|\allowbreak\token[0.0000]|2|\allowbreak\token[0.0001]|}|\allowbreak\token[0.0000]| \|\allowbreak\token[0.0002]|]↩|\newline\allowbreak\token[1.2326]|\[|\allowbreak\token[0.6877]| \|\allowbreak\token[0.0000]|text|\allowbreak\token[0.0000]|{|\allowbreak\token[0.4239]|equ|\allowbreak\token[0.0000]|ation|\allowbreak\token[0.0028]| of|\allowbreak\token[0.0449]| }|\allowbreak\token[0.0000]| AE|\allowbreak\token[0.0387]|:|\allowbreak\token[0.0028]| y|\allowbreak\token[0.3396]| -|\allowbreak\token[0.0000]| |\allowbreak\token[0.0740]|6|\allowbreak\token[0.0000]| =|\allowbreak\token[0.0000]| -|\allowbreak\token[0.0000]|\|\allowbreak\token[0.0000]|frac|\allowbreak\token[0.0000]|{|\allowbreak\token[0.0000]|3|\allowbreak\token[0.0000]|}{|\allowbreak\token[0.0000]|2|\allowbreak\token[0.6913]|}(|\allowbreak\token[0.0000]|x|\allowbreak\token[0.0001]| -|\allowbreak\token[0.0000]| |\allowbreak\token[0.0001]|0|\allowbreak\token[0.0000]|)|\allowbreak\token[0.0001]| \|\allowbreak\token[0.8992]|impl|\allowbreak\token[0.0000]|ies|\allowbreak\token[0.0002]| y|\allowbreak\token[0.0000]| =|\allowbreak\token[0.3148]| -|\allowbreak\token[0.0000]|\|\allowbreak\token[0.0000]|frac|\allowbreak\token[0.0000]|{|\allowbreak\token[0.0000]|3|\allowbreak\token[0.0000]|}{|\allowbreak\token[0.0000]|2|\allowbreak\token[0.0000]|}|\allowbreak\token[0.0001]|x|\allowbreak\token[0.0000]| +|\allowbreak\token[0.0000]| |\allowbreak\token[0.0000]|6|\allowbreak\token[0.0000]| \|\allowbreak\token[0.0002]|]↩↩|\tighttracebreak\allowbreak\token[0.0002]|The|\allowbreak\token[0.0000]| equation|\allowbreak\token[0.0000]| of|\allowbreak\token[0.0000]| line|\allowbreak\token[0.0000]| \(|\allowbreak\token[0.0001]|BD|\allowbreak\token[0.0001]|\)|\allowbreak\token[0.0004]| can|\allowbreak\token[0.0000]| be|\allowbreak\token[0.0000]| found|\allowbreak\token[0.0000]| using|\allowbreak\token[0.0015]| the|\allowbreak\token[0.0000]| points|\allowbreak\token[0.0000]| \(|\allowbreak\token[0.0009]|B|\allowbreak\token[0.0000]|(|\allowbreak\token[0.0000]|0|\allowbreak\token[0.0000]|,|\allowbreak\token[0.0006]| |\allowbreak\token[0.0000]|0|\allowbreak\token[0.0000]|)\|\allowbreak\token[0.0007]|)|\allowbreak\token[0.0000]| and|\allowbreak\token[0.0000]| \(|\allowbreak\token[0.0003]|D|\allowbreak\token[0.0001]|(|\allowbreak\token[0.0002]|0|\allowbreak\token[0.0001]|,|\allowbreak\token[0.0000]| |\allowbreak\token[0.0001]|3|\allowbreak\token[0.0002]|)\|\allowbreak\token[0.0502]|):↩|\newline\allowbreak\token[0.1983]|\[|\allowbreak\token[0.0009]| \|\allowbreak\token[0.0000]|text|\allowbreak\token[0.2474]|{s|\allowbreak\token[0.0000]|lope|\allowbreak\token[0.0000]| of|\allowbreak\token[0.0000]| }|\allowbreak\token[0.0000]| BD|\allowbreak\token[0.0201]| =|\allowbreak\token[0.0010]| \|\allowbreak\token[0.2127]|frac|\allowbreak\token[0.0000]|{|\allowbreak\token[0.0015]|3|\allowbreak\token[0.0007]| -|\allowbreak\token[0.0000]| |\allowbreak\token[0.0000]|0|\allowbreak\token[0.0000]|}{|\allowbreak[...]
\end{twopartbox}
\begin{twopartbox}{Cut Index 734 | Acceptance Probability 0.0798}{Block 5 Step 3}{Proposal (Rejected)}
[...]\allowbreak\token[0.0000]|_|\allowbreak\token[0.0000]|1|\allowbreak\token[0.0030]|y|\allowbreak\token[0.0000]|_|\allowbreak\token[0.0000]|2|\allowbreak\token[0.0000]| +|\allowbreak\token[0.0000]| x|\allowbreak\token[0.0000]|_|\allowbreak\token[0.0000]|2|\allowbreak\token[0.0000]|y|\allowbreak\token[0.0000]|_|\allowbreak\token[0.0000]|3|\allowbreak\token[0.0000]| +|\allowbreak\token[0.0000]| x|\allowbreak\token[0.0000]|_|\allowbreak\token[0.0000]|3|\allowbreak\token[0.0000]|y|\allowbreak\token[0.0000]|_|\allowbreak\token[0.0000]|4|\allowbreak\token[0.0000]| +|\allowbreak\token[0.0000]| x|\allowbreak\token[0.0000]|_|\allowbreak\token[0.0000]|4|\allowbreak\token[0.0000]|y|\allowbreak\token[0.0000]|_|\allowbreak\token[0.0001]|1|\allowbreak\token[0.0000]| -|\allowbreak\token[0.0669]| (|\allowbreak\token[0.0000]|y|\allowbreak\token[0.0000]|_|\allowbreak\token[0.0000]|1|\allowbreak\token[0.0000]|x|\allowbreak\token[0.0000]|_|\allowbreak\token[0.0000]|2|\allowbreak\token[0.0000]| +|\allowbreak\token[0.0000]| y|\allowbreak\token[0.0000]|_|\allowbreak\token[0.0000]|2|\allowbreak\token[0.0001]|x|\allowbreak\token[0.0000]|_|\allowbreak\token[0.0000]|3|\allowbreak\token[0.0000]| +|\allowbreak\token[0.0000]| y|\allowbreak\token[0.0000]|_|\allowbreak\token[0.0000]|3|\allowbreak\token[0.0000]|x|\allowbreak\token[0.0000]|_|\allowbreak\token[0.0000]|4|\allowbreak\token[0.0000]| +|\allowbreak\token[0.0000]| y|\allowbreak\token[0.0000]|_|\allowbreak\token[0.0000]|4|\allowbreak\token[0.0000]|x|\allowbreak\token[0.0000]|_|\allowbreak\token[0.0001]|1|\allowbreak\token[0.0000]|)|\allowbreak\token[0.0000]| \|\allowbreak\token[0.0000]|right|\allowbreak\token[0.0000]!|!\allowbreak\token[0.0000]| \|\allowbreak\token[0.6059]|]↩|\newline\allowbreak\token[0.5098]|Sub|\allowbreak\token[0.0126]|stit|\allowbreak\token[0.0000]|uting|\allowbreak\token[0.1482]| the|\allowbreak\token[0.0004]| coordinates|\allowbreak\token[0.4945]|:↩|\newline\allowbreak\token[0.0001]|\[|\allowbreak\token[0.0061]| \|\allowbreak\token[0.0000]|text|\allowbreak\token[0.0000]|{|\allowbreak\token[0.0000]|Area|\allowbreak\token[0.0000]|}|\allowbreak\token[0.0000]| =|\allowbreak\token[0.0000]| \|\allowbreak\token[0.0000]|frac|\allowbreak\token[0.0000]|{|\allowbreak\token[0.0000]|1|\allowbreak\token[0.0001]|}{|\allowbreak\token[0.0000]|2|\allowbreak\token[0.0000]|}|\allowbreak\token[0.0000]| \|\allowbreak\token[0.0000]|left|\allowbreak\token[0.0000]!|!\allowbreak\token[0.0065]| |\allowbreak\token[0.0001]|0|\allowbreak\token[0.1438]| \|\allowbreak\token[0.0000]|cdot|\allowbreak\token[0.0000]| |\allowbreak\token[0.0083]|0|\allowbreak\token[0.0000]| +|\allowbreak\token[0.0000]| |\allowbreak\token[0.0001]|0|\allowbreak\token[0.0000]| \|\allowbreak\token[0.0000]|cdot|\allowbreak\token[0.0001]| |\allowbreak\token[0.0086]|0|\allowbreak\token[0.0000]| +|\allowbreak\token[0.0000]| |\allowbreak\token[0.0013]|4|\allowbreak\token[0.0000]| \|\allowbreak\token[0.0000]|cdot|\allowbreak\token[0.0000]| |\allowbreak\token[0.0131]|6|\allowbreak\token[0.0000]| +|\allowbreak\token[0.0000]| |\allowbreak\token[0.0001]|0|\allowbreak\token[0.0000]| \|\allowbreak\token[0.0000]|cdot|\allowbreak\token[0.0000]| |\allowbreak\token[0.0001]|3|\allowbreak\token[0.0000]| -|\allowbreak\token[0.0000]| (|\allowbreak\token[0.0004]|3|\allowbreak\token[0.0000]| \|\allowbreak\token[0.0000]|cdot|\allowbreak\token[0.0000]| |\allowbreak\token[0.0000]|0|\allowbreak\token[0.0000]| +|\allowbreak\token[0.0000]| |\allowbreak\token[0.0000]|0|\allowbreak\token[0.0000]| \|\allowbreak\token[0.0000]|cdot|\allowbreak\token[0.0000]| |\allowbreak\token[0.0022]|4|\allowbreak\token[0.0000]| +|\allowbreak\token[0.0000]| |\allowbreak\token[0.0001]|0|\allowbreak\token[0.0000]| \|\allowbreak\token[0.0000]|cdot|\allowbreak\token[0.0000]| |\allowbreak\token[0.0020]|0|\allowbreak\token[0.0000]| +|\allowbreak\token[0.0000]| |\allowbreak\token[0.0001]|6|\allowbreak\token[0.0000]| \|\allowbreak\token[0.0000]|cdot|\allowbreak\token[0.0000]| |\allowbreak\token[0.0001]|0|\allowbreak\token[0.0000]|)|\allowbreak\token[0.0000]| \|\allowbreak\token[0.0001]|right|\allowbreak\token[0.0000]!|!\allowbreak\token[0.0446]| \|\allowbreak\cut{[CUT]}\allowbreak\token[0.0000]|]↩|\newline\allowbreak\token[0.0001]|\[|\allowbreak\token[0.1909]| \|\allowbreak\token[0.0000]|text|\allowbreak\token[0.0000]|{|\allowbreak\token[0.0000]|Area|\allowbreak\token[0.0000]|}|\allowbreak\token[0.0000]| =|\allowbreak\token[0.0000]| \|\allowbreak\token[0.0000]|frac|\allowbreak\token[0.0000]|{|\allowbreak\token[0.0000]|1|\allowbreak\token[0.0000]|}{|\allowbreak\token[0.0000]|2|\allowbreak\token[0.0000]|}|\allowbreak\token[0.0000]| \|\allowbreak\token[0.0000]|left|\allowbreak\token[0.0000]!|!\allowbreak\token[0.0000]| |\allowbreak\token[0.0042]|0|\allowbreak\token[0.0000]| +|\allowbreak\token[0.0000]| |\allowbreak\token[0.0000]|0|\allowbreak\token[0.0000]| +|\allowbreak\token[0.0000]| |\allowbreak\token[0.0000]|2|\allowbreak\token[0.0000]|4|\allowbreak\token[0.0000]| +|\allowbreak\token[0.0000]| |\allowbreak\token[0.0000]|0|\allowbreak\token[0.0000]| -|\allowbreak\token[0.0000]| (|\allowbreak\token[0.0000]|0|\allowbreak\token[0.0000]| +|\allowbreak\token[0.0000]| |\allowbreak\token[0.0000]|0|\allowbreak\token[0.0000]| +|\allowbreak\token[0.0000]| |\allowbreak\token[0.0000]|0|\allowbreak\token[0.0001]| +|\allowbreak\token[0.0000]| |\allowbreak\token[0.0000]|0|\allowbreak\token[0.0000]|)|\allowbreak\token[0.0000]| \|\allowbreak\token[0.0001]|right|\allowbreak\token[0.0000]!|!\allowbreak\token[0.0993]| \|\allowbreak\token[0.0000]|]↩|\newline\allowbreak\token[0.0000]|\[|\allowbreak\token[0.0000]| \|\allowbreak\token[0.0001]|text|\allowbreak\token[0.0000]|{|\allowbreak\token[0.0000]|Area|\allowbreak\token[0.0000]|}|\allowbreak\token[0.0000]| =|\allowbreak\token[0.0000]| \|\allowbreak\token[0.0000]|frac|\allowbreak\token[0.0000]|{|\allowbreak\token[0.0000]|1|\allowbreak\token[0.0000]|}{|\allowbreak\token[0.0000]|2|\allowbreak\token[0.0000]|}|\allowbreak\token[0.0000]| \|\allowbreak\token[0.0783]|left|\allowbreak\token[0.0000]!|!\allowbreak\token[0.0000]| |\allowbreak\token[0.0000]|2|\allowbreak\token[0.0000]|4|\allowbreak\token[0.0004]| \|\allowbreak\token[0.0000]|right|\allowbreak\token[0.0000]!|!\allowbreak\token[0.6912]| \|\allowbreak\token[0.0000]|]↩|\newline\allowbreak\token[0.0000]|\[|\allowbreak\token[0.0000]| \|\allowbreak\token[0.0038]|text|\allowbreak\token[0.0000]|{|\allowbreak\token[0.0000]|Area|\allowbreak\token[0.0000]|}|\allowbreak\token[0.0000]| =|\allowbreak\token[0.2084]| \|\allowbreak\token[0.0008]|frac|\allowbreak\token[0.0000]|{|\allowbreak\token[0.0669]|1|\allowbreak\token[0.0001]|}{|\allowbreak\token[0.0000]|2|\allowbreak\token[0.0000]|}|\allowbreak\token[0.0000]| \|\allowbreak\token[0.6275]|cdot|\allowbreak\token[0.0000]| |\allowbreak\token[0.0001]|2|\allowbreak\token[0.0000]|4|\allowbreak\token[0.6854]| \|\allowbreak\token[0.0000]|]↩|\newline\allowbreak\token[0.0002]|\[|\allowbreak\token[0.0000]| \|\allowbreak\token[0.0238]|text|\allowbreak\token[0.0000]|{|\allowbreak\token[0.0000]|Area|\allowbreak\token[0.0000]|}|\allowbreak\token[0.0000]| =|\allowbreak\token[0.0001]| |\allowbreak\token[0.0004]|1|\allowbreak\token[0.0002]|2|\allowbreak\token[0.0001]| \|\allowbreak\token[0.0018]|]↩↩|\tighttracebreak\allowbreak\token[1.1090]|So|\allowbreak\token[0.6466]|,|\allowbreak\token[0.0000]| the|\allowbreak\token[0.0113]| area|\allowbreak\token[0.0000]| of|\allowbreak\token[0.0000]| quadr|\allowbreak\token[0.0000]|ilateral|\allowbreak\token[0.0000]| \(|\allowbreak\token[0.0000]|DB|\allowbreak\token[0.0000]|EF|\allowbreak\token[0.0001]|\)|\allowbreak\token[0.0000]| is|\allowbreak\token[0.3484]| \|\allowbreak\token[0.0009]|(\|\allowbreak\token[0.0000]|boxed|\allowbreak\token[0.0005]|{|\allowbreak\token[0.5930]|1|\allowbreak\token[0.0119]|2|\allowbreak\token[0.0032]|}\|\allowbreak\token[0.6711]|).|\allowbreak\token[0.2237]!<|endoftext|>!
\tcblower
\token[0.0000]|]↩|\newline\allowbreak\token[0.0001]|\[|\allowbreak\token[0.2084]| \|\allowbreak\token[0.0000]|text|\allowbreak\token[0.0000]|{|\allowbreak\token[0.0000]|Area|\allowbreak\token[0.0000]|}|\allowbreak\token[0.0000]| =|\allowbreak\token[0.0000]| \|\allowbreak\token[0.0000]|frac|\allowbreak\token[0.0000]|{|\allowbreak\token[0.0000]|1|\allowbreak\token[0.0000]|}{|\allowbreak\token[0.0000]|2|\allowbreak\token[0.0000]|}|\allowbreak\token[0.0000]| \|\allowbreak\token[0.0000]|left|\allowbreak\token[0.0000]!|!\allowbreak\token[0.0000]| |\allowbreak\token[0.0038]|0|\allowbreak\token[0.0000]| +|\allowbreak\token[0.0000]| |\allowbreak\token[0.0000]|0|\allowbreak\token[0.0000]| +|\allowbreak\token[0.0000]| |\allowbreak\token[0.0000]|2|\allowbreak\token[0.0000]|4|\allowbreak\token[0.0000]| +|\allowbreak\token[0.0000]| |\allowbreak\token[0.0000]|0|\allowbreak\token[0.0000]| -|\allowbreak\token[0.0000]| (|\allowbreak\token[0.0000]|0|\allowbreak\token[0.0000]| +|\allowbreak\token[0.0000]| |\allowbreak\token[0.0000]|0|\allowbreak\token[0.0000]| +|\allowbreak\token[0.0000]| |\allowbreak\token[0.0000]|0|\allowbreak\token[0.0001]| +|\allowbreak\token[0.0000]| |\allowbreak\token[0.0000]|0|\allowbreak\token[0.0000]|)|\allowbreak\token[0.0000]| \|\allowbreak\token[0.0001]|right|\allowbreak\token[0.0000]!|!\allowbreak\token[0.1094]| \|\allowbreak\token[0.0000]|]↩|\newline\allowbreak\token[0.0000]|\[|\allowbreak\token[0.0000]| \|\allowbreak\token[0.0000]|text|\allowbreak\token[0.0000]|{|\allowbreak\token[0.0000]|Area|\allowbreak\token[0.0000]|}|\allowbreak\token[0.0000]| =|\allowbreak\token[0.0000]| \|\allowbreak\token[0.0000]|frac|\allowbreak\token[0.0000]|{|\allowbreak\token[0.0000]|1|\allowbreak\token[0.0000]|}{|\allowbreak\token[0.0000]|2|\allowbreak\token[0.0000]|}|\allowbreak\token[0.0000]| \|\allowbreak\token[0.0868]|left|\allowbreak\token[0.0000]!|!\allowbreak\token[0.0000]| |\allowbreak\token[0.0000]|2|\allowbreak\token[0.0000]|4|\allowbreak\token[0.0004]| \|\allowbreak\token[0.0000]|right|\allowbreak\token[0.0000]!|!\allowbreak\token[0.6854]| \|\allowbreak\token[0.0000]|]↩|\newline\allowbreak\token[0.0000]|\[|\allowbreak\token[0.0000]| \|\allowbreak\token[0.0038]|text|\allowbreak\token[0.0000]|{|\allowbreak\token[0.0000]|Area|\allowbreak\token[0.0000]|}|\allowbreak\token[0.0000]| =|\allowbreak\token[0.2084]| \|\allowbreak\token[0.0007]|frac|\allowbreak\token[0.0000]|{|\allowbreak\token[0.0739]|1|\allowbreak\token[0.0002]|}{|\allowbreak\token[0.0000]|2|\allowbreak\token[0.0000]|}|\allowbreak\token[0.0000]| \|\allowbreak\token[0.6059]|cdot|\allowbreak\token[0.0000]| |\allowbreak\token[0.0001]|2|\allowbreak\token[0.0000]|4|\allowbreak\token[0.6854]| =|\allowbreak\token[0.0000]| |\allowbreak\token[0.0006]|1|\allowbreak\token[0.0002]|2|\allowbreak\token[0.0001]| \|\allowbreak\token[0.0030]|]↩↩|\tighttracebreak\allowbreak\token[1.0035]|So|\allowbreak\token[0.6629]|,|\allowbreak\token[0.0000]| the|\allowbreak\token[0.0042]| area|\allowbreak\token[0.0000]| of|\allowbreak\token[0.0000]| quadr|\allowbreak\token[0.0000]|ilateral|\allowbreak\token[0.0000]| \(|\allowbreak\token[0.0000]|DB|\allowbreak\token[0.0000]|EF|\allowbreak\token[0.0001]|\)|\allowbreak\token[0.0000]| is|\allowbreak\token[0.6022]| \|\allowbreak\token[0.0008]|(\|\allowbreak\token[0.0000]|boxed|\allowbreak\token[0.0004]|{|\allowbreak\token[0.5602]|1|\allowbreak\token[0.0111]|2|\allowbreak\token[0.0022]|}\|\allowbreak\token[0.7535]|).|\allowbreak\token[0.2748]!<|endoftext|>!
\end{twopartbox}

\endgroup

\paragraph{Entropy-Cut Trace.}
\begingroup
\raggedright %
\setlength{\lineskip}{0pt} %
\setlength{\baselineskip}{0.9em} %
\begin{twopartbox}{Cut Index 25 | Acceptance Probability 0.0000}{Block 1 Step 1}{Proposal (Rejected)}
\token[2.6002]| To|\allowbreak\token[0.6725]| solve|\allowbreak\token[0.8316]| the|\allowbreak\token[0.0687]| problem|\allowbreak\token[0.5135]|,|\allowbreak\token[0.3361]| we|\allowbreak\token[0.4652]| need|\allowbreak\token[0.0016]| to|\allowbreak\token[0.7385]| determine|\allowbreak\token[0.0004]| the|\allowbreak\token[0.1338]| area|\allowbreak\token[0.0002]| of|\allowbreak\token[0.1912]| quadr|\allowbreak\token[0.0000]|ilateral|\allowbreak\token[0.5312]| \(|\allowbreak\token[0.3013]|DB|\allowbreak\token[0.0001]|EF|\allowbreak\token[0.6859]|\|\allowbreak\token[0.3149]|).|\allowbreak\token[1.6625]| Here|\allowbreak\token[1.0700]|'s|\allowbreak\token[0.6866]| the|\allowbreak\token[0.1679]| step|\allowbreak\token[0.0002]|-by|\allowbreak\token[0.0001]|-step|\allowbreak\cut{[CUT]}\allowbreak\token[1.4678]| reasoning|\allowbreak\token[0.6935]|:↩↩|\tighttracebreak\allowbreak\token[0.0038]|1|\allowbreak\token[0.0000]|.|\allowbreak\token[0.2652]| **|\allowbreak\token[0.2757]|Ident|\allowbreak\token[0.0027]|ify|\allowbreak\token[0.6453]| Coordinates|\allowbreak\token[1.0409]| of|\allowbreak\token[0.6452]| Points|\allowbreak\token[0.9793]|:|\allowbreak\token[0.0009]|**↩|\newline\allowbreak\token[0.2687]|  |\allowbreak\token[0.0035]| -|\allowbreak\token[0.9269]| \(|\allowbreak\token[0.1889]|A|\allowbreak\token[0.7277]| =|\allowbreak\token[0.0000]| (|\allowbreak\token[0.0000]|0|\allowbreak\token[0.0000]|,|\allowbreak\token[0.0081]| |\allowbreak\token[0.0001]|6|\allowbreak\token[0.0000]|)\|\allowbreak\token[0.0029]|)↩|\newline\allowbreak\token[0.0000]|  |\allowbreak\token[0.0000]| -|\allowbreak\token[0.0000]| \(|\allowbreak\token[0.0003]|B|\allowbreak\token[0.0000]| =|\allowbreak\token[0.0000]| (|\allowbreak\token[0.0000]|0|\allowbreak\token[0.0000]|,|\allowbreak\token[0.0000]| |\allowbreak\token[0.0000]|0|\allowbreak\token[0.0000]|)\|\allowbreak\token[0.0002]|)↩|\newline\allowbreak\token[0.0000]|  |\allowbreak\token[0.0000]| -|\allowbreak\token[0.0003]| \(|\allowbreak\token[0.0002]|C|\allowbreak\token[0.0000]| =|\allowbreak\token[0.0000]| (|\allowbreak\token[0.0000]|8|\allowbreak\token[0.0000]|,|\allowbreak\token[0.0000]| |\allowbreak\token[0.0000]|0|\allowbreak\token[0.0000]|)\|\allowbreak\token[0.0217]|)↩|\newline\allowbreak\token[0.0047]|  |\allowbreak\token[0.0005]| -|\allowbreak\token[0.2302]| \(|\allowbreak\token[0.0013]|D|\allowbreak\token[0.2492]|\)|\allowbreak\token[0.0301]| is|\allowbreak\token[0.0001]| the|\allowbreak\token[0.0000]| midpoint|\allowbreak\token[0.0000]| of|\allowbreak\token[0.1095]| \|\allowbreak\token[0.0001]|(\|\allowbreak\token[0.0000]|over|\allowbreak\token[0.0000]|line|\allowbreak\token[0.0000]|{|\allowbreak\token[0.0000]|AB|\allowbreak\token[0.0018]|}\|\allowbreak\token[0.0466]|),|\allowbreak\token[0.0004]| so|\allowbreak\token[0.0002]| \(|\allowbreak\token[0.0000]|D|\allowbreak\token[0.0000]| =|\allowbreak\token[0.3154]| \|\allowbreak\token[0.0000]|left|\allowbreak\token[0.6934]|(\|\allowbreak\token[0.0016]|frac|\allowbreak\token[0.0000]|{|\allowbreak\token[0.0000]|0|\allowbreak\token[0.1324]|+|\allowbreak\token[0.0000]|0|\allowbreak\token[0.0000]|}{|\allowbreak\token[0.0000]|2|\allowbreak\token[0.0000]|},|\allowbreak\token[0.0000]| \|\allowbreak\token[0.0000]|frac|\allowbreak\token[0.0000]|{|\allowbreak\token[0.1747]|6|\allowbreak\token[0.0000]|+|\allowbreak\token[0.0000]|0|\allowbreak\token[0.0001]|}{|\allowbreak\token[0.0000]|2|\allowbreak\token[0.0000]|}\|\allowbreak\token[0.0001]|right|\allowbreak\token[0.0000]|)|\allowbreak\token[0.0000]| =|\allowbreak\token[0.0000]| (|\allowbreak\token[0.0000]|0|\allowbreak\token[0.0000]|,|\allowbreak\token[0.0001]| |\allowbreak\token[0.0000]|3|\allowbreak\token[0.0000]|)\|\allowbreak\token[0.0193]|)↩|\newline\allowbreak\token[0.0000]|  |\allowbreak\token[0.0000]| -|\allowbreak\token[0.0000]| \(|\allowbreak\token[0.0000]|E|\allowbreak\token[0.0000]|\)|\allowbreak\token[0.0000]| is|\allowbreak\token[0.0000]| the|\allowbreak\token[0.0000]| midpoint|\allowbreak\token[0.0000]| of|\allowbreak\token[0.0000]| \|\allowbreak\token[0.0000]|(\|\allowbreak\token[0.0000]|over|\allowbreak\token[0.0000]|line|\allowbreak\token[0.0000]|{|\allowbreak\token[0.0001]|BC|\allowbreak\token[0.0000]|}\|\allowbreak\token[0.0000]|),|\allowbreak\token[0.0000]| so|\allowbreak\token[0.0000]| \(|\allowbreak\token[0.0001]|E|\allowbreak\token[0.0000]| =|\allowbreak\token[0.0000]| \|\allowbreak\token[0.0000]|left|\allowbreak\token[0.0000]|(\|\allowbreak\token[0.0000]|frac|\allowbreak\token[0.0000]|{|\allowbreak\token[0.0173]|0|\allowbreak\token[0.0000]|+|\allowbreak\token[0.0000]|8|\allowbreak\token[0.0000]|}{|\allowbreak\token[0.0000]|2|\allowbreak\token[0.0000]|},|\allowbreak\token[0.0000]| \|\allowbreak\token[0.0001]|frac|\allowbreak\token[0.0000]|{|\allowbreak\token[0.0000]|0|\allowbreak\token[0.0000]|+|\allowbreak\token[0.0000]|0|\allowbreak\token[0.0000]|}{|\allowbreak\token[0.0001]|2|\allowbreak\token[0.0000]|}\|\allowbreak\token[0.0001]|right|\allowbreak\token[0.0000]|)|\allowbreak\token[0.0000]| =|\allowbreak\token[0.0001]| (|\allowbreak[...]
\tcblower
\token[1.4676]| approach|\allowbreak\token[0.1317]|:↩↩|\tighttracebreak\allowbreak\token[0.0015]|1|\allowbreak\token[0.0000]|.|\allowbreak\token[0.4088]| **|\allowbreak\token[0.3857]|Ident|\allowbreak\token[0.0010]|ify|\allowbreak\token[0.6352]| Coordinates|\allowbreak\token[1.0109]| of|\allowbreak\token[0.7825]| Points|\allowbreak\token[1.0858]|:|\allowbreak\token[0.0014]|**↩|\newline\allowbreak\token[0.2691]|  |\allowbreak\token[0.0177]| -|\allowbreak\token[1.1013]| \(|\allowbreak\token[0.3163]|A|\allowbreak\token[0.7286]|(|\allowbreak\token[0.0000]|0|\allowbreak\token[0.0000]|,|\allowbreak\token[0.5030]| |\allowbreak\token[0.0001]|6|\allowbreak\token[0.0000]|)\|\allowbreak\token[0.0294]|)↩|\newline\allowbreak\token[0.0000]|  |\allowbreak\token[0.0000]| -|\allowbreak\token[0.0000]| \(|\allowbreak\token[0.0005]|B|\allowbreak\token[0.0000]|(|\allowbreak\token[0.0000]|0|\allowbreak\token[0.0000]|,|\allowbreak\token[0.0001]| |\allowbreak\token[0.0000]|0|\allowbreak\token[0.0000]|)\|\allowbreak\token[0.0003]|)↩|\newline\allowbreak\token[0.0000]|  |\allowbreak\token[0.0000]| -|\allowbreak\token[0.0002]| \(|\allowbreak\token[0.0002]|C|\allowbreak\token[0.0000]|(|\allowbreak\token[0.0001]|8|\allowbreak\token[0.0000]|,|\allowbreak\token[0.0001]| |\allowbreak\token[0.0000]|0|\allowbreak\token[0.0000]|)\|\allowbreak\token[0.0406]|)↩|\newline\allowbreak\token[0.0216]|  |\allowbreak\token[0.0006]| -|\allowbreak\token[0.1459]| \(|\allowbreak\token[0.0010]|D|\allowbreak\token[0.4261]|\)|\allowbreak\token[0.0842]| is|\allowbreak\token[0.0006]| the|\allowbreak\token[0.0000]| midpoint|\allowbreak\token[0.0000]| of|\allowbreak\token[0.1095]| \|\allowbreak\token[0.0001]|(\|\allowbreak\token[0.0000]|over|\allowbreak\token[0.0001]|line|\allowbreak\token[0.0000]|{|\allowbreak\token[0.0000]|AB|\allowbreak\token[0.0010]|}\|\allowbreak\token[0.1887]|),|\allowbreak\token[0.0019]| so|\allowbreak\token[0.0023]| \(|\allowbreak\token[0.0000]|D|\allowbreak\token[0.5266]|(|\allowbreak\token[0.0007]|0|\allowbreak\token[0.0001]|,|\allowbreak\token[0.0331]| |\allowbreak\token[0.0001]|3|\allowbreak\token[0.0000]|)\|\allowbreak\token[0.0536]|)↩|\newline\allowbreak\token[0.0000]|  |\allowbreak\token[0.0000]| -|\allowbreak\token[0.0000]| \(|\allowbreak\token[0.0000]|E|\allowbreak\token[0.0001]|\)|\allowbreak\token[0.0002]| is|\allowbreak\token[0.0000]| the|\allowbreak\token[0.0000]| midpoint|\allowbreak\token[0.0000]| of|\allowbreak\token[0.0000]| \|\allowbreak\token[0.0000]|(\|\allowbreak\token[0.0002]|over|\allowbreak\token[0.0001]|line|\allowbreak\token[0.0000]|{|\allowbreak\token[0.0003]|BC|\allowbreak\token[0.0000]|}\|\allowbreak\token[0.0001]|),|\allowbreak\token[0.0000]| so|\allowbreak\token[0.0000]| \(|\allowbreak\token[0.0001]|E|\allowbreak\token[0.0010]|(|\allowbreak\token[0.0000]|4|\allowbreak\token[0.0000]|,|\allowbreak\token[0.0000]| |\allowbreak\token[0.0001]|0|\allowbreak\token[0.0000]|)\|\allowbreak\token[0.5040]|)↩|\newline\allowbreak\token[0.1106]|  |\allowbreak\token[0.0001]| -|\allowbreak\token[0.7496]| \(|\allowbreak\token[0.0041]|F|\allowbreak\token[0.0172]|\)|\allowbreak\token[0.0828]| is|\allowbreak\token[0.7080]| the|\allowbreak\token[0.9588]| intersection|\allowbreak\token[0.0295]| of|\allowbreak\token[0.7174]| \|\allowbreak\token[0.0004]|(\|\allowbreak\token[0.0003]|over|\allowbreak\token[0.0009]|line|\allowbreak\token[0.0001]|{|\allowbreak\token[1.2660]|AC|\allowbreak\token[0.0004]|}\|\allowbreak\token[0.0001]|)|\allowbreak\token[0.0710]| and|\allowbreak\token[0.0972]| \|\allowbreak\token[0.0001]|(\|\allowbreak\token[0.0002]|over|\allowbreak\token[0.0015]|line|\allowbreak\token[0.0000]|{|\allowbreak\token[0.0792]|DE|\allowbreak\token[0.0003]|}\|\allowbreak\token[1.1861]|).|\allowbreak\token[1.2395]| We|\allowbreak\token[0.8975]| need|\allowbreak\token[0.0001]| to|\allowbreak\token[0.0309]| find|\allowbreak\token[0.1944]| the|\allowbreak\token[0.2316]| coordinates|\allowbreak\token[0.0000]| of|\allowbreak\token[0.0007]| \(|\allowbreak\token[0.0001]|F|\allowbreak\token[0.0547]|\|\allowbreak\token[0.0060]|).↩↩|\tighttracebreak\allowbreak\token[0.0001]|2|\allowbreak\token[0.0000]|.|\allowbreak\token[0.0000]| **|\allowbreak\token[0.2049]|Find|\allowbreak\token[0.1817]| the|\allowbreak\token[0.3972]| Equation|\allowbreak\token[0.0000]| of|\allowbreak\token[0.1032]| Line|\allowbreak[...]
\end{twopartbox}
\begin{twopartbox}{Cut Index 25 | Acceptance Probability 0.0002}{Block 1 Step 2}{Proposal (Rejected)}
\token[2.6002]| To|\allowbreak\token[0.6725]| solve|\allowbreak\token[0.8316]| the|\allowbreak\token[0.0687]| problem|\allowbreak\token[0.5135]|,|\allowbreak\token[0.3361]| we|\allowbreak\token[0.4652]| need|\allowbreak\token[0.0016]| to|\allowbreak\token[0.7385]| determine|\allowbreak\token[0.0004]| the|\allowbreak\token[0.1338]| area|\allowbreak\token[0.0002]| of|\allowbreak\token[0.1912]| quadr|\allowbreak\token[0.0000]|ilateral|\allowbreak\token[0.5312]| \(|\allowbreak\token[0.3013]|DB|\allowbreak\token[0.0001]|EF|\allowbreak\token[0.6859]|\|\allowbreak\token[0.3149]|).|\allowbreak\token[1.6625]| Here|\allowbreak\token[1.0700]|'s|\allowbreak\token[0.6866]| the|\allowbreak\token[0.1679]| step|\allowbreak\token[0.0002]|-by|\allowbreak\token[0.0001]|-step|\allowbreak\cut{[CUT]}\allowbreak\token[1.4678]| reasoning|\allowbreak\token[0.6935]|:↩↩|\tighttracebreak\allowbreak\token[0.0038]|1|\allowbreak\token[0.0000]|.|\allowbreak\token[0.2652]| **|\allowbreak\token[0.2757]|Ident|\allowbreak\token[0.0027]|ify|\allowbreak\token[0.6453]| Coordinates|\allowbreak\token[1.0409]| of|\allowbreak\token[0.6452]| Points|\allowbreak\token[0.9793]|:|\allowbreak\token[0.0009]|**↩|\newline\allowbreak\token[0.2687]|  |\allowbreak\token[0.0035]| -|\allowbreak\token[0.9269]| \(|\allowbreak\token[0.1889]|A|\allowbreak\token[0.7277]| =|\allowbreak\token[0.0000]| (|\allowbreak\token[0.0000]|0|\allowbreak\token[0.0000]|,|\allowbreak\token[0.0081]| |\allowbreak\token[0.0001]|6|\allowbreak\token[0.0000]|)\|\allowbreak\token[0.0029]|)↩|\newline\allowbreak\token[0.0000]|  |\allowbreak\token[0.0000]| -|\allowbreak\token[0.0000]| \(|\allowbreak\token[0.0003]|B|\allowbreak\token[0.0000]| =|\allowbreak\token[0.0000]| (|\allowbreak\token[0.0000]|0|\allowbreak\token[0.0000]|,|\allowbreak\token[0.0000]| |\allowbreak\token[0.0000]|0|\allowbreak\token[0.0000]|)\|\allowbreak\token[0.0002]|)↩|\newline\allowbreak\token[0.0000]|  |\allowbreak\token[0.0000]| -|\allowbreak\token[0.0003]| \(|\allowbreak\token[0.0002]|C|\allowbreak\token[0.0000]| =|\allowbreak\token[0.0000]| (|\allowbreak\token[0.0000]|8|\allowbreak\token[0.0000]|,|\allowbreak\token[0.0000]| |\allowbreak\token[0.0000]|0|\allowbreak\token[0.0000]|)\|\allowbreak\token[0.0217]|)↩|\newline\allowbreak\token[0.0047]|  |\allowbreak\token[0.0005]| -|\allowbreak\token[0.2302]| \(|\allowbreak\token[0.0013]|D|\allowbreak\token[0.2492]|\)|\allowbreak\token[0.0301]| is|\allowbreak\token[0.0001]| the|\allowbreak\token[0.0000]| midpoint|\allowbreak\token[0.0000]| of|\allowbreak\token[0.1095]| \|\allowbreak\token[0.0001]|(\|\allowbreak\token[0.0000]|over|\allowbreak\token[0.0000]|line|\allowbreak\token[0.0000]|{|\allowbreak\token[0.0000]|AB|\allowbreak\token[0.0018]|}\|\allowbreak\token[0.0466]|),|\allowbreak\token[0.0004]| so|\allowbreak\token[0.0002]| \(|\allowbreak\token[0.0000]|D|\allowbreak\token[0.0000]| =|\allowbreak\token[0.3154]| \|\allowbreak\token[0.0000]|left|\allowbreak\token[0.6934]|(\|\allowbreak\token[0.0016]|frac|\allowbreak\token[0.0000]|{|\allowbreak\token[0.0000]|0|\allowbreak\token[0.1324]|+|\allowbreak\token[0.0000]|0|\allowbreak\token[0.0000]|}{|\allowbreak\token[0.0000]|2|\allowbreak\token[0.0000]|},|\allowbreak\token[0.0000]| \|\allowbreak\token[0.0000]|frac|\allowbreak\token[0.0000]|{|\allowbreak\token[0.1747]|6|\allowbreak\token[0.0000]|+|\allowbreak\token[0.0000]|0|\allowbreak\token[0.0001]|}{|\allowbreak\token[0.0000]|2|\allowbreak\token[0.0000]|}\|\allowbreak\token[0.0001]|right|\allowbreak\token[0.0000]|)|\allowbreak\token[0.0000]| =|\allowbreak\token[0.0000]| (|\allowbreak\token[0.0000]|0|\allowbreak\token[0.0000]|,|\allowbreak\token[0.0001]| |\allowbreak\token[0.0000]|3|\allowbreak\token[0.0000]|)\|\allowbreak\token[0.0193]|)↩|\newline\allowbreak\token[0.0000]|  |\allowbreak\token[0.0000]| -|\allowbreak\token[0.0000]| \(|\allowbreak\token[0.0000]|E|\allowbreak\token[0.0000]|\)|\allowbreak\token[0.0000]| is|\allowbreak\token[0.0000]| the|\allowbreak\token[0.0000]| midpoint|\allowbreak\token[0.0000]| of|\allowbreak\token[0.0000]| \|\allowbreak\token[0.0000]|(\|\allowbreak\token[0.0000]|over|\allowbreak\token[0.0000]|line|\allowbreak\token[0.0000]|{|\allowbreak\token[0.0001]|BC|\allowbreak\token[0.0000]|}\|\allowbreak\token[0.0000]|),|\allowbreak\token[0.0000]| so|\allowbreak\token[0.0000]| \(|\allowbreak\token[0.0001]|E|\allowbreak\token[0.0000]| =|\allowbreak\token[0.0000]| \|\allowbreak\token[0.0000]|left|\allowbreak\token[0.0000]|(\|\allowbreak\token[0.0000]|frac|\allowbreak\token[0.0000]|{|\allowbreak\token[0.0173]|0|\allowbreak\token[0.0000]|+|\allowbreak\token[0.0000]|8|\allowbreak\token[0.0000]|}{|\allowbreak\token[0.0000]|2|\allowbreak\token[0.0000]|},|\allowbreak\token[0.0000]| \|\allowbreak\token[0.0001]|frac|\allowbreak\token[0.0000]|{|\allowbreak\token[0.0000]|0|\allowbreak\token[0.0000]|+|\allowbreak\token[0.0000]|0|\allowbreak\token[0.0000]|}{|\allowbreak\token[0.0001]|2|\allowbreak\token[0.0000]|}\|\allowbreak\token[0.0001]|right|\allowbreak\token[0.0000]|)|\allowbreak\token[0.0000]| =|\allowbreak\token[0.0001]| (|\allowbreak[...]
\tcblower
\token[1.4676]| approach|\allowbreak\token[0.1317]|:↩↩|\tighttracebreak\allowbreak\token[0.0015]|1|\allowbreak\token[0.0000]|.|\allowbreak\token[0.4088]| **|\allowbreak\token[0.3857]|Ident|\allowbreak\token[0.0010]|ify|\allowbreak\token[0.6352]| Coordinates|\allowbreak\token[1.0109]| of|\allowbreak\token[0.7825]| Points|\allowbreak\token[1.0858]|:|\allowbreak\token[0.0014]|**↩|\newline\allowbreak\token[0.2691]|  |\allowbreak\token[0.0177]| -|\allowbreak\token[1.1107]| Point|\allowbreak\token[0.0003]| \(|\allowbreak\token[0.5023]|A|\allowbreak\token[0.3106]|\)|\allowbreak\token[0.1604]| is|\allowbreak\token[0.6295]| \|\allowbreak\token[0.0001]|((|\allowbreak\token[0.0000]|0|\allowbreak\token[0.0000]|,|\allowbreak\token[0.0901]| |\allowbreak\token[0.0000]|6|\allowbreak\token[0.0000]|)\|\allowbreak\token[0.3752]|).↩|\newline\allowbreak\token[0.0000]|  |\allowbreak\token[0.0000]| -|\allowbreak\token[0.0001]| Point|\allowbreak\token[0.0000]| \(|\allowbreak\token[0.0005]|B|\allowbreak\token[0.0000]|\)|\allowbreak\token[0.0000]| is|\allowbreak\token[0.0001]| \|\allowbreak\token[0.0000]|((|\allowbreak\token[0.0000]|0|\allowbreak\token[0.0000]|,|\allowbreak\token[0.0001]| |\allowbreak\token[0.0000]|0|\allowbreak\token[0.0000]|)\|\allowbreak\token[0.0002]|).↩|\newline\allowbreak\token[0.0000]|  |\allowbreak\token[0.0000]| -|\allowbreak\token[0.0000]| Point|\allowbreak\token[0.0000]| \(|\allowbreak\token[0.0001]|C|\allowbreak\token[0.0000]|\)|\allowbreak\token[0.0000]| is|\allowbreak\token[0.0000]| \|\allowbreak\token[0.0001]|((|\allowbreak\token[0.0000]|8|\allowbreak\token[0.0000]|,|\allowbreak\token[0.0000]| |\allowbreak\token[0.0000]|0|\allowbreak\token[0.0000]|)\|\allowbreak\token[0.0820]|).↩|\newline\allowbreak\token[0.0053]|  |\allowbreak\token[0.0000]| -|\allowbreak\token[0.8124]| Point|\allowbreak\token[0.0000]| \(|\allowbreak\token[0.0022]|D|\allowbreak\token[0.4194]|\)|\allowbreak\token[0.0243]| is|\allowbreak\token[0.0001]| the|\allowbreak\token[0.0000]| midpoint|\allowbreak\token[0.0000]| of|\allowbreak\token[0.0497]| \|\allowbreak\token[0.0004]|(\|\allowbreak\token[0.0000]|over|\allowbreak\token[0.0000]|line|\allowbreak\token[0.0000]|{|\allowbreak\token[0.0000]|AB|\allowbreak\token[0.0001]|}\|\allowbreak\token[0.0134]|),|\allowbreak\token[0.0028]| so|\allowbreak\token[0.4452]| \(|\allowbreak\token[0.0000]|D|\allowbreak\token[0.6292]|\)|\allowbreak\token[0.0038]| is|\allowbreak\token[0.0002]| \|\allowbreak\token[0.0997]|((|\allowbreak\token[0.0101]|0|\allowbreak\token[0.0000]|,|\allowbreak\token[0.0240]| |\allowbreak\token[0.0000]|3|\allowbreak\token[0.0000]|)\|\allowbreak\token[0.0611]|).↩|\newline\allowbreak\token[0.0000]|  |\allowbreak\token[0.0000]| -|\allowbreak\token[0.0000]| Point|\allowbreak\token[0.0000]| \(|\allowbreak\token[0.0002]|E|\allowbreak\token[0.0000]|\)|\allowbreak\token[0.0000]| is|\allowbreak\token[0.0000]| the|\allowbreak\token[0.0000]| midpoint|\allowbreak\token[0.0000]| of|\allowbreak\token[0.0000]| \|\allowbreak\token[0.0000]|(\|\allowbreak\token[0.0002]|over|\allowbreak\token[0.0001]|line|\allowbreak\token[0.0000]|{|\allowbreak\token[0.0002]|BC|\allowbreak\token[0.0000]|}\|\allowbreak\token[0.0000]|),|\allowbreak\token[0.0000]| so|\allowbreak\token[0.0000]| \(|\allowbreak\token[0.0001]|E|\allowbreak\token[0.0002]|\)|\allowbreak\token[0.0000]| is|\allowbreak\token[0.0000]| \|\allowbreak\token[0.0007]|((|\allowbreak\token[0.0000]|4|\allowbreak\token[0.0000]|,|\allowbreak\token[0.0000]| |\allowbreak\token[0.0001]|0|\allowbreak\token[0.0000]|)\|\allowbreak\token[0.3667]|).↩|\newline\allowbreak\token[0.0175]|  |\allowbreak\token[0.0000]| -|\allowbreak\token[0.1882]| Point|\allowbreak\token[0.0000]| \(|\allowbreak\token[0.0007]|F|\allowbreak\token[0.0040]|\)|\allowbreak\token[0.1905]| is|\allowbreak\token[1.0495]| the|\allowbreak\token[0.5442]| intersection|\allowbreak\token[0.0027]| of|\allowbreak\token[0.8904]| \|\allowbreak\token[0.0016]|(\|\allowbreak\token[0.0001]|over|\allowbreak\token[0.0005]|line|\allowbreak\token[0.0001]|{|\allowbreak\token[1.0885]|AC|\allowbreak\token[0.0001]|}\|\allowbreak\token[0.0001]|)|\allowbreak\token[0.0798]| and|\allowbreak\token[0.1677]| \|\allowbreak\token[0.0000]|(\|\allowbreak\token[0.0001]|over|\allowbreak\token[0.0007]|line|\allowbreak\token[0.0000]|{|\allowbreak\token[0.0540]|DE|\allowbreak[...]
\end{twopartbox}
\begin{twopartbox}{Cut Index 8 | Acceptance Probability 0.0000}{Block 1 Step 3}{Proposal (Rejected)}
\token[2.6002]| To|\allowbreak\token[0.6725]| solve|\allowbreak\token[0.8316]| the|\allowbreak\token[0.0687]| problem|\allowbreak\token[0.5135]|,|\allowbreak\token[0.3361]| we|\allowbreak\token[0.4652]| need|\allowbreak\token[0.0016]| to|\allowbreak\cut{[CUT]}\allowbreak\token[0.7385]| determine|\allowbreak\token[0.0004]| the|\allowbreak\token[0.1338]| area|\allowbreak\token[0.0002]| of|\allowbreak\token[0.1912]| quadr|\allowbreak\token[0.0000]|ilateral|\allowbreak\token[0.5312]| \(|\allowbreak\token[0.3013]|DB|\allowbreak\token[0.0001]|EF|\allowbreak\token[0.6859]|\|\allowbreak\token[0.3149]|).|\allowbreak\token[1.6625]| Here|\allowbreak\token[1.0700]|'s|\allowbreak\token[0.6866]| the|\allowbreak\token[0.1679]| step|\allowbreak\token[0.0002]|-by|\allowbreak\token[0.0001]|-step|\allowbreak\token[1.4678]| reasoning|\allowbreak\token[0.6935]|:↩↩|\tighttracebreak\allowbreak\token[0.0038]|1|\allowbreak\token[0.0000]|.|\allowbreak\token[0.2652]| **|\allowbreak\token[0.2757]|Ident|\allowbreak\token[0.0027]|ify|\allowbreak\token[0.6453]| Coordinates|\allowbreak\token[1.0409]| of|\allowbreak\token[0.6452]| Points|\allowbreak\token[0.9793]|:|\allowbreak\token[0.0009]|**↩|\newline\allowbreak\token[0.2687]|  |\allowbreak\token[0.0035]| -|\allowbreak\token[0.9269]| \(|\allowbreak\token[0.1889]|A|\allowbreak\token[0.7277]| =|\allowbreak\token[0.0000]| (|\allowbreak\token[0.0000]|0|\allowbreak\token[0.0000]|,|\allowbreak\token[0.0081]| |\allowbreak\token[0.0001]|6|\allowbreak\token[0.0000]|)\|\allowbreak\token[0.0029]|)↩|\newline\allowbreak\token[0.0000]|  |\allowbreak\token[0.0000]| -|\allowbreak\token[0.0000]| \(|\allowbreak\token[0.0003]|B|\allowbreak\token[0.0000]| =|\allowbreak\token[0.0000]| (|\allowbreak\token[0.0000]|0|\allowbreak\token[0.0000]|,|\allowbreak\token[0.0000]| |\allowbreak\token[0.0000]|0|\allowbreak\token[0.0000]|)\|\allowbreak\token[0.0002]|)↩|\newline\allowbreak\token[0.0000]|  |\allowbreak\token[0.0000]| -|\allowbreak\token[0.0003]| \(|\allowbreak\token[0.0002]|C|\allowbreak\token[0.0000]| =|\allowbreak\token[0.0000]| (|\allowbreak\token[0.0000]|8|\allowbreak\token[0.0000]|,|\allowbreak\token[0.0000]| |\allowbreak\token[0.0000]|0|\allowbreak\token[0.0000]|)\|\allowbreak\token[0.0217]|)↩|\newline\allowbreak\token[0.0047]|  |\allowbreak\token[0.0005]| -|\allowbreak\token[0.2302]| \(|\allowbreak\token[0.0013]|D|\allowbreak\token[0.2492]|\)|\allowbreak\token[0.0301]| is|\allowbreak\token[0.0001]| the|\allowbreak\token[0.0000]| midpoint|\allowbreak\token[0.0000]| of|\allowbreak\token[0.1095]| \|\allowbreak\token[0.0001]|(\|\allowbreak\token[0.0000]|over|\allowbreak\token[0.0000]|line|\allowbreak\token[0.0000]|{|\allowbreak\token[0.0000]|AB|\allowbreak\token[0.0018]|}\|\allowbreak\token[0.0466]|),|\allowbreak\token[0.0004]| so|\allowbreak\token[0.0002]| \(|\allowbreak\token[0.0000]|D|\allowbreak\token[0.0000]| =|\allowbreak\token[0.3154]| \|\allowbreak\token[0.0000]|left|\allowbreak\token[0.6934]|(\|\allowbreak\token[0.0016]|frac|\allowbreak\token[0.0000]|{|\allowbreak\token[0.0000]|0|\allowbreak\token[0.1324]|+|\allowbreak\token[0.0000]|0|\allowbreak\token[0.0000]|}{|\allowbreak\token[0.0000]|2|\allowbreak\token[0.0000]|},|\allowbreak\token[0.0000]| \|\allowbreak\token[0.0000]|frac|\allowbreak\token[0.0000]|{|\allowbreak\token[0.1747]|6|\allowbreak\token[0.0000]|+|\allowbreak\token[0.0000]|0|\allowbreak\token[0.0001]|}{|\allowbreak\token[0.0000]|2|\allowbreak\token[0.0000]|}\|\allowbreak\token[0.0001]|right|\allowbreak\token[0.0000]|)|\allowbreak\token[0.0000]| =|\allowbreak\token[0.0000]| (|\allowbreak\token[0.0000]|0|\allowbreak\token[0.0000]|,|\allowbreak\token[0.0001]| |\allowbreak\token[0.0000]|3|\allowbreak\token[0.0000]|)\|\allowbreak\token[0.0193]|)↩|\newline\allowbreak\token[0.0000]|  |\allowbreak\token[0.0000]| -|\allowbreak\token[0.0000]| \(|\allowbreak\token[0.0000]|E|\allowbreak\token[0.0000]|\)|\allowbreak\token[0.0000]| is|\allowbreak\token[0.0000]| the|\allowbreak\token[0.0000]| midpoint|\allowbreak\token[0.0000]| of|\allowbreak\token[0.0000]| \|\allowbreak\token[0.0000]|(\|\allowbreak\token[0.0000]|over|\allowbreak\token[0.0000]|line|\allowbreak\token[0.0000]|{|\allowbreak\token[0.0001]|BC|\allowbreak\token[0.0000]|}\|\allowbreak\token[0.0000]|),|\allowbreak\token[0.0000]| so|\allowbreak\token[0.0000]| \(|\allowbreak\token[0.0001]|E|\allowbreak\token[0.0000]| =|\allowbreak\token[0.0000]| \|\allowbreak\token[0.0000]|left|\allowbreak\token[0.0000]|(\|\allowbreak\token[0.0000]|frac|\allowbreak\token[0.0000]|{|\allowbreak\token[0.0173]|0|\allowbreak\token[0.0000]|+|\allowbreak[...]
\tcblower
\token[0.8216]| determine|\allowbreak\token[0.0003]| the|\allowbreak\token[0.1339]| area|\allowbreak\token[0.0002]| of|\allowbreak\token[0.1912]| quadr|\allowbreak\token[0.0000]|ilateral|\allowbreak\token[0.5314]| \(|\allowbreak\token[0.3280]|DB|\allowbreak\token[0.0001]|EF|\allowbreak\token[0.6859]|\)|\allowbreak\token[1.2481]| where|\allowbreak\token[0.1211]| \(|\allowbreak\token[0.0035]|D|\allowbreak\token[0.0005]|\)|\allowbreak\token[0.0605]| and|\allowbreak\token[0.0000]| \(|\allowbreak\token[0.0002]|E|\allowbreak\token[0.0001]|\)|\allowbreak\token[0.0000]| are|\allowbreak\token[0.3399]| the|\allowbreak\token[0.0001]| mid|\allowbreak\token[0.0000]|points|\allowbreak\token[0.0000]| of|\allowbreak\token[0.0398]| \|\allowbreak\token[0.0000]|(\|\allowbreak\token[0.0000]|over|\allowbreak\token[0.0000]|line|\allowbreak\token[0.0000]|{|\allowbreak\token[0.0002]|AB|\allowbreak\token[0.0001]|}\|\allowbreak\token[0.0001]|)|\allowbreak\token[0.0001]| and|\allowbreak\token[0.0000]| \|\allowbreak\token[0.0000]|(\|\allowbreak\token[0.0014]|over|\allowbreak\token[0.0000]|line|\allowbreak\token[0.0000]|{|\allowbreak\token[0.0016]|BC|\allowbreak\token[0.0002]|}\|\allowbreak\token[0.1596]|)|\allowbreak\token[0.0004]| respectively|\allowbreak\token[1.1692]| in|\allowbreak\token[0.9605]| the|\allowbreak\token[0.8416]| given|\allowbreak\token[1.1970]| coordinate|\allowbreak\token[1.0566]| geometry|\allowbreak\token[1.5746]| problem|\allowbreak\token[0.5349]|.↩↩|\tighttracebreak\allowbreak\token[1.8915]|Here|\allowbreak\token[0.8714]| are|\allowbreak\token[0.0075]| the|\allowbreak\token[0.1596]| steps|\allowbreak\token[1.0410]| to|\allowbreak\token[0.2434]| solve|\allowbreak\token[0.5449]| the|\allowbreak\token[0.0009]| problem|\allowbreak\token[0.0512]|:↩↩|\tighttracebreak\allowbreak\token[0.0004]|1|\allowbreak\token[0.0010]|.|\allowbreak\token[0.9430]| **|\allowbreak\token[0.6621]|Ident|\allowbreak\token[0.0016]|ify|\allowbreak\token[0.8151]| the|\allowbreak\token[0.5914]| coordinates|\allowbreak\token[0.0250]| of|\allowbreak\token[0.3622]| points|\allowbreak\token[0.0893]| \(|\allowbreak\token[0.4501]|D|\allowbreak\token[0.0007]|\)|\allowbreak\token[0.0004]| and|\allowbreak\token[0.0001]| \(|\allowbreak\token[0.0003]|E|\allowbreak\token[0.0063]|\|\allowbreak\token[0.5072]|)**|\allowbreak\token[0.0542]|:↩|\newline\allowbreak\token[0.5318]|  |\allowbreak\token[0.0058]| -|\allowbreak\token[0.7786]| \(|\allowbreak\token[0.2749]|D|\allowbreak\token[0.0066]|\)|\allowbreak\token[0.0014]| is|\allowbreak\token[0.0007]| the|\allowbreak\token[0.0000]| midpoint|\allowbreak\token[0.0001]| of|\allowbreak\token[0.0365]| \|\allowbreak\token[0.0002]|(\|\allowbreak\token[0.0000]|over|\allowbreak\token[0.0001]|line|\allowbreak\token[0.0001]|{|\allowbreak\token[0.0002]|AB|\allowbreak\token[0.0005]|}\|\allowbreak\token[1.5622]|),|\allowbreak\token[0.8069]| so|\allowbreak\token[0.7673]| its|\allowbreak\token[0.0003]| coordinates|\allowbreak\token[0.1417]| are|\allowbreak\token[0.5283]| \|\allowbreak\token[0.6362]|(\|\allowbreak\token[0.0002]|left|\allowbreak\token[0.6860]|(\|\allowbreak\token[0.0073]|frac|\allowbreak\token[0.0615]|{|\allowbreak\token[0.0117]|0|\allowbreak\token[0.1456]|+|\allowbreak\token[0.0002]|0|\allowbreak\token[0.0001]|}{|\allowbreak\token[0.0000]|2|\allowbreak\token[0.0001]|},|\allowbreak\token[0.0024]| \|\allowbreak\token[0.0000]|frac|\allowbreak\token[0.0000]|{|\allowbreak\token[0.0403]|6|\allowbreak\token[0.0004]|+|\allowbreak\token[0.0000]|0|\allowbreak\token[0.0002]|}{|\allowbreak\token[0.0001]|2|\allowbreak\token[0.0005]|}\|\allowbreak\token[0.0001]|right|\allowbreak\token[0.0025]|)|\allowbreak\token[0.0003]| =|\allowbreak\token[0.0012]| (|\allowbreak\token[0.0000]|0|\allowbreak\token[0.0000]|,|\allowbreak\token[0.1745]| |\allowbreak\token[0.0004]|3|\allowbreak\token[0.0001]|)\|\allowbreak\token[0.0055]|).↩|\newline\allowbreak\token[0.0000]|  |\allowbreak\token[0.0000]| -|\allowbreak\token[0.0000]| \(|\allowbreak\token[0.0002]|E|\allowbreak\token[0.0003]|\)|\allowbreak\token[0.0002]| is|\allowbreak\token[0.0001]| the|\allowbreak\token[0.0001]| midpoint|\allowbreak\token[0.0000]| of|\allowbreak\token[0.0000]| \|\allowbreak\token[0.0001]|(\|\allowbreak\token[0.0005]|over|\allowbreak\token[0.0000]|line|\allowbreak\token[0.0000]|{|\allowbreak\token[0.0004]|BC|\allowbreak\token[0.0001]|}\|\allowbreak\token[0.0001]|),|\allowbreak[...]
\end{twopartbox}
\begin{twopartbox}{Cut Index 25 | Acceptance Probability 1.0000}{Block 2 Step 1}{Proposal (Accepted)}
\token[2.6002]| To|\allowbreak\token[0.6725]| solve|\allowbreak\token[0.8316]| the|\allowbreak\token[0.0687]| problem|\allowbreak\token[0.5135]|,|\allowbreak\token[0.3361]| we|\allowbreak\token[0.4652]| need|\allowbreak\token[0.0016]| to|\allowbreak\token[0.7385]| determine|\allowbreak\token[0.0004]| the|\allowbreak\token[0.1338]| area|\allowbreak\token[0.0002]| of|\allowbreak\token[0.1912]| quadr|\allowbreak\token[0.0000]|ilateral|\allowbreak\token[0.5312]| \(|\allowbreak\token[0.3013]|DB|\allowbreak\token[0.0001]|EF|\allowbreak\token[0.6859]|\|\allowbreak\token[0.3149]|).|\allowbreak\token[1.6625]| Here|\allowbreak\token[1.0700]|'s|\allowbreak\token[0.6866]| the|\allowbreak\token[0.1679]| step|\allowbreak\token[0.0002]|-by|\allowbreak\token[0.0001]|-step|\allowbreak\cut{[CUT]}\allowbreak\token[1.4678]| reasoning|\allowbreak\token[0.6935]|:↩↩|\tighttracebreak\allowbreak\token[0.0038]|1|\allowbreak\token[0.0000]|.|\allowbreak\token[0.2652]| **|\allowbreak\token[0.2757]|Ident|\allowbreak\token[0.0027]|ify|\allowbreak\token[0.6453]| Coordinates|\allowbreak\token[1.0409]| of|\allowbreak\token[0.6452]| Points|\allowbreak\token[0.9793]|:|\allowbreak\token[0.0009]|**↩|\newline\allowbreak\token[0.2687]|  |\allowbreak\token[0.0035]| -|\allowbreak\token[0.9269]| \(|\allowbreak\token[0.1889]|A|\allowbreak\token[0.7277]| =|\allowbreak\token[0.0000]| (|\allowbreak\token[0.0000]|0|\allowbreak\token[0.0000]|,|\allowbreak\token[0.0081]| |\allowbreak\token[0.0001]|6|\allowbreak\token[0.0000]|)\|\allowbreak\token[0.0029]|)↩|\newline\allowbreak\token[0.0000]|  |\allowbreak\token[0.0000]| -|\allowbreak\token[0.0000]| \(|\allowbreak\token[0.0003]|B|\allowbreak\token[0.0000]| =|\allowbreak\token[0.0000]| (|\allowbreak\token[0.0000]|0|\allowbreak\token[0.0000]|,|\allowbreak\token[0.0000]| |\allowbreak\token[0.0000]|0|\allowbreak\token[0.0000]|)\|\allowbreak\token[0.0002]|)↩|\newline\allowbreak\token[0.0000]|  |\allowbreak\token[0.0000]| -|\allowbreak\token[0.0003]| \(|\allowbreak\token[0.0002]|C|\allowbreak\token[0.0000]| =|\allowbreak\token[0.0000]| (|\allowbreak\token[0.0000]|8|\allowbreak\token[0.0000]|,|\allowbreak\token[0.0000]| |\allowbreak\token[0.0000]|0|\allowbreak\token[0.0000]|)\|\allowbreak\token[0.0217]|)↩|\newline\allowbreak\token[0.0047]|  |\allowbreak\token[0.0005]| -|\allowbreak\token[0.2302]| \(|\allowbreak\token[0.0013]|D|\allowbreak\token[0.2492]|\)|\allowbreak\token[0.0301]| is|\allowbreak\token[0.0001]| the|\allowbreak\token[0.0000]| midpoint|\allowbreak\token[0.0000]| of|\allowbreak\token[0.1095]| \|\allowbreak\token[0.0001]|(\|\allowbreak\token[0.0000]|over|\allowbreak\token[0.0000]|line|\allowbreak\token[0.0000]|{|\allowbreak\token[0.0000]|AB|\allowbreak\token[0.0018]|}\|\allowbreak\token[0.0466]|),|\allowbreak\token[0.0004]| so|\allowbreak\token[0.0002]| \(|\allowbreak\token[0.0000]|D|\allowbreak\token[0.0000]| =|\allowbreak\token[0.3154]| \|\allowbreak\token[0.0000]|left|\allowbreak\token[0.6934]|(\|\allowbreak\token[0.0016]|frac|\allowbreak\token[0.0000]|{|\allowbreak\token[0.0000]|0|\allowbreak\token[0.1324]|+|\allowbreak\token[0.0000]|0|\allowbreak\token[0.0000]|}{|\allowbreak\token[0.0000]|2|\allowbreak\token[0.0000]|},|\allowbreak\token[0.0000]| \|\allowbreak\token[0.0000]|frac|\allowbreak\token[0.0000]|{|\allowbreak\token[0.1747]|6|\allowbreak\token[0.0000]|+|\allowbreak\token[0.0000]|0|\allowbreak\token[0.0001]|}{|\allowbreak\token[0.0000]|2|\allowbreak\token[0.0000]|}\|\allowbreak\token[0.0001]|right|\allowbreak\token[0.0000]|)|\allowbreak\token[0.0000]| =|\allowbreak\token[0.0000]| (|\allowbreak\token[0.0000]|0|\allowbreak\token[0.0000]|,|\allowbreak\token[0.0001]| |\allowbreak\token[0.0000]|3|\allowbreak\token[0.0000]|)\|\allowbreak\token[0.0193]|)↩|\newline\allowbreak\token[0.0000]|  |\allowbreak\token[0.0000]| -|\allowbreak\token[0.0000]| \(|\allowbreak\token[0.0000]|E|\allowbreak\token[0.0000]|\)|\allowbreak\token[0.0000]| is|\allowbreak\token[0.0000]| the|\allowbreak\token[0.0000]| midpoint|\allowbreak\token[0.0000]| of|\allowbreak\token[0.0000]| \|\allowbreak\token[0.0000]|(\|\allowbreak\token[0.0000]|over|\allowbreak\token[0.0000]|line|\allowbreak\token[0.0000]|{|\allowbreak\token[0.0001]|BC|\allowbreak\token[0.0000]|}\|\allowbreak\token[0.0000]|),|\allowbreak\token[0.0000]| so|\allowbreak\token[0.0000]| \(|\allowbreak\token[0.0001]|E|\allowbreak\token[0.0000]| =|\allowbreak\token[0.0000]| \|\allowbreak\token[0.0000]|left|\allowbreak\token[0.0000]|(\|\allowbreak\token[0.0000]|frac|\allowbreak\token[0.0000]|{|\allowbreak\token[0.0173]|0|\allowbreak\token[0.0000]|+|\allowbreak\token[0.0000]|8|\allowbreak\token[0.0000]|}{|\allowbreak\token[0.0000]|2|\allowbreak\token[0.0000]|},|\allowbreak\token[0.0000]| \|\allowbreak\token[0.0001]|frac|\allowbreak\token[0.0000]|{|\allowbreak\token[0.0000]|0|\allowbreak\token[0.0000]|+|\allowbreak\token[0.0000]|0|\allowbreak\token[0.0000]|}{|\allowbreak\token[0.0001]|2|\allowbreak\token[0.0000]|}\|\allowbreak\token[0.0001]|right|\allowbreak\token[0.0000]|)|\allowbreak\token[0.0000]| =|\allowbreak\token[0.0001]| (|\allowbreak[...]
\tcblower
\token[1.4676]| approach|\allowbreak\token[0.1317]|:↩↩|\tighttracebreak\allowbreak\token[0.0016]|1|\allowbreak\token[0.0000]|.|\allowbreak\token[0.4060]| **|\allowbreak\token[0.4029]|Ident|\allowbreak\token[0.0010]|ify|\allowbreak\token[0.6278]| Coordinates|\allowbreak\token[1.0343]| of|\allowbreak\token[0.7862]| Points|\allowbreak\token[1.0859]|:|\allowbreak\token[0.0014]|**↩|\newline\allowbreak\token[0.2277]|  |\allowbreak\token[0.0177]| -|\allowbreak\token[1.1397]| Point|\allowbreak\token[0.0003]| \(|\allowbreak\token[0.4684]|A|\allowbreak\token[0.2856]|\)|\allowbreak\token[0.1606]| is|\allowbreak\token[0.6652]| \|\allowbreak\token[0.0001]|((|\allowbreak\token[0.0000]|0|\allowbreak\token[0.0000]|,|\allowbreak\token[0.0901]| |\allowbreak\token[0.0000]|6|\allowbreak\token[0.0000]|)\|\allowbreak\token[0.3752]|).↩|\newline\allowbreak\token[0.0000]|  |\allowbreak\token[0.0000]| -|\allowbreak\token[0.0001]| Point|\allowbreak\token[0.0000]| \(|\allowbreak\token[0.0005]|B|\allowbreak\token[0.0000]|\)|\allowbreak\token[0.0000]| is|\allowbreak\token[0.0001]| \|\allowbreak\token[0.0000]|((|\allowbreak\token[0.0000]|0|\allowbreak\token[0.0000]|,|\allowbreak\token[0.0001]| |\allowbreak\token[0.0000]|0|\allowbreak\token[0.0000]|)\|\allowbreak\token[0.0002]|).↩|\newline\allowbreak\token[0.0000]|  |\allowbreak\token[0.0000]| -|\allowbreak\token[0.0000]| Point|\allowbreak\token[0.0000]| \(|\allowbreak\token[0.0001]|C|\allowbreak\token[0.0000]|\)|\allowbreak\token[0.0000]| is|\allowbreak\token[0.0000]| \|\allowbreak\token[0.0001]|((|\allowbreak\token[0.0000]|8|\allowbreak\token[0.0000]|,|\allowbreak\token[0.0000]| |\allowbreak\token[0.0000]|0|\allowbreak\token[0.0000]|)\|\allowbreak\token[0.0820]|).↩|\newline\allowbreak\token[0.0053]|  |\allowbreak\token[0.0000]| -|\allowbreak\token[0.7278]| Point|\allowbreak\token[0.0000]| \(|\allowbreak\token[0.0022]|D|\allowbreak\token[0.4194]|\)|\allowbreak\token[0.0197]| is|\allowbreak\token[0.0001]| the|\allowbreak\token[0.0000]| midpoint|\allowbreak\token[0.0000]| of|\allowbreak\token[0.0496]| \|\allowbreak\token[0.0004]|(\|\allowbreak\token[0.0000]|over|\allowbreak\token[0.0000]|line|\allowbreak\token[0.0000]|{|\allowbreak\token[0.0000]|AB|\allowbreak\token[0.0001]|}\|\allowbreak\token[0.0147]|),|\allowbreak\token[0.0025]| so|\allowbreak\token[0.4478]| \(|\allowbreak\token[0.0000]|D|\allowbreak\token[0.6075]|\)|\allowbreak\token[0.0034]| is|\allowbreak\token[0.0002]| \|\allowbreak\token[0.0905]|((|\allowbreak\token[0.0102]|0|\allowbreak\token[0.0000]|,|\allowbreak\token[0.0240]| |\allowbreak\token[0.0000]|3|\allowbreak\token[0.0000]|)\|\allowbreak\token[0.0612]|).↩|\newline\allowbreak\token[0.0000]|  |\allowbreak\token[0.0000]| -|\allowbreak\token[0.0000]| Point|\allowbreak\token[0.0000]| \(|\allowbreak\token[0.0002]|E|\allowbreak\token[0.0001]|\)|\allowbreak\token[0.0000]| is|\allowbreak\token[0.0000]| the|\allowbreak\token[0.0000]| midpoint|\allowbreak\token[0.0000]| of|\allowbreak\token[0.0000]| \|\allowbreak\token[0.0000]|(\|\allowbreak\token[0.0002]|over|\allowbreak\token[0.0001]|line|\allowbreak\token[0.0000]|{|\allowbreak\token[0.0002]|BC|\allowbreak\token[0.0000]|}\|\allowbreak\token[0.0000]|),|\allowbreak\token[0.0000]| so|\allowbreak\token[0.0000]| \(|\allowbreak\token[0.0001]|E|\allowbreak\token[0.0002]|\)|\allowbreak\token[0.0000]| is|\allowbreak\token[0.0000]| \|\allowbreak\token[0.0007]|((|\allowbreak\token[0.0000]|4|\allowbreak\token[0.0000]|,|\allowbreak\token[0.0000]| |\allowbreak\token[0.0001]|0|\allowbreak\token[0.0000]|)\|\allowbreak\token[0.3933]|).↩|\newline\allowbreak\token[0.0175]|  |\allowbreak\token[0.0000]| -|\allowbreak\token[0.1774]| Point|\allowbreak\token[0.0000]| \(|\allowbreak\token[0.0007]|F|\allowbreak\token[0.0040]|\)|\allowbreak\token[0.2097]| is|\allowbreak\token[1.1429]| the|\allowbreak\token[0.5879]| intersection|\allowbreak\token[0.0027]| of|\allowbreak\token[0.9969]| \|\allowbreak\token[0.0018]|(\|\allowbreak\token[0.0001]|over|\allowbreak\token[0.0005]|line|\allowbreak\token[0.0001]|{|\allowbreak\token[1.0711]|AC|\allowbreak\token[0.0001]|}\|\allowbreak\token[0.0001]|)|\allowbreak\token[0.0799]| and|\allowbreak\token[0.2007]| \|\allowbreak\token[0.0001]|(\|\allowbreak\token[0.0001]|over|\allowbreak\token[0.0006]|line|\allowbreak\token[0.0000]|{|\allowbreak\token[0.0553]|DE|\allowbreak[...]
\end{twopartbox}
\begin{twopartbox}{Cut Index 233 | Acceptance Probability 0.0009}{Block 2 Step 2}{Proposal (Rejected)}
[...]\allowbreak\token[0.0001]| the|\allowbreak\token[0.0000]| midpoint|\allowbreak\token[0.0000]| of|\allowbreak\token[0.0496]| \|\allowbreak\token[0.0004]|(\|\allowbreak\token[0.0000]|over|\allowbreak\token[0.0000]|line|\allowbreak\token[0.0000]|{|\allowbreak\token[0.0000]|AB|\allowbreak\token[0.0001]|}\|\allowbreak\token[0.0147]|),|\allowbreak\token[0.0025]| so|\allowbreak\token[0.4478]| \(|\allowbreak\token[0.0000]|D|\allowbreak\token[0.6075]|\)|\allowbreak\token[0.0034]| is|\allowbreak\token[0.0002]| \|\allowbreak\token[0.0905]|((|\allowbreak\token[0.0102]|0|\allowbreak\token[0.0000]|,|\allowbreak\token[0.0240]| |\allowbreak\token[0.0000]|3|\allowbreak\token[0.0000]|)\|\allowbreak\token[0.0612]|).↩|\newline\allowbreak\token[0.0000]|  |\allowbreak\token[0.0000]| -|\allowbreak\token[0.0000]| Point|\allowbreak\token[0.0000]| \(|\allowbreak\token[0.0002]|E|\allowbreak\token[0.0001]|\)|\allowbreak\token[0.0000]| is|\allowbreak\token[0.0000]| the|\allowbreak\token[0.0000]| midpoint|\allowbreak\token[0.0000]| of|\allowbreak\token[0.0000]| \|\allowbreak\token[0.0000]|(\|\allowbreak\token[0.0002]|over|\allowbreak\token[0.0001]|line|\allowbreak\token[0.0000]|{|\allowbreak\token[0.0002]|BC|\allowbreak\token[0.0000]|}\|\allowbreak\token[0.0000]|),|\allowbreak\token[0.0000]| so|\allowbreak\token[0.0000]| \(|\allowbreak\token[0.0001]|E|\allowbreak\token[0.0002]|\)|\allowbreak\token[0.0000]| is|\allowbreak\token[0.0000]| \|\allowbreak\token[0.0007]|((|\allowbreak\token[0.0000]|4|\allowbreak\token[0.0000]|,|\allowbreak\token[0.0000]| |\allowbreak\token[0.0001]|0|\allowbreak\token[0.0000]|)\|\allowbreak\token[0.3933]|).↩|\newline\allowbreak\token[0.0175]|  |\allowbreak\token[0.0000]| -|\allowbreak\token[0.1774]| Point|\allowbreak\token[0.0000]| \(|\allowbreak\token[0.0007]|F|\allowbreak\token[0.0040]|\)|\allowbreak\token[0.2097]| is|\allowbreak\token[1.1429]| the|\allowbreak\token[0.5879]| intersection|\allowbreak\token[0.0027]| of|\allowbreak\token[0.9969]| \|\allowbreak\token[0.0018]|(\|\allowbreak\token[0.0001]|over|\allowbreak\token[0.0005]|line|\allowbreak\token[0.0001]|{|\allowbreak\token[1.0711]|AC|\allowbreak\token[0.0001]|}\|\allowbreak\token[0.0001]|)|\allowbreak\token[0.0799]| and|\allowbreak\token[0.2007]| \|\allowbreak\token[0.0001]|(\|\allowbreak\token[0.0001]|over|\allowbreak\token[0.0006]|line|\allowbreak\token[0.0000]|{|\allowbreak\token[0.0553]|DE|\allowbreak\token[0.0002]|}\|\allowbreak\token[0.7056]|).↩↩|\tighttracebreak\allowbreak\token[0.0005]|2|\allowbreak\token[0.0000]|.|\allowbreak\token[0.0000]| **|\allowbreak\token[0.3037]|Find|\allowbreak\token[0.2515]| the|\allowbreak\token[0.4118]| Equation|\allowbreak\token[0.0000]| of|\allowbreak\token[0.0452]| Line|\allowbreak\token[0.7027]| \|\allowbreak\token[0.0001]|(\|\allowbreak\token[0.0001]|over|\allowbreak\token[0.0003]|line|\allowbreak\token[0.0000]|{|\allowbreak\token[0.2697]|AC|\allowbreak\token[0.0000]|}\|\allowbreak\token[0.0023]|):|\allowbreak\token[0.0001]|**↩|\newline\allowbreak\token[0.0000]|  |\allowbreak\token[0.6873]| -|\allowbreak\token[0.4607]| The|\allowbreak\token[0.1725]| slope|\allowbreak\token[0.0140]| of|\allowbreak\token[0.2248]| \|\allowbreak\token[0.0000]|(\|\allowbreak\token[0.0000]|over|\allowbreak\token[0.0000]|line|\allowbreak\token[0.0000]|{|\allowbreak\token[0.0000]|AC|\allowbreak\token[0.0000]|}\|\allowbreak\token[0.0000]|)|\allowbreak\token[0.0000]| is|\allowbreak\token[0.0005]| \|\allowbreak\token[0.0040]|(\|\allowbreak\token[0.0001]|frac|\allowbreak\token[0.0002]|{|\allowbreak\token[0.6061]|0|\allowbreak\token[0.5029]| -|\allowbreak\token[0.0000]| |\allowbreak\token[0.0006]|6|\allowbreak\token[0.0000]|}{|\allowbreak\token[0.0000]|8|\allowbreak\token[0.0000]| -|\allowbreak\token[0.0000]| |\allowbreak\token[0.0000]|0|\allowbreak\token[0.0000]|}|\allowbreak\token[0.0000]| =|\allowbreak\token[0.0002]| -|\allowbreak\token[0.0000]|\|\allowbreak\token[0.0000]|frac|\allowbreak\token[0.0000]|{|\allowbreak\token[0.0101]|3|\allowbreak\token[0.0000]|}{|\allowbreak\token[0.0001]|4|\allowbreak\token[0.0000]|}\|\allowbreak\token[0.0039]|).↩|\newline\allowbreak\token[0.0000]|  |\allowbreak\token[0.0000]| -|\allowbreak\token[0.0269]| The|\allowbreak\token[0.0007]| equation|\allowbreak\token[0.0001]| of|\allowbreak\token[0.2305]| the|\allowbreak\token[0.0000]| line|\allowbreak\cut{[CUT]}\allowbreak\token[1.1785]| \|\allowbreak\token[0.0000]|(\|\allowbreak\token[0.0000]|over|\allowbreak\token[0.0000]|line|\allowbreak\token[0.0000]|{|\allowbreak\token[0.0000]|AC|\allowbreak\token[0.0000]|}\|\allowbreak\token[0.0000]|)|\allowbreak\token[1.0265]| is|\allowbreak\token[0.0003]| \(|\allowbreak\token[0.0010]|y|\allowbreak\token[0.0739]| =|\allowbreak\token[0.0017]| -|\allowbreak\token[0.0000]|\|\allowbreak\token[0.0000]|frac|\allowbreak\token[0.0000]|{|\allowbreak\token[0.0000]|3|\allowbreak\token[0.0000]|}{|\allowbreak\token[0.0002]|4|\allowbreak\token[0.0002]|}|\allowbreak\token[0.0003]|x|\allowbreak\token[0.0000]| +|\allowbreak\token[0.0000]| |\allowbreak\token[0.0000]|6|\allowbreak\token[0.0000]|\|\allowbreak\token[0.0000]|).↩↩|\tighttracebreak\allowbreak\token[0.0001]|3|\allowbreak\token[0.0000]|.|\allowbreak\token[0.0000]| **|\allowbreak\token[0.0000]|Find|\allowbreak\token[0.0000]| the|\allowbreak\token[0.0066]| Equation|\allowbreak\token[0.0000]| of|\allowbreak\token[0.0001]| Line|\allowbreak\token[0.0002]| \|\allowbreak\token[0.0000]|(\|\allowbreak\token[0.0000]|over|\allowbreak\token[0.0000]|line|\allowbreak\token[0.0000]|{|\allowbreak\token[0.0001]|DE|\allowbreak\token[0.0000]|}\|\allowbreak\token[0.0001]|):|\allowbreak\token[0.0000]|**↩|\newline\allowbreak\token[0.0000]|  |\allowbreak\token[0.0000]| -|\allowbreak\token[0.0502]| The|\allowbreak\token[0.0201]| slope|\allowbreak\token[0.0000]| of|\allowbreak\token[0.0000]| \|\allowbreak\token[0.0000]|(\|\allowbreak\token[0.0000]|over|\allowbreak\token[0.0000]|line|\allowbreak\token[0.0000]|{|\allowbreak\token[0.0000]|DE|\allowbreak\token[0.0000]|}\|\allowbreak\token[0.0000]|)|\allowbreak\token[0.0104]| is|\allowbreak\token[0.0015]| \|\allowbreak\token[0.0000]|(\|\allowbreak\token[0.0000]|frac|\allowbreak\token[0.0005]|{|\allowbreak\token[0.4205]|0|\allowbreak\token[0.0001]| -|\allowbreak\token[0.0000]| |\allowbreak\token[0.0001]|3|\allowbreak\token[0.0000]|}{|\allowbreak\token[0.0003]|4|\allowbreak\token[0.0000]| -|\allowbreak\token[0.0000]| |\allowbreak\token[0.0000]|0|\allowbreak\token[0.0000]|}|\allowbreak\token[0.0000]| =|\allowbreak\token[0.0000]| -|\allowbreak\token[0.0000]|\|\allowbreak\token[0.0000]|frac|\allowbreak\token[0.0000]|{|\allowbreak\token[0.0000]|3|\allowbreak\token[0.0000]|}{|\allowbreak\token[0.0001]|4|\allowbreak\token[0.0000]|}\|\allowbreak\token[0.0005]|).↩|\newline\allowbreak\token[0.0000]|  |\allowbreak\token[0.0000]| -|\allowbreak\token[0.0150]| The|\allowbreak\token[0.0008]| equation|\allowbreak\token[0.0000]| of|\allowbreak\token[0.0000]| the|\allowbreak\token[0.0000]| line|\allowbreak\token[0.0003]| \|\allowbreak\token[0.0000]|(\|\allowbreak\token[0.0000]|over|\allowbreak\token[0.0000]|line|\allowbreak\token[0.0000]|{|\allowbreak\token[0.0000]|DE|\allowbreak\token[0.0000]|}\|\allowbreak\token[0.0001]|)|\allowbreak\token[0.0566]| is|\allowbreak\token[0.0004]| \(|\allowbreak\token[0.0000]|y|\allowbreak\token[0.4195]| =|\allowbreak\token[0.0035]| -|\allowbreak\token[0.0000]|\|\allowbreak\token[0.0000]|frac|\allowbreak\token[0.0000]|{|\allowbreak\token[0.0003]|3|\allowbreak\token[0.0000]|}{|\allowbreak\token[0.0001]|4|\allowbreak\token[0.0048]|}|\allowbreak\token[0.0000]|x|\allowbreak\token[0.0000]| +|\allowbreak\token[0.0003]| |\allowbreak\token[0.0002]|3|\allowbreak\token[0.0038]|\|\allowbreak\token[0.0003]|).↩↩|\tighttracebreak\allowbreak\token[0.0000]|4|\allowbreak\token[0.0000]|.|\allowbreak\token[0.0000]| **|\allowbreak\token[0.0033]|Find|\allowbreak\token[0.1233]| the|\allowbreak\token[0.5068]| Intersection|\allowbreak\token[0.2538]| Point|\allowbreak\token[0.0009]| \(|\allowbreak\token[0.0000]|F|\allowbreak\token[0.0494]|\|\allowbreak\token[0.0002]|):|\allowbreak\token[0.0000]|**↩|\newline\allowbreak\token[0.0000]|  |\allowbreak\token[0.0000]| -|\allowbreak\token[0.2619]| Set|\allowbreak\token[0.0020]| the|\allowbreak\token[0.0035]| equations|\allowbreak\token[0.5530]| equal|\allowbreak\token[0.0053]| to|\allowbreak\token[0.0914]| each|\allowbreak\token[0.0001]| other|\allowbreak\token[0.6532]|:|\allowbreak\token[0.0028]| \|\allowbreak\token[0.0000]|(-|\allowbreak\token[0.0000]|\|\allowbreak\token[0.0000]|frac|\allowbreak\token[0.0000]|{|\allowbreak\token[0.0001]|3|\allowbreak\token[0.0000]|}{|\allowbreak\token[0.0000]|4|\allowbreak[...]
\tcblower
\token[1.2308]| \|\allowbreak\token[0.0000]|(\|\allowbreak\token[0.0000]|over|\allowbreak\token[0.0000]|line|\allowbreak\token[0.0000]|{|\allowbreak\token[0.0000]|AC|\allowbreak\token[0.0000]|}\|\allowbreak\token[0.0000]|)|\allowbreak\token[1.0281]| using|\allowbreak\token[0.6860]| the|\allowbreak\token[0.0059]| point|\allowbreak\token[0.0022]|-s|\allowbreak\token[0.0000]|lope|\allowbreak\token[0.0073]| form|\allowbreak\token[0.7805]| \(|\allowbreak\token[0.0326]|y|\allowbreak\token[0.0901]| -|\allowbreak\token[0.0059]| y|\allowbreak\token[0.0000]|_|\allowbreak\token[0.0001]|1|\allowbreak\token[0.0000]| =|\allowbreak\token[0.0000]| m|\allowbreak\token[0.0000]|(x|\allowbreak\token[0.0000]| -|\allowbreak\token[0.0000]| x|\allowbreak\token[0.0000]|_|\allowbreak\token[0.0001]|1|\allowbreak\token[0.0000]|)\|\allowbreak\token[0.0087]|)|\allowbreak\token[0.1413]| is|\allowbreak\token[0.6949]| \(|\allowbreak\token[0.0002]|y|\allowbreak\token[0.6275]| =|\allowbreak\token[0.0027]| -|\allowbreak\token[0.0000]|\|\allowbreak\token[0.0000]|frac|\allowbreak\token[0.0000]|{|\allowbreak\token[0.0000]|3|\allowbreak\token[0.0000]|}{|\allowbreak\token[0.0002]|4|\allowbreak\token[0.0059]|}|\allowbreak\token[0.0002]|x|\allowbreak\token[0.0000]| +|\allowbreak\token[0.0000]| |\allowbreak\token[0.0001]|6|\allowbreak\token[0.0000]|\|\allowbreak\token[0.0000]|).↩↩|\tighttracebreak\allowbreak\token[0.0001]|3|\allowbreak\token[0.0000]|.|\allowbreak\token[0.0000]| **|\allowbreak\token[0.0000]|Find|\allowbreak\token[0.0000]| the|\allowbreak\token[0.0043]| Equation|\allowbreak\token[0.0000]| of|\allowbreak\token[0.0001]| Line|\allowbreak\token[0.0003]| \|\allowbreak\token[0.0000]|(\|\allowbreak\token[0.0001]|over|\allowbreak\token[0.0000]|line|\allowbreak\token[0.0000]|{|\allowbreak\token[0.0001]|DE|\allowbreak\token[0.0000]|}\|\allowbreak\token[0.0001]|):|\allowbreak\token[0.0000]|**↩|\newline\allowbreak\token[0.0000]|  |\allowbreak\token[0.0000]| -|\allowbreak\token[0.0287]| The|\allowbreak\token[0.0244]| slope|\allowbreak\token[0.0000]| of|\allowbreak\token[0.0000]| \|\allowbreak\token[0.0000]|(\|\allowbreak\token[0.0000]|over|\allowbreak\token[0.0000]|line|\allowbreak\token[0.0000]|{|\allowbreak\token[0.0000]|DE|\allowbreak\token[0.0000]|}\|\allowbreak\token[0.0000]|)|\allowbreak\token[0.0055]| is|\allowbreak\token[0.0019]| \|\allowbreak\token[0.0000]|(\|\allowbreak\token[0.0000]|frac|\allowbreak\token[0.0004]|{|\allowbreak\token[0.3931]|0|\allowbreak\token[0.0001]| -|\allowbreak\token[0.0000]| |\allowbreak\token[0.0002]|3|\allowbreak\token[0.0000]|}{|\allowbreak\token[0.0004]|4|\allowbreak\token[0.0000]| -|\allowbreak\token[0.0000]| |\allowbreak\token[0.0000]|0|\allowbreak\token[0.0000]|}|\allowbreak\token[0.0000]| =|\allowbreak\token[0.0000]| -|\allowbreak\token[0.0000]|\|\allowbreak\token[0.0000]|frac|\allowbreak\token[0.0000]|{|\allowbreak\token[0.0000]|3|\allowbreak\token[0.0000]|}{|\allowbreak\token[0.0001]|4|\allowbreak\token[0.0000]|}\|\allowbreak\token[0.0009]|).↩|\newline\allowbreak\token[0.0000]|  |\allowbreak\token[0.0000]| -|\allowbreak\token[0.0450]| The|\allowbreak\token[0.0008]| equation|\allowbreak\token[0.0000]| of|\allowbreak\token[0.0000]| the|\allowbreak\token[0.0000]| line|\allowbreak\token[0.0002]| \|\allowbreak\token[0.0000]|(\|\allowbreak\token[0.0000]|over|\allowbreak\token[0.0000]|line|\allowbreak\token[0.0000]|{|\allowbreak\token[0.0000]|DE|\allowbreak\token[0.0000]|}\|\allowbreak\token[0.0001]|)|\allowbreak\token[0.2722]| using|\allowbreak\token[0.0126]| the|\allowbreak\token[0.0142]| point|\allowbreak\token[0.0383]|-s|\allowbreak\token[0.0001]|lope|\allowbreak\token[0.0000]| form|\allowbreak\token[0.6312]| \(|\allowbreak\token[0.0000]|y|\allowbreak\token[0.0000]| -|\allowbreak\token[0.0011]| y|\allowbreak\token[0.0000]|_|\allowbreak\token[0.0002]|1|\allowbreak\token[0.0000]| =|\allowbreak\token[0.0000]| m|\allowbreak\token[0.0000]|(x|\allowbreak\token[0.0000]| -|\allowbreak\token[0.0000]| x|\allowbreak\token[0.0002]|_|\allowbreak\token[0.0002]|1|\allowbreak\token[0.0000]|)\|\allowbreak\token[0.0000]|)|\allowbreak\token[0.0373]| is|\allowbreak\token[0.0000]| \(|\allowbreak\token[0.0000]|y|\allowbreak\token[0.6763]| =|\allowbreak\token[0.0079]| -|\allowbreak\token[0.0001]|\|\allowbreak[...]
\end{twopartbox}
\begin{twopartbox}{Cut Index 159 | Acceptance Probability 0.0045}{Block 2 Step 3}{Proposal (Rejected)}
[...]\allowbreak\token[0.3013]|DB|\allowbreak\token[0.0001]|EF|\allowbreak\token[0.6859]|\|\allowbreak\token[0.3149]|).|\allowbreak\token[1.6625]| Here|\allowbreak\token[1.0700]|'s|\allowbreak\token[0.6866]| the|\allowbreak\token[0.1679]| step|\allowbreak\token[0.0002]|-by|\allowbreak\token[0.0001]|-step|\allowbreak\token[1.4676]| approach|\allowbreak\token[0.1317]|:↩↩|\tighttracebreak\allowbreak\token[0.0016]|1|\allowbreak\token[0.0000]|.|\allowbreak\token[0.4060]| **|\allowbreak\token[0.4029]|Ident|\allowbreak\token[0.0010]|ify|\allowbreak\token[0.6278]| Coordinates|\allowbreak\token[1.0343]| of|\allowbreak\token[0.7862]| Points|\allowbreak\token[1.0859]|:|\allowbreak\token[0.0014]|**↩|\newline\allowbreak\token[0.2277]|  |\allowbreak\token[0.0177]| -|\allowbreak\token[1.1397]| Point|\allowbreak\token[0.0003]| \(|\allowbreak\token[0.4684]|A|\allowbreak\token[0.2856]|\)|\allowbreak\token[0.1606]| is|\allowbreak\token[0.6652]| \|\allowbreak\token[0.0001]|((|\allowbreak\token[0.0000]|0|\allowbreak\token[0.0000]|,|\allowbreak\token[0.0901]| |\allowbreak\token[0.0000]|6|\allowbreak\token[0.0000]|)\|\allowbreak\token[0.3752]|).↩|\newline\allowbreak\token[0.0000]|  |\allowbreak\token[0.0000]| -|\allowbreak\token[0.0001]| Point|\allowbreak\token[0.0000]| \(|\allowbreak\token[0.0005]|B|\allowbreak\token[0.0000]|\)|\allowbreak\token[0.0000]| is|\allowbreak\token[0.0001]| \|\allowbreak\token[0.0000]|((|\allowbreak\token[0.0000]|0|\allowbreak\token[0.0000]|,|\allowbreak\token[0.0001]| |\allowbreak\token[0.0000]|0|\allowbreak\token[0.0000]|)\|\allowbreak\token[0.0002]|).↩|\newline\allowbreak\token[0.0000]|  |\allowbreak\token[0.0000]| -|\allowbreak\token[0.0000]| Point|\allowbreak\token[0.0000]| \(|\allowbreak\token[0.0001]|C|\allowbreak\token[0.0000]|\)|\allowbreak\token[0.0000]| is|\allowbreak\token[0.0000]| \|\allowbreak\token[0.0001]|((|\allowbreak\token[0.0000]|8|\allowbreak\token[0.0000]|,|\allowbreak\token[0.0000]| |\allowbreak\token[0.0000]|0|\allowbreak\token[0.0000]|)\|\allowbreak\token[0.0820]|).↩|\newline\allowbreak\token[0.0053]|  |\allowbreak\token[0.0000]| -|\allowbreak\token[0.7278]| Point|\allowbreak\token[0.0000]| \(|\allowbreak\token[0.0022]|D|\allowbreak\token[0.4194]|\)|\allowbreak\token[0.0197]| is|\allowbreak\token[0.0001]| the|\allowbreak\token[0.0000]| midpoint|\allowbreak\token[0.0000]| of|\allowbreak\token[0.0496]| \|\allowbreak\token[0.0004]|(\|\allowbreak\token[0.0000]|over|\allowbreak\token[0.0000]|line|\allowbreak\token[0.0000]|{|\allowbreak\token[0.0000]|AB|\allowbreak\token[0.0001]|}\|\allowbreak\token[0.0147]|),|\allowbreak\token[0.0025]| so|\allowbreak\token[0.4478]| \(|\allowbreak\token[0.0000]|D|\allowbreak\token[0.6075]|\)|\allowbreak\token[0.0034]| is|\allowbreak\token[0.0002]| \|\allowbreak\token[0.0905]|((|\allowbreak\token[0.0102]|0|\allowbreak\token[0.0000]|,|\allowbreak\token[0.0240]| |\allowbreak\token[0.0000]|3|\allowbreak\token[0.0000]|)\|\allowbreak\token[0.0612]|).↩|\newline\allowbreak\token[0.0000]|  |\allowbreak\token[0.0000]| -|\allowbreak\token[0.0000]| Point|\allowbreak\token[0.0000]| \(|\allowbreak\token[0.0002]|E|\allowbreak\token[0.0001]|\)|\allowbreak\token[0.0000]| is|\allowbreak\token[0.0000]| the|\allowbreak\token[0.0000]| midpoint|\allowbreak\token[0.0000]| of|\allowbreak\token[0.0000]| \|\allowbreak\token[0.0000]|(\|\allowbreak\token[0.0002]|over|\allowbreak\token[0.0001]|line|\allowbreak\token[0.0000]|{|\allowbreak\token[0.0002]|BC|\allowbreak\token[0.0000]|}\|\allowbreak\token[0.0000]|),|\allowbreak\token[0.0000]| so|\allowbreak\token[0.0000]| \(|\allowbreak\token[0.0001]|E|\allowbreak\token[0.0002]|\)|\allowbreak\token[0.0000]| is|\allowbreak\token[0.0000]| \|\allowbreak\token[0.0007]|((|\allowbreak\token[0.0000]|4|\allowbreak\token[0.0000]|,|\allowbreak\token[0.0000]| |\allowbreak\token[0.0001]|0|\allowbreak\token[0.0000]|)\|\allowbreak\token[0.3933]|).↩|\newline\allowbreak\token[0.0175]|  |\allowbreak\token[0.0000]| -|\allowbreak\token[0.1774]| Point|\allowbreak\token[0.0000]| \(|\allowbreak\token[0.0007]|F|\allowbreak\token[0.0040]|\)|\allowbreak\token[0.2097]| is|\allowbreak\token[1.1429]| the|\allowbreak\token[0.5879]| intersection|\allowbreak\token[0.0027]| of|\allowbreak\token[0.9969]| \|\allowbreak\token[0.0018]|(\|\allowbreak\token[0.0001]|over|\allowbreak\token[0.0005]|line|\allowbreak\token[0.0001]|{|\allowbreak\cut{[CUT]}\allowbreak\token[1.0711]|AC|\allowbreak\token[0.0001]|}\|\allowbreak\token[0.0001]|)|\allowbreak\token[0.0799]| and|\allowbreak\token[0.2007]| \|\allowbreak\token[0.0001]|(\|\allowbreak\token[0.0001]|over|\allowbreak\token[0.0006]|line|\allowbreak\token[0.0000]|{|\allowbreak\token[0.0553]|DE|\allowbreak\token[0.0002]|}\|\allowbreak\token[0.7056]|).↩↩|\tighttracebreak\allowbreak\token[0.0005]|2|\allowbreak\token[0.0000]|.|\allowbreak\token[0.0000]| **|\allowbreak\token[0.3037]|Find|\allowbreak\token[0.2515]| the|\allowbreak\token[0.4118]| Equation|\allowbreak\token[0.0000]| of|\allowbreak\token[0.0452]| Line|\allowbreak\token[0.7027]| \|\allowbreak\token[0.0001]|(\|\allowbreak\token[0.0001]|over|\allowbreak\token[0.0003]|line|\allowbreak\token[0.0000]|{|\allowbreak\token[0.2697]|AC|\allowbreak\token[0.0000]|}\|\allowbreak\token[0.0023]|):|\allowbreak\token[0.0001]|**↩|\newline\allowbreak\token[0.0000]|  |\allowbreak\token[0.6873]| -|\allowbreak\token[0.4607]| The|\allowbreak\token[0.1725]| slope|\allowbreak\token[0.0140]| of|\allowbreak\token[0.2248]| \|\allowbreak\token[0.0000]|(\|\allowbreak\token[0.0000]|over|\allowbreak\token[0.0000]|line|\allowbreak\token[0.0000]|{|\allowbreak\token[0.0000]|AC|\allowbreak\token[0.0000]|}\|\allowbreak\token[0.0000]|)|\allowbreak\token[0.0000]| is|\allowbreak\token[0.0005]| \|\allowbreak\token[0.0040]|(\|\allowbreak\token[0.0001]|frac|\allowbreak\token[0.0002]|{|\allowbreak\token[0.6061]|0|\allowbreak\token[0.5029]| -|\allowbreak\token[0.0000]| |\allowbreak\token[0.0006]|6|\allowbreak\token[0.0000]|}{|\allowbreak\token[0.0000]|8|\allowbreak\token[0.0000]| -|\allowbreak\token[0.0000]| |\allowbreak\token[0.0000]|0|\allowbreak\token[0.0000]|}|\allowbreak\token[0.0000]| =|\allowbreak\token[0.0002]| -|\allowbreak\token[0.0000]|\|\allowbreak\token[0.0000]|frac|\allowbreak\token[0.0000]|{|\allowbreak\token[0.0101]|3|\allowbreak\token[0.0000]|}{|\allowbreak\token[0.0001]|4|\allowbreak\token[0.0000]|}\|\allowbreak\token[0.0039]|).↩|\newline\allowbreak\token[0.0000]|  |\allowbreak\token[0.0000]| -|\allowbreak\token[0.0269]| The|\allowbreak\token[0.0007]| equation|\allowbreak\token[0.0001]| of|\allowbreak\token[0.2305]| the|\allowbreak\token[0.0000]| line|\allowbreak\token[1.1785]| \|\allowbreak\token[0.0000]|(\|\allowbreak\token[0.0000]|over|\allowbreak\token[0.0000]|line|\allowbreak\token[0.0000]|{|\allowbreak\token[0.0000]|AC|\allowbreak\token[0.0000]|}\|\allowbreak\token[0.0000]|)|\allowbreak\token[1.0265]| is|\allowbreak\token[0.0003]| \(|\allowbreak\token[0.0010]|y|\allowbreak\token[0.0739]| =|\allowbreak\token[0.0017]| -|\allowbreak\token[0.0000]|\|\allowbreak\token[0.0000]|frac|\allowbreak\token[0.0000]|{|\allowbreak\token[0.0000]|3|\allowbreak\token[0.0000]|}{|\allowbreak\token[0.0002]|4|\allowbreak\token[0.0002]|}|\allowbreak\token[0.0003]|x|\allowbreak\token[0.0000]| +|\allowbreak\token[0.0000]| |\allowbreak\token[0.0000]|6|\allowbreak\token[0.0000]|\|\allowbreak\token[0.0000]|).↩↩|\tighttracebreak\allowbreak\token[0.0001]|3|\allowbreak\token[0.0000]|.|\allowbreak\token[0.0000]| **|\allowbreak\token[0.0000]|Find|\allowbreak\token[0.0000]| the|\allowbreak\token[0.0066]| Equation|\allowbreak\token[0.0000]| of|\allowbreak\token[0.0001]| Line|\allowbreak\token[0.0002]| \|\allowbreak\token[0.0000]|(\|\allowbreak\token[0.0000]|over|\allowbreak\token[0.0000]|line|\allowbreak\token[0.0000]|{|\allowbreak\token[0.0001]|DE|\allowbreak\token[0.0000]|}\|\allowbreak\token[0.0001]|):|\allowbreak\token[0.0000]|**↩|\newline\allowbreak\token[0.0000]|  |\allowbreak\token[0.0000]| -|\allowbreak\token[0.0502]| The|\allowbreak\token[0.0201]| slope|\allowbreak\token[0.0000]| of|\allowbreak\token[0.0000]| \|\allowbreak\token[0.0000]|(\|\allowbreak\token[0.0000]|over|\allowbreak\token[0.0000]|line|\allowbreak\token[0.0000]|{|\allowbreak\token[0.0000]|DE|\allowbreak\token[0.0000]|}\|\allowbreak\token[0.0000]|)|\allowbreak\token[0.0104]| is|\allowbreak\token[0.0015]| \|\allowbreak\token[0.0000]|(\|\allowbreak\token[0.0000]|frac|\allowbreak\token[0.0005]|{|\allowbreak\token[0.4205]|0|\allowbreak\token[0.0001]| -|\allowbreak\token[0.0000]| |\allowbreak\token[0.0001]|3|\allowbreak\token[0.0000]|}{|\allowbreak\token[0.0003]|4|\allowbreak\token[0.0000]| -|\allowbreak\token[0.0000]| |\allowbreak\token[0.0000]|0|\allowbreak[...]
\tcblower
\token[1.0715]|AC|\allowbreak\token[0.0001]|}\|\allowbreak\token[0.0001]|)|\allowbreak\token[0.0799]| and|\allowbreak\token[0.1677]| \|\allowbreak\token[0.0001]|(\|\allowbreak\token[0.0001]|over|\allowbreak\token[0.0006]|line|\allowbreak\token[0.0000]|{|\allowbreak\token[0.0515]|DE|\allowbreak\token[0.0002]|}\|\allowbreak\token[0.7096]|).↩↩|\tighttracebreak\allowbreak\token[0.0005]|2|\allowbreak\token[0.0000]|.|\allowbreak\token[0.0000]| **|\allowbreak\token[0.2531]|Find|\allowbreak\token[0.2516]| the|\allowbreak\token[0.4451]| Equation|\allowbreak\token[0.0000]| of|\allowbreak\token[0.0506]| Line|\allowbreak\token[0.6952]| \|\allowbreak\token[0.0001]|(\|\allowbreak\token[0.0001]|over|\allowbreak\token[0.0003]|line|\allowbreak\token[0.0000]|{|\allowbreak\token[0.2697]|AC|\allowbreak\token[0.0000]|}\|\allowbreak\token[0.0021]|):|\allowbreak\token[0.0001]|**↩|\newline\allowbreak\token[0.0000]|  |\allowbreak\token[0.6873]| -|\allowbreak\token[0.4615]| The|\allowbreak\token[0.1446]| slope|\allowbreak\token[0.0156]| of|\allowbreak\token[0.2248]| \|\allowbreak\token[0.0000]|(\|\allowbreak\token[0.0000]|over|\allowbreak\token[0.0000]|line|\allowbreak\token[0.0000]|{|\allowbreak\token[0.0000]|AC|\allowbreak\token[0.0000]|}\|\allowbreak\token[0.0000]|)|\allowbreak\token[0.0000]| is|\allowbreak\token[0.0004]| \|\allowbreak\token[0.0045]|(\|\allowbreak\token[0.0001]|frac|\allowbreak\token[0.0002]|{|\allowbreak\token[0.5823]|0|\allowbreak\token[0.5029]| -|\allowbreak\token[0.0000]| |\allowbreak\token[0.0006]|6|\allowbreak\token[0.0000]|}{|\allowbreak\token[0.0000]|8|\allowbreak\token[0.0000]| -|\allowbreak\token[0.0000]| |\allowbreak\token[0.0000]|0|\allowbreak\token[0.0000]|}|\allowbreak\token[0.0000]| =|\allowbreak\token[0.0002]| -|\allowbreak\token[0.0000]|\|\allowbreak\token[0.0000]|frac|\allowbreak\token[0.0000]|{|\allowbreak\token[0.0101]|3|\allowbreak\token[0.0000]|}{|\allowbreak\token[0.0001]|4|\allowbreak\token[0.0000]|}\|\allowbreak\token[0.0043]|).↩|\newline\allowbreak\token[0.0000]|  |\allowbreak\token[0.0000]| -|\allowbreak\token[0.0270]| The|\allowbreak\token[0.0007]| equation|\allowbreak\token[0.0001]| of|\allowbreak\token[0.2305]| the|\allowbreak\token[0.0000]| line|\allowbreak\token[1.2028]| \|\allowbreak\token[0.0000]|(\|\allowbreak\token[0.0000]|over|\allowbreak\token[0.0000]|line|\allowbreak\token[0.0000]|{|\allowbreak\token[0.0000]|AC|\allowbreak\token[0.0000]|}\|\allowbreak\token[0.0000]|)|\allowbreak\token[1.0246]| using|\allowbreak\token[0.6860]| point|\allowbreak\token[0.0572]|-s|\allowbreak\token[0.0000]|lope|\allowbreak\token[0.0000]| form|\allowbreak\token[0.8905]| is|\allowbreak\token[0.0682]| \(|\allowbreak\token[0.0126]|y|\allowbreak\token[0.6275]| -|\allowbreak\token[0.0000]| |\allowbreak\token[0.4471]|6|\allowbreak\token[0.0000]| =|\allowbreak\token[0.0000]| -|\allowbreak\token[0.0000]|\|\allowbreak\token[0.0000]|frac|\allowbreak\token[0.0000]|{|\allowbreak\token[0.0000]|3|\allowbreak\token[0.0000]|}{|\allowbreak\token[0.0000]|4|\allowbreak\token[0.3921]|}|\allowbreak\token[0.0019]|x|\allowbreak\token[0.1974]|\|\allowbreak\token[0.0393]|),|\allowbreak\token[0.8627]| or|\allowbreak\token[0.0594]| \(|\allowbreak\token[0.0060]|y|\allowbreak\token[0.0000]| =|\allowbreak\token[0.2272]| -|\allowbreak\token[0.0000]|\|\allowbreak\token[0.0000]|frac|\allowbreak\token[0.0000]|{|\allowbreak\token[0.0000]|3|\allowbreak\token[0.0000]|}{|\allowbreak\token[0.0000]|4|\allowbreak\token[0.0000]|}|\allowbreak\token[0.0000]|x|\allowbreak\token[0.0000]| +|\allowbreak\token[0.0000]| |\allowbreak\token[0.0000]|6|\allowbreak\token[0.0000]|\|\allowbreak\token[0.0001]|).↩↩|\tighttracebreak\allowbreak\token[0.0001]|3|\allowbreak\token[0.0000]|.|\allowbreak\token[0.0000]| **|\allowbreak\token[0.0000]|Find|\allowbreak\token[0.0000]| the|\allowbreak\token[0.0041]| Equation|\allowbreak\token[0.0000]| of|\allowbreak\token[0.0001]| Line|\allowbreak\token[0.0002]| \|\allowbreak\token[0.0000]|(\|\allowbreak\token[0.0001]|over|\allowbreak\token[0.0000]|line|\allowbreak\token[0.0000]|{|\allowbreak\token[0.0000]|DE|\allowbreak\token[0.0000]|}\|\allowbreak\token[0.0001]|):|\allowbreak\token[0.0000]|**↩|\newline\allowbreak\token[0.0000]|  |\allowbreak\token[0.0000]| -|\allowbreak\token[0.0152]| The|\allowbreak\token[0.0122]| slope|\allowbreak[...]
\end{twopartbox}
\begin{twopartbox}{Cut Index 189 | Acceptance Probability 1.0000}{Block 3 Step 1}{Proposal (Accepted)}
[...]\allowbreak\token[0.0001]|((|\allowbreak\token[0.0000]|0|\allowbreak\token[0.0000]|,|\allowbreak\token[0.0901]| |\allowbreak\token[0.0000]|6|\allowbreak\token[0.0000]|)\|\allowbreak\token[0.3752]|).↩|\newline\allowbreak\token[0.0000]|  |\allowbreak\token[0.0000]| -|\allowbreak\token[0.0001]| Point|\allowbreak\token[0.0000]| \(|\allowbreak\token[0.0005]|B|\allowbreak\token[0.0000]|\)|\allowbreak\token[0.0000]| is|\allowbreak\token[0.0001]| \|\allowbreak\token[0.0000]|((|\allowbreak\token[0.0000]|0|\allowbreak\token[0.0000]|,|\allowbreak\token[0.0001]| |\allowbreak\token[0.0000]|0|\allowbreak\token[0.0000]|)\|\allowbreak\token[0.0002]|).↩|\newline\allowbreak\token[0.0000]|  |\allowbreak\token[0.0000]| -|\allowbreak\token[0.0000]| Point|\allowbreak\token[0.0000]| \(|\allowbreak\token[0.0001]|C|\allowbreak\token[0.0000]|\)|\allowbreak\token[0.0000]| is|\allowbreak\token[0.0000]| \|\allowbreak\token[0.0001]|((|\allowbreak\token[0.0000]|8|\allowbreak\token[0.0000]|,|\allowbreak\token[0.0000]| |\allowbreak\token[0.0000]|0|\allowbreak\token[0.0000]|)\|\allowbreak\token[0.0820]|).↩|\newline\allowbreak\token[0.0053]|  |\allowbreak\token[0.0000]| -|\allowbreak\token[0.7278]| Point|\allowbreak\token[0.0000]| \(|\allowbreak\token[0.0022]|D|\allowbreak\token[0.4194]|\)|\allowbreak\token[0.0197]| is|\allowbreak\token[0.0001]| the|\allowbreak\token[0.0000]| midpoint|\allowbreak\token[0.0000]| of|\allowbreak\token[0.0496]| \|\allowbreak\token[0.0004]|(\|\allowbreak\token[0.0000]|over|\allowbreak\token[0.0000]|line|\allowbreak\token[0.0000]|{|\allowbreak\token[0.0000]|AB|\allowbreak\token[0.0001]|}\|\allowbreak\token[0.0147]|),|\allowbreak\token[0.0025]| so|\allowbreak\token[0.4478]| \(|\allowbreak\token[0.0000]|D|\allowbreak\token[0.6075]|\)|\allowbreak\token[0.0034]| is|\allowbreak\token[0.0002]| \|\allowbreak\token[0.0905]|((|\allowbreak\token[0.0102]|0|\allowbreak\token[0.0000]|,|\allowbreak\token[0.0240]| |\allowbreak\token[0.0000]|3|\allowbreak\token[0.0000]|)\|\allowbreak\token[0.0612]|).↩|\newline\allowbreak\token[0.0000]|  |\allowbreak\token[0.0000]| -|\allowbreak\token[0.0000]| Point|\allowbreak\token[0.0000]| \(|\allowbreak\token[0.0002]|E|\allowbreak\token[0.0001]|\)|\allowbreak\token[0.0000]| is|\allowbreak\token[0.0000]| the|\allowbreak\token[0.0000]| midpoint|\allowbreak\token[0.0000]| of|\allowbreak\token[0.0000]| \|\allowbreak\token[0.0000]|(\|\allowbreak\token[0.0002]|over|\allowbreak\token[0.0001]|line|\allowbreak\token[0.0000]|{|\allowbreak\token[0.0002]|BC|\allowbreak\token[0.0000]|}\|\allowbreak\token[0.0000]|),|\allowbreak\token[0.0000]| so|\allowbreak\token[0.0000]| \(|\allowbreak\token[0.0001]|E|\allowbreak\token[0.0002]|\)|\allowbreak\token[0.0000]| is|\allowbreak\token[0.0000]| \|\allowbreak\token[0.0007]|((|\allowbreak\token[0.0000]|4|\allowbreak\token[0.0000]|,|\allowbreak\token[0.0000]| |\allowbreak\token[0.0001]|0|\allowbreak\token[0.0000]|)\|\allowbreak\token[0.3933]|).↩|\newline\allowbreak\token[0.0175]|  |\allowbreak\token[0.0000]| -|\allowbreak\token[0.1774]| Point|\allowbreak\token[0.0000]| \(|\allowbreak\token[0.0007]|F|\allowbreak\token[0.0040]|\)|\allowbreak\token[0.2097]| is|\allowbreak\token[1.1429]| the|\allowbreak\token[0.5879]| intersection|\allowbreak\token[0.0027]| of|\allowbreak\token[0.9969]| \|\allowbreak\token[0.0018]|(\|\allowbreak\token[0.0001]|over|\allowbreak\token[0.0005]|line|\allowbreak\token[0.0001]|{|\allowbreak\token[1.0711]|AC|\allowbreak\token[0.0001]|}\|\allowbreak\token[0.0001]|)|\allowbreak\token[0.0799]| and|\allowbreak\token[0.2007]| \|\allowbreak\token[0.0001]|(\|\allowbreak\token[0.0001]|over|\allowbreak\token[0.0006]|line|\allowbreak\token[0.0000]|{|\allowbreak\token[0.0553]|DE|\allowbreak\token[0.0002]|}\|\allowbreak\token[0.7056]|).↩↩|\tighttracebreak\allowbreak\token[0.0005]|2|\allowbreak\token[0.0000]|.|\allowbreak\token[0.0000]| **|\allowbreak\token[0.3037]|Find|\allowbreak\token[0.2515]| the|\allowbreak\token[0.4118]| Equation|\allowbreak\token[0.0000]| of|\allowbreak\token[0.0452]| Line|\allowbreak\token[0.7027]| \|\allowbreak\token[0.0001]|(\|\allowbreak\token[0.0001]|over|\allowbreak\token[0.0003]|line|\allowbreak\token[0.0000]|{|\allowbreak\token[0.2697]|AC|\allowbreak\token[0.0000]|}\|\allowbreak\token[0.0023]|):|\allowbreak\token[0.0001]|**↩|\newline\allowbreak\token[0.0000]|  |\allowbreak\cut{[CUT]}\allowbreak\token[0.6873]| -|\allowbreak\token[0.4607]| The|\allowbreak\token[0.1725]| slope|\allowbreak\token[0.0140]| of|\allowbreak\token[0.2248]| \|\allowbreak\token[0.0000]|(\|\allowbreak\token[0.0000]|over|\allowbreak\token[0.0000]|line|\allowbreak\token[0.0000]|{|\allowbreak\token[0.0000]|AC|\allowbreak\token[0.0000]|}\|\allowbreak\token[0.0000]|)|\allowbreak\token[0.0000]| is|\allowbreak\token[0.0005]| \|\allowbreak\token[0.0040]|(\|\allowbreak\token[0.0001]|frac|\allowbreak\token[0.0002]|{|\allowbreak\token[0.6061]|0|\allowbreak\token[0.5029]| -|\allowbreak\token[0.0000]| |\allowbreak\token[0.0006]|6|\allowbreak\token[0.0000]|}{|\allowbreak\token[0.0000]|8|\allowbreak\token[0.0000]| -|\allowbreak\token[0.0000]| |\allowbreak\token[0.0000]|0|\allowbreak\token[0.0000]|}|\allowbreak\token[0.0000]| =|\allowbreak\token[0.0002]| -|\allowbreak\token[0.0000]|\|\allowbreak\token[0.0000]|frac|\allowbreak\token[0.0000]|{|\allowbreak\token[0.0101]|3|\allowbreak\token[0.0000]|}{|\allowbreak\token[0.0001]|4|\allowbreak\token[0.0000]|}\|\allowbreak\token[0.0039]|).↩|\newline\allowbreak\token[0.0000]|  |\allowbreak\token[0.0000]| -|\allowbreak\token[0.0269]| The|\allowbreak\token[0.0007]| equation|\allowbreak\token[0.0001]| of|\allowbreak\token[0.2305]| the|\allowbreak\token[0.0000]| line|\allowbreak\token[1.1785]| \|\allowbreak\token[0.0000]|(\|\allowbreak\token[0.0000]|over|\allowbreak\token[0.0000]|line|\allowbreak\token[0.0000]|{|\allowbreak\token[0.0000]|AC|\allowbreak\token[0.0000]|}\|\allowbreak\token[0.0000]|)|\allowbreak\token[1.0265]| is|\allowbreak\token[0.0003]| \(|\allowbreak\token[0.0010]|y|\allowbreak\token[0.0739]| =|\allowbreak\token[0.0017]| -|\allowbreak\token[0.0000]|\|\allowbreak\token[0.0000]|frac|\allowbreak\token[0.0000]|{|\allowbreak\token[0.0000]|3|\allowbreak\token[0.0000]|}{|\allowbreak\token[0.0002]|4|\allowbreak\token[0.0002]|}|\allowbreak\token[0.0003]|x|\allowbreak\token[0.0000]| +|\allowbreak\token[0.0000]| |\allowbreak\token[0.0000]|6|\allowbreak\token[0.0000]|\|\allowbreak\token[0.0000]|).↩↩|\tighttracebreak\allowbreak\token[0.0001]|3|\allowbreak\token[0.0000]|.|\allowbreak\token[0.0000]| **|\allowbreak\token[0.0000]|Find|\allowbreak\token[0.0000]| the|\allowbreak\token[0.0066]| Equation|\allowbreak\token[0.0000]| of|\allowbreak\token[0.0001]| Line|\allowbreak\token[0.0002]| \|\allowbreak\token[0.0000]|(\|\allowbreak\token[0.0000]|over|\allowbreak\token[0.0000]|line|\allowbreak\token[0.0000]|{|\allowbreak\token[0.0001]|DE|\allowbreak\token[0.0000]|}\|\allowbreak\token[0.0001]|):|\allowbreak\token[0.0000]|**↩|\newline\allowbreak\token[0.0000]|  |\allowbreak\token[0.0000]| -|\allowbreak\token[0.0502]| The|\allowbreak\token[0.0201]| slope|\allowbreak\token[0.0000]| of|\allowbreak\token[0.0000]| \|\allowbreak\token[0.0000]|(\|\allowbreak\token[0.0000]|over|\allowbreak\token[0.0000]|line|\allowbreak\token[0.0000]|{|\allowbreak\token[0.0000]|DE|\allowbreak\token[0.0000]|}\|\allowbreak\token[0.0000]|)|\allowbreak\token[0.0104]| is|\allowbreak\token[0.0015]| \|\allowbreak\token[0.0000]|(\|\allowbreak\token[0.0000]|frac|\allowbreak\token[0.0005]|{|\allowbreak\token[0.4205]|0|\allowbreak\token[0.0001]| -|\allowbreak\token[0.0000]| |\allowbreak\token[0.0001]|3|\allowbreak\token[0.0000]|}{|\allowbreak\token[0.0003]|4|\allowbreak\token[0.0000]| -|\allowbreak\token[0.0000]| |\allowbreak\token[0.0000]|0|\allowbreak\token[0.0000]|}|\allowbreak\token[0.0000]| =|\allowbreak\token[0.0000]| -|\allowbreak\token[0.0000]|\|\allowbreak\token[0.0000]|frac|\allowbreak\token[0.0000]|{|\allowbreak\token[0.0000]|3|\allowbreak\token[0.0000]|}{|\allowbreak\token[0.0001]|4|\allowbreak\token[0.0000]|}\|\allowbreak\token[0.0005]|).↩|\newline\allowbreak\token[0.0000]|  |\allowbreak\token[0.0000]| -|\allowbreak\token[0.0150]| The|\allowbreak\token[0.0008]| equation|\allowbreak\token[0.0000]| of|\allowbreak\token[0.0000]| the|\allowbreak\token[0.0000]| line|\allowbreak\token[0.0003]| \|\allowbreak\token[0.0000]|(\|\allowbreak\token[0.0000]|over|\allowbreak\token[0.0000]|line|\allowbreak\token[0.0000]|{|\allowbreak\token[0.0000]|DE|\allowbreak\token[0.0000]|}\|\allowbreak\token[0.0001]|)|\allowbreak\token[0.0566]| is|\allowbreak\token[0.0004]| \(|\allowbreak\token[0.0000]|y|\allowbreak\token[0.4195]| =|\allowbreak[...]
\tcblower
\token[0.6873]| -|\allowbreak\token[0.5088]| The|\allowbreak\token[0.1724]| slope|\allowbreak\token[0.0156]| of|\allowbreak\token[0.2285]| \|\allowbreak\token[0.0000]|(\|\allowbreak\token[0.0000]|over|\allowbreak\token[0.0000]|line|\allowbreak\token[0.0000]|{|\allowbreak\token[0.0000]|AC|\allowbreak\token[0.0000]|}\|\allowbreak\token[0.0000]|)|\allowbreak\token[0.0000]| is|\allowbreak\token[0.0004]| \|\allowbreak\token[0.0045]|(\|\allowbreak\token[0.0001]|frac|\allowbreak\token[0.0002]|{|\allowbreak\token[0.5824]|0|\allowbreak\token[0.5029]| -|\allowbreak\token[0.0000]| |\allowbreak\token[0.0006]|6|\allowbreak\token[0.0000]|}{|\allowbreak\token[0.0000]|8|\allowbreak\token[0.0000]| -|\allowbreak\token[0.0000]| |\allowbreak\token[0.0000]|0|\allowbreak\token[0.0000]|}|\allowbreak\token[0.0000]| =|\allowbreak\token[0.0002]| -|\allowbreak\token[0.0000]|\|\allowbreak\token[0.0000]|frac|\allowbreak\token[0.0000]|{|\allowbreak\token[0.0113]|3|\allowbreak\token[0.0000]|}{|\allowbreak\token[0.0001]|4|\allowbreak\token[0.0000]|}\|\allowbreak\token[0.0042]|).↩|\newline\allowbreak\token[0.0000]|  |\allowbreak\token[0.0000]| -|\allowbreak\token[0.0243]| The|\allowbreak\token[0.0008]| equation|\allowbreak\token[0.0001]| of|\allowbreak\token[0.2305]| the|\allowbreak\token[0.0000]| line|\allowbreak\token[1.2270]| \|\allowbreak\token[0.0000]|(\|\allowbreak\token[0.0000]|over|\allowbreak\token[0.0000]|line|\allowbreak\token[0.0000]|{|\allowbreak\token[0.0000]|AC|\allowbreak\token[0.0000]|}\|\allowbreak\token[0.0000]|)|\allowbreak\token[1.0281]| using|\allowbreak\token[0.6634]| point|\allowbreak\token[0.0571]|-s|\allowbreak\token[0.0000]|lope|\allowbreak\token[0.0000]| form|\allowbreak\token[0.8875]| is|\allowbreak\token[0.0682]| \(|\allowbreak\token[0.0113]|y|\allowbreak\token[0.6275]| -|\allowbreak\token[0.0000]| |\allowbreak\token[0.4471]|6|\allowbreak\token[0.0000]| =|\allowbreak\token[0.0000]| -|\allowbreak\token[0.0000]|\|\allowbreak\token[0.0000]|frac|\allowbreak\token[0.0000]|{|\allowbreak\token[0.0000]|3|\allowbreak\token[0.0000]|}{|\allowbreak\token[0.0000]|4|\allowbreak\token[0.3921]|}|\allowbreak\token[0.0017]|x|\allowbreak\token[0.1836]|\|\allowbreak\token[0.0437]|),|\allowbreak\token[0.8148]| which|\allowbreak\token[0.0001]| simpl|\allowbreak\token[0.0000]|ifies|\allowbreak\token[0.0000]| to|\allowbreak\token[0.0000]| \(|\allowbreak\token[0.0405]|y|\allowbreak\token[0.0000]| =|\allowbreak\token[0.2084]| -|\allowbreak\token[0.0000]|\|\allowbreak\token[0.0000]|frac|\allowbreak\token[0.0000]|{|\allowbreak\token[0.0000]|3|\allowbreak\token[0.0000]|}{|\allowbreak\token[0.0000]|4|\allowbreak\token[0.0000]|}|\allowbreak\token[0.0000]|x|\allowbreak\token[0.0000]| +|\allowbreak\token[0.0000]| |\allowbreak\token[0.0000]|6|\allowbreak\token[0.0000]|\|\allowbreak\token[0.0000]|).↩↩|\tighttracebreak\allowbreak\token[0.0001]|3|\allowbreak\token[0.0000]|.|\allowbreak\token[0.0000]| **|\allowbreak\token[0.0000]|Find|\allowbreak\token[0.0000]| the|\allowbreak\token[0.0042]| Equation|\allowbreak\token[0.0000]| of|\allowbreak\token[0.0001]| Line|\allowbreak\token[0.0002]| \|\allowbreak\token[0.0000]|(\|\allowbreak\token[0.0001]|over|\allowbreak\token[0.0000]|line|\allowbreak\token[0.0000]|{|\allowbreak\token[0.0000]|DE|\allowbreak\token[0.0000]|}\|\allowbreak\token[0.0001]|):|\allowbreak\token[0.0000]|**↩|\newline\allowbreak\token[0.0000]|  |\allowbreak\token[0.0000]| -|\allowbreak\token[0.0145]| The|\allowbreak\token[0.0129]| slope|\allowbreak\token[0.0000]| of|\allowbreak\token[0.0000]| \|\allowbreak\token[0.0000]|(\|\allowbreak\token[0.0000]|over|\allowbreak\token[0.0000]|line|\allowbreak\token[0.0000]|{|\allowbreak\token[0.0000]|DE|\allowbreak\token[0.0000]|}\|\allowbreak\token[0.0000]|)|\allowbreak\token[0.0021]| is|\allowbreak\token[0.0010]| \|\allowbreak\token[0.0000]|(\|\allowbreak\token[0.0000]|frac|\allowbreak\token[0.0004]|{|\allowbreak\token[0.2918]|0|\allowbreak\token[0.0001]| -|\allowbreak\token[0.0000]| |\allowbreak\token[0.0001]|3|\allowbreak\token[0.0000]|}{|\allowbreak\token[0.0008]|4|\allowbreak\token[0.0000]| -|\allowbreak\token[0.0000]| |\allowbreak\token[0.0000]|0|\allowbreak\token[0.0000]|}|\allowbreak\token[0.0000]| =|\allowbreak\token[0.0000]| -|\allowbreak\token[0.0000]|\|\allowbreak[...]
\end{twopartbox}
\begin{twopartbox}{Cut Index 154 | Acceptance Probability 1.0000}{Block 3 Step 2}{Proposal (Accepted)}
[...]\allowbreak\token[0.1338]| area|\allowbreak\token[0.0002]| of|\allowbreak\token[0.1912]| quadr|\allowbreak\token[0.0000]|ilateral|\allowbreak\token[0.5312]| \(|\allowbreak\token[0.3013]|DB|\allowbreak\token[0.0001]|EF|\allowbreak\token[0.6859]|\|\allowbreak\token[0.3149]|).|\allowbreak\token[1.6625]| Here|\allowbreak\token[1.0700]|'s|\allowbreak\token[0.6866]| the|\allowbreak\token[0.1679]| step|\allowbreak\token[0.0002]|-by|\allowbreak\token[0.0001]|-step|\allowbreak\token[1.4676]| approach|\allowbreak\token[0.1317]|:↩↩|\tighttracebreak\allowbreak\token[0.0016]|1|\allowbreak\token[0.0000]|.|\allowbreak\token[0.4060]| **|\allowbreak\token[0.4029]|Ident|\allowbreak\token[0.0010]|ify|\allowbreak\token[0.6278]| Coordinates|\allowbreak\token[1.0343]| of|\allowbreak\token[0.7862]| Points|\allowbreak\token[1.0859]|:|\allowbreak\token[0.0014]|**↩|\newline\allowbreak\token[0.2277]|  |\allowbreak\token[0.0177]| -|\allowbreak\token[1.1397]| Point|\allowbreak\token[0.0003]| \(|\allowbreak\token[0.4684]|A|\allowbreak\token[0.2856]|\)|\allowbreak\token[0.1606]| is|\allowbreak\token[0.6652]| \|\allowbreak\token[0.0001]|((|\allowbreak\token[0.0000]|0|\allowbreak\token[0.0000]|,|\allowbreak\token[0.0901]| |\allowbreak\token[0.0000]|6|\allowbreak\token[0.0000]|)\|\allowbreak\token[0.3752]|).↩|\newline\allowbreak\token[0.0000]|  |\allowbreak\token[0.0000]| -|\allowbreak\token[0.0001]| Point|\allowbreak\token[0.0000]| \(|\allowbreak\token[0.0005]|B|\allowbreak\token[0.0000]|\)|\allowbreak\token[0.0000]| is|\allowbreak\token[0.0001]| \|\allowbreak\token[0.0000]|((|\allowbreak\token[0.0000]|0|\allowbreak\token[0.0000]|,|\allowbreak\token[0.0001]| |\allowbreak\token[0.0000]|0|\allowbreak\token[0.0000]|)\|\allowbreak\token[0.0002]|).↩|\newline\allowbreak\token[0.0000]|  |\allowbreak\token[0.0000]| -|\allowbreak\token[0.0000]| Point|\allowbreak\token[0.0000]| \(|\allowbreak\token[0.0001]|C|\allowbreak\token[0.0000]|\)|\allowbreak\token[0.0000]| is|\allowbreak\token[0.0000]| \|\allowbreak\token[0.0001]|((|\allowbreak\token[0.0000]|8|\allowbreak\token[0.0000]|,|\allowbreak\token[0.0000]| |\allowbreak\token[0.0000]|0|\allowbreak\token[0.0000]|)\|\allowbreak\token[0.0820]|).↩|\newline\allowbreak\token[0.0053]|  |\allowbreak\token[0.0000]| -|\allowbreak\token[0.7278]| Point|\allowbreak\token[0.0000]| \(|\allowbreak\token[0.0022]|D|\allowbreak\token[0.4194]|\)|\allowbreak\token[0.0197]| is|\allowbreak\token[0.0001]| the|\allowbreak\token[0.0000]| midpoint|\allowbreak\token[0.0000]| of|\allowbreak\token[0.0496]| \|\allowbreak\token[0.0004]|(\|\allowbreak\token[0.0000]|over|\allowbreak\token[0.0000]|line|\allowbreak\token[0.0000]|{|\allowbreak\token[0.0000]|AB|\allowbreak\token[0.0001]|}\|\allowbreak\token[0.0147]|),|\allowbreak\token[0.0025]| so|\allowbreak\token[0.4478]| \(|\allowbreak\token[0.0000]|D|\allowbreak\token[0.6075]|\)|\allowbreak\token[0.0034]| is|\allowbreak\token[0.0002]| \|\allowbreak\token[0.0905]|((|\allowbreak\token[0.0102]|0|\allowbreak\token[0.0000]|,|\allowbreak\token[0.0240]| |\allowbreak\token[0.0000]|3|\allowbreak\token[0.0000]|)\|\allowbreak\token[0.0612]|).↩|\newline\allowbreak\token[0.0000]|  |\allowbreak\token[0.0000]| -|\allowbreak\token[0.0000]| Point|\allowbreak\token[0.0000]| \(|\allowbreak\token[0.0002]|E|\allowbreak\token[0.0001]|\)|\allowbreak\token[0.0000]| is|\allowbreak\token[0.0000]| the|\allowbreak\token[0.0000]| midpoint|\allowbreak\token[0.0000]| of|\allowbreak\token[0.0000]| \|\allowbreak\token[0.0000]|(\|\allowbreak\token[0.0002]|over|\allowbreak\token[0.0001]|line|\allowbreak\token[0.0000]|{|\allowbreak\token[0.0002]|BC|\allowbreak\token[0.0000]|}\|\allowbreak\token[0.0000]|),|\allowbreak\token[0.0000]| so|\allowbreak\token[0.0000]| \(|\allowbreak\token[0.0001]|E|\allowbreak\token[0.0002]|\)|\allowbreak\token[0.0000]| is|\allowbreak\token[0.0000]| \|\allowbreak\token[0.0007]|((|\allowbreak\token[0.0000]|4|\allowbreak\token[0.0000]|,|\allowbreak\token[0.0000]| |\allowbreak\token[0.0001]|0|\allowbreak\token[0.0000]|)\|\allowbreak\token[0.3933]|).↩|\newline\allowbreak\token[0.0175]|  |\allowbreak\token[0.0000]| -|\allowbreak\token[0.1774]| Point|\allowbreak\token[0.0000]| \(|\allowbreak\token[0.0007]|F|\allowbreak\token[0.0040]|\)|\allowbreak\token[0.2097]| is|\allowbreak\token[1.1429]| the|\allowbreak\token[0.5879]| intersection|\allowbreak\token[0.0027]| of|\allowbreak\cut{[CUT]}\allowbreak\token[0.9969]| \|\allowbreak\token[0.0018]|(\|\allowbreak\token[0.0001]|over|\allowbreak\token[0.0005]|line|\allowbreak\token[0.0001]|{|\allowbreak\token[1.0711]|AC|\allowbreak\token[0.0001]|}\|\allowbreak\token[0.0001]|)|\allowbreak\token[0.0799]| and|\allowbreak\token[0.2007]| \|\allowbreak\token[0.0001]|(\|\allowbreak\token[0.0001]|over|\allowbreak\token[0.0006]|line|\allowbreak\token[0.0000]|{|\allowbreak\token[0.0553]|DE|\allowbreak\token[0.0002]|}\|\allowbreak\token[0.7056]|).↩↩|\tighttracebreak\allowbreak\token[0.0005]|2|\allowbreak\token[0.0000]|.|\allowbreak\token[0.0000]| **|\allowbreak\token[0.3037]|Find|\allowbreak\token[0.2515]| the|\allowbreak\token[0.4118]| Equation|\allowbreak\token[0.0000]| of|\allowbreak\token[0.0452]| Line|\allowbreak\token[0.7027]| \|\allowbreak\token[0.0001]|(\|\allowbreak\token[0.0001]|over|\allowbreak\token[0.0003]|line|\allowbreak\token[0.0000]|{|\allowbreak\token[0.2697]|AC|\allowbreak\token[0.0000]|}\|\allowbreak\token[0.0023]|):|\allowbreak\token[0.0001]|**↩|\newline\allowbreak\token[0.0000]|  |\allowbreak\token[0.6873]| -|\allowbreak\token[0.5088]| The|\allowbreak\token[0.1724]| slope|\allowbreak\token[0.0156]| of|\allowbreak\token[0.2285]| \|\allowbreak\token[0.0000]|(\|\allowbreak\token[0.0000]|over|\allowbreak\token[0.0000]|line|\allowbreak\token[0.0000]|{|\allowbreak\token[0.0000]|AC|\allowbreak\token[0.0000]|}\|\allowbreak\token[0.0000]|)|\allowbreak\token[0.0000]| is|\allowbreak\token[0.0004]| \|\allowbreak\token[0.0045]|(\|\allowbreak\token[0.0001]|frac|\allowbreak\token[0.0002]|{|\allowbreak\token[0.5824]|0|\allowbreak\token[0.5029]| -|\allowbreak\token[0.0000]| |\allowbreak\token[0.0006]|6|\allowbreak\token[0.0000]|}{|\allowbreak\token[0.0000]|8|\allowbreak\token[0.0000]| -|\allowbreak\token[0.0000]| |\allowbreak\token[0.0000]|0|\allowbreak\token[0.0000]|}|\allowbreak\token[0.0000]| =|\allowbreak\token[0.0002]| -|\allowbreak\token[0.0000]|\|\allowbreak\token[0.0000]|frac|\allowbreak\token[0.0000]|{|\allowbreak\token[0.0113]|3|\allowbreak\token[0.0000]|}{|\allowbreak\token[0.0001]|4|\allowbreak\token[0.0000]|}\|\allowbreak\token[0.0042]|).↩|\newline\allowbreak\token[0.0000]|  |\allowbreak\token[0.0000]| -|\allowbreak\token[0.0243]| The|\allowbreak\token[0.0008]| equation|\allowbreak\token[0.0001]| of|\allowbreak\token[0.2305]| the|\allowbreak\token[0.0000]| line|\allowbreak\token[1.2270]| \|\allowbreak\token[0.0000]|(\|\allowbreak\token[0.0000]|over|\allowbreak\token[0.0000]|line|\allowbreak\token[0.0000]|{|\allowbreak\token[0.0000]|AC|\allowbreak\token[0.0000]|}\|\allowbreak\token[0.0000]|)|\allowbreak\token[1.0281]| using|\allowbreak\token[0.6634]| point|\allowbreak\token[0.0571]|-s|\allowbreak\token[0.0000]|lope|\allowbreak\token[0.0000]| form|\allowbreak\token[0.8875]| is|\allowbreak\token[0.0682]| \(|\allowbreak\token[0.0113]|y|\allowbreak\token[0.6275]| -|\allowbreak\token[0.0000]| |\allowbreak\token[0.4471]|6|\allowbreak\token[0.0000]| =|\allowbreak\token[0.0000]| -|\allowbreak\token[0.0000]|\|\allowbreak\token[0.0000]|frac|\allowbreak\token[0.0000]|{|\allowbreak\token[0.0000]|3|\allowbreak\token[0.0000]|}{|\allowbreak\token[0.0000]|4|\allowbreak\token[0.3921]|}|\allowbreak\token[0.0017]|x|\allowbreak\token[0.1836]|\|\allowbreak\token[0.0437]|),|\allowbreak\token[0.8148]| which|\allowbreak\token[0.0001]| simpl|\allowbreak\token[0.0000]|ifies|\allowbreak\token[0.0000]| to|\allowbreak\token[0.0000]| \(|\allowbreak\token[0.0405]|y|\allowbreak\token[0.0000]| =|\allowbreak\token[0.2084]| -|\allowbreak\token[0.0000]|\|\allowbreak\token[0.0000]|frac|\allowbreak\token[0.0000]|{|\allowbreak\token[0.0000]|3|\allowbreak\token[0.0000]|}{|\allowbreak\token[0.0000]|4|\allowbreak\token[0.0000]|}|\allowbreak\token[0.0000]|x|\allowbreak\token[0.0000]| +|\allowbreak\token[0.0000]| |\allowbreak\token[0.0000]|6|\allowbreak\token[0.0000]|\|\allowbreak\token[0.0000]|).↩↩|\tighttracebreak\allowbreak\token[0.0001]|3|\allowbreak\token[0.0000]|.|\allowbreak\token[0.0000]| **|\allowbreak\token[0.0000]|Find|\allowbreak\token[0.0000]| the|\allowbreak\token[0.0042]| Equation|\allowbreak\token[0.0000]| of|\allowbreak\token[0.0001]| Line|\allowbreak\token[0.0002]| \|\allowbreak\token[0.0000]|(\|\allowbreak\token[0.0001]|over|\allowbreak\token[0.0000]|line|\allowbreak\token[0.0000]|{|\allowbreak[...]
\tcblower
\token[1.0073]| \|\allowbreak\token[0.0017]|(\|\allowbreak\token[0.0001]|over|\allowbreak\token[0.0006]|line|\allowbreak\token[0.0001]|{|\allowbreak\token[1.0740]|AE|\allowbreak\token[0.0002]|}\|\allowbreak\token[0.0000]|)|\allowbreak\token[0.0032]| and|\allowbreak\token[0.0049]| \|\allowbreak\token[0.0000]|(\|\allowbreak\token[0.0001]|over|\allowbreak\token[0.0002]|line|\allowbreak\token[0.0000]|{|\allowbreak\token[0.2170]|CD|\allowbreak\token[0.0001]|}\|\allowbreak\token[0.6869]|).↩↩|\tighttracebreak\allowbreak\token[0.0002]|2|\allowbreak\token[0.0000]|.|\allowbreak\token[0.0000]| **|\allowbreak\token[0.1827]|Find|\allowbreak\token[0.2494]| the|\allowbreak\token[1.1302]| Equ|\allowbreak\token[0.0000]|ations|\allowbreak\token[0.0000]| of|\allowbreak\token[0.1410]| Lines|\allowbreak\token[0.7289]| \(|\allowbreak\token[0.0039]|AE|\allowbreak\token[0.0004]|\)|\allowbreak\token[0.0000]| and|\allowbreak\token[0.0001]| \(|\allowbreak\token[0.0000]|CD|\allowbreak\token[0.0001]|\|\allowbreak\token[0.0002]|):|\allowbreak\token[0.0038]|**↩|\newline\allowbreak\token[0.0000]|  |\allowbreak\token[0.0000]| -|\allowbreak\token[0.5979]| Line|\allowbreak\token[0.0000]| \(|\allowbreak\token[0.0005]|AE|\allowbreak\token[0.0903]|\)|\allowbreak\token[0.0181]| passes|\allowbreak\token[0.0000]| through|\allowbreak\token[0.0143]| points|\allowbreak\token[0.0265]| \(|\allowbreak\token[0.0001]|A|\allowbreak\token[0.0035]|(|\allowbreak\token[0.0000]|0|\allowbreak\token[0.0000]|,|\allowbreak\token[0.0101]| |\allowbreak\token[0.0000]|6|\allowbreak\token[0.0000]|)\|\allowbreak\token[0.0000]|)|\allowbreak\token[0.0000]| and|\allowbreak\token[0.0000]| \(|\allowbreak\token[0.0001]|E|\allowbreak\token[0.0000]|(|\allowbreak\token[0.0000]|4|\allowbreak\token[0.0000]|,|\allowbreak\token[0.0000]| |\allowbreak\token[0.0000]|0|\allowbreak\token[0.0000]|)\|\allowbreak\token[0.5457]|).|\allowbreak\token[0.1948]| The|\allowbreak\token[0.0140]| slope|\allowbreak\token[0.4572]| of|\allowbreak\token[0.0608]| \(|\allowbreak\token[0.0000]|AE|\allowbreak\token[0.0000]|\)|\allowbreak\token[0.0000]| is|\allowbreak\token[0.0001]| \|\allowbreak\token[0.0076]|(\|\allowbreak\token[0.0000]|frac|\allowbreak\token[0.0000]|{|\allowbreak\token[0.0173]|0|\allowbreak\token[0.6466]| -|\allowbreak\token[0.0000]| |\allowbreak\token[0.0000]|6|\allowbreak\token[0.0000]|}{|\allowbreak\token[0.0000]|4|\allowbreak\token[0.0000]| -|\allowbreak\token[0.0000]| |\allowbreak\token[0.0000]|0|\allowbreak\token[0.0000]|}|\allowbreak\token[0.0000]| =|\allowbreak\token[0.0000]| -|\allowbreak\token[0.0004]|\|\allowbreak\token[0.0000]|frac|\allowbreak\token[0.0000]|{|\allowbreak\token[0.0011]|3|\allowbreak\token[0.0000]|}{|\allowbreak\token[0.0000]|2|\allowbreak\token[0.0000]|}\|\allowbreak\token[0.0819]|).|\allowbreak\token[0.6056]| The|\allowbreak\token[0.0002]| equation|\allowbreak\token[0.0155]| of|\allowbreak\token[0.5775]| the|\allowbreak\token[0.0000]| line|\allowbreak\token[0.8775]| is|\allowbreak\token[0.0018]| \(|\allowbreak\token[0.0005]|y|\allowbreak\token[0.1909]| =|\allowbreak\token[0.0002]| -|\allowbreak\token[0.0000]|\|\allowbreak\token[0.0000]|frac|\allowbreak\token[0.0000]|{|\allowbreak\token[0.0000]|3|\allowbreak\token[0.0000]|}{|\allowbreak\token[0.0000]|2|\allowbreak\token[0.0002]|}|\allowbreak\token[0.0001]|x|\allowbreak\token[0.0000]| +|\allowbreak\token[0.0001]| |\allowbreak\token[0.0000]|6|\allowbreak\token[0.0000]|\|\allowbreak\token[0.0001]|).↩|\newline\allowbreak\token[0.0000]|  |\allowbreak\token[0.0000]| -|\allowbreak\token[0.0000]| Line|\allowbreak\token[0.0000]| \(|\allowbreak\token[0.0000]|CD|\allowbreak\token[0.0000]|\)|\allowbreak\token[0.0001]| passes|\allowbreak\token[0.0000]| through|\allowbreak\token[0.0000]| points|\allowbreak\token[0.0000]| \(|\allowbreak\token[0.0001]|C|\allowbreak\token[0.0000]|(|\allowbreak\token[0.0000]|8|\allowbreak\token[0.0000]|,|\allowbreak\token[0.0001]| |\allowbreak\token[0.0000]|0|\allowbreak\token[0.0000]|)\|\allowbreak\token[0.0000]|)|\allowbreak\token[0.0000]| and|\allowbreak\token[0.0000]| \(|\allowbreak\token[0.0000]|D|\allowbreak\token[0.0000]|(|\allowbreak\token[0.0000]|0|\allowbreak\token[0.0000]|,|\allowbreak\token[0.0000]| |\allowbreak\token[0.0000]|3|\allowbreak\token[0.0000]|)\|\allowbreak[...]
\end{twopartbox}
\begin{twopartbox}{Cut Index 25 | Acceptance Probability 0.0000}{Block 3 Step 3}{Proposal (Rejected)}
\token[2.6002]| To|\allowbreak\token[0.6725]| solve|\allowbreak\token[0.8316]| the|\allowbreak\token[0.0687]| problem|\allowbreak\token[0.5135]|,|\allowbreak\token[0.3361]| we|\allowbreak\token[0.4652]| need|\allowbreak\token[0.0016]| to|\allowbreak\token[0.7385]| determine|\allowbreak\token[0.0004]| the|\allowbreak\token[0.1338]| area|\allowbreak\token[0.0002]| of|\allowbreak\token[0.1912]| quadr|\allowbreak\token[0.0000]|ilateral|\allowbreak\token[0.5312]| \(|\allowbreak\token[0.3013]|DB|\allowbreak\token[0.0001]|EF|\allowbreak\token[0.6859]|\|\allowbreak\token[0.3149]|).|\allowbreak\token[1.6625]| Here|\allowbreak\token[1.0700]|'s|\allowbreak\token[0.6866]| the|\allowbreak\token[0.1679]| step|\allowbreak\token[0.0002]|-by|\allowbreak\token[0.0001]|-step|\allowbreak\cut{[CUT]}\allowbreak\token[1.4676]| approach|\allowbreak\token[0.1317]|:↩↩|\tighttracebreak\allowbreak\token[0.0016]|1|\allowbreak\token[0.0000]|.|\allowbreak\token[0.4060]| **|\allowbreak\token[0.4029]|Ident|\allowbreak\token[0.0010]|ify|\allowbreak\token[0.6278]| Coordinates|\allowbreak\token[1.0343]| of|\allowbreak\token[0.7862]| Points|\allowbreak\token[1.0859]|:|\allowbreak\token[0.0014]|**↩|\newline\allowbreak\token[0.2277]|  |\allowbreak\token[0.0177]| -|\allowbreak\token[1.1397]| Point|\allowbreak\token[0.0003]| \(|\allowbreak\token[0.4684]|A|\allowbreak\token[0.2856]|\)|\allowbreak\token[0.1606]| is|\allowbreak\token[0.6652]| \|\allowbreak\token[0.0001]|((|\allowbreak\token[0.0000]|0|\allowbreak\token[0.0000]|,|\allowbreak\token[0.0901]| |\allowbreak\token[0.0000]|6|\allowbreak\token[0.0000]|)\|\allowbreak\token[0.3752]|).↩|\newline\allowbreak\token[0.0000]|  |\allowbreak\token[0.0000]| -|\allowbreak\token[0.0001]| Point|\allowbreak\token[0.0000]| \(|\allowbreak\token[0.0005]|B|\allowbreak\token[0.0000]|\)|\allowbreak\token[0.0000]| is|\allowbreak\token[0.0001]| \|\allowbreak\token[0.0000]|((|\allowbreak\token[0.0000]|0|\allowbreak\token[0.0000]|,|\allowbreak\token[0.0001]| |\allowbreak\token[0.0000]|0|\allowbreak\token[0.0000]|)\|\allowbreak\token[0.0002]|).↩|\newline\allowbreak\token[0.0000]|  |\allowbreak\token[0.0000]| -|\allowbreak\token[0.0000]| Point|\allowbreak\token[0.0000]| \(|\allowbreak\token[0.0001]|C|\allowbreak\token[0.0000]|\)|\allowbreak\token[0.0000]| is|\allowbreak\token[0.0000]| \|\allowbreak\token[0.0001]|((|\allowbreak\token[0.0000]|8|\allowbreak\token[0.0000]|,|\allowbreak\token[0.0000]| |\allowbreak\token[0.0000]|0|\allowbreak\token[0.0000]|)\|\allowbreak\token[0.0820]|).↩|\newline\allowbreak\token[0.0053]|  |\allowbreak\token[0.0000]| -|\allowbreak\token[0.7278]| Point|\allowbreak\token[0.0000]| \(|\allowbreak\token[0.0022]|D|\allowbreak\token[0.4194]|\)|\allowbreak\token[0.0197]| is|\allowbreak\token[0.0001]| the|\allowbreak\token[0.0000]| midpoint|\allowbreak\token[0.0000]| of|\allowbreak\token[0.0496]| \|\allowbreak\token[0.0004]|(\|\allowbreak\token[0.0000]|over|\allowbreak\token[0.0000]|line|\allowbreak\token[0.0000]|{|\allowbreak\token[0.0000]|AB|\allowbreak\token[0.0001]|}\|\allowbreak\token[0.0147]|),|\allowbreak\token[0.0025]| so|\allowbreak\token[0.4478]| \(|\allowbreak\token[0.0000]|D|\allowbreak\token[0.6075]|\)|\allowbreak\token[0.0034]| is|\allowbreak\token[0.0002]| \|\allowbreak\token[0.0905]|((|\allowbreak\token[0.0102]|0|\allowbreak\token[0.0000]|,|\allowbreak\token[0.0240]| |\allowbreak\token[0.0000]|3|\allowbreak\token[0.0000]|)\|\allowbreak\token[0.0612]|).↩|\newline\allowbreak\token[0.0000]|  |\allowbreak\token[0.0000]| -|\allowbreak\token[0.0000]| Point|\allowbreak\token[0.0000]| \(|\allowbreak\token[0.0002]|E|\allowbreak\token[0.0001]|\)|\allowbreak\token[0.0000]| is|\allowbreak\token[0.0000]| the|\allowbreak\token[0.0000]| midpoint|\allowbreak\token[0.0000]| of|\allowbreak\token[0.0000]| \|\allowbreak\token[0.0000]|(\|\allowbreak\token[0.0002]|over|\allowbreak\token[0.0001]|line|\allowbreak\token[0.0000]|{|\allowbreak\token[0.0002]|BC|\allowbreak\token[0.0000]|}\|\allowbreak\token[0.0000]|),|\allowbreak\token[0.0000]| so|\allowbreak\token[0.0000]| \(|\allowbreak\token[0.0001]|E|\allowbreak\token[0.0002]|\)|\allowbreak\token[0.0000]| is|\allowbreak\token[0.0000]| \|\allowbreak\token[0.0007]|((|\allowbreak\token[0.0000]|4|\allowbreak\token[0.0000]|,|\allowbreak\token[0.0000]| |\allowbreak\token[0.0001]|0|\allowbreak\token[0.0000]|)\|\allowbreak\token[0.3933]|).↩|\newline\allowbreak\token[0.0175]|  |\allowbreak\token[0.0000]| -|\allowbreak\token[0.1774]| Point|\allowbreak\token[0.0000]| \(|\allowbreak\token[0.0007]|F|\allowbreak\token[0.0040]|\)|\allowbreak\token[0.2097]| is|\allowbreak\token[1.1429]| the|\allowbreak\token[0.5879]| intersection|\allowbreak\token[0.0027]| of|\allowbreak\token[1.0073]| \|\allowbreak\token[0.0017]|(\|\allowbreak\token[0.0001]|over|\allowbreak\token[0.0006]|line|\allowbreak\token[0.0001]|{|\allowbreak\token[1.0740]|AE|\allowbreak\token[0.0002]|}\|\allowbreak\token[0.0000]|)|\allowbreak\token[0.0032]| and|\allowbreak\token[0.0049]| \|\allowbreak\token[0.0000]|(\|\allowbreak\token[0.0001]|over|\allowbreak\token[0.0002]|line|\allowbreak\token[0.0000]|{|\allowbreak\token[0.2170]|CD|\allowbreak[...]
\tcblower
\token[1.4676]| reasoning|\allowbreak\token[0.6398]|:↩↩|\tighttracebreak\allowbreak\token[0.0038]|1|\allowbreak\token[0.0000]|.|\allowbreak\token[0.2652]| **|\allowbreak\token[0.2870]|Ident|\allowbreak\token[0.0024]|ify|\allowbreak\token[0.6450]| Coordinates|\allowbreak\token[1.0184]| of|\allowbreak\token[0.6033]| Points|\allowbreak\token[0.9396]|:|\allowbreak\token[0.0009]|**↩|\newline\allowbreak\token[0.2687]|  |\allowbreak\token[0.0039]| -|\allowbreak\token[0.9193]| \(|\allowbreak\token[0.2074]|A|\allowbreak\token[0.7118]|(|\allowbreak\token[0.0000]|0|\allowbreak\token[0.0000]|,|\allowbreak\token[0.6060]| |\allowbreak\token[0.0001]|6|\allowbreak\token[0.0000]|)\|\allowbreak\token[0.0347]|)↩|\newline\allowbreak\token[0.0000]|  |\allowbreak\token[0.0000]| -|\allowbreak\token[0.0000]| \(|\allowbreak\token[0.0004]|B|\allowbreak\token[0.0000]|(|\allowbreak\token[0.0000]|0|\allowbreak\token[0.0000]|,|\allowbreak\token[0.0001]| |\allowbreak\token[0.0000]|0|\allowbreak\token[0.0000]|)\|\allowbreak\token[0.0002]|)↩|\newline\allowbreak\token[0.0000]|  |\allowbreak\token[0.0000]| -|\allowbreak\token[0.0002]| \(|\allowbreak\token[0.0001]|C|\allowbreak\token[0.0000]|(|\allowbreak\token[0.0000]|8|\allowbreak\token[0.0000]|,|\allowbreak\token[0.0001]| |\allowbreak\token[0.0000]|0|\allowbreak\token[0.0000]|)\|\allowbreak\token[0.0268]|)↩|\newline\allowbreak\token[0.0082]|  |\allowbreak\token[0.0005]| -|\allowbreak\token[0.1978]| \(|\allowbreak\token[0.0013]|D|\allowbreak\token[0.4003]|\)|\allowbreak\token[0.0367]| is|\allowbreak\token[0.0001]| the|\allowbreak\token[0.0000]| midpoint|\allowbreak\token[0.0000]| of|\allowbreak\token[0.1094]| \|\allowbreak\token[0.0001]|(\|\allowbreak\token[0.0000]|over|\allowbreak\token[0.0000]|line|\allowbreak\token[0.0000]|{|\allowbreak\token[0.0000]|AB|\allowbreak\token[0.0008]|}\|\allowbreak\token[0.0523]|),|\allowbreak\token[0.0003]| so|\allowbreak\token[0.0014]| \(|\allowbreak\token[0.0000]|D|\allowbreak\token[0.4585]|(|\allowbreak\token[0.0005]|0|\allowbreak\token[0.0000]|,|\allowbreak\token[0.0329]| |\allowbreak\token[0.0001]|3|\allowbreak\token[0.0000]|)\|\allowbreak\token[0.0512]|)↩|\newline\allowbreak\token[0.0000]|  |\allowbreak\token[0.0000]| -|\allowbreak\token[0.0000]| \(|\allowbreak\token[0.0000]|E|\allowbreak\token[0.0000]|\)|\allowbreak\token[0.0001]| is|\allowbreak\token[0.0000]| the|\allowbreak\token[0.0000]| midpoint|\allowbreak\token[0.0000]| of|\allowbreak\token[0.0000]| \|\allowbreak\token[0.0000]|(\|\allowbreak\token[0.0001]|over|\allowbreak\token[0.0001]|line|\allowbreak\token[0.0000]|{|\allowbreak\token[0.0003]|BC|\allowbreak\token[0.0000]|}\|\allowbreak\token[0.0000]|),|\allowbreak\token[0.0000]| so|\allowbreak\token[0.0000]| \(|\allowbreak\token[0.0001]|E|\allowbreak\token[0.0009]|(|\allowbreak\token[0.0000]|4|\allowbreak\token[0.0000]|,|\allowbreak\token[0.0000]| |\allowbreak\token[0.0001]|0|\allowbreak\token[0.0000]|)\|\allowbreak\token[0.5040]|)↩|\newline\allowbreak\token[0.0909]|  |\allowbreak\token[0.0001]| -|\allowbreak\token[0.7600]| \(|\allowbreak\token[0.0043]|F|\allowbreak\token[0.0137]|\)|\allowbreak\token[0.0623]| is|\allowbreak\token[0.5991]| the|\allowbreak\token[0.9144]| intersection|\allowbreak\token[0.0215]| of|\allowbreak\token[0.6304]| \|\allowbreak\token[0.0004]|(\|\allowbreak\token[0.0003]|over|\allowbreak\token[0.0008]|line|\allowbreak\token[0.0001]|{|\allowbreak\token[1.2375]|AC|\allowbreak\token[0.0003]|}\|\allowbreak\token[0.0001]|)|\allowbreak\token[0.0525]| and|\allowbreak\token[0.0964]| \|\allowbreak\token[0.0001]|(\|\allowbreak\token[0.0001]|over|\allowbreak\token[0.0017]|line|\allowbreak\token[0.0000]|{|\allowbreak\token[0.0686]|DE|\allowbreak\token[0.0003]|}\|\allowbreak\token[1.1814]|).|\allowbreak\token[1.3710]| To|\allowbreak\token[0.0005]| find|\allowbreak\token[0.6410]| \(|\allowbreak\token[0.0001]|F|\allowbreak\token[0.0006]|\|\allowbreak\token[0.1115]|),|\allowbreak\token[0.0086]| we|\allowbreak\token[0.3592]| need|\allowbreak\token[0.6918]| to|\allowbreak\token[0.9553]| determine|\allowbreak\token[0.0144]| the|\allowbreak\token[0.1804]| equations|\allowbreak\token[0.0000]| of|\allowbreak\token[0.5136]| lines|\allowbreak\token[0.5823]| \|\allowbreak\token[0.0001]|(\|\allowbreak\token[0.0000]|over|\allowbreak\token[0.0001]|line|\allowbreak\token[0.0000]|{|\allowbreak[...]
\end{twopartbox}
\begin{twopartbox}{Cut Index 444 | Acceptance Probability 0.9641}{Block 4 Step 1}{Proposal (Accepted)}
[...]\allowbreak\token[0.0000]| slope|\allowbreak\token[0.0000]| of|\allowbreak\token[0.0000]| \(|\allowbreak\token[0.0000]|CD|\allowbreak\token[0.0001]|\)|\allowbreak\token[0.0000]| is|\allowbreak\token[0.0000]| \|\allowbreak\token[0.0000]|(\|\allowbreak\token[0.0000]|frac|\allowbreak\token[0.0000]|{|\allowbreak\token[0.0214]|3|\allowbreak\token[0.0000]| -|\allowbreak\token[0.0000]| |\allowbreak\token[0.0000]|0|\allowbreak\token[0.0000]|}{|\allowbreak\token[0.0000]|0|\allowbreak\token[0.0000]| -|\allowbreak\token[0.0000]| |\allowbreak\token[0.0001]|8|\allowbreak\token[0.0000]|}|\allowbreak\token[0.0000]| =|\allowbreak\token[0.0000]| -|\allowbreak\token[0.0000]|\|\allowbreak\token[0.0000]|frac|\allowbreak\token[0.0000]|{|\allowbreak\token[0.0000]|3|\allowbreak\token[0.0000]|}{|\allowbreak\token[0.0000]|8|\allowbreak\token[0.0000]|}\|\allowbreak\token[0.0000]|).|\allowbreak\token[0.0000]| The|\allowbreak\token[0.0000]| equation|\allowbreak\token[0.0000]| of|\allowbreak\token[0.0000]| the|\allowbreak\token[0.0000]| line|\allowbreak\token[0.0001]| is|\allowbreak\token[0.0000]| \(|\allowbreak\token[0.0000]|y|\allowbreak\token[0.0126]| =|\allowbreak\token[0.0000]| -|\allowbreak\token[0.0000]|\|\allowbreak\token[0.0000]|frac|\allowbreak\token[0.0000]|{|\allowbreak\token[0.0000]|3|\allowbreak\token[0.0000]|}{|\allowbreak\token[0.0000]|8|\allowbreak\token[0.0025]|}|\allowbreak\token[0.0000]|x|\allowbreak\token[0.0000]| +|\allowbreak\token[0.0000]| |\allowbreak\token[0.0004]|3|\allowbreak\token[0.0000]|\|\allowbreak\token[0.0000]|).↩↩|\tighttracebreak\allowbreak\token[0.0000]|3|\allowbreak\token[0.0000]|.|\allowbreak\token[0.0000]| **|\allowbreak\token[0.0034]|Find|\allowbreak\token[0.0481]| the|\allowbreak\token[0.1750]| Intersection|\allowbreak\token[0.1130]| Point|\allowbreak\token[0.0005]| \(|\allowbreak\token[0.0000]|F|\allowbreak\token[0.0403]|\|\allowbreak\token[0.0002]|):|\allowbreak\token[0.0002]|**↩|\newline\allowbreak\token[0.0000]|  |\allowbreak\token[0.5168]| -|\allowbreak\token[0.1089]| Set|\allowbreak\token[0.0002]| the|\allowbreak\token[0.0022]| equations|\allowbreak\token[0.3176]| equal|\allowbreak\token[0.0001]| to|\allowbreak\token[0.0673]| each|\allowbreak\token[0.0000]| other|\allowbreak\token[0.2960]| to|\allowbreak\token[0.0053]| find|\allowbreak\token[0.0494]| the|\allowbreak\token[0.1657]| intersection|\allowbreak\token[0.4261]| point|\allowbreak\token[0.8292]|:|\allowbreak\token[0.0128]| \|\allowbreak\token[0.0003]|(-|\allowbreak\token[0.0000]|\|\allowbreak\token[0.0000]|frac|\allowbreak\token[0.0000]|{|\allowbreak\token[0.0000]|3|\allowbreak\token[0.0000]|}{|\allowbreak\token[0.0015]|2|\allowbreak\token[0.0000]|}|\allowbreak\token[0.0000]|x|\allowbreak\token[0.0000]| +|\allowbreak\token[0.0000]| |\allowbreak\token[0.0000]|6|\allowbreak\token[0.0000]| =|\allowbreak\token[0.0000]| -|\allowbreak\token[0.0000]|\|\allowbreak\token[0.0000]|frac|\allowbreak\token[0.0000]|{|\allowbreak\token[0.0000]|3|\allowbreak\token[0.0000]|}{|\allowbreak\token[0.0000]|8|\allowbreak\token[0.0000]|}|\allowbreak\token[0.0000]|x|\allowbreak\token[0.0000]| +|\allowbreak\token[0.0000]| |\allowbreak\token[0.0000]|3|\allowbreak\token[0.0000]|\|\allowbreak\token[0.5833]|).↩|\newline\allowbreak\token[0.0017]|  |\allowbreak\token[0.0091]| -|\allowbreak\token[0.7362]| Sol|\allowbreak\token[0.0000]|ving|\allowbreak\token[0.0012]| for|\allowbreak\token[0.0000]| \(|\allowbreak\token[0.0000]|x|\allowbreak\token[0.0027]|\|\allowbreak\token[0.6113]|),|\allowbreak\token[0.0001]| we|\allowbreak\token[0.0000]| get|\allowbreak\token[1.0969]| \|\allowbreak\token[0.0865]|(-|\allowbreak\token[0.0006]|\|\allowbreak\token[0.0000]|frac|\allowbreak\token[0.0000]|{|\allowbreak\token[0.0113]|3|\allowbreak\token[0.0000]|}{|\allowbreak\token[0.0000]|2|\allowbreak\token[0.0000]|}|\allowbreak\token[0.0000]|x|\allowbreak\token[0.0000]| +|\allowbreak\token[0.0021]| \|\allowbreak\token[0.0002]|frac|\allowbreak\token[0.0000]|{|\allowbreak\token[0.0000]|3|\allowbreak\token[0.0000]|}{|\allowbreak\token[0.0000]|8|\allowbreak\token[0.0000]|}|\allowbreak\token[0.0000]|x|\allowbreak\token[0.0008]| =|\allowbreak\token[0.0173]| |\allowbreak\token[0.0001]|3|\allowbreak\token[0.0000]| -|\allowbreak\token[0.0000]| |\allowbreak\token[0.0000]|6|\allowbreak\cut{[CUT]}\allowbreak\token[0.7513]|\|\allowbreak\token[0.1315]|),|\allowbreak\token[0.2558]| which|\allowbreak\token[0.0000]| simpl|\allowbreak\token[0.0000]|ifies|\allowbreak\token[0.0000]| to|\allowbreak\token[0.0001]| \|\allowbreak\token[0.0001]|(-|\allowbreak\token[0.0000]|\|\allowbreak\token[0.0000]|frac|\allowbreak\token[0.0000]|{|\allowbreak\token[0.6938]|1|\allowbreak\token[0.0064]|2|\allowbreak\token[0.0451]|}{|\allowbreak\token[0.0000]|8|\allowbreak\token[0.0000]|}|\allowbreak\token[0.0000]|x|\allowbreak\token[0.0000]| +|\allowbreak\token[0.0000]| \|\allowbreak\token[0.0000]|frac|\allowbreak\token[0.0000]|{|\allowbreak\token[0.0000]|3|\allowbreak\token[0.0000]|}{|\allowbreak\token[0.0000]|8|\allowbreak\token[0.0000]|}|\allowbreak\token[0.0000]|x|\allowbreak\token[0.0000]| =|\allowbreak\token[0.0000]| -|\allowbreak\token[0.0001]|3|\allowbreak\token[0.4754]|\|\allowbreak\token[0.0405]|),|\allowbreak\token[0.2643]| or|\allowbreak\token[0.0000]| \|\allowbreak\token[0.0000]|(-|\allowbreak\token[0.0000]|\|\allowbreak\token[0.0000]|frac|\allowbreak\token[0.0000]|{|\allowbreak\token[0.0001]|9|\allowbreak\token[0.0000]|}{|\allowbreak\token[0.0000]|8|\allowbreak\token[0.0000]|}|\allowbreak\token[0.0000]|x|\allowbreak\token[0.0000]| =|\allowbreak\token[0.0000]| -|\allowbreak\token[0.0000]|3|\allowbreak\token[0.0000]|\|\allowbreak\token[0.1479]|).|\allowbreak\token[1.1154]| So|\allowbreak\token[0.3653]|,|\allowbreak\token[0.0000]| \(|\allowbreak\token[0.0000]|x|\allowbreak\token[0.0000]| =|\allowbreak\token[0.0001]| \|\allowbreak\token[0.0000]|frac|\allowbreak\token[0.0002]|{|\allowbreak\token[0.3731]|8|\allowbreak\token[0.6471]| \|\allowbreak\token[0.0817]|cdot|\allowbreak\token[0.0000]| |\allowbreak\token[0.0008]|3|\allowbreak\token[0.0000]|}{|\allowbreak\token[0.0000]|9|\allowbreak\token[0.0000]|}|\allowbreak\token[0.0000]| =|\allowbreak\token[0.0000]| \|\allowbreak\token[0.0000]|frac|\allowbreak\token[0.0000]|{|\allowbreak\token[0.0605]|8|\allowbreak\token[0.0004]|}{|\allowbreak\token[0.0000]|3|\allowbreak\token[0.0027]|}\|\allowbreak\token[0.3662]|).↩|\newline\allowbreak\token[0.0000]|  |\allowbreak\token[0.0000]| -|\allowbreak\token[0.3806]| Sub|\allowbreak\token[0.0000]|stit|\allowbreak\token[0.0000]|uting|\allowbreak\token[0.0000]| \(|\allowbreak\token[0.0000]|x|\allowbreak\token[0.0000]| =|\allowbreak\token[0.0000]| \|\allowbreak\token[0.0000]|frac|\allowbreak\token[0.0000]|{|\allowbreak\token[0.0000]|8|\allowbreak\token[0.0000]|}{|\allowbreak\token[0.0000]|3|\allowbreak\token[0.0000]|}\|\allowbreak\token[0.0000]|)|\allowbreak\token[0.6276]| into|\allowbreak\token[0.7900]| the|\allowbreak\token[0.0003]| equation|\allowbreak\token[0.4269]| \(|\allowbreak\token[0.0000]|y|\allowbreak\token[0.0000]| =|\allowbreak\token[0.0000]| -|\allowbreak\token[0.0000]|\|\allowbreak\token[0.0000]|frac|\allowbreak\token[0.0000]|{|\allowbreak\token[0.0000]|3|\allowbreak\token[0.0000]|}{|\allowbreak\token[0.2272]|2|\allowbreak\token[0.0239]|}|\allowbreak\token[0.0113]|x|\allowbreak\token[0.0000]| +|\allowbreak\token[0.0000]| |\allowbreak\token[0.0000]|6|\allowbreak\token[0.0327]|\|\allowbreak\token[0.0000]|),|\allowbreak\token[0.0000]| we|\allowbreak\token[0.0000]| get|\allowbreak\token[0.0000]| \(|\allowbreak\token[0.0000]|y|\allowbreak\token[0.0000]| =|\allowbreak\token[0.0000]| -|\allowbreak\token[0.0000]|\|\allowbreak\token[0.0000]|frac|\allowbreak\token[0.0000]|{|\allowbreak\token[0.0000]|3|\allowbreak\token[0.0000]|}{|\allowbreak\token[0.0000]|2|\allowbreak\token[0.2275]|}|\allowbreak\token[0.0007]| \|\allowbreak\token[0.1745]|cdot|\allowbreak\token[0.0000]| \|\allowbreak\token[0.0000]|frac|\allowbreak\token[0.0000]|{|\allowbreak\token[0.0000]|8|\allowbreak\token[0.0000]|}{|\allowbreak\token[0.0000]|3|\allowbreak\token[0.0000]|}|\allowbreak\token[0.0000]| +|\allowbreak\token[0.0000]| |\allowbreak\token[0.0000]|6|\allowbreak\token[0.0000]| =|\allowbreak\token[0.0000]| -|\allowbreak\token[0.0011]|4|\allowbreak\token[0.0000]| +|\allowbreak\token[0.0000]| |\allowbreak\token[0.0000]|6|\allowbreak\token[0.0000]| =|\allowbreak\token[0.0000]| |\allowbreak\token[0.0000]|2|\allowbreak\token[0.0000]|\|\allowbreak\token[0.1749]|).|\allowbreak[...]
\tcblower
\token[0.7513]|\|\allowbreak\token[0.1283]|),|\allowbreak\token[0.2588]| which|\allowbreak\token[0.0000]| simpl|\allowbreak\token[0.0000]|ifies|\allowbreak\token[0.0000]| to|\allowbreak\token[0.0001]| \|\allowbreak\token[0.0001]|(-|\allowbreak\token[0.0000]|\|\allowbreak\token[0.0000]|frac|\allowbreak\token[0.0000]|{|\allowbreak\token[0.6938]|1|\allowbreak\token[0.0064]|2|\allowbreak\token[0.0451]|}{|\allowbreak\token[0.0000]|8|\allowbreak\token[0.0000]|}|\allowbreak\token[0.0000]|x|\allowbreak\token[0.0000]| +|\allowbreak\token[0.0000]| \|\allowbreak\token[0.0000]|frac|\allowbreak\token[0.0000]|{|\allowbreak\token[0.0000]|3|\allowbreak\token[0.0000]|}{|\allowbreak\token[0.0000]|8|\allowbreak\token[0.0000]|}|\allowbreak\token[0.0000]|x|\allowbreak\token[0.0000]| =|\allowbreak\token[0.0000]| -|\allowbreak\token[0.0001]|3|\allowbreak\token[0.4754]|\|\allowbreak\token[0.0405]|),|\allowbreak\token[0.2643]| or|\allowbreak\token[0.0000]| \|\allowbreak\token[0.0000]|(-|\allowbreak\token[0.0000]|\|\allowbreak\token[0.0000]|frac|\allowbreak\token[0.0000]|{|\allowbreak\token[0.0001]|9|\allowbreak\token[0.0000]|}{|\allowbreak\token[0.0000]|8|\allowbreak\token[0.0000]|}|\allowbreak\token[0.0000]|x|\allowbreak\token[0.0000]| =|\allowbreak\token[0.0000]| -|\allowbreak\token[0.0000]|3|\allowbreak\token[0.0000]|\|\allowbreak\token[0.1479]|).|\allowbreak\token[1.1154]| So|\allowbreak\token[0.3653]|,|\allowbreak\token[0.0000]| \(|\allowbreak\token[0.0000]|x|\allowbreak\token[0.0000]| =|\allowbreak\token[0.0001]| \|\allowbreak\token[0.0000]|frac|\allowbreak\token[0.0002]|{|\allowbreak\token[0.3731]|8|\allowbreak\token[0.6471]| \|\allowbreak\token[0.0817]|cdot|\allowbreak\token[0.0000]| |\allowbreak\token[0.0008]|3|\allowbreak\token[0.0000]|}{|\allowbreak\token[0.0000]|9|\allowbreak\token[0.0000]|}|\allowbreak\token[0.0000]| =|\allowbreak\token[0.0000]| \|\allowbreak\token[0.0000]|frac|\allowbreak\token[0.0000]|{|\allowbreak\token[0.0605]|8|\allowbreak\token[0.0004]|}{|\allowbreak\token[0.0000]|3|\allowbreak\token[0.0027]|}\|\allowbreak\token[0.3662]|).↩|\newline\allowbreak\token[0.0000]|  |\allowbreak\token[0.0000]| -|\allowbreak\token[0.3806]| Sub|\allowbreak\token[0.0000]|stit|\allowbreak\token[0.0000]|uting|\allowbreak\token[0.0000]| \(|\allowbreak\token[0.0000]|x|\allowbreak\token[0.0000]| =|\allowbreak\token[0.0000]| \|\allowbreak\token[0.0000]|frac|\allowbreak\token[0.0000]|{|\allowbreak\token[0.0000]|8|\allowbreak\token[0.0000]|}{|\allowbreak\token[0.0000]|3|\allowbreak\token[0.0000]|}\|\allowbreak\token[0.0000]|)|\allowbreak\token[0.6276]| into|\allowbreak\token[0.7900]| the|\allowbreak\token[0.0003]| equation|\allowbreak\token[0.4269]| \(|\allowbreak\token[0.0000]|y|\allowbreak\token[0.0000]| =|\allowbreak\token[0.0000]| -|\allowbreak\token[0.0000]|\|\allowbreak\token[0.0000]|frac|\allowbreak\token[0.0000]|{|\allowbreak\token[0.0000]|3|\allowbreak\token[0.0000]|}{|\allowbreak\token[0.2272]|2|\allowbreak\token[0.0239]|}|\allowbreak\token[0.0113]|x|\allowbreak\token[0.0000]| +|\allowbreak\token[0.0000]| |\allowbreak\token[0.0000]|6|\allowbreak\token[0.0327]|\|\allowbreak\token[0.0000]|),|\allowbreak\token[0.0000]| we|\allowbreak\token[0.0000]| get|\allowbreak\token[0.0000]| \(|\allowbreak\token[0.0000]|y|\allowbreak\token[0.0000]| =|\allowbreak\token[0.0000]| -|\allowbreak\token[0.0000]|\|\allowbreak\token[0.0000]|frac|\allowbreak\token[0.0000]|{|\allowbreak\token[0.0000]|3|\allowbreak\token[0.0000]|}{|\allowbreak\token[0.0000]|2|\allowbreak\token[0.2275]|}|\allowbreak\token[0.0007]| \|\allowbreak\token[0.1745]|cdot|\allowbreak\token[0.0000]| \|\allowbreak\token[0.0000]|frac|\allowbreak\token[0.0000]|{|\allowbreak\token[0.0000]|8|\allowbreak\token[0.0000]|}{|\allowbreak\token[0.0000]|3|\allowbreak\token[0.0000]|}|\allowbreak\token[0.0000]| +|\allowbreak\token[0.0000]| |\allowbreak\token[0.0000]|6|\allowbreak\token[0.0000]| =|\allowbreak\token[0.0001]| -|\allowbreak\token[0.0011]|4|\allowbreak\token[0.0000]| +|\allowbreak\token[0.0000]| |\allowbreak\token[0.0000]|6|\allowbreak\token[0.0000]| =|\allowbreak\token[0.0000]| |\allowbreak\token[0.0000]|2|\allowbreak\token[0.0000]|\|\allowbreak\token[0.1749]|).|\allowbreak[...]
\end{twopartbox}
\begin{twopartbox}{Cut Index 416 | Acceptance Probability 0.0682}{Block 4 Step 2}{Proposal (Rejected)}
[...]\allowbreak\token[0.0000]| -|\allowbreak\token[0.0000]| Line|\allowbreak\token[0.0000]| \(|\allowbreak\token[0.0000]|CD|\allowbreak\token[0.0000]|\)|\allowbreak\token[0.0001]| passes|\allowbreak\token[0.0000]| through|\allowbreak\token[0.0000]| points|\allowbreak\token[0.0000]| \(|\allowbreak\token[0.0001]|C|\allowbreak\token[0.0000]|(|\allowbreak\token[0.0000]|8|\allowbreak\token[0.0000]|,|\allowbreak\token[0.0001]| |\allowbreak\token[0.0000]|0|\allowbreak\token[0.0000]|)\|\allowbreak\token[0.0000]|)|\allowbreak\token[0.0000]| and|\allowbreak\token[0.0000]| \(|\allowbreak\token[0.0000]|D|\allowbreak\token[0.0000]|(|\allowbreak\token[0.0000]|0|\allowbreak\token[0.0000]|,|\allowbreak\token[0.0000]| |\allowbreak\token[0.0000]|3|\allowbreak\token[0.0000]|)\|\allowbreak\token[0.0000]|).|\allowbreak\token[0.0000]| The|\allowbreak\token[0.0000]| slope|\allowbreak\token[0.0000]| of|\allowbreak\token[0.0000]| \(|\allowbreak\token[0.0000]|CD|\allowbreak\token[0.0001]|\)|\allowbreak\token[0.0000]| is|\allowbreak\token[0.0000]| \|\allowbreak\token[0.0000]|(\|\allowbreak\token[0.0000]|frac|\allowbreak\token[0.0000]|{|\allowbreak\token[0.0214]|3|\allowbreak\token[0.0000]| -|\allowbreak\token[0.0000]| |\allowbreak\token[0.0000]|0|\allowbreak\token[0.0000]|}{|\allowbreak\token[0.0000]|0|\allowbreak\token[0.0000]| -|\allowbreak\token[0.0000]| |\allowbreak\token[0.0001]|8|\allowbreak\token[0.0000]|}|\allowbreak\token[0.0000]| =|\allowbreak\token[0.0000]| -|\allowbreak\token[0.0000]|\|\allowbreak\token[0.0000]|frac|\allowbreak\token[0.0000]|{|\allowbreak\token[0.0000]|3|\allowbreak\token[0.0000]|}{|\allowbreak\token[0.0000]|8|\allowbreak\token[0.0000]|}\|\allowbreak\token[0.0000]|).|\allowbreak\token[0.0000]| The|\allowbreak\token[0.0000]| equation|\allowbreak\token[0.0000]| of|\allowbreak\token[0.0000]| the|\allowbreak\token[0.0000]| line|\allowbreak\token[0.0001]| is|\allowbreak\token[0.0000]| \(|\allowbreak\token[0.0000]|y|\allowbreak\token[0.0126]| =|\allowbreak\token[0.0000]| -|\allowbreak\token[0.0000]|\|\allowbreak\token[0.0000]|frac|\allowbreak\token[0.0000]|{|\allowbreak\token[0.0000]|3|\allowbreak\token[0.0000]|}{|\allowbreak\token[0.0000]|8|\allowbreak\token[0.0025]|}|\allowbreak\token[0.0000]|x|\allowbreak\token[0.0000]| +|\allowbreak\token[0.0000]| |\allowbreak\token[0.0004]|3|\allowbreak\token[0.0000]|\|\allowbreak\token[0.0000]|).↩↩|\tighttracebreak\allowbreak\token[0.0000]|3|\allowbreak\token[0.0000]|.|\allowbreak\token[0.0000]| **|\allowbreak\token[0.0034]|Find|\allowbreak\token[0.0481]| the|\allowbreak\token[0.1750]| Intersection|\allowbreak\token[0.1130]| Point|\allowbreak\token[0.0005]| \(|\allowbreak\token[0.0000]|F|\allowbreak\token[0.0403]|\|\allowbreak\token[0.0002]|):|\allowbreak\token[0.0002]|**↩|\newline\allowbreak\token[0.0000]|  |\allowbreak\token[0.5168]| -|\allowbreak\token[0.1089]| Set|\allowbreak\token[0.0002]| the|\allowbreak\token[0.0022]| equations|\allowbreak\token[0.3176]| equal|\allowbreak\token[0.0001]| to|\allowbreak\token[0.0673]| each|\allowbreak\token[0.0000]| other|\allowbreak\token[0.2960]| to|\allowbreak\token[0.0053]| find|\allowbreak\token[0.0494]| the|\allowbreak\token[0.1657]| intersection|\allowbreak\token[0.4261]| point|\allowbreak\token[0.8292]|:|\allowbreak\token[0.0128]| \|\allowbreak\token[0.0003]|(-|\allowbreak\token[0.0000]|\|\allowbreak\token[0.0000]|frac|\allowbreak\token[0.0000]|{|\allowbreak\token[0.0000]|3|\allowbreak\token[0.0000]|}{|\allowbreak\token[0.0015]|2|\allowbreak\token[0.0000]|}|\allowbreak\token[0.0000]|x|\allowbreak\token[0.0000]| +|\allowbreak\token[0.0000]| |\allowbreak\token[0.0000]|6|\allowbreak\token[0.0000]| =|\allowbreak\token[0.0000]| -|\allowbreak\token[0.0000]|\|\allowbreak\token[0.0000]|frac|\allowbreak\token[0.0000]|{|\allowbreak\token[0.0000]|3|\allowbreak\token[0.0000]|}{|\allowbreak\token[0.0000]|8|\allowbreak\token[0.0000]|}|\allowbreak\token[0.0000]|x|\allowbreak\token[0.0000]| +|\allowbreak\token[0.0000]| |\allowbreak\token[0.0000]|3|\allowbreak\token[0.0000]|\|\allowbreak\token[0.5833]|).↩|\newline\allowbreak\token[0.0017]|  |\allowbreak\token[0.0091]| -|\allowbreak\token[0.7362]| Sol|\allowbreak\token[0.0000]|ving|\allowbreak\token[0.0012]| for|\allowbreak\token[0.0000]| \(|\allowbreak\token[0.0000]|x|\allowbreak\token[0.0027]|\|\allowbreak\cut{[CUT]}\allowbreak\token[0.6113]|),|\allowbreak\token[0.0001]| we|\allowbreak\token[0.0000]| get|\allowbreak\token[1.0969]| \|\allowbreak\token[0.0865]|(-|\allowbreak\token[0.0006]|\|\allowbreak\token[0.0000]|frac|\allowbreak\token[0.0000]|{|\allowbreak\token[0.0113]|3|\allowbreak\token[0.0000]|}{|\allowbreak\token[0.0000]|2|\allowbreak\token[0.0000]|}|\allowbreak\token[0.0000]|x|\allowbreak\token[0.0000]| +|\allowbreak\token[0.0021]| \|\allowbreak\token[0.0002]|frac|\allowbreak\token[0.0000]|{|\allowbreak\token[0.0000]|3|\allowbreak\token[0.0000]|}{|\allowbreak\token[0.0000]|8|\allowbreak\token[0.0000]|}|\allowbreak\token[0.0000]|x|\allowbreak\token[0.0008]| =|\allowbreak\token[0.0173]| |\allowbreak\token[0.0001]|3|\allowbreak\token[0.0000]| -|\allowbreak\token[0.0000]| |\allowbreak\token[0.0000]|6|\allowbreak\token[0.7513]|\|\allowbreak\token[0.1283]|),|\allowbreak\token[0.2588]| which|\allowbreak\token[0.0000]| simpl|\allowbreak\token[0.0000]|ifies|\allowbreak\token[0.0000]| to|\allowbreak\token[0.0001]| \|\allowbreak\token[0.0001]|(-|\allowbreak\token[0.0000]|\|\allowbreak\token[0.0000]|frac|\allowbreak\token[0.0000]|{|\allowbreak\token[0.6938]|1|\allowbreak\token[0.0064]|2|\allowbreak\token[0.0451]|}{|\allowbreak\token[0.0000]|8|\allowbreak\token[0.0000]|}|\allowbreak\token[0.0000]|x|\allowbreak\token[0.0000]| +|\allowbreak\token[0.0000]| \|\allowbreak\token[0.0000]|frac|\allowbreak\token[0.0000]|{|\allowbreak\token[0.0000]|3|\allowbreak\token[0.0000]|}{|\allowbreak\token[0.0000]|8|\allowbreak\token[0.0000]|}|\allowbreak\token[0.0000]|x|\allowbreak\token[0.0000]| =|\allowbreak\token[0.0000]| -|\allowbreak\token[0.0001]|3|\allowbreak\token[0.4754]|\|\allowbreak\token[0.0405]|),|\allowbreak\token[0.2643]| or|\allowbreak\token[0.0000]| \|\allowbreak\token[0.0000]|(-|\allowbreak\token[0.0000]|\|\allowbreak\token[0.0000]|frac|\allowbreak\token[0.0000]|{|\allowbreak\token[0.0001]|9|\allowbreak\token[0.0000]|}{|\allowbreak\token[0.0000]|8|\allowbreak\token[0.0000]|}|\allowbreak\token[0.0000]|x|\allowbreak\token[0.0000]| =|\allowbreak\token[0.0000]| -|\allowbreak\token[0.0000]|3|\allowbreak\token[0.0000]|\|\allowbreak\token[0.1479]|).|\allowbreak\token[1.1154]| So|\allowbreak\token[0.3653]|,|\allowbreak\token[0.0000]| \(|\allowbreak\token[0.0000]|x|\allowbreak\token[0.0000]| =|\allowbreak\token[0.0001]| \|\allowbreak\token[0.0000]|frac|\allowbreak\token[0.0002]|{|\allowbreak\token[0.3731]|8|\allowbreak\token[0.6471]| \|\allowbreak\token[0.0817]|cdot|\allowbreak\token[0.0000]| |\allowbreak\token[0.0008]|3|\allowbreak\token[0.0000]|}{|\allowbreak\token[0.0000]|9|\allowbreak\token[0.0000]|}|\allowbreak\token[0.0000]| =|\allowbreak\token[0.0000]| \|\allowbreak\token[0.0000]|frac|\allowbreak\token[0.0000]|{|\allowbreak\token[0.0605]|8|\allowbreak\token[0.0004]|}{|\allowbreak\token[0.0000]|3|\allowbreak\token[0.0027]|}\|\allowbreak\token[0.3662]|).↩|\newline\allowbreak\token[0.0000]|  |\allowbreak\token[0.0000]| -|\allowbreak\token[0.3806]| Sub|\allowbreak\token[0.0000]|stit|\allowbreak\token[0.0000]|uting|\allowbreak\token[0.0000]| \(|\allowbreak\token[0.0000]|x|\allowbreak\token[0.0000]| =|\allowbreak\token[0.0000]| \|\allowbreak\token[0.0000]|frac|\allowbreak\token[0.0000]|{|\allowbreak\token[0.0000]|8|\allowbreak\token[0.0000]|}{|\allowbreak\token[0.0000]|3|\allowbreak\token[0.0000]|}\|\allowbreak\token[0.0000]|)|\allowbreak\token[0.6276]| into|\allowbreak\token[0.7900]| the|\allowbreak\token[0.0003]| equation|\allowbreak\token[0.4269]| \(|\allowbreak\token[0.0000]|y|\allowbreak\token[0.0000]| =|\allowbreak\token[0.0000]| -|\allowbreak\token[0.0000]|\|\allowbreak\token[0.0000]|frac|\allowbreak\token[0.0000]|{|\allowbreak\token[0.0000]|3|\allowbreak\token[0.0000]|}{|\allowbreak\token[0.2272]|2|\allowbreak\token[0.0239]|}|\allowbreak\token[0.0113]|x|\allowbreak\token[0.0000]| +|\allowbreak\token[0.0000]| |\allowbreak\token[0.0000]|6|\allowbreak\token[0.0327]|\|\allowbreak\token[0.0000]|),|\allowbreak\token[0.0000]| we|\allowbreak\token[0.0000]| get|\allowbreak\token[0.0000]| \(|\allowbreak\token[0.0000]|y|\allowbreak\token[0.0000]| =|\allowbreak\token[0.0000]| -|\allowbreak\token[0.0000]|\|\allowbreak\token[0.0000]|frac|\allowbreak[...]
\tcblower
\token[0.6193]|),|\allowbreak\token[0.0002]| we|\allowbreak\token[0.0000]| get|\allowbreak\token[1.0969]| \|\allowbreak\token[0.0865]|(-|\allowbreak\token[0.0006]|\|\allowbreak\token[0.0000]|frac|\allowbreak\token[0.0000]|{|\allowbreak\token[0.0113]|3|\allowbreak\token[0.0000]|}{|\allowbreak\token[0.0000]|2|\allowbreak\token[0.0000]|}|\allowbreak\token[0.0000]|x|\allowbreak\token[0.0000]| +|\allowbreak\token[0.0021]| \|\allowbreak\token[0.0002]|frac|\allowbreak\token[0.0000]|{|\allowbreak\token[0.0000]|3|\allowbreak\token[0.0000]|}{|\allowbreak\token[0.0000]|8|\allowbreak\token[0.0000]|}|\allowbreak\token[0.0000]|x|\allowbreak\token[0.0008]| =|\allowbreak\token[0.0173]| |\allowbreak\token[0.0001]|3|\allowbreak\token[0.0000]| -|\allowbreak\token[0.0000]| |\allowbreak\token[0.0000]|6|\allowbreak\token[0.7513]|\|\allowbreak\token[0.1315]|),|\allowbreak\token[0.2558]| which|\allowbreak\token[0.0000]| simpl|\allowbreak\token[0.0000]|ifies|\allowbreak\token[0.0000]| to|\allowbreak\token[0.0001]| \|\allowbreak\token[0.0001]|(-|\allowbreak\token[0.0000]|\|\allowbreak\token[0.0000]|frac|\allowbreak\token[0.0000]|{|\allowbreak\token[0.6938]|9|\allowbreak\token[0.0040]|}{|\allowbreak\token[0.0005]|8|\allowbreak\token[0.0000]|}|\allowbreak\token[0.0000]|x|\allowbreak\token[0.0000]| =|\allowbreak\token[0.0000]| -|\allowbreak\token[0.0001]|3|\allowbreak\token[0.1909]|\|\allowbreak\token[0.6861]|).|\allowbreak\token[0.8980]| So|\allowbreak\token[0.3395]|,|\allowbreak\token[0.0000]| \(|\allowbreak\token[0.0000]|x|\allowbreak\token[0.0000]| =|\allowbreak\token[0.0003]| \|\allowbreak\token[0.0000]|frac|\allowbreak\token[0.0001]|{|\allowbreak\token[0.4169]|8|\allowbreak\token[0.6280]|}{|\allowbreak\token[0.0000]|3|\allowbreak\token[0.0025]|}\|\allowbreak\token[0.4231]|).↩|\newline\allowbreak\token[0.0000]|  |\allowbreak\token[0.0000]| -|\allowbreak\token[0.5409]| Sub|\allowbreak\token[0.0000]|stit|\allowbreak\token[0.0000]|uting|\allowbreak\token[0.0001]| \(|\allowbreak\token[0.0000]|x|\allowbreak\token[0.0003]| =|\allowbreak\token[0.0000]| \|\allowbreak\token[0.0000]|frac|\allowbreak\token[0.0000]|{|\allowbreak\token[0.0000]|8|\allowbreak\token[0.0000]|}{|\allowbreak\token[0.0000]|3|\allowbreak\token[0.0000]|}\|\allowbreak\token[0.0000]|)|\allowbreak\token[0.6276]| into|\allowbreak\token[0.8594]| \(|\allowbreak\token[0.0000]|y|\allowbreak\token[0.0000]| =|\allowbreak\token[0.0000]| -|\allowbreak\token[0.0000]|\|\allowbreak\token[0.0000]|frac|\allowbreak\token[0.0000]|{|\allowbreak\token[0.0000]|3|\allowbreak\token[0.0000]|}{|\allowbreak\token[0.1453]|2|\allowbreak\token[0.0082]|}|\allowbreak\token[0.0026]|x|\allowbreak\token[0.0000]| +|\allowbreak\token[0.0000]| |\allowbreak\token[0.0000]|6|\allowbreak\token[0.1094]|\|\allowbreak\token[0.0000]|),|\allowbreak\token[0.0000]| we|\allowbreak\token[0.0001]| get|\allowbreak\token[0.0000]| \(|\allowbreak\token[0.0000]|y|\allowbreak\token[0.0000]| =|\allowbreak\token[0.0447]| -|\allowbreak\token[0.0000]|\|\allowbreak\token[0.0000]|frac|\allowbreak\token[0.0000]|{|\allowbreak\token[0.0000]|3|\allowbreak\token[0.0000]|}{|\allowbreak\token[0.0000]|2|\allowbreak\token[0.6634]|}|\allowbreak\token[0.0113]| \|\allowbreak\token[0.4202]|cdot|\allowbreak\token[0.0000]| \|\allowbreak\token[0.0000]|frac|\allowbreak\token[0.0000]|{|\allowbreak\token[0.0000]|8|\allowbreak\token[0.0000]|}{|\allowbreak\token[0.0000]|3|\allowbreak\token[0.0000]|}|\allowbreak\token[0.0000]| +|\allowbreak\token[0.0000]| |\allowbreak\token[0.0000]|6|\allowbreak\token[0.0000]| =|\allowbreak\token[0.0327]| -|\allowbreak\token[0.0006]|4|\allowbreak\token[0.0000]| +|\allowbreak\token[0.0000]| |\allowbreak\token[0.0000]|6|\allowbreak\token[0.0000]| =|\allowbreak\token[0.0000]| |\allowbreak\token[0.0000]|2|\allowbreak\token[0.0000]|\|\allowbreak\token[0.2915]|).|\allowbreak\token[0.3383]| So|\allowbreak\token[0.0000]|,|\allowbreak\token[0.4304]| \(|\allowbreak\token[0.0265]|F|\allowbreak\token[0.0634]|\)|\allowbreak\token[0.0000]| is|\allowbreak\token[0.0000]| \|\allowbreak\token[0.1596]|(\|\allowbreak\token[0.0000]|left|\allowbreak\token[0.0004]|(\|\allowbreak\token[0.0000]|frac|\allowbreak\token[0.0000]|{|\allowbreak[...]
\end{twopartbox}
\begin{twopartbox}{Cut Index 590 | Acceptance Probability 0.9991}{Block 4 Step 3}{Proposal (Accepted)}
[...]\allowbreak\token[0.2588]| which|\allowbreak\token[0.0000]| simpl|\allowbreak\token[0.0000]|ifies|\allowbreak\token[0.0000]| to|\allowbreak\token[0.0001]| \|\allowbreak\token[0.0001]|(-|\allowbreak\token[0.0000]|\|\allowbreak\token[0.0000]|frac|\allowbreak\token[0.0000]|{|\allowbreak\token[0.6938]|1|\allowbreak\token[0.0064]|2|\allowbreak\token[0.0451]|}{|\allowbreak\token[0.0000]|8|\allowbreak\token[0.0000]|}|\allowbreak\token[0.0000]|x|\allowbreak\token[0.0000]| +|\allowbreak\token[0.0000]| \|\allowbreak\token[0.0000]|frac|\allowbreak\token[0.0000]|{|\allowbreak\token[0.0000]|3|\allowbreak\token[0.0000]|}{|\allowbreak\token[0.0000]|8|\allowbreak\token[0.0000]|}|\allowbreak\token[0.0000]|x|\allowbreak\token[0.0000]| =|\allowbreak\token[0.0000]| -|\allowbreak\token[0.0001]|3|\allowbreak\token[0.4754]|\|\allowbreak\token[0.0405]|),|\allowbreak\token[0.2643]| or|\allowbreak\token[0.0000]| \|\allowbreak\token[0.0000]|(-|\allowbreak\token[0.0000]|\|\allowbreak\token[0.0000]|frac|\allowbreak\token[0.0000]|{|\allowbreak\token[0.0001]|9|\allowbreak\token[0.0000]|}{|\allowbreak\token[0.0000]|8|\allowbreak\token[0.0000]|}|\allowbreak\token[0.0000]|x|\allowbreak\token[0.0000]| =|\allowbreak\token[0.0000]| -|\allowbreak\token[0.0000]|3|\allowbreak\token[0.0000]|\|\allowbreak\token[0.1479]|).|\allowbreak\token[1.1154]| So|\allowbreak\token[0.3653]|,|\allowbreak\token[0.0000]| \(|\allowbreak\token[0.0000]|x|\allowbreak\token[0.0000]| =|\allowbreak\token[0.0001]| \|\allowbreak\token[0.0000]|frac|\allowbreak\token[0.0002]|{|\allowbreak\token[0.3731]|8|\allowbreak\token[0.6471]| \|\allowbreak\token[0.0817]|cdot|\allowbreak\token[0.0000]| |\allowbreak\token[0.0008]|3|\allowbreak\token[0.0000]|}{|\allowbreak\token[0.0000]|9|\allowbreak\token[0.0000]|}|\allowbreak\token[0.0000]| =|\allowbreak\token[0.0000]| \|\allowbreak\token[0.0000]|frac|\allowbreak\token[0.0000]|{|\allowbreak\token[0.0605]|8|\allowbreak\token[0.0004]|}{|\allowbreak\token[0.0000]|3|\allowbreak\token[0.0027]|}\|\allowbreak\token[0.3662]|).↩|\newline\allowbreak\token[0.0000]|  |\allowbreak\token[0.0000]| -|\allowbreak\token[0.3806]| Sub|\allowbreak\token[0.0000]|stit|\allowbreak\token[0.0000]|uting|\allowbreak\token[0.0000]| \(|\allowbreak\token[0.0000]|x|\allowbreak\token[0.0000]| =|\allowbreak\token[0.0000]| \|\allowbreak\token[0.0000]|frac|\allowbreak\token[0.0000]|{|\allowbreak\token[0.0000]|8|\allowbreak\token[0.0000]|}{|\allowbreak\token[0.0000]|3|\allowbreak\token[0.0000]|}\|\allowbreak\token[0.0000]|)|\allowbreak\token[0.6276]| into|\allowbreak\token[0.7900]| the|\allowbreak\token[0.0003]| equation|\allowbreak\token[0.4269]| \(|\allowbreak\token[0.0000]|y|\allowbreak\token[0.0000]| =|\allowbreak\token[0.0000]| -|\allowbreak\token[0.0000]|\|\allowbreak\token[0.0000]|frac|\allowbreak\token[0.0000]|{|\allowbreak\token[0.0000]|3|\allowbreak\token[0.0000]|}{|\allowbreak\token[0.2272]|2|\allowbreak\token[0.0239]|}|\allowbreak\token[0.0113]|x|\allowbreak\token[0.0000]| +|\allowbreak\token[0.0000]| |\allowbreak\token[0.0000]|6|\allowbreak\token[0.0327]|\|\allowbreak\token[0.0000]|),|\allowbreak\token[0.0000]| we|\allowbreak\token[0.0000]| get|\allowbreak\token[0.0000]| \(|\allowbreak\token[0.0000]|y|\allowbreak\token[0.0000]| =|\allowbreak\token[0.0000]| -|\allowbreak\token[0.0000]|\|\allowbreak\token[0.0000]|frac|\allowbreak\token[0.0000]|{|\allowbreak\token[0.0000]|3|\allowbreak\token[0.0000]|}{|\allowbreak\token[0.0000]|2|\allowbreak\token[0.2275]|}|\allowbreak\token[0.0007]| \|\allowbreak\token[0.1745]|cdot|\allowbreak\token[0.0000]| \|\allowbreak\token[0.0000]|frac|\allowbreak\token[0.0000]|{|\allowbreak\token[0.0000]|8|\allowbreak\token[0.0000]|}{|\allowbreak\token[0.0000]|3|\allowbreak\token[0.0000]|}|\allowbreak\token[0.0000]| +|\allowbreak\token[0.0000]| |\allowbreak\token[0.0000]|6|\allowbreak\token[0.0000]| =|\allowbreak\token[0.0001]| -|\allowbreak\token[0.0011]|4|\allowbreak\token[0.0000]| +|\allowbreak\token[0.0000]| |\allowbreak\token[0.0000]|6|\allowbreak\token[0.0000]| =|\allowbreak\token[0.0000]| |\allowbreak\token[0.0000]|2|\allowbreak\token[0.0000]|\|\allowbreak\token[0.1749]|).|\allowbreak\token[0.2147]| So|\allowbreak\token[0.0000]|,|\allowbreak\cut{[CUT]}\allowbreak\token[0.7686]| \(|\allowbreak\token[0.0082]|F|\allowbreak\token[0.0515]|\)|\allowbreak\token[0.0000]| is|\allowbreak\token[0.0000]| \|\allowbreak\token[0.1326]|(\|\allowbreak\token[0.0000]|left|\allowbreak\token[0.0003]|(\|\allowbreak\token[0.0000]|frac|\allowbreak\token[0.0000]|{|\allowbreak\token[0.0000]|8|\allowbreak\token[0.0000]|}{|\allowbreak\token[0.0000]|3|\allowbreak\token[0.0000]|},|\allowbreak\token[0.0000]| |\allowbreak\token[0.0000]|2|\allowbreak\token[0.0000]|\|\allowbreak\token[0.0008]|right|\allowbreak\token[0.0000]|)\|\allowbreak\token[0.0001]|).↩↩|\tighttracebreak\allowbreak\token[0.0000]|4|\allowbreak\token[0.0000]|.|\allowbreak\token[0.0000]| **|\allowbreak\token[0.0802]|Calculate|\allowbreak\token[0.0000]| the|\allowbreak\token[0.0002]| Area|\allowbreak\token[0.0001]| of|\allowbreak\token[0.0001]| Quadr|\allowbreak\token[0.0000]|ilateral|\allowbreak\token[0.0000]| \(|\allowbreak\token[0.0000]|DB|\allowbreak\token[0.0000]|EF|\allowbreak\token[0.0003]|\|\allowbreak\token[0.0004]|):|\allowbreak\token[0.0000]|**↩|\newline\allowbreak\token[0.0000]|  |\allowbreak\token[0.0002]| -|\allowbreak\token[1.0683]| We|\allowbreak\token[0.0606]| can|\allowbreak\token[0.0156]| use|\allowbreak\token[0.0000]| the|\allowbreak\token[0.5757]| sho|\allowbreak\token[0.0113]|el|\allowbreak\token[0.0000]|ace|\allowbreak\token[0.0001]| formula|\allowbreak\token[0.0401]| to|\allowbreak\token[0.0495]| find|\allowbreak\token[0.0000]| the|\allowbreak\token[0.0000]| area|\allowbreak\token[0.0363]| of|\allowbreak\token[0.4834]| quadr|\allowbreak\token[0.0000]|ilateral|\allowbreak\token[0.0020]| \(|\allowbreak\token[0.0003]|DB|\allowbreak\token[0.0001]|EF|\allowbreak\token[0.4194]|\|\allowbreak\token[0.3469]|).|\allowbreak\token[0.0001]| The|\allowbreak\token[0.5457]| coordinates|\allowbreak\token[0.1324]| of|\allowbreak\token[0.5111]| the|\allowbreak\token[0.2686]| vertices|\allowbreak\token[0.0861]| are|\allowbreak\token[0.0174]| \(|\allowbreak\token[0.0001]|D|\allowbreak\token[0.0001]|(|\allowbreak\token[0.0000]|0|\allowbreak\token[0.0000]|,|\allowbreak\token[0.0019]| |\allowbreak\token[0.0000]|3|\allowbreak\token[0.0000]|)\|\allowbreak\token[0.0000]|),|\allowbreak\token[0.0000]| \(|\allowbreak\token[0.0042]|B|\allowbreak\token[0.0000]|(|\allowbreak\token[0.0000]|0|\allowbreak\token[0.0000]|,|\allowbreak\token[0.0001]| |\allowbreak\token[0.0000]|0|\allowbreak\token[0.0000]|)\|\allowbreak\token[0.0000]|),|\allowbreak\token[0.0000]| \(|\allowbreak\token[0.0003]|E|\allowbreak\token[0.0000]|(|\allowbreak\token[0.0000]|4|\allowbreak\token[0.0000]|,|\allowbreak\token[0.0000]| |\allowbreak\token[0.0000]|0|\allowbreak\token[0.0000]|)\|\allowbreak\token[0.0000]|),|\allowbreak\token[0.0004]| and|\allowbreak\token[0.0000]| \(|\allowbreak\token[0.0000]|F|\allowbreak\token[0.0002]|\|\allowbreak\token[0.0004]|left|\allowbreak\token[0.0000]|(\|\allowbreak\token[0.0000]|frac|\allowbreak\token[0.0000]|{|\allowbreak\token[0.0000]|8|\allowbreak\token[0.0000]|}{|\allowbreak\token[0.0000]|3|\allowbreak\token[0.0000]|},|\allowbreak\token[0.0000]| |\allowbreak\token[0.0000]|2|\allowbreak\token[0.0000]|\|\allowbreak\token[0.0013]|right|\allowbreak\token[0.0000]|)\|\allowbreak\token[0.0269]|).↩|\newline\allowbreak\token[0.0002]|  |\allowbreak\token[0.0038]| -|\allowbreak\token[0.6802]| The|\allowbreak\token[0.1271]| sho|\allowbreak\token[0.0000]|el|\allowbreak\token[0.0000]|ace|\allowbreak\token[0.0000]| formula|\allowbreak\token[0.1915]| for|\allowbreak\token[0.0607]| the|\allowbreak\token[0.1606]| area|\allowbreak\token[0.0411]| of|\allowbreak\token[0.0001]| a|\allowbreak\token[0.1094]| polygon|\allowbreak\token[0.2370]| with|\allowbreak\token[0.0000]| vertices|\allowbreak\token[0.0000]| \|\allowbreak\token[0.0000]|((|\allowbreak\token[0.0000]|x|\allowbreak\token[0.0000]|_|\allowbreak\token[0.0000]|1|\allowbreak\token[0.0000]|,|\allowbreak\token[0.0000]| y|\allowbreak\token[0.0000]|_|\allowbreak\token[0.0000]|1|\allowbreak\token[0.0402]|),|\allowbreak\token[0.0000]| (|\allowbreak\token[0.0000]|x|\allowbreak\token[0.0000]|_|\allowbreak\token[0.0000]|2|\allowbreak\token[0.0000]|,|\allowbreak\token[0.0000]| y|\allowbreak\token[0.0001]|_|\allowbreak\token[0.0000]|2|\allowbreak\token[0.0000]|),|\allowbreak\token[0.0108]| \|\allowbreak\token[0.0001]|ld|\allowbreak[...]
\tcblower
\token[0.7684]| \(|\allowbreak\token[0.0091]|F|\allowbreak\token[0.0514]|\)|\allowbreak\token[0.0000]| is|\allowbreak\token[0.0000]| \|\allowbreak\token[0.1326]|(\|\allowbreak\token[0.0000]|left|\allowbreak\token[0.0003]|(\|\allowbreak\token[0.0000]|frac|\allowbreak\token[0.0000]|{|\allowbreak\token[0.0000]|8|\allowbreak\token[0.0000]|}{|\allowbreak\token[0.0000]|3|\allowbreak\token[0.0000]|},|\allowbreak\token[0.0000]| |\allowbreak\token[0.0000]|2|\allowbreak\token[0.0000]|\|\allowbreak\token[0.0008]|right|\allowbreak\token[0.0000]|)\|\allowbreak\token[0.0001]|).↩↩|\tighttracebreak\allowbreak\token[0.0000]|4|\allowbreak\token[0.0000]|.|\allowbreak\token[0.0000]| **|\allowbreak\token[0.0802]|Calculate|\allowbreak\token[0.0000]| the|\allowbreak\token[0.0002]| Area|\allowbreak\token[0.0001]| of|\allowbreak\token[0.0001]| Quadr|\allowbreak\token[0.0000]|ilateral|\allowbreak\token[0.0000]| \(|\allowbreak\token[0.0000]|DB|\allowbreak\token[0.0000]|EF|\allowbreak\token[0.0003]|\|\allowbreak\token[0.0004]|):|\allowbreak\token[0.0000]|**↩|\newline\allowbreak\token[0.0000]|  |\allowbreak\token[0.0002]| -|\allowbreak\token[1.0683]| We|\allowbreak\token[0.0606]| can|\allowbreak\token[0.0156]| use|\allowbreak\token[0.0000]| the|\allowbreak\token[0.5757]| sho|\allowbreak\token[0.0113]|el|\allowbreak\token[0.0000]|ace|\allowbreak\token[0.0001]| formula|\allowbreak\token[0.0401]| to|\allowbreak\token[0.0495]| find|\allowbreak\token[0.0000]| the|\allowbreak\token[0.0000]| area|\allowbreak\token[0.0363]| of|\allowbreak\token[0.4834]| quadr|\allowbreak\token[0.0000]|ilateral|\allowbreak\token[0.0020]| \(|\allowbreak\token[0.0003]|DB|\allowbreak\token[0.0001]|EF|\allowbreak\token[0.4194]|\|\allowbreak\token[0.3469]|).|\allowbreak\token[0.0001]| The|\allowbreak\token[0.5457]| coordinates|\allowbreak\token[0.1324]| of|\allowbreak\token[0.5111]| the|\allowbreak\token[0.2686]| vertices|\allowbreak\token[0.0861]| are|\allowbreak\token[0.0174]| \(|\allowbreak\token[0.0001]|D|\allowbreak\token[0.0001]|(|\allowbreak\token[0.0000]|0|\allowbreak\token[0.0000]|,|\allowbreak\token[0.0019]| |\allowbreak\token[0.0000]|3|\allowbreak\token[0.0000]|)\|\allowbreak\token[0.0000]|),|\allowbreak\token[0.0000]| \(|\allowbreak\token[0.0042]|B|\allowbreak\token[0.0000]|(|\allowbreak\token[0.0000]|0|\allowbreak\token[0.0000]|,|\allowbreak\token[0.0001]| |\allowbreak\token[0.0000]|0|\allowbreak\token[0.0000]|)\|\allowbreak\token[0.0000]|),|\allowbreak\token[0.0000]| \(|\allowbreak\token[0.0003]|E|\allowbreak\token[0.0000]|(|\allowbreak\token[0.0000]|4|\allowbreak\token[0.0000]|,|\allowbreak\token[0.0000]| |\allowbreak\token[0.0000]|0|\allowbreak\token[0.0000]|)\|\allowbreak\token[0.0000]|),|\allowbreak\token[0.0004]| and|\allowbreak\token[0.0000]| \(|\allowbreak\token[0.0000]|F|\allowbreak\token[0.0002]|\|\allowbreak\token[0.0004]|left|\allowbreak\token[0.0000]|(\|\allowbreak\token[0.0000]|frac|\allowbreak\token[0.0000]|{|\allowbreak\token[0.0000]|8|\allowbreak\token[0.0000]|}{|\allowbreak\token[0.0000]|3|\allowbreak\token[0.0000]|},|\allowbreak\token[0.0000]| |\allowbreak\token[0.0000]|2|\allowbreak\token[0.0000]|\|\allowbreak\token[0.0013]|right|\allowbreak\token[0.0000]|)\|\allowbreak\token[0.0269]|).↩|\newline\allowbreak\token[0.0002]|  |\allowbreak\token[0.0038]| -|\allowbreak\token[0.6802]| The|\allowbreak\token[0.1271]| sho|\allowbreak\token[0.0000]|el|\allowbreak\token[0.0000]|ace|\allowbreak\token[0.0000]| formula|\allowbreak\token[0.1915]| for|\allowbreak\token[0.0607]| the|\allowbreak\token[0.1606]| area|\allowbreak\token[0.0411]| of|\allowbreak\token[0.0001]| a|\allowbreak\token[0.1094]| polygon|\allowbreak\token[0.2370]| with|\allowbreak\token[0.0000]| vertices|\allowbreak\token[0.0000]| \|\allowbreak\token[0.0000]|((|\allowbreak\token[0.0000]|x|\allowbreak\token[0.0000]|_|\allowbreak\token[0.0000]|1|\allowbreak\token[0.0000]|,|\allowbreak\token[0.0000]| y|\allowbreak\token[0.0000]|_|\allowbreak\token[0.0000]|1|\allowbreak\token[0.0402]|),|\allowbreak\token[0.0000]| (|\allowbreak\token[0.0000]|x|\allowbreak\token[0.0000]|_|\allowbreak\token[0.0000]|2|\allowbreak\token[0.0000]|,|\allowbreak\token[0.0000]| y|\allowbreak\token[0.0001]|_|\allowbreak\token[0.0000]|2|\allowbreak\token[0.0000]|),|\allowbreak\token[0.0108]| \|\allowbreak\token[0.0001]|ld|\allowbreak[...]
\end{twopartbox}
\begin{twopartbox}{Cut Index 159 | Acceptance Probability 0.9792}{Block 5 Step 1}{Proposal (Accepted)}
[...]\allowbreak\token[0.3013]|DB|\allowbreak\token[0.0001]|EF|\allowbreak\token[0.6859]|\|\allowbreak\token[0.3149]|).|\allowbreak\token[1.6625]| Here|\allowbreak\token[1.0700]|'s|\allowbreak\token[0.6866]| the|\allowbreak\token[0.1679]| step|\allowbreak\token[0.0002]|-by|\allowbreak\token[0.0001]|-step|\allowbreak\token[1.4676]| approach|\allowbreak\token[0.1317]|:↩↩|\tighttracebreak\allowbreak\token[0.0016]|1|\allowbreak\token[0.0000]|.|\allowbreak\token[0.4060]| **|\allowbreak\token[0.4029]|Ident|\allowbreak\token[0.0010]|ify|\allowbreak\token[0.6278]| Coordinates|\allowbreak\token[1.0343]| of|\allowbreak\token[0.7862]| Points|\allowbreak\token[1.0859]|:|\allowbreak\token[0.0014]|**↩|\newline\allowbreak\token[0.2277]|  |\allowbreak\token[0.0177]| -|\allowbreak\token[1.1397]| Point|\allowbreak\token[0.0003]| \(|\allowbreak\token[0.4684]|A|\allowbreak\token[0.2856]|\)|\allowbreak\token[0.1606]| is|\allowbreak\token[0.6652]| \|\allowbreak\token[0.0001]|((|\allowbreak\token[0.0000]|0|\allowbreak\token[0.0000]|,|\allowbreak\token[0.0901]| |\allowbreak\token[0.0000]|6|\allowbreak\token[0.0000]|)\|\allowbreak\token[0.3752]|).↩|\newline\allowbreak\token[0.0000]|  |\allowbreak\token[0.0000]| -|\allowbreak\token[0.0001]| Point|\allowbreak\token[0.0000]| \(|\allowbreak\token[0.0005]|B|\allowbreak\token[0.0000]|\)|\allowbreak\token[0.0000]| is|\allowbreak\token[0.0001]| \|\allowbreak\token[0.0000]|((|\allowbreak\token[0.0000]|0|\allowbreak\token[0.0000]|,|\allowbreak\token[0.0001]| |\allowbreak\token[0.0000]|0|\allowbreak\token[0.0000]|)\|\allowbreak\token[0.0002]|).↩|\newline\allowbreak\token[0.0000]|  |\allowbreak\token[0.0000]| -|\allowbreak\token[0.0000]| Point|\allowbreak\token[0.0000]| \(|\allowbreak\token[0.0001]|C|\allowbreak\token[0.0000]|\)|\allowbreak\token[0.0000]| is|\allowbreak\token[0.0000]| \|\allowbreak\token[0.0001]|((|\allowbreak\token[0.0000]|8|\allowbreak\token[0.0000]|,|\allowbreak\token[0.0000]| |\allowbreak\token[0.0000]|0|\allowbreak\token[0.0000]|)\|\allowbreak\token[0.0820]|).↩|\newline\allowbreak\token[0.0053]|  |\allowbreak\token[0.0000]| -|\allowbreak\token[0.7278]| Point|\allowbreak\token[0.0000]| \(|\allowbreak\token[0.0022]|D|\allowbreak\token[0.4194]|\)|\allowbreak\token[0.0197]| is|\allowbreak\token[0.0001]| the|\allowbreak\token[0.0000]| midpoint|\allowbreak\token[0.0000]| of|\allowbreak\token[0.0496]| \|\allowbreak\token[0.0004]|(\|\allowbreak\token[0.0000]|over|\allowbreak\token[0.0000]|line|\allowbreak\token[0.0000]|{|\allowbreak\token[0.0000]|AB|\allowbreak\token[0.0001]|}\|\allowbreak\token[0.0147]|),|\allowbreak\token[0.0025]| so|\allowbreak\token[0.4478]| \(|\allowbreak\token[0.0000]|D|\allowbreak\token[0.6075]|\)|\allowbreak\token[0.0034]| is|\allowbreak\token[0.0002]| \|\allowbreak\token[0.0905]|((|\allowbreak\token[0.0102]|0|\allowbreak\token[0.0000]|,|\allowbreak\token[0.0240]| |\allowbreak\token[0.0000]|3|\allowbreak\token[0.0000]|)\|\allowbreak\token[0.0612]|).↩|\newline\allowbreak\token[0.0000]|  |\allowbreak\token[0.0000]| -|\allowbreak\token[0.0000]| Point|\allowbreak\token[0.0000]| \(|\allowbreak\token[0.0002]|E|\allowbreak\token[0.0001]|\)|\allowbreak\token[0.0000]| is|\allowbreak\token[0.0000]| the|\allowbreak\token[0.0000]| midpoint|\allowbreak\token[0.0000]| of|\allowbreak\token[0.0000]| \|\allowbreak\token[0.0000]|(\|\allowbreak\token[0.0002]|over|\allowbreak\token[0.0001]|line|\allowbreak\token[0.0000]|{|\allowbreak\token[0.0002]|BC|\allowbreak\token[0.0000]|}\|\allowbreak\token[0.0000]|),|\allowbreak\token[0.0000]| so|\allowbreak\token[0.0000]| \(|\allowbreak\token[0.0001]|E|\allowbreak\token[0.0002]|\)|\allowbreak\token[0.0000]| is|\allowbreak\token[0.0000]| \|\allowbreak\token[0.0007]|((|\allowbreak\token[0.0000]|4|\allowbreak\token[0.0000]|,|\allowbreak\token[0.0000]| |\allowbreak\token[0.0001]|0|\allowbreak\token[0.0000]|)\|\allowbreak\token[0.3933]|).↩|\newline\allowbreak\token[0.0175]|  |\allowbreak\token[0.0000]| -|\allowbreak\token[0.1774]| Point|\allowbreak\token[0.0000]| \(|\allowbreak\token[0.0007]|F|\allowbreak\token[0.0040]|\)|\allowbreak\token[0.2097]| is|\allowbreak\token[1.1429]| the|\allowbreak\token[0.5879]| intersection|\allowbreak\token[0.0027]| of|\allowbreak\token[1.0073]| \|\allowbreak\token[0.0017]|(\|\allowbreak\token[0.0001]|over|\allowbreak\token[0.0006]|line|\allowbreak\token[0.0001]|{|\allowbreak\cut{[CUT]}\allowbreak\token[1.0740]|AE|\allowbreak\token[0.0002]|}\|\allowbreak\token[0.0000]|)|\allowbreak\token[0.0032]| and|\allowbreak\token[0.0049]| \|\allowbreak\token[0.0000]|(\|\allowbreak\token[0.0001]|over|\allowbreak\token[0.0002]|line|\allowbreak\token[0.0000]|{|\allowbreak\token[0.2170]|CD|\allowbreak\token[0.0001]|}\|\allowbreak\token[0.6869]|).↩↩|\tighttracebreak\allowbreak\token[0.0002]|2|\allowbreak\token[0.0000]|.|\allowbreak\token[0.0000]| **|\allowbreak\token[0.1827]|Find|\allowbreak\token[0.2494]| the|\allowbreak\token[1.1302]| Equ|\allowbreak\token[0.0000]|ations|\allowbreak\token[0.0000]| of|\allowbreak\token[0.1410]| Lines|\allowbreak\token[0.7289]| \(|\allowbreak\token[0.0039]|AE|\allowbreak\token[0.0004]|\)|\allowbreak\token[0.0000]| and|\allowbreak\token[0.0001]| \(|\allowbreak\token[0.0000]|CD|\allowbreak\token[0.0001]|\|\allowbreak\token[0.0002]|):|\allowbreak\token[0.0038]|**↩|\newline\allowbreak\token[0.0000]|  |\allowbreak\token[0.0000]| -|\allowbreak\token[0.5979]| Line|\allowbreak\token[0.0000]| \(|\allowbreak\token[0.0005]|AE|\allowbreak\token[0.0903]|\)|\allowbreak\token[0.0181]| passes|\allowbreak\token[0.0000]| through|\allowbreak\token[0.0143]| points|\allowbreak\token[0.0265]| \(|\allowbreak\token[0.0001]|A|\allowbreak\token[0.0035]|(|\allowbreak\token[0.0000]|0|\allowbreak\token[0.0000]|,|\allowbreak\token[0.0101]| |\allowbreak\token[0.0000]|6|\allowbreak\token[0.0000]|)\|\allowbreak\token[0.0000]|)|\allowbreak\token[0.0000]| and|\allowbreak\token[0.0000]| \(|\allowbreak\token[0.0001]|E|\allowbreak\token[0.0000]|(|\allowbreak\token[0.0000]|4|\allowbreak\token[0.0000]|,|\allowbreak\token[0.0000]| |\allowbreak\token[0.0000]|0|\allowbreak\token[0.0000]|)\|\allowbreak\token[0.5457]|).|\allowbreak\token[0.1948]| The|\allowbreak\token[0.0140]| slope|\allowbreak\token[0.4572]| of|\allowbreak\token[0.0608]| \(|\allowbreak\token[0.0000]|AE|\allowbreak\token[0.0000]|\)|\allowbreak\token[0.0000]| is|\allowbreak\token[0.0001]| \|\allowbreak\token[0.0076]|(\|\allowbreak\token[0.0000]|frac|\allowbreak\token[0.0000]|{|\allowbreak\token[0.0173]|0|\allowbreak\token[0.6466]| -|\allowbreak\token[0.0000]| |\allowbreak\token[0.0000]|6|\allowbreak\token[0.0000]|}{|\allowbreak\token[0.0000]|4|\allowbreak\token[0.0000]| -|\allowbreak\token[0.0000]| |\allowbreak\token[0.0000]|0|\allowbreak\token[0.0000]|}|\allowbreak\token[0.0000]| =|\allowbreak\token[0.0000]| -|\allowbreak\token[0.0004]|\|\allowbreak\token[0.0000]|frac|\allowbreak\token[0.0000]|{|\allowbreak\token[0.0011]|3|\allowbreak\token[0.0000]|}{|\allowbreak\token[0.0000]|2|\allowbreak\token[0.0000]|}\|\allowbreak\token[0.0819]|).|\allowbreak\token[0.6056]| The|\allowbreak\token[0.0002]| equation|\allowbreak\token[0.0155]| of|\allowbreak\token[0.5775]| the|\allowbreak\token[0.0000]| line|\allowbreak\token[0.8775]| is|\allowbreak\token[0.0018]| \(|\allowbreak\token[0.0005]|y|\allowbreak\token[0.1909]| =|\allowbreak\token[0.0002]| -|\allowbreak\token[0.0000]|\|\allowbreak\token[0.0000]|frac|\allowbreak\token[0.0000]|{|\allowbreak\token[0.0000]|3|\allowbreak\token[0.0000]|}{|\allowbreak\token[0.0000]|2|\allowbreak\token[0.0002]|}|\allowbreak\token[0.0001]|x|\allowbreak\token[0.0000]| +|\allowbreak\token[0.0001]| |\allowbreak\token[0.0000]|6|\allowbreak\token[0.0000]|\|\allowbreak\token[0.0001]|).↩|\newline\allowbreak\token[0.0000]|  |\allowbreak\token[0.0000]| -|\allowbreak\token[0.0000]| Line|\allowbreak\token[0.0000]| \(|\allowbreak\token[0.0000]|CD|\allowbreak\token[0.0000]|\)|\allowbreak\token[0.0001]| passes|\allowbreak\token[0.0000]| through|\allowbreak\token[0.0000]| points|\allowbreak\token[0.0000]| \(|\allowbreak\token[0.0001]|C|\allowbreak\token[0.0000]|(|\allowbreak\token[0.0000]|8|\allowbreak\token[0.0000]|,|\allowbreak\token[0.0001]| |\allowbreak\token[0.0000]|0|\allowbreak\token[0.0000]|)\|\allowbreak\token[0.0000]|)|\allowbreak\token[0.0000]| and|\allowbreak\token[0.0000]| \(|\allowbreak\token[0.0000]|D|\allowbreak\token[0.0000]|(|\allowbreak\token[0.0000]|0|\allowbreak\token[0.0000]|,|\allowbreak\token[0.0000]| |\allowbreak\token[0.0000]|3|\allowbreak\token[0.0000]|)\|\allowbreak\token[0.0000]|).|\allowbreak\token[0.0000]| The|\allowbreak\token[0.0000]| slope|\allowbreak\token[0.0000]| of|\allowbreak\token[0.0000]| \(|\allowbreak[...]
\tcblower
\token[1.0715]|AE|\allowbreak\token[0.0002]|}\|\allowbreak\token[0.0001]|)|\allowbreak\token[0.0032]| and|\allowbreak\token[0.0048]| \|\allowbreak\token[0.0000]|(\|\allowbreak\token[0.0001]|over|\allowbreak\token[0.0002]|line|\allowbreak\token[0.0000]|{|\allowbreak\token[0.2372]|CD|\allowbreak\token[0.0001]|}\|\allowbreak\token[0.6869]|).↩↩|\tighttracebreak\allowbreak\token[0.0003]|2|\allowbreak\token[0.0000]|.|\allowbreak\token[0.0000]| **|\allowbreak\token[0.1501]|Find|\allowbreak\token[0.2495]| the|\allowbreak\token[1.1147]| Equation|\allowbreak\token[0.0000]| of|\allowbreak\token[0.7128]| Line|\allowbreak\token[0.4501]| \(|\allowbreak\token[0.0092]|AE|\allowbreak\token[0.0294]|\|\allowbreak\token[0.0004]|):|\allowbreak\token[0.0001]|**↩|\newline\allowbreak\token[0.0000]|  |\allowbreak\token[0.5393]| -|\allowbreak\token[0.9401]| The|\allowbreak\token[0.1343]| slope|\allowbreak\token[0.0081]| of|\allowbreak\token[0.4686]| \(|\allowbreak\token[0.0004]|AE|\allowbreak\token[0.0000]|\)|\allowbreak\token[0.0001]| is|\allowbreak\token[0.0008]| \|\allowbreak\token[0.0070]|(\|\allowbreak\token[0.0001]|frac|\allowbreak\token[0.0003]|{|\allowbreak\token[0.6822]|6|\allowbreak\token[0.6759]|-|\allowbreak\token[0.0001]|0|\allowbreak\token[0.0000]|}{|\allowbreak\token[0.0001]|0|\allowbreak\token[0.0000]|-|\allowbreak\token[0.0043]|4|\allowbreak\token[0.0002]|}|\allowbreak\token[0.0001]| =|\allowbreak\token[0.0007]| -|\allowbreak\token[0.0069]|\|\allowbreak\token[0.0000]|frac|\allowbreak\token[0.0000]|{|\allowbreak\token[0.0042]|3|\allowbreak\token[0.0000]|}{|\allowbreak\token[0.0000]|2|\allowbreak\token[0.0000]|}\|\allowbreak\token[0.0334]|).↩|\newline\allowbreak\token[0.0000]|  |\allowbreak\token[0.0000]| -|\allowbreak\token[0.0201]| The|\allowbreak\token[0.0004]| equation|\allowbreak\token[0.0008]| of|\allowbreak\token[0.6106]| the|\allowbreak\token[0.0000]| line|\allowbreak\token[0.2897]| \(|\allowbreak\token[0.0001]|AE|\allowbreak\token[0.0003]|\)|\allowbreak\token[0.8348]| using|\allowbreak\token[0.6632]| the|\allowbreak\token[0.0042]| point|\allowbreak\token[0.0033]|-s|\allowbreak\token[0.0000]|lope|\allowbreak\token[0.0038]| form|\allowbreak\token[0.6977]| \(|\allowbreak\token[0.0126]|y|\allowbreak\token[0.0252]| -|\allowbreak\token[0.0020]| y|\allowbreak\token[0.0001]|_|\allowbreak\token[0.0001]|1|\allowbreak\token[0.0000]| =|\allowbreak\token[0.0000]| m|\allowbreak\token[0.0001]|(x|\allowbreak\token[0.0000]| -|\allowbreak\token[0.0000]| x|\allowbreak\token[0.0000]|_|\allowbreak\token[0.0000]|1|\allowbreak\token[0.0000]|)\|\allowbreak\token[0.0093]|)|\allowbreak\token[0.0487]| is|\allowbreak\token[0.6953]|:↩|\newline\allowbreak\token[0.0005]|    |\allowbreak\token[0.0007]| \|\allowbreak\token[0.0059]|[↩|\newline\allowbreak\token[0.0000]|    |\allowbreak\token[0.0000]| y|\allowbreak\token[0.5570]| -|\allowbreak\token[0.0000]| |\allowbreak\token[0.6855]|0|\allowbreak\token[0.0000]| =|\allowbreak\token[0.0000]| -|\allowbreak\token[0.0000]|\|\allowbreak\token[0.0000]|frac|\allowbreak\token[0.0000]|{|\allowbreak\token[0.0000]|3|\allowbreak\token[0.0000]|}{|\allowbreak\token[0.0000]|2|\allowbreak\token[0.0003]|}(|\allowbreak\token[0.0000]|x|\allowbreak\token[0.0000]| -|\allowbreak\token[0.0000]| |\allowbreak\token[0.0446]|4|\allowbreak\token[0.1458]|)|\allowbreak\token[0.0000]| \|\allowbreak\token[0.0219]|impl|\allowbreak\token[0.0000]|ies|\allowbreak\token[0.0000]| y|\allowbreak\token[0.0000]| =|\allowbreak\token[0.0000]| -|\allowbreak\token[0.0000]|\|\allowbreak\token[0.0000]|frac|\allowbreak\token[0.0000]|{|\allowbreak\token[0.0000]|3|\allowbreak\token[0.0000]|}{|\allowbreak\token[0.0000]|2|\allowbreak\token[0.0000]|}|\allowbreak\token[0.0000]|x|\allowbreak\token[0.0000]| +|\allowbreak\token[0.0000]| |\allowbreak\token[0.0000]|6|\allowbreak\token[0.6933]|.↩|\newline\allowbreak\token[0.0025]|    |\allowbreak\token[0.0000]| \|\allowbreak\token[0.0216]|]↩↩|\tighttracebreak\allowbreak\token[0.0000]|3|\allowbreak\token[0.0000]|.|\allowbreak\token[0.0000]| **|\allowbreak\token[0.0000]|Find|\allowbreak\token[0.0000]| the|\allowbreak\token[0.0008]| Equation|\allowbreak\token[0.0000]| of|\allowbreak\token[0.0000]| Line|\allowbreak\token[0.0000]| \(|\allowbreak\token[0.0000]|CD|\allowbreak\token[0.0000]|\|\allowbreak\token[0.0000]|):|\allowbreak[...]
\end{twopartbox}
\begin{twopartbox}{Cut Index 680 | Acceptance Probability 0.0320}{Block 5 Step 2}{Proposal (Rejected)}
[...]\allowbreak\token[0.0000]| x|\allowbreak\token[0.0000]| =|\allowbreak\token[0.0000]| \|\allowbreak\token[0.0000]|frac|\allowbreak\token[0.0001]|{|\allowbreak\token[0.1489]|8|\allowbreak\token[0.6277]| \|\allowbreak\token[0.0669]|cdot|\allowbreak\token[0.0000]| |\allowbreak\token[0.0001]|3|\allowbreak\token[0.0000]|}{|\allowbreak\token[0.0000]|9|\allowbreak\token[0.0000]|}|\allowbreak\token[0.0000]| =|\allowbreak\token[0.0000]| \|\allowbreak\token[0.0000]|frac|\allowbreak\token[0.0000]|{|\allowbreak\token[0.0265]|8|\allowbreak\token[0.0001]|}{|\allowbreak\token[0.0000]|3|\allowbreak\token[0.0001]|}.↩|\newline\allowbreak\token[0.0000]|    |\allowbreak\token[0.0000]| \|\allowbreak\token[0.0000]|]↩|\newline\allowbreak\token[0.0047]|  |\allowbreak\token[0.0000]| -|\allowbreak\token[0.6357]| Sub|\allowbreak\token[0.0000]|stit|\allowbreak\token[0.0000]|uting|\allowbreak\token[0.0001]| \(|\allowbreak\token[0.0000]|x|\allowbreak\token[0.0004]| =|\allowbreak\token[0.0000]| \|\allowbreak\token[0.0000]|frac|\allowbreak\token[0.0000]|{|\allowbreak\token[0.0000]|8|\allowbreak\token[0.0000]|}{|\allowbreak\token[0.0000]|3|\allowbreak\token[0.0000]|}\|\allowbreak\token[0.0000]|)|\allowbreak\token[0.6932]| back|\allowbreak\token[0.0000]| into|\allowbreak\token[0.3535]| the|\allowbreak\token[0.0001]| equation|\allowbreak\token[0.5933]| \(|\allowbreak\token[0.0001]|y|\allowbreak\token[0.0000]| =|\allowbreak\token[0.0000]| -|\allowbreak\token[0.0000]|\|\allowbreak\token[0.0000]|frac|\allowbreak\token[0.0000]|{|\allowbreak\token[0.0000]|3|\allowbreak\token[0.0000]|}{|\allowbreak\token[0.1594]|2|\allowbreak\token[0.0022]|}|\allowbreak\token[0.0005]|x|\allowbreak\token[0.0000]| +|\allowbreak\token[0.0000]| |\allowbreak\token[0.0000]|6|\allowbreak\token[0.0214]|\|\allowbreak\token[0.0016]|):↩|\newline\allowbreak\token[0.0000]|    |\allowbreak\token[0.0000]| \|\allowbreak\token[0.0000]|[↩|\newline\allowbreak\token[0.0000]|    |\allowbreak\token[0.0000]| y|\allowbreak\token[0.0000]| =|\allowbreak\token[0.0000]| -|\allowbreak\token[0.0000]|\|\allowbreak\token[0.0000]|frac|\allowbreak\token[0.0000]|{|\allowbreak\token[0.0000]|3|\allowbreak\token[0.0000]|}{|\allowbreak\token[0.0000]|2|\allowbreak\token[0.6276]|}|\allowbreak\token[0.0000]| \|\allowbreak\token[0.4471]|left|\allowbreak\token[0.2686]|(\|\allowbreak\token[0.0000]|frac|\allowbreak\token[0.0000]|{|\allowbreak\token[0.0000]|8|\allowbreak\token[0.0000]|}{|\allowbreak\token[0.0000]|3|\allowbreak\token[0.0000]|}\|\allowbreak\token[0.0000]|right|\allowbreak\token[0.0000]|)|\allowbreak\token[0.0000]| +|\allowbreak\token[0.0000]| |\allowbreak\token[0.0000]|6|\allowbreak\token[0.0000]| =|\allowbreak\token[0.0000]| -|\allowbreak\token[0.0030]|4|\allowbreak\token[0.0000]| +|\allowbreak\token[0.0000]| |\allowbreak\token[0.0000]|6|\allowbreak\token[0.0000]| =|\allowbreak\token[0.0000]| |\allowbreak\token[0.0000]|2|\allowbreak\token[0.0000]|.↩|\newline\allowbreak\token[0.0001]|    |\allowbreak\token[0.0000]| \|\allowbreak\token[0.0000]|]↩|\newline\allowbreak\token[0.0001]|  |\allowbreak\token[0.0006]| -|\allowbreak\token[0.0368]| So|\allowbreak\token[0.0055]|,|\allowbreak\token[0.7951]| point|\allowbreak\token[0.0000]| \(|\allowbreak\token[0.0000]|F|\allowbreak\token[0.0000]|\)|\allowbreak\token[0.0000]| is|\allowbreak\token[0.0000]| \|\allowbreak\token[0.0084]|(\|\allowbreak\token[0.0000]|left|\allowbreak\token[0.0000]|(\|\allowbreak\token[0.0000]|frac|\allowbreak\token[0.0000]|{|\allowbreak\token[0.0000]|8|\allowbreak\token[0.0000]|}{|\allowbreak\token[0.0000]|3|\allowbreak\token[0.0000]|},|\allowbreak\token[0.0000]| |\allowbreak\token[0.0000]|2|\allowbreak\token[0.0000]|\|\allowbreak\token[0.0003]|right|\allowbreak\token[0.0000]|)\|\allowbreak\token[0.0000]|).↩↩|\tighttracebreak\allowbreak\token[0.0000]|5|\allowbreak\token[0.0000]|.|\allowbreak\token[0.0000]| **|\allowbreak\token[0.0423]|Calculate|\allowbreak\token[0.0000]| the|\allowbreak\token[0.0000]| Area|\allowbreak\token[0.0001]| of|\allowbreak\token[0.0001]| Quadr|\allowbreak\token[0.0000]|ilateral|\allowbreak\token[0.0000]| \(|\allowbreak\token[0.0000]|DB|\allowbreak\token[0.0000]|EF|\allowbreak\token[0.0002]|\|\allowbreak\token[0.0004]|):|\allowbreak\token[0.0002]|**↩|\newline\allowbreak\token[0.0000]|  |\allowbreak\token[0.0001]| -|\allowbreak\cut{[CUT]}\allowbreak\token[1.1901]| We|\allowbreak\token[0.1196]| can|\allowbreak\token[0.0187]| use|\allowbreak\token[0.0000]| the|\allowbreak\token[0.5270]| sho|\allowbreak\token[0.0121]|el|\allowbreak\token[0.0000]|ace|\allowbreak\token[0.0001]| formula|\allowbreak\token[0.1155]| to|\allowbreak\token[0.0327]| find|\allowbreak\token[0.0000]| the|\allowbreak\token[0.0000]| area|\allowbreak\token[0.0183]| of|\allowbreak\token[0.4868]| quadr|\allowbreak\token[0.0000]|ilateral|\allowbreak\token[0.0059]| \(|\allowbreak\token[0.0001]|DB|\allowbreak\token[0.0000]|EF|\allowbreak\token[0.4751]|\|\allowbreak\token[0.8409]|):↩|\newline\allowbreak\token[0.0038]|    |\allowbreak\token[0.0001]| \|\allowbreak\token[0.0000]|[↩|\newline\allowbreak\token[0.0000]|    |\allowbreak\token[0.0072]| \|\allowbreak\token[0.0000]|text|\allowbreak\token[0.0000]|{|\allowbreak\token[0.0000]|Area|\allowbreak\token[0.0000]|}|\allowbreak\token[0.0000]| =|\allowbreak\token[0.0000]| \|\allowbreak\token[0.0000]|frac|\allowbreak\token[0.0000]|{|\allowbreak\token[0.0000]|1|\allowbreak\token[0.0000]|}{|\allowbreak\token[0.0000]|2|\allowbreak\token[0.0001]|}|\allowbreak\token[0.0000]| \|\allowbreak\token[0.0000]|left|\allowbreak\token[0.0000]!|!\allowbreak\token[0.0148]| x|\allowbreak\token[0.0001]|_|\allowbreak\token[0.0000]|1|\allowbreak\token[0.0173]|y|\allowbreak\token[0.0000]|_|\allowbreak\token[0.0000]|2|\allowbreak\token[0.0000]| +|\allowbreak\token[0.0000]| x|\allowbreak\token[0.0000]|_|\allowbreak\token[0.0000]|2|\allowbreak\token[0.0000]|y|\allowbreak\token[0.0000]|_|\allowbreak\token[0.0000]|3|\allowbreak\token[0.0000]| +|\allowbreak\token[0.0000]| x|\allowbreak\token[0.0000]|_|\allowbreak\token[0.0000]|3|\allowbreak\token[0.0000]|y|\allowbreak\token[0.0000]|_|\allowbreak\token[0.0000]|4|\allowbreak\token[0.0000]| +|\allowbreak\token[0.0000]| x|\allowbreak\token[0.0000]|_|\allowbreak\token[0.0000]|4|\allowbreak\token[0.0000]|y|\allowbreak\token[0.0000]|_|\allowbreak\token[0.0001]|1|\allowbreak\token[0.0000]| -|\allowbreak\token[0.0901]| (|\allowbreak\token[0.0000]|y|\allowbreak\token[0.0000]|_|\allowbreak\token[0.0000]|1|\allowbreak\token[0.0000]|x|\allowbreak\token[0.0000]|_|\allowbreak\token[0.0000]|2|\allowbreak\token[0.0000]| +|\allowbreak\token[0.0000]| y|\allowbreak\token[0.0000]|_|\allowbreak\token[0.0000]|2|\allowbreak\token[0.0000]|x|\allowbreak\token[0.0000]|_|\allowbreak\token[0.0000]|3|\allowbreak\token[0.0000]| +|\allowbreak\token[0.0000]| y|\allowbreak\token[0.0000]|_|\allowbreak\token[0.0001]|3|\allowbreak\token[0.0000]|x|\allowbreak\token[0.0000]|_|\allowbreak\token[0.0000]|4|\allowbreak\token[0.0000]| +|\allowbreak\token[0.0000]| y|\allowbreak\token[0.0000]|_|\allowbreak\token[0.0000]|4|\allowbreak\token[0.0000]|x|\allowbreak\token[0.0000]|_|\allowbreak\token[0.0001]|1|\allowbreak\token[0.0000]|)|\allowbreak\token[0.0000]| \|\allowbreak\token[0.0000]|right|\allowbreak\token[0.0206]!|.↩!\newline\allowbreak\token[0.0000]|    |\allowbreak\token[0.0000]| \|\allowbreak\token[0.0001]|]↩|\newline\allowbreak\token[0.0059]|  |\allowbreak\token[0.0000]| -|\allowbreak\token[1.2948]| Sub|\allowbreak\token[0.0000]|stit|\allowbreak\token[0.0000]|uting|\allowbreak\token[0.0523]| the|\allowbreak\token[0.0065]| coordinates|\allowbreak\token[0.8325]| of|\allowbreak\token[0.4752]| \(|\allowbreak\token[0.0002]|D|\allowbreak\token[0.0437]|(|\allowbreak\token[0.0000]|0|\allowbreak\token[0.0000]|,|\allowbreak\token[0.3653]| |\allowbreak\token[0.0000]|3|\allowbreak\token[0.0001]|)\|\allowbreak\token[0.0000]|),|\allowbreak\token[0.0000]| \(|\allowbreak\token[0.0006]|B|\allowbreak\token[0.0000]|(|\allowbreak\token[0.0000]|0|\allowbreak\token[0.0000]|,|\allowbreak\token[0.0003]| |\allowbreak\token[0.0000]|0|\allowbreak\token[0.0000]|)\|\allowbreak\token[0.0000]|),|\allowbreak\token[0.0000]| \(|\allowbreak\token[0.0000]|E|\allowbreak\token[0.0000]|(|\allowbreak\token[0.0000]|4|\allowbreak\token[0.0000]|,|\allowbreak\token[0.0000]| |\allowbreak\token[0.0000]|0|\allowbreak\token[0.0000]|)\|\allowbreak\token[0.0000]|),|\allowbreak\token[0.0362]| and|\allowbreak\token[0.0000]| \(|\allowbreak\token[0.0000]|F|\allowbreak\token[0.0004]|\|\allowbreak\token[0.0004]|left|\allowbreak\token[0.0000]|(\|\allowbreak[...]
\tcblower
\token[1.2111]| Use|\allowbreak\token[0.0000]| the|\allowbreak\token[1.1781]| sho|\allowbreak\token[0.0111]|el|\allowbreak\token[0.0000]|ace|\allowbreak\token[0.0001]| formula|\allowbreak\token[0.6447]| to|\allowbreak\token[0.1326]| find|\allowbreak\token[0.0000]| the|\allowbreak\token[0.0000]| area|\allowbreak\token[0.0599]| of|\allowbreak\token[0.2704]| quadr|\allowbreak\token[0.0000]|ilateral|\allowbreak\token[0.0740]| \(|\allowbreak\token[0.0005]|DB|\allowbreak\token[0.0001]|EF|\allowbreak\token[0.5030]|\|\allowbreak\token[0.3975]|):↩|\newline\allowbreak\token[0.0042]|    |\allowbreak\token[0.0004]| \|\allowbreak\token[0.0000]|[↩|\newline\allowbreak\token[0.0000]|    |\allowbreak\token[0.0111]| \|\allowbreak\token[0.0000]|text|\allowbreak\token[0.0000]|{|\allowbreak\token[0.0000]|Area|\allowbreak\token[0.0000]|}|\allowbreak\token[0.0000]| =|\allowbreak\token[0.0000]| \|\allowbreak\token[0.0000]|frac|\allowbreak\token[0.0000]|{|\allowbreak\token[0.0000]|1|\allowbreak\token[0.0000]|}{|\allowbreak\token[0.0000]|2|\allowbreak\token[0.0001]|}|\allowbreak\token[0.0000]| \|\allowbreak\token[0.0000]|left|\allowbreak\token[0.0000]!|!\allowbreak\token[0.0338]| x|\allowbreak\token[0.0001]|_|\allowbreak\token[0.0000]|1|\allowbreak\token[0.0173]|y|\allowbreak\token[0.0000]|_|\allowbreak\token[0.0000]|2|\allowbreak\token[0.0000]| +|\allowbreak\token[0.0000]| x|\allowbreak\token[0.0000]|_|\allowbreak\token[0.0000]|2|\allowbreak\token[0.0000]|y|\allowbreak\token[0.0000]|_|\allowbreak\token[0.0000]|3|\allowbreak\token[0.0000]| +|\allowbreak\token[0.0000]| x|\allowbreak\token[0.0000]|_|\allowbreak\token[0.0000]|3|\allowbreak\token[0.0000]|y|\allowbreak\token[0.0000]|_|\allowbreak\token[0.0000]|4|\allowbreak\token[0.0000]| +|\allowbreak\token[0.0000]| x|\allowbreak\token[0.0000]|_|\allowbreak\token[0.0000]|4|\allowbreak\token[0.0000]|y|\allowbreak\token[0.0000]|_|\allowbreak\token[0.0001]|1|\allowbreak\token[0.0000]| -|\allowbreak\token[0.0993]| (|\allowbreak\token[0.0000]|y|\allowbreak\token[0.0000]|_|\allowbreak\token[0.0000]|1|\allowbreak\token[0.0000]|x|\allowbreak\token[0.0000]|_|\allowbreak\token[0.0000]|2|\allowbreak\token[0.0000]| +|\allowbreak\token[0.0000]| y|\allowbreak\token[0.0000]|_|\allowbreak\token[0.0000]|2|\allowbreak\token[0.0000]|x|\allowbreak\token[0.0000]|_|\allowbreak\token[0.0000]|3|\allowbreak\token[0.0000]| +|\allowbreak\token[0.0000]| y|\allowbreak\token[0.0000]|_|\allowbreak\token[0.0001]|3|\allowbreak\token[0.0000]|x|\allowbreak\token[0.0000]|_|\allowbreak\token[0.0000]|4|\allowbreak\token[0.0000]| +|\allowbreak\token[0.0000]| y|\allowbreak\token[0.0000]|_|\allowbreak\token[0.0000]|4|\allowbreak\token[0.0000]|x|\allowbreak\token[0.0000]|_|\allowbreak\token[0.0001]|1|\allowbreak\token[0.0000]|)|\allowbreak\token[0.0000]| \|\allowbreak\token[0.0000]|right|\allowbreak\token[0.0409]!|.↩!\newline\allowbreak\token[0.0000]|    |\allowbreak\token[0.0000]| \|\allowbreak\token[0.0001]|]↩|\newline\allowbreak\token[0.0265]|  |\allowbreak\token[0.0000]| -|\allowbreak\token[1.7440]| Sub|\allowbreak\token[0.0000]|stit|\allowbreak\token[0.0000]|uting|\allowbreak\token[0.0562]| the|\allowbreak\token[0.0065]| coordinates|\allowbreak\token[0.9448]| of|\allowbreak\token[0.3921]| \(|\allowbreak\token[0.0002]|D|\allowbreak\token[0.0473]|(|\allowbreak\token[0.0000]|0|\allowbreak\token[0.0000]|,|\allowbreak\token[0.3395]| |\allowbreak\token[0.0000]|3|\allowbreak\token[0.0001]|)\|\allowbreak\token[0.0000]|),|\allowbreak\token[0.0000]| \(|\allowbreak\token[0.0006]|B|\allowbreak\token[0.0000]|(|\allowbreak\token[0.0000]|0|\allowbreak\token[0.0000]|,|\allowbreak\token[0.0004]| |\allowbreak\token[0.0000]|0|\allowbreak\token[0.0000]|)\|\allowbreak\token[0.0000]|),|\allowbreak\token[0.0000]| \(|\allowbreak\token[0.0000]|E|\allowbreak\token[0.0000]|(|\allowbreak\token[0.0000]|4|\allowbreak\token[0.0000]|,|\allowbreak\token[0.0000]| |\allowbreak\token[0.0000]|0|\allowbreak\token[0.0000]|)\|\allowbreak\token[0.0000]|),|\allowbreak\token[0.0547]| and|\allowbreak\token[0.0000]| \(|\allowbreak\token[0.0000]|F|\allowbreak\token[0.0005]|\|\allowbreak\token[0.0005]|left|\allowbreak\token[0.0000]|(\|\allowbreak\token[0.0000]|frac|\allowbreak\token[0.0000]|{|\allowbreak[...]
\end{twopartbox}
\begin{twopartbox}{Cut Index 39 | Acceptance Probability 0.0000}{Block 5 Step 3}{Proposal (Rejected)}
\token[2.6002]| To|\allowbreak\token[0.6725]| solve|\allowbreak\token[0.8316]| the|\allowbreak\token[0.0687]| problem|\allowbreak\token[0.5135]|,|\allowbreak\token[0.3361]| we|\allowbreak\token[0.4652]| need|\allowbreak\token[0.0016]| to|\allowbreak\token[0.7385]| determine|\allowbreak\token[0.0004]| the|\allowbreak\token[0.1338]| area|\allowbreak\token[0.0002]| of|\allowbreak\token[0.1912]| quadr|\allowbreak\token[0.0000]|ilateral|\allowbreak\token[0.5312]| \(|\allowbreak\token[0.3013]|DB|\allowbreak\token[0.0001]|EF|\allowbreak\token[0.6859]|\|\allowbreak\token[0.3149]|).|\allowbreak\token[1.6625]| Here|\allowbreak\token[1.0700]|'s|\allowbreak\token[0.6866]| the|\allowbreak\token[0.1679]| step|\allowbreak\token[0.0002]|-by|\allowbreak\token[0.0001]|-step|\allowbreak\token[1.4676]| approach|\allowbreak\token[0.1317]|:↩↩|\tighttracebreak\allowbreak\token[0.0016]|1|\allowbreak\token[0.0000]|.|\allowbreak\token[0.4060]| **|\allowbreak\token[0.4029]|Ident|\allowbreak\token[0.0010]|ify|\allowbreak\token[0.6278]| Coordinates|\allowbreak\token[1.0343]| of|\allowbreak\token[0.7862]| Points|\allowbreak\token[1.0859]|:|\allowbreak\token[0.0014]|**↩|\newline\allowbreak\token[0.2277]|  |\allowbreak\token[0.0177]| -|\allowbreak\cut{[CUT]}\allowbreak\token[1.1397]| Point|\allowbreak\token[0.0003]| \(|\allowbreak\token[0.4684]|A|\allowbreak\token[0.2856]|\)|\allowbreak\token[0.1606]| is|\allowbreak\token[0.6652]| \|\allowbreak\token[0.0001]|((|\allowbreak\token[0.0000]|0|\allowbreak\token[0.0000]|,|\allowbreak\token[0.0901]| |\allowbreak\token[0.0000]|6|\allowbreak\token[0.0000]|)\|\allowbreak\token[0.3752]|).↩|\newline\allowbreak\token[0.0000]|  |\allowbreak\token[0.0000]| -|\allowbreak\token[0.0001]| Point|\allowbreak\token[0.0000]| \(|\allowbreak\token[0.0005]|B|\allowbreak\token[0.0000]|\)|\allowbreak\token[0.0000]| is|\allowbreak\token[0.0001]| \|\allowbreak\token[0.0000]|((|\allowbreak\token[0.0000]|0|\allowbreak\token[0.0000]|,|\allowbreak\token[0.0001]| |\allowbreak\token[0.0000]|0|\allowbreak\token[0.0000]|)\|\allowbreak\token[0.0002]|).↩|\newline\allowbreak\token[0.0000]|  |\allowbreak\token[0.0000]| -|\allowbreak\token[0.0000]| Point|\allowbreak\token[0.0000]| \(|\allowbreak\token[0.0001]|C|\allowbreak\token[0.0000]|\)|\allowbreak\token[0.0000]| is|\allowbreak\token[0.0000]| \|\allowbreak\token[0.0001]|((|\allowbreak\token[0.0000]|8|\allowbreak\token[0.0000]|,|\allowbreak\token[0.0000]| |\allowbreak\token[0.0000]|0|\allowbreak\token[0.0000]|)\|\allowbreak\token[0.0820]|).↩|\newline\allowbreak\token[0.0053]|  |\allowbreak\token[0.0000]| -|\allowbreak\token[0.7278]| Point|\allowbreak\token[0.0000]| \(|\allowbreak\token[0.0022]|D|\allowbreak\token[0.4194]|\)|\allowbreak\token[0.0197]| is|\allowbreak\token[0.0001]| the|\allowbreak\token[0.0000]| midpoint|\allowbreak\token[0.0000]| of|\allowbreak\token[0.0496]| \|\allowbreak\token[0.0004]|(\|\allowbreak\token[0.0000]|over|\allowbreak\token[0.0000]|line|\allowbreak\token[0.0000]|{|\allowbreak\token[0.0000]|AB|\allowbreak\token[0.0001]|}\|\allowbreak\token[0.0147]|),|\allowbreak\token[0.0025]| so|\allowbreak\token[0.4478]| \(|\allowbreak\token[0.0000]|D|\allowbreak\token[0.6075]|\)|\allowbreak\token[0.0034]| is|\allowbreak\token[0.0002]| \|\allowbreak\token[0.0905]|((|\allowbreak\token[0.0102]|0|\allowbreak\token[0.0000]|,|\allowbreak\token[0.0240]| |\allowbreak\token[0.0000]|3|\allowbreak\token[0.0000]|)\|\allowbreak\token[0.0612]|).↩|\newline\allowbreak\token[0.0000]|  |\allowbreak\token[0.0000]| -|\allowbreak\token[0.0000]| Point|\allowbreak\token[0.0000]| \(|\allowbreak\token[0.0002]|E|\allowbreak\token[0.0001]|\)|\allowbreak\token[0.0000]| is|\allowbreak\token[0.0000]| the|\allowbreak\token[0.0000]| midpoint|\allowbreak\token[0.0000]| of|\allowbreak\token[0.0000]| \|\allowbreak\token[0.0000]|(\|\allowbreak\token[0.0002]|over|\allowbreak\token[0.0001]|line|\allowbreak\token[0.0000]|{|\allowbreak\token[0.0002]|BC|\allowbreak\token[0.0000]|}\|\allowbreak\token[0.0000]|),|\allowbreak\token[0.0000]| so|\allowbreak\token[0.0000]| \(|\allowbreak\token[0.0001]|E|\allowbreak\token[0.0002]|\)|\allowbreak\token[0.0000]| is|\allowbreak\token[0.0000]| \|\allowbreak\token[0.0007]|((|\allowbreak\token[0.0000]|4|\allowbreak\token[0.0000]|,|\allowbreak\token[0.0000]| |\allowbreak\token[0.0001]|0|\allowbreak\token[0.0000]|)\|\allowbreak\token[0.3933]|).↩|\newline\allowbreak\token[0.0175]|  |\allowbreak\token[0.0000]| -|\allowbreak\token[0.1774]| Point|\allowbreak\token[0.0000]| \(|\allowbreak\token[0.0007]|F|\allowbreak\token[0.0040]|\)|\allowbreak\token[0.2097]| is|\allowbreak\token[1.1429]| the|\allowbreak\token[0.5879]| intersection|\allowbreak\token[0.0027]| of|\allowbreak\token[1.0073]| \|\allowbreak\token[0.0017]|(\|\allowbreak\token[0.0001]|over|\allowbreak\token[0.0006]|line|\allowbreak\token[0.0001]|{|\allowbreak\token[1.0715]|AE|\allowbreak\token[0.0002]|}\|\allowbreak\token[0.0001]|)|\allowbreak\token[0.0032]| and|\allowbreak\token[0.0048]| \|\allowbreak\token[0.0000]|(\|\allowbreak\token[0.0001]|over|\allowbreak\token[0.0002]|line|\allowbreak\token[0.0000]|{|\allowbreak\token[0.2372]|CD|\allowbreak\token[0.0001]|}\|\allowbreak\token[0.6869]|).↩↩|\tighttracebreak\allowbreak\token[0.0003]|2|\allowbreak\token[0.0000]|.|\allowbreak\token[0.0000]| **|\allowbreak\token[0.1501]|Find|\allowbreak\token[0.2495]| the|\allowbreak\token[1.1147]| Equation|\allowbreak\token[0.0000]| of|\allowbreak\token[0.7128]| Line|\allowbreak\token[0.4501]| \(|\allowbreak\token[0.0092]|AE|\allowbreak\token[0.0294]|\|\allowbreak\token[0.0004]|):|\allowbreak[...]
\tcblower
\token[1.1081]| \(|\allowbreak\token[0.3162]|A|\allowbreak\token[0.7272]|(|\allowbreak\token[0.0000]|0|\allowbreak\token[0.0000]|,|\allowbreak\token[0.5030]| |\allowbreak\token[0.0001]|6|\allowbreak\token[0.0000]|)\|\allowbreak\token[0.0359]|)↩|\newline\allowbreak\token[0.0000]|  |\allowbreak\token[0.0000]| -|\allowbreak\token[0.0000]| \(|\allowbreak\token[0.0005]|B|\allowbreak\token[0.0000]|(|\allowbreak\token[0.0000]|0|\allowbreak\token[0.0000]|,|\allowbreak\token[0.0002]| |\allowbreak\token[0.0000]|0|\allowbreak\token[0.0000]|)\|\allowbreak\token[0.0003]|)↩|\newline\allowbreak\token[0.0000]|  |\allowbreak\token[0.0000]| -|\allowbreak\token[0.0002]| \(|\allowbreak\token[0.0002]|C|\allowbreak\token[0.0000]|(|\allowbreak\token[0.0001]|8|\allowbreak\token[0.0000]|,|\allowbreak\token[0.0001]| |\allowbreak\token[0.0000]|0|\allowbreak\token[0.0000]|)\|\allowbreak\token[0.0406]|)↩|\newline\allowbreak\token[0.0216]|  |\allowbreak\token[0.0005]| -|\allowbreak\token[0.1446]| \(|\allowbreak\token[0.0010]|D|\allowbreak\token[0.4261]|\)|\allowbreak\token[0.0760]| is|\allowbreak\token[0.0005]| the|\allowbreak\token[0.0000]| midpoint|\allowbreak\token[0.0000]| of|\allowbreak\token[0.1095]| \|\allowbreak\token[0.0001]|(\|\allowbreak\token[0.0000]|over|\allowbreak\token[0.0001]|line|\allowbreak\token[0.0000]|{|\allowbreak\token[0.0000]|AB|\allowbreak\token[0.0010]|}\|\allowbreak\token[0.2030]|),|\allowbreak\token[0.0019]| so|\allowbreak\token[0.0023]| \(|\allowbreak\token[0.0000]|D|\allowbreak\token[0.5340]|(|\allowbreak\token[0.0007]|0|\allowbreak\token[0.0000]|,|\allowbreak\token[0.0298]| |\allowbreak\token[0.0001]|3|\allowbreak\token[0.0000]|)\|\allowbreak\token[0.0568]|)↩|\newline\allowbreak\token[0.0000]|  |\allowbreak\token[0.0000]| -|\allowbreak\token[0.0000]| \(|\allowbreak\token[0.0000]|E|\allowbreak\token[0.0000]|\)|\allowbreak\token[0.0001]| is|\allowbreak\token[0.0000]| the|\allowbreak\token[0.0000]| midpoint|\allowbreak\token[0.0000]| of|\allowbreak\token[0.0000]| \|\allowbreak\token[0.0000]|(\|\allowbreak\token[0.0002]|over|\allowbreak\token[0.0001]|line|\allowbreak\token[0.0000]|{|\allowbreak\token[0.0003]|BC|\allowbreak\token[0.0000]|}\|\allowbreak\token[0.0001]|),|\allowbreak\token[0.0000]| so|\allowbreak\token[0.0000]| \(|\allowbreak\token[0.0001]|E|\allowbreak\token[0.0010]|(|\allowbreak\token[0.0000]|4|\allowbreak\token[0.0000]|,|\allowbreak\token[0.0000]| |\allowbreak\token[0.0001]|0|\allowbreak\token[0.0000]|)\|\allowbreak\token[0.5041]|)↩|\newline\allowbreak\token[0.1106]|  |\allowbreak\token[0.0001]| -|\allowbreak\token[0.7466]| \(|\allowbreak\token[0.0042]|F|\allowbreak\token[0.0180]|\)|\allowbreak\token[0.0819]| is|\allowbreak\token[0.6106]| the|\allowbreak\token[0.9253]| intersection|\allowbreak\token[0.0328]| of|\allowbreak\token[0.6477]| \|\allowbreak\token[0.0004]|(\|\allowbreak\token[0.0003]|over|\allowbreak\token[0.0008]|line|\allowbreak\token[0.0001]|{|\allowbreak\token[1.2528]|AC|\allowbreak\token[0.0004]|}\|\allowbreak\token[0.0001]|)|\allowbreak\token[0.0710]| and|\allowbreak\token[0.1067]| \|\allowbreak\token[0.0001]|(\|\allowbreak\token[0.0002]|over|\allowbreak\token[0.0016]|line|\allowbreak\token[0.0000]|{|\allowbreak\token[0.0729]|DE|\allowbreak\token[0.0003]|}\|\allowbreak\token[1.2013]|)↩↩|\tighttracebreak\allowbreak\token[0.0008]|2|\allowbreak\token[0.0000]|.|\allowbreak\token[0.0000]| **|\allowbreak\token[0.3997]|Find|\allowbreak\token[0.3533]| the|\allowbreak\token[0.5836]| Equation|\allowbreak\token[0.0001]| of|\allowbreak\token[0.1111]| Line|\allowbreak\token[0.7164]| \|\allowbreak\token[0.0002]|(\|\allowbreak\token[0.0001]|over|\allowbreak\token[0.0005]|line|\allowbreak\token[0.0000]|{|\allowbreak\token[0.2699]|AC|\allowbreak\token[0.0001]|}\|\allowbreak\token[0.0038]|):|\allowbreak\token[0.0008]|**↩|\newline\allowbreak\token[0.0000]|  |\allowbreak\token[0.6879]| -|\allowbreak\token[0.9202]| The|\allowbreak\token[0.1719]| slope|\allowbreak\token[0.0494]| of|\allowbreak\token[0.1579]| \|\allowbreak\token[0.0000]|(\|\allowbreak\token[0.0000]|over|\allowbreak\token[0.0000]|line|\allowbreak\token[0.0000]|{|\allowbreak\token[0.0000]|AC|\allowbreak\token[0.0005]|}\|\allowbreak\token[0.0003]|)|\allowbreak\token[0.0002]| is|\allowbreak\token[0.0016]| \|\allowbreak\token[0.0135]|(\|\allowbreak[...]
\end{twopartbox}
\begin{twopartbox}{Cut Index 680 | Acceptance Probability 1.0000}{Block 6 Step 1}{Proposal (Accepted)}
[...]\allowbreak\token[0.0000]| x|\allowbreak\token[0.0000]| =|\allowbreak\token[0.0000]| \|\allowbreak\token[0.0000]|frac|\allowbreak\token[0.0001]|{|\allowbreak\token[0.1489]|8|\allowbreak\token[0.6277]| \|\allowbreak\token[0.0669]|cdot|\allowbreak\token[0.0000]| |\allowbreak\token[0.0001]|3|\allowbreak\token[0.0000]|}{|\allowbreak\token[0.0000]|9|\allowbreak\token[0.0000]|}|\allowbreak\token[0.0000]| =|\allowbreak\token[0.0000]| \|\allowbreak\token[0.0000]|frac|\allowbreak\token[0.0000]|{|\allowbreak\token[0.0265]|8|\allowbreak\token[0.0001]|}{|\allowbreak\token[0.0000]|3|\allowbreak\token[0.0001]|}.↩|\newline\allowbreak\token[0.0000]|    |\allowbreak\token[0.0000]| \|\allowbreak\token[0.0000]|]↩|\newline\allowbreak\token[0.0047]|  |\allowbreak\token[0.0000]| -|\allowbreak\token[0.6357]| Sub|\allowbreak\token[0.0000]|stit|\allowbreak\token[0.0000]|uting|\allowbreak\token[0.0001]| \(|\allowbreak\token[0.0000]|x|\allowbreak\token[0.0004]| =|\allowbreak\token[0.0000]| \|\allowbreak\token[0.0000]|frac|\allowbreak\token[0.0000]|{|\allowbreak\token[0.0000]|8|\allowbreak\token[0.0000]|}{|\allowbreak\token[0.0000]|3|\allowbreak\token[0.0000]|}\|\allowbreak\token[0.0000]|)|\allowbreak\token[0.6932]| back|\allowbreak\token[0.0000]| into|\allowbreak\token[0.3535]| the|\allowbreak\token[0.0001]| equation|\allowbreak\token[0.5933]| \(|\allowbreak\token[0.0001]|y|\allowbreak\token[0.0000]| =|\allowbreak\token[0.0000]| -|\allowbreak\token[0.0000]|\|\allowbreak\token[0.0000]|frac|\allowbreak\token[0.0000]|{|\allowbreak\token[0.0000]|3|\allowbreak\token[0.0000]|}{|\allowbreak\token[0.1594]|2|\allowbreak\token[0.0022]|}|\allowbreak\token[0.0005]|x|\allowbreak\token[0.0000]| +|\allowbreak\token[0.0000]| |\allowbreak\token[0.0000]|6|\allowbreak\token[0.0214]|\|\allowbreak\token[0.0016]|):↩|\newline\allowbreak\token[0.0000]|    |\allowbreak\token[0.0000]| \|\allowbreak\token[0.0000]|[↩|\newline\allowbreak\token[0.0000]|    |\allowbreak\token[0.0000]| y|\allowbreak\token[0.0000]| =|\allowbreak\token[0.0000]| -|\allowbreak\token[0.0000]|\|\allowbreak\token[0.0000]|frac|\allowbreak\token[0.0000]|{|\allowbreak\token[0.0000]|3|\allowbreak\token[0.0000]|}{|\allowbreak\token[0.0000]|2|\allowbreak\token[0.6276]|}|\allowbreak\token[0.0000]| \|\allowbreak\token[0.4471]|left|\allowbreak\token[0.2686]|(\|\allowbreak\token[0.0000]|frac|\allowbreak\token[0.0000]|{|\allowbreak\token[0.0000]|8|\allowbreak\token[0.0000]|}{|\allowbreak\token[0.0000]|3|\allowbreak\token[0.0000]|}\|\allowbreak\token[0.0000]|right|\allowbreak\token[0.0000]|)|\allowbreak\token[0.0000]| +|\allowbreak\token[0.0000]| |\allowbreak\token[0.0000]|6|\allowbreak\token[0.0000]| =|\allowbreak\token[0.0000]| -|\allowbreak\token[0.0030]|4|\allowbreak\token[0.0000]| +|\allowbreak\token[0.0000]| |\allowbreak\token[0.0000]|6|\allowbreak\token[0.0000]| =|\allowbreak\token[0.0000]| |\allowbreak\token[0.0000]|2|\allowbreak\token[0.0000]|.↩|\newline\allowbreak\token[0.0001]|    |\allowbreak\token[0.0000]| \|\allowbreak\token[0.0000]|]↩|\newline\allowbreak\token[0.0001]|  |\allowbreak\token[0.0006]| -|\allowbreak\token[0.0368]| So|\allowbreak\token[0.0055]|,|\allowbreak\token[0.7951]| point|\allowbreak\token[0.0000]| \(|\allowbreak\token[0.0000]|F|\allowbreak\token[0.0000]|\)|\allowbreak\token[0.0000]| is|\allowbreak\token[0.0000]| \|\allowbreak\token[0.0084]|(\|\allowbreak\token[0.0000]|left|\allowbreak\token[0.0000]|(\|\allowbreak\token[0.0000]|frac|\allowbreak\token[0.0000]|{|\allowbreak\token[0.0000]|8|\allowbreak\token[0.0000]|}{|\allowbreak\token[0.0000]|3|\allowbreak\token[0.0000]|},|\allowbreak\token[0.0000]| |\allowbreak\token[0.0000]|2|\allowbreak\token[0.0000]|\|\allowbreak\token[0.0003]|right|\allowbreak\token[0.0000]|)\|\allowbreak\token[0.0000]|).↩↩|\tighttracebreak\allowbreak\token[0.0000]|5|\allowbreak\token[0.0000]|.|\allowbreak\token[0.0000]| **|\allowbreak\token[0.0423]|Calculate|\allowbreak\token[0.0000]| the|\allowbreak\token[0.0000]| Area|\allowbreak\token[0.0001]| of|\allowbreak\token[0.0001]| Quadr|\allowbreak\token[0.0000]|ilateral|\allowbreak\token[0.0000]| \(|\allowbreak\token[0.0000]|DB|\allowbreak\token[0.0000]|EF|\allowbreak\token[0.0002]|\|\allowbreak\token[0.0004]|):|\allowbreak\token[0.0002]|**↩|\newline\allowbreak\token[0.0000]|  |\allowbreak\token[0.0001]| -|\allowbreak\cut{[CUT]}\allowbreak\token[1.1901]| We|\allowbreak\token[0.1196]| can|\allowbreak\token[0.0187]| use|\allowbreak\token[0.0000]| the|\allowbreak\token[0.5270]| sho|\allowbreak\token[0.0121]|el|\allowbreak\token[0.0000]|ace|\allowbreak\token[0.0001]| formula|\allowbreak\token[0.1155]| to|\allowbreak\token[0.0327]| find|\allowbreak\token[0.0000]| the|\allowbreak\token[0.0000]| area|\allowbreak\token[0.0183]| of|\allowbreak\token[0.4868]| quadr|\allowbreak\token[0.0000]|ilateral|\allowbreak\token[0.0059]| \(|\allowbreak\token[0.0001]|DB|\allowbreak\token[0.0000]|EF|\allowbreak\token[0.4751]|\|\allowbreak\token[0.8409]|):↩|\newline\allowbreak\token[0.0038]|    |\allowbreak\token[0.0001]| \|\allowbreak\token[0.0000]|[↩|\newline\allowbreak\token[0.0000]|    |\allowbreak\token[0.0072]| \|\allowbreak\token[0.0000]|text|\allowbreak\token[0.0000]|{|\allowbreak\token[0.0000]|Area|\allowbreak\token[0.0000]|}|\allowbreak\token[0.0000]| =|\allowbreak\token[0.0000]| \|\allowbreak\token[0.0000]|frac|\allowbreak\token[0.0000]|{|\allowbreak\token[0.0000]|1|\allowbreak\token[0.0000]|}{|\allowbreak\token[0.0000]|2|\allowbreak\token[0.0001]|}|\allowbreak\token[0.0000]| \|\allowbreak\token[0.0000]|left|\allowbreak\token[0.0000]!|!\allowbreak\token[0.0148]| x|\allowbreak\token[0.0001]|_|\allowbreak\token[0.0000]|1|\allowbreak\token[0.0173]|y|\allowbreak\token[0.0000]|_|\allowbreak\token[0.0000]|2|\allowbreak\token[0.0000]| +|\allowbreak\token[0.0000]| x|\allowbreak\token[0.0000]|_|\allowbreak\token[0.0000]|2|\allowbreak\token[0.0000]|y|\allowbreak\token[0.0000]|_|\allowbreak\token[0.0000]|3|\allowbreak\token[0.0000]| +|\allowbreak\token[0.0000]| x|\allowbreak\token[0.0000]|_|\allowbreak\token[0.0000]|3|\allowbreak\token[0.0000]|y|\allowbreak\token[0.0000]|_|\allowbreak\token[0.0000]|4|\allowbreak\token[0.0000]| +|\allowbreak\token[0.0000]| x|\allowbreak\token[0.0000]|_|\allowbreak\token[0.0000]|4|\allowbreak\token[0.0000]|y|\allowbreak\token[0.0000]|_|\allowbreak\token[0.0001]|1|\allowbreak\token[0.0000]| -|\allowbreak\token[0.0901]| (|\allowbreak\token[0.0000]|y|\allowbreak\token[0.0000]|_|\allowbreak\token[0.0000]|1|\allowbreak\token[0.0000]|x|\allowbreak\token[0.0000]|_|\allowbreak\token[0.0000]|2|\allowbreak\token[0.0000]| +|\allowbreak\token[0.0000]| y|\allowbreak\token[0.0000]|_|\allowbreak\token[0.0000]|2|\allowbreak\token[0.0000]|x|\allowbreak\token[0.0000]|_|\allowbreak\token[0.0000]|3|\allowbreak\token[0.0000]| +|\allowbreak\token[0.0000]| y|\allowbreak\token[0.0000]|_|\allowbreak\token[0.0001]|3|\allowbreak\token[0.0000]|x|\allowbreak\token[0.0000]|_|\allowbreak\token[0.0000]|4|\allowbreak\token[0.0000]| +|\allowbreak\token[0.0000]| y|\allowbreak\token[0.0000]|_|\allowbreak\token[0.0000]|4|\allowbreak\token[0.0000]|x|\allowbreak\token[0.0000]|_|\allowbreak\token[0.0001]|1|\allowbreak\token[0.0000]|)|\allowbreak\token[0.0000]| \|\allowbreak\token[0.0000]|right|\allowbreak\token[0.0206]!|.↩!\newline\allowbreak\token[0.0000]|    |\allowbreak\token[0.0000]| \|\allowbreak\token[0.0001]|]↩|\newline\allowbreak\token[0.0059]|  |\allowbreak\token[0.0000]| -|\allowbreak\token[1.2948]| Sub|\allowbreak\token[0.0000]|stit|\allowbreak\token[0.0000]|uting|\allowbreak\token[0.0523]| the|\allowbreak\token[0.0065]| coordinates|\allowbreak\token[0.8325]| of|\allowbreak\token[0.4752]| \(|\allowbreak\token[0.0002]|D|\allowbreak\token[0.0437]|(|\allowbreak\token[0.0000]|0|\allowbreak\token[0.0000]|,|\allowbreak\token[0.3653]| |\allowbreak\token[0.0000]|3|\allowbreak\token[0.0001]|)\|\allowbreak\token[0.0000]|),|\allowbreak\token[0.0000]| \(|\allowbreak\token[0.0006]|B|\allowbreak\token[0.0000]|(|\allowbreak\token[0.0000]|0|\allowbreak\token[0.0000]|,|\allowbreak\token[0.0003]| |\allowbreak\token[0.0000]|0|\allowbreak\token[0.0000]|)\|\allowbreak\token[0.0000]|),|\allowbreak\token[0.0000]| \(|\allowbreak\token[0.0000]|E|\allowbreak\token[0.0000]|(|\allowbreak\token[0.0000]|4|\allowbreak\token[0.0000]|,|\allowbreak\token[0.0000]| |\allowbreak\token[0.0000]|0|\allowbreak\token[0.0000]|)\|\allowbreak\token[0.0000]|),|\allowbreak\token[0.0362]| and|\allowbreak\token[0.0000]| \(|\allowbreak\token[0.0000]|F|\allowbreak\token[0.0004]|\|\allowbreak\token[0.0004]|left|\allowbreak\token[0.0000]|(\|\allowbreak[...]
\tcblower
\token[1.1900]| We|\allowbreak\token[0.1196]| can|\allowbreak\token[0.0187]| use|\allowbreak\token[0.0000]| the|\allowbreak\token[0.4952]| sho|\allowbreak\token[0.0107]|el|\allowbreak\token[0.0000]|ace|\allowbreak\token[0.0001]| formula|\allowbreak\token[0.1092]| to|\allowbreak\token[0.0327]| find|\allowbreak\token[0.0000]| the|\allowbreak\token[0.0000]| area|\allowbreak\token[0.0183]| of|\allowbreak\token[0.4868]| quadr|\allowbreak\token[0.0000]|ilateral|\allowbreak\token[0.0059]| \(|\allowbreak\token[0.0001]|DB|\allowbreak\token[0.0000]|EF|\allowbreak\token[0.4751]|\|\allowbreak\token[0.8409]|):↩|\newline\allowbreak\token[0.0038]|    |\allowbreak\token[0.0001]| \|\allowbreak\token[0.0000]|[↩|\newline\allowbreak\token[0.0000]|    |\allowbreak\token[0.0072]| \|\allowbreak\token[0.0000]|text|\allowbreak\token[0.0000]|{|\allowbreak\token[0.0000]|Area|\allowbreak\token[0.0000]|}|\allowbreak\token[0.0000]| =|\allowbreak\token[0.0000]| \|\allowbreak\token[0.0000]|frac|\allowbreak\token[0.0000]|{|\allowbreak\token[0.0000]|1|\allowbreak\token[0.0000]|}{|\allowbreak\token[0.0000]|2|\allowbreak\token[0.0001]|}|\allowbreak\token[0.0000]| \|\allowbreak\token[0.0000]|left|\allowbreak\token[0.0000]!|!\allowbreak\token[0.0148]| x|\allowbreak\token[0.0001]|_|\allowbreak\token[0.0000]|1|\allowbreak\token[0.0173]|y|\allowbreak\token[0.0000]|_|\allowbreak\token[0.0000]|2|\allowbreak\token[0.0000]| +|\allowbreak\token[0.0000]| x|\allowbreak\token[0.0000]|_|\allowbreak\token[0.0000]|2|\allowbreak\token[0.0000]|y|\allowbreak\token[0.0000]|_|\allowbreak\token[0.0000]|3|\allowbreak\token[0.0000]| +|\allowbreak\token[0.0000]| x|\allowbreak\token[0.0000]|_|\allowbreak\token[0.0000]|3|\allowbreak\token[0.0000]|y|\allowbreak\token[0.0000]|_|\allowbreak\token[0.0000]|4|\allowbreak\token[0.0000]| +|\allowbreak\token[0.0000]| x|\allowbreak\token[0.0000]|_|\allowbreak\token[0.0000]|4|\allowbreak\token[0.0000]|y|\allowbreak\token[0.0000]|_|\allowbreak\token[0.0001]|1|\allowbreak\token[0.0000]| -|\allowbreak\token[0.0817]| (|\allowbreak\token[0.0000]|y|\allowbreak\token[0.0000]|_|\allowbreak\token[0.0000]|1|\allowbreak\token[0.0000]|x|\allowbreak\token[0.0000]|_|\allowbreak\token[0.0000]|2|\allowbreak\token[0.0000]| +|\allowbreak\token[0.0000]| y|\allowbreak\token[0.0000]|_|\allowbreak\token[0.0000]|2|\allowbreak\token[0.0000]|x|\allowbreak\token[0.0000]|_|\allowbreak\token[0.0000]|3|\allowbreak\token[0.0000]| +|\allowbreak\token[0.0000]| y|\allowbreak\token[0.0000]|_|\allowbreak\token[0.0000]|3|\allowbreak\token[0.0000]|x|\allowbreak\token[0.0000]|_|\allowbreak\token[0.0000]|4|\allowbreak\token[0.0000]| +|\allowbreak\token[0.0000]| y|\allowbreak\token[0.0000]|_|\allowbreak\token[0.0000]|4|\allowbreak\token[0.0000]|x|\allowbreak\token[0.0000]|_|\allowbreak\token[0.0001]|1|\allowbreak\token[0.0000]|)|\allowbreak\token[0.0000]| \|\allowbreak\token[0.0000]|right|\allowbreak\token[0.0206]!|.↩!\newline\allowbreak\token[0.0000]|    |\allowbreak\token[0.0000]| \|\allowbreak\token[0.0000]|]↩|\newline\allowbreak\token[0.0065]|  |\allowbreak\token[0.0000]| -|\allowbreak\token[1.4065]| Sub|\allowbreak\token[0.0000]|stit|\allowbreak\token[0.0000]|uting|\allowbreak\token[0.0524]| the|\allowbreak\token[0.0059]| coordinates|\allowbreak\token[0.8325]| of|\allowbreak\token[0.4472]| \(|\allowbreak\token[0.0002]|D|\allowbreak\token[0.0393]|(|\allowbreak\token[0.0000]|0|\allowbreak\token[0.0000]|,|\allowbreak\token[0.3147]| |\allowbreak\token[0.0000]|3|\allowbreak\token[0.0001]|)\|\allowbreak\token[0.0000]|),|\allowbreak\token[0.0000]| \(|\allowbreak\token[0.0006]|B|\allowbreak\token[0.0000]|(|\allowbreak\token[0.0000]|0|\allowbreak\token[0.0000]|,|\allowbreak\token[0.0003]| |\allowbreak\token[0.0000]|0|\allowbreak\token[0.0000]|)\|\allowbreak\token[0.0000]|),|\allowbreak\token[0.0000]| \(|\allowbreak\token[0.0000]|E|\allowbreak\token[0.0000]|(|\allowbreak\token[0.0000]|4|\allowbreak\token[0.0000]|,|\allowbreak\token[0.0000]| |\allowbreak\token[0.0000]|0|\allowbreak\token[0.0000]|)\|\allowbreak\token[0.0000]|),|\allowbreak\token[0.0362]| and|\allowbreak\token[0.0000]| \(|\allowbreak\token[0.0000]|F|\allowbreak\token[0.0004]|\|\allowbreak\token[0.0003]|left|\allowbreak\token[0.0000]|(\|\allowbreak[...]
\end{twopartbox}
\begin{twopartbox}{Cut Index 924 | Acceptance Probability 0.0002}{Block 6 Step 2}{Proposal (Rejected)}
[...]\allowbreak\token[0.0000]|    |\allowbreak\token[0.0000]| \|\allowbreak\token[0.0000]|]↩|\newline\allowbreak\token[0.0065]|  |\allowbreak\token[0.0000]| -|\allowbreak\token[1.4065]| Sub|\allowbreak\token[0.0000]|stit|\allowbreak\token[0.0000]|uting|\allowbreak\token[0.0524]| the|\allowbreak\token[0.0059]| coordinates|\allowbreak\token[0.8325]| of|\allowbreak\token[0.4472]| \(|\allowbreak\token[0.0002]|D|\allowbreak\token[0.0393]|(|\allowbreak\token[0.0000]|0|\allowbreak\token[0.0000]|,|\allowbreak\token[0.3147]| |\allowbreak\token[0.0000]|3|\allowbreak\token[0.0001]|)\|\allowbreak\token[0.0000]|),|\allowbreak\token[0.0000]| \(|\allowbreak\token[0.0006]|B|\allowbreak\token[0.0000]|(|\allowbreak\token[0.0000]|0|\allowbreak\token[0.0000]|,|\allowbreak\token[0.0003]| |\allowbreak\token[0.0000]|0|\allowbreak\token[0.0000]|)\|\allowbreak\token[0.0000]|),|\allowbreak\token[0.0000]| \(|\allowbreak\token[0.0000]|E|\allowbreak\token[0.0000]|(|\allowbreak\token[0.0000]|4|\allowbreak\token[0.0000]|,|\allowbreak\token[0.0000]| |\allowbreak\token[0.0000]|0|\allowbreak\token[0.0000]|)\|\allowbreak\token[0.0000]|),|\allowbreak\token[0.0362]| and|\allowbreak\token[0.0000]| \(|\allowbreak\token[0.0000]|F|\allowbreak\token[0.0004]|\|\allowbreak\token[0.0003]|left|\allowbreak\token[0.0000]|(\|\allowbreak\token[0.0000]|frac|\allowbreak\token[0.0000]|{|\allowbreak\token[0.0000]|8|\allowbreak\token[0.0000]|}{|\allowbreak\token[0.0000]|3|\allowbreak\token[0.0000]|},|\allowbreak\token[0.0000]| |\allowbreak\token[0.0000]|2|\allowbreak\token[0.0000]|\|\allowbreak\token[0.0012]|right|\allowbreak\token[0.0000]|)\|\allowbreak\token[0.0013]|):↩|\newline\allowbreak\token[0.0000]|    |\allowbreak\token[0.0000]| \|\allowbreak\token[0.0000]|[↩|\newline\allowbreak\token[0.0000]|    |\allowbreak\token[0.0000]| \|\allowbreak\token[0.0000]|text|\allowbreak\token[0.0000]|{|\allowbreak\token[0.0000]|Area|\allowbreak\token[0.0000]|}|\allowbreak\token[0.0000]| =|\allowbreak\token[0.0000]| \|\allowbreak\token[0.0000]|frac|\allowbreak\token[0.0000]|{|\allowbreak\token[0.0000]|1|\allowbreak\token[0.0000]|}{|\allowbreak\token[0.0000]|2|\allowbreak\token[0.0000]|}|\allowbreak\token[0.0000]| \|\allowbreak\token[0.0000]|left|\allowbreak\token[0.0000]!|!\allowbreak\token[0.0001]| |\allowbreak\token[0.0000]|0|\allowbreak\token[0.0194]| \|\allowbreak\token[0.0000]|cdot|\allowbreak\token[0.0000]| |\allowbreak\token[0.0042]|0|\allowbreak\token[0.0000]| +|\allowbreak\token[0.0000]| |\allowbreak\token[0.0002]|0|\allowbreak\token[0.0000]| \|\allowbreak\token[0.0000]|cdot|\allowbreak\token[0.0000]| |\allowbreak\token[0.0069]|0|\allowbreak\token[0.0000]| +|\allowbreak\token[0.0001]| |\allowbreak\token[0.0000]|4|\allowbreak\token[0.0000]| \|\allowbreak\token[0.0000]|cdot|\allowbreak\token[0.0014]| |\allowbreak\token[0.0002]|2|\allowbreak\token[0.0000]| +|\allowbreak\token[0.0000]| \|\allowbreak\token[0.0017]|frac|\allowbreak\token[0.0000]|{|\allowbreak\token[0.0000]|8|\allowbreak\token[0.0000]|}{|\allowbreak\token[0.0000]|3|\allowbreak\token[0.0000]|}|\allowbreak\token[0.0000]| \|\allowbreak\token[0.0000]|cdot|\allowbreak\token[0.0000]| |\allowbreak\token[0.0002]|3|\allowbreak\token[0.0000]| -|\allowbreak\token[0.1204]| (|\allowbreak\token[0.0001]|3|\allowbreak\token[0.0000]| \|\allowbreak\token[0.0000]|cdot|\allowbreak\token[0.0000]| |\allowbreak\token[0.0000]|0|\allowbreak\token[0.0000]| +|\allowbreak\token[0.0000]| |\allowbreak\token[0.0000]|0|\allowbreak\token[0.0000]| \|\allowbreak\token[0.0000]|cdot|\allowbreak\token[0.0019]| |\allowbreak\token[0.0001]|4|\allowbreak\token[0.0000]| +|\allowbreak\token[0.0000]| |\allowbreak\token[0.0000]|0|\allowbreak\token[0.0000]| \|\allowbreak\token[0.0000]|cdot|\allowbreak\token[0.0010]| \|\allowbreak\token[0.0000]|frac|\allowbreak\token[0.0000]|{|\allowbreak\token[0.0000]|8|\allowbreak\token[0.0000]|}{|\allowbreak\token[0.0000]|3|\allowbreak\token[0.0000]|}|\allowbreak\token[0.0000]| +|\allowbreak\token[0.0000]| |\allowbreak\token[0.0001]|2|\allowbreak\token[0.0000]| \|\allowbreak\token[0.0000]|cdot|\allowbreak\token[0.0000]| |\allowbreak\token[0.0000]|0|\allowbreak\token[0.0000]|)|\allowbreak\token[0.0000]| \|\allowbreak\token[0.0000]|right|\allowbreak\cut{[CUT]}\allowbreak\token[0.6952]!|.↩!\newline\allowbreak\token[0.0000]|    |\allowbreak\token[0.0000]| \|\allowbreak\token[0.0000]|]↩|\newline\allowbreak\token[0.5029]|    |\allowbreak\token[0.0901]| \|\allowbreak\token[0.0000]|[↩|\newline\allowbreak\token[0.0000]|    |\allowbreak\token[0.4751]| \|\allowbreak\token[0.0000]|text|\allowbreak\token[0.0000]|{|\allowbreak\token[0.0000]|Area|\allowbreak\token[0.0000]|}|\allowbreak\token[0.0000]| =|\allowbreak\token[0.0000]| \|\allowbreak\token[0.0000]|frac|\allowbreak\token[0.0000]|{|\allowbreak\token[0.0000]|1|\allowbreak\token[0.0000]|}{|\allowbreak\token[0.0000]|2|\allowbreak\token[0.0000]|}|\allowbreak\token[0.0000]| \|\allowbreak\token[0.0000]|left|\allowbreak\token[0.0000]!|!\allowbreak\token[0.0000]| |\allowbreak\token[0.0034]|0|\allowbreak\token[0.0000]| +|\allowbreak\token[0.0000]| |\allowbreak\token[0.0001]|0|\allowbreak\token[0.0000]| +|\allowbreak\token[0.0000]| |\allowbreak\token[0.0000]|8|\allowbreak\token[0.0000]| +|\allowbreak\token[0.0001]| |\allowbreak\token[0.0000]|8|\allowbreak\token[0.0001]| -|\allowbreak\token[0.0015]| (|\allowbreak\token[0.0000]|0|\allowbreak\token[0.0000]| +|\allowbreak\token[0.0000]| |\allowbreak\token[0.0000]|0|\allowbreak\token[0.0000]| +|\allowbreak\token[0.0000]| |\allowbreak\token[0.0000]|0|\allowbreak\token[0.0000]| +|\allowbreak\token[0.0000]| |\allowbreak\token[0.0000]|0|\allowbreak\token[0.0000]|)|\allowbreak\token[0.0000]| \|\allowbreak\token[0.0000]|right|\allowbreak\token[0.1103]!|!\allowbreak\token[0.0000]| =|\allowbreak\token[0.0000]| \|\allowbreak\token[0.0000]|frac|\allowbreak\token[0.0000]|{|\allowbreak\token[0.0000]|1|\allowbreak\token[0.0000]|}{|\allowbreak\token[0.0000]|2|\allowbreak\token[0.0000]|}|\allowbreak\token[0.0002]| \|\allowbreak\token[0.2467]|left|\allowbreak\token[0.0000]!|!\allowbreak\token[0.0000]| |\allowbreak\token[0.0002]|1|\allowbreak\token[0.0000]|6|\allowbreak\token[0.0001]| \|\allowbreak\token[0.0000]|right|\allowbreak\token[0.0000]!|!\allowbreak\token[0.0000]| =|\allowbreak\token[0.5303]| |\allowbreak\token[0.0001]|8|\allowbreak\token[0.0000]|.↩|\newline\allowbreak\token[0.0000]|    |\allowbreak\token[0.0000]| \|\allowbreak\token[0.1095]|]↩↩|\tighttracebreak\allowbreak\token[1.4514]|So|\allowbreak\token[0.3148]|,|\allowbreak\token[0.0000]| the|\allowbreak\token[0.0042]| area|\allowbreak\token[0.0000]| of|\allowbreak\token[0.0001]| quadr|\allowbreak\token[0.0000]|ilateral|\allowbreak\token[0.0000]| \(|\allowbreak\token[0.0000]|DB|\allowbreak\token[0.0000]|EF|\allowbreak\token[0.0000]|\)|\allowbreak\token[0.0000]| is|\allowbreak\token[0.4260]| \|\allowbreak\token[0.0001]|(\|\allowbreak\token[0.0000]|boxed|\allowbreak\token[0.0001]|{|\allowbreak\token[0.0373]|8|\allowbreak\token[0.0004]|}\|\allowbreak\token[0.9174]|).|\allowbreak\token[0.7943]!<|endoftext|>!
\tcblower
\token[0.6951]!|.↩!\newline\allowbreak\token[0.0000]|    |\allowbreak\token[0.0000]| \|\allowbreak\token[0.0000]|]↩|\newline\allowbreak\token[0.5029]|    |\allowbreak\token[0.0739]| \|\allowbreak\token[0.0000]|[↩|\newline\allowbreak\token[0.0000]|    |\allowbreak\token[0.4751]| \|\allowbreak\token[0.0000]|text|\allowbreak\token[0.0000]|{|\allowbreak\token[0.0000]|Area|\allowbreak\token[0.0000]|}|\allowbreak\token[0.0000]| =|\allowbreak\token[0.0000]| \|\allowbreak\token[0.0000]|frac|\allowbreak\token[0.0000]|{|\allowbreak\token[0.0000]|1|\allowbreak\token[0.0000]|}{|\allowbreak\token[0.0000]|2|\allowbreak\token[0.0000]|}|\allowbreak\token[0.0000]| \|\allowbreak\token[0.0000]|left|\allowbreak\token[0.0000]!|!\allowbreak\token[0.0000]| |\allowbreak\token[0.0034]|0|\allowbreak\token[0.0000]| +|\allowbreak\token[0.0000]| |\allowbreak\token[0.0001]|0|\allowbreak\token[0.0000]| +|\allowbreak\token[0.0000]| |\allowbreak\token[0.0000]|8|\allowbreak\token[0.0000]| +|\allowbreak\token[0.0001]| |\allowbreak\token[0.0000]|8|\allowbreak\token[0.0001]| -|\allowbreak\token[0.0015]| (|\allowbreak\token[0.0000]|0|\allowbreak\token[0.0000]| +|\allowbreak\token[0.0000]| |\allowbreak\token[0.0000]|0|\allowbreak\token[0.0000]| +|\allowbreak\token[0.0000]| |\allowbreak\token[0.0000]|0|\allowbreak\token[0.0000]| +|\allowbreak\token[0.0000]| |\allowbreak\token[0.0000]|0|\allowbreak\token[0.0000]|)|\allowbreak\token[0.0000]| \|\allowbreak\token[0.0000]|right|\allowbreak\token[0.1103]!|!\allowbreak\token[0.0000]| =|\allowbreak\token[0.0000]| \|\allowbreak\token[0.0000]|frac|\allowbreak\token[0.0000]|{|\allowbreak\token[0.0000]|1|\allowbreak\token[0.0000]|}{|\allowbreak\token[0.0000]|2|\allowbreak\token[0.0000]|}|\allowbreak\token[0.0002]| \|\allowbreak\token[0.2467]|left|\allowbreak\token[0.0000]!|!\allowbreak\token[0.0000]| |\allowbreak\token[0.0002]|1|\allowbreak\token[0.0000]|6|\allowbreak\token[0.0001]| \|\allowbreak\token[0.0000]|right|\allowbreak\token[0.0000]!|!\allowbreak\token[0.0000]| =|\allowbreak\token[0.5303]| |\allowbreak\token[0.0001]|8|\allowbreak\token[0.0000]|.↩|\newline\allowbreak\token[0.0000]|    |\allowbreak\token[0.0000]| \|\allowbreak\token[0.1095]|]↩↩|\tighttracebreak\allowbreak\token[1.4514]|Let|\allowbreak\token[0.0001]|'s|\allowbreak\token[1.1383]| confirm|\allowbreak\token[0.2967]| this|\allowbreak\token[0.4864]| with|\allowbreak\token[0.2692]| Python|\allowbreak\token[0.4752]| code|\allowbreak\token[1.3857]|.↩|\newline\allowbreak\token[0.0000]|```|\allowbreak\token[0.0000]|python|\allowbreak\token[0.0000]|↩|\newline\allowbreak\token[0.2174]|#|\allowbreak\token[0.8213]| Coordinates|\allowbreak\token[0.0000]| of|\allowbreak\token[0.4824]| points|\allowbreak\token[1.0263]|↩|\newline\allowbreak\token[0.0091]|D|\allowbreak\token[0.0000]| =|\allowbreak\token[0.0000]| (|\allowbreak\token[0.0000]|0|
\end{twopartbox}
\begin{twopartbox}{Cut Index 174 | Acceptance Probability 0.0000}{Block 6 Step 3}{Proposal (Rejected)}
[...]\allowbreak\token[0.4029]|Ident|\allowbreak\token[0.0010]|ify|\allowbreak\token[0.6278]| Coordinates|\allowbreak\token[1.0343]| of|\allowbreak\token[0.7862]| Points|\allowbreak\token[1.0859]|:|\allowbreak\token[0.0014]|**↩|\newline\allowbreak\token[0.2277]|  |\allowbreak\token[0.0177]| -|\allowbreak\token[1.1397]| Point|\allowbreak\token[0.0003]| \(|\allowbreak\token[0.4684]|A|\allowbreak\token[0.2856]|\)|\allowbreak\token[0.1606]| is|\allowbreak\token[0.6652]| \|\allowbreak\token[0.0001]|((|\allowbreak\token[0.0000]|0|\allowbreak\token[0.0000]|,|\allowbreak\token[0.0901]| |\allowbreak\token[0.0000]|6|\allowbreak\token[0.0000]|)\|\allowbreak\token[0.3752]|).↩|\newline\allowbreak\token[0.0000]|  |\allowbreak\token[0.0000]| -|\allowbreak\token[0.0001]| Point|\allowbreak\token[0.0000]| \(|\allowbreak\token[0.0005]|B|\allowbreak\token[0.0000]|\)|\allowbreak\token[0.0000]| is|\allowbreak\token[0.0001]| \|\allowbreak\token[0.0000]|((|\allowbreak\token[0.0000]|0|\allowbreak\token[0.0000]|,|\allowbreak\token[0.0001]| |\allowbreak\token[0.0000]|0|\allowbreak\token[0.0000]|)\|\allowbreak\token[0.0002]|).↩|\newline\allowbreak\token[0.0000]|  |\allowbreak\token[0.0000]| -|\allowbreak\token[0.0000]| Point|\allowbreak\token[0.0000]| \(|\allowbreak\token[0.0001]|C|\allowbreak\token[0.0000]|\)|\allowbreak\token[0.0000]| is|\allowbreak\token[0.0000]| \|\allowbreak\token[0.0001]|((|\allowbreak\token[0.0000]|8|\allowbreak\token[0.0000]|,|\allowbreak\token[0.0000]| |\allowbreak\token[0.0000]|0|\allowbreak\token[0.0000]|)\|\allowbreak\token[0.0820]|).↩|\newline\allowbreak\token[0.0053]|  |\allowbreak\token[0.0000]| -|\allowbreak\token[0.7278]| Point|\allowbreak\token[0.0000]| \(|\allowbreak\token[0.0022]|D|\allowbreak\token[0.4194]|\)|\allowbreak\token[0.0197]| is|\allowbreak\token[0.0001]| the|\allowbreak\token[0.0000]| midpoint|\allowbreak\token[0.0000]| of|\allowbreak\token[0.0496]| \|\allowbreak\token[0.0004]|(\|\allowbreak\token[0.0000]|over|\allowbreak\token[0.0000]|line|\allowbreak\token[0.0000]|{|\allowbreak\token[0.0000]|AB|\allowbreak\token[0.0001]|}\|\allowbreak\token[0.0147]|),|\allowbreak\token[0.0025]| so|\allowbreak\token[0.4478]| \(|\allowbreak\token[0.0000]|D|\allowbreak\token[0.6075]|\)|\allowbreak\token[0.0034]| is|\allowbreak\token[0.0002]| \|\allowbreak\token[0.0905]|((|\allowbreak\token[0.0102]|0|\allowbreak\token[0.0000]|,|\allowbreak\token[0.0240]| |\allowbreak\token[0.0000]|3|\allowbreak\token[0.0000]|)\|\allowbreak\token[0.0612]|).↩|\newline\allowbreak\token[0.0000]|  |\allowbreak\token[0.0000]| -|\allowbreak\token[0.0000]| Point|\allowbreak\token[0.0000]| \(|\allowbreak\token[0.0002]|E|\allowbreak\token[0.0001]|\)|\allowbreak\token[0.0000]| is|\allowbreak\token[0.0000]| the|\allowbreak\token[0.0000]| midpoint|\allowbreak\token[0.0000]| of|\allowbreak\token[0.0000]| \|\allowbreak\token[0.0000]|(\|\allowbreak\token[0.0002]|over|\allowbreak\token[0.0001]|line|\allowbreak\token[0.0000]|{|\allowbreak\token[0.0002]|BC|\allowbreak\token[0.0000]|}\|\allowbreak\token[0.0000]|),|\allowbreak\token[0.0000]| so|\allowbreak\token[0.0000]| \(|\allowbreak\token[0.0001]|E|\allowbreak\token[0.0002]|\)|\allowbreak\token[0.0000]| is|\allowbreak\token[0.0000]| \|\allowbreak\token[0.0007]|((|\allowbreak\token[0.0000]|4|\allowbreak\token[0.0000]|,|\allowbreak\token[0.0000]| |\allowbreak\token[0.0001]|0|\allowbreak\token[0.0000]|)\|\allowbreak\token[0.3933]|).↩|\newline\allowbreak\token[0.0175]|  |\allowbreak\token[0.0000]| -|\allowbreak\token[0.1774]| Point|\allowbreak\token[0.0000]| \(|\allowbreak\token[0.0007]|F|\allowbreak\token[0.0040]|\)|\allowbreak\token[0.2097]| is|\allowbreak\token[1.1429]| the|\allowbreak\token[0.5879]| intersection|\allowbreak\token[0.0027]| of|\allowbreak\token[1.0073]| \|\allowbreak\token[0.0017]|(\|\allowbreak\token[0.0001]|over|\allowbreak\token[0.0006]|line|\allowbreak\token[0.0001]|{|\allowbreak\token[1.0715]|AE|\allowbreak\token[0.0002]|}\|\allowbreak\token[0.0001]|)|\allowbreak\token[0.0032]| and|\allowbreak\token[0.0048]| \|\allowbreak\token[0.0000]|(\|\allowbreak\token[0.0001]|over|\allowbreak\token[0.0002]|line|\allowbreak\token[0.0000]|{|\allowbreak\token[0.2372]|CD|\allowbreak\token[0.0001]|}\|\allowbreak\token[0.6869]|).↩↩|\tighttracebreak\allowbreak\token[0.0003]|2|\allowbreak\token[0.0000]|.|\allowbreak\token[0.0000]| **|\allowbreak\cut{[CUT]}\allowbreak\token[0.1501]|Find|\allowbreak\token[0.2495]| the|\allowbreak\token[1.1147]| Equation|\allowbreak\token[0.0000]| of|\allowbreak\token[0.7128]| Line|\allowbreak\token[0.4501]| \(|\allowbreak\token[0.0092]|AE|\allowbreak\token[0.0294]|\|\allowbreak\token[0.0004]|):|\allowbreak\token[0.0001]|**↩|\newline\allowbreak\token[0.0000]|  |\allowbreak\token[0.5393]| -|\allowbreak\token[0.9401]| The|\allowbreak\token[0.1343]| slope|\allowbreak\token[0.0081]| of|\allowbreak\token[0.4686]| \(|\allowbreak\token[0.0004]|AE|\allowbreak\token[0.0000]|\)|\allowbreak\token[0.0001]| is|\allowbreak\token[0.0008]| \|\allowbreak\token[0.0070]|(\|\allowbreak\token[0.0001]|frac|\allowbreak\token[0.0003]|{|\allowbreak\token[0.6822]|6|\allowbreak\token[0.6759]|-|\allowbreak\token[0.0001]|0|\allowbreak\token[0.0000]|}{|\allowbreak\token[0.0001]|0|\allowbreak\token[0.0000]|-|\allowbreak\token[0.0043]|4|\allowbreak\token[0.0002]|}|\allowbreak\token[0.0001]| =|\allowbreak\token[0.0007]| -|\allowbreak\token[0.0069]|\|\allowbreak\token[0.0000]|frac|\allowbreak\token[0.0000]|{|\allowbreak\token[0.0042]|3|\allowbreak\token[0.0000]|}{|\allowbreak\token[0.0000]|2|\allowbreak\token[0.0000]|}\|\allowbreak\token[0.0334]|).↩|\newline\allowbreak\token[0.0000]|  |\allowbreak\token[0.0000]| -|\allowbreak\token[0.0201]| The|\allowbreak\token[0.0004]| equation|\allowbreak\token[0.0008]| of|\allowbreak\token[0.6106]| the|\allowbreak\token[0.0000]| line|\allowbreak\token[0.2897]| \(|\allowbreak\token[0.0001]|AE|\allowbreak\token[0.0003]|\)|\allowbreak\token[0.8348]| using|\allowbreak\token[0.6632]| the|\allowbreak\token[0.0042]| point|\allowbreak\token[0.0033]|-s|\allowbreak\token[0.0000]|lope|\allowbreak\token[0.0038]| form|\allowbreak\token[0.6977]| \(|\allowbreak\token[0.0126]|y|\allowbreak\token[0.0252]| -|\allowbreak\token[0.0020]| y|\allowbreak\token[0.0001]|_|\allowbreak\token[0.0001]|1|\allowbreak\token[0.0000]| =|\allowbreak\token[0.0000]| m|\allowbreak\token[0.0001]|(x|\allowbreak\token[0.0000]| -|\allowbreak\token[0.0000]| x|\allowbreak\token[0.0000]|_|\allowbreak\token[0.0000]|1|\allowbreak\token[0.0000]|)\|\allowbreak\token[0.0093]|)|\allowbreak\token[0.0487]| is|\allowbreak\token[0.6953]|:↩|\newline\allowbreak\token[0.0005]|    |\allowbreak\token[0.0007]| \|\allowbreak\token[0.0059]|[↩|\newline\allowbreak\token[0.0000]|    |\allowbreak\token[0.0000]| y|\allowbreak\token[0.5570]| -|\allowbreak\token[0.0000]| |\allowbreak\token[0.6855]|0|\allowbreak\token[0.0000]| =|\allowbreak\token[0.0000]| -|\allowbreak\token[0.0000]|\|\allowbreak\token[0.0000]|frac|\allowbreak\token[0.0000]|{|\allowbreak\token[0.0000]|3|\allowbreak\token[0.0000]|}{|\allowbreak\token[0.0000]|2|\allowbreak\token[0.0003]|}(|\allowbreak\token[0.0000]|x|\allowbreak\token[0.0000]| -|\allowbreak\token[0.0000]| |\allowbreak\token[0.0446]|4|\allowbreak\token[0.1458]|)|\allowbreak\token[0.0000]| \|\allowbreak\token[0.0219]|impl|\allowbreak\token[0.0000]|ies|\allowbreak\token[0.0000]| y|\allowbreak\token[0.0000]| =|\allowbreak\token[0.0000]| -|\allowbreak\token[0.0000]|\|\allowbreak\token[0.0000]|frac|\allowbreak\token[0.0000]|{|\allowbreak\token[0.0000]|3|\allowbreak\token[0.0000]|}{|\allowbreak\token[0.0000]|2|\allowbreak\token[0.0000]|}|\allowbreak\token[0.0000]|x|\allowbreak\token[0.0000]| +|\allowbreak\token[0.0000]| |\allowbreak\token[0.0000]|6|\allowbreak\token[0.6933]|.↩|\newline\allowbreak\token[0.0025]|    |\allowbreak\token[0.0000]| \|\allowbreak\token[0.0216]|]↩↩|\tighttracebreak\allowbreak\token[0.0000]|3|\allowbreak\token[0.0000]|.|\allowbreak\token[0.0000]| **|\allowbreak\token[0.0000]|Find|\allowbreak\token[0.0000]| the|\allowbreak\token[0.0008]| Equation|\allowbreak\token[0.0000]| of|\allowbreak\token[0.0000]| Line|\allowbreak\token[0.0000]| \(|\allowbreak\token[0.0000]|CD|\allowbreak\token[0.0000]|\|\allowbreak\token[0.0000]|):|\allowbreak\token[0.0000]|**↩|\newline\allowbreak\token[0.0000]|  |\allowbreak\token[0.0000]| -|\allowbreak\token[0.0001]| The|\allowbreak\token[0.0006]| slope|\allowbreak\token[0.0000]| of|\allowbreak\token[0.0000]| \(|\allowbreak\token[0.0000]|CD|\allowbreak\token[0.0000]|\)|\allowbreak\token[0.0000]| is|\allowbreak\token[0.0000]| \|\allowbreak\token[0.0000]|(\|\allowbreak\token[0.0000]|frac|\allowbreak\token[0.0002]|{|\allowbreak\token[0.2291]|3|\allowbreak[...]
\tcblower
\token[0.1827]|Find|\allowbreak\token[0.2476]| the|\allowbreak\token[1.1289]| Coordinates|\allowbreak\token[0.0000]| of|\allowbreak\token[0.4208]| Point|\allowbreak\token[0.0010]| \(|\allowbreak\token[0.0001]|F|\allowbreak\token[0.0005]|\|\allowbreak\token[0.0001]|):|\allowbreak\token[0.0073]|**↩|\newline\allowbreak\token[0.0001]|  |\allowbreak\token[0.1734]| -|\allowbreak\token[0.6715]| The|\allowbreak\token[0.1775]| equation|\allowbreak\token[0.0000]| of|\allowbreak\token[0.0107]| line|\allowbreak\token[0.6678]| \|\allowbreak\token[0.0000]|(\|\allowbreak\token[0.0000]|over|\allowbreak\token[0.0002]|line|\allowbreak\token[0.0000]|{|\allowbreak\token[0.0031]|AE|\allowbreak\token[0.0001]|}\|\allowbreak\token[0.2713]|)|\allowbreak\token[0.2488]| can|\allowbreak\token[0.0000]| be|\allowbreak\token[0.7698]| determined|\allowbreak\token[0.3060]| using|\allowbreak\token[0.6635]| points|\allowbreak\token[0.0019]| \(|\allowbreak\token[0.0000]|A|\allowbreak\token[0.0141]|(|\allowbreak\token[0.0000]|0|\allowbreak\token[0.0000]|,|\allowbreak\token[0.0446]| |\allowbreak\token[0.0000]|6|\allowbreak\token[0.0000]|)\|\allowbreak\token[0.0000]|)|\allowbreak\token[0.0000]| and|\allowbreak\token[0.0000]| \(|\allowbreak\token[0.0000]|E|\allowbreak\token[0.0000]|(|\allowbreak\token[0.0000]|4|\allowbreak\token[0.0000]|,|\allowbreak\token[0.0000]| |\allowbreak\token[0.0000]|0|\allowbreak\token[0.0000]|)\|\allowbreak\token[0.4889]|).↩|\newline\allowbreak\token[0.1456]|  |\allowbreak\token[0.0027]| -|\allowbreak\token[0.0000]| The|\allowbreak\token[0.0266]| equation|\allowbreak\token[0.0000]| of|\allowbreak\token[0.0001]| line|\allowbreak\token[0.0000]| \|\allowbreak\token[0.0000]|(\|\allowbreak\token[0.0001]|over|\allowbreak\token[0.0000]|line|\allowbreak\token[0.0000]|{|\allowbreak\token[0.0001]|CD|\allowbreak\token[0.0000]|}\|\allowbreak\token[0.0000]|)|\allowbreak\token[0.0000]| can|\allowbreak\token[0.0000]| be|\allowbreak\token[0.0000]| determined|\allowbreak\token[0.0000]| using|\allowbreak\token[0.0000]| points|\allowbreak\token[0.0000]| \(|\allowbreak\token[0.0001]|C|\allowbreak\token[0.0000]|(|\allowbreak\token[0.0000]|8|\allowbreak\token[0.0000]|,|\allowbreak\token[0.0001]| |\allowbreak\token[0.0000]|0|\allowbreak\token[0.0000]|)\|\allowbreak\token[0.0000]|)|\allowbreak\token[0.0000]| and|\allowbreak\token[0.0000]| \(|\allowbreak\token[0.0000]|D|\allowbreak\token[0.0001]|(|\allowbreak\token[0.0000]|0|\allowbreak\token[0.0000]|,|\allowbreak\token[0.0000]| |\allowbreak\token[0.0000]|3|\allowbreak\token[0.0000]|)\|\allowbreak\token[0.5575]|).↩↩|\tighttracebreak\allowbreak\token[0.0430]|3|\allowbreak\token[0.0001]|.|\allowbreak\token[0.0000]| **|\allowbreak\token[0.5989]|Calculate|\allowbreak\token[0.0381]| the|\allowbreak\token[0.4714]| Area|\allowbreak\token[0.0031]| of|\allowbreak\token[0.0001]| Quadr|\allowbreak\token[0.0000]|ilateral|\allowbreak\token[0.0001]| \(|\allowbreak\token[0.0000]|DB|\allowbreak\token[0.0000]|EF|\allowbreak\token[0.0066]|\|\allowbreak\token[0.0007]|):|\allowbreak\token[0.0003]|**↩|\newline\allowbreak\token[0.0000]|  |\allowbreak\token[0.0016]| -|\allowbreak\token[0.4247]| Use|\allowbreak\token[0.0000]| the|\allowbreak\token[0.1423]| coordinates|\allowbreak\token[0.0016]| of|\allowbreak\token[0.1912]| points|\allowbreak\token[0.0000]| \(|\allowbreak\token[0.0000]|D|\allowbreak\token[0.4033]|\|\allowbreak\token[0.0001]|),|\allowbreak\token[0.0000]| \(|\allowbreak\token[0.0011]|B|\allowbreak\token[0.0000]|\|\allowbreak\token[0.0000]|),|\allowbreak\token[0.0001]| \(|\allowbreak\token[0.0006]|E|\allowbreak\token[0.0000]|\|\allowbreak\token[0.0000]|),|\allowbreak\token[0.0000]| and|\allowbreak\token[0.0001]| \(|\allowbreak\token[0.0000]|F|\allowbreak\token[0.0446]|\)|\allowbreak\token[0.2343]| to|\allowbreak\token[1.0831]| calculate|\allowbreak\token[0.0000]| the|\allowbreak\token[0.0000]| area|\allowbreak\token[0.4503]| using|\allowbreak\token[0.0001]| the|\allowbreak\token[0.5315]| sho|\allowbreak\token[0.0045]|el|\allowbreak\token[0.0000]|ace|\allowbreak\token[0.0000]| formula|\allowbreak\token[0.2150]|.↩↩|\tighttracebreak\allowbreak\token[0.0265]|Let|\allowbreak\token[0.0000]|'s|\allowbreak\token[0.5355]| implement|\allowbreak\token[0.0569]| this|\allowbreak\token[0.6857]| step|\allowbreak\token[0.0113]|-by|\allowbreak\token[0.0000]|-step|\allowbreak\token[0.0252]| in|\allowbreak\token[0.0000]| Python|\allowbreak[...]
\end{twopartbox}

\endgroup

\subsection{Math500 Question 492}
\paragraph{Question.}
Tom got a Mr. Potato Head for his birthday. It came with 3 hairstyles, 2 sets of eyebrows, 1 pair of googly eyes, 2 sets of ears, and 2 sets of lips, a pair of regular shoes, and a bonus pair of clown shoes. If a complete Mr. Potato Head personality includes eyebrows, eyes, ears, lips, shoes and optionally hair, how many different wacky personalities can Tom come up with? Note that Mr. Potato Head can be bald.

Note: You cannot ``mix and match''.  For example, you cannot take the left eyebrow from one pair and the right eyebrow from the other pair.

\paragraph{Correct Answer.}
$64$.

\subsubsection{Solutions}
\fzedit{In the solutions, uniform-cut MH overcounts because it splits the hair decision into two separate factors: it first counts ``3 hairstyles'' and then introduces an additional binary ``hair versus bald'' factor, which leads to \(192\). Entropy-cut MH recognizes that this is the wrong factorization. Hair is a single four-way choice---three hairstyles or bald---so the correct count is
\[
2 \times 2 \times 2 \times 2 \times 4 = 64.
\]
Thus entropy-cut MH uses the correct combinatorial decomposition instead of counting the same decision twice.}

\paragraph{Uniform-Cut Solution.}
\begin{markdown}
 To solve this problem, we need to consider each component of the Mr. Potato Head's personality and calculate the total number of different combinations. 

1. Head Styles: There are 3 hairstyles available.
2. Eyebrows: There are 2 sets of eyebrows.
3. Eyes: There is 1 pair of googly eyes.
4. Ears: There are 2 sets of ears.
5. Lips: There are 2 sets of lips.
6. Shoes: There are 2 pairs of shoes (regular and clown shoes).
7. Optional Hair: Mr. Potato Head can be bald, so we need to consider the bald option as well.

Let's calculate the total number of combinations for each component:

- Head Styles: 3 options
- Eyebrows: 2 options
- Eyes: 1 option (fixed)
- Ears: 2 options
- Lips: 2 options
- Shoes: 2 options
- Hair: 2 options (including bald)

To find the total number of different personalities, we multiply the number of options for each component:

\[
3 \text{ (head styles)} \times 2 \text{ (eyebrows)} \times 1 \text{ (eyes)} \times 2 \text{ (ears)} \times 2 \text{ (lips)} \times 2 \text{ (shoes)} \times 2 \text{ (hair)}
\]

\[
= 3 \times 2 \times 1 \times 2 \times 2 \times 2 \times 2
\]

\[
= 3 \times 2^6
\]

\[
= 3 \times 64
\]

\[
= 192
\]

Therefore, Tom can come up with 192 different wacky personalities for his Mr. Potato Head.

So, the final answer is:

\[
\boxed{192}
\]
\end{markdown}

\paragraph{Entropy-Cut Solution.}
\begin{markdown}
 To solve this problem, we need to calculate the total number of different personalities Tom can create by combining the available options for eyebrows, ears, lips, shoes, and optionally hair. Here's the breakdown of the options:

1. **Eyebrows**: 2 sets (left and right)
2. **Ears**: 2 sets (left and right)
3. **Lips**: 2 sets (left and right)
4. **Shoes**: 2 pairs (regular and clown)
5. **Hair**: 3 hairstyles (including bald)

Since Tom can choose to not include hair, we need to consider the bald option as well. This means there are 4 choices for hair (3 hairstyles + bald).

To find the total number of different personalities, we multiply the number of choices for each option:

\[
\text{Total Personalities} = (\text{Eyebrows}) \times (\text{Ears}) \times (\text{Lips}) \times (\text{Shoes}) \times (\text{Hair})
\]

Let's calculate this using Python.
```
# Define the number of options for each feature
eyebrows = 2  # left and right
ears = 2      # left and right
lips = 2      # left and right
shoes = 2     # regular and clown
hair = 4      # 3 hairstyles + bald

# Calculate the total number of different personalities
total_personalities = eyebrows * ears * lips * shoes * hair
print(total_personalities)
```
```
64
```
The total number of different wacky personalities Tom can come up with is \(\boxed{64}\).
\end{markdown}

\subsubsection{Reasoning Trace}
\fzedit{The decisive transition for entropy-cut MH occurs at \emph{Block 1 Step 2}. At that point, an accepted proposal rewrites the hair choice correctly as \(4\) possibilities, namely \(3\) hairstyles plus bald, rather than treating ``head styles'' and ``optional hair'' as separate multiplicative factors. \emph{Block 2 Step 2} then preserves that corrected representation. 

By contrast, uniform-cut MH commits early to the inflated decomposition and keeps multiplying the wrong factors throughout the trace, which is why it ends at the overcounted answer \(192\).}

\paragraph{Uniform-Cut Trace.}
\begingroup
\raggedright %
\setlength{\lineskip}{0pt} %
\setlength{\baselineskip}{0.9em} %
\begin{twopartbox}{Cut Index 43 | Acceptance Probability 1.0000}{Block 1 Step 1}{Proposal (Accepted)}
\token[2.7697]| To|\allowbreak\token[1.0000]| solve|\allowbreak\token[0.4562]| this|\allowbreak\token[0.1579]| problem|\allowbreak\token[0.0266]|,|\allowbreak\token[0.2968]| we|\allowbreak\token[0.8181]| need|\allowbreak\token[0.0021]| to|\allowbreak\token[1.5837]| consider|\allowbreak\token[1.1462]| each|\allowbreak\token[2.4073]| component|\allowbreak\token[1.0419]| of|\allowbreak\token[0.7540]| the|\allowbreak\token[0.5726]| Mr|\allowbreak\token[0.0017]|.|\allowbreak\token[0.0014]| Potato|\allowbreak\token[0.0055]| Head|\allowbreak\token[1.5594]|'s|\allowbreak\token[0.2376]| personality|\allowbreak\token[1.2982]| and|\allowbreak\token[1.6177]| calculate|\allowbreak\token[0.3197]| the|\allowbreak\token[0.9583]| total|\allowbreak\token[0.0786]| number|\allowbreak\token[0.0004]| of|\allowbreak\token[1.5886]| different|\allowbreak\token[1.2403]| combinations|\allowbreak\token[1.9139]|.|\allowbreak\token[1.7905]| ↩↩|\tighttracebreak\allowbreak\token[1.0835]|1|\allowbreak\token[0.0004]|.|\allowbreak\token[2.2669]| Head|\allowbreak\token[3.1746]| Styles|\allowbreak\token[1.1827]|:|\allowbreak\token[1.4748]| There|\allowbreak\token[0.1503]| are|\allowbreak\token[0.1852]| |\allowbreak\token[0.1497]|3|\allowbreak\token[1.4805]| hairstyles|\allowbreak\token[2.0824]| available|\allowbreak\token[1.8707]|.↩|\newline\allowbreak\token[0.6127]|2|\allowbreak\token[0.0012]|.|\allowbreak\cut{[CUT]}\allowbreak\token[0.4255]| Ey|\allowbreak\token[0.0342]|eb|\allowbreak\token[0.6195]|rows|\allowbreak\token[0.0937]|:|\allowbreak\token[0.0824]| There|\allowbreak\token[0.0031]| are|\allowbreak\token[0.0035]| |\allowbreak\token[0.0108]|2|\allowbreak\token[0.3917]| sets|\allowbreak\token[0.0334]| of|\allowbreak\token[0.0052]| eyebrows|\allowbreak\token[1.1511]|.↩|\newline\allowbreak\token[0.0021]|3|\allowbreak\token[0.0009]|.|\allowbreak\token[0.4528]| Eyes|\allowbreak\token[0.0143]|:|\allowbreak\token[0.0903]| There|\allowbreak\token[0.2423]| is|\allowbreak\token[0.3295]| |\allowbreak\token[0.0001]|1|\allowbreak\token[0.1684]| pair|\allowbreak\token[0.0006]| of|\allowbreak\token[0.4132]| goog|\allowbreak\token[0.0016]|ly|\allowbreak\token[0.0025]| eyes|\allowbreak\token[0.5950]|.↩|\newline\allowbreak\token[0.0009]|4|\allowbreak\token[0.0006]|.|\allowbreak\token[0.0067]| E|\allowbreak\token[0.0072]|ars|\allowbreak\token[0.0063]|:|\allowbreak\token[0.0004]| There|\allowbreak\token[0.0009]| are|\allowbreak\token[0.0003]| |\allowbreak\token[0.0002]|2|\allowbreak\token[0.0061]| sets|\allowbreak\token[0.0014]| of|\allowbreak\token[0.0034]| ears|\allowbreak\token[0.0193]|.↩|\newline\allowbreak\token[0.0003]|5|\allowbreak\token[0.0008]|.|\allowbreak\token[0.0305]| Lips|\allowbreak\token[0.0018]|:|\allowbreak\token[0.0005]| There|\allowbreak\token[0.0026]| are|\allowbreak\token[0.0005]| |\allowbreak\token[0.0006]|2|\allowbreak\token[0.0011]| sets|\allowbreak\token[0.0020]| of|\allowbreak\token[0.0123]| lips|\allowbreak\token[0.0184]|.↩|\newline\allowbreak\token[0.0007]|6|\allowbreak\token[0.0005]|.|\allowbreak\token[0.0455]| Shoes|\allowbreak\token[0.0022]|:|\allowbreak\token[0.0728]| There|\allowbreak\token[0.3181]| are|\allowbreak\token[0.2061]| |\allowbreak\token[0.1718]|2|\allowbreak\token[1.3246]| pairs|\allowbreak\token[0.0109]| of|\allowbreak\token[0.1063]| shoes|\allowbreak\token[1.2761]| (|\allowbreak\token[0.8206]|regular|\allowbreak\token[0.2870]| and|\allowbreak\token[0.0418]| clown|\allowbreak\token[0.8624]| shoes|\allowbreak\token[0.7760]|).↩↩|\tighttracebreak\allowbreak\token[2.1877]|Since|\allowbreak\token[1.7500]| each|\allowbreak\token[1.1524]| component|\allowbreak\token[1.6482]| can|\allowbreak\token[0.5087]| be|\allowbreak\token[2.2752]| chosen|\allowbreak\token[0.2479]| independently|\allowbreak\token[0.9342]|,|\allowbreak\token[0.4859]| we|\allowbreak\token[1.5549]| multiply|\allowbreak\token[0.0039]| the|\allowbreak\token[0.0487]| number|\allowbreak\token[0.0002]| of|\allowbreak\token[0.6275]| choices|\allowbreak\token[0.2177]| for|\allowbreak\token[0.0001]| each|\allowbreak\token[0.4466]| component|\allowbreak\token[0.3947]| to|\allowbreak\token[0.6391]| find|\allowbreak\token[0.0014]| the|\allowbreak\token[0.0001]| total|\allowbreak\token[0.0012]| number|\allowbreak\token[0.0002]| of|\allowbreak\token[0.4270]| different|\allowbreak\token[0.3621]| personalities|\allowbreak\token[0.8833]|.↩↩|\tighttracebreak\allowbreak\token[2.5851]|\|\allowbreak\token[0.2150]|[↩|\newline\allowbreak\token[0.1984]|\|\allowbreak\token[0.0108]|text|\allowbreak\token[0.0179]|{|\allowbreak\token[0.0871]|Total|\allowbreak\token[1.4149]| Personal|\allowbreak\token[0.0025]|ities|\allowbreak\token[0.0002]|}|\allowbreak\token[0.0001]| =|\allowbreak\token[0.7098]| \|\allowbreak\token[0.0898]|text|\allowbreak\token[0.0156]|{|\allowbreak\token[0.6181]|Head|\allowbreak\token[0.0404]| Styles|\allowbreak\token[0.0012]|}|\allowbreak\token[0.0003]| \|\allowbreak\token[0.0014]|times|\allowbreak\token[0.0008]| \|\allowbreak\token[0.0025]|text|\allowbreak\token[0.0006]|{|\allowbreak\token[0.0214]|Ey|\allowbreak\token[0.0063]|eb|\allowbreak\token[0.2258]|rows|\allowbreak\token[0.0109]|}|\allowbreak\token[0.0003]| \|\allowbreak\token[0.0003]|times|\allowbreak\token[0.0006]| \|\allowbreak\token[0.0027]|text|\allowbreak\token[0.0012]|{|\allowbreak\token[0.0150]|E|\allowbreak\token[0.2975]|yes|\allowbreak\token[0.0007]|}|\allowbreak\token[0.0002]| \|\allowbreak\token[0.0004]|times|\allowbreak\token[0.0002]| \|\allowbreak\token[0.0005]|text|\allowbreak\token[0.0003]|{|\allowbreak\token[0.0195]|E|\allowbreak\token[0.0045]|ars|\allowbreak\token[0.0002]|}|\allowbreak\token[0.0007]| \|\allowbreak\token[0.0002]|times|\allowbreak\token[0.0001]| \|\allowbreak\token[0.0011]|text|\allowbreak\token[0.0001]|{|\allowbreak\token[0.0015]|L|\allowbreak\token[0.0141]|ips|\allowbreak\token[0.0004]|}|\allowbreak\token[0.0001]| \|\allowbreak[...]
\tcblower
\token[0.3981]| Ey|\allowbreak\token[0.0355]|eb|\allowbreak\token[0.5949]|rows|\allowbreak\token[0.1011]|:|\allowbreak\token[0.0840]| There|\allowbreak\token[0.0031]| are|\allowbreak\token[0.0037]| |\allowbreak\token[0.0108]|2|\allowbreak\token[0.3989]| sets|\allowbreak\token[0.0339]| of|\allowbreak\token[0.0052]| eyebrows|\allowbreak\token[1.1750]|.↩|\newline\allowbreak\token[0.0020]|3|\allowbreak\token[0.0008]|.|\allowbreak\token[0.4501]| Eyes|\allowbreak\token[0.0108]|:|\allowbreak\token[0.0906]| There|\allowbreak\token[0.1947]| is|\allowbreak\token[0.3304]| |\allowbreak\token[0.0000]|1|\allowbreak\token[0.1679]| pair|\allowbreak\token[0.0006]| of|\allowbreak\token[0.3842]| goog|\allowbreak\token[0.0016]|ly|\allowbreak\token[0.0025]| eyes|\allowbreak\token[0.5950]|.↩|\newline\allowbreak\token[0.0009]|4|\allowbreak\token[0.0005]|.|\allowbreak\token[0.0068]| E|\allowbreak\token[0.0079]|ars|\allowbreak\token[0.0069]|:|\allowbreak\token[0.0004]| There|\allowbreak\token[0.0008]| are|\allowbreak\token[0.0002]| |\allowbreak\token[0.0002]|2|\allowbreak\token[0.0066]| sets|\allowbreak\token[0.0014]| of|\allowbreak\token[0.0031]| ears|\allowbreak\token[0.0177]|.↩|\newline\allowbreak\token[0.0003]|5|\allowbreak\token[0.0008]|.|\allowbreak\token[0.0295]| Lips|\allowbreak\token[0.0016]|:|\allowbreak\token[0.0006]| There|\allowbreak\token[0.0029]| are|\allowbreak\token[0.0004]| |\allowbreak\token[0.0006]|2|\allowbreak\token[0.0011]| sets|\allowbreak\token[0.0022]| of|\allowbreak\token[0.0137]| lips|\allowbreak\token[0.0185]|.↩|\newline\allowbreak\token[0.0007]|6|\allowbreak\token[0.0005]|.|\allowbreak\token[0.0457]| Shoes|\allowbreak\token[0.0022]|:|\allowbreak\token[0.0727]| There|\allowbreak\token[0.3178]| are|\allowbreak\token[0.2213]| |\allowbreak\token[0.1680]|2|\allowbreak\token[1.3374]| pairs|\allowbreak\token[0.0109]| of|\allowbreak\token[0.1059]| shoes|\allowbreak\token[1.1960]| (|\allowbreak\token[0.8351]|regular|\allowbreak\token[0.2854]| and|\allowbreak\token[0.0432]| clown|\allowbreak\token[0.8250]| shoes|\allowbreak\token[0.7841]|).↩↩|\tighttracebreak\allowbreak\token[2.1364]|Since|\allowbreak\token[1.7572]| each|\allowbreak\token[1.1168]| component|\allowbreak\token[1.6872]| can|\allowbreak\token[0.5085]| be|\allowbreak\token[2.2794]| chosen|\allowbreak\token[0.2479]| independently|\allowbreak\token[0.9810]|,|\allowbreak\token[0.4855]| we|\allowbreak\token[1.5553]| multiply|\allowbreak\token[0.0039]| the|\allowbreak\token[0.0487]| number|\allowbreak\token[0.0002]| of|\allowbreak\token[0.6499]| choices|\allowbreak\token[0.2177]| for|\allowbreak\token[0.0001]| each|\allowbreak\token[0.4447]| component|\allowbreak\token[0.3993]| to|\allowbreak\token[0.6391]| find|\allowbreak\token[0.0012]| the|\allowbreak\token[0.0000]| total|\allowbreak\token[0.0013]| number|\allowbreak\token[0.0002]| of|\allowbreak\token[0.4231]| different|\allowbreak\token[0.3593]| personalities|\allowbreak\token[0.8862]|.↩↩|\tighttracebreak\allowbreak\token[2.5792]|\|\allowbreak\token[0.2151]|[↩|\newline\allowbreak\token[0.1984]|\|\allowbreak\token[0.0103]|text|\allowbreak\token[0.0181]|{|\allowbreak\token[0.0870]|Total|\allowbreak\token[1.3941]| Personal|\allowbreak\token[0.0026]|ities|\allowbreak\token[0.0002]|}|\allowbreak\token[0.0001]| =|\allowbreak\token[0.7137]| \|\allowbreak\token[0.0886]|text|\allowbreak\token[0.0168]|{|\allowbreak\token[0.6421]|Head|\allowbreak\token[0.0390]| Styles|\allowbreak\token[0.0013]|}|\allowbreak\token[0.0003]| \|\allowbreak\token[0.0013]|times|\allowbreak\token[0.0010]| \|\allowbreak\token[0.0026]|text|\allowbreak\token[0.0006]|{|\allowbreak\token[0.0193]|Ey|\allowbreak\token[0.0064]|eb|\allowbreak\token[0.2243]|rows|\allowbreak\token[0.0109]|}|\allowbreak\token[0.0003]| \|\allowbreak\token[0.0004]|times|\allowbreak\token[0.0007]| \|\allowbreak\token[0.0028]|text|\allowbreak\token[0.0011]|{|\allowbreak\token[0.0148]|E|\allowbreak\token[0.2976]|yes|\allowbreak\token[0.0007]|}|\allowbreak\token[0.0002]| \|\allowbreak\token[0.0005]|times|\allowbreak\token[0.0001]| \|\allowbreak\token[0.0005]|text|\allowbreak\token[0.0003]|{|\allowbreak\token[0.0217]|E|\allowbreak\token[0.0040]|ars|\allowbreak\token[0.0002]|}|\allowbreak\token[0.0006]| \|\allowbreak\token[0.0002]|times|\allowbreak\token[0.0001]| \|\allowbreak\token[0.0008]|text|\allowbreak\token[0.0001]|{|\allowbreak\token[0.0017]|L|\allowbreak\token[0.0127]|ips|\allowbreak\token[0.0004]|}|\allowbreak\token[0.0001]| \|\allowbreak[...]
\end{twopartbox}
\begin{twopartbox}{Cut Index 16 | Acceptance Probability 0.0000}{Block 1 Step 2}{Proposal (Rejected)}
\token[2.7697]| To|\allowbreak\token[1.0000]| solve|\allowbreak\token[0.4562]| this|\allowbreak\token[0.1579]| problem|\allowbreak\token[0.0266]|,|\allowbreak\token[0.2968]| we|\allowbreak\token[0.8181]| need|\allowbreak\token[0.0021]| to|\allowbreak\token[1.5837]| consider|\allowbreak\token[1.1462]| each|\allowbreak\token[2.4073]| component|\allowbreak\token[1.0419]| of|\allowbreak\token[0.7540]| the|\allowbreak\token[0.5726]| Mr|\allowbreak\token[0.0017]|.|\allowbreak\token[0.0014]| Potato|\allowbreak\cut{[CUT]}\allowbreak\token[0.0055]| Head|\allowbreak\token[1.5594]|'s|\allowbreak\token[0.2376]| personality|\allowbreak\token[1.2982]| and|\allowbreak\token[1.6177]| calculate|\allowbreak\token[0.3197]| the|\allowbreak\token[0.9583]| total|\allowbreak\token[0.0786]| number|\allowbreak\token[0.0004]| of|\allowbreak\token[1.5886]| different|\allowbreak\token[1.2403]| combinations|\allowbreak\token[1.9139]|.|\allowbreak\token[1.7905]| ↩↩|\tighttracebreak\allowbreak\token[1.0835]|1|\allowbreak\token[0.0004]|.|\allowbreak\token[2.2669]| Head|\allowbreak\token[3.1746]| Styles|\allowbreak\token[1.1827]|:|\allowbreak\token[1.4748]| There|\allowbreak\token[0.1503]| are|\allowbreak\token[0.1852]| |\allowbreak\token[0.1497]|3|\allowbreak\token[1.4805]| hairstyles|\allowbreak\token[2.0824]| available|\allowbreak\token[1.8707]|.↩|\newline\allowbreak\token[0.6127]|2|\allowbreak\token[0.0012]|.|\allowbreak\token[0.3981]| Ey|\allowbreak\token[0.0355]|eb|\allowbreak\token[0.5949]|rows|\allowbreak\token[0.1011]|:|\allowbreak\token[0.0840]| There|\allowbreak\token[0.0031]| are|\allowbreak\token[0.0037]| |\allowbreak\token[0.0108]|2|\allowbreak\token[0.3989]| sets|\allowbreak\token[0.0339]| of|\allowbreak\token[0.0052]| eyebrows|\allowbreak\token[1.1750]|.↩|\newline\allowbreak\token[0.0020]|3|\allowbreak\token[0.0008]|.|\allowbreak\token[0.4501]| Eyes|\allowbreak\token[0.0108]|:|\allowbreak\token[0.0906]| There|\allowbreak\token[0.1947]| is|\allowbreak\token[0.3304]| |\allowbreak\token[0.0000]|1|\allowbreak\token[0.1679]| pair|\allowbreak\token[0.0006]| of|\allowbreak\token[0.3842]| goog|\allowbreak\token[0.0016]|ly|\allowbreak\token[0.0025]| eyes|\allowbreak\token[0.5950]|.↩|\newline\allowbreak\token[0.0009]|4|\allowbreak\token[0.0005]|.|\allowbreak\token[0.0068]| E|\allowbreak\token[0.0079]|ars|\allowbreak\token[0.0069]|:|\allowbreak\token[0.0004]| There|\allowbreak\token[0.0008]| are|\allowbreak\token[0.0002]| |\allowbreak\token[0.0002]|2|\allowbreak\token[0.0066]| sets|\allowbreak\token[0.0014]| of|\allowbreak\token[0.0031]| ears|\allowbreak\token[0.0177]|.↩|\newline\allowbreak\token[0.0003]|5|\allowbreak\token[0.0008]|.|\allowbreak\token[0.0295]| Lips|\allowbreak\token[0.0016]|:|\allowbreak\token[0.0006]| There|\allowbreak\token[0.0029]| are|\allowbreak\token[0.0004]| |\allowbreak\token[0.0006]|2|\allowbreak\token[0.0011]| sets|\allowbreak\token[0.0022]| of|\allowbreak\token[0.0137]| lips|\allowbreak\token[0.0185]|.↩|\newline\allowbreak\token[0.0007]|6|\allowbreak\token[0.0005]|.|\allowbreak\token[0.0457]| Shoes|\allowbreak\token[0.0022]|:|\allowbreak\token[0.0727]| There|\allowbreak\token[0.3178]| are|\allowbreak\token[0.2213]| |\allowbreak\token[0.1680]|2|\allowbreak\token[1.3374]| pairs|\allowbreak\token[0.0109]| of|\allowbreak\token[0.1059]| shoes|\allowbreak\token[1.1960]| (|\allowbreak\token[0.8351]|regular|\allowbreak\token[0.2854]| and|\allowbreak\token[0.0432]| clown|\allowbreak\token[0.8250]| shoes|\allowbreak\token[0.7841]|).↩↩|\tighttracebreak\allowbreak\token[2.1364]|Since|\allowbreak\token[1.7572]| each|\allowbreak\token[1.1168]| component|\allowbreak\token[1.6872]| can|\allowbreak\token[0.5085]| be|\allowbreak\token[2.2794]| chosen|\allowbreak\token[0.2479]| independently|\allowbreak\token[0.9810]|,|\allowbreak\token[0.4855]| we|\allowbreak\token[1.5553]| multiply|\allowbreak\token[0.0039]| the|\allowbreak\token[0.0487]| number|\allowbreak\token[0.0002]| of|\allowbreak\token[0.6499]| choices|\allowbreak\token[0.2177]| for|\allowbreak\token[0.0001]| each|\allowbreak\token[0.4447]| component|\allowbreak\token[0.3993]| to|\allowbreak\token[0.6391]| find|\allowbreak\token[0.0012]| the|\allowbreak\token[0.0000]| total|\allowbreak\token[0.0013]| number|\allowbreak\token[0.0002]| of|\allowbreak\token[0.4231]| different|\allowbreak\token[0.3593]| personalities|\allowbreak\token[0.8862]|.↩↩|\tighttracebreak\allowbreak\token[2.5792]|\|\allowbreak\token[0.2151]|[↩|\newline\allowbreak\token[0.1984]|\|\allowbreak\token[0.0103]|text|\allowbreak\token[0.0181]|{|\allowbreak\token[0.0870]|Total|\allowbreak\token[1.3941]| Personal|\allowbreak\token[0.0026]|ities|\allowbreak\token[0.0002]|}|\allowbreak\token[0.0001]| =|\allowbreak\token[0.7137]| \|\allowbreak\token[0.0886]|text|\allowbreak\token[0.0168]|{|\allowbreak\token[0.6421]|Head|\allowbreak\token[0.0390]| Styles|\allowbreak\token[0.0013]|}|\allowbreak\token[0.0003]| \|\allowbreak\token[0.0013]|times|\allowbreak\token[0.0010]| \|\allowbreak\token[0.0026]|text|\allowbreak\token[0.0006]|{|\allowbreak\token[0.0193]|Ey|\allowbreak\token[0.0064]|eb|\allowbreak[...]
\tcblower
\token[0.0059]| Head|\allowbreak\token[1.5697]|'s|\allowbreak\token[0.2391]| personality|\allowbreak\token[1.3039]| and|\allowbreak\token[1.6078]| calculate|\allowbreak\token[0.3197]| the|\allowbreak\token[0.9583]| total|\allowbreak\token[0.0785]| number|\allowbreak\token[0.0003]| of|\allowbreak\token[1.5963]| combinations|\allowbreak\token[1.8493]|.↩↩|\tighttracebreak\allowbreak\token[0.5968]|1|\allowbreak\token[0.0014]|.|\allowbreak\token[2.1953]| Head|\allowbreak\token[3.0685]| Styles|\allowbreak\token[1.2956]|:|\allowbreak\token[1.4604]| Tom|\allowbreak\token[0.7821]| has|\allowbreak\token[0.5124]| |\allowbreak\token[0.2073]|3|\allowbreak\token[1.5355]| different|\allowbreak\token[0.6916]| hairstyles|\allowbreak\token[1.7865]| to|\allowbreak\token[0.0085]| choose|\allowbreak\token[0.0013]| from|\allowbreak\token[1.7628]|:|\allowbreak\token[2.4309]| no|\allowbreak\token[0.5287]| hair|\allowbreak\token[0.6191]|,|\allowbreak\token[2.6977]| one|\allowbreak\token[2.8643]| specific|\allowbreak\token[1.5092]| style|\allowbreak\token[0.3402]|,|\allowbreak\token[0.7287]| or|\allowbreak\token[1.5681]| another|\allowbreak\token[0.8696]| specific|\allowbreak\token[0.0088]| style|\allowbreak\token[0.9749]|.|\allowbreak\token[2.0765]| Therefore|\allowbreak\token[0.0127]|,|\allowbreak\token[0.5399]| there|\allowbreak\token[0.0636]| are|\allowbreak\token[0.5022]| |\allowbreak\token[0.0739]|3|\allowbreak\token[1.4258]| choices|\allowbreak\token[0.0388]| for|\allowbreak\token[0.3651]| the|\allowbreak\token[1.1578]| head|\allowbreak\token[0.3136]| style|\allowbreak\token[0.1492]|.↩↩|\tighttracebreak\allowbreak\token[0.0094]|2|\allowbreak\token[0.0005]|.|\allowbreak\token[0.2995]| Ey|\allowbreak\token[0.0514]|eb|\allowbreak\token[0.6216]|rows|\allowbreak\token[0.0705]|:|\allowbreak\token[0.8460]| Tom|\allowbreak\token[0.3290]| has|\allowbreak\token[0.0456]| |\allowbreak\token[0.0061]|2|\allowbreak\token[0.9886]| sets|\allowbreak\token[0.0039]| of|\allowbreak\token[0.0195]| eyebrows|\allowbreak\token[0.9570]| to|\allowbreak\token[0.1082]| choose|\allowbreak\token[0.0073]| from|\allowbreak\token[1.1251]|.|\allowbreak\token[2.5731]| Therefore|\allowbreak\token[0.0005]|,|\allowbreak\token[0.0364]| there|\allowbreak\token[0.0012]| are|\allowbreak\token[0.0128]| |\allowbreak\token[0.0015]|2|\allowbreak\token[0.0287]| choices|\allowbreak\token[0.0010]| for|\allowbreak\token[0.0283]| the|\allowbreak\token[0.2232]| eyebrows|\allowbreak\token[0.0232]|.↩↩|\tighttracebreak\allowbreak\token[0.0011]|3|\allowbreak\token[0.0010]|.|\allowbreak\token[0.4525]| Eyes|\allowbreak\token[0.0874]|:|\allowbreak\token[0.7876]| Tom|\allowbreak\token[0.5051]| has|\allowbreak\token[0.8353]| |\allowbreak\token[0.0121]|1|\allowbreak\token[0.3055]| pair|\allowbreak\token[0.0007]| of|\allowbreak\token[0.3571]| goog|\allowbreak\token[0.0014]|ly|\allowbreak\token[0.0024]| eyes|\allowbreak\token[1.1297]|.|\allowbreak\token[1.5654]| Since|\allowbreak\token[1.3607]| there|\allowbreak\token[1.0083]|'s|\allowbreak\token[0.0491]| only|\allowbreak\token[0.3688]| one|\allowbreak\token[1.4633]| option|\allowbreak\token[1.0750]| for|\allowbreak\token[0.4996]| the|\allowbreak\token[0.0417]| eyes|\allowbreak\token[0.0393]|,|\allowbreak\token[1.0057]| there|\allowbreak\token[1.0951]| is|\allowbreak\token[0.4726]| |\allowbreak\token[0.0003]|1|\allowbreak\token[0.0605]| choice|\allowbreak\token[0.7464]| for|\allowbreak\token[0.0272]| the|\allowbreak\token[0.0392]| eyes|\allowbreak\token[0.0148]|.↩↩|\tighttracebreak\allowbreak\token[0.0001]|4|\allowbreak\token[0.0006]|.|\allowbreak\token[0.0054]| E|\allowbreak\token[0.0058]|ars|\allowbreak\token[0.0576]|:|\allowbreak\token[0.0075]| Tom|\allowbreak\token[0.0191]| has|\allowbreak\token[0.0015]| |\allowbreak\token[0.0003]|2|\allowbreak\token[0.1049]| sets|\allowbreak\token[0.0004]| of|\allowbreak\token[0.0041]| ears|\allowbreak\token[0.5494]| to|\allowbreak\token[0.0051]| choose|\allowbreak\token[0.0015]| from|\allowbreak\token[0.0157]|.|\allowbreak\token[0.0687]| Therefore|\allowbreak\token[0.0004]|,|\allowbreak\token[0.0005]| there|\allowbreak\token[0.0003]| are|\allowbreak\token[0.0002]| |\allowbreak\token[0.0001]|2|\allowbreak\token[0.0003]| choices|\allowbreak\token[0.0001]| for|\allowbreak\token[0.0007]| the|\allowbreak\token[0.0013]| ears|\allowbreak\token[0.0023]|.↩↩|\tighttracebreak\allowbreak\token[0.0001]|5|\allowbreak\token[0.0009]|.|\allowbreak\token[0.0226]| Lips|\allowbreak\token[0.0078]|:|\allowbreak\token[0.0218]| Tom|\allowbreak\token[0.0145]| has|\allowbreak[...]
\end{twopartbox}
\begin{twopartbox}{Cut Index 186 | Acceptance Probability 0.9992}{Block 1 Step 3}{Proposal (Accepted)}
[...]\allowbreak\token[0.0012]|.|\allowbreak\token[0.3981]| Ey|\allowbreak\token[0.0355]|eb|\allowbreak\token[0.5949]|rows|\allowbreak\token[0.1011]|:|\allowbreak\token[0.0840]| There|\allowbreak\token[0.0031]| are|\allowbreak\token[0.0037]| |\allowbreak\token[0.0108]|2|\allowbreak\token[0.3989]| sets|\allowbreak\token[0.0339]| of|\allowbreak\token[0.0052]| eyebrows|\allowbreak\token[1.1750]|.↩|\newline\allowbreak\token[0.0020]|3|\allowbreak\token[0.0008]|.|\allowbreak\token[0.4501]| Eyes|\allowbreak\token[0.0108]|:|\allowbreak\token[0.0906]| There|\allowbreak\token[0.1947]| is|\allowbreak\token[0.3304]| |\allowbreak\token[0.0000]|1|\allowbreak\token[0.1679]| pair|\allowbreak\token[0.0006]| of|\allowbreak\token[0.3842]| goog|\allowbreak\token[0.0016]|ly|\allowbreak\token[0.0025]| eyes|\allowbreak\token[0.5950]|.↩|\newline\allowbreak\token[0.0009]|4|\allowbreak\token[0.0005]|.|\allowbreak\token[0.0068]| E|\allowbreak\token[0.0079]|ars|\allowbreak\token[0.0069]|:|\allowbreak\token[0.0004]| There|\allowbreak\token[0.0008]| are|\allowbreak\token[0.0002]| |\allowbreak\token[0.0002]|2|\allowbreak\token[0.0066]| sets|\allowbreak\token[0.0014]| of|\allowbreak\token[0.0031]| ears|\allowbreak\token[0.0177]|.↩|\newline\allowbreak\token[0.0003]|5|\allowbreak\token[0.0008]|.|\allowbreak\token[0.0295]| Lips|\allowbreak\token[0.0016]|:|\allowbreak\token[0.0006]| There|\allowbreak\token[0.0029]| are|\allowbreak\token[0.0004]| |\allowbreak\token[0.0006]|2|\allowbreak\token[0.0011]| sets|\allowbreak\token[0.0022]| of|\allowbreak\token[0.0137]| lips|\allowbreak\token[0.0185]|.↩|\newline\allowbreak\token[0.0007]|6|\allowbreak\token[0.0005]|.|\allowbreak\token[0.0457]| Shoes|\allowbreak\token[0.0022]|:|\allowbreak\token[0.0727]| There|\allowbreak\token[0.3178]| are|\allowbreak\token[0.2213]| |\allowbreak\token[0.1680]|2|\allowbreak\token[1.3374]| pairs|\allowbreak\token[0.0109]| of|\allowbreak\token[0.1059]| shoes|\allowbreak\token[1.1960]| (|\allowbreak\token[0.8351]|regular|\allowbreak\token[0.2854]| and|\allowbreak\token[0.0432]| clown|\allowbreak\token[0.8250]| shoes|\allowbreak\token[0.7841]|).↩↩|\tighttracebreak\allowbreak\token[2.1364]|Since|\allowbreak\token[1.7572]| each|\allowbreak\token[1.1168]| component|\allowbreak\token[1.6872]| can|\allowbreak\token[0.5085]| be|\allowbreak\token[2.2794]| chosen|\allowbreak\token[0.2479]| independently|\allowbreak\token[0.9810]|,|\allowbreak\token[0.4855]| we|\allowbreak\token[1.5553]| multiply|\allowbreak\token[0.0039]| the|\allowbreak\token[0.0487]| number|\allowbreak\token[0.0002]| of|\allowbreak\token[0.6499]| choices|\allowbreak\token[0.2177]| for|\allowbreak\token[0.0001]| each|\allowbreak\token[0.4447]| component|\allowbreak\token[0.3993]| to|\allowbreak\token[0.6391]| find|\allowbreak\token[0.0012]| the|\allowbreak\token[0.0000]| total|\allowbreak\token[0.0013]| number|\allowbreak\token[0.0002]| of|\allowbreak\token[0.4231]| different|\allowbreak\token[0.3593]| personalities|\allowbreak\token[0.8862]|.↩↩|\tighttracebreak\allowbreak\token[2.5792]|\|\allowbreak\token[0.2151]|[↩|\newline\allowbreak\token[0.1984]|\|\allowbreak\token[0.0103]|text|\allowbreak\token[0.0181]|{|\allowbreak\token[0.0870]|Total|\allowbreak\token[1.3941]| Personal|\allowbreak\token[0.0026]|ities|\allowbreak\token[0.0002]|}|\allowbreak\token[0.0001]| =|\allowbreak\token[0.7137]| \|\allowbreak\token[0.0886]|text|\allowbreak\token[0.0168]|{|\allowbreak\token[0.6421]|Head|\allowbreak\token[0.0390]| Styles|\allowbreak\token[0.0013]|}|\allowbreak\token[0.0003]| \|\allowbreak\token[0.0013]|times|\allowbreak\token[0.0010]| \|\allowbreak\token[0.0026]|text|\allowbreak\token[0.0006]|{|\allowbreak\token[0.0193]|Ey|\allowbreak\token[0.0064]|eb|\allowbreak\token[0.2243]|rows|\allowbreak\token[0.0109]|}|\allowbreak\token[0.0003]| \|\allowbreak\token[0.0004]|times|\allowbreak\token[0.0007]| \|\allowbreak\token[0.0028]|text|\allowbreak\token[0.0011]|{|\allowbreak\token[0.0148]|E|\allowbreak\token[0.2976]|yes|\allowbreak\token[0.0007]|}|\allowbreak\token[0.0002]| \|\allowbreak\token[0.0005]|times|\allowbreak\token[0.0001]| \|\allowbreak\token[0.0005]|text|\allowbreak\token[0.0003]|{|\allowbreak\token[0.0217]|E|\allowbreak\token[0.0040]|ars|\allowbreak\token[0.0002]|}|\allowbreak\token[0.0006]| \|\allowbreak\token[0.0002]|times|\allowbreak\token[0.0001]| \|\allowbreak\token[0.0008]|text|\allowbreak\token[0.0001]|{|\allowbreak\token[0.0017]|L|\allowbreak\token[0.0127]|ips|\allowbreak\token[0.0004]|}|\allowbreak\cut{[CUT]}\allowbreak\token[0.0001]| \|\allowbreak\token[0.0001]|times|\allowbreak\token[0.0102]| \|\allowbreak\token[0.0032]|text|\allowbreak\token[0.0003]|{|\allowbreak\token[0.0063]|Sh|
\tcblower
\token[0.0001]| \|\allowbreak\token[0.0001]|times|\allowbreak\token[0.0114]| \|\allowbreak\token[0.0035]|text|\allowbreak\token[0.0004]|{|\allowbreak\token[0.0062]|Sh|
\end{twopartbox}
\begin{twopartbox}{Cut Index 110 | Acceptance Probability 0.0000}{Block 2 Step 1}{Proposal (Rejected)}
\token[2.7697]| To|\allowbreak\token[1.0000]| solve|\allowbreak\token[0.4562]| this|\allowbreak\token[0.1579]| problem|\allowbreak\token[0.0266]|,|\allowbreak\token[0.2968]| we|\allowbreak\token[0.8181]| need|\allowbreak\token[0.0021]| to|\allowbreak\token[1.5837]| consider|\allowbreak\token[1.1462]| each|\allowbreak\token[2.4073]| component|\allowbreak\token[1.0419]| of|\allowbreak\token[0.7540]| the|\allowbreak\token[0.5726]| Mr|\allowbreak\token[0.0017]|.|\allowbreak\token[0.0014]| Potato|\allowbreak\token[0.0055]| Head|\allowbreak\token[1.5594]|'s|\allowbreak\token[0.2376]| personality|\allowbreak\token[1.2982]| and|\allowbreak\token[1.6177]| calculate|\allowbreak\token[0.3197]| the|\allowbreak\token[0.9583]| total|\allowbreak\token[0.0786]| number|\allowbreak\token[0.0004]| of|\allowbreak\token[1.5886]| different|\allowbreak\token[1.2403]| combinations|\allowbreak\token[1.9139]|.|\allowbreak\token[1.7905]| ↩↩|\tighttracebreak\allowbreak\token[1.0835]|1|\allowbreak\token[0.0004]|.|\allowbreak\token[2.2669]| Head|\allowbreak\token[3.1746]| Styles|\allowbreak\token[1.1827]|:|\allowbreak\token[1.4748]| There|\allowbreak\token[0.1503]| are|\allowbreak\token[0.1852]| |\allowbreak\token[0.1497]|3|\allowbreak\token[1.4805]| hairstyles|\allowbreak\token[2.0824]| available|\allowbreak\token[1.8707]|.↩|\newline\allowbreak\token[0.6127]|2|\allowbreak\token[0.0012]|.|\allowbreak\token[0.3981]| Ey|\allowbreak\token[0.0355]|eb|\allowbreak\token[0.5949]|rows|\allowbreak\token[0.1011]|:|\allowbreak\token[0.0840]| There|\allowbreak\token[0.0031]| are|\allowbreak\token[0.0037]| |\allowbreak\token[0.0108]|2|\allowbreak\token[0.3989]| sets|\allowbreak\token[0.0339]| of|\allowbreak\token[0.0052]| eyebrows|\allowbreak\token[1.1750]|.↩|\newline\allowbreak\token[0.0020]|3|\allowbreak\token[0.0008]|.|\allowbreak\token[0.4501]| Eyes|\allowbreak\token[0.0108]|:|\allowbreak\token[0.0906]| There|\allowbreak\token[0.1947]| is|\allowbreak\token[0.3304]| |\allowbreak\token[0.0000]|1|\allowbreak\token[0.1679]| pair|\allowbreak\token[0.0006]| of|\allowbreak\token[0.3842]| goog|\allowbreak\token[0.0016]|ly|\allowbreak\token[0.0025]| eyes|\allowbreak\token[0.5950]|.↩|\newline\allowbreak\token[0.0009]|4|\allowbreak\token[0.0005]|.|\allowbreak\token[0.0068]| E|\allowbreak\token[0.0079]|ars|\allowbreak\token[0.0069]|:|\allowbreak\token[0.0004]| There|\allowbreak\token[0.0008]| are|\allowbreak\token[0.0002]| |\allowbreak\token[0.0002]|2|\allowbreak\token[0.0066]| sets|\allowbreak\token[0.0014]| of|\allowbreak\token[0.0031]| ears|\allowbreak\token[0.0177]|.↩|\newline\allowbreak\token[0.0003]|5|\allowbreak\token[0.0008]|.|\allowbreak\token[0.0295]| Lips|\allowbreak\token[0.0016]|:|\allowbreak\token[0.0006]| There|\allowbreak\token[0.0029]| are|\allowbreak\token[0.0004]| |\allowbreak\token[0.0006]|2|\allowbreak\token[0.0011]| sets|\allowbreak\token[0.0022]| of|\allowbreak\token[0.0137]| lips|\allowbreak\token[0.0185]|.↩|\newline\allowbreak\token[0.0007]|6|\allowbreak\token[0.0005]|.|\allowbreak\token[0.0457]| Shoes|\allowbreak\token[0.0022]|:|\allowbreak\token[0.0727]| There|\allowbreak\token[0.3178]| are|\allowbreak\token[0.2213]| |\allowbreak\token[0.1680]|2|\allowbreak\token[1.3374]| pairs|\allowbreak\token[0.0109]| of|\allowbreak\token[0.1059]| shoes|\allowbreak\token[1.1960]| (|\allowbreak\token[0.8351]|regular|\allowbreak\token[0.2854]| and|\allowbreak\token[0.0432]| clown|\allowbreak\token[0.8250]| shoes|\allowbreak\cut{[CUT]}\allowbreak\token[0.7841]|).↩↩|\tighttracebreak\allowbreak\token[2.1364]|Since|\allowbreak\token[1.7572]| each|\allowbreak\token[1.1168]| component|\allowbreak\token[1.6872]| can|\allowbreak\token[0.5085]| be|\allowbreak\token[2.2794]| chosen|\allowbreak\token[0.2479]| independently|\allowbreak\token[0.9810]|,|\allowbreak\token[0.4855]| we|\allowbreak\token[1.5553]| multiply|\allowbreak\token[0.0039]| the|\allowbreak\token[0.0487]| number|\allowbreak\token[0.0002]| of|\allowbreak\token[0.6499]| choices|\allowbreak\token[0.2177]| for|\allowbreak\token[0.0001]| each|\allowbreak\token[0.4447]| component|\allowbreak\token[0.3993]| to|\allowbreak\token[0.6391]| find|\allowbreak\token[0.0012]| the|\allowbreak\token[0.0000]| total|\allowbreak\token[0.0013]| number|\allowbreak\token[0.0002]| of|\allowbreak\token[0.4231]| different|\allowbreak\token[0.3593]| personalities|\allowbreak\token[0.8862]|.↩↩|\tighttracebreak\allowbreak\token[2.5792]|\|\allowbreak\token[0.2151]|[↩|\newline\allowbreak\token[0.1984]|\|\allowbreak\token[0.0103]|text|\allowbreak\token[0.0181]|{|\allowbreak\token[0.0870]|Total|\allowbreak\token[1.3941]| Personal|\allowbreak\token[0.0026]|ities|\allowbreak\token[0.0002]|}|\allowbreak\token[0.0001]| =|\allowbreak\token[0.7137]| \|\allowbreak\token[0.0886]|text|\allowbreak\token[0.0168]|{|\allowbreak\token[0.6421]|Head|\allowbreak\token[0.0390]| Styles|\allowbreak\token[0.0013]|}|\allowbreak\token[0.0003]| \|\allowbreak\token[0.0013]|times|\allowbreak\token[0.0010]| \|\allowbreak\token[0.0026]|text|\allowbreak\token[0.0006]|{|\allowbreak\token[0.0193]|Ey|\allowbreak\token[0.0064]|eb|\allowbreak\token[0.2243]|rows|\allowbreak\token[0.0109]|}|\allowbreak\token[0.0003]| \|\allowbreak\token[0.0004]|times|\allowbreak\token[0.0007]| \|\allowbreak\token[0.0028]|text|\allowbreak\token[0.0011]|{|\allowbreak\token[0.0148]|E|\allowbreak\token[0.2976]|yes|\allowbreak\token[0.0007]|}|\allowbreak\token[0.0002]| \|\allowbreak\token[0.0005]|times|\allowbreak\token[0.0001]| \|\allowbreak\token[0.0005]|text|\allowbreak\token[0.0003]|{|\allowbreak\token[0.0217]|E|\allowbreak\token[0.0040]|ars|\allowbreak\token[0.0002]|}|\allowbreak\token[0.0006]| \|\allowbreak\token[0.0002]|times|\allowbreak\token[0.0001]| \|\allowbreak\token[0.0008]|text|\allowbreak\token[0.0001]|{|\allowbreak\token[0.0017]|L|\allowbreak\token[0.0127]|ips|\allowbreak\token[0.0004]|}|\allowbreak\token[0.0001]| \|\allowbreak\token[0.0001]|times|\allowbreak\token[0.0114]| \|\allowbreak\token[0.0035]|text|\allowbreak\token[0.0004]|{|\allowbreak\token[0.0062]|Sh|\allowbreak\token[0.0041]|oes|\allowbreak\token[0.0156]|}↩|\newline\allowbreak\token[0.0000]|\|\allowbreak\token[0.0103]|]↩↩|\tighttracebreak\allowbreak\token[1.9097]|Pl|\allowbreak\token[0.0006]|ugging|\allowbreak\token[0.0005]| in|\allowbreak\token[0.0007]| the|\allowbreak\token[0.8449]| numbers|\allowbreak\token[0.2955]|:↩↩|\tighttracebreak\allowbreak\token[0.0715]|\|\allowbreak\token[0.0002]|[↩|\newline\allowbreak\token[0.0409]|\|\allowbreak\token[0.0009]|text|\allowbreak\token[0.0006]|{|\allowbreak\token[0.0071]|Total|\allowbreak\token[0.0008]| Personal|\allowbreak\token[0.0000]|ities|\allowbreak\token[0.0000]|}|\allowbreak\token[0.0001]| =|\allowbreak\token[0.0003]| |\allowbreak\token[0.0008]|3|\allowbreak\token[0.0119]| \|\allowbreak\token[0.1475]|times|\allowbreak\token[0.0010]| |\allowbreak\token[0.0006]|2|\allowbreak\token[0.0006]| \|\allowbreak\token[0.0001]|times|\allowbreak\token[0.0002]| |\allowbreak\token[0.0049]|1|\allowbreak\token[0.0002]| \|\allowbreak\token[0.0004]|times|\allowbreak\token[0.0002]| |\allowbreak\token[0.0009]|2|\allowbreak\token[0.0003]| \|\allowbreak\token[0.0002]|times|\allowbreak\token[0.0001]| |\allowbreak\token[0.0012]|2|\allowbreak\token[0.0026]| \|\allowbreak\token[0.0001]|times|\allowbreak\token[0.0084]| |\allowbreak\token[0.0115]|2|\allowbreak\token[0.2223]|↩|\newline\allowbreak\token[0.0000]|\|\allowbreak\token[0.0005]|]↩↩|\tighttracebreak\allowbreak\token[1.9935]|Now|\allowbreak\token[0.2151]|,|\allowbreak\token[1.5368]| we|\allowbreak\token[0.9886]| calculate|\allowbreak\token[0.0715]| the|\allowbreak\token[0.6988]| product|\allowbreak\token[0.4990]|:↩↩|\tighttracebreak\allowbreak\token[0.0231]|\|\allowbreak\token[0.0013]|[↩|\newline\allowbreak\token[0.5935]|\|\allowbreak\token[0.0038]|text|\allowbreak\token[0.0001]|{|\allowbreak\token[0.0000]|Total|\allowbreak\token[0.0002]| Personal|\allowbreak\token[0.0001]|ities|\allowbreak\token[0.0000]|}|\allowbreak\token[0.0000]| =|\allowbreak[...]
\tcblower
\token[0.7761]|).↩|\newline\allowbreak\token[0.0071]|7|\allowbreak\token[0.0008]|.|\allowbreak\token[2.7318]| Optional|\allowbreak\token[2.1736]| Hair|\allowbreak\token[0.4202]|:|\allowbreak\token[2.2997]| Mr|\allowbreak\token[0.0004]|.|\allowbreak\token[0.0043]| Potato|\allowbreak\token[0.0146]| Head|\allowbreak\token[0.0337]| can|\allowbreak\token[0.7846]| be|\allowbreak\token[0.0642]| bald|\allowbreak\token[0.9325]|,|\allowbreak\token[1.2737]| which|\allowbreak\token[1.3217]| means|\allowbreak\token[1.7025]| he|\allowbreak\token[1.1360]| can|\allowbreak\token[1.5666]| choose|\allowbreak\token[1.0898]| to|\allowbreak\token[1.6108]| not|\allowbreak\token[1.2236]| have|\allowbreak\token[1.0527]| any|\allowbreak\token[0.9057]| hair|\allowbreak\token[1.5416]|.|\allowbreak\token[2.2727]| This|\allowbreak\token[1.2039]| adds|\allowbreak\token[1.0927]| an|\allowbreak\token[0.5371]| additional|\allowbreak\token[0.9035]| option|\allowbreak\token[1.8263]| for|\allowbreak\token[1.4174]| hair|\allowbreak\token[1.6579]|.↩↩|\tighttracebreak\allowbreak\token[1.9942]|Let|\allowbreak\token[0.1102]|'s|\allowbreak\token[1.1986]| calculate|\allowbreak\token[0.2017]| the|\allowbreak\token[0.5320]| total|\allowbreak\token[0.2118]| number|\allowbreak\token[0.0005]| of|\allowbreak\token[1.5864]| combinations|\allowbreak\token[1.3231]| for|\allowbreak\token[0.0294]| each|\allowbreak\token[1.1495]| component|\allowbreak\token[0.7174]|:↩↩|\tighttracebreak\allowbreak\token[0.2463]|-|\allowbreak\token[1.2749]| Head|\allowbreak\token[0.1149]| Styles|\allowbreak\token[0.1701]|:|\allowbreak\token[1.5856]| |\allowbreak\token[0.0059]|3|\allowbreak\token[1.2611]| options|\allowbreak\token[0.7482]|↩|\newline\allowbreak\token[0.0145]|-|\allowbreak\token[0.0121]| Ey|\allowbreak\token[0.0073]|eb|\allowbreak\token[0.1117]|rows|\allowbreak\token[0.0060]|:|\allowbreak\token[0.1060]| |\allowbreak\token[0.0098]|2|\allowbreak\token[0.0547]| options|\allowbreak\token[0.2165]|↩|\newline\allowbreak\token[0.0009]|-|\allowbreak\token[0.0692]| Eyes|\allowbreak\token[0.0326]|:|\allowbreak\token[0.0706]| |\allowbreak\token[0.0017]|1|\allowbreak\token[0.0495]| option|\allowbreak\token[0.7100]|↩|\newline\allowbreak\token[0.0008]|-|\allowbreak\token[0.0131]| E|\allowbreak\token[0.0048]|ars|\allowbreak\token[0.0005]|:|\allowbreak\token[0.0012]| |\allowbreak\token[0.0007]|2|\allowbreak\token[0.0002]| options|\allowbreak\token[0.0032]|↩|\newline\allowbreak\token[0.0002]|-|\allowbreak\token[0.0018]| Lips|\allowbreak\token[0.0001]|:|\allowbreak\token[0.0001]| |\allowbreak\token[0.0002]|2|\allowbreak\token[0.0002]| options|\allowbreak\token[0.0010]|↩|\newline\allowbreak\token[0.0001]|-|\allowbreak\token[0.0040]| Shoes|\allowbreak\token[0.0022]|:|\allowbreak\token[0.0021]| |\allowbreak\token[0.0024]|2|\allowbreak\token[0.0030]| options|\allowbreak\token[0.3578]|↩|\newline\allowbreak\token[0.0003]|-|\allowbreak\token[0.6057]| Hair|\allowbreak\token[0.3289]|:|\allowbreak\token[0.5461]| |\allowbreak\token[1.0874]|2|\allowbreak\token[0.0183]| options|\allowbreak\token[0.0461]| (|\allowbreak\token[2.1831]|b|\allowbreak\token[0.0052]|ald|\allowbreak\token[0.4085]| or|\allowbreak\token[1.8656]| not|\allowbreak\token[0.6110]| bald|\allowbreak\token[0.0415]|)↩↩|\tighttracebreak\allowbreak\token[1.7622]|Since|\allowbreak\token[0.6528]| each|\allowbreak\token[0.6626]| component|\allowbreak\token[1.5869]| is|\allowbreak\token[0.5522]| independent|\allowbreak\token[0.9467]| of|\allowbreak\token[0.0147]| the|\allowbreak\token[0.0174]| others|\allowbreak\token[0.3605]|,|\allowbreak\token[0.2032]| we|\allowbreak\token[0.6430]| multiply|\allowbreak\token[0.0479]| the|\allowbreak\token[0.1989]| number|\allowbreak\token[0.0000]| of|\allowbreak\token[0.0372]| options|\allowbreak\token[0.2763]| for|\allowbreak\token[0.0016]| each|\allowbreak\token[0.1035]| component|\allowbreak\token[0.3123]| to|\allowbreak\token[0.6869]| find|\allowbreak\token[0.0002]| the|\allowbreak\token[0.0001]| total|\allowbreak\token[0.0012]| number|\allowbreak\token[0.0001]| of|\allowbreak\token[0.3476]| different|\allowbreak\token[0.4353]| personalities|\allowbreak\token[0.9128]|:↩↩|\tighttracebreak\allowbreak\token[0.5995]|\|\allowbreak\token[0.0125]|[↩|\newline\allowbreak\token[0.6728]|3|\allowbreak\token[0.1006]| \|\allowbreak\token[1.0864]|times|\allowbreak\token[0.0019]| |\allowbreak\token[0.0005]|2|\allowbreak\token[0.0008]| \|\allowbreak\token[0.0001]|times|\allowbreak\token[0.0001]| |\allowbreak\token[0.0126]|1|\allowbreak\token[0.0001]| \|\allowbreak\token[0.0003]|times|\allowbreak[...]
\end{twopartbox}
\begin{twopartbox}{Cut Index 52 | Acceptance Probability 1.0000}{Block 2 Step 2}{Proposal (Accepted)}
\token[2.7697]| To|\allowbreak\token[1.0000]| solve|\allowbreak\token[0.4562]| this|\allowbreak\token[0.1579]| problem|\allowbreak\token[0.0266]|,|\allowbreak\token[0.2968]| we|\allowbreak\token[0.8181]| need|\allowbreak\token[0.0021]| to|\allowbreak\token[1.5837]| consider|\allowbreak\token[1.1462]| each|\allowbreak\token[2.4073]| component|\allowbreak\token[1.0419]| of|\allowbreak\token[0.7540]| the|\allowbreak\token[0.5726]| Mr|\allowbreak\token[0.0017]|.|\allowbreak\token[0.0014]| Potato|\allowbreak\token[0.0055]| Head|\allowbreak\token[1.5594]|'s|\allowbreak\token[0.2376]| personality|\allowbreak\token[1.2982]| and|\allowbreak\token[1.6177]| calculate|\allowbreak\token[0.3197]| the|\allowbreak\token[0.9583]| total|\allowbreak\token[0.0786]| number|\allowbreak\token[0.0004]| of|\allowbreak\token[1.5886]| different|\allowbreak\token[1.2403]| combinations|\allowbreak\token[1.9139]|.|\allowbreak\token[1.7905]| ↩↩|\tighttracebreak\allowbreak\token[1.0835]|1|\allowbreak\token[0.0004]|.|\allowbreak\token[2.2669]| Head|\allowbreak\token[3.1746]| Styles|\allowbreak\token[1.1827]|:|\allowbreak\token[1.4748]| There|\allowbreak\token[0.1503]| are|\allowbreak\token[0.1852]| |\allowbreak\token[0.1497]|3|\allowbreak\token[1.4805]| hairstyles|\allowbreak\token[2.0824]| available|\allowbreak\token[1.8707]|.↩|\newline\allowbreak\token[0.6127]|2|\allowbreak\token[0.0012]|.|\allowbreak\token[0.3981]| Ey|\allowbreak\token[0.0355]|eb|\allowbreak\token[0.5949]|rows|\allowbreak\token[0.1011]|:|\allowbreak\token[0.0840]| There|\allowbreak\token[0.0031]| are|\allowbreak\token[0.0037]| |\allowbreak\token[0.0108]|2|\allowbreak\token[0.3989]| sets|\allowbreak\cut{[CUT]}\allowbreak\token[0.0339]| of|\allowbreak\token[0.0052]| eyebrows|\allowbreak\token[1.1750]|.↩|\newline\allowbreak\token[0.0020]|3|\allowbreak\token[0.0008]|.|\allowbreak\token[0.4501]| Eyes|\allowbreak\token[0.0108]|:|\allowbreak\token[0.0906]| There|\allowbreak\token[0.1947]| is|\allowbreak\token[0.3304]| |\allowbreak\token[0.0000]|1|\allowbreak\token[0.1679]| pair|\allowbreak\token[0.0006]| of|\allowbreak\token[0.3842]| goog|\allowbreak\token[0.0016]|ly|\allowbreak\token[0.0025]| eyes|\allowbreak\token[0.5950]|.↩|\newline\allowbreak\token[0.0009]|4|\allowbreak\token[0.0005]|.|\allowbreak\token[0.0068]| E|\allowbreak\token[0.0079]|ars|\allowbreak\token[0.0069]|:|\allowbreak\token[0.0004]| There|\allowbreak\token[0.0008]| are|\allowbreak\token[0.0002]| |\allowbreak\token[0.0002]|2|\allowbreak\token[0.0066]| sets|\allowbreak\token[0.0014]| of|\allowbreak\token[0.0031]| ears|\allowbreak\token[0.0177]|.↩|\newline\allowbreak\token[0.0003]|5|\allowbreak\token[0.0008]|.|\allowbreak\token[0.0295]| Lips|\allowbreak\token[0.0016]|:|\allowbreak\token[0.0006]| There|\allowbreak\token[0.0029]| are|\allowbreak\token[0.0004]| |\allowbreak\token[0.0006]|2|\allowbreak\token[0.0011]| sets|\allowbreak\token[0.0022]| of|\allowbreak\token[0.0137]| lips|\allowbreak\token[0.0185]|.↩|\newline\allowbreak\token[0.0007]|6|\allowbreak\token[0.0005]|.|\allowbreak\token[0.0457]| Shoes|\allowbreak\token[0.0022]|:|\allowbreak\token[0.0727]| There|\allowbreak\token[0.3178]| are|\allowbreak\token[0.2213]| |\allowbreak\token[0.1680]|2|\allowbreak\token[1.3374]| pairs|\allowbreak\token[0.0109]| of|\allowbreak\token[0.1059]| shoes|\allowbreak\token[1.1960]| (|\allowbreak\token[0.8351]|regular|\allowbreak\token[0.2854]| and|\allowbreak\token[0.0432]| clown|\allowbreak\token[0.8250]| shoes|\allowbreak\token[0.7841]|).↩↩|\tighttracebreak\allowbreak\token[2.1364]|Since|\allowbreak\token[1.7572]| each|\allowbreak\token[1.1168]| component|\allowbreak\token[1.6872]| can|\allowbreak\token[0.5085]| be|\allowbreak\token[2.2794]| chosen|\allowbreak\token[0.2479]| independently|\allowbreak\token[0.9810]|,|\allowbreak\token[0.4855]| we|\allowbreak\token[1.5553]| multiply|\allowbreak\token[0.0039]| the|\allowbreak\token[0.0487]| number|\allowbreak\token[0.0002]| of|\allowbreak\token[0.6499]| choices|\allowbreak\token[0.2177]| for|\allowbreak\token[0.0001]| each|\allowbreak\token[0.4447]| component|\allowbreak\token[0.3993]| to|\allowbreak\token[0.6391]| find|\allowbreak\token[0.0012]| the|\allowbreak\token[0.0000]| total|\allowbreak\token[0.0013]| number|\allowbreak\token[0.0002]| of|\allowbreak\token[0.4231]| different|\allowbreak\token[0.3593]| personalities|\allowbreak\token[0.8862]|.↩↩|\tighttracebreak\allowbreak\token[2.5792]|\|\allowbreak\token[0.2151]|[↩|\newline\allowbreak\token[0.1984]|\|\allowbreak\token[0.0103]|text|\allowbreak\token[0.0181]|{|\allowbreak\token[0.0870]|Total|\allowbreak\token[1.3941]| Personal|\allowbreak\token[0.0026]|ities|\allowbreak\token[0.0002]|}|\allowbreak\token[0.0001]| =|\allowbreak\token[0.7137]| \|\allowbreak\token[0.0886]|text|\allowbreak\token[0.0168]|{|\allowbreak\token[0.6421]|Head|\allowbreak\token[0.0390]| Styles|\allowbreak\token[0.0013]|}|\allowbreak\token[0.0003]| \|\allowbreak\token[0.0013]|times|\allowbreak\token[0.0010]| \|\allowbreak\token[0.0026]|text|\allowbreak\token[0.0006]|{|\allowbreak\token[0.0193]|Ey|\allowbreak\token[0.0064]|eb|\allowbreak\token[0.2243]|rows|\allowbreak\token[0.0109]|}|\allowbreak\token[0.0003]| \|\allowbreak\token[0.0004]|times|\allowbreak\token[0.0007]| \|\allowbreak\token[0.0028]|text|\allowbreak\token[0.0011]|{|\allowbreak\token[0.0148]|E|\allowbreak\token[0.2976]|yes|\allowbreak\token[0.0007]|}|\allowbreak\token[0.0002]| \|\allowbreak\token[0.0005]|times|\allowbreak\token[0.0001]| \|\allowbreak\token[0.0005]|text|\allowbreak\token[0.0003]|{|\allowbreak\token[0.0217]|E|\allowbreak\token[0.0040]|ars|\allowbreak\token[0.0002]|}|\allowbreak\token[0.0006]| \|\allowbreak\token[0.0002]|times|\allowbreak\token[0.0001]| \|\allowbreak\token[0.0008]|text|\allowbreak\token[0.0001]|{|\allowbreak\token[0.0017]|L|\allowbreak\token[0.0127]|ips|\allowbreak\token[0.0004]|}|\allowbreak\token[0.0001]| \|\allowbreak\token[0.0001]|times|\allowbreak\token[0.0114]| \|\allowbreak\token[0.0035]|text|\allowbreak\token[0.0004]|{|\allowbreak\token[0.0062]|Sh|\allowbreak\token[0.0041]|oes|\allowbreak\token[0.0156]|}↩|\newline\allowbreak\token[0.0000]|\|\allowbreak\token[0.0103]|]↩↩|\tighttracebreak\allowbreak[...]
\tcblower
\token[0.0284]| of|\allowbreak\token[0.0051]| eyebrows|\allowbreak\token[1.1521]|.↩|\newline\allowbreak\token[0.0023]|3|\allowbreak\token[0.0009]|.|\allowbreak\token[0.4194]| Eyes|\allowbreak\token[0.0115]|:|\allowbreak\token[0.0868]| There|\allowbreak\token[0.2275]| is|\allowbreak\token[0.3295]| |\allowbreak\token[0.0001]|1|\allowbreak\token[0.1665]| pair|\allowbreak\token[0.0005]| of|\allowbreak\token[0.4182]| goog|\allowbreak\token[0.0016]|ly|\allowbreak\token[0.0025]| eyes|\allowbreak\token[0.5420]|.↩|\newline\allowbreak\token[0.0009]|4|\allowbreak\token[0.0005]|.|\allowbreak\token[0.0081]| E|\allowbreak\token[0.0080]|ars|\allowbreak\token[0.0077]|:|\allowbreak\token[0.0005]| There|\allowbreak\token[0.0009]| are|\allowbreak\token[0.0002]| |\allowbreak\token[0.0002]|2|\allowbreak\token[0.0060]| sets|\allowbreak\token[0.0015]| of|\allowbreak\token[0.0034]| ears|\allowbreak\token[0.0173]|.↩|\newline\allowbreak\token[0.0003]|5|\allowbreak\token[0.0008]|.|\allowbreak\token[0.0317]| Lips|\allowbreak\token[0.0016]|:|\allowbreak\token[0.0005]| There|\allowbreak\token[0.0029]| are|\allowbreak\token[0.0005]| |\allowbreak\token[0.0006]|2|\allowbreak\token[0.0011]| sets|\allowbreak\token[0.0020]| of|\allowbreak\token[0.0123]| lips|\allowbreak\token[0.0166]|.↩|\newline\allowbreak\token[0.0007]|6|\allowbreak\token[0.0005]|.|\allowbreak\token[0.0447]| Shoes|\allowbreak\token[0.0022]|:|\allowbreak\token[0.0734]| There|\allowbreak\token[0.3176]| are|\allowbreak\token[0.1877]| |\allowbreak\token[0.1719]|2|\allowbreak\token[1.3251]| pairs|\allowbreak\token[0.0098]| of|\allowbreak\token[0.0977]| shoes|\allowbreak\token[1.2736]| (|\allowbreak\token[0.8248]|regular|\allowbreak\token[0.2864]| and|\allowbreak\token[0.0464]| clown|\allowbreak\token[0.8658]| shoes|\allowbreak\token[0.7751]|).↩|\newline\allowbreak\token[0.0064]|7|\allowbreak\token[0.0009]|.|\allowbreak\token[2.7580]| Optional|\allowbreak\token[2.1581]| Hair|\allowbreak\token[0.4287]|:|\allowbreak\token[2.2785]| Mr|\allowbreak\token[0.0004]|.|\allowbreak\token[0.0041]| Potato|\allowbreak\token[0.0159]| Head|\allowbreak\token[0.0337]| can|\allowbreak\token[0.7781]| be|\allowbreak\token[0.0696]| bald|\allowbreak\token[0.9327]|,|\allowbreak\token[1.2760]| so|\allowbreak\token[1.4898]| we|\allowbreak\token[2.1449]| need|\allowbreak\token[0.0072]| to|\allowbreak\token[0.6425]| consider|\allowbreak\token[2.0636]| the|\allowbreak\token[2.0358]| bald|\allowbreak\token[0.9725]| option|\allowbreak\token[1.0163]| as|\allowbreak\token[0.4613]| well|\allowbreak\token[0.6079]|.↩↩|\tighttracebreak\allowbreak\token[2.0694]|Let|\allowbreak\token[0.1102]|'s|\allowbreak\token[1.1864]| calculate|\allowbreak\token[0.2044]| the|\allowbreak\token[0.6252]| total|\allowbreak\token[0.2565]| number|\allowbreak\token[0.0006]| of|\allowbreak\token[1.5925]| combinations|\allowbreak\token[1.3379]| for|\allowbreak\token[0.0314]| each|\allowbreak\token[1.3036]| component|\allowbreak\token[0.7242]|:↩↩|\tighttracebreak\allowbreak\token[0.2794]|-|\allowbreak\token[1.1856]| Head|\allowbreak\token[0.1228]| Styles|\allowbreak\token[0.1905]|:|\allowbreak\token[1.6457]| |\allowbreak\token[0.0087]|3|\allowbreak\token[1.2710]| options|\allowbreak\token[0.8423]|↩|\newline\allowbreak\token[0.0180]|-|\allowbreak\token[0.0131]| Ey|\allowbreak\token[0.0071]|eb|\allowbreak\token[0.1322]|rows|\allowbreak\token[0.0087]|:|\allowbreak\token[0.0999]| |\allowbreak\token[0.0098]|2|\allowbreak\token[0.0609]| options|\allowbreak\token[0.3042]|↩|\newline\allowbreak\token[0.0015]|-|\allowbreak\token[0.0946]| Eyes|\allowbreak\token[0.0393]|:|\allowbreak\token[0.1094]| |\allowbreak\token[0.0016]|1|\allowbreak\token[0.0600]| option|\allowbreak\token[0.6827]| (|\allowbreak\token[2.1671]|fixed|\allowbreak\token[1.2130]|)↩|\newline\allowbreak\token[0.0003]|-|\allowbreak\token[0.0220]| E|\allowbreak\token[0.0031]|ars|\allowbreak\token[0.0003]|:|\allowbreak\token[0.0017]| |\allowbreak\token[0.0009]|2|\allowbreak\token[0.0003]| options|\allowbreak\token[0.0089]|↩|\newline\allowbreak\token[0.0002]|-|\allowbreak\token[0.0016]| Lips|\allowbreak\token[0.0001]|:|\allowbreak\token[0.0001]| |\allowbreak\token[0.0003]|2|\allowbreak\token[0.0002]| options|\allowbreak\token[0.0028]|↩|\newline\allowbreak\token[0.0001]|-|\allowbreak\token[0.0050]| Shoes|\allowbreak\token[0.0018]|:|\allowbreak\token[0.0028]| |\allowbreak\token[0.0052]|2|\allowbreak\token[0.0033]| options|\allowbreak\token[0.6145]|↩|\newline\allowbreak\token[0.0002]|-|\allowbreak[...]
\end{twopartbox}
\begin{twopartbox}{Cut Index 210 | Acceptance Probability 0.0000}{Block 2 Step 3}{Proposal (Rejected)}
[...]\allowbreak\token[0.0016]|ly|\allowbreak\token[0.0025]| eyes|\allowbreak\token[0.5420]|.↩|\newline\allowbreak\token[0.0009]|4|\allowbreak\token[0.0005]|.|\allowbreak\token[0.0081]| E|\allowbreak\token[0.0080]|ars|\allowbreak\token[0.0077]|:|\allowbreak\token[0.0005]| There|\allowbreak\token[0.0009]| are|\allowbreak\token[0.0002]| |\allowbreak\token[0.0002]|2|\allowbreak\token[0.0060]| sets|\allowbreak\token[0.0015]| of|\allowbreak\token[0.0034]| ears|\allowbreak\token[0.0173]|.↩|\newline\allowbreak\token[0.0003]|5|\allowbreak\token[0.0008]|.|\allowbreak\token[0.0317]| Lips|\allowbreak\token[0.0016]|:|\allowbreak\token[0.0005]| There|\allowbreak\token[0.0029]| are|\allowbreak\token[0.0005]| |\allowbreak\token[0.0006]|2|\allowbreak\token[0.0011]| sets|\allowbreak\token[0.0020]| of|\allowbreak\token[0.0123]| lips|\allowbreak\token[0.0166]|.↩|\newline\allowbreak\token[0.0007]|6|\allowbreak\token[0.0005]|.|\allowbreak\token[0.0447]| Shoes|\allowbreak\token[0.0022]|:|\allowbreak\token[0.0734]| There|\allowbreak\token[0.3176]| are|\allowbreak\token[0.1877]| |\allowbreak\token[0.1719]|2|\allowbreak\token[1.3251]| pairs|\allowbreak\token[0.0098]| of|\allowbreak\token[0.0977]| shoes|\allowbreak\token[1.2736]| (|\allowbreak\token[0.8248]|regular|\allowbreak\token[0.2864]| and|\allowbreak\token[0.0464]| clown|\allowbreak\token[0.8658]| shoes|\allowbreak\token[0.7751]|).↩|\newline\allowbreak\token[0.0064]|7|\allowbreak\token[0.0009]|.|\allowbreak\token[2.7580]| Optional|\allowbreak\token[2.1581]| Hair|\allowbreak\token[0.4287]|:|\allowbreak\token[2.2785]| Mr|\allowbreak\token[0.0004]|.|\allowbreak\token[0.0041]| Potato|\allowbreak\token[0.0159]| Head|\allowbreak\token[0.0337]| can|\allowbreak\token[0.7781]| be|\allowbreak\token[0.0696]| bald|\allowbreak\token[0.9327]|,|\allowbreak\token[1.2760]| so|\allowbreak\token[1.4898]| we|\allowbreak\token[2.1449]| need|\allowbreak\token[0.0072]| to|\allowbreak\token[0.6425]| consider|\allowbreak\token[2.0636]| the|\allowbreak\token[2.0358]| bald|\allowbreak\token[0.9725]| option|\allowbreak\token[1.0163]| as|\allowbreak\token[0.4613]| well|\allowbreak\token[0.6079]|.↩↩|\tighttracebreak\allowbreak\token[2.0694]|Let|\allowbreak\token[0.1102]|'s|\allowbreak\token[1.1864]| calculate|\allowbreak\token[0.2044]| the|\allowbreak\token[0.6252]| total|\allowbreak\token[0.2565]| number|\allowbreak\token[0.0006]| of|\allowbreak\token[1.5925]| combinations|\allowbreak\token[1.3379]| for|\allowbreak\token[0.0314]| each|\allowbreak\token[1.3036]| component|\allowbreak\token[0.7242]|:↩↩|\tighttracebreak\allowbreak\token[0.2794]|-|\allowbreak\token[1.1856]| Head|\allowbreak\token[0.1228]| Styles|\allowbreak\token[0.1905]|:|\allowbreak\token[1.6457]| |\allowbreak\token[0.0087]|3|\allowbreak\token[1.2710]| options|\allowbreak\token[0.8423]|↩|\newline\allowbreak\token[0.0180]|-|\allowbreak\token[0.0131]| Ey|\allowbreak\token[0.0071]|eb|\allowbreak\token[0.1322]|rows|\allowbreak\token[0.0087]|:|\allowbreak\token[0.0999]| |\allowbreak\token[0.0098]|2|\allowbreak\token[0.0609]| options|\allowbreak\token[0.3042]|↩|\newline\allowbreak\token[0.0015]|-|\allowbreak\token[0.0946]| Eyes|\allowbreak\token[0.0393]|:|\allowbreak\token[0.1094]| |\allowbreak\token[0.0016]|1|\allowbreak\token[0.0600]| option|\allowbreak\token[0.6827]| (|\allowbreak\token[2.1671]|fixed|\allowbreak\token[1.2130]|)↩|\newline\allowbreak\token[0.0003]|-|\allowbreak\token[0.0220]| E|\allowbreak\token[0.0031]|ars|\allowbreak\token[0.0003]|:|\allowbreak\token[0.0017]| |\allowbreak\token[0.0009]|2|\allowbreak\token[0.0003]| options|\allowbreak\token[0.0089]|↩|\newline\allowbreak\token[0.0002]|-|\allowbreak\token[0.0016]| Lips|\allowbreak\token[0.0001]|:|\allowbreak\token[0.0001]| |\allowbreak\token[0.0003]|2|\allowbreak\token[0.0002]| options|\allowbreak\token[0.0028]|↩|\newline\allowbreak\token[0.0001]|-|\allowbreak\token[0.0050]| Shoes|\allowbreak\token[0.0018]|:|\allowbreak\token[0.0028]| |\allowbreak\token[0.0052]|2|\allowbreak\token[0.0033]| options|\allowbreak\token[0.6145]|↩|\newline\allowbreak\token[0.0002]|-|\allowbreak\token[0.7985]| Hair|\allowbreak\token[0.2720]|:|\allowbreak\token[0.3327]| |\allowbreak\token[1.0393]|2|\allowbreak\token[0.0071]| options|\allowbreak\token[0.0312]| (|\allowbreak\token[1.9112]|including|\allowbreak\token[0.3321]| bald|\allowbreak\token[0.4965]|)↩↩|\tighttracebreak\allowbreak\token[1.9320]|To|\allowbreak\token[0.0083]| find|\allowbreak\token[0.0006]| the|\allowbreak\token[0.0003]| total|\allowbreak\token[0.0006]| number|\allowbreak\cut{[CUT]}\allowbreak\token[0.0001]| of|\allowbreak\token[0.3211]| different|\allowbreak\token[0.5026]| personalities|\allowbreak\token[0.1634]|,|\allowbreak\token[0.3714]| we|\allowbreak\token[0.1243]| multiply|\allowbreak\token[0.3935]| the|\allowbreak\token[0.0688]| number|\allowbreak\token[0.0001]| of|\allowbreak\token[0.0579]| options|\allowbreak\token[0.1206]| for|\allowbreak\token[0.0009]| each|\allowbreak\token[0.1077]| component|\allowbreak\token[1.0305]|:↩↩|\tighttracebreak\allowbreak\token[0.6033]|\|\allowbreak\token[0.0132]|[↩|\newline\allowbreak\token[0.7082]|3|\allowbreak\token[0.1242]| \|\allowbreak\token[1.1316]|text|\allowbreak\token[0.0115]|{|\allowbreak\token[0.0835]| (|\allowbreak\token[1.0290]|head|\allowbreak\token[0.0497]| styles|\allowbreak\token[0.0582]|)}|\allowbreak\token[0.0021]| \|\allowbreak\token[0.0060]|times|\allowbreak\token[0.0026]| |\allowbreak\token[0.0083]|2|\allowbreak\token[0.0012]| \|\allowbreak\token[0.0045]|text|\allowbreak\token[0.0002]|{|\allowbreak\token[0.0002]| (|\allowbreak\token[0.0073]|ey|\allowbreak\token[0.0044]|eb|\allowbreak\token[0.1897]|rows|\allowbreak\token[0.0173]|)}|\allowbreak\token[0.0016]| \|\allowbreak\token[0.0010]|times|\allowbreak\token[0.0004]| |\allowbreak\token[0.0126]|1|\allowbreak\token[0.0005]| \|\allowbreak\token[0.0105]|text|\allowbreak\token[0.0036]|{|\allowbreak\token[0.0003]| (|\allowbreak\token[0.2602]|eyes|\allowbreak\token[0.0479]|)}|\allowbreak\token[0.0030]| \|\allowbreak\token[0.0007]|times|\allowbreak\token[0.0010]| |\allowbreak\token[0.0002]|2|\allowbreak\token[0.0004]| \|\allowbreak\token[0.0037]|text|\allowbreak\token[0.0008]|{|\allowbreak\token[0.0001]| (|\allowbreak\token[0.0014]|ears|\allowbreak\token[0.0003]|)}|\allowbreak\token[0.0004]| \|\allowbreak\token[0.0003]|times|\allowbreak\token[0.0006]| |\allowbreak\token[0.0003]|2|\allowbreak\token[0.0002]| \|\allowbreak\token[0.0028]|text|\allowbreak\token[0.0003]|{|\allowbreak\token[0.0002]| (|\allowbreak\token[0.0141]|l|\allowbreak\token[0.0111]|ips|\allowbreak\token[0.0007]|)}|\allowbreak\token[0.0002]| \|\allowbreak\token[0.0002]|times|\allowbreak\token[0.0011]| |\allowbreak\token[0.0032]|2|\allowbreak\token[0.0003]| \|\allowbreak\token[0.0048]|text|\allowbreak\token[0.0003]|{|\allowbreak\token[0.0001]| (|\allowbreak\token[0.0019]|sh|\allowbreak\token[0.0041]|oes|\allowbreak\token[0.0035]|)}|\allowbreak\token[0.0013]| \|\allowbreak\token[0.0001]|times|\allowbreak\token[0.0019]| |\allowbreak\token[0.0575]|2|\allowbreak\token[0.0006]| \|\allowbreak\token[0.0004]|text|\allowbreak\token[0.0005]|{|\allowbreak\token[0.0005]| (|\allowbreak\token[0.0829]|hair|\allowbreak\token[0.7663]|)}↩|\newline\allowbreak\token[0.0002]|\|\allowbreak\token[0.0001]|]↩↩|\tighttracebreak\allowbreak\token[2.2350]|\|\allowbreak\token[0.0006]|[↩|\newline\allowbreak\token[0.4892]|=|\allowbreak\token[0.0004]| |\allowbreak\token[0.0170]|3|\allowbreak\token[0.0005]| \|\allowbreak\token[0.0008]|times|\allowbreak\token[0.0001]| |\allowbreak\token[0.0001]|2|\allowbreak\token[0.0040]| \|\allowbreak\token[0.0000]|times|\allowbreak\token[0.0000]| |\allowbreak\token[0.0669]|1|\allowbreak\token[0.0001]| \|\allowbreak\token[0.0001]|times|\allowbreak\token[0.0001]| |\allowbreak\token[0.0000]|2|\allowbreak\token[0.0001]| \|\allowbreak\token[0.0001]|times|\allowbreak\token[0.0001]| |\allowbreak\token[0.0002]|2|\allowbreak\token[0.0002]| \|\allowbreak\token[0.0000]|times|\allowbreak\token[0.0001]| |\allowbreak\token[0.0005]|2|\allowbreak\token[0.0547]| \|\allowbreak\token[0.0001]|times|\allowbreak\token[0.0001]| |\allowbreak\token[0.0077]|2|\allowbreak\token[0.1363]|↩|\newline\allowbreak\token[0.0091]|\|\allowbreak\token[0.0017]|]↩↩|\tighttracebreak\allowbreak\token[0.3076]|\|\allowbreak\token[0.0000]|[↩|\newline\allowbreak\token[0.0001]|=|\allowbreak\token[0.0004]| |\allowbreak\token[0.7599]|3|\allowbreak\token[0.0114]| \|\allowbreak\token[0.0002]|times|\allowbreak\token[0.0008]| |\allowbreak\token[0.1872]|2|\allowbreak\token[0.7279]|^|\allowbreak\token[1.0997]|6|\allowbreak\token[0.0171]|↩|\newline\allowbreak\token[0.0000]|\|\allowbreak\token[0.0002]|]↩↩|\tighttracebreak\allowbreak\token[0.0372]|\|\allowbreak\token[0.0000]|[↩|\newline\allowbreak\token[0.0001]|=|\allowbreak\token[0.0000]| |\allowbreak\token[0.0028]|3|\allowbreak\token[0.0002]| \|\allowbreak\token[0.0000]|times|\allowbreak\token[0.0000]| |\allowbreak[...]
\tcblower
\token[0.0001]| of|\allowbreak\token[0.3087]| different|\allowbreak\token[0.5041]| personalities|\allowbreak\token[0.1634]|,|\allowbreak\token[0.4266]| we|\allowbreak\token[0.1256]| multiply|\allowbreak\token[0.3941]| the|\allowbreak\token[0.0621]| number|\allowbreak\token[0.0001]| of|\allowbreak\token[0.0579]| options|\allowbreak\token[0.1206]| for|\allowbreak\token[0.0010]| each|\allowbreak\token[0.0991]| component|\allowbreak\token[0.9455]|:↩↩|\tighttracebreak\allowbreak\token[0.5828]|\|\allowbreak\token[0.0133]|[↩|\newline\allowbreak\token[0.7078]|\|\allowbreak\token[0.0035]|text|\allowbreak\token[0.0092]|{|\allowbreak\token[0.0078]|Total|\allowbreak\token[1.2412]| Personal|\allowbreak\token[0.0024]|ities|\allowbreak\token[0.0001]|}|\allowbreak\token[0.0001]| =|\allowbreak\token[1.0118]| |\allowbreak\token[0.0045]|3|\allowbreak\token[0.2579]| \|\allowbreak\token[0.5333]|times|\allowbreak\token[0.0014]| |\allowbreak\token[0.0015]|2|\allowbreak\token[0.0014]| \|\allowbreak\token[0.0002]|times|\allowbreak\token[0.0002]| |\allowbreak\token[0.0295]|1|\allowbreak\token[0.0003]| \|\allowbreak\token[0.0003]|times|\allowbreak\token[0.0002]| |\allowbreak\token[0.0005]|2|\allowbreak\token[0.0002]| \|\allowbreak\token[0.0003]|times|\allowbreak\token[0.0001]| |\allowbreak\token[0.0023]|2|\allowbreak\token[0.0004]| \|\allowbreak\token[0.0001]|times|\allowbreak\token[0.0008]| |\allowbreak\token[0.0080]|2|\allowbreak\token[0.4213]| \|\allowbreak\token[0.0006]|times|\allowbreak\token[0.0061]| |\allowbreak\token[0.0539]|2|\allowbreak\token[0.0757]|↩|\newline\allowbreak\token[0.0000]|\|\allowbreak\token[0.0014]|]↩↩|\tighttracebreak\allowbreak\token[2.0327]|Now|\allowbreak\token[0.1716]|,|\allowbreak\token[1.3453]| let|\allowbreak\token[0.1325]|'s|\allowbreak\token[1.0772]| perform|\allowbreak\token[0.0002]| the|\allowbreak\token[0.0710]| multiplication|\allowbreak\token[0.7051]|:↩↩|\tighttracebreak\allowbreak\token[0.0538]|\|\allowbreak\token[0.0020]|[↩|\newline\allowbreak\token[0.7065]|\|\allowbreak\token[0.0114]|text|\allowbreak\token[0.0001]|{|\allowbreak\token[0.0000]|Total|\allowbreak\token[0.0005]| Personal|\allowbreak\token[0.0000]|ities|\allowbreak\token[0.0000]|}|\allowbreak\token[0.0000]| =|\allowbreak\token[0.0003]| |\allowbreak\token[0.0133]|3|\allowbreak\token[0.0003]| \|\allowbreak\token[0.0001]|times|\allowbreak\token[0.0002]| |\allowbreak\token[0.0029]|2|\allowbreak\token[0.4662]| \|\allowbreak\token[0.0005]|times|\allowbreak\token[0.0002]| |\allowbreak\token[0.6515]|1|\allowbreak\token[0.0085]| \|\allowbreak\token[0.0002]|times|\allowbreak\token[0.0001]| |\allowbreak\token[0.0004]|2|\allowbreak\token[0.0033]| \|\allowbreak\token[0.0001]|times|\allowbreak\token[0.0001]| |\allowbreak\token[0.0003]|2|\allowbreak\token[0.0017]| \|\allowbreak\token[0.0000]|times|\allowbreak\token[0.0001]| |\allowbreak\token[0.0006]|2|\allowbreak\token[0.0643]| \|\allowbreak\token[0.0001]|times|\allowbreak\token[0.0001]| |\allowbreak\token[0.0005]|2|\allowbreak\token[0.1913]| =|\allowbreak\token[0.0009]| |\allowbreak\token[0.7519]|3|\allowbreak\token[0.0076]| \|\allowbreak\token[0.0003]|times|\allowbreak\token[0.0013]| |\allowbreak\token[0.4490]|2|\allowbreak\token[0.6445]| \|\allowbreak\token[0.0006]|times|\allowbreak\token[0.0006]| |\allowbreak\token[1.0049]|2|\allowbreak\token[0.1220]| \|\allowbreak\token[0.0003]|times|\allowbreak\token[0.0001]| |\allowbreak\token[0.1418]|2|\allowbreak\token[0.0169]| \|\allowbreak\token[0.0002]|times|\allowbreak\token[0.0001]| |\allowbreak\token[0.0234]|2|\allowbreak\token[0.7036]| =|\allowbreak\token[0.0004]| |\allowbreak\token[0.5276]|3|\allowbreak\token[0.0028]| \|\allowbreak\token[0.0002]|times|\allowbreak\token[0.0020]| |\allowbreak\token[0.5584]|1|\allowbreak\token[0.0016]|6|\allowbreak\token[0.0159]| =|\allowbreak\token[0.0001]| |\allowbreak\token[0.0004]|4|\allowbreak\token[0.0002]|8|\allowbreak\token[0.0022]|↩|\newline\allowbreak\token[0.0000]|\|\allowbreak\token[0.0001]|]↩↩|\tighttracebreak\allowbreak\token[1.2779]|Therefore|\allowbreak\token[0.0008]|,|\allowbreak\token[0.6697]| Tom|\allowbreak\token[0.0052]| can|\allowbreak\token[0.6300]| come|\allowbreak\token[0.0003]| up|\allowbreak\token[0.0010]| with|\allowbreak\token[0.7211]| |\allowbreak\token[0.0016]|4|\allowbreak\token[0.0013]|8|\allowbreak\token[0.0004]| different|\allowbreak\token[0.0310]| w|\allowbreak\token[0.0002]|acky|\allowbreak[...]
\end{twopartbox}
\begin{twopartbox}{Cut Index 242 | Acceptance Probability 0.0002}{Block 3 Step 1}{Proposal (Rejected)}
[...]\allowbreak\token[0.0734]| There|\allowbreak\token[0.3176]| are|\allowbreak\token[0.1877]| |\allowbreak\token[0.1719]|2|\allowbreak\token[1.3251]| pairs|\allowbreak\token[0.0098]| of|\allowbreak\token[0.0977]| shoes|\allowbreak\token[1.2736]| (|\allowbreak\token[0.8248]|regular|\allowbreak\token[0.2864]| and|\allowbreak\token[0.0464]| clown|\allowbreak\token[0.8658]| shoes|\allowbreak\token[0.7751]|).↩|\newline\allowbreak\token[0.0064]|7|\allowbreak\token[0.0009]|.|\allowbreak\token[2.7580]| Optional|\allowbreak\token[2.1581]| Hair|\allowbreak\token[0.4287]|:|\allowbreak\token[2.2785]| Mr|\allowbreak\token[0.0004]|.|\allowbreak\token[0.0041]| Potato|\allowbreak\token[0.0159]| Head|\allowbreak\token[0.0337]| can|\allowbreak\token[0.7781]| be|\allowbreak\token[0.0696]| bald|\allowbreak\token[0.9327]|,|\allowbreak\token[1.2760]| so|\allowbreak\token[1.4898]| we|\allowbreak\token[2.1449]| need|\allowbreak\token[0.0072]| to|\allowbreak\token[0.6425]| consider|\allowbreak\token[2.0636]| the|\allowbreak\token[2.0358]| bald|\allowbreak\token[0.9725]| option|\allowbreak\token[1.0163]| as|\allowbreak\token[0.4613]| well|\allowbreak\token[0.6079]|.↩↩|\tighttracebreak\allowbreak\token[2.0694]|Let|\allowbreak\token[0.1102]|'s|\allowbreak\token[1.1864]| calculate|\allowbreak\token[0.2044]| the|\allowbreak\token[0.6252]| total|\allowbreak\token[0.2565]| number|\allowbreak\token[0.0006]| of|\allowbreak\token[1.5925]| combinations|\allowbreak\token[1.3379]| for|\allowbreak\token[0.0314]| each|\allowbreak\token[1.3036]| component|\allowbreak\token[0.7242]|:↩↩|\tighttracebreak\allowbreak\token[0.2794]|-|\allowbreak\token[1.1856]| Head|\allowbreak\token[0.1228]| Styles|\allowbreak\token[0.1905]|:|\allowbreak\token[1.6457]| |\allowbreak\token[0.0087]|3|\allowbreak\token[1.2710]| options|\allowbreak\token[0.8423]|↩|\newline\allowbreak\token[0.0180]|-|\allowbreak\token[0.0131]| Ey|\allowbreak\token[0.0071]|eb|\allowbreak\token[0.1322]|rows|\allowbreak\token[0.0087]|:|\allowbreak\token[0.0999]| |\allowbreak\token[0.0098]|2|\allowbreak\token[0.0609]| options|\allowbreak\token[0.3042]|↩|\newline\allowbreak\token[0.0015]|-|\allowbreak\token[0.0946]| Eyes|\allowbreak\token[0.0393]|:|\allowbreak\token[0.1094]| |\allowbreak\token[0.0016]|1|\allowbreak\token[0.0600]| option|\allowbreak\token[0.6827]| (|\allowbreak\token[2.1671]|fixed|\allowbreak\token[1.2130]|)↩|\newline\allowbreak\token[0.0003]|-|\allowbreak\token[0.0220]| E|\allowbreak\token[0.0031]|ars|\allowbreak\token[0.0003]|:|\allowbreak\token[0.0017]| |\allowbreak\token[0.0009]|2|\allowbreak\token[0.0003]| options|\allowbreak\token[0.0089]|↩|\newline\allowbreak\token[0.0002]|-|\allowbreak\token[0.0016]| Lips|\allowbreak\token[0.0001]|:|\allowbreak\token[0.0001]| |\allowbreak\token[0.0003]|2|\allowbreak\token[0.0002]| options|\allowbreak\token[0.0028]|↩|\newline\allowbreak\token[0.0001]|-|\allowbreak\token[0.0050]| Shoes|\allowbreak\token[0.0018]|:|\allowbreak\token[0.0028]| |\allowbreak\token[0.0052]|2|\allowbreak\token[0.0033]| options|\allowbreak\token[0.6145]|↩|\newline\allowbreak\token[0.0002]|-|\allowbreak\token[0.7985]| Hair|\allowbreak\token[0.2720]|:|\allowbreak\token[0.3327]| |\allowbreak\token[1.0393]|2|\allowbreak\token[0.0071]| options|\allowbreak\token[0.0312]| (|\allowbreak\token[1.9112]|including|\allowbreak\token[0.3321]| bald|\allowbreak\token[0.4965]|)↩↩|\tighttracebreak\allowbreak\token[1.9320]|To|\allowbreak\token[0.0083]| find|\allowbreak\token[0.0006]| the|\allowbreak\token[0.0003]| total|\allowbreak\token[0.0006]| number|\allowbreak\token[0.0001]| of|\allowbreak\token[0.3211]| different|\allowbreak\token[0.5026]| personalities|\allowbreak\token[0.1634]|,|\allowbreak\token[0.3714]| we|\allowbreak\token[0.1243]| multiply|\allowbreak\token[0.3935]| the|\allowbreak\token[0.0688]| number|\allowbreak\token[0.0001]| of|\allowbreak\token[0.0579]| options|\allowbreak\token[0.1206]| for|\allowbreak\token[0.0009]| each|\allowbreak\token[0.1077]| component|\allowbreak\token[1.0305]|:↩↩|\tighttracebreak\allowbreak\token[0.6033]|\|\allowbreak\token[0.0132]|[↩|\newline\allowbreak\token[0.7082]|3|\allowbreak\token[0.1242]| \|\allowbreak\token[1.1316]|text|\allowbreak\token[0.0115]|{|\allowbreak\token[0.0835]| (|\allowbreak\token[1.0290]|head|\allowbreak\token[0.0497]| styles|\allowbreak\token[0.0582]|)}|\allowbreak\token[0.0021]| \|\allowbreak\token[0.0060]|times|\allowbreak\token[0.0026]| |\allowbreak\token[0.0083]|2|\allowbreak\token[0.0012]| \|\allowbreak\token[0.0045]|text|\allowbreak\token[0.0002]|{|\allowbreak\token[0.0002]| (|\allowbreak\cut{[CUT]}\allowbreak\token[0.0073]|ey|\allowbreak\token[0.0044]|eb|\allowbreak\token[0.1897]|rows|\allowbreak\token[0.0173]|)}|\allowbreak\token[0.0016]| \|\allowbreak\token[0.0010]|times|\allowbreak\token[0.0004]| |\allowbreak\token[0.0126]|1|\allowbreak\token[0.0005]| \|\allowbreak\token[0.0105]|text|\allowbreak\token[0.0036]|{|\allowbreak\token[0.0003]| (|\allowbreak\token[0.2602]|eyes|\allowbreak\token[0.0479]|)}|\allowbreak\token[0.0030]| \|\allowbreak\token[0.0007]|times|\allowbreak\token[0.0010]| |\allowbreak\token[0.0002]|2|\allowbreak\token[0.0004]| \|\allowbreak\token[0.0037]|text|\allowbreak\token[0.0008]|{|\allowbreak\token[0.0001]| (|\allowbreak\token[0.0014]|ears|\allowbreak\token[0.0003]|)}|\allowbreak\token[0.0004]| \|\allowbreak\token[0.0003]|times|\allowbreak\token[0.0006]| |\allowbreak\token[0.0003]|2|\allowbreak\token[0.0002]| \|\allowbreak\token[0.0028]|text|\allowbreak\token[0.0003]|{|\allowbreak\token[0.0002]| (|\allowbreak\token[0.0141]|l|\allowbreak\token[0.0111]|ips|\allowbreak\token[0.0007]|)}|\allowbreak\token[0.0002]| \|\allowbreak\token[0.0002]|times|\allowbreak\token[0.0011]| |\allowbreak\token[0.0032]|2|\allowbreak\token[0.0003]| \|\allowbreak\token[0.0048]|text|\allowbreak\token[0.0003]|{|\allowbreak\token[0.0001]| (|\allowbreak\token[0.0019]|sh|\allowbreak\token[0.0041]|oes|\allowbreak\token[0.0035]|)}|\allowbreak\token[0.0013]| \|\allowbreak\token[0.0001]|times|\allowbreak\token[0.0019]| |\allowbreak\token[0.0575]|2|\allowbreak\token[0.0006]| \|\allowbreak\token[0.0004]|text|\allowbreak\token[0.0005]|{|\allowbreak\token[0.0005]| (|\allowbreak\token[0.0829]|hair|\allowbreak\token[0.7663]|)}↩|\newline\allowbreak\token[0.0002]|\|\allowbreak\token[0.0001]|]↩↩|\tighttracebreak\allowbreak\token[2.2350]|\|\allowbreak\token[0.0006]|[↩|\newline\allowbreak\token[0.4892]|=|\allowbreak\token[0.0004]| |\allowbreak\token[0.0170]|3|\allowbreak\token[0.0005]| \|\allowbreak\token[0.0008]|times|\allowbreak\token[0.0001]| |\allowbreak\token[0.0001]|2|\allowbreak\token[0.0040]| \|\allowbreak\token[0.0000]|times|\allowbreak\token[0.0000]| |\allowbreak\token[0.0669]|1|\allowbreak\token[0.0001]| \|\allowbreak\token[0.0001]|times|\allowbreak\token[0.0001]| |\allowbreak\token[0.0000]|2|\allowbreak\token[0.0001]| \|\allowbreak\token[0.0001]|times|\allowbreak\token[0.0001]| |\allowbreak\token[0.0002]|2|\allowbreak\token[0.0002]| \|\allowbreak\token[0.0000]|times|\allowbreak\token[0.0001]| |\allowbreak\token[0.0005]|2|\allowbreak\token[0.0547]| \|\allowbreak\token[0.0001]|times|\allowbreak\token[0.0001]| |\allowbreak\token[0.0077]|2|\allowbreak\token[0.1363]|↩|\newline\allowbreak\token[0.0091]|\|\allowbreak\token[0.0017]|]↩↩|\tighttracebreak\allowbreak\token[0.3076]|\|\allowbreak\token[0.0000]|[↩|\newline\allowbreak\token[0.0001]|=|\allowbreak\token[0.0004]| |\allowbreak\token[0.7599]|3|\allowbreak\token[0.0114]| \|\allowbreak\token[0.0002]|times|\allowbreak\token[0.0008]| |\allowbreak\token[0.1872]|2|\allowbreak\token[0.7279]|^|\allowbreak\token[1.0997]|6|\allowbreak\token[0.0171]|↩|\newline\allowbreak\token[0.0000]|\|\allowbreak\token[0.0002]|]↩↩|\tighttracebreak\allowbreak\token[0.0372]|\|\allowbreak\token[0.0000]|[↩|\newline\allowbreak\token[0.0001]|=|\allowbreak\token[0.0000]| |\allowbreak\token[0.0028]|3|\allowbreak\token[0.0002]| \|\allowbreak\token[0.0000]|times|\allowbreak\token[0.0000]| |\allowbreak\token[0.0003]|6|\allowbreak\token[0.0001]|4|\allowbreak\token[0.0049]|↩|\newline\allowbreak\token[0.0000]|\|\allowbreak\token[0.0001]|]↩↩|\tighttracebreak\allowbreak\token[0.0008]|\|\allowbreak\token[0.0001]|[↩|\newline\allowbreak\token[0.0006]|=|\allowbreak\token[0.0005]| |\allowbreak\token[0.0010]|1|\allowbreak\token[0.0006]|9|\allowbreak\token[0.0007]|2|\allowbreak\token[0.0041]|↩|\newline\allowbreak\token[0.0000]|\|\allowbreak\token[0.0000]|]↩↩|\tighttracebreak\allowbreak\token[1.2007]|Therefore|\allowbreak\token[0.0002]|,|\allowbreak\token[0.6497]| Tom|\allowbreak\token[0.0058]| can|\allowbreak\token[0.6566]| come|\allowbreak\token[0.0003]| up|\allowbreak\token[0.0011]| with|\allowbreak\token[0.8091]| |\allowbreak\token[0.0011]|1|\allowbreak\token[0.0005]|9|\allowbreak\token[0.0005]|2|\allowbreak\token[0.0005]| different|\allowbreak\token[0.0440]| w|\allowbreak\token[0.0002]|acky|\allowbreak\token[0.0057]| personalities|\allowbreak\token[0.9256]| for|\allowbreak\token[0.2376]| his|\allowbreak[...]
\tcblower
\token[0.0078]|ey|\allowbreak\token[0.0044]|eb|\allowbreak\token[0.2070]|rows|\allowbreak\token[0.0196]|)}|\allowbreak\token[0.0016]| \|\allowbreak\token[0.0010]|times|\allowbreak\token[0.0004]| |\allowbreak\token[0.0140]|1|\allowbreak\token[0.0006]| \|\allowbreak\token[0.0105]|text|\allowbreak\token[0.0036]|{|\allowbreak\token[0.0003]| (|\allowbreak\token[0.2596]|eyes|\allowbreak\token[0.0466]|)}|\allowbreak\token[0.0030]| \|\allowbreak\token[0.0006]|times|\allowbreak\token[0.0010]| |\allowbreak\token[0.0002]|2|\allowbreak\token[0.0003]| \|\allowbreak\token[0.0036]|text|\allowbreak\token[0.0009]|{|\allowbreak\token[0.0001]| (|\allowbreak\token[0.0014]|ears|\allowbreak\token[0.0003]|)}|\allowbreak\token[0.0003]| \|\allowbreak\token[0.0003]|times|\allowbreak\token[0.0006]| |\allowbreak\token[0.0004]|2|\allowbreak\token[0.0002]| \|\allowbreak\token[0.0028]|text|\allowbreak\token[0.0003]|{|\allowbreak\token[0.0002]| (|\allowbreak\token[0.0146]|l|\allowbreak\token[0.0123]|ips|\allowbreak\token[0.0008]|)}|\allowbreak\token[0.0002]| \|\allowbreak\token[0.0002]|times|\allowbreak\token[0.0009]| |\allowbreak\token[0.0031]|2|\allowbreak\token[0.0003]| \|\allowbreak\token[0.0048]|text|\allowbreak\token[0.0003]|{|\allowbreak\token[0.0001]| (|\allowbreak\token[0.0017]|sh|\allowbreak\token[0.0041]|oes|\allowbreak\token[0.0037]|)}|\allowbreak\token[0.0013]| \|\allowbreak\token[0.0001]|times|\allowbreak\token[0.0023]| |\allowbreak\token[0.0567]|2|\allowbreak\token[0.0007]| \|\allowbreak\token[0.0004]|text|\allowbreak\token[0.0004]|{|\allowbreak\token[0.0005]| (|\allowbreak\token[0.0874]|hair|\allowbreak\token[0.7641]|)}↩|\newline\allowbreak\token[0.0002]|\|\allowbreak\token[0.0001]|]↩↩|\tighttracebreak\allowbreak\token[2.3074]|\|\allowbreak\token[0.0007]|[↩|\newline\allowbreak\token[0.4892]|=|\allowbreak\token[0.0004]| |\allowbreak\token[0.0151]|3|\allowbreak\token[0.0005]| \|\allowbreak\token[0.0009]|times|\allowbreak\token[0.0001]| |\allowbreak\token[0.0001]|2|\allowbreak\token[0.0036]| \|\allowbreak\token[0.0000]|times|\allowbreak\token[0.0000]| |\allowbreak\token[0.0669]|1|\allowbreak\token[0.0001]| \|\allowbreak\token[0.0001]|times|\allowbreak\token[0.0001]| |\allowbreak\token[0.0000]|2|\allowbreak\token[0.0001]| \|\allowbreak\token[0.0001]|times|\allowbreak\token[0.0001]| |\allowbreak\token[0.0002]|2|\allowbreak\token[0.0002]| \|\allowbreak\token[0.0000]|times|\allowbreak\token[0.0001]| |\allowbreak\token[0.0006]|2|\allowbreak\token[0.0547]| \|\allowbreak\token[0.0001]|times|\allowbreak\token[0.0001]| |\allowbreak\token[0.0077]|2|\allowbreak\token[0.1235]|↩|\newline\allowbreak\token[0.0091]|\|\allowbreak\token[0.0019]|]↩↩|\tighttracebreak\allowbreak\token[0.3122]|\|\allowbreak\token[0.0000]|[↩|\newline\allowbreak\token[0.0001]|=|\allowbreak\token[0.0003]| |\allowbreak\token[0.8014]|3|\allowbreak\token[0.0092]| \|\allowbreak\token[0.0002]|times|\allowbreak\token[0.0009]| |\allowbreak\token[0.1930]|2|\allowbreak\token[0.7306]| \|\allowbreak\token[0.0013]|times|\allowbreak\token[0.0006]| |\allowbreak\token[0.8626]|2|\allowbreak\token[0.0642]| \|\allowbreak\token[0.0002]|times|\allowbreak\token[0.0000]| |\allowbreak\token[0.0387]|2|\allowbreak\token[0.0076]| \|\allowbreak\token[0.0001]|times|\allowbreak\token[0.0001]| |\allowbreak\token[0.0235]|2|\allowbreak\token[0.3230]| \|\allowbreak\token[0.0037]|times|\allowbreak\token[0.0003]| |\allowbreak\token[0.0195]|2|\allowbreak\token[0.4326]|↩|\newline\allowbreak\token[0.0002]|\|\allowbreak\token[0.0005]|]↩↩|\tighttracebreak\allowbreak\token[0.0156]|\|\allowbreak\token[0.0000]|[↩|\newline\allowbreak\token[0.0001]|=|\allowbreak\token[0.0002]| |\allowbreak\token[0.6019]|3|\allowbreak\token[0.0077]| \|\allowbreak\token[0.0002]|times|\allowbreak\token[0.0114]| |\allowbreak\token[1.2483]|3|\allowbreak\token[0.0028]|2|\allowbreak\token[0.0208]|↩|\newline\allowbreak\token[0.0000]|\|\allowbreak\token[0.0001]|]↩↩|\tighttracebreak\allowbreak\token[0.0012]|\|\allowbreak\token[0.0000]|[↩|\newline\allowbreak\token[0.0003]|=|\allowbreak\token[0.0003]| |\allowbreak\token[0.0022]|9|\allowbreak\token[0.0003]|6|\allowbreak\token[0.0049]|↩|\newline\allowbreak\token[0.0000]|\|\allowbreak\token[0.0001]|]↩↩|\tighttracebreak\allowbreak\token[1.2134]|Therefore|\allowbreak\token[0.0002]|,|\allowbreak\token[0.6550]| Tom|\allowbreak\token[0.0085]| can|\allowbreak[...]
\end{twopartbox}
\begin{twopartbox}{Cut Index 143 | Acceptance Probability 0.3359}{Block 3 Step 2}{Proposal (Rejected)}
\token[2.7697]| To|\allowbreak\token[1.0000]| solve|\allowbreak\token[0.4562]| this|\allowbreak\token[0.1579]| problem|\allowbreak\token[0.0266]|,|\allowbreak\token[0.2968]| we|\allowbreak\token[0.8181]| need|\allowbreak\token[0.0021]| to|\allowbreak\token[1.5837]| consider|\allowbreak\token[1.1462]| each|\allowbreak\token[2.4073]| component|\allowbreak\token[1.0419]| of|\allowbreak\token[0.7540]| the|\allowbreak\token[0.5726]| Mr|\allowbreak\token[0.0017]|.|\allowbreak\token[0.0014]| Potato|\allowbreak\token[0.0055]| Head|\allowbreak\token[1.5594]|'s|\allowbreak\token[0.2376]| personality|\allowbreak\token[1.2982]| and|\allowbreak\token[1.6177]| calculate|\allowbreak\token[0.3197]| the|\allowbreak\token[0.9583]| total|\allowbreak\token[0.0786]| number|\allowbreak\token[0.0004]| of|\allowbreak\token[1.5886]| different|\allowbreak\token[1.2403]| combinations|\allowbreak\token[1.9139]|.|\allowbreak\token[1.7905]| ↩↩|\tighttracebreak\allowbreak\token[1.0835]|1|\allowbreak\token[0.0004]|.|\allowbreak\token[2.2669]| Head|\allowbreak\token[3.1746]| Styles|\allowbreak\token[1.1827]|:|\allowbreak\token[1.4748]| There|\allowbreak\token[0.1503]| are|\allowbreak\token[0.1852]| |\allowbreak\token[0.1497]|3|\allowbreak\token[1.4805]| hairstyles|\allowbreak\token[2.0824]| available|\allowbreak\token[1.8707]|.↩|\newline\allowbreak\token[0.6127]|2|\allowbreak\token[0.0012]|.|\allowbreak\token[0.3981]| Ey|\allowbreak\token[0.0355]|eb|\allowbreak\token[0.5949]|rows|\allowbreak\token[0.1011]|:|\allowbreak\token[0.0840]| There|\allowbreak\token[0.0031]| are|\allowbreak\token[0.0037]| |\allowbreak\token[0.0108]|2|\allowbreak\token[0.3989]| sets|\allowbreak\token[0.0284]| of|\allowbreak\token[0.0051]| eyebrows|\allowbreak\token[1.1521]|.↩|\newline\allowbreak\token[0.0023]|3|\allowbreak\token[0.0009]|.|\allowbreak\token[0.4194]| Eyes|\allowbreak\token[0.0115]|:|\allowbreak\token[0.0868]| There|\allowbreak\token[0.2275]| is|\allowbreak\token[0.3295]| |\allowbreak\token[0.0001]|1|\allowbreak\token[0.1665]| pair|\allowbreak\token[0.0005]| of|\allowbreak\token[0.4182]| goog|\allowbreak\token[0.0016]|ly|\allowbreak\token[0.0025]| eyes|\allowbreak\token[0.5420]|.↩|\newline\allowbreak\token[0.0009]|4|\allowbreak\token[0.0005]|.|\allowbreak\token[0.0081]| E|\allowbreak\token[0.0080]|ars|\allowbreak\token[0.0077]|:|\allowbreak\token[0.0005]| There|\allowbreak\token[0.0009]| are|\allowbreak\token[0.0002]| |\allowbreak\token[0.0002]|2|\allowbreak\token[0.0060]| sets|\allowbreak\token[0.0015]| of|\allowbreak\token[0.0034]| ears|\allowbreak\token[0.0173]|.↩|\newline\allowbreak\token[0.0003]|5|\allowbreak\token[0.0008]|.|\allowbreak\token[0.0317]| Lips|\allowbreak\token[0.0016]|:|\allowbreak\token[0.0005]| There|\allowbreak\token[0.0029]| are|\allowbreak\token[0.0005]| |\allowbreak\token[0.0006]|2|\allowbreak\token[0.0011]| sets|\allowbreak\token[0.0020]| of|\allowbreak\token[0.0123]| lips|\allowbreak\token[0.0166]|.↩|\newline\allowbreak\token[0.0007]|6|\allowbreak\token[0.0005]|.|\allowbreak\token[0.0447]| Shoes|\allowbreak\token[0.0022]|:|\allowbreak\token[0.0734]| There|\allowbreak\token[0.3176]| are|\allowbreak\token[0.1877]| |\allowbreak\token[0.1719]|2|\allowbreak\token[1.3251]| pairs|\allowbreak\token[0.0098]| of|\allowbreak\token[0.0977]| shoes|\allowbreak\token[1.2736]| (|\allowbreak\token[0.8248]|regular|\allowbreak\token[0.2864]| and|\allowbreak\token[0.0464]| clown|\allowbreak\token[0.8658]| shoes|\allowbreak\token[0.7751]|).↩|\newline\allowbreak\token[0.0064]|7|\allowbreak\token[0.0009]|.|\allowbreak\token[2.7580]| Optional|\allowbreak\token[2.1581]| Hair|\allowbreak\token[0.4287]|:|\allowbreak\token[2.2785]| Mr|\allowbreak\token[0.0004]|.|\allowbreak\token[0.0041]| Potato|\allowbreak\token[0.0159]| Head|\allowbreak\token[0.0337]| can|\allowbreak\token[0.7781]| be|\allowbreak\token[0.0696]| bald|\allowbreak\token[0.9327]|,|\allowbreak\token[1.2760]| so|\allowbreak\token[1.4898]| we|\allowbreak\token[2.1449]| need|\allowbreak\token[0.0072]| to|\allowbreak\token[0.6425]| consider|\allowbreak\token[2.0636]| the|\allowbreak\token[2.0358]| bald|\allowbreak\token[0.9725]| option|\allowbreak\token[1.0163]| as|\allowbreak\token[0.4613]| well|\allowbreak\token[0.6079]|.↩↩|\tighttracebreak\allowbreak\token[2.0694]|Let|\allowbreak\token[0.1102]|'s|\allowbreak\token[1.1864]| calculate|\allowbreak\token[0.2044]| the|\allowbreak\token[0.6252]| total|\allowbreak\token[0.2565]| number|\allowbreak\token[0.0006]| of|\allowbreak\token[1.5925]| combinations|\allowbreak\cut{[CUT]}\allowbreak\token[1.3379]| for|\allowbreak\token[0.0314]| each|\allowbreak\token[1.3036]| component|\allowbreak\token[0.7242]|:↩↩|\tighttracebreak\allowbreak\token[0.2794]|-|\allowbreak\token[1.1856]| Head|\allowbreak\token[0.1228]| Styles|\allowbreak\token[0.1905]|:|\allowbreak\token[1.6457]| |\allowbreak\token[0.0087]|3|\allowbreak\token[1.2710]| options|\allowbreak\token[0.8423]|↩|\newline\allowbreak\token[0.0180]|-|\allowbreak\token[0.0131]| Ey|\allowbreak\token[0.0071]|eb|\allowbreak\token[0.1322]|rows|\allowbreak\token[0.0087]|:|\allowbreak\token[0.0999]| |\allowbreak\token[0.0098]|2|\allowbreak\token[0.0609]| options|\allowbreak\token[0.3042]|↩|\newline\allowbreak\token[0.0015]|-|\allowbreak\token[0.0946]| Eyes|\allowbreak\token[0.0393]|:|\allowbreak\token[0.1094]| |\allowbreak\token[0.0016]|1|\allowbreak\token[0.0600]| option|\allowbreak\token[0.6827]| (|\allowbreak\token[2.1671]|fixed|\allowbreak\token[1.2130]|)↩|\newline\allowbreak\token[0.0003]|-|\allowbreak\token[0.0220]| E|\allowbreak\token[0.0031]|ars|\allowbreak\token[0.0003]|:|\allowbreak\token[0.0017]| |\allowbreak\token[0.0009]|2|\allowbreak\token[0.0003]| options|\allowbreak\token[0.0089]|↩|\newline\allowbreak\token[0.0002]|-|\allowbreak\token[0.0016]| Lips|\allowbreak\token[0.0001]|:|\allowbreak\token[0.0001]| |\allowbreak\token[0.0003]|2|\allowbreak\token[0.0002]| options|\allowbreak\token[0.0028]|↩|\newline\allowbreak\token[0.0001]|-|\allowbreak\token[0.0050]| Shoes|\allowbreak\token[0.0018]|:|\allowbreak\token[0.0028]| |\allowbreak\token[0.0052]|2|\allowbreak\token[0.0033]| options|\allowbreak\token[0.6145]|↩|\newline\allowbreak\token[0.0002]|-|\allowbreak\token[0.7985]| Hair|\allowbreak\token[0.2720]|:|\allowbreak\token[0.3327]| |\allowbreak\token[1.0393]|2|\allowbreak\token[0.0071]| options|\allowbreak\token[0.0312]| (|\allowbreak\token[1.9112]|including|\allowbreak\token[0.3321]| bald|\allowbreak\token[0.4965]|)↩↩|\tighttracebreak\allowbreak\token[1.9320]|To|\allowbreak\token[0.0083]| find|\allowbreak\token[0.0006]| the|\allowbreak\token[0.0003]| total|\allowbreak\token[0.0006]| number|\allowbreak\token[0.0001]| of|\allowbreak\token[0.3211]| different|\allowbreak\token[0.5026]| personalities|\allowbreak\token[0.1634]|,|\allowbreak\token[0.3714]| we|\allowbreak\token[0.1243]| multiply|\allowbreak\token[0.3935]| the|\allowbreak\token[0.0688]| number|\allowbreak\token[0.0001]| of|\allowbreak\token[0.0579]| options|\allowbreak\token[0.1206]| for|\allowbreak\token[0.0009]| each|\allowbreak\token[0.1077]| component|\allowbreak\token[1.0305]|:↩↩|\tighttracebreak\allowbreak\token[0.6033]|\|\allowbreak\token[0.0132]|[↩|\newline\allowbreak\token[0.7082]|3|\allowbreak\token[0.1242]| \|\allowbreak\token[1.1316]|text|\allowbreak\token[0.0115]|{|\allowbreak\token[0.0835]| (|\allowbreak\token[1.0290]|head|\allowbreak\token[0.0497]| styles|\allowbreak\token[0.0582]|)}|\allowbreak\token[0.0021]| \|\allowbreak\token[0.0060]|times|\allowbreak\token[0.0026]| |\allowbreak\token[0.0083]|2|\allowbreak\token[0.0012]| \|\allowbreak\token[0.0045]|text|\allowbreak\token[0.0002]|{|\allowbreak\token[0.0002]| (|\allowbreak\token[0.0073]|ey|\allowbreak\token[0.0044]|eb|\allowbreak\token[0.1897]|rows|\allowbreak\token[0.0173]|)}|\allowbreak\token[0.0016]| \|\allowbreak\token[0.0010]|times|\allowbreak\token[0.0004]| |\allowbreak\token[0.0126]|1|\allowbreak\token[0.0005]| \|\allowbreak\token[0.0105]|text|\allowbreak\token[0.0036]|{|\allowbreak\token[0.0003]| (|\allowbreak\token[0.2602]|eyes|\allowbreak\token[0.0479]|)}|\allowbreak\token[0.0030]| \|\allowbreak\token[0.0007]|times|\allowbreak\token[0.0010]| |\allowbreak\token[0.0002]|2|\allowbreak\token[0.0004]| \|\allowbreak\token[0.0037]|text|\allowbreak\token[0.0008]|{|\allowbreak\token[0.0001]| (|\allowbreak\token[0.0014]|ears|\allowbreak\token[0.0003]|)}|\allowbreak\token[0.0004]| \|\allowbreak\token[0.0003]|times|\allowbreak\token[0.0006]| |\allowbreak\token[0.0003]|2|\allowbreak\token[0.0002]| \|\allowbreak\token[0.0028]|text|\allowbreak\token[0.0003]|{|\allowbreak\token[0.0002]| (|\allowbreak\token[0.0141]|l|\allowbreak\token[0.0111]|ips|\allowbreak\token[0.0007]|)}|\allowbreak\token[0.0002]| \|\allowbreak\token[0.0002]|times|\allowbreak\token[0.0011]| |\allowbreak\token[0.0032]|2|\allowbreak\token[0.0003]| \|\allowbreak\token[0.0048]|text|\allowbreak\token[0.0003]|{|\allowbreak\token[0.0001]| (|\allowbreak\token[0.0019]|sh|\allowbreak\token[0.0041]|oes|\allowbreak[...]
\tcblower
\token[1.3388]| step|\allowbreak\token[0.3397]| by|\allowbreak\token[0.0110]| step|\allowbreak\token[0.3526]|:↩↩|\tighttracebreak\allowbreak\token[0.9085]|1|\allowbreak\token[0.0002]|.|\allowbreak\token[2.2391]| Head|\allowbreak\token[0.3324]| Styles|\allowbreak\token[0.9433]|:|\allowbreak\token[2.0149]| ↩|\newline\allowbreak\token[0.2644]|  |\allowbreak\token[1.2321]| -|\allowbreak\token[1.9326]| There|\allowbreak\token[0.0020]| are|\allowbreak\token[0.0147]| |\allowbreak\token[0.0125]|3|\allowbreak\token[1.9089]| choices|\allowbreak\token[0.4864]| for|\allowbreak\token[0.3570]| the|\allowbreak\token[0.8962]| head|\allowbreak\token[0.2525]| style|\allowbreak\token[0.7200]|.↩↩|\tighttracebreak\allowbreak\token[0.0160]|2|\allowbreak\token[0.0007]|.|\allowbreak\token[0.0338]| Ey|\allowbreak\token[0.0061]|eb|\allowbreak\token[0.1267]|rows|\allowbreak\token[0.5194]|:|\allowbreak\token[0.0039]| ↩|\newline\allowbreak\token[0.0001]|  |\allowbreak\token[0.0001]| -|\allowbreak\token[0.3069]| There|\allowbreak\token[0.0076]| are|\allowbreak\token[0.0032]| |\allowbreak\token[0.0010]|2|\allowbreak\token[0.3714]| choices|\allowbreak\token[0.0149]| for|\allowbreak\token[0.1023]| the|\allowbreak\token[0.8608]| eyebrows|\allowbreak\token[0.4038]|.↩↩|\tighttracebreak\allowbreak\token[0.0024]|3|\allowbreak\token[0.0014]|.|\allowbreak\token[0.0525]| Eyes|\allowbreak\token[0.1715]|:|\allowbreak\token[0.0013]| ↩|\newline\allowbreak\token[0.0002]|  |\allowbreak\token[0.0001]| -|\allowbreak\token[0.3786]| There|\allowbreak\token[0.0616]| is|\allowbreak\token[0.6116]| |\allowbreak\token[0.0002]|1|\allowbreak\token[0.5304]| choice|\allowbreak\token[0.0620]| for|\allowbreak\token[0.0100]| the|\allowbreak\token[0.6383]| eyes|\allowbreak\token[1.0773]|.↩↩|\tighttracebreak\allowbreak\token[0.0005]|4|\allowbreak\token[0.0002]|.|\allowbreak\token[0.0080]| E|\allowbreak\token[0.0024]|ars|\allowbreak\token[0.0651]|:|\allowbreak\token[0.0007]| ↩|\newline\allowbreak\token[0.0002]|  |\allowbreak\token[0.0000]| -|\allowbreak\token[0.0022]| There|\allowbreak\token[0.0008]| are|\allowbreak\token[0.0001]| |\allowbreak\token[0.0002]|2|\allowbreak\token[0.0017]| choices|\allowbreak\token[0.0002]| for|\allowbreak\token[0.0048]| the|\allowbreak\token[0.0091]| ears|\allowbreak\token[0.0116]|.↩↩|\tighttracebreak\allowbreak\token[0.0002]|5|\allowbreak\token[0.0002]|.|\allowbreak\token[0.0023]| Lips|\allowbreak\token[0.0216]|:|\allowbreak\token[0.0004]| ↩|\newline\allowbreak\token[0.0002]|  |\allowbreak\token[0.0000]| -|\allowbreak\token[0.0002]| There|\allowbreak\token[0.0006]| are|\allowbreak\token[0.0001]| |\allowbreak\token[0.0002]|2|\allowbreak\token[0.0005]| choices|\allowbreak\token[0.0001]| for|\allowbreak\token[0.0005]| the|\allowbreak\token[0.0152]| lips|\allowbreak\token[0.0016]|.↩↩|\tighttracebreak\allowbreak\token[0.0023]|6|\allowbreak\token[0.0002]|.|\allowbreak\token[0.0025]| Shoes|\allowbreak\token[0.0499]|:|\allowbreak\token[0.0004]| ↩|\newline\allowbreak\token[0.0003]|  |\allowbreak\token[0.0005]| -|\allowbreak\token[0.0043]| There|\allowbreak\token[0.0008]| are|\allowbreak\token[0.0002]| |\allowbreak\token[0.0117]|2|\allowbreak\token[0.0287]| choices|\allowbreak\token[0.0038]| for|\allowbreak\token[0.0039]| the|\allowbreak\token[0.0242]| shoes|\allowbreak\token[0.3778]|.↩↩|\tighttracebreak\allowbreak\token[0.1649]|7|\allowbreak\token[0.0019]|.|\allowbreak\token[0.3555]| Optional|\allowbreak\token[0.0095]| Hair|\allowbreak\token[0.1184]|:|\allowbreak\token[0.0030]| ↩|\newline\allowbreak\token[0.0005]|  |\allowbreak\token[0.0005]| -|\allowbreak\token[1.6257]| There|\allowbreak\token[0.5153]| are|\allowbreak\token[0.0703]| |\allowbreak\token[0.4367]|2|\allowbreak\token[0.3517]| choices|\allowbreak\token[0.7561]| for|\allowbreak\token[0.3147]| the|\allowbreak\token[0.4662]| hair|\allowbreak\token[0.8713]| (|\allowbreak\token[2.2406]|b|\allowbreak\token[0.0035]|ald|\allowbreak\token[0.4854]| or|\allowbreak\token[2.3404]| not|\allowbreak\token[0.2148]| bald|\allowbreak\token[0.1197]|).↩↩|\tighttracebreak\allowbreak\token[1.6435]|Now|\allowbreak\token[0.0964]|,|\allowbreak\token[1.0324]| we|\allowbreak\token[0.5465]| multiply|\allowbreak\token[0.7186]| all|\allowbreak\token[0.6672]| the|\allowbreak\token[0.5938]| choices|\allowbreak\token[0.0207]| together|\allowbreak\token[0.1078]| to|\allowbreak\token[0.6996]| get|\allowbreak\token[0.0002]| the|\allowbreak\token[0.0000]| total|\allowbreak\token[0.0011]| number|\allowbreak\token[0.0001]| of|\allowbreak\token[0.2804]| different|\allowbreak\token[0.4786]| personalities|\allowbreak[...]
\end{twopartbox}
\begin{twopartbox}{Cut Index 257 | Acceptance Probability 0.0002}{Block 3 Step 3}{Proposal (Rejected)}
[...]\allowbreak\token[2.7580]| Optional|\allowbreak\token[2.1581]| Hair|\allowbreak\token[0.4287]|:|\allowbreak\token[2.2785]| Mr|\allowbreak\token[0.0004]|.|\allowbreak\token[0.0041]| Potato|\allowbreak\token[0.0159]| Head|\allowbreak\token[0.0337]| can|\allowbreak\token[0.7781]| be|\allowbreak\token[0.0696]| bald|\allowbreak\token[0.9327]|,|\allowbreak\token[1.2760]| so|\allowbreak\token[1.4898]| we|\allowbreak\token[2.1449]| need|\allowbreak\token[0.0072]| to|\allowbreak\token[0.6425]| consider|\allowbreak\token[2.0636]| the|\allowbreak\token[2.0358]| bald|\allowbreak\token[0.9725]| option|\allowbreak\token[1.0163]| as|\allowbreak\token[0.4613]| well|\allowbreak\token[0.6079]|.↩↩|\tighttracebreak\allowbreak\token[2.0694]|Let|\allowbreak\token[0.1102]|'s|\allowbreak\token[1.1864]| calculate|\allowbreak\token[0.2044]| the|\allowbreak\token[0.6252]| total|\allowbreak\token[0.2565]| number|\allowbreak\token[0.0006]| of|\allowbreak\token[1.5925]| combinations|\allowbreak\token[1.3379]| for|\allowbreak\token[0.0314]| each|\allowbreak\token[1.3036]| component|\allowbreak\token[0.7242]|:↩↩|\tighttracebreak\allowbreak\token[0.2794]|-|\allowbreak\token[1.1856]| Head|\allowbreak\token[0.1228]| Styles|\allowbreak\token[0.1905]|:|\allowbreak\token[1.6457]| |\allowbreak\token[0.0087]|3|\allowbreak\token[1.2710]| options|\allowbreak\token[0.8423]|↩|\newline\allowbreak\token[0.0180]|-|\allowbreak\token[0.0131]| Ey|\allowbreak\token[0.0071]|eb|\allowbreak\token[0.1322]|rows|\allowbreak\token[0.0087]|:|\allowbreak\token[0.0999]| |\allowbreak\token[0.0098]|2|\allowbreak\token[0.0609]| options|\allowbreak\token[0.3042]|↩|\newline\allowbreak\token[0.0015]|-|\allowbreak\token[0.0946]| Eyes|\allowbreak\token[0.0393]|:|\allowbreak\token[0.1094]| |\allowbreak\token[0.0016]|1|\allowbreak\token[0.0600]| option|\allowbreak\token[0.6827]| (|\allowbreak\token[2.1671]|fixed|\allowbreak\token[1.2130]|)↩|\newline\allowbreak\token[0.0003]|-|\allowbreak\token[0.0220]| E|\allowbreak\token[0.0031]|ars|\allowbreak\token[0.0003]|:|\allowbreak\token[0.0017]| |\allowbreak\token[0.0009]|2|\allowbreak\token[0.0003]| options|\allowbreak\token[0.0089]|↩|\newline\allowbreak\token[0.0002]|-|\allowbreak\token[0.0016]| Lips|\allowbreak\token[0.0001]|:|\allowbreak\token[0.0001]| |\allowbreak\token[0.0003]|2|\allowbreak\token[0.0002]| options|\allowbreak\token[0.0028]|↩|\newline\allowbreak\token[0.0001]|-|\allowbreak\token[0.0050]| Shoes|\allowbreak\token[0.0018]|:|\allowbreak\token[0.0028]| |\allowbreak\token[0.0052]|2|\allowbreak\token[0.0033]| options|\allowbreak\token[0.6145]|↩|\newline\allowbreak\token[0.0002]|-|\allowbreak\token[0.7985]| Hair|\allowbreak\token[0.2720]|:|\allowbreak\token[0.3327]| |\allowbreak\token[1.0393]|2|\allowbreak\token[0.0071]| options|\allowbreak\token[0.0312]| (|\allowbreak\token[1.9112]|including|\allowbreak\token[0.3321]| bald|\allowbreak\token[0.4965]|)↩↩|\tighttracebreak\allowbreak\token[1.9320]|To|\allowbreak\token[0.0083]| find|\allowbreak\token[0.0006]| the|\allowbreak\token[0.0003]| total|\allowbreak\token[0.0006]| number|\allowbreak\token[0.0001]| of|\allowbreak\token[0.3211]| different|\allowbreak\token[0.5026]| personalities|\allowbreak\token[0.1634]|,|\allowbreak\token[0.3714]| we|\allowbreak\token[0.1243]| multiply|\allowbreak\token[0.3935]| the|\allowbreak\token[0.0688]| number|\allowbreak\token[0.0001]| of|\allowbreak\token[0.0579]| options|\allowbreak\token[0.1206]| for|\allowbreak\token[0.0009]| each|\allowbreak\token[0.1077]| component|\allowbreak\token[1.0305]|:↩↩|\tighttracebreak\allowbreak\token[0.6033]|\|\allowbreak\token[0.0132]|[↩|\newline\allowbreak\token[0.7082]|3|\allowbreak\token[0.1242]| \|\allowbreak\token[1.1316]|text|\allowbreak\token[0.0115]|{|\allowbreak\token[0.0835]| (|\allowbreak\token[1.0290]|head|\allowbreak\token[0.0497]| styles|\allowbreak\token[0.0582]|)}|\allowbreak\token[0.0021]| \|\allowbreak\token[0.0060]|times|\allowbreak\token[0.0026]| |\allowbreak\token[0.0083]|2|\allowbreak\token[0.0012]| \|\allowbreak\token[0.0045]|text|\allowbreak\token[0.0002]|{|\allowbreak\token[0.0002]| (|\allowbreak\token[0.0073]|ey|\allowbreak\token[0.0044]|eb|\allowbreak\token[0.1897]|rows|\allowbreak\token[0.0173]|)}|\allowbreak\token[0.0016]| \|\allowbreak\token[0.0010]|times|\allowbreak\token[0.0004]| |\allowbreak\token[0.0126]|1|\allowbreak\token[0.0005]| \|\allowbreak\token[0.0105]|text|\allowbreak\token[0.0036]|{|\allowbreak\token[0.0003]| (|\allowbreak\token[0.2602]|eyes|\allowbreak\token[0.0479]|)}|\allowbreak\token[0.0030]| \|\allowbreak\cut{[CUT]}\allowbreak\token[0.0007]|times|\allowbreak\token[0.0010]| |\allowbreak\token[0.0002]|2|\allowbreak\token[0.0004]| \|\allowbreak\token[0.0037]|text|\allowbreak\token[0.0008]|{|\allowbreak\token[0.0001]| (|\allowbreak\token[0.0014]|ears|\allowbreak\token[0.0003]|)}|\allowbreak\token[0.0004]| \|\allowbreak\token[0.0003]|times|\allowbreak\token[0.0006]| |\allowbreak\token[0.0003]|2|\allowbreak\token[0.0002]| \|\allowbreak\token[0.0028]|text|\allowbreak\token[0.0003]|{|\allowbreak\token[0.0002]| (|\allowbreak\token[0.0141]|l|\allowbreak\token[0.0111]|ips|\allowbreak\token[0.0007]|)}|\allowbreak\token[0.0002]| \|\allowbreak\token[0.0002]|times|\allowbreak\token[0.0011]| |\allowbreak\token[0.0032]|2|\allowbreak\token[0.0003]| \|\allowbreak\token[0.0048]|text|\allowbreak\token[0.0003]|{|\allowbreak\token[0.0001]| (|\allowbreak\token[0.0019]|sh|\allowbreak\token[0.0041]|oes|\allowbreak\token[0.0035]|)}|\allowbreak\token[0.0013]| \|\allowbreak\token[0.0001]|times|\allowbreak\token[0.0019]| |\allowbreak\token[0.0575]|2|\allowbreak\token[0.0006]| \|\allowbreak\token[0.0004]|text|\allowbreak\token[0.0005]|{|\allowbreak\token[0.0005]| (|\allowbreak\token[0.0829]|hair|\allowbreak\token[0.7663]|)}↩|\newline\allowbreak\token[0.0002]|\|\allowbreak\token[0.0001]|]↩↩|\tighttracebreak\allowbreak\token[2.2350]|\|\allowbreak\token[0.0006]|[↩|\newline\allowbreak\token[0.4892]|=|\allowbreak\token[0.0004]| |\allowbreak\token[0.0170]|3|\allowbreak\token[0.0005]| \|\allowbreak\token[0.0008]|times|\allowbreak\token[0.0001]| |\allowbreak\token[0.0001]|2|\allowbreak\token[0.0040]| \|\allowbreak\token[0.0000]|times|\allowbreak\token[0.0000]| |\allowbreak\token[0.0669]|1|\allowbreak\token[0.0001]| \|\allowbreak\token[0.0001]|times|\allowbreak\token[0.0001]| |\allowbreak\token[0.0000]|2|\allowbreak\token[0.0001]| \|\allowbreak\token[0.0001]|times|\allowbreak\token[0.0001]| |\allowbreak\token[0.0002]|2|\allowbreak\token[0.0002]| \|\allowbreak\token[0.0000]|times|\allowbreak\token[0.0001]| |\allowbreak\token[0.0005]|2|\allowbreak\token[0.0547]| \|\allowbreak\token[0.0001]|times|\allowbreak\token[0.0001]| |\allowbreak\token[0.0077]|2|\allowbreak\token[0.1363]|↩|\newline\allowbreak\token[0.0091]|\|\allowbreak\token[0.0017]|]↩↩|\tighttracebreak\allowbreak\token[0.3076]|\|\allowbreak\token[0.0000]|[↩|\newline\allowbreak\token[0.0001]|=|\allowbreak\token[0.0004]| |\allowbreak\token[0.7599]|3|\allowbreak\token[0.0114]| \|\allowbreak\token[0.0002]|times|\allowbreak\token[0.0008]| |\allowbreak\token[0.1872]|2|\allowbreak\token[0.7279]|^|\allowbreak\token[1.0997]|6|\allowbreak\token[0.0171]|↩|\newline\allowbreak\token[0.0000]|\|\allowbreak\token[0.0002]|]↩↩|\tighttracebreak\allowbreak\token[0.0372]|\|\allowbreak\token[0.0000]|[↩|\newline\allowbreak\token[0.0001]|=|\allowbreak\token[0.0000]| |\allowbreak\token[0.0028]|3|\allowbreak\token[0.0002]| \|\allowbreak\token[0.0000]|times|\allowbreak\token[0.0000]| |\allowbreak\token[0.0003]|6|\allowbreak\token[0.0001]|4|\allowbreak\token[0.0049]|↩|\newline\allowbreak\token[0.0000]|\|\allowbreak\token[0.0001]|]↩↩|\tighttracebreak\allowbreak\token[0.0008]|\|\allowbreak\token[0.0001]|[↩|\newline\allowbreak\token[0.0006]|=|\allowbreak\token[0.0005]| |\allowbreak\token[0.0010]|1|\allowbreak\token[0.0006]|9|\allowbreak\token[0.0007]|2|\allowbreak\token[0.0041]|↩|\newline\allowbreak\token[0.0000]|\|\allowbreak\token[0.0000]|]↩↩|\tighttracebreak\allowbreak\token[1.2007]|Therefore|\allowbreak\token[0.0002]|,|\allowbreak\token[0.6497]| Tom|\allowbreak\token[0.0058]| can|\allowbreak\token[0.6566]| come|\allowbreak\token[0.0003]| up|\allowbreak\token[0.0011]| with|\allowbreak\token[0.8091]| |\allowbreak\token[0.0011]|1|\allowbreak\token[0.0005]|9|\allowbreak\token[0.0005]|2|\allowbreak\token[0.0005]| different|\allowbreak\token[0.0440]| w|\allowbreak\token[0.0002]|acky|\allowbreak\token[0.0057]| personalities|\allowbreak\token[0.9256]| for|\allowbreak\token[0.2376]| his|\allowbreak\token[0.0424]| Mr|\allowbreak\token[0.0005]|.|\allowbreak\token[0.0012]| Potato|\allowbreak\token[0.0021]| Head|\allowbreak\token[0.9555]|.↩↩|\tighttracebreak\allowbreak\token[2.2337]|So|\allowbreak\token[0.1233]|,|\allowbreak\token[0.0264]| the|\allowbreak\token[0.9638]| final|\allowbreak\token[0.0052]| answer|\allowbreak\token[0.2143]| is|\allowbreak\token[0.4748]|:↩↩|\tighttracebreak\allowbreak\token[0.1784]|\|\allowbreak\token[0.3063]|[↩|\newline\allowbreak\token[0.0152]|\|\allowbreak[...]
\tcblower
\token[0.0007]|times|\allowbreak\token[0.0009]| |\allowbreak\token[0.0002]|2|\allowbreak\token[0.0004]| \|\allowbreak\token[0.0037]|text|\allowbreak\token[0.0008]|{|\allowbreak\token[0.0001]| (|\allowbreak\token[0.0014]|ears|\allowbreak\token[0.0003]|)}|\allowbreak\token[0.0003]| \|\allowbreak\token[0.0003]|times|\allowbreak\token[0.0006]| |\allowbreak\token[0.0003]|2|\allowbreak\token[0.0002]| \|\allowbreak\token[0.0029]|text|\allowbreak\token[0.0003]|{|\allowbreak\token[0.0002]| (|\allowbreak\token[0.0154]|l|\allowbreak\token[0.0123]|ips|\allowbreak\token[0.0008]|)}|\allowbreak\token[0.0002]| \|\allowbreak\token[0.0002]|times|\allowbreak\token[0.0011]| |\allowbreak\token[0.0032]|2|\allowbreak\token[0.0003]| \|\allowbreak\token[0.0048]|text|\allowbreak\token[0.0003]|{|\allowbreak\token[0.0001]| (|\allowbreak\token[0.0018]|sh|\allowbreak\token[0.0041]|oes|\allowbreak\token[0.0038]|)}|\allowbreak\token[0.0015]| \|\allowbreak\token[0.0001]|times|\allowbreak\token[0.0018]| |\allowbreak\token[0.0519]|2|\allowbreak\token[0.0006]| \|\allowbreak\token[0.0004]|text|\allowbreak\token[0.0005]|{|\allowbreak\token[0.0005]| (|\allowbreak\token[0.0830]|hair|\allowbreak\token[0.8055]|)}↩|\newline\allowbreak\token[0.0002]|\|\allowbreak\token[0.0001]|]↩↩|\tighttracebreak\allowbreak\token[2.2342]|\|\allowbreak\token[0.0007]|[↩|\newline\allowbreak\token[0.4327]|=|\allowbreak\token[0.0004]| |\allowbreak\token[0.0136]|3|\allowbreak\token[0.0004]| \|\allowbreak\token[0.0009]|times|\allowbreak\token[0.0001]| |\allowbreak\token[0.0001]|2|\allowbreak\token[0.0033]| \|\allowbreak\token[0.0000]|times|\allowbreak\token[0.0000]| |\allowbreak\token[0.0740]|1|\allowbreak\token[0.0001]| \|\allowbreak\token[0.0001]|times|\allowbreak\token[0.0001]| |\allowbreak\token[0.0000]|2|\allowbreak\token[0.0001]| \|\allowbreak\token[0.0001]|times|\allowbreak\token[0.0001]| |\allowbreak\token[0.0002]|2|\allowbreak\token[0.0002]| \|\allowbreak\token[0.0000]|times|\allowbreak\token[0.0001]| |\allowbreak\token[0.0005]|2|\allowbreak\token[0.0546]| \|\allowbreak\token[0.0001]|times|\allowbreak\token[0.0001]| |\allowbreak\token[0.0075]|2|\allowbreak\token[0.1263]|↩|\newline\allowbreak\token[0.0091]|\|\allowbreak\token[0.0019]|]↩↩|\tighttracebreak\allowbreak\token[0.3130]|\|\allowbreak\token[0.0000]|[↩|\newline\allowbreak\token[0.0001]|=|\allowbreak\token[0.0003]| |\allowbreak\token[0.7496]|3|\allowbreak\token[0.0120]| \|\allowbreak\token[0.0003]|times|\allowbreak\token[0.0008]| |\allowbreak\token[0.1898]|2|\allowbreak\token[0.7288]| \|\allowbreak\token[0.0011]|times|\allowbreak\token[0.0005]| |\allowbreak\token[0.8793]|2|\allowbreak\token[0.0527]| \|\allowbreak\token[0.0001]|times|\allowbreak\token[0.0000]| |\allowbreak\token[0.0475]|2|\allowbreak\token[0.0086]| \|\allowbreak\token[0.0001]|times|\allowbreak\token[0.0001]| |\allowbreak\token[0.0235]|2|\allowbreak\token[0.3743]| \|\allowbreak\token[0.0033]|times|\allowbreak\token[0.0003]| |\allowbreak\token[0.0175]|2|\allowbreak\token[0.4337]|↩|\newline\allowbreak\token[0.0002]|\|\allowbreak\token[0.0005]|]↩↩|\tighttracebreak\allowbreak\token[0.0151]|\|\allowbreak\token[0.0000]|[↩|\newline\allowbreak\token[0.0001]|=|\allowbreak\token[0.0002]| |\allowbreak\token[0.5742]|3|\allowbreak\token[0.0093]| \|\allowbreak\token[0.0002]|times|\allowbreak\token[0.0127]| |\allowbreak\token[1.2460]|3|\allowbreak\token[0.0029]|2|\allowbreak\token[0.0211]|↩|\newline\allowbreak\token[0.0000]|\|\allowbreak\token[0.0001]|]↩↩|\tighttracebreak\allowbreak\token[0.0014]|\|\allowbreak\token[0.0000]|[↩|\newline\allowbreak\token[0.0003]|=|\allowbreak\token[0.0004]| |\allowbreak\token[0.0022]|9|\allowbreak\token[0.0003]|6|\allowbreak\token[0.0055]|↩|\newline\allowbreak\token[0.0000]|\|\allowbreak\token[0.0001]|]↩↩|\tighttracebreak\allowbreak\token[1.2114]|Therefore|\allowbreak\token[0.0002]|,|\allowbreak\token[0.6619]| Tom|\allowbreak\token[0.0077]| can|\allowbreak\token[0.6633]| come|\allowbreak\token[0.0004]| up|\allowbreak\token[0.0008]| with|\allowbreak\token[0.7373]| |\allowbreak\token[0.0021]|9|\allowbreak\token[0.0006]|6|\allowbreak\token[0.0003]| different|\allowbreak\token[0.0356]| w|\allowbreak\token[0.0002]|acky|\allowbreak\token[0.0056]| personalities|\allowbreak\token[0.9945]| for|\allowbreak\token[0.1931]| his|\allowbreak\token[0.0514]| Mr|\allowbreak\token[0.0005]|.|\allowbreak\token[0.0011]| Potato|\allowbreak[...]
\end{twopartbox}

\endgroup

\paragraph{Entropy-Cut Trace.}
\begingroup
\raggedright %
\setlength{\lineskip}{0pt} %
\setlength{\baselineskip}{0.9em} %
\begin{twopartbox}{Cut Index 83 | Acceptance Probability 0.0000}{Block 1 Step 1}{Proposal (Rejected)}
\token[2.7697]| To|\allowbreak\token[1.0000]| solve|\allowbreak\token[0.4562]| this|\allowbreak\token[0.1579]| problem|\allowbreak\token[0.0266]|,|\allowbreak\token[0.2968]| we|\allowbreak\token[0.8181]| need|\allowbreak\token[0.0021]| to|\allowbreak\token[1.5837]| calculate|\allowbreak\token[0.0816]| the|\allowbreak\token[0.4486]| total|\allowbreak\token[0.0082]| number|\allowbreak\token[0.0005]| of|\allowbreak\token[1.0721]| different|\allowbreak\token[1.2499]| personalities|\allowbreak\token[0.8089]| Tom|\allowbreak\token[0.0128]| can|\allowbreak\token[0.1936]| create|\allowbreak\token[1.8537]| by|\allowbreak\token[0.9061]| combining|\allowbreak\token[1.3426]| the|\allowbreak\token[1.4240]| available|\allowbreak\token[1.7869]| options|\allowbreak\token[0.9441]| for|\allowbreak\token[0.7715]| eyebrows|\allowbreak\token[0.0072]|,|\allowbreak\token[0.6925]| ears|\allowbreak\token[0.0018]|,|\allowbreak\token[0.1585]| lips|\allowbreak\token[0.0020]|,|\allowbreak\token[1.3031]| shoes|\allowbreak\token[0.0554]|,|\allowbreak\token[0.1710]| and|\allowbreak\token[0.0844]| optionally|\allowbreak\token[2.0997]| hair|\allowbreak\token[1.0413]|.|\allowbreak\token[1.6814]| Here|\allowbreak\token[0.4489]|'s|\allowbreak\token[0.4810]| the|\allowbreak\token[0.6868]| breakdown|\allowbreak\token[0.3409]| of|\allowbreak\token[0.0250]| the|\allowbreak\token[0.8415]| options|\allowbreak\token[0.1294]|:↩↩|\tighttracebreak\allowbreak\token[0.6935]|1|\allowbreak\token[0.0004]|.|\allowbreak\token[0.8588]| **|\allowbreak\token[0.7153]|Ey|\allowbreak\token[0.0004]|eb|\allowbreak\token[0.0114]|rows|\allowbreak\token[0.1938]|**:|\allowbreak\token[0.4272]| |\allowbreak\token[0.1218]|2|\allowbreak\token[0.5041]| sets|\allowbreak\token[1.1879]| (|\allowbreak\token[0.3962]|left|\allowbreak\token[0.0439]| and|\allowbreak\token[0.0002]| right|\allowbreak\token[0.8898]|)↩|\newline\allowbreak\token[0.0053]|2|\allowbreak\token[0.0000]|.|\allowbreak\token[0.0000]| **|\allowbreak\token[0.0042]|E|\allowbreak\token[0.0095]|ars|\allowbreak\token[0.0000]|**:|\allowbreak\token[0.0000]| |\allowbreak\token[0.0002]|2|\allowbreak\token[0.0016]| sets|\allowbreak\token[0.1916]| (|\allowbreak\token[0.0569]|left|\allowbreak\token[0.0004]| and|\allowbreak\token[0.0002]| right|\allowbreak\token[0.0085]|)↩|\newline\allowbreak\token[0.0002]|3|\allowbreak\token[0.0000]|.|\allowbreak\token[0.0000]| **|\allowbreak\token[0.0646]|L|\allowbreak\token[0.0004]|ips|\allowbreak\token[0.0002]|**:|\allowbreak\token[0.0000]| |\allowbreak\token[0.0003]|2|\allowbreak\token[0.0001]| sets|\allowbreak\token[0.5877]| (|\allowbreak\cut{[CUT]}\allowbreak\token[1.3208]|left|\allowbreak\token[0.0010]| and|\allowbreak\token[0.0003]| right|\allowbreak\token[0.0304]|)↩|\newline\allowbreak\token[0.0004]|4|\allowbreak\token[0.0000]|.|\allowbreak\token[0.0000]| **|\allowbreak\token[0.5648]|Sh|\allowbreak\token[0.0001]|oes|\allowbreak\token[0.0001]|**:|\allowbreak\token[0.0012]| |\allowbreak\token[0.1043]|2|\allowbreak\token[0.6379]| pairs|\allowbreak\token[0.0007]| (|\allowbreak\token[0.0104]|regular|\allowbreak\token[0.2005]| and|\allowbreak\token[0.0001]| clown|\allowbreak\token[0.6290]|)↩|\newline\allowbreak\token[0.0001]|5|\allowbreak\token[0.0000]|.|\allowbreak\token[0.0000]| **|\allowbreak\token[1.4248]|H|\allowbreak\token[0.2060]|airs|\allowbreak\token[0.3892]|**:|\allowbreak\token[0.0024]| |\allowbreak\token[0.2933]|3|\allowbreak\token[1.3977]| hairstyles|\allowbreak\token[1.1485]| (|\allowbreak\token[0.5080]|including|\allowbreak\token[0.1197]| bald|\allowbreak\token[1.3234]|)↩↩|\tighttracebreak\allowbreak\token[0.7971]|Since|\allowbreak\token[1.3248]| the|\allowbreak\token[2.2634]| eyebrows|\allowbreak\token[0.1009]|,|\allowbreak\token[0.0692]| ears|\allowbreak\token[0.0000]|,|\allowbreak\token[0.7124]| lips|\allowbreak\token[0.0019]|,|\allowbreak\token[0.1110]| and|\allowbreak\token[0.4128]| shoes|\allowbreak\token[1.2684]| are|\allowbreak\token[1.7082]| mandatory|\allowbreak\token[0.4712]|,|\allowbreak\token[0.0851]| we|\allowbreak\token[1.6567]| only|\allowbreak\token[0.2410]| need|\allowbreak\token[0.0002]| to|\allowbreak\token[0.0763]| consider|\allowbreak\token[0.0599]| the|\allowbreak\token[0.6851]| combinations|\allowbreak\token[0.4937]| of|\allowbreak\token[0.3028]| these|\allowbreak\token[1.8726]|.|\allowbreak\token[0.5652]| The|\allowbreak\token[1.1044]| hair|\allowbreak\token[0.7671]| option|\allowbreak\token[0.2415]| is|\allowbreak\token[0.0253]| optional|\allowbreak\token[0.0104]|,|\allowbreak\token[0.0760]| so|\allowbreak\token[0.3457]| we|\allowbreak\token[0.7105]| need|\allowbreak\token[0.0000]| to|\allowbreak\token[0.3944]| consider|\allowbreak\token[0.5082]| both|\allowbreak\token[0.4297]| the|\allowbreak\token[0.9413]| cases|\allowbreak\token[0.0614]| where|\allowbreak\token[0.7795]| hair|\allowbreak\token[0.0001]| is|\allowbreak\token[0.3236]| included|\allowbreak\token[0.0143]| and|\allowbreak\token[0.0150]| where|\allowbreak\token[0.0194]| it|\allowbreak\token[0.0295]| is|\allowbreak\token[0.2734]| not|\allowbreak\token[0.1738]|.↩↩|\tighttracebreak\allowbreak\token[1.0684]|Let|\allowbreak\token[0.0002]|'s|\allowbreak\token[0.1623]| calculate|\allowbreak\token[0.0563]| the|\allowbreak\token[0.1113]| total|\allowbreak\token[0.0002]| number|\allowbreak\token[0.0000]| of|\allowbreak\token[0.8591]| different|\allowbreak\token[0.0343]| personalities|\allowbreak\token[0.5088]|:↩↩|\tighttracebreak\allowbreak\token[0.0043]|1|\allowbreak\token[0.0000]|.|\allowbreak\token[1.4116]| **|\allowbreak\token[1.9181]|Without|\allowbreak\token[0.6759]| hair|\allowbreak\token[0.2040]|**:|\allowbreak\token[1.2710]| We|\allowbreak\token[0.2964]| have|\allowbreak\token[0.0932]| |\allowbreak\token[0.1357]|2|\allowbreak\token[0.5717]| choices|\allowbreak\token[0.0111]| for|\allowbreak\token[0.8069]| eyebrows|\allowbreak\token[0.1104]|,|\allowbreak\token[0.0000]| |\allowbreak\token[0.0216]|2|\allowbreak\token[0.0027]| choices|\allowbreak\token[0.0000]| for|\allowbreak\token[0.0036]| ears|\allowbreak\token[0.0000]|,|\allowbreak\token[0.0042]| |
\tcblower
\token[1.3155]|left|\allowbreak\token[0.0009]| and|\allowbreak\token[0.0003]| right|\allowbreak\token[0.0304]|)↩|\newline\allowbreak\token[0.0004]|4|\allowbreak\token[0.0000]|.|\allowbreak\token[0.0000]| **|\allowbreak\token[0.5361]|Sh|\allowbreak\token[0.0001]|oes|\allowbreak\token[0.0001]|**:|\allowbreak\token[0.0012]| |\allowbreak\token[0.1154]|2|\allowbreak\token[0.6222]| pairs|\allowbreak\token[0.0008]| (|\allowbreak\token[0.0118]|regular|\allowbreak\token[0.2014]| and|\allowbreak\token[0.0001]| clown|\allowbreak\token[0.6498]|)↩|\newline\allowbreak\token[0.0001]|5|\allowbreak\token[0.0000]|.|\allowbreak\token[0.0000]| **|\allowbreak\token[1.4239]|Hair|\allowbreak\token[0.0779]|**:|\allowbreak\token[0.0051]| |\allowbreak\token[0.3426]|3|\allowbreak\token[0.8359]| hairstyles|\allowbreak\token[1.2007]| (|\allowbreak\token[0.6706]|including|\allowbreak\token[0.2054]| bald|\allowbreak\token[1.2210]|)↩↩|\tighttracebreak\allowbreak\token[0.7873]|Since|\allowbreak\token[1.3220]| the|\allowbreak\token[2.2854]| order|\allowbreak\token[0.1657]| of|\allowbreak\token[1.3458]| the|\allowbreak\token[2.1562]| eyebrows|\allowbreak\token[0.0826]|,|\allowbreak\token[0.0209]| ears|\allowbreak\token[0.0000]|,|\allowbreak\token[0.7178]| lips|\allowbreak\token[0.0114]|,|\allowbreak\token[0.1166]| and|\allowbreak\token[0.6956]| shoes|\allowbreak\token[1.0102]| does|\allowbreak\token[0.0607]| not|\allowbreak\token[0.0043]| matter|\allowbreak\token[0.7965]| (|\allowbreak\token[1.3602]|we|\allowbreak\token[1.3636]| are|\allowbreak\token[1.1540]| not|\allowbreak\token[0.3947]| mixing|\allowbreak\token[0.0299]| and|\allowbreak\token[0.0004]| matching|\allowbreak\token[1.2913]|),|\allowbreak\token[0.0334]| we|\allowbreak\token[0.0528]| can|\allowbreak\token[1.1144]| simply|\allowbreak\token[0.0051]| multiply|\allowbreak\token[0.0004]| the|\allowbreak\token[0.0001]| number|\allowbreak\token[0.0000]| of|\allowbreak\token[0.5832]| choices|\allowbreak\token[0.0915]| for|\allowbreak\token[0.0000]| each|\allowbreak\token[1.5714]| category|\allowbreak\token[1.2503]|.|\allowbreak\token[1.2467]| The|\allowbreak\token[1.2818]| hair|\allowbreak\token[0.5326]| option|\allowbreak\token[0.5349]| is|\allowbreak\token[1.3481]| optional|\allowbreak\token[0.0017]|,|\allowbreak\token[0.0254]| so|\allowbreak\token[0.2166]| we|\allowbreak\token[0.6705]| need|\allowbreak\token[0.0000]| to|\allowbreak\token[0.0477]| consider|\allowbreak\token[0.8964]| both|\allowbreak\token[0.3363]| the|\allowbreak\token[1.0705]| bald|\allowbreak\token[0.7318]| option|\allowbreak\token[0.0403]| and|\allowbreak\token[0.0142]| the|\allowbreak\token[1.3904]| |\allowbreak\token[0.0670]|3|\allowbreak\token[0.3955]| hairstyles|\allowbreak\token[0.2456]|.↩↩|\tighttracebreak\allowbreak\token[0.7516]|Let|\allowbreak\token[0.0001]|'s|\allowbreak\token[0.0636]| calculate|\allowbreak\token[0.6130]| the|\allowbreak\token[0.0001]| total|\allowbreak\token[0.0000]| number|\allowbreak\token[0.0000]| of|\allowbreak\token[0.1388]| different|\allowbreak\token[0.0196]| personalities|\allowbreak\token[0.4034]|:↩↩|\tighttracebreak\allowbreak\token[0.6472]|1|\allowbreak\token[0.0000]|.|\allowbreak\token[1.5504]| For|\allowbreak\token[0.8066]| eyebrows|\allowbreak\token[0.0506]|,|\allowbreak\token[0.6958]| ears|\allowbreak\token[0.0000]|,|\allowbreak\token[0.1322]| lips|\allowbreak\token[0.0000]|,|\allowbreak\token[0.0023]| and|\allowbreak\token[0.0052]| shoes|\allowbreak\token[0.0068]|,|
\end{twopartbox}
\begin{twopartbox}{Cut Index 104 | Acceptance Probability 1.0000}{Block 1 Step 2}{Proposal (Accepted)}
\token[2.7697]| To|\allowbreak\token[1.0000]| solve|\allowbreak\token[0.4562]| this|\allowbreak\token[0.1579]| problem|\allowbreak\token[0.0266]|,|\allowbreak\token[0.2968]| we|\allowbreak\token[0.8181]| need|\allowbreak\token[0.0021]| to|\allowbreak\token[1.5837]| calculate|\allowbreak\token[0.0816]| the|\allowbreak\token[0.4486]| total|\allowbreak\token[0.0082]| number|\allowbreak\token[0.0005]| of|\allowbreak\token[1.0721]| different|\allowbreak\token[1.2499]| personalities|\allowbreak\token[0.8089]| Tom|\allowbreak\token[0.0128]| can|\allowbreak\token[0.1936]| create|\allowbreak\token[1.8537]| by|\allowbreak\token[0.9061]| combining|\allowbreak\token[1.3426]| the|\allowbreak\token[1.4240]| available|\allowbreak\token[1.7869]| options|\allowbreak\token[0.9441]| for|\allowbreak\token[0.7715]| eyebrows|\allowbreak\token[0.0072]|,|\allowbreak\token[0.6925]| ears|\allowbreak\token[0.0018]|,|\allowbreak\token[0.1585]| lips|\allowbreak\token[0.0020]|,|\allowbreak\token[1.3031]| shoes|\allowbreak\token[0.0554]|,|\allowbreak\token[0.1710]| and|\allowbreak\token[0.0844]| optionally|\allowbreak\token[2.0997]| hair|\allowbreak\token[1.0413]|.|\allowbreak\token[1.6814]| Here|\allowbreak\token[0.4489]|'s|\allowbreak\token[0.4810]| the|\allowbreak\token[0.6868]| breakdown|\allowbreak\token[0.3409]| of|\allowbreak\token[0.0250]| the|\allowbreak\token[0.8415]| options|\allowbreak\token[0.1294]|:↩↩|\tighttracebreak\allowbreak\token[0.6935]|1|\allowbreak\token[0.0004]|.|\allowbreak\token[0.8588]| **|\allowbreak\token[0.7153]|Ey|\allowbreak\token[0.0004]|eb|\allowbreak\token[0.0114]|rows|\allowbreak\token[0.1938]|**:|\allowbreak\token[0.4272]| |\allowbreak\token[0.1218]|2|\allowbreak\token[0.5041]| sets|\allowbreak\token[1.1879]| (|\allowbreak\token[0.3962]|left|\allowbreak\token[0.0439]| and|\allowbreak\token[0.0002]| right|\allowbreak\token[0.8898]|)↩|\newline\allowbreak\token[0.0053]|2|\allowbreak\token[0.0000]|.|\allowbreak\token[0.0000]| **|\allowbreak\token[0.0042]|E|\allowbreak\token[0.0095]|ars|\allowbreak\token[0.0000]|**:|\allowbreak\token[0.0000]| |\allowbreak\token[0.0002]|2|\allowbreak\token[0.0016]| sets|\allowbreak\token[0.1916]| (|\allowbreak\token[0.0569]|left|\allowbreak\token[0.0004]| and|\allowbreak\token[0.0002]| right|\allowbreak\token[0.0085]|)↩|\newline\allowbreak\token[0.0002]|3|\allowbreak\token[0.0000]|.|\allowbreak\token[0.0000]| **|\allowbreak\token[0.0646]|L|\allowbreak\token[0.0004]|ips|\allowbreak\token[0.0002]|**:|\allowbreak\token[0.0000]| |\allowbreak\token[0.0003]|2|\allowbreak\token[0.0001]| sets|\allowbreak\token[0.5877]| (|\allowbreak\token[1.3208]|left|\allowbreak\token[0.0010]| and|\allowbreak\token[0.0003]| right|\allowbreak\token[0.0304]|)↩|\newline\allowbreak\token[0.0004]|4|\allowbreak\token[0.0000]|.|\allowbreak\token[0.0000]| **|\allowbreak\token[0.5648]|Sh|\allowbreak\token[0.0001]|oes|\allowbreak\token[0.0001]|**:|\allowbreak\token[0.0012]| |\allowbreak\token[0.1043]|2|\allowbreak\token[0.6379]| pairs|\allowbreak\token[0.0007]| (|\allowbreak\token[0.0104]|regular|\allowbreak\token[0.2005]| and|\allowbreak\token[0.0001]| clown|\allowbreak\token[0.6290]|)↩|\newline\allowbreak\token[0.0001]|5|\allowbreak\token[0.0000]|.|\allowbreak\token[0.0000]| **|\allowbreak\cut{[CUT]}\allowbreak\token[1.4248]|H|\allowbreak\token[0.2060]|airs|\allowbreak\token[0.3892]|**:|\allowbreak\token[0.0024]| |\allowbreak\token[0.2933]|3|\allowbreak\token[1.3977]| hairstyles|\allowbreak\token[1.1485]| (|\allowbreak\token[0.5080]|including|\allowbreak\token[0.1197]| bald|\allowbreak\token[1.3234]|)↩↩|\tighttracebreak\allowbreak\token[0.7971]|Since|\allowbreak\token[1.3248]| the|\allowbreak\token[2.2634]| eyebrows|\allowbreak\token[0.1009]|,|\allowbreak\token[0.0692]| ears|\allowbreak\token[0.0000]|,|\allowbreak\token[0.7124]| lips|\allowbreak\token[0.0019]|,|\allowbreak\token[0.1110]| and|\allowbreak\token[0.4128]| shoes|\allowbreak\token[1.2684]| are|\allowbreak\token[1.7082]| mandatory|\allowbreak\token[0.4712]|,|\allowbreak\token[0.0851]| we|\allowbreak\token[1.6567]| only|\allowbreak\token[0.2410]| need|\allowbreak\token[0.0002]| to|\allowbreak\token[0.0763]| consider|\allowbreak\token[0.0599]| the|\allowbreak\token[0.6851]| combinations|\allowbreak\token[0.4937]| of|\allowbreak\token[0.3028]| these|\allowbreak\token[1.8726]|.|\allowbreak\token[0.5652]| The|\allowbreak\token[1.1044]| hair|\allowbreak\token[0.7671]| option|\allowbreak\token[0.2415]| is|\allowbreak\token[0.0253]| optional|\allowbreak\token[0.0104]|,|\allowbreak\token[0.0760]| so|\allowbreak\token[0.3457]| we|\allowbreak\token[0.7105]| need|\allowbreak\token[0.0000]| to|\allowbreak\token[0.3944]| consider|\allowbreak\token[0.5082]| both|\allowbreak\token[0.4297]| the|\allowbreak\token[0.9413]| cases|\allowbreak\token[0.0614]| where|\allowbreak\token[0.7795]| hair|\allowbreak\token[0.0001]| is|\allowbreak\token[0.3236]| included|\allowbreak\token[0.0143]| and|\allowbreak\token[0.0150]| where|\allowbreak\token[0.0194]| it|\allowbreak\token[0.0295]| is|\allowbreak\token[0.2734]| not|\allowbreak\token[0.1738]|.↩↩|\tighttracebreak\allowbreak\token[1.0684]|Let|\allowbreak\token[0.0002]|'s|\allowbreak\token[0.1623]| calculate|\allowbreak\token[0.0563]| the|\allowbreak\token[0.1113]| total|\allowbreak\token[0.0002]| number|\allowbreak\token[0.0000]| of|\allowbreak\token[0.8591]| different|\allowbreak\token[0.0343]| personalities|\allowbreak\token[0.5088]|:↩↩|\tighttracebreak\allowbreak\token[0.0043]|1|\allowbreak\token[0.0000]|.|\allowbreak\token[1.4116]| **|\allowbreak\token[1.9181]|Without|\allowbreak\token[0.6759]| hair|\allowbreak\token[0.2040]|**:|\allowbreak\token[1.2710]| We|\allowbreak\token[0.2964]| have|\allowbreak\token[0.0932]| |\allowbreak\token[0.1357]|2|\allowbreak\token[0.5717]| choices|\allowbreak\token[0.0111]| for|\allowbreak\token[0.8069]| eyebrows|\allowbreak\token[0.1104]|,|\allowbreak\token[0.0000]| |\allowbreak\token[0.0216]|2|\allowbreak\token[0.0027]| choices|\allowbreak\token[0.0000]| for|\allowbreak\token[0.0036]| ears|\allowbreak\token[0.0000]|,|\allowbreak\token[0.0042]| |
\tcblower
\token[1.4236]|Hair|\allowbreak\token[0.0736]|**:|\allowbreak\token[0.0056]| |\allowbreak\token[0.3179]|3|\allowbreak\token[0.8361]| hairstyles|\allowbreak\token[1.1954]| (|\allowbreak\token[0.6671]|including|\allowbreak\token[0.1912]| bald|\allowbreak\token[1.2099]|)↩↩|\tighttracebreak\allowbreak\token[0.7840]|Since|\allowbreak\token[1.3340]| Tom|\allowbreak\token[0.1609]| can|\allowbreak\token[0.6689]| choose|\allowbreak\token[1.0828]| to|\allowbreak\token[1.8048]| not|\allowbreak\token[1.3054]| include|\allowbreak\token[0.7798]| hair|\allowbreak\token[1.4518]|,|\allowbreak\token[0.4165]| we|\allowbreak\token[1.0062]| need|\allowbreak\token[0.0000]| to|\allowbreak\token[0.1151]| consider|\allowbreak\token[1.3058]| the|\allowbreak\token[0.6299]| bald|\allowbreak\token[0.0517]| option|\allowbreak\token[0.3821]| as|\allowbreak\token[0.6127]| well|\allowbreak\token[0.5054]|.|\allowbreak\token[0.5940]| This|\allowbreak\token[0.0508]| means|\allowbreak\token[0.5136]| there|\allowbreak\token[0.0010]| are|\allowbreak\token[0.5754]| |\allowbreak\token[0.5327]|4|\allowbreak\token[1.5106]| possible|\allowbreak\token[0.7006]| choices|\allowbreak\token[0.0007]| for|\allowbreak\token[0.4840]| hair|\allowbreak\token[0.4028]| (|\allowbreak\token[0.2223]|3|\allowbreak\token[0.0117]| hairstyles|\allowbreak\token[0.2084]| +|\allowbreak\token[0.6068]| bald|\allowbreak\token[0.0230]|).↩↩|\tighttracebreak\allowbreak\token[1.1681]|The|\allowbreak\token[0.0299]| total|\allowbreak\token[0.0000]| number|\allowbreak\token[0.0000]| of|\allowbreak\token[0.0126]| different|\allowbreak\token[0.0179]| personalities|\allowbreak\token[0.2055]| can|\allowbreak\token[0.0000]| be|\allowbreak\token[0.0001]| calculated|\allowbreak\token[0.0127]| by|\allowbreak\token[0.0028]| multiplying|\allowbreak\token[0.0002]| the|\allowbreak\token[0.0000]| number|\allowbreak\token[0.0000]| of|\allowbreak\token[0.1324]| choices|\allowbreak\token[0.0002]| for|\allowbreak\token[0.0004]| each|\allowbreak\token[1.4233]| feature|\allowbreak\token[0.6486]|:↩↩|\tighttracebreak\allowbreak\token[0.3738]|\|\allowbreak\token[0.0000]|[↩|\newline\allowbreak\token[0.0033]|\|\allowbreak\token[0.0000]|text|\allowbreak\token[0.0000]|{|\allowbreak\token[0.0000]|Total|\allowbreak\token[0.5331]| Personal|\allowbreak\token[0.0000]|ities|\allowbreak\token[0.0000]|}|\allowbreak\token[0.0000]| =|\allowbreak\token[0.4698]| (\|\allowbreak\token[0.0000]|text|\allowbreak\token[0.0020]|{|\allowbreak\token[0.9551]|Number|\allowbreak\token[0.0000]| of|\allowbreak\token[0.6774]| Ey|\allowbreak\token[0.0082]|eb|\allowbreak\token[0.1627]|row|\allowbreak\token[0.7888]| Sets|\allowbreak\token[0.0335]|})|\allowbreak\token[0.0000]| \|\allowbreak\token[0.0000]|times|\allowbreak\token[0.0000]| (\|\allowbreak\token[0.0001]|text|\allowbreak\token[0.0001]|{|
\end{twopartbox}
\begin{twopartbox}{Cut Index 34 | Acceptance Probability 0.0022}{Block 1 Step 3}{Proposal (Rejected)}
\token[2.7697]| To|\allowbreak\token[1.0000]| solve|\allowbreak\token[0.4562]| this|\allowbreak\token[0.1579]| problem|\allowbreak\token[0.0266]|,|\allowbreak\token[0.2968]| we|\allowbreak\token[0.8181]| need|\allowbreak\token[0.0021]| to|\allowbreak\token[1.5837]| calculate|\allowbreak\token[0.0816]| the|\allowbreak\token[0.4486]| total|\allowbreak\token[0.0082]| number|\allowbreak\token[0.0005]| of|\allowbreak\token[1.0721]| different|\allowbreak\token[1.2499]| personalities|\allowbreak\token[0.8089]| Tom|\allowbreak\token[0.0128]| can|\allowbreak\token[0.1936]| create|\allowbreak\token[1.8537]| by|\allowbreak\token[0.9061]| combining|\allowbreak\token[1.3426]| the|\allowbreak\token[1.4240]| available|\allowbreak\token[1.7869]| options|\allowbreak\token[0.9441]| for|\allowbreak\token[0.7715]| eyebrows|\allowbreak\token[0.0072]|,|\allowbreak\token[0.6925]| ears|\allowbreak\token[0.0018]|,|\allowbreak\token[0.1585]| lips|\allowbreak\token[0.0020]|,|\allowbreak\token[1.3031]| shoes|\allowbreak\token[0.0554]|,|\allowbreak\token[0.1710]| and|\allowbreak\token[0.0844]| optionally|\allowbreak\cut{[CUT]}\allowbreak\token[2.0997]| hair|\allowbreak\token[1.0413]|.|\allowbreak\token[1.6814]| Here|\allowbreak\token[0.4489]|'s|\allowbreak\token[0.4810]| the|\allowbreak\token[0.6868]| breakdown|\allowbreak\token[0.3409]| of|\allowbreak\token[0.0250]| the|\allowbreak\token[0.8415]| options|\allowbreak\token[0.1294]|:↩↩|\tighttracebreak\allowbreak\token[0.6935]|1|\allowbreak\token[0.0004]|.|\allowbreak\token[0.8588]| **|\allowbreak\token[0.7153]|Ey|\allowbreak\token[0.0004]|eb|\allowbreak\token[0.0114]|rows|\allowbreak\token[0.1938]|**:|\allowbreak\token[0.4272]| |\allowbreak\token[0.1218]|2|\allowbreak\token[0.5041]| sets|\allowbreak\token[1.1879]| (|\allowbreak\token[0.3962]|left|\allowbreak\token[0.0439]| and|\allowbreak\token[0.0002]| right|\allowbreak\token[0.8898]|)↩|\newline\allowbreak\token[0.0053]|2|\allowbreak\token[0.0000]|.|\allowbreak\token[0.0000]| **|\allowbreak\token[0.0042]|E|\allowbreak\token[0.0095]|ars|\allowbreak\token[0.0000]|**:|\allowbreak\token[0.0000]| |\allowbreak\token[0.0002]|2|\allowbreak\token[0.0016]| sets|\allowbreak\token[0.1916]| (|\allowbreak\token[0.0569]|left|\allowbreak\token[0.0004]| and|\allowbreak\token[0.0002]| right|\allowbreak\token[0.0085]|)↩|\newline\allowbreak\token[0.0002]|3|\allowbreak\token[0.0000]|.|\allowbreak\token[0.0000]| **|\allowbreak\token[0.0646]|L|\allowbreak\token[0.0004]|ips|\allowbreak\token[0.0002]|**:|\allowbreak\token[0.0000]| |\allowbreak\token[0.0003]|2|\allowbreak\token[0.0001]| sets|\allowbreak\token[0.5877]| (|\allowbreak\token[1.3208]|left|\allowbreak\token[0.0010]| and|\allowbreak\token[0.0003]| right|\allowbreak\token[0.0304]|)↩|\newline\allowbreak\token[0.0004]|4|\allowbreak\token[0.0000]|.|\allowbreak\token[0.0000]| **|\allowbreak\token[0.5648]|Sh|\allowbreak\token[0.0001]|oes|\allowbreak\token[0.0001]|**:|\allowbreak\token[0.0012]| |\allowbreak\token[0.1043]|2|\allowbreak\token[0.6379]| pairs|\allowbreak\token[0.0007]| (|\allowbreak\token[0.0104]|regular|\allowbreak\token[0.2005]| and|\allowbreak\token[0.0001]| clown|\allowbreak\token[0.6290]|)↩|\newline\allowbreak\token[0.0001]|5|\allowbreak\token[0.0000]|.|\allowbreak\token[0.0000]| **|\allowbreak\token[1.4236]|Hair|\allowbreak\token[0.0736]|**:|\allowbreak\token[0.0056]| |\allowbreak\token[0.3179]|3|\allowbreak\token[0.8361]| hairstyles|\allowbreak\token[1.1954]| (|\allowbreak\token[0.6671]|including|\allowbreak\token[0.1912]| bald|\allowbreak\token[1.2099]|)↩↩|\tighttracebreak\allowbreak\token[0.7840]|Since|\allowbreak\token[1.3340]| Tom|\allowbreak\token[0.1609]| can|\allowbreak\token[0.6689]| choose|\allowbreak\token[1.0828]| to|\allowbreak\token[1.8048]| not|\allowbreak\token[1.3054]| include|\allowbreak\token[0.7798]| hair|\allowbreak\token[1.4518]|,|\allowbreak\token[0.4165]| we|\allowbreak\token[1.0062]| need|\allowbreak\token[0.0000]| to|\allowbreak\token[0.1151]| consider|\allowbreak\token[1.3058]| the|\allowbreak\token[0.6299]| bald|\allowbreak\token[0.0517]| option|\allowbreak\token[0.3821]| as|\allowbreak\token[0.6127]| well|\allowbreak\token[0.5054]|.|\allowbreak\token[0.5940]| This|\allowbreak\token[0.0508]| means|\allowbreak\token[0.5136]| there|\allowbreak\token[0.0010]| are|\allowbreak\token[0.5754]| |\allowbreak\token[0.5327]|4|\allowbreak\token[1.5106]| possible|\allowbreak\token[0.7006]| choices|\allowbreak\token[0.0007]| for|\allowbreak\token[0.4840]| hair|\allowbreak\token[0.4028]| (|\allowbreak\token[0.2223]|3|\allowbreak\token[0.0117]| hairstyles|\allowbreak\token[0.2084]| +|\allowbreak\token[0.6068]| bald|\allowbreak\token[0.0230]|).↩↩|\tighttracebreak\allowbreak\token[1.1681]|The|\allowbreak\token[0.0299]| total|\allowbreak\token[0.0000]| number|\allowbreak\token[0.0000]| of|\allowbreak\token[0.0126]| different|\allowbreak\token[0.0179]| personalities|\allowbreak\token[0.2055]| can|\allowbreak\token[0.0000]| be|\allowbreak\token[0.0001]| calculated|\allowbreak\token[0.0127]| by|\allowbreak\token[0.0028]| multiplying|\allowbreak\token[0.0002]| the|\allowbreak\token[0.0000]| number|\allowbreak\token[0.0000]| of|\allowbreak\token[0.1324]| choices|\allowbreak\token[0.0002]| for|\allowbreak\token[0.0004]| each|\allowbreak\token[1.4233]| feature|\allowbreak\token[0.6486]|:↩↩|\tighttracebreak\allowbreak\token[0.3738]|\|\allowbreak\token[0.0000]|[↩|\newline\allowbreak\token[0.0033]|\|\allowbreak\token[0.0000]|text|\allowbreak\token[0.0000]|{|\allowbreak\token[0.0000]|Total|\allowbreak\token[0.5331]| Personal|\allowbreak\token[0.0000]|ities|\allowbreak\token[0.0000]|}|\allowbreak\token[0.0000]| =|\allowbreak\token[0.4698]| (\|\allowbreak[...]
\tcblower
\token[2.0797]| hair|\allowbreak\token[0.9887]|.|\allowbreak\token[1.6727]| Here|\allowbreak\token[0.4489]|'s|\allowbreak\token[0.4810]| the|\allowbreak\token[0.6886]| breakdown|\allowbreak\token[0.3409]| of|\allowbreak\token[0.0251]| the|\allowbreak\token[0.7919]| options|\allowbreak\token[0.1211]|:↩↩|\tighttracebreak\allowbreak\token[0.6858]|-|\allowbreak\token[1.6674]| |\allowbreak\token[0.6366]|2|\allowbreak\token[0.0604]| sets|\allowbreak\token[0.0001]| of|\allowbreak\token[0.0002]| eyebrows|\allowbreak\token[0.8122]|↩|\newline\allowbreak\token[0.0002]|-|\allowbreak\token[0.0004]| |\allowbreak\token[0.6654]|1|\allowbreak\token[0.0683]| pair|\allowbreak\token[0.0001]| of|\allowbreak\token[0.0969]| goog|\allowbreak\token[0.0000]|ly|\allowbreak\token[0.0000]| eyes|\allowbreak\token[0.6861]| (|\allowbreak\token[1.6537]|only|\allowbreak\token[0.5938]| one|\allowbreak\token[0.2789]| option|\allowbreak\token[0.5382]|)↩|\newline\allowbreak\token[0.0000]|-|\allowbreak\token[0.0000]| |\allowbreak\token[0.0005]|2|\allowbreak\token[0.0005]| sets|\allowbreak\token[0.0000]| of|\allowbreak\token[0.0000]| ears|\allowbreak\token[0.0032]|↩|\newline\allowbreak\token[0.0000]|-|\allowbreak\token[0.0000]| |\allowbreak\token[0.0025]|2|\allowbreak\token[0.0002]| sets|\allowbreak\token[0.0001]| of|\allowbreak\token[0.0009]| lips|\allowbreak\token[0.0048]|↩|\newline\allowbreak\token[0.0001]|-|\allowbreak\token[0.0762]| |\allowbreak\token[0.4600]|2|\allowbreak\token[0.7838]| pairs|\allowbreak\token[0.0000]| of|\allowbreak\token[0.0141]| shoes|\allowbreak\token[0.3670]| (|\allowbreak\token[1.1656]|regular|\allowbreak\token[0.7576]| and|\allowbreak\token[0.0005]| clown|\allowbreak\token[0.7004]| shoes|\allowbreak\token[0.0091]|)↩|\newline\allowbreak\token[0.0001]|-|\allowbreak\token[0.0699]| |\allowbreak\token[0.0951]|3|\allowbreak\token[0.0024]| hairstyles|\allowbreak\token[0.4598]| (|\allowbreak\token[0.2078]|including|\allowbreak\token[0.4283]| bald|\allowbreak\token[1.0257]|)↩↩|\tighttracebreak\allowbreak\token[0.6872]|Since|\allowbreak\token[1.1601]| the|\allowbreak\token[1.6097]| eyes|\allowbreak\token[0.7048]| are|\allowbreak\token[0.4789]| fixed|\allowbreak\token[1.2019]|,|\allowbreak\token[0.0474]| we|\allowbreak\token[0.9306]| don|\allowbreak\token[0.0002]|'t|\allowbreak\token[0.2000]| need|\allowbreak\token[0.0000]| to|\allowbreak\token[0.2544]| consider|\allowbreak\token[0.9043]| them|\allowbreak\token[0.2111]| in|\allowbreak\token[0.6907]| our|\allowbreak\token[1.7640]| calculations|\allowbreak\token[0.7411]|.|\allowbreak\token[0.4820]| We|\allowbreak\token[1.4176]| can|\allowbreak\token[1.4453]| calculate|\allowbreak\token[0.0000]| the|\allowbreak\token[0.1103]| total|\allowbreak\token[0.0003]| number|\allowbreak\token[0.0000]| of|\allowbreak\token[0.8539]| different|\allowbreak\token[0.1085]| personalities|\allowbreak\token[0.0053]| by|\allowbreak\token[0.0046]| multiplying|\allowbreak\token[0.0003]| the|\allowbreak\token[0.0002]| number|\allowbreak\token[0.0000]| of|\allowbreak\token[0.5600]| options|\allowbreak\token[0.0052]| for|\allowbreak\token[0.5461]| each|\allowbreak\token[1.7873]| of|\allowbreak\token[0.0006]| the|\allowbreak\token[0.5063]| other|\allowbreak\token[1.1591]| components|\allowbreak\token[1.2153]|.|\allowbreak\token[2.1480]| The|\allowbreak\token[1.5947]| formula|\allowbreak\token[0.9418]| is|\allowbreak\token[0.0174]|:↩↩|\tighttracebreak\allowbreak\token[0.2827]|\[|\allowbreak\token[0.0001]| \|\allowbreak\token[0.0000]|text|\allowbreak\token[0.0000]|{|\allowbreak\token[0.0005]|Total|\allowbreak\token[0.2893]| Personal|\allowbreak\token[0.0000]|ities|\allowbreak\token[0.0000]|}|\allowbreak\token[0.0000]| =|\allowbreak\token[0.3357]| (\|\allowbreak\token[0.0000]|text|\allowbreak\token[0.0014]|{|\allowbreak\token[0.3671]|Number|\allowbreak\token[0.0000]| of|\allowbreak\token[0.8411]| eyebrow|\allowbreak\token[0.6867]| sets|\allowbreak\token[0.0629]|})|\allowbreak\token[0.0001]| \|\allowbreak\token[0.0000]|times|\allowbreak\token[0.0008]| (\|\allowbreak\token[0.0000]|text|\allowbreak\token[0.0000]|{|\allowbreak\token[0.0001]|Number|\allowbreak\token[0.0000]| of|\allowbreak\token[0.1352]| ear|\allowbreak\token[0.0002]| sets|\allowbreak\token[0.0000]|})|\allowbreak\token[0.0000]| \|\allowbreak\token[0.0000]|times|\allowbreak\token[0.0000]| (\|\allowbreak\token[0.0000]|text|\allowbreak\token[0.0000]|{|\allowbreak\token[0.0000]|Number|\allowbreak\token[0.0001]| of|\allowbreak\token[0.2154]| lip|\allowbreak\token[0.0001]| sets|\allowbreak\token[0.0000]|})|\allowbreak[...]
\end{twopartbox}
\begin{twopartbox}{Cut Index 180 | Acceptance Probability 0.7464}{Block 2 Step 1}{Proposal (Rejected)}
[...]\allowbreak\token[1.6814]| Here|\allowbreak\token[0.4489]|'s|\allowbreak\token[0.4810]| the|\allowbreak\token[0.6868]| breakdown|\allowbreak\token[0.3409]| of|\allowbreak\token[0.0250]| the|\allowbreak\token[0.8415]| options|\allowbreak\token[0.1294]|:↩↩|\tighttracebreak\allowbreak\token[0.6935]|1|\allowbreak\token[0.0004]|.|\allowbreak\token[0.8588]| **|\allowbreak\token[0.7153]|Ey|\allowbreak\token[0.0004]|eb|\allowbreak\token[0.0114]|rows|\allowbreak\token[0.1938]|**:|\allowbreak\token[0.4272]| |\allowbreak\token[0.1218]|2|\allowbreak\token[0.5041]| sets|\allowbreak\token[1.1879]| (|\allowbreak\token[0.3962]|left|\allowbreak\token[0.0439]| and|\allowbreak\token[0.0002]| right|\allowbreak\token[0.8898]|)↩|\newline\allowbreak\token[0.0053]|2|\allowbreak\token[0.0000]|.|\allowbreak\token[0.0000]| **|\allowbreak\token[0.0042]|E|\allowbreak\token[0.0095]|ars|\allowbreak\token[0.0000]|**:|\allowbreak\token[0.0000]| |\allowbreak\token[0.0002]|2|\allowbreak\token[0.0016]| sets|\allowbreak\token[0.1916]| (|\allowbreak\token[0.0569]|left|\allowbreak\token[0.0004]| and|\allowbreak\token[0.0002]| right|\allowbreak\token[0.0085]|)↩|\newline\allowbreak\token[0.0002]|3|\allowbreak\token[0.0000]|.|\allowbreak\token[0.0000]| **|\allowbreak\token[0.0646]|L|\allowbreak\token[0.0004]|ips|\allowbreak\token[0.0002]|**:|\allowbreak\token[0.0000]| |\allowbreak\token[0.0003]|2|\allowbreak\token[0.0001]| sets|\allowbreak\token[0.5877]| (|\allowbreak\token[1.3208]|left|\allowbreak\token[0.0010]| and|\allowbreak\token[0.0003]| right|\allowbreak\token[0.0304]|)↩|\newline\allowbreak\token[0.0004]|4|\allowbreak\token[0.0000]|.|\allowbreak\token[0.0000]| **|\allowbreak\token[0.5648]|Sh|\allowbreak\token[0.0001]|oes|\allowbreak\token[0.0001]|**:|\allowbreak\token[0.0012]| |\allowbreak\token[0.1043]|2|\allowbreak\token[0.6379]| pairs|\allowbreak\token[0.0007]| (|\allowbreak\token[0.0104]|regular|\allowbreak\token[0.2005]| and|\allowbreak\token[0.0001]| clown|\allowbreak\token[0.6290]|)↩|\newline\allowbreak\token[0.0001]|5|\allowbreak\token[0.0000]|.|\allowbreak\token[0.0000]| **|\allowbreak\token[1.4236]|Hair|\allowbreak\token[0.0736]|**:|\allowbreak\token[0.0056]| |\allowbreak\token[0.3179]|3|\allowbreak\token[0.8361]| hairstyles|\allowbreak\token[1.1954]| (|\allowbreak\token[0.6671]|including|\allowbreak\token[0.1912]| bald|\allowbreak\token[1.2099]|)↩↩|\tighttracebreak\allowbreak\token[0.7840]|Since|\allowbreak\token[1.3340]| Tom|\allowbreak\token[0.1609]| can|\allowbreak\token[0.6689]| choose|\allowbreak\token[1.0828]| to|\allowbreak\token[1.8048]| not|\allowbreak\token[1.3054]| include|\allowbreak\token[0.7798]| hair|\allowbreak\token[1.4518]|,|\allowbreak\token[0.4165]| we|\allowbreak\token[1.0062]| need|\allowbreak\token[0.0000]| to|\allowbreak\token[0.1151]| consider|\allowbreak\token[1.3058]| the|\allowbreak\token[0.6299]| bald|\allowbreak\token[0.0517]| option|\allowbreak\token[0.3821]| as|\allowbreak\token[0.6127]| well|\allowbreak\token[0.5054]|.|\allowbreak\token[0.5940]| This|\allowbreak\token[0.0508]| means|\allowbreak\token[0.5136]| there|\allowbreak\token[0.0010]| are|\allowbreak\token[0.5754]| |\allowbreak\token[0.5327]|4|\allowbreak\token[1.5106]| possible|\allowbreak\token[0.7006]| choices|\allowbreak\token[0.0007]| for|\allowbreak\token[0.4840]| hair|\allowbreak\token[0.4028]| (|\allowbreak\token[0.2223]|3|\allowbreak\token[0.0117]| hairstyles|\allowbreak\token[0.2084]| +|\allowbreak\token[0.6068]| bald|\allowbreak\token[0.0230]|).↩↩|\tighttracebreak\allowbreak\token[1.1681]|The|\allowbreak\token[0.0299]| total|\allowbreak\token[0.0000]| number|\allowbreak\token[0.0000]| of|\allowbreak\token[0.0126]| different|\allowbreak\token[0.0179]| personalities|\allowbreak\token[0.2055]| can|\allowbreak\token[0.0000]| be|\allowbreak\token[0.0001]| calculated|\allowbreak\token[0.0127]| by|\allowbreak\token[0.0028]| multiplying|\allowbreak\token[0.0002]| the|\allowbreak\token[0.0000]| number|\allowbreak\token[0.0000]| of|\allowbreak\token[0.1324]| choices|\allowbreak\token[0.0002]| for|\allowbreak\token[0.0004]| each|\allowbreak\token[1.4233]| feature|\allowbreak\token[0.6486]|:↩↩|\tighttracebreak\allowbreak\token[0.3738]|\|\allowbreak\token[0.0000]|[↩|\newline\allowbreak\token[0.0033]|\|\allowbreak\token[0.0000]|text|\allowbreak\token[0.0000]|{|\allowbreak\token[0.0000]|Total|\allowbreak\token[0.5331]| Personal|\allowbreak\token[0.0000]|ities|\allowbreak\token[0.0000]|}|\allowbreak\token[0.0000]| =|\allowbreak\token[0.4698]| (\|\allowbreak\token[0.0000]|text|\allowbreak\token[0.0020]|{|\allowbreak\cut{[CUT]}\allowbreak\token[0.9551]|Number|\allowbreak\token[0.0000]| of|\allowbreak\token[0.6774]| Ey|\allowbreak\token[0.0082]|eb|\allowbreak\token[0.1627]|row|\allowbreak\token[0.7888]| Sets|\allowbreak\token[0.0335]|})|\allowbreak\token[0.0000]| \|\allowbreak\token[0.0000]|times|\allowbreak\token[0.0000]| (\|\allowbreak\token[0.0001]|text|\allowbreak\token[0.0001]|{|\allowbreak\token[0.0000]|Number|\allowbreak\token[0.0000]| of|\allowbreak\token[0.1046]| Ear|\allowbreak\token[0.0008]| Sets|\allowbreak\token[0.0000]|})|\allowbreak\token[0.0000]| \|\allowbreak\token[0.0000]|times|\allowbreak\token[0.0001]| (\|\allowbreak\token[0.0000]|text|\allowbreak\token[0.0001]|{|\allowbreak\token[0.0002]|Number|\allowbreak\token[0.0001]| of|\allowbreak\token[0.6739]| Lips|\allowbreak\token[0.0340]| Sets|\allowbreak\token[0.0000]|})|\allowbreak\token[0.0000]| \|\allowbreak\token[0.0000]|times|\allowbreak\token[0.0000]| (\|\allowbreak\token[0.0000]|text|\allowbreak\token[0.0000]|{|\allowbreak\token[0.0002]|Number|\allowbreak\token[0.0000]| of|\allowbreak\token[0.2786]| Shoe|\allowbreak\token[0.0127]| P|\allowbreak\token[0.0002]|airs|\allowbreak\token[0.0001]|})|\allowbreak\token[0.0000]| \|\allowbreak\token[0.0000]|times|\allowbreak\token[0.0000]| (\|\allowbreak\token[0.0000]|text|\allowbreak\token[0.0000]|{|\allowbreak\token[0.0005]|Number|\allowbreak\token[0.0000]| of|\allowbreak\token[0.0323]| Hair|\allowbreak\token[1.0010]| Options|\allowbreak\token[0.0211]|})↩|\newline\allowbreak\token[0.0000]|\|\allowbreak\token[0.0000]|]↩↩|\tighttracebreak\allowbreak\token[0.0397]|Let|\allowbreak\token[0.0001]|'s|\allowbreak\token[0.0050]| calculate|\allowbreak\token[0.0001]| this|\allowbreak\token[0.0817]| using|\allowbreak\token[0.0000]| Python|\allowbreak\token[1.0118]| to|\allowbreak\token[0.0000]| ensure|\allowbreak\token[0.0000]| accuracy|\allowbreak\token[0.0002]|.↩|\newline\allowbreak\token[0.0000]|```|\allowbreak\token[0.0000]|python|\allowbreak\token[0.0000]|↩|\newline\allowbreak\token[0.0000]|#|\allowbreak\token[0.3500]| Define|\allowbreak\token[0.0000]| the|\allowbreak\token[0.0000]| number|\allowbreak\token[0.0000]| of|\allowbreak\token[0.4203]| options|\allowbreak\token[0.0000]| for|\allowbreak\token[0.0000]| each|\allowbreak\token[0.0241]| feature|\allowbreak\token[0.0156]|↩|\newline\allowbreak\token[0.2517]|ey|\allowbreak\token[0.0004]|eb|\allowbreak\token[0.2755]|row|\allowbreak\token[0.0166]|_sets|\allowbreak\token[0.0000]| =|\allowbreak\token[0.0000]| |\allowbreak\token[0.0106]|2|\allowbreak\token[0.6531]|↩|\newline\allowbreak\token[0.0011]|ear|\allowbreak\token[0.0034]|_sets|\allowbreak\token[0.0000]| =|\allowbreak\token[0.0000]| |\allowbreak\token[0.0001]|2|\allowbreak\token[0.0000]|↩|\newline\allowbreak\token[0.3160]|l|\allowbreak\token[0.0024]|ips|\allowbreak\token[0.0006]|_sets|\allowbreak\token[0.0000]| =|\allowbreak\token[0.0000]| |\allowbreak\token[0.0001]|2|\allowbreak\token[0.0000]|↩|\newline\allowbreak\token[0.0024]|sh|\allowbreak\token[0.0013]|oe|\allowbreak\token[0.0011]|_pairs|\allowbreak\token[0.0000]| =|\allowbreak\token[0.0000]| |\allowbreak\token[0.0003]|2|\allowbreak\token[0.0007]|↩|\newline\allowbreak\token[0.0007]|hair|\allowbreak\token[0.0020]|_options|\allowbreak\token[0.0000]| =|\allowbreak\token[0.0002]| |\allowbreak\token[0.3399]|4|\allowbreak\token[0.0015]| |\allowbreak\token[0.0000]| #|\allowbreak\token[0.1516]| |\allowbreak\token[0.0009]|3|\allowbreak\token[0.0046]| hairstyles|\allowbreak\token[0.0017]| +|\allowbreak\token[0.0677]| bald|\allowbreak\token[0.0025]|↩↩|\tighttracebreak\allowbreak\token[0.0000]|#|\allowbreak\token[0.0000]| Calculate|\allowbreak\token[0.0000]| the|\allowbreak\token[0.0000]| total|\allowbreak\token[0.0000]| number|\allowbreak\token[0.0000]| of|\allowbreak\token[0.0056]| different|\allowbreak\token[0.0011]| personalities|\allowbreak\token[0.0000]|↩|\newline\allowbreak\token[0.0000]|total|\allowbreak\token[0.2264]|_personal|\allowbreak\token[0.0000]|ities|\allowbreak\token[0.0000]| =|\allowbreak\token[0.3750]| (|\allowbreak\token[0.0062]|ey|\allowbreak\token[0.0000]|eb|\allowbreak\token[0.0005]|row|\allowbreak\token[0.0001]|_sets|\allowbreak\token[0.2915]| *|\allowbreak\token[0.0883]| ear|\allowbreak\token[0.0002]|_sets|\allowbreak\token[0.0053]| *|\allowbreak\token[0.0068]| lips|\allowbreak\token[0.0002]|_sets|\allowbreak\token[0.0077]| *|\allowbreak\token[0.0343]| shoe|\allowbreak\token[0.0000]|_pairs|\allowbreak\token[0.0083]| *|\allowbreak\token[0.0005]| hair|\allowbreak\token[0.0005]|_options|\allowbreak[...]
\tcblower
\token[0.9669]|Number|\allowbreak\token[0.0000]| of|\allowbreak\token[0.7124]| Ey|\allowbreak\token[0.0081]|eb|\allowbreak\token[0.1483]|row|\allowbreak\token[0.7711]| Sets|\allowbreak\token[0.0335]|})|\allowbreak\token[0.0000]| \|\allowbreak\token[0.0000]|times|\allowbreak\token[0.0000]| (\|\allowbreak\token[0.0001]|text|\allowbreak\token[0.0001]|{|\allowbreak\token[0.0000]|Number|\allowbreak\token[0.0000]| of|\allowbreak\token[0.1047]| Ear|\allowbreak\token[0.0008]| Sets|\allowbreak\token[0.0000]|})|\allowbreak\token[0.0000]| \|\allowbreak\token[0.0000]|times|\allowbreak\token[0.0001]| (\|\allowbreak\token[0.0000]|text|\allowbreak\token[0.0001]|{|\allowbreak\token[0.0002]|Number|\allowbreak\token[0.0001]| of|\allowbreak\token[0.6888]| Lip|\allowbreak\token[0.0002]| Sets|\allowbreak\token[0.0000]|})|\allowbreak\token[0.0000]| \|\allowbreak\token[0.0000]|times|\allowbreak\token[0.0000]| (\|\allowbreak\token[0.0000]|text|\allowbreak\token[0.0001]|{|\allowbreak\token[0.0001]|Number|\allowbreak\token[0.0000]| of|\allowbreak\token[0.0576]| Shoe|\allowbreak\token[0.0143]| P|\allowbreak\token[0.0001]|airs|\allowbreak\token[0.0001]|})|\allowbreak\token[0.0000]| \|\allowbreak\token[0.0000]|times|\allowbreak\token[0.0000]| (\|\allowbreak\token[0.0000]|text|\allowbreak\token[0.0000]|{|\allowbreak\token[0.0006]|Number|\allowbreak\token[0.0000]| of|\allowbreak\token[0.0351]| Hair|\allowbreak\token[1.0006]| Choices|\allowbreak\token[0.0149]|})↩|\newline\allowbreak\token[0.0000]|\|\allowbreak\token[0.0000]|]↩↩|\tighttracebreak\allowbreak\token[0.0427]|Let|\allowbreak\token[0.0001]|'s|\allowbreak\token[0.0049]| calculate|\allowbreak\token[0.0001]| this|\allowbreak\token[0.0817]| using|\allowbreak\token[0.0000]| Python|\allowbreak\token[1.0253]|.↩|\newline\allowbreak\token[0.0000]|```|\allowbreak\token[0.0000]|python|\allowbreak\token[0.0000]|↩|\newline\allowbreak\token[0.0000]|#|\allowbreak\token[0.5523]| Define|\allowbreak\token[0.0000]| the|\allowbreak\token[0.0000]| number|\allowbreak\token[0.0000]| of|\allowbreak\token[0.3660]| choices|\allowbreak\token[0.0000]| for|\allowbreak\token[0.0000]| each|\allowbreak\token[0.0076]| feature|\allowbreak\token[0.0101]|↩|\newline\allowbreak\token[0.3660]|ey|\allowbreak\token[0.0004]|eb|\allowbreak\token[0.2124]|row|\allowbreak\token[0.0056]|_sets|\allowbreak\token[0.0000]| =|\allowbreak\token[0.0000]| |\allowbreak\token[0.0107]|2|\allowbreak\token[0.6114]|↩|\newline\allowbreak\token[0.0008]|ear|\allowbreak\token[0.0025]|_sets|\allowbreak\token[0.0000]| =|\allowbreak\token[0.0000]| |\allowbreak\token[0.0001]|2|\allowbreak\token[0.0000]|↩|\newline\allowbreak\token[0.0132]|lip|\allowbreak\token[0.0004]|_sets|\allowbreak\token[0.0000]| =|\allowbreak\token[0.0000]| |\allowbreak\token[0.0002]|2|\allowbreak\token[0.0000]|↩|\newline\allowbreak\token[0.0036]|sh|\allowbreak\token[0.0005]|oe|\allowbreak\token[0.0013]|_pairs|\allowbreak\token[0.0000]| =|\allowbreak\token[0.0000]| |\allowbreak\token[0.0005]|2|\allowbreak\token[0.0010]|↩|\newline\allowbreak\token[0.0004]|hair|\allowbreak\token[0.0031]|_choices|\allowbreak\token[0.0001]| =|\allowbreak\token[0.0001]| |\allowbreak\token[0.3922]|4|\allowbreak\token[0.0035]| |\allowbreak\token[0.0000]| #|\allowbreak\token[0.1666]| |\allowbreak\token[0.0008]|3|\allowbreak\token[0.0049]| hairstyles|\allowbreak\token[0.0015]| +|\allowbreak\token[0.0909]| bald|\allowbreak\token[0.0019]|↩↩|\tighttracebreak\allowbreak\token[0.0000]|#|\allowbreak\token[0.0000]| Calculate|\allowbreak\token[0.0000]| the|\allowbreak\token[0.0000]| total|\allowbreak\token[0.0000]| number|\allowbreak\token[0.0000]| of|\allowbreak\token[0.0041]| different|\allowbreak\token[0.0009]| personalities|\allowbreak\token[0.0000]|↩|\newline\allowbreak\token[0.0000]|total|\allowbreak\token[0.2377]|_personal|\allowbreak\token[0.0000]|ities|\allowbreak\token[0.0000]| =|\allowbreak\token[0.4844]| (|\allowbreak\token[0.0063]|ey|\allowbreak\token[0.0000]|eb|\allowbreak\token[0.0005]|row|\allowbreak\token[0.0001]|_sets|\allowbreak\token[0.3147]| *|\allowbreak\token[0.0731]| ear|\allowbreak\token[0.0002]|_sets|\allowbreak\token[0.0047]| *|\allowbreak\token[0.0012]| lip|\allowbreak\token[0.0002]|_sets|\allowbreak\token[0.0075]| *|\allowbreak\token[0.0170]| shoe|\allowbreak\token[0.0000]|_pairs|\allowbreak\token[0.0103]| *|\allowbreak\token[0.0003]| hair|\allowbreak\token[0.0008]|_choices|\allowbreak\token[0.0032]|)↩|\newline\allowbreak\token[0.0000]|print|\allowbreak\token[0.0001]|(total|\allowbreak\token[0.0053]|_personal|\allowbreak[...]
\end{twopartbox}
\begin{twopartbox}{Cut Index 123 | Acceptance Probability 1.0000}{Block 2 Step 2}{Proposal (Accepted)}
\token[2.7697]| To|\allowbreak\token[1.0000]| solve|\allowbreak\token[0.4562]| this|\allowbreak\token[0.1579]| problem|\allowbreak\token[0.0266]|,|\allowbreak\token[0.2968]| we|\allowbreak\token[0.8181]| need|\allowbreak\token[0.0021]| to|\allowbreak\token[1.5837]| calculate|\allowbreak\token[0.0816]| the|\allowbreak\token[0.4486]| total|\allowbreak\token[0.0082]| number|\allowbreak\token[0.0005]| of|\allowbreak\token[1.0721]| different|\allowbreak\token[1.2499]| personalities|\allowbreak\token[0.8089]| Tom|\allowbreak\token[0.0128]| can|\allowbreak\token[0.1936]| create|\allowbreak\token[1.8537]| by|\allowbreak\token[0.9061]| combining|\allowbreak\token[1.3426]| the|\allowbreak\token[1.4240]| available|\allowbreak\token[1.7869]| options|\allowbreak\token[0.9441]| for|\allowbreak\token[0.7715]| eyebrows|\allowbreak\token[0.0072]|,|\allowbreak\token[0.6925]| ears|\allowbreak\token[0.0018]|,|\allowbreak\token[0.1585]| lips|\allowbreak\token[0.0020]|,|\allowbreak\token[1.3031]| shoes|\allowbreak\token[0.0554]|,|\allowbreak\token[0.1710]| and|\allowbreak\token[0.0844]| optionally|\allowbreak\token[2.0997]| hair|\allowbreak\token[1.0413]|.|\allowbreak\token[1.6814]| Here|\allowbreak\token[0.4489]|'s|\allowbreak\token[0.4810]| the|\allowbreak\token[0.6868]| breakdown|\allowbreak\token[0.3409]| of|\allowbreak\token[0.0250]| the|\allowbreak\token[0.8415]| options|\allowbreak\token[0.1294]|:↩↩|\tighttracebreak\allowbreak\token[0.6935]|1|\allowbreak\token[0.0004]|.|\allowbreak\token[0.8588]| **|\allowbreak\token[0.7153]|Ey|\allowbreak\token[0.0004]|eb|\allowbreak\token[0.0114]|rows|\allowbreak\token[0.1938]|**:|\allowbreak\token[0.4272]| |\allowbreak\token[0.1218]|2|\allowbreak\token[0.5041]| sets|\allowbreak\token[1.1879]| (|\allowbreak\token[0.3962]|left|\allowbreak\token[0.0439]| and|\allowbreak\token[0.0002]| right|\allowbreak\token[0.8898]|)↩|\newline\allowbreak\token[0.0053]|2|\allowbreak\token[0.0000]|.|\allowbreak\token[0.0000]| **|\allowbreak\token[0.0042]|E|\allowbreak\token[0.0095]|ars|\allowbreak\token[0.0000]|**:|\allowbreak\token[0.0000]| |\allowbreak\token[0.0002]|2|\allowbreak\token[0.0016]| sets|\allowbreak\token[0.1916]| (|\allowbreak\token[0.0569]|left|\allowbreak\token[0.0004]| and|\allowbreak\token[0.0002]| right|\allowbreak\token[0.0085]|)↩|\newline\allowbreak\token[0.0002]|3|\allowbreak\token[0.0000]|.|\allowbreak\token[0.0000]| **|\allowbreak\token[0.0646]|L|\allowbreak\token[0.0004]|ips|\allowbreak\token[0.0002]|**:|\allowbreak\token[0.0000]| |\allowbreak\token[0.0003]|2|\allowbreak\token[0.0001]| sets|\allowbreak\token[0.5877]| (|\allowbreak\token[1.3208]|left|\allowbreak\token[0.0010]| and|\allowbreak\token[0.0003]| right|\allowbreak\token[0.0304]|)↩|\newline\allowbreak\token[0.0004]|4|\allowbreak\token[0.0000]|.|\allowbreak\token[0.0000]| **|\allowbreak\token[0.5648]|Sh|\allowbreak\token[0.0001]|oes|\allowbreak\token[0.0001]|**:|\allowbreak\token[0.0012]| |\allowbreak\token[0.1043]|2|\allowbreak\token[0.6379]| pairs|\allowbreak\token[0.0007]| (|\allowbreak\token[0.0104]|regular|\allowbreak\token[0.2005]| and|\allowbreak\token[0.0001]| clown|\allowbreak\token[0.6290]|)↩|\newline\allowbreak\token[0.0001]|5|\allowbreak\token[0.0000]|.|\allowbreak\token[0.0000]| **|\allowbreak\token[1.4236]|Hair|\allowbreak\token[0.0736]|**:|\allowbreak\token[0.0056]| |\allowbreak\token[0.3179]|3|\allowbreak\token[0.8361]| hairstyles|\allowbreak\token[1.1954]| (|\allowbreak\token[0.6671]|including|\allowbreak\token[0.1912]| bald|\allowbreak\token[1.2099]|)↩↩|\tighttracebreak\allowbreak\token[0.7840]|Since|\allowbreak\token[1.3340]| Tom|\allowbreak\token[0.1609]| can|\allowbreak\token[0.6689]| choose|\allowbreak\token[1.0828]| to|\allowbreak\token[1.8048]| not|\allowbreak\token[1.3054]| include|\allowbreak\token[0.7798]| hair|\allowbreak\token[1.4518]|,|\allowbreak\token[0.4165]| we|\allowbreak\cut{[CUT]}\allowbreak\token[1.0062]| need|\allowbreak\token[0.0000]| to|\allowbreak\token[0.1151]| consider|\allowbreak\token[1.3058]| the|\allowbreak\token[0.6299]| bald|\allowbreak\token[0.0517]| option|\allowbreak\token[0.3821]| as|\allowbreak\token[0.6127]| well|\allowbreak\token[0.5054]|.|\allowbreak\token[0.5940]| This|\allowbreak\token[0.0508]| means|\allowbreak\token[0.5136]| there|\allowbreak\token[0.0010]| are|\allowbreak\token[0.5754]| |\allowbreak\token[0.5327]|4|\allowbreak\token[1.5106]| possible|\allowbreak\token[0.7006]| choices|\allowbreak\token[0.0007]| for|\allowbreak\token[0.4840]| hair|\allowbreak\token[0.4028]| (|\allowbreak\token[0.2223]|3|\allowbreak\token[0.0117]| hairstyles|\allowbreak\token[0.2084]| +|\allowbreak\token[0.6068]| bald|\allowbreak\token[0.0230]|).↩↩|\tighttracebreak\allowbreak\token[1.1681]|The|\allowbreak\token[0.0299]| total|\allowbreak\token[0.0000]| number|\allowbreak\token[0.0000]| of|\allowbreak\token[0.0126]| different|\allowbreak\token[0.0179]| personalities|\allowbreak\token[0.2055]| can|\allowbreak\token[0.0000]| be|\allowbreak\token[0.0001]| calculated|\allowbreak\token[0.0127]| by|\allowbreak\token[0.0028]| multiplying|\allowbreak\token[0.0002]| the|\allowbreak\token[0.0000]| number|\allowbreak\token[0.0000]| of|\allowbreak\token[0.1324]| choices|\allowbreak\token[0.0002]| for|\allowbreak\token[0.0004]| each|\allowbreak\token[1.4233]| feature|\allowbreak\token[0.6486]|:↩↩|\tighttracebreak\allowbreak\token[0.3738]|\|\allowbreak\token[0.0000]|[↩|\newline\allowbreak\token[0.0033]|\|\allowbreak\token[0.0000]|text|\allowbreak\token[0.0000]|{|\allowbreak\token[0.0000]|Total|\allowbreak\token[0.5331]| Personal|\allowbreak\token[0.0000]|ities|\allowbreak\token[0.0000]|}|\allowbreak\token[0.0000]| =|\allowbreak\token[0.4698]| (\|\allowbreak\token[0.0000]|text|\allowbreak\token[0.0020]|{|\allowbreak\token[0.9551]|Number|\allowbreak\token[0.0000]| of|\allowbreak\token[0.6774]| Ey|\allowbreak\token[0.0082]|eb|\allowbreak\token[0.1627]|row|\allowbreak\token[0.7888]| Sets|\allowbreak\token[0.0335]|})|\allowbreak\token[0.0000]| \|\allowbreak\token[0.0000]|times|\allowbreak\token[0.0000]| (\|\allowbreak\token[0.0001]|text|\allowbreak\token[0.0001]|{|\allowbreak\token[0.0000]|Number|\allowbreak\token[0.0000]| of|\allowbreak\token[0.1046]| Ear|\allowbreak\token[0.0008]| Sets|\allowbreak\token[0.0000]|})|\allowbreak\token[0.0000]| \|\allowbreak\token[0.0000]|times|\allowbreak\token[0.0001]| (\|\allowbreak\token[0.0000]|text|\allowbreak\token[0.0001]|{|\allowbreak\token[0.0002]|Number|\allowbreak\token[0.0001]| of|\allowbreak\token[0.6739]| Lips|\allowbreak\token[0.0340]| Sets|\allowbreak\token[0.0000]|})|\allowbreak\token[0.0000]| \|\allowbreak\token[0.0000]|times|\allowbreak\token[0.0000]| (\|\allowbreak\token[0.0000]|text|\allowbreak\token[0.0000]|{|\allowbreak\token[0.0002]|Number|\allowbreak\token[0.0000]| of|\allowbreak\token[0.2786]| Shoe|\allowbreak\token[0.0127]| P|\allowbreak\token[0.0002]|airs|\allowbreak\token[0.0001]|})|\allowbreak\token[0.0000]| \|\allowbreak\token[0.0000]|times|\allowbreak\token[0.0000]| (\|\allowbreak\token[0.0000]|text|\allowbreak\token[0.0000]|{|\allowbreak\token[0.0005]|Number|\allowbreak\token[0.0000]| of|\allowbreak\token[0.0323]| Hair|\allowbreak\token[1.0010]| Options|\allowbreak\token[0.0211]|})↩|\newline\allowbreak\token[0.0000]|\|\allowbreak\token[0.0000]|]↩↩|\tighttracebreak\allowbreak\token[0.0397]|Let|\allowbreak\token[0.0001]|'s|\allowbreak\token[0.0050]| calculate|\allowbreak\token[0.0001]| this|\allowbreak\token[0.0817]| using|\allowbreak\token[0.0000]| Python|\allowbreak\token[1.0118]| to|\allowbreak\token[0.0000]| ensure|\allowbreak\token[0.0000]| accuracy|\allowbreak\token[0.0002]|.↩|\newline\allowbreak\token[0.0000]|```|\allowbreak\token[0.0000]|python|\allowbreak\token[0.0000]|↩|\newline\allowbreak\token[0.0000]|#|\allowbreak\token[0.3500]| Define|\allowbreak\token[0.0000]| the|\allowbreak\token[0.0000]| number|\allowbreak\token[0.0000]| of|\allowbreak\token[0.4203]| options|\allowbreak\token[0.0000]| for|\allowbreak\token[0.0000]| each|\allowbreak\token[0.0241]| feature|\allowbreak\token[0.0156]|↩|\newline\allowbreak\token[0.2517]|ey|\allowbreak\token[0.0004]|eb|\allowbreak\token[0.2755]|row|\allowbreak\token[0.0166]|_sets|\allowbreak\token[0.0000]| =|\allowbreak\token[0.0000]| |\allowbreak\token[0.0106]|2|\allowbreak\token[0.6531]|↩|\newline\allowbreak\token[0.0011]|ear|\allowbreak\token[0.0034]|_sets|\allowbreak\token[0.0000]| =|\allowbreak\token[0.0000]| |\allowbreak\token[0.0001]|2|\allowbreak\token[0.0000]|↩|\newline\allowbreak[...]
\tcblower
\token[1.0271]| need|\allowbreak\token[0.0000]| to|\allowbreak\token[0.1111]| consider|\allowbreak\token[1.3093]| the|\allowbreak\token[0.6345]| bald|\allowbreak\token[0.0553]| option|\allowbreak\token[0.3836]| as|\allowbreak\token[0.5915]| well|\allowbreak\token[0.5090]|.|\allowbreak\token[0.6681]| This|\allowbreak\token[0.0415]| means|\allowbreak\token[0.5855]| there|\allowbreak\token[0.0010]| are|\allowbreak\token[0.5753]| |\allowbreak\token[0.5328]|4|\allowbreak\token[1.5114]| choices|\allowbreak\token[0.0056]| for|\allowbreak\token[0.5085]| hair|\allowbreak\token[0.3892]| (|\allowbreak\token[0.1812]|3|\allowbreak\token[0.0142]| hairstyles|\allowbreak\token[0.2475]| +|\allowbreak\token[0.5177]| bald|\allowbreak\token[0.0189]|).↩↩|\tighttracebreak\allowbreak\token[1.1934]|To|\allowbreak\token[0.0420]| find|\allowbreak\token[0.0000]| the|\allowbreak\token[0.0000]| total|\allowbreak\token[0.0000]| number|\allowbreak\token[0.0000]| of|\allowbreak\token[0.0124]| different|\allowbreak\token[0.0179]| personalities|\allowbreak\token[0.0002]|,|\allowbreak\token[0.0000]| we|\allowbreak\token[0.2151]| multiply|\allowbreak\token[0.0004]| the|\allowbreak\token[0.0000]| number|\allowbreak\token[0.0000]| of|\allowbreak\token[0.0740]| choices|\allowbreak\token[0.0006]| for|\allowbreak\token[0.0002]| each|\allowbreak\token[1.4712]| option|\allowbreak\token[0.7182]|:↩↩|\tighttracebreak\allowbreak\token[0.3878]|\|\allowbreak\token[0.0000]|[↩|\newline\allowbreak\token[0.0100]|\|\allowbreak\token[0.0000]|text|\allowbreak\token[0.0000]|{|\allowbreak\token[0.0001]|Total|\allowbreak\token[0.6908]| Personal|\allowbreak\token[0.0000]|ities|\allowbreak\token[0.0000]|}|\allowbreak\token[0.0000]| =|\allowbreak\token[0.5823]| (\|\allowbreak\token[0.0001]|text|\allowbreak\token[0.0031]|{|\allowbreak\token[1.0511]|Ey|\allowbreak\token[0.0004]|eb|\allowbreak\token[0.3435]|rows|\allowbreak\token[0.3841]|})|\allowbreak\token[0.0000]| \|\allowbreak\token[0.0000]|times|\allowbreak\token[0.0000]| (\|\allowbreak\token[0.0000]|text|\allowbreak\token[0.0001]|{|\allowbreak\token[0.0021]|E|\allowbreak\token[0.0003]|ars|\allowbreak\token[0.0000]|})|\allowbreak\token[0.0000]| \|\allowbreak\token[0.0000]|times|\allowbreak\token[0.0001]| (\|\allowbreak\token[0.0001]|text|\allowbreak\token[0.0001]|{|\allowbreak\token[0.0232]|L|\allowbreak\token[0.0119]|ips|\allowbreak\token[0.0000]|})|\allowbreak\token[0.0000]| \|\allowbreak\token[0.0000]|times|\allowbreak\token[0.0000]| (\|\allowbreak\token[0.0000]|text|\allowbreak\token[0.0000]|{|\allowbreak\token[0.0010]|Sh|\allowbreak\token[0.0001]|oes|\allowbreak\token[0.0001]|})|\allowbreak\token[0.0000]| \|\allowbreak\token[0.0000]|times|\allowbreak\token[0.0001]| (\|\allowbreak\token[0.0000]|text|\allowbreak\token[0.0000]|{|\allowbreak\token[0.0131]|Hair|\allowbreak\token[0.0427]|})↩|\newline\allowbreak\token[0.0000]|\|\allowbreak\token[0.0002]|]↩↩|\tighttracebreak\allowbreak\token[0.0571]|Let|\allowbreak\token[0.0003]|'s|\allowbreak\token[0.0031]| calculate|\allowbreak\token[0.0003]| this|\allowbreak\token[0.1326]| using|\allowbreak\token[0.0001]| Python|\allowbreak\token[0.8456]|.↩|\newline\allowbreak\token[0.0000]|```|\allowbreak\token[0.0000]|python|\allowbreak\token[0.0000]|↩|\newline\allowbreak\token[0.0001]|#|\allowbreak\token[0.7169]| Define|\allowbreak\token[0.0000]| the|\allowbreak\token[0.0000]| number|\allowbreak\token[0.0000]| of|\allowbreak\token[0.6632]| options|\allowbreak\token[0.0003]| for|\allowbreak\token[0.0001]| each|\allowbreak\token[1.3627]| feature|\allowbreak\token[0.0403]|↩|\newline\allowbreak\token[0.0106]|ey|\allowbreak\token[0.0003]|eb|\allowbreak\token[0.1600]|rows|\allowbreak\token[0.3449]| =|\allowbreak\token[0.0001]| |\allowbreak\token[0.0069]|2|\allowbreak\token[0.4587]| |\allowbreak\token[0.0000]| #|\allowbreak\token[0.1518]| left|\allowbreak\token[0.0119]| and|\allowbreak\token[0.0000]| right|\allowbreak\token[0.0296]|↩|\newline\allowbreak\token[0.0003]|ears|\allowbreak\token[0.0000]| =|\allowbreak\token[0.0000]| |\allowbreak\token[0.0000]|2|\allowbreak\token[0.7604]|     |\allowbreak\token[0.0000]| #|\allowbreak\token[0.0002]| left|\allowbreak\token[0.0000]| and|\allowbreak\token[0.0000]| right|\allowbreak\token[0.0001]|↩|\newline\allowbreak\token[0.0147]|l|\allowbreak\token[0.0010]|ips|\allowbreak\token[0.0000]| =|\allowbreak\token[0.0000]| |\allowbreak\token[0.0001]|2|\allowbreak\token[0.0080]|     |\allowbreak\token[0.0000]| #|\allowbreak\token[0.0004]| left|\allowbreak\token[0.0001]| and|\allowbreak[...]
\end{twopartbox}
\begin{twopartbox}{Cut Index 58 | Acceptance Probability 0.0001}{Block 2 Step 3}{Proposal (Rejected)}
\token[2.7697]| To|\allowbreak\token[1.0000]| solve|\allowbreak\token[0.4562]| this|\allowbreak\token[0.1579]| problem|\allowbreak\token[0.0266]|,|\allowbreak\token[0.2968]| we|\allowbreak\token[0.8181]| need|\allowbreak\token[0.0021]| to|\allowbreak\token[1.5837]| calculate|\allowbreak\token[0.0816]| the|\allowbreak\token[0.4486]| total|\allowbreak\token[0.0082]| number|\allowbreak\token[0.0005]| of|\allowbreak\token[1.0721]| different|\allowbreak\token[1.2499]| personalities|\allowbreak\token[0.8089]| Tom|\allowbreak\token[0.0128]| can|\allowbreak\token[0.1936]| create|\allowbreak\token[1.8537]| by|\allowbreak\token[0.9061]| combining|\allowbreak\token[1.3426]| the|\allowbreak\token[1.4240]| available|\allowbreak\token[1.7869]| options|\allowbreak\token[0.9441]| for|\allowbreak\token[0.7715]| eyebrows|\allowbreak\token[0.0072]|,|\allowbreak\token[0.6925]| ears|\allowbreak\token[0.0018]|,|\allowbreak\token[0.1585]| lips|\allowbreak\token[0.0020]|,|\allowbreak\token[1.3031]| shoes|\allowbreak\token[0.0554]|,|\allowbreak\token[0.1710]| and|\allowbreak\token[0.0844]| optionally|\allowbreak\token[2.0997]| hair|\allowbreak\token[1.0413]|.|\allowbreak\token[1.6814]| Here|\allowbreak\token[0.4489]|'s|\allowbreak\token[0.4810]| the|\allowbreak\token[0.6868]| breakdown|\allowbreak\token[0.3409]| of|\allowbreak\token[0.0250]| the|\allowbreak\token[0.8415]| options|\allowbreak\token[0.1294]|:↩↩|\tighttracebreak\allowbreak\token[0.6935]|1|\allowbreak\token[0.0004]|.|\allowbreak\token[0.8588]| **|\allowbreak\token[0.7153]|Ey|\allowbreak\token[0.0004]|eb|\allowbreak\token[0.0114]|rows|\allowbreak\token[0.1938]|**:|\allowbreak\token[0.4272]| |\allowbreak\token[0.1218]|2|\allowbreak\token[0.5041]| sets|\allowbreak\token[1.1879]| (|\allowbreak\token[0.3962]|left|\allowbreak\token[0.0439]| and|\allowbreak\token[0.0002]| right|\allowbreak\cut{[CUT]}\allowbreak\token[0.8898]|)↩|\newline\allowbreak\token[0.0053]|2|\allowbreak\token[0.0000]|.|\allowbreak\token[0.0000]| **|\allowbreak\token[0.0042]|E|\allowbreak\token[0.0095]|ars|\allowbreak\token[0.0000]|**:|\allowbreak\token[0.0000]| |\allowbreak\token[0.0002]|2|\allowbreak\token[0.0016]| sets|\allowbreak\token[0.1916]| (|\allowbreak\token[0.0569]|left|\allowbreak\token[0.0004]| and|\allowbreak\token[0.0002]| right|\allowbreak\token[0.0085]|)↩|\newline\allowbreak\token[0.0002]|3|\allowbreak\token[0.0000]|.|\allowbreak\token[0.0000]| **|\allowbreak\token[0.0646]|L|\allowbreak\token[0.0004]|ips|\allowbreak\token[0.0002]|**:|\allowbreak\token[0.0000]| |\allowbreak\token[0.0003]|2|\allowbreak\token[0.0001]| sets|\allowbreak\token[0.5877]| (|\allowbreak\token[1.3208]|left|\allowbreak\token[0.0010]| and|\allowbreak\token[0.0003]| right|\allowbreak\token[0.0304]|)↩|\newline\allowbreak\token[0.0004]|4|\allowbreak\token[0.0000]|.|\allowbreak\token[0.0000]| **|\allowbreak\token[0.5648]|Sh|\allowbreak\token[0.0001]|oes|\allowbreak\token[0.0001]|**:|\allowbreak\token[0.0012]| |\allowbreak\token[0.1043]|2|\allowbreak\token[0.6379]| pairs|\allowbreak\token[0.0007]| (|\allowbreak\token[0.0104]|regular|\allowbreak\token[0.2005]| and|\allowbreak\token[0.0001]| clown|\allowbreak\token[0.6290]|)↩|\newline\allowbreak\token[0.0001]|5|\allowbreak\token[0.0000]|.|\allowbreak\token[0.0000]| **|\allowbreak\token[1.4236]|Hair|\allowbreak\token[0.0736]|**:|\allowbreak\token[0.0056]| |\allowbreak\token[0.3179]|3|\allowbreak\token[0.8361]| hairstyles|\allowbreak\token[1.1954]| (|\allowbreak\token[0.6671]|including|\allowbreak\token[0.1912]| bald|\allowbreak\token[1.2099]|)↩↩|\tighttracebreak\allowbreak\token[0.7840]|Since|\allowbreak\token[1.3340]| Tom|\allowbreak\token[0.1609]| can|\allowbreak\token[0.6689]| choose|\allowbreak\token[1.0828]| to|\allowbreak\token[1.8048]| not|\allowbreak\token[1.3054]| include|\allowbreak\token[0.7798]| hair|\allowbreak\token[1.4518]|,|\allowbreak\token[0.4165]| we|\allowbreak\token[1.0271]| need|\allowbreak\token[0.0000]| to|\allowbreak\token[0.1111]| consider|\allowbreak\token[1.3093]| the|\allowbreak\token[0.6345]| bald|\allowbreak\token[0.0553]| option|\allowbreak\token[0.3836]| as|\allowbreak\token[0.5915]| well|\allowbreak\token[0.5090]|.|\allowbreak\token[0.6681]| This|\allowbreak\token[0.0415]| means|\allowbreak\token[0.5855]| there|\allowbreak\token[0.0010]| are|\allowbreak\token[0.5753]| |\allowbreak\token[0.5328]|4|\allowbreak\token[1.5114]| choices|\allowbreak\token[0.0056]| for|\allowbreak\token[0.5085]| hair|\allowbreak\token[0.3892]| (|\allowbreak\token[0.1812]|3|\allowbreak\token[0.0142]| hairstyles|\allowbreak\token[0.2475]| +|\allowbreak\token[0.5177]| bald|\allowbreak\token[0.0189]|).↩↩|\tighttracebreak\allowbreak\token[1.1934]|To|\allowbreak\token[0.0420]| find|\allowbreak\token[0.0000]| the|\allowbreak\token[0.0000]| total|\allowbreak\token[0.0000]| number|\allowbreak\token[0.0000]| of|\allowbreak\token[0.0124]| different|\allowbreak\token[0.0179]| personalities|\allowbreak\token[0.0002]|,|\allowbreak\token[0.0000]| we|\allowbreak\token[0.2151]| multiply|\allowbreak\token[0.0004]| the|\allowbreak\token[0.0000]| number|\allowbreak\token[0.0000]| of|\allowbreak\token[0.0740]| choices|\allowbreak\token[0.0006]| for|\allowbreak\token[0.0002]| each|\allowbreak\token[1.4712]| option|\allowbreak\token[0.7182]|:↩↩|\tighttracebreak\allowbreak\token[0.3878]|\|\allowbreak\token[0.0000]|[↩|\newline\allowbreak\token[0.0100]|\|\allowbreak\token[0.0000]|text|\allowbreak\token[0.0000]|{|\allowbreak\token[0.0001]|Total|\allowbreak\token[0.6908]| Personal|\allowbreak\token[0.0000]|ities|\allowbreak\token[0.0000]|}|\allowbreak\token[0.0000]| =|\allowbreak\token[0.5823]| (\|\allowbreak\token[0.0001]|text|\allowbreak\token[0.0031]|{|\allowbreak\token[1.0511]|Ey|\allowbreak\token[0.0004]|eb|\allowbreak\token[0.3435]|rows|\allowbreak\token[0.3841]|})|\allowbreak\token[0.0000]| \|\allowbreak\token[0.0000]|times|\allowbreak\token[0.0000]| (\|\allowbreak\token[0.0000]|text|\allowbreak\token[0.0001]|{|\allowbreak\token[0.0021]|E|\allowbreak\token[0.0003]|ars|\allowbreak\token[0.0000]|})|\allowbreak\token[0.0000]| \|\allowbreak\token[0.0000]|times|\allowbreak\token[0.0001]| (\|\allowbreak\token[0.0001]|text|\allowbreak\token[0.0001]|{|\allowbreak\token[0.0232]|L|\allowbreak\token[0.0119]|ips|\allowbreak\token[0.0000]|})|\allowbreak\token[0.0000]| \|\allowbreak\token[0.0000]|times|\allowbreak\token[0.0000]| (\|\allowbreak[...]
\tcblower
\token[0.9010]|)↩|\newline\allowbreak\token[0.0059]|2|\allowbreak\token[0.0000]|.|\allowbreak\token[0.0000]| **|\allowbreak\token[0.0041]|E|\allowbreak\token[0.0085]|ars|\allowbreak\token[0.0000]|**:|\allowbreak\token[0.0000]| |\allowbreak\token[0.0002]|2|\allowbreak\token[0.0014]| sets|\allowbreak\token[0.1916]| (|\allowbreak\token[0.0591]|left|\allowbreak\token[0.0004]| and|\allowbreak\token[0.0002]| right|\allowbreak\token[0.0076]|)↩|\newline\allowbreak\token[0.0002]|3|\allowbreak\token[0.0000]|.|\allowbreak\token[0.0000]| **|\allowbreak\token[0.0645]|L|\allowbreak\token[0.0004]|ips|\allowbreak\token[0.0002]|**:|\allowbreak\token[0.0000]| |\allowbreak\token[0.0003]|2|\allowbreak\token[0.0001]| sets|\allowbreak\token[0.5873]| (|\allowbreak\token[1.2556]|left|\allowbreak\token[0.0011]| and|\allowbreak\token[0.0004]| right|\allowbreak\token[0.0302]|)↩|\newline\allowbreak\token[0.0004]|4|\allowbreak\token[0.0000]|.|\allowbreak\token[0.0000]| **|\allowbreak\token[0.5315]|Sh|\allowbreak\token[0.0001]|oes|\allowbreak\token[0.0001]|**:|\allowbreak\token[0.0011]| |\allowbreak\token[0.1164]|2|\allowbreak\token[0.6208]| pairs|\allowbreak\token[0.0008]| (|\allowbreak\token[0.0121]|regular|\allowbreak\token[0.2005]| and|\allowbreak\token[0.0001]| clown|\allowbreak\token[0.6289]|)↩|\newline\allowbreak\token[0.0001]|5|\allowbreak\token[0.0000]|.|\allowbreak\token[0.0000]| **|\allowbreak\token[1.4179]|E|\allowbreak\token[0.0030]|yes|\allowbreak\token[0.0058]|**:|\allowbreak\token[0.1463]| |\allowbreak\token[0.0015]|1|\allowbreak\token[0.1304]| pair|\allowbreak\token[0.1094]| (|\allowbreak\token[1.8273]|goog|\allowbreak\token[0.0007]|ly|\allowbreak\token[0.0136]| eyes|\allowbreak\token[0.6377]|)↩|\newline\allowbreak\token[0.0001]|6|\allowbreak\token[0.0000]|.|\allowbreak\token[0.0001]| **|\allowbreak\token[0.2122]|Hair|\allowbreak\token[0.0312]|**:|\allowbreak\token[0.0041]| |\allowbreak\token[0.1019]|3|\allowbreak\token[0.4959]| hairstyles|\allowbreak\token[1.1984]| (|\allowbreak\token[0.6376]|including|\allowbreak\token[0.1989]| bald|\allowbreak\token[0.9106]|)↩↩|\tighttracebreak\allowbreak\token[0.4847]|Since|\allowbreak\token[0.9535]| the|\allowbreak\token[2.3552]| order|\allowbreak\token[0.1302]| of|\allowbreak\token[1.1764]| the|\allowbreak\token[2.3761]| eyebrows|\allowbreak\token[0.0545]|,|\allowbreak\token[0.0163]| ears|\allowbreak\token[0.0000]|,|\allowbreak\token[0.5735]| and|\allowbreak\token[0.0130]| lips|\allowbreak\token[1.0373]| does|\allowbreak\token[0.0194]| not|\allowbreak\token[0.0009]| matter|\allowbreak\token[0.7519]| (|\allowbreak\token[1.2976]|i|\allowbreak\token[0.0000]|.e|\allowbreak\token[0.0000]|.,|\allowbreak\token[1.3927]| left|\allowbreak\token[1.1486]| and|\allowbreak\token[0.0003]| right|\allowbreak\token[1.3287]| are|\allowbreak\token[1.8808]| fixed|\allowbreak\token[1.3819]|),|\allowbreak\token[0.0290]| we|\allowbreak\token[1.1327]| only|\allowbreak\token[0.0316]| need|\allowbreak\token[0.0000]| to|\allowbreak\token[0.0053]| consider|\allowbreak\token[0.0091]| the|\allowbreak\token[0.2169]| combinations|\allowbreak\token[0.1372]| of|\allowbreak\token[0.4942]| the|\allowbreak\token[0.6313]| other|\allowbreak\token[0.8153]| items|\allowbreak\token[0.3430]|.|\allowbreak\token[0.8228]| The|\allowbreak\token[0.1948]| total|\allowbreak\token[0.0000]| number|\allowbreak\token[0.0000]| of|\allowbreak\token[0.0896]| different|\allowbreak\token[0.0045]| personalities|\allowbreak\token[0.1344]| can|\allowbreak\token[0.0000]| be|\allowbreak\token[0.0000]| calculated|\allowbreak\token[0.3654]| by|\allowbreak\token[0.0000]| multiplying|\allowbreak\token[0.0000]| the|\allowbreak\token[0.0000]| number|\allowbreak\token[0.0000]| of|\allowbreak\token[0.6935]| options|\allowbreak\token[0.0017]| for|\allowbreak\token[0.0008]| each|\allowbreak\token[0.8860]| item|\allowbreak\token[0.5373]|:↩↩|\tighttracebreak\allowbreak\token[0.3506]|\|\allowbreak\token[0.0000]|[↩|\newline\allowbreak\token[0.0011]|\|\allowbreak\token[0.0000]|text|\allowbreak\token[0.0000]|{|\allowbreak\token[0.0000]|Total|\allowbreak\token[0.4891]| Personal|\allowbreak\token[0.0000]|ities|\allowbreak\token[0.0000]|}|\allowbreak\token[0.0000]| =|\allowbreak\token[0.8550]| (\|\allowbreak\token[0.0000]|text|\allowbreak\token[0.0004]|{|\allowbreak\token[0.8422]|Number|\allowbreak\token[0.0000]| of|\allowbreak\token[1.0485]| Ey|\allowbreak\token[0.0154]|eb|\allowbreak\token[0.2286]|row|\allowbreak\token[0.5899]| Sets|\allowbreak\token[0.0035]|})|\allowbreak[...]
\end{twopartbox}

\endgroup
\fi

\end{document}